\documentclass{article}
\usepackage[table,x11names]{xcolor}
\usepackage[final]{corl_2026} 

\usepackage{mymacros}

\usepackage{upgreek}
\usepackage{enumitem}
\usepackage{cleveref}
\usepackage{comment}
\usepackage{amsmath,amsthm}
\usepackage{amsmath,amssymb,amsthm}
\usepackage{bm}
\usepackage{cleveref}
\usepackage{color}
\usepackage{cleveref}
\usepackage{algorithm}
\usepackage{algpseudocode}
\usepackage{wrapfig}
\usepackage{booktabs}
\usepackage{graphicx}
\usepackage{subcaption}
\usepackage{tabularx}
\usepackage{enumitem}

\usepackage[most]{tcolorbox}

\definecolor{exampleborder}{RGB}{110,40,75}
\definecolor{examplebackground}{RGB}{255,248,252}

\newtcolorbox{examplebox}[2][]{%
    enhanced,
    float,
    floatplacement=!t, 
    title={#2},
    colback=examplebackground,
    colframe=exampleborder,
    fonttitle=\bfseries,
    before upper={%
        \setlength{\parskip}{0.6\baselineskip}%
        \setlength{\parindent}{0pt}%
    },
    #1
}

\definecolor{conclusionborder}{RGB}{40,140,150}
\definecolor{conclusionbackground}{RGB}{235,246,248}

\newtcolorbox{conclusionbox}[1]{
    colback=conclusionbackground,
    colframe=conclusionborder,
    title={#1},
    fonttitle=\bfseries,
    rounded corners,
}

\newcommand{\sectionend}{%
    \unskip\nobreak\hfill
    \penalty50\hskip1em\hbox{$\lozenge$}%
    \par
}

\newtheorem{theorem}{Theorem}

\newcommand{\loose}{\looseness=-1}

\makeatletter
\renewcommand{\@makefnmark}{%
  \hbox{\textcolor{red}{\@textsuperscript{\normalfont\@thefnmark}}}%
}
\makeatother

\title{Why Does Action Chunking Improve Behavioral Cloning Performance in Robotic Control?}

\author{
  Filippo Lazzati\\
  Politecnico di Milano\\
  \And
  Kyle Stachowicz \\
  UC Berkeley\\
  \And
  William Chen \\
  UC Berkeley\\
  \AND
  Alberto Maria Metelli \\
  Politecnico di Milano\\
  \And
  Andrew Wagenmaker \\
  UC Berkeley\\
  \And
  Sergey Levine\\
  UC Berkeley\\
}

\begin{document}
\maketitle

\begingroup
\renewcommand\thefootnote{}
\footnotetext{Correspondence to: Filippo Lazzati \{\texttt{filippo.lazzati@polimi.it}\} and Andrew Wagenmaker \{\texttt{ajwagen@berkeley.edu}\}.}
\endgroup
\setcounter{footnote}{0}


\begin{abstract}
Action chunking---predicting and executing multiple actions instead of a single action---has proven to be a critical component for learning effective robotic control policies. However, our precise understanding of \emph{why} action chunking improves performance has remained limited. In this work we seek to close this gap. Through rigorous experimental evaluations in both simulated and real-world settings, we show that existing hypotheses for the success of action chunking---temporal consistency, horizon reduction, and representation learning---fail to explain the success of action chunking. Instead, we find that action chunking benefits from greater non-Markovian expressivity and reduced compounding error compared to Markovian policies, but, in many settings of interest, these effects can be fully captured by \emph{delayed} policies, which at each step predict a single action based on the observation $k$ steps in the past. We then show that there exists an additional benefit of action chunking that we refer to as \emph{implicit ensembling}. In particular, by learning a diversity of temporal relationships (that is, $a_t | o_t$, $a_t | o_{t-1}$, $\ldots$), action-chunked policies exhibit behavior matching that of a model ensemble, increasing their robustness and generalization ability over policies that only learn a \emph{single} temporal relationship. Building on these insights, we show that in simulated and real-world robotic control settings, we can match the performance of action chunking without action chunking---by deploying an action chunking policy as an ensemble of policies with randomized delays. Furthermore, we propose a policy class that amplifies the benefits of action chunking by explicitly instantiating an ensemble, and which we show significantly improves over the performance of action chunking in many domains.\loose

\textbf{Website:} \url{https://action-chunking.github.io}.
\end{abstract}

\newcommand{\frakD}{\mathfrak{D}}
\newcommand{\traj}{\uptau}
\newcommand{\ba}{\bm{a}}

\newcommand{\pidemo}{\pi_{\mathrm{demo}}}
\newcommand{\pihat}{\widehat{\pi}}
\newcommand{\Exp}{\mathbb{E}}
\newcommand{\hist}{\bm{h}}
\newcommand{\wass}{W_1}
\newcommand{\sdemo}{s^{\pidemo}}
\newcommand{\histdemo}{\hist^{\pidemo}}
\newcommand{\shat}{s^{\pihat}}
\newcommand{\ademo}{a^{\pidemo}}
\newcommand{\ahat}{a^{\pihat}}
\newcommand{\ds}{d_{\mathcal{S}}}
\newcommand{\da}{d_{\mathcal{A}}}
\newcommand{\deltabar}{\bar{\delta}}
\newcommand{\Pimarkov}{\Pi_{\mathrm{markov}}}
\newcommand{\Pinonmarkov}{\Pi_{\mathrm{hist}}}

\section{Introduction}

Action chunking---predicting and executing a \emph{sequence} of actions rather than a \emph{single} action---is an essential ingredient in modern approaches to behavioral cloning for robotic control \cite{zhao2023actionchunking, chi2023diffusion}. Virtually all state-of-the-art policies for robotic control---from generalist policies capable of solving a wide range of tasks~\cite{black2024pi0, pertsch2025fast}, to single-task policies able to perform dexterous, high-precision maneuvers~\cite{zhao2024aloha}---use some form of action chunking to achieve high performance.

In behavioral cloning, action chunking simply requires training a policy to predict $k$ demonstrator actions from a single observation, modeling $(a_t, a_{t+1}, \dots, a_{t+k-1}) \mid o_t$. Here the prediction target $\ba_{t:t+k} = (a_t, a_{t+1}, \dots, a_{t+k-1})$ is referred to as the \textit{action chunk}. At inference time, action chunking policies execute all or part of this action chunk in an \textit{open-loop} fashion, effectively relying on an action prediction from a ``stale'' observation $o_t$ for multiple timesteps, rather than computing a fresh action at each timestep. Across a range of settings, executing action chunks in this manner achieves higher performance than recomputing and executing a single action at each timestep.

While empirically the necessity of action chunking is largely undisputed, our understanding of \emph{why} action chunking is necessary remains rudimentary. Common hypothesis in the literature include:
\begin{enumerate}[nosep,leftmargin=*]
    \item \textbf{Temporal consistency}: Action-chunked policies can better represent the temporally correlated behaviors exhibited by human demonstrators \cite{zhao2023actionchunking,chi2023diffusion, li2025reinforcement}.
    \item \textbf{Horizon reduction}: Action chunks reduce the effective horizon of the environment, mitigating compounding error \cite{zhao2023actionchunking}.
    \item \textbf{Representation learning}: Action chunking serves as an auxiliary loss in policy training, improving representation learning and generalization \cite{chi2023diffusion,torne2025learning}.
\end{enumerate}

\begin{figure}[t]
\centering
\includegraphics[width=\linewidth]{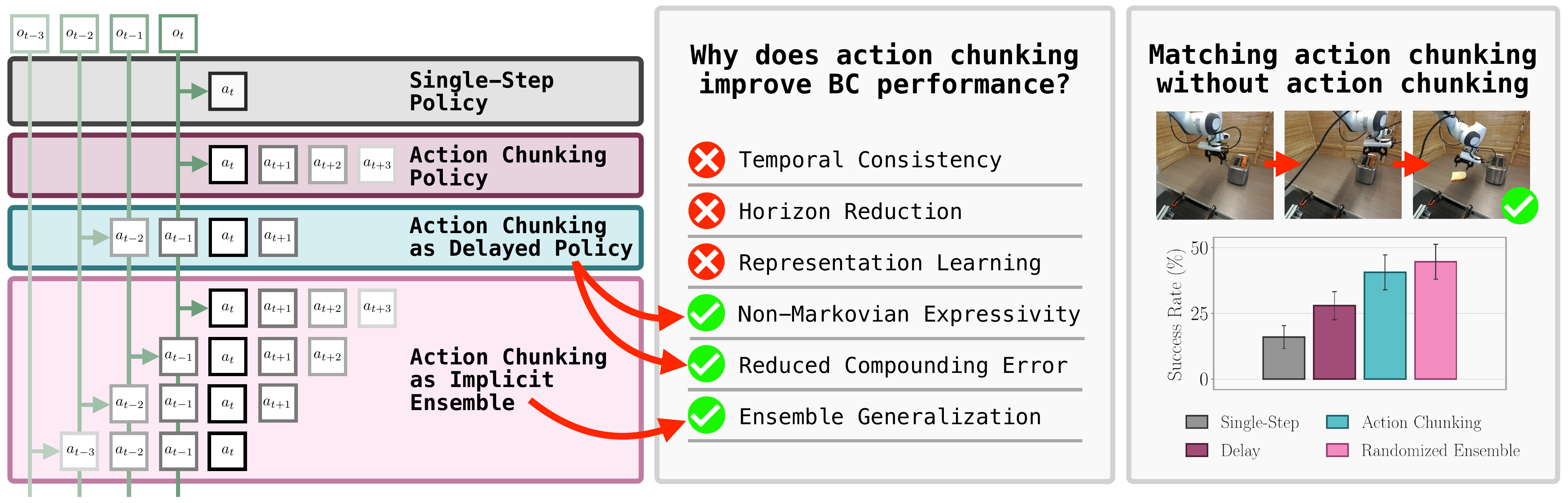}
\caption{In this work we investigate why action chunking improves the performance of behavioral cloning in robotic control. We show that common hypothesis---temporal consistency, horizon reduction, and representation learning---are not able to explain the performance of action chunking. Instead, we demonstrate that a combination of three factors---expressing non-Markovian delays in human behavior, reducing compounding error by predicting actions based on past observations, and ensemble-like effects induced by learning a variety of temporal relationships---almost fully explains the success of action chunking. Motivated by our insights, we show that in both simulated and real-world robotic control, we can match the performance of action chunking without explicitly utilizing action chunking.}
\label{fig:paper_fig}
\vspace{-1em}
\end{figure}

In this work we seek to develop a more complete understanding of the mechanisms enabling the success of action chunking in robotic control.
Our first key observation is that, in many settings, the performance of action chunking can in fact be replicated by deploying a \textit{delayed} policy $\pi(a_t \mid o_{t-k})$. Through theoretical analysis and controlled experiments on both simulated and real-world robotic manipulation settings, we show that, while humans do exhibit non-Markovian behaviors and action chunking does meaningfully reduce compounding error, predicting actions based on delayed observations is often sufficient to capture the necessary non-Markovian behaviors and achieve sufficient reductions in compounding error---the temporal consistency and horizon reduction provided by action chunking is not necessary for effective performance.

While non-Markovian expressivity and the reduction in compounding error provided by predicting actions based on past observations can partially explain the success of action chunking, we find settings where these effects are not fully explanatory. 
We identify an additional mechanism driving action chunking's improved performance, which we refer to as \textit{implicit ensembling}: an action-chunked policy $\pi(\ba_{t:t+k}|o_t)$ can be interpreted as an \textit{ensemble} of $k$ delayed policies $\{\pi(a_t | o_{t-i})\}_{0 \le i < k}$ and, by learning a diversity of temporal relationships, inherits many of the benefits of ensemble-based methods \cite{krogh1994neural,ho1998random,breiman2001random}. Across all simulated and real-world settings we consider, we show that we can match the performance of action chunking \emph{without action chunking} by deploying the action-chunked policy as an ensemble of policies with random delay.

 While the focus of this work is primarily analysis of existing approaches, our results also provide insight into how we can improve on the performance of action chunking.
In particular, motivated by the implicit ensembling effect of action chunking, we show that by training an \textit{explicit} ensemble of delayed policies, we can exceed the performance of action chunking in many settings.
Altogether, our results provide a much more complete picture for why action chunking enables improved performance in behavioral cloning for robotic control and show that, in many settings, the typical instantiation of action chunking is not necessary for effective policy performance.

\textbf{Organization.} The rest of this paper is organized as follows. In \Cref{sec:related} we discuss related work, and in \Cref{sec:prelim} outline our problem setting.  
Next, in \Cref{sec:analysis} we investigate existing hypothesis for why action chunking improves performance---temporal consistency (\Cref{sec:hypoth_expressivity}), horizon reduction (\Cref{sec:hypoth_horizon}), and representation learning (\Cref{sec:hypoth_inductive})---ultimately arguing that these are inadequate to explain the success of action chunking. In \Cref{sec:ac_ensembles}, we show that action chunking policies act as implicit ensembles, and that this effect is critical to their success.
\Cref{sec:results_real} then extends our simulated results to the real-world, demonstrating that our conclusions hold there as well. 
Motivated by our findings, in \Cref{sec:ensembles} we show that amplifying the implicit ensembling effects of action chunking with explicit ensembles can lead to even further performance improvements. Finally, in \Cref{sec:conclusion}, we close by highlighting several interesting directions for future work.

\section{Related Work}\label{sec:related}

\textbf{Behavioral cloning in robotics.}
Imitation learning via behavioral cloning (BC, \cite{pmerleau1988alvinn})---where a policy is trained via supervised learning to mimic the demonstrator's actions---is widely used in robotics~\cite{argall2009survey,ross2011reduction,bojarski2016end,zhang2018deep,rahmatizadeh2018vision,mandlekar2021matters}. Recent work scales BC by collecting large-scale human demonstration datasets \cite{walke2023bridgedata,o2024open,khazatsky2024droid} and training expressive generative models on these datasets \cite{shafiullah2022behavior,cui2022play,zhao2023actionchunking,chi2023diffusion,dasari2024ingredients,zhao2024aloha,ankile2024juicer,ze20243d,sridhar2024nomad}, resulting in generalist vision-language-action (VLA) policies capable of performing a wide variety of tasks 
~\cite{brohan2022rt,gu2023rt,team2024octo,kim2024openvla,black2024pi_0,bjorck2025gr00t,intelligence2025pi_,zha2026lap}. To the best of our knowledge, \textit{nearly every large-scale robot policy training effort since ACT \cite{zhao2023actionchunking} predicts and executes chunked actions}~\cite{team2024octo,black2024pi0,pertsch2025fast,molmoact2,geminirobotics,lbm,mimicvideo,alohaunleashed}.
In contrast to these works, rather than simply demonstrating another example of applying BC to robotic control, the focus of this paper is in understanding why action chunking helps, and showing how we can move beyond action chunking. We remark as well that the ensembling approach we propose in \Cref{sec:ensembles} is somewhat related to the work of \cite{wagenmaker2025posterior}, although the RL finetuning objective considered there is tangential to our objective.

\textbf{Action chunking.}
To the best of our knowledge, the first works to explicitly propose the use of action chunking in robotics in its current form are the concurrent works on ACT \citep{zhao2023actionchunking} and Diffusion Policy \citep{chi2023diffusion}. While many previous works also explored the idea of a policy \textit{producing} multiple actions, typically these works learned receding-horizon policies~\cite{janner2022planning}, where only the first action is executed.
Several works have sought to understand or improve on action chunking.
In particular, several approaches have been proposed to adjust the policy's sampling process in order to maintain temporal consistency \emph{across} action chunks \cite{liu2025bidirectional,malhotra2025self,park2025acg,black2026real, black2025training}. Other works have sought to obtain the benefits of action chunking while ensuring policy reactivity by mixing action chunks of different length \cite{jing2025mixture}, training a policy via reinforcement learning to adaptively select the chunk length \cite{weng2025temporal}, or using a world model to adaptively switch between action chunks at test-time \cite{chen2026dream}.

Perhaps most related to our work is the work of \citet{simchowitz2025pitfalls} and \citet{zhang2025imitation}, which seek to obtain a theoretical understanding for why action chunking helps, largely from a control-theoretic perspective. This line of work primarily aims to understand the \emph{minimal} conditions under which action chunking enables improved performance, showing that action chunking has benefits even if the demonstrator is deterministic and Markovian and the environment is fully observed. In particular, \citet{zhang2025imitation} shows that, under certain control-theoretic notions of stability, action chunking can mitigate compounding error in continuous control domains. While this is related to our investigation of the compounding error hypothesis (see \Cref{sec:hypoth_horizon}), we find that in practice a variety of other effects also contribute to the success of action chunking.

\newcommand{\cOp}{\cO^+}
\newcommand{\Po}{P_{\cO}}
\newcommand{\D}{D}
\newcommand{\Did}{D_{\mathrm{id}}}
\newcommand{\frakDtrain}{\frakD_{\mathrm{train}}}
\newcommand{\frakDval}{\frakD_{\mathrm{val}}}
\newcommand{\Lval}{L_{\mathrm{val}}}
\newcommand{\pidelay}{\pi_{\mathrm{delay}}}
\newcommand{\Ttrain}{T_{\mathrm{train}}}
\newcommand{\Ltrain}{L_{\mathrm{train}}}
\newcommand{\pibctwenty}{\pibc_{20}}
\newcommand{\dima}{d_{\mathcal{A}}}
\newcommand{\tti}{\texttt{i}}
\newcommand{\abar}{\bar{a}}

\section{Preliminaries}\label{sec:prelim}

In this work we study action chunking in the context of behavioral cloning. We consider interaction with a potentially non-Markovian, partially observed environment $\cM = (\cS,\cA,\cO,P,P_0,\Po,r)$, where $\cS$, $\cA := \mathbb{R}^{\dima}$, and $\cO$ are the state, action, and observation spaces, $P : \cH \times \cA \rightarrow \triangle_{\cS}$ is the transition kernel for histories $\cH$, $P_0 \in \triangle_{\cS}$ is the initial state distribution, $\Po : \cS \rightarrow \triangle_{\cO}$ is the observation distribution, and $r : \cO \rightarrow [0,1]$ the reward. An episode consists of a sequence of states, observations, and actions $(s_1, o_1, a_1, s_2, o_2, a_2, s_3, \dots)$ where $s_1 \sim P_0$, $s_{t+1} \sim P(\cdot \mid h_t, a_t)$, for $h_t$ the history of states and actions up to $t$, and $o_t \sim \Po(s_t)$. This continues for $H$ steps or until a successful state is reached, i.e., $r(o_t) = 1$.
We denote by $\cJ(\pi)$ the expected episode reward of policy $\pi$,
where the expectation is over trajectories induced by $\pi$ on $\cM$ (for a 0-1 success reward, this corresponds to the probability that some successful state will be reached before step $H$). While we consider non-Markovian, partially observed environments for the sake of generality, the majority of our conclusions also hold in Markovian (i.e. environments where $P(\cdot \mid a_t, s_t, s_{t-1}, s_{t-2}, \ldots) = P(\cdot \mid a_t, s_t)$) and fully observed settings.\loose

We assume access to a dataset of demonstration trajectories $\frakD = \{ \traj_{\tti} \}_{\tti=1}^N$, with each trajectory $\traj_{\tti} = (o_1^{\tti}, a_1^{\tti}, \ldots , o_{H }^{\tti}, a_{H}^{\tti}, o_{H+1}^{\tti})$ generated by some demonstrator $\pidemo$ on $\cM$. The demonstrator may be non-Markovian, meaning that it may condition on the full history of observations. In general, we assume that $\frakD$ consists of successful (although not necessarily optimal) behaviors, and are interested in learning policies that maximize success rate $\cJ(\pi)$. In the following, it will be convenient to split $\frakD$ into train and validation sets, $\frakDtrain$ and $\frakDval$. Throughout this work, we denote $[k] := \{ 1, \ldots , k \}$.\loose

\textbf{Behavioral cloning.} 
Behavioral cloning is a standard approach to learning from demonstrations that trains a policy via supervised learning to mimic the actions present in $\frakDtrain$. In particular, in the simplest, Markovian case, behavioral cloning flattens the training dataset into $(o_t^{\tti}, a_t^{\tti})$ pairs and trains a policy $\pibc(a_t \mid o_t)$ to predict the actions from the dataset:
\begin{align}\label{eq:bc_objective}
\textstyle   \pibc \leftarrow \argmax_\pi \sum_{(o_t, a_t) \in \frakDtrain} \log \pi(a_t \mid o_t).
\end{align}
At deployment, given observation $o_t$ an action $a_t \sim \pibc(\cdot \mid o_t)$ is sampled and executed in the environment.
In practice, modern approaches to BC in robotics typically parameterize $\pi$ with a generative model---usually either an autoregressive transformer~\cite{kim2024openvla,pertsch2025fast} or diffusion/flow model~\cite{chi2023diffusion,black2024pi0}---and aim to model the full demonstrator distribution. 
While our results are independent of the exact instantiation of the BC policy, for all experiments we parameterize $\pibc$ as a diffusion model \cite{chi2023diffusion}.

\textbf{Action chunking.} 
Action-chunked policies are trained to predict the next $k$ actions produced by the demonstrator instead of only the single next action. That is, action chunking models the conditional distribution of sequences of actions:
\begin{align}\label{eq:ac_objective}
 \textstyle   \pibc \leftarrow \argmin_\pi \sum_{(o_t,\ba_{t:t+k}) \in \frakDtrain} \log \pi(\ba_{t:t+k} \mid o_t) 
\end{align}
for $\ba_{t:t+k} = (a_t, a_{t+1}, \ldots , a_{t+k-1})$ 
the sequence of $k$ actions produced from the demonstrator following each $o_t$ in $\frakDtrain$ (the ``action chunk''). At deployment, action-chunked policies typically sample $\ba_{t:t+k} \sim \pihat(\cdot \mid o_t)$ and execute either part or all of $\ba_{t:t+k}$ open-loop before querying the policy for a new sequence of actions.

\newcommand{\ellval}{\ell_{\mathrm{val}}}
\newcommand{\Lvaldiff}{\Lval^{\mathrm{diff}}}
\newcommand{\Lvalact}{\Lval^{\mathrm{act}}}

\textbf{Policy notation.}
We let $\pihat_k$ denote a policy that produces action chunks of length $k$. In practice, it is common to only execute a portion of an action chunk, e.g. the first $n < k$ steps of the action chunk, and then recompute a fresh action chunk. We denote a policy that is deployed in this fashion as $\pihat_k^n$. 
We will also consider \emph{single-step delayed policies}, where we compute a new action at each step, but condition on some observation $d$ steps in the past. We denote such a policy as $\pidelay^d$ (so, for example, $a_t \sim \pidelay^d(\cdot \mid o_{t-d})$).
We can also execute an action-chunking policy $\pihat_k$ as a delayed policy for delay $d < k$ by sampling {\thinmuskip=2mu
\medmuskip=2mu\thickmuskip=3mu$\ba_{t-d:t-d+k} \sim \pihat_k(\cdot \mid o_{t-d})$} and taking the $(d+1)$th action in $\ba_{t-d:t-d+k}$. We denote this induced delayed policy as $\pidelay^d[\pihat_k]$. For brevity, we refer to single-step Markovian BC policies as ``Markov'' or ``single-step'', action chunking policies with chunk size $k$ as AC$(k)$, and delayed policies with delay $d$ as Delay$(d+1)$.

\textbf{Validation error.}
Throughout this work, we consider validation error---computed as mean-squared error (MSE) of \emph{mean} actions sampled from the policy on $\frakDval$---to measure how well a policy fits the data. We do not use this as a proxy for task success, only for analysis purposes.
Formally, for an action-chunk $k$ policy $\pibc_k$, we define the validation error:\loose
\begin{align*}
\textstyle    \Lval(\pibc_k) := \frac{1}{\dima \cdot k} \cdot \frac{1}{|\frakDval|} \sum_{(o_t,\ba_{t:t+k}) \in \frakDval}  \| \Exp_{\ba \sim \pibc_k(\cdot \mid o_t)}[\ba]- \ba_{t:t+k} \|_2^2,
\end{align*}
where, for action chunks, we take $\| \ba \|_2^2 := \sum_i \| [\ba]_i \|_2^2$ for $[\ba]_i$ the $i$th element in the chunk.
If $\pihat$ is a delayed policy (i.e. $\pidelay^d$ or $\pidelay^d[\pihat_k]$) or an action-chunked policy where we only execute the first $n$ actions in the chunk (i.e. $\pihat_k^n$), we define $\Lval(\pihat)$ analogously, but only compute the loss over the actions that the policy would execute ($a_{t+d} \mid o_t$ for a delayed policy, and $a_t,\ldots,a_{t+n-1} \mid o_t$ for $\pihat_k^n$), in either case normalizing $\Lval$ by the number of actions the loss is computed over rather than $1/k$.
We avoid the standard denoising validation loss due to ambiguity with ensembled policies (see \Cref{sec:val_error_explanation} for further discussion of this point).
While this notion of error does not capture multi-modality, in practice diffusion policies rarely exhibit significant multi-modality \cite{pan2025much} so this is not a major shortcoming (in particular, we found that the policies we trained in this work did not exhibit significant multi-modality, see e.g. \Cref{apx: prod marg}).

\section{Do Existing Hypotheses Explain the Performance of Action Chunking?}\label{sec:analysis}

In this section we seek to understand why BC policies trained to predict and execute action chunks outperform those trained to predict only single actions. We proceed by investigating each of the commonly cited hypotheses for the success of action chunking---\textbf{temporal consistency} (\Cref{sec:hypoth_expressivity}), \textbf{horizon reduction} (\Cref{sec:hypoth_horizon}), and \textbf{representation learning} (\Cref{sec:hypoth_inductive}).\loose

To investigate these hypothesis, in this section we consider the \texttt{Libero} behavioral cloning benchmark \cite{liu2023libero} (in the following sections we also consider the \texttt{Robomimic} benchmark \cite{robomimic2021}, as well as real-world robotic settings). \texttt{Libero} contains 130 total tasks, each provided with 50 successful human demonstrations---we primarily consider the \texttt{Libero-90} suite of 90 tasks. We train diffusion policies \cite{chi2023diffusion} on these demonstrations, using image-based observations. For each point in the following results, we average over at least three random seeds; error bars denote standard errors. Please see \Cref{sec:app_exp_details} for detailed experimental setup and \Cref{apx:addplots} for results on individual tasks.

\subsection{Hypothesis 1: Temporal Consistency}\label{sec:hypoth_expressivity}

The temporal consistency hypothesis starts from the observation that human demonstrators may not be truly Markovian, and, in particular, often produce smoothly-varying, correlated action sequences that are better modeled by policies explicitly encoding temporal correlations than Markovian policies \cite{zhao2023actionchunking,chi2023diffusion, li2025reinforcement,pertsch2025fast}. 
Here we investigate this claim experimentally.

\refstepcounter{experiment}
\textbf{Experiment \theexperiment: Human demonstrations are better modeled by delayed policies.}
To test this hypothesis, we first consider whether non-Markovian policies are able to better model the demonstration data in \texttt{Libero-90}
compared to single-step Markovian policies. 
In particular, we select $k=20$ and train an action chunk 20 policy $\pibctwenty$ on $\frakDtrain$. We consider the action prediction MSE on $\frakDval$ for policy $\pibctwenty^n$ and $n \in [20]$, $\Lval(\pibctwenty^n)$---the average action prediction error over the first $n$ steps in the action chunk. We also plot $\Lval(\pidelay^n[\pibctwenty])$---the action prediction error of the induced delayed policy. Conceptually, $\Lval(\pibctwenty^n)$ measures how effectively we can predict $(a_{t}, \ldots , a_{t+n-1}) \mid o_t$ while $\Lval(\pidelay^n[\pibctwenty])$ measures $a_{t+n} \mid o_t$ (equivalently, $a_t \mid o_{t-n}$). We compare both of these for different values of $n$ to $\Lval(\pihat_k^1)$, the validation MSE of a Markovian policy.\loose

Our results, aggregated across \texttt{Libero-90}, are given in \Cref{fig:val_loss_libero}. We confirm that we can more easily predict $a_t$ given $o_{t-n}$ than given $o_t$, for all $n < 18$. Furthermore, we find that we can more easily predict an action based on a delayed observation than we can predict the full action chunk.
Note that, if the demonstrator were purely Markovian, we would expect the mutual information between $a_t$ and $o_t$ to be at least as large as the mutual information between $a_t$ and $o_{t-n}$ (by the data processing inequality). However, we see the opposite: $a_t$ is best predicted by conditioning on $o_{t-10}$. This suggests that human demonstrators exhibit significant non-Markovian behavior\footnote{Note that, in the case of a purely Markovian, fully-observed environment, if the demonstrator is non-Markovian, we can perfectly replicate their state distribution with a Markovian policy \emph{so long as this policy is indexed on timestep $t$}. However, if the policy is not indexed on timestep, as is usually the case in practice, then a Markovian policy does not suffice to model a non-Markovian demonstrator. In other words, even in the simplest settings, we require non-Markovian policies to model the behavior of a non-Markovian demonstrator. See \Cref{apx:randomwalk} for further discussion of this point.}. 
\sectionend

\begin{figure*}[t]
    \centering

    \begin{minipage}[t]{0.32\textwidth}
        \centering
        \vspace{-3cm}
        \includegraphics[width=\linewidth]{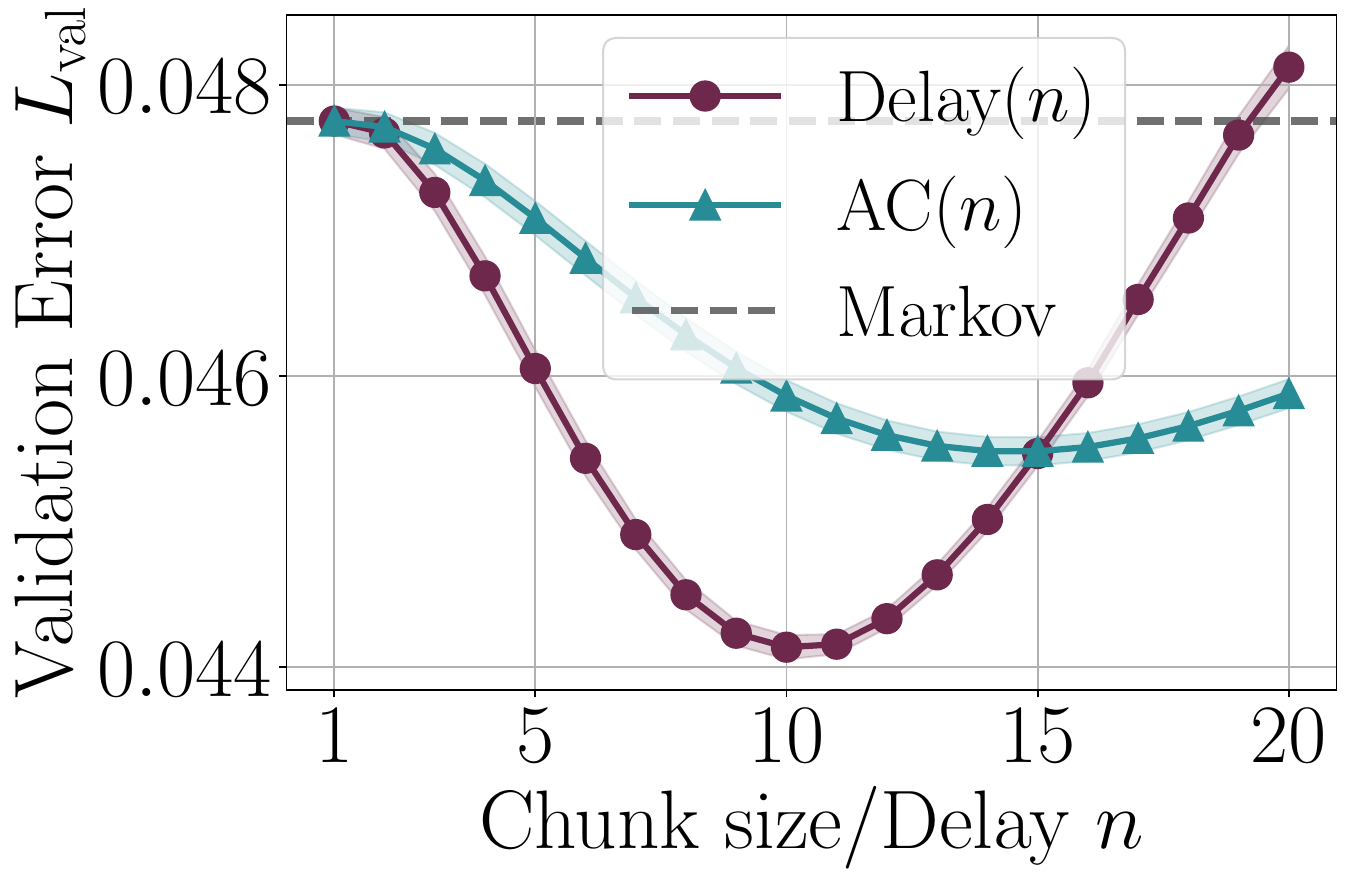}
        \caption{Action prediction validation error of delayed policies and action chunking policies, aggregated over \texttt{Libero-90}.}
        \label{fig:val_loss_libero}
    \end{minipage}
    \hfill
    \begin{minipage}[t]{0.31\textwidth}
        \centering
        \includegraphics[width=\linewidth]{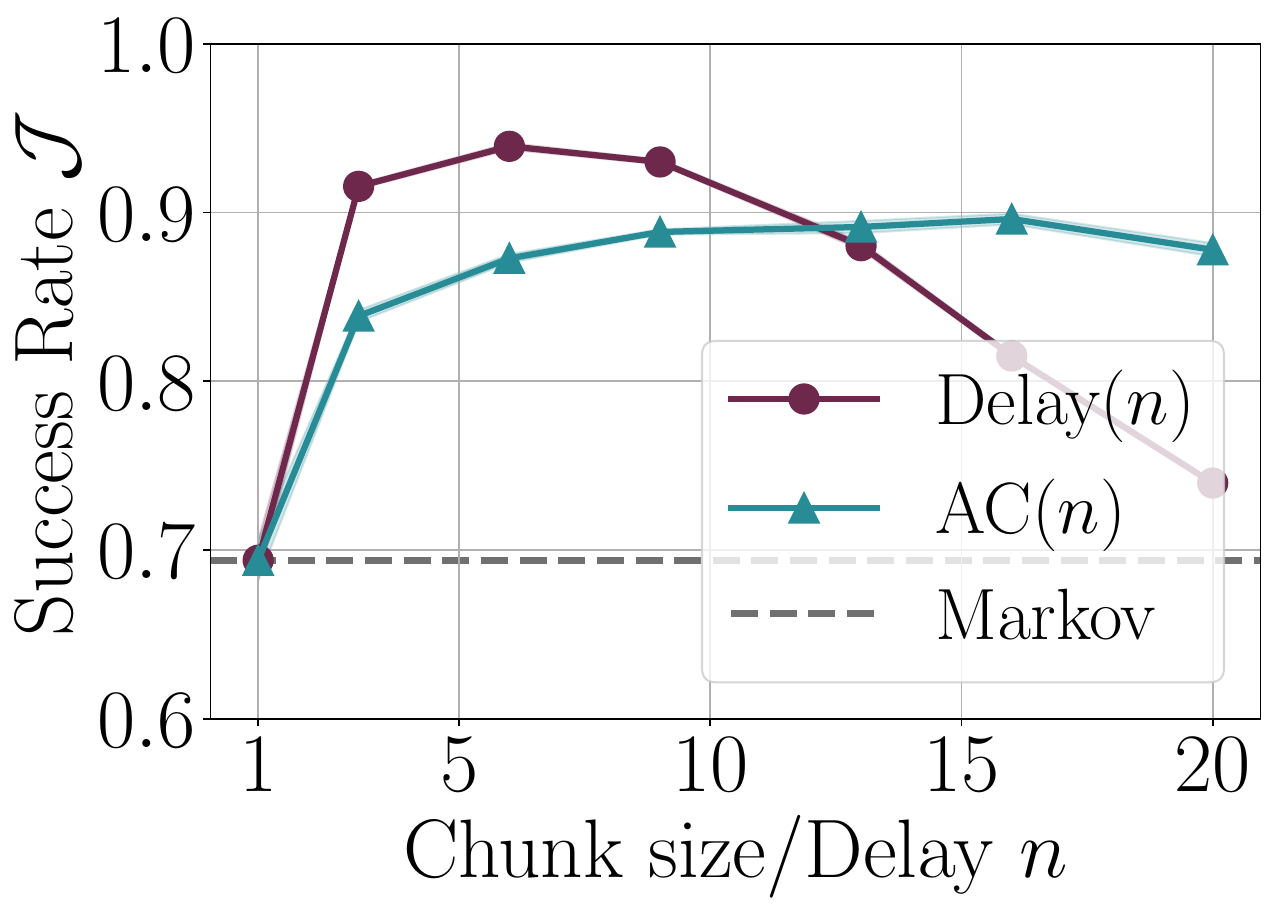}
        \caption{Average task success of delayed policies and action chunking policies, aggregated over \texttt{Libero-90}.}
        \label{fig:success_libero}
    \end{minipage}
    \hfill
    \begin{minipage}[t]{0.315\textwidth}
        \centering
        \includegraphics[width=\linewidth]{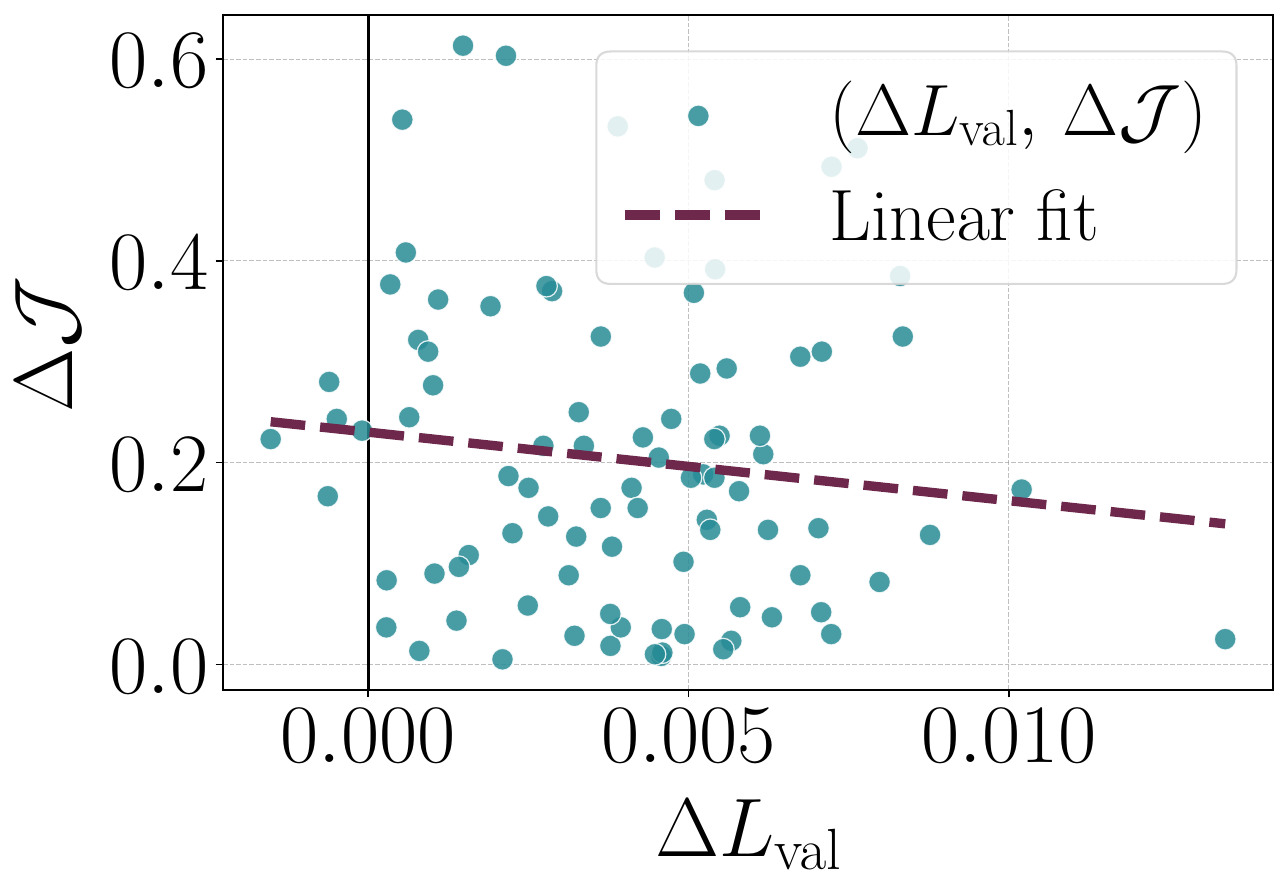}
        \caption{Success difference of AC$(10)$ and Markovian policy, vs non-Markovian-ness, for each task in \texttt{Libero-90}.}
        \label{fig:val_vs_success_libero}
    \end{minipage}
    
\end{figure*}

\refstepcounter{experiment}
\textbf{Experiment \theexperiment: Delayed policies match the performance of action chunking.}
Notably, while \Cref{fig:val_loss_libero} shows that $a_t \mid o_{t-n}$ achieves lower validation error than $a_t \mid o_t$, it also shows that the validation error averaged over the entire prediction $\ba_{t:t+n} \mid o_t$ is greater than $a_t \mid o_{t-n}$, for proper choice of $n$. To test how this impacts policy performance, we compare the success rate of $\pibctwenty^n$ to $\pidelay^n[\pibctwenty]$ for various choices of $n$. 
We illustrate the performance of these approaches in \Cref{fig:success_libero}, averaging the success rate achieved by each policy over all \texttt{Libero-90} tasks. We see that simply executing a delayed policy matches or exceeds the performance achieved by an action-chunking policy. 
This suggests that, on \texttt{Libero-90}, delayed policies capture the non-Markovian demonstrator behaviors relevant for policy performance. 
Thus, while action chunking may model human non-Markovian behavior more effectively than Markovian policies, leading to more effective performance, the full non-Markovian expressivity of action-chunked policies is not required---it suffices to simply predict an action based on a delayed observation. Concretely, this suggests that in standard robotic manipulation settings (in particular in this case \texttt{Libero}), temporal consistency is not necessary to perform effectively, and the benefits of action chunking are not due to improving temporal consistency. In \Cref{apx: openpi}, we show that this is true for VLAs as well---we find that the \texttt{Libero} finetune of $\pi_{0.5}$ \cite{intelligence2025pi_} performs just as well when executed as a delayed policy as a standard action-chunked policy.\loose
\sectionend

\begin{conclusionbox}{Takeaway 1: Action chunking improves performance by capturing non-Markovianity in human demonstrators, but temporal consistency is not, in general, required}
We conclude that humans do exhibit non-Markovian behavior, which can be more effectively modeled by an action-chunking policy than a Markovian policy. However, delayed policies model human demonstrators just as effectively, and perform equivalently to or better than action-chunked policies in terms of success rate, suggesting that the relevant non-Markovian behavior is captured by delayed policies, and temporal consistency is not required.
\end{conclusionbox}

\subsection{Hypothesis 2: Horizon Reduction}\label{sec:hypoth_horizon}

While humans exhibit non-Markovian behavior that action chunking is able to capture, our results suggest that delayed policies are able to capture the relevant non-Markovian behavior as well. Here we investigate if other effects also contribute to the performance of action chunking, and begin by investigating the explanatory power of non-Markovian expressivity.

\refstepcounter{experiment}
\textbf{Experiment \theexperiment: Non-Markovian expressivity only partially explains the success of action chunking.}
If non-Markovian expressivity fully explained the superior performance of action chunking, we would expect that in settings where action chunking does \textit{not} provide additional expressivity benefits (for example, when the demonstrator is Markovian), Markovian policies would perform as well as action chunking policies.
To investigate this, we evaluate the behavior of $\pibctwenty^{10}$ (executing action chunks of length 10) compared to $\pibctwenty^{1}$ (executing a Markovian policy) on each \texttt{Libero} task individually. We consider the metric $ \Lval(\pibctwenty^1)- \min_{n\ge1} \Lval(\pidelay^n[\pibctwenty])$, which we denote $\Delta \Lval$, as a proxy for measuring how non-Markovian the demonstrator is, evaluated for each \texttt{Libero} task. $\Lval(\pibctwenty^1)$ corresponds to how well a Markovian policy can predict the demonstrator's behavior, while $\min_{n\ge1} \Lval(\pidelay^n[\pibctwenty])$ corresponds to how well the \emph{best} delayed policy can predict their behavior---we would therefore expect this value to be large only if non-Markovian policies are useful for predicting demonstrator behavior on a particular task. In \Cref{fig:val_vs_success_libero}, we plot the success rate difference between $\pibctwenty^{10}$ and $\pibctwenty^{1}$ for each task in \texttt{Libero-90} (denoted as $\Delta \cJ$) against this measure of non-Markovian behavior.\loose

\begin{examplebox}{Example: When do non-Markovian behaviors arise in practice?}

{
\centering

\newlength{\examplecontentwidth}
\setlength{\examplecontentwidth}{\linewidth}

\begin{minipage}[t]{0.63\examplecontentwidth}
\captionsetup{type=figure}
\centering

\begin{subfigure}[t]{0.32\linewidth}
    \centering
    \includegraphics[width=\linewidth]
        {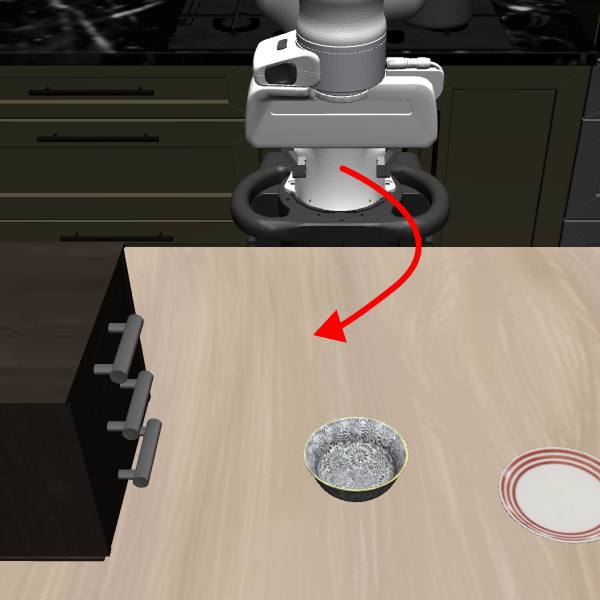}
    \caption{$t=0$.}
    \label{fig:closed_drawer}
\end{subfigure}
\hfill
\begin{subfigure}[t]{0.32\linewidth}
    \centering
    \includegraphics[width=\linewidth]
        {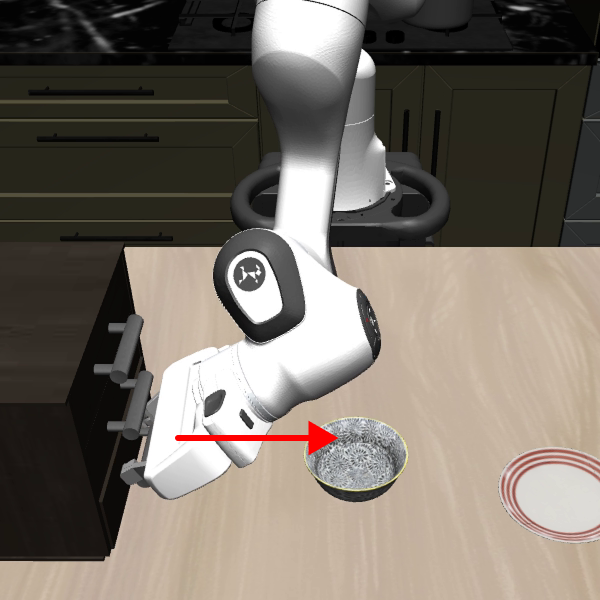}
    \caption{$t=80$.}
    \label{fig:open_drawer}
\end{subfigure}
\hfill
\begin{subfigure}[t]{0.32\linewidth}
    \centering
    \includegraphics[width=\linewidth]
        {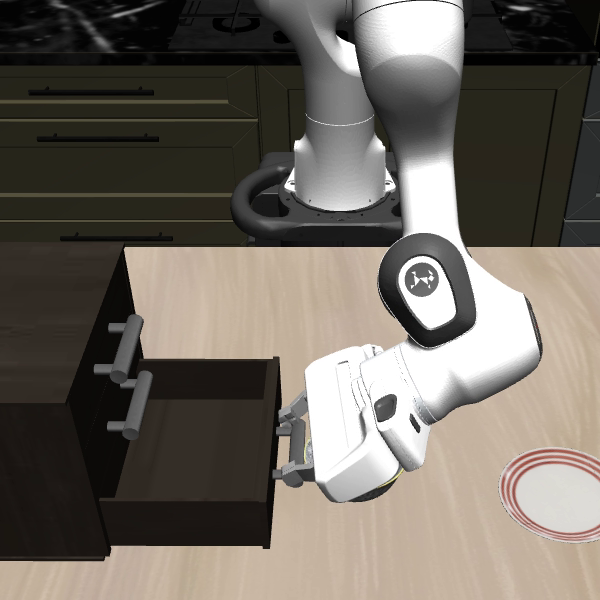}
    \caption{$t=110$.}
    \label{fig:pause_drawer}
\end{subfigure}

\caption{%
Illustration of the task ``open the bottom drawer of the cabinet.''
The panels show
(a) the initial state,
(b) the ``decision boundary'' at which the demonstrator pauses and changes direction,
and
(c) the nearly completed task.
}
\label{fig:drawer_examples}

\end{minipage}
\hfill
\begin{minipage}[t]{0.33\examplecontentwidth}
\captionsetup{type=figure}
\centering

\includegraphics[
    width=\linewidth,
    keepaspectratio
]{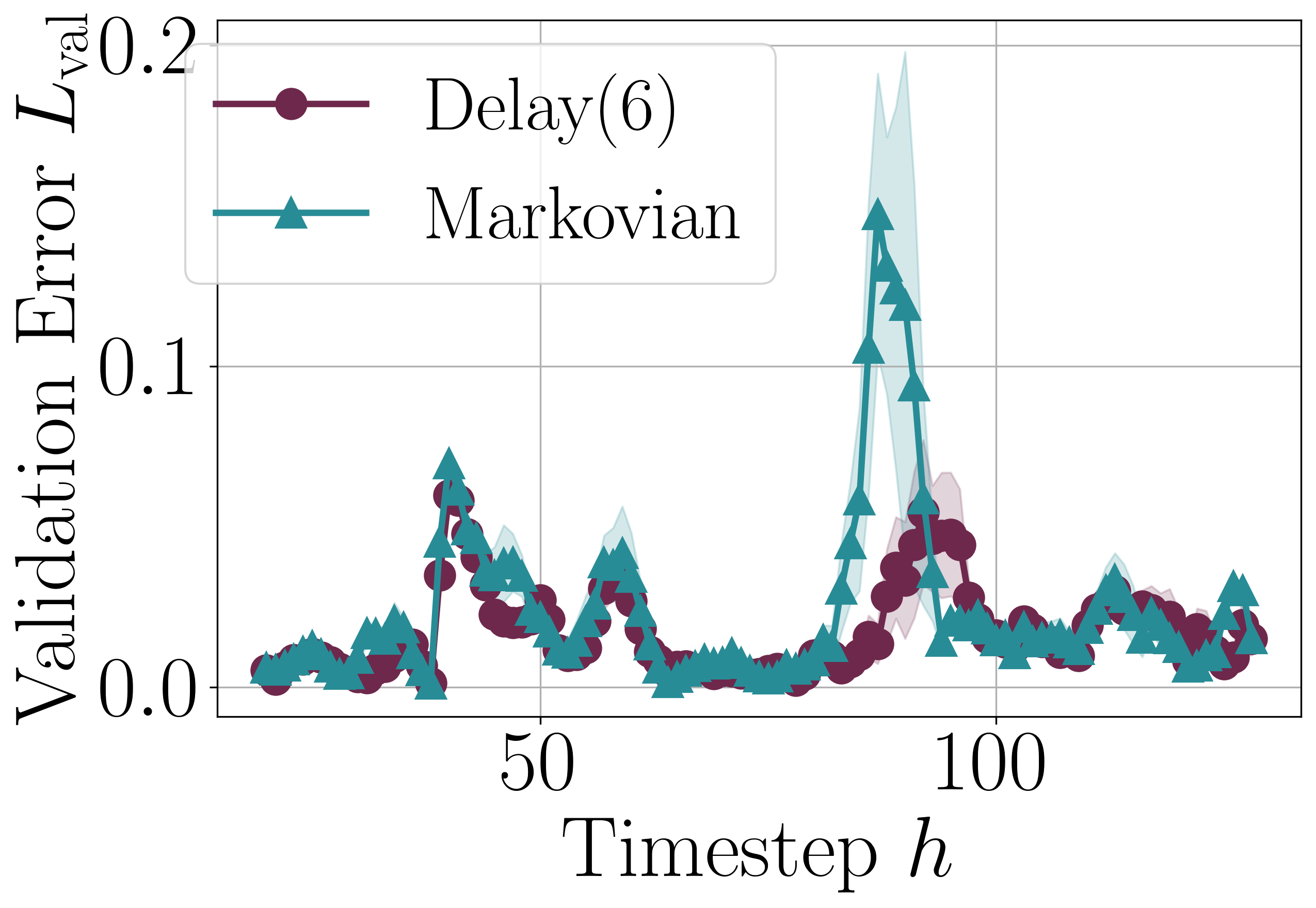}

\caption{%
Action prediction validation error vs. episode timestep.
The demonstrator pauses near timestep 80 while opening the drawer.
}
\label{fig:drawer_val_error}

\end{minipage}
}

To illustrate where non-Markovian behavior can arise in practice, we consider Task 6 of \texttt{Libero-90}, ``open the bottom drawer of the cabinet'', which we illustrate in \Cref{fig:drawer_examples}. This task requires the robot to move down, grasp the bottom handle, and pull the drawer open. While this primarily involves straightforward free-space motion, at the point where the robot reaches the drawer it must both fasten onto the handle and change direction. We refer to this point as the \emph{decision boundary}---in order to complete the task, the robot must ``decide'' at this point to transition from motion in one direction to motion in another direction.

We observe that the human demonstrations exhibit significant pauses at this decision boundary. Instead of immediately changing direction, the human demonstrator stops for several steps, perhaps to ensure the handle is effectively grasped, before proceeding. Such pauses are non-Markovian---while the state has not changed meaningfully, the demonstrator consistently waits for $n$ steps. As such, a Markovian policy will not effectively fit this behavior---instead of predicting ``wait'' for $n$ steps deterministically, it will predict ``wait'' with some probability and ``move'' with some probability (i.e. the marginal action distribution at the decision-boundary). In contrast, either an action-chunked or delayed policy \emph{can} model this non-Markovian behavior, as long as the chunk size and delay are at least $n$. To illustrate this quantitatively, in \Cref{fig:drawer_val_error} we plot the action prediction error for a Markovian policy and a delayed policy along the held-out demonstration trajectory shown in \Cref{fig:drawer_examples}. We see that the Markovian and delayed policy have similar prediction error except around step 80---the precise point this ``decision boundary'' is encountered---when the Markovian policy has significantly higher prediction error.\loose

In policy rollouts, we find that this behavior causes the Markovian policy to deviate from the demonstrator behavior at this decision boundary, leading to out-of-distribution states and poor performance, while delayed or action-chunked policies effectively model the human demonstrator at this state, enabling successful task completion. Such ``decision boundaries'' are very common across robotic manipulation tasks, and, as such, the non-Markovian expressivity of action chunk and delayed policies can enable significant improvements over Markovian policies by fully capturing these behaviors.
\end{examplebox}

If non-Markovian expressivity were predictive of the success of action chunking, we would expect there to be a strong correlation between $\Delta \Lval$ and $\Delta \cJ$.
Somewhat surprisingly, however, we find little correlation between how non-Markovian the demonstrator is for a given task and how much action chunking improves on a single-step policy.
Indeed, in several tasks for which action chunking still improves significantly on the Markovian policy, there is little difference between $\Lval$ of a Markovian and non-Markovian policy---even though the demonstrator's behavior can be accurately modeled by a Markovian policy, non-Markovian policies can still lead to substantial performance improvements in terms of success rate\footnote{Note that, for fully expressing non-Markovian behaviors, a \emph{history-conditioned} policy, e.g. modeling $a_t \mid o_t, o_{t-1}, o_{t-2}, \ldots$, may be necessary. In practice, however, the performance of history-conditioned policies is typically very poor, and we find that delayed policies modeling $a_t \mid o_{t-d}$ perform much better. We believe this is due to the increased sample complexity of history-conditioned policies: by conditioning on a full history, we increase the effective observation space \emph{exponentially} with the length of the history, increasing the challenge of generalization in the low-data regime.}. \sectionend

These results suggest that non-Markovian expressivity only partially explains the success of action chunking, leading us to conclude that other effects must be present as well.
A second hypothesis commonly proposed in the literature is that action chunking reduces the \emph{effective horizon} of the environment, and as such reduces compounding error.
Formally, the suboptimality of a BC policy $\pihat$
typically scales with the horizon of the problem $H$ as well as the supervised learning loss $\epsilon$ (i.e., the population validation error under the demonstrator's trajectory distribution) \cite{foster2024behavior}: $\cJ(\pidemo) - \cJ(\pihat) \le C(H) \cdot \epsilon$,
where $C(H)$ denotes some measure of how quickly the error compounds in our environment (that is, how quickly supervised learning error leads to performance reduction).
Though the precise scaling of $C(H)$ depends on environment properties, in general $C(H)$ increases with $H$. Thus, if we can decrease $H$---for example, to $H/k$---while, critically, keeping $\epsilon$ fixed, this should reduce the compounding error and total suboptimality. By only computing an action every $k$ steps, action chunking \emph{does} decrease the effective horizon in this fashion. 
Here we investigate whether this effect can explain the performance of action chunking.

As a starting point, we note that, as shown in \Cref{fig:success_libero}, delayed policies can match the performance of action chunking in aggregate across \texttt{Libero-90}. This immediately suggests that the performance of action chunking is not due, strictly speaking, to horizon reduction: delayed policies recompute actions at each step, so do not reduce the effective horizon as action chunking does, yet still perform comparably to action-chunked policies.   
However, for the tasks in \Cref{fig:val_vs_success_libero} where $\pibctwenty^{10}$ and $\pibctwenty^{1}$ achieve approximately the same validation loss ($\epsilon$), several of these tasks exhibit significant gaps in performance between action-chunked and single-step policies. This suggests that, while horizon reduction may not be necessary, action chunking nonetheless mitigates compounding error: $C(H)$ is still lower for an action-chunked policy. To explain this, we investigate the effect of action chunked and delayed predictions theoretically.

\textbf{Theory: Predicting demonstrator actions from past states provably reduces compounding error.}
Consider the setting with deterministic dynamics, so that $s_{t+1} = P(s_t,a_t)$, and $P$ is 1-Lipschitz in both state and action. Assume that all policies considered and the reward $r$ are 1-Lipschitz in the state (please see \Cref{sec:proofs} for a precise statement of these assumptions). 
In this setting, we have the following lower bound on the performance of a Markovian policy.
\begin{theorem}\label{thm:lb_markov}
    There exists an environment satisfying the above assumptions, a Markov demonstrator $\pidemo$, and a Markov policy $\pihat$ satisfying $\max_t \Exp^{\pidemo}[\wass(\pidemo(s_t), \pihat(s_t))] \le \epsilon$, for $W_1(\cdot, \cdot)$ the Wasserstein-1 metric,
    such that $\cJ(\pidemo) \ge \cJ(\pihat) + \Omega(2^H \cdot \epsilon)$.
\end{theorem}
\Cref{thm:lb_markov} shows that, if $\pihat$ satisfies a standard bound on supervised learning error, then this error can compound exponentially in horizon, as $2^H$. Please see \Cref{sec:proofs} for a proof of this result. We next show that this exponential dependence can be reduced by playing an action-chunked policy, but that it can \emph{also} be reduced by simply playing a delayed policy. 
\begin{theorem}\label{thm:delayed_upper_bound}
    Assume that for all $t \in [H]$ and some $n \ge 0$: 
    \begin{align}\label{eq:delayed_upper_bound_condition}
        \Exp^{\pidemo}[ \max_{i \in [n]} \wass(\pidemo(s_{t+i-1}), [\pihat_n(s_{t})]_i)] \le \epsilon.
    \end{align}
   for $[\pihat_n(s_{t})]_i$ the $i$th step in the action chunk. Then, for $k<n$:
    \begin{align*}
   \resizebox{\textwidth}{!}{
        $\cJ(\pidemo) - \cJ(\pihat_n^k) \le \cO((k+1)^{H/k} \cdot \epsilon) \quad \text{and} \quad  \cJ(\pidemo) - \cJ(\pidelay^k[\pihat_n]) \le \cO((k+1)^{H/k} \cdot \epsilon).$
    }
    \end{align*}
\end{theorem}
\Cref{thm:delayed_upper_bound} shows that, if the action-chunked policy is able to fit the demonstrator effectively over the entire chunk, then we can bound the compounding error by $\cO((k+1)^{H/k})$\footnote{We remark that, in the case of absolute position control where the dynamics can be approximately modeled by $s_{t+1}=P(s_t,a_t)\approx a_t$, rates on the order of $\Omega(H^2 \cdot \epsilon)$ and $\cO(H^2/k \cdot \epsilon)$ can be obtained for, respectively, the Markovian learner and the action chunking/delayed learners, using similar proof techniques.}. In particular, if $k = c \cdot H$ for some constant fraction $c$, this scales polynomially in $H$ rather than exponentially, offering an exponential improvement over a Markovian policy.
Notably, however, the delayed policy achieves the exact same reduction in compounding error. Furthermore, in \Cref{sec:proofs} we show that this is tight---action-chunked policies must incur compounding error that scales at least as $\Omega((k+1)^{H/k})$, showing that, in the setting of deterministic, smooth dynamics, action chunking does not exhibit any benefits over simply playing a delayed policy in terms of compounding error. \sectionend

These results show that, while action chunking can mitigate compounding error, the primary mechanism is simply because it conditions on previous observations, rather than because it reduces horizon. Intuitively, this arises because earlier states are likely to have compounded \emph{less} error---they are more in-distribution---than later states, if the initial state distributions for train and test are identical.
Thus, if we can effectively model the demonstrator's action at a future state given only a delayed observation, we would expect this prediction to be \emph{more accurate} than if we attempt this inference conditioned on the current state.
Notably, our experimental results in \Cref{fig:val_loss_libero} show that we \emph{can} model the demonstrator's action at a future state effectively, suggesting that the key criteria of \Cref{thm:delayed_upper_bound} hold in practice. While these results rely on the assumption that dynamics are smooth, we show in \Cref{apx: smooth dynamics} that this assumption holds in many robotic control settings of interest.

\begin{conclusionbox}{Takeaway 2: Action chunking reduces compounding error by predicting actions based on past observations}
Action chunking reduces compounding error but, in many settings, this results not from horizon reduction, but from action-chunked policies predicting actions based on delayed observations. We find that, as a result, we can often replicate the success of action chunking simply by training a policy to predict actions based on a delayed observations.
\end{conclusionbox}

\subsection{Hypothesis 3: Representation Learning}\label{sec:hypoth_inductive}

The final hypothesis we evaluate is whether action chunking has representational benefits. 
A commonly-observed phenomenon is that training an action chunk $k$ policy and only executing the first $n < k$ actions from the chunk still performs better than a Markovian policy, even for small $n$. This is commonly attributed to representational benefits of training with action chunking. Here we evaluate whether this effect can also contribute to the superior performance of action chunking. \loose

\refstepcounter{experiment}
\textbf{Experiment \theexperiment: Delayed policies match the representational benefits of action chunking.}
To test this hypothesis, we compare the performance of $\pibctwenty^1$ to $\pibc_1$---executing a single action at each step, but training on action chunks vs. single actions---and also compare this to $\pidelay^5[\pibctwenty]$ and $\widehat{\pi}_{\mathrm{delay}}^{5}$, where $\widehat{\pi}_{\mathrm{delay}}^{5}$ denotes the policy trained directly to predict $a_t \mid o_{t-5}$ (while  $\pidelay^5[\pibctwenty]$ is the delay policy induced by  $\pibctwenty$). We plot aggregate performance on \texttt{Libero-90} in \Cref{fig:libero_repr_learn}. We find that, while action chunking \emph{does} provide representational benefits when compared to a single-step policy, it \emph{does not} provide meaningful benefits when compared to training a delayed policy. \sectionend

\begin{conclusionbox}{Takeaway 3: Action chunking only gives representational benefits for short horizons}
Action chunking provides representation-learning benefits only when predicting the first few steps in the action chunk, but not for predicting actions with more substantial delays. Thus, the benefits of action chunking can again be fully replicated with a delayed policy.
\end{conclusionbox}

\section{Action-Chunked Policies as Implicit Ensembles}\label{sec:ac_ensembles}

\begin{figure*}[t]
    \centering
\begin{minipage}[t]{0.29\textwidth}
    \centering

    \makebox[\linewidth][c]{%
        \rlap{%
            \hspace{-0.2cm}%
            \raisebox{+0.2cm}{%
                \includegraphics[width=\linewidth]{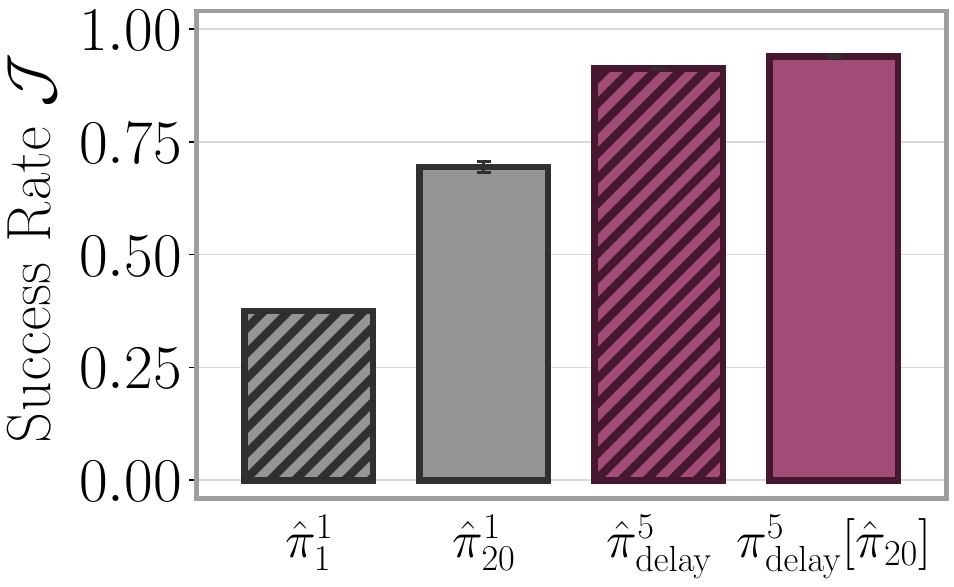}%
            }%
        }%
        \phantom{%
             \includegraphics[width=\linewidth]{images/libero_repr_learn.pdf}%
        }%
    }

    \vspace{-0cm}
    \caption{Success of Markovian and delayed policies trained with and without action chunks, aggregated over \texttt{Libero-90}.}
    \label{fig:libero_repr_learn}
\end{minipage}
    \hfill
    \begin{minipage}[t]{0.32\textwidth}
        \centering
        \includegraphics[width=0.89\linewidth]{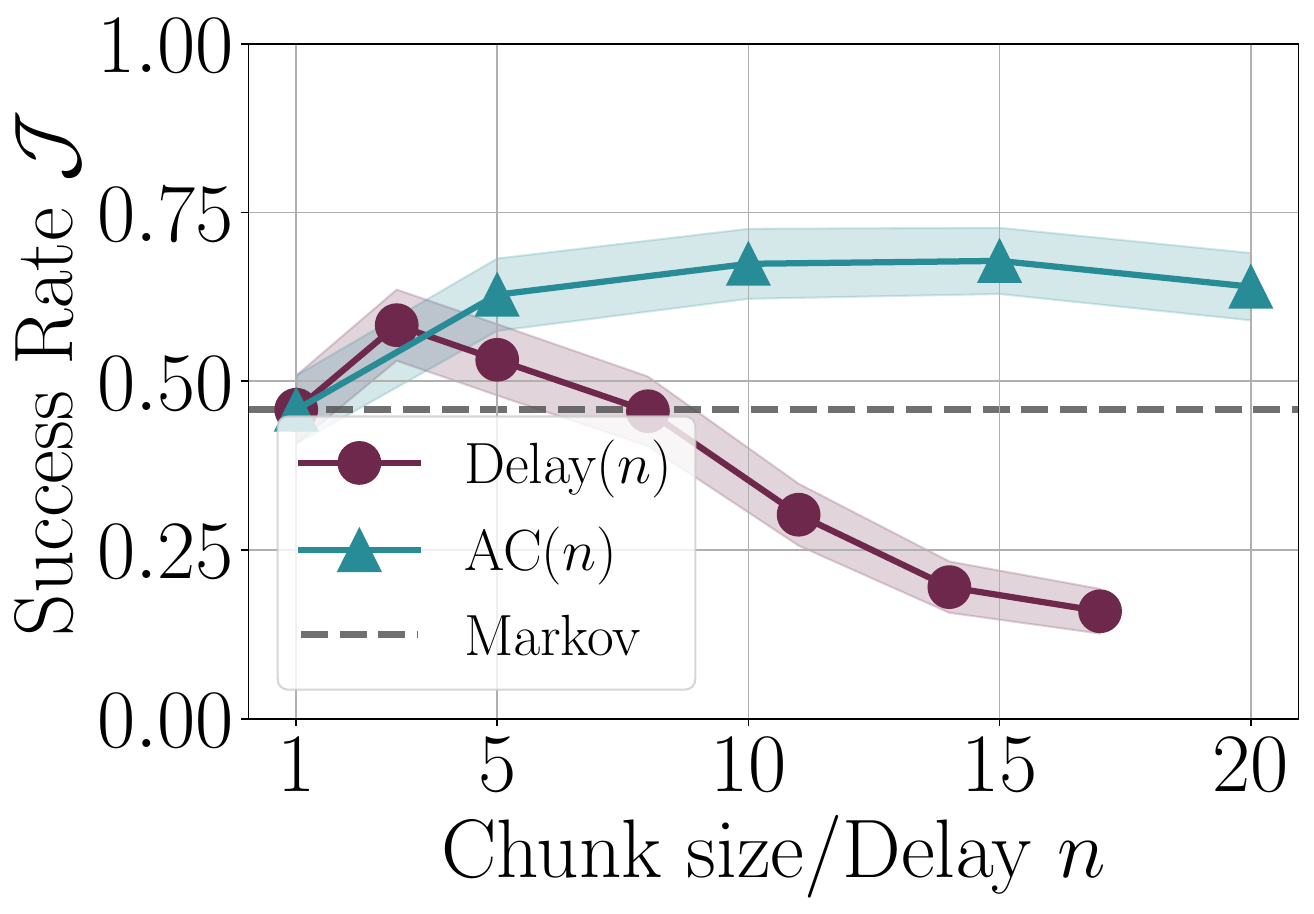}
        \caption{Average task success of delayed policies and action chunking policies, aggregated over all \texttt{Robomimic} tasks.}
        \label{fig:success_robomimic}
    \end{minipage}
    \hfill
    \begin{minipage}[t]{0.32\textwidth}
      \centering
           \includegraphics[width=0.89\linewidth]{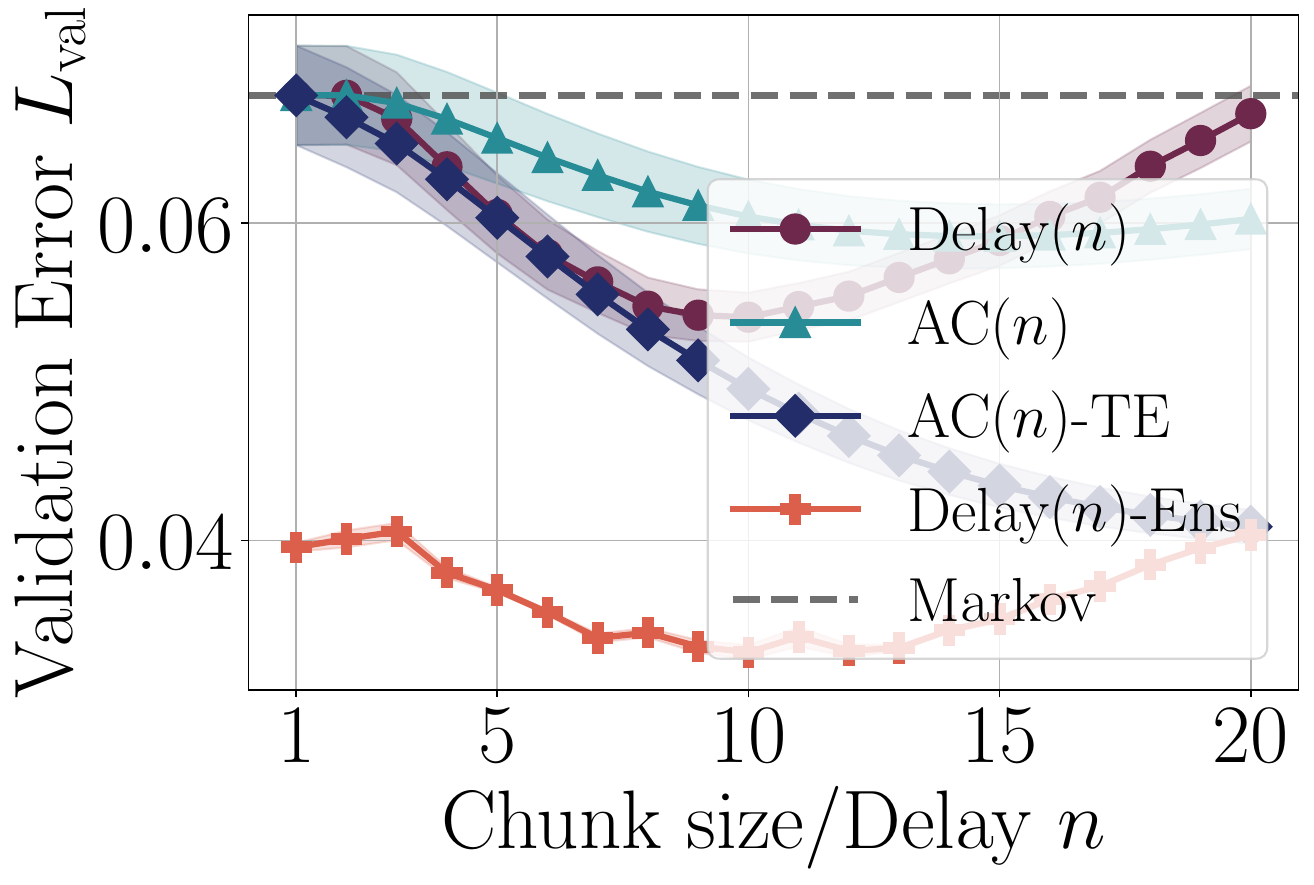}
    \caption{Validation error of Delay$(n)$, AC$(n)$, AC$(n)$ temporal ensemble, Delay$(n)$ ensemble, over \texttt{Robomimic} tasks.}
    \label{fig:robomimic_val}
    \end{minipage}
\vspace{-1em}
\end{figure*}

The preceding results suggest that non-Markovian expressivity and reduction in compounding error contribute to the success of action chunking but that, on \texttt{Libero}, these effects can be fully captured with only a delayed policy---action chunking is not required.
We next seek to understand whether these results generalize beyond \texttt{Libero}, and consider an additional benchmark, the \texttt{Robomimic} benchmark \cite{robomimic2021}. In particular, we consider the four most challenging from \texttt{Robomimic}: \texttt{Can}, \texttt{Square}, \texttt{Transport}, and \texttt{Tool Hang}, and we report results training on the Proficient-Human demonstrations (please see \Cref{apx: robomimic mh} for results on Multi-Human; our results hold identically there).
In \Cref{fig:success_robomimic}, we compare the performance of action-chunked policies to that of delayed policies aggregated across these \texttt{Robomimic} tasks, and in columns 2 and 3 of \Cref{table:final_success} provide numerical success rates for the best-case action-chunked policy compared to the best-case delayed policy. We see that, while delayed policies do improve on single-step policies significantly in \texttt{Robomimic}, they fail to match the performance of action chunking in general, contrasting our results on \texttt{Libero}.

In this section we seek to offer an explanation for why this is the case. Our key insight is that, in addition to the aforementioned benefits of action chunking, action-chunked policies act as \emph{implicit ensembles}. When training an action-chunked policy, for a given action $a_t$ in our training data, we learn the relationships $a_t \mid o_t$, $a_t \mid o_{t-1}$, $\ldots$, $a_t \mid o_{t-k+1}$. In other words, rather than just fitting a \emph{single} $a_t \mid o_{t-d}$, we fit what, effectively, is an \emph{ensemble} of predictors, each learning a different temporal relationship. By learning the relationships $a_t \mid o_t$, $a_t \mid o_{t-1}$, $a_t \mid o_{t-2}$, $\ldots$, $a_t \mid o_{t-k+1}$, an action-chunked policy learns to predict demonstrator actions based on subsets of the true ``feature'' $(o_t,o_{t-1},\dotsc,o_1)$. Conceptually, this is often how ensemble-based methods such as random forests operate---aggregating predictors trained on different features---and it is known that such approaches can achieve lower generalization error than the average error across ensemble members, leading to improved performance \cite{krogh1994neural,ho1998random,breiman2001random}. We conjecture that, by aggregating action predictions over time through interaction with the environment, action chunking implicitly acts as an ensemble, inheriting such benefits. Here we investigate experimentally to what extent this holds in practice.\loose

\refstepcounter{experiment}
\textbf{Experiment \theexperiment: Temporal ensembles enable more accurate action prediction than delayed policies.}
We first evaluate the validation error on \texttt{Robomimic}. We consider an identical setup to that considered in \Cref{fig:val_loss_libero}, but add in two additional methods. First, we consider the temporal ensemble induced by $\pihat_k^n$. That is, for some sequence of observations $(o_1, \ldots, o_k)$, we compute $a^i \leftarrow [\pihat_k(o_{k-i+1})]_i$ for $i\in [n]$; thus, $\abar^n \leftarrow \frac{1}{n} \sum_{i=1}^n a^i$, so that $\abar^n$ is the prediction induced by combining the actions produced by the previous $n$ predictions induced by $\pihat_k$, the \emph{induced temporal ensemble} (which we refer to as AC$(n)$-TE)\footnote{We note that existing works such as ACT \cite{zhao2023actionchunking} propose similar temporal ensembles. The primary difference between our instantiation of a temporal ensemble and existing approaches is that we combine predictions linearly, while previous works apply an exponential weighting, significantly downweighting the contributions of more recent timesteps, and mitigating the ensembling effects.}. We also train an \emph{actual ensemble}, training policies $\{ \pibc_{20,i} \}_{i=1}^m$, where each $\pibc_{20,i}$ is trained on the same dataset $\frakDtrain$, but from different random initializations. To aggregate the ensemble predictions, we average the actions predicted by each ensemble member. In particular, here we only consider the delayed ensemble, that ensembles $\pidelay^n[\pibc_{20,i}]$ (which we denote as Delay$(n+1)$-Ens).

We plot our results in \Cref{fig:robomimic_val}. As can be seen, the temporal ensemble induced by the action-chunked policy achieves lower validation loss than that induced by the standard action-chunked policy, and nearly as low as that of the true ensemble. This suggests that the predictive properties of the ensemble induced by an action-chunked policy are indeed similar to that of a true ensemble, and significantly outperform a single delayed policy. We also note that, for the correct choice of delay, the delayed policy achieves lower validation loss than that of the action-chunked policy, suggesting that non-Markovian expressivity is not the key limitation of the delayed policies in \texttt{Robomimic}. 
Together, this suggests that the implicit ensemble induced by $\pibctwenty$ significantly improves on the predictive ability of any single delayed policy, and comes close to that of a true ensemble. \sectionend

\refstepcounter{experiment}
\textbf{Experiment \theexperiment: Randomized delay deployment of action chunk policies.}
If the cumulative benefits of action chunking are its ability to (a) express non-Markovian demonstrators (\Cref{sec:hypoth_expressivity}), (b) reduce compounding error by predicting on past observations (\Cref{sec:hypoth_horizon}), and (c) instantiate an implicit ensemble, then the key shortcoming of $\pidelay^d[\pibctwenty]$ is simply that it only leverages one of the temporal relationships learned by $\pibctwenty$ (in particular $a_t \mid o_{t-d}$), while typical deployment of action chunking leverages all learned relationships. Thus, if we could enable prediction based on past observations while leveraging all learned temporal relationships in a different manner than standard action chunking deployment, we would expect to achieve similar performance as action chunking.

To test this, we consider the following procedure. At each step $t$ first sample $i \sim \mathrm{unif}(\{0,1,...,n-1\})$, then sample $a_t \sim \pidelay^i[\pihat_k](o_{t-i})$. That is, we randomize over the delay of the possible delayed policies induced by $\pihat_k^n$ at each step. We state the results for this approach in \Cref{table:final_success} (we denote this approach as the ``randomized delay ensemble'', or AC$(n)$-RDE). We see that this approach (nearly) matches or outperforms the performance of action chunking across all settings, \emph{even when the delayed policies perform worse than the action-chunked policy}.
Furthermore, the induced random delay ensemble also exceeds the performance of the temporal ensemble achieved by averaging the predictions (AC$(n)$-TE), suggesting that the ensemble randomization is critical to achieving effective performances. \sectionend

\vspace{0.3cm}
\begin{conclusionbox}{Takeaway 4: Action chunking implicitly instantiates an ensemble}
Action chunking instantiates an implicit ensemble of delayed policies, learning all temporal relationships $a_t \mid o_{t-i+1}$ for $i \in [k]$. This provides significant ``ensemble-like'' benefits, improving action prediction ability and success rate. Combined with the other two hypotheses, this almost fully explains the benefits of action chunking across \texttt{Libero-90} and \texttt{Robomimic}.
\end{conclusionbox}

\begin{table*}[t!]
\centering
\scalebox{0.75}
{
\begin{tabular}{l!{\vrule width 1pt}ccc|cc}
\toprule
\texttt{Task}  & Markovian & AC$(n)$ & Delay$(n)$ & AC$(n)$-RDE & AC$(n)$-TE \\
\midrule

\texttt{Libero-90}  & $68.9$ {\tiny $\pm {0.6}$} & $89.2$ {\tiny $\pm {0.2}$} & $\mathbf{94.0}$ {\tiny $\pm {0.1}$} & $93.6$ {\tiny $\pm {0.1}$} & $92.7$ {\tiny $\pm {0.1}$}   \\

\texttt{Libero-10}  & $19.8$ {\tiny $\pm {1.5}$} & $\mathbf{88.7}$ {\tiny $\pm {0.4}$} & ${86.8}$ {\tiny $\pm {0.4}$} & $\mathbf{88.5}$ {\tiny $\pm {0.4}$} &  $86.0$ {\tiny $\pm {0.5}$}  \\

\texttt{Libero-Spatial}  & $58.1$ {\tiny $\pm {0.9}$} & $91.6$ {\tiny $\pm {0.4}$} & $\mathbf{92.0}$ {\tiny $\pm {0.5}$} & $\mathbf{92.7}$ {\tiny $\pm {0.3}$} & $91.1$ {\tiny $\pm {0.3}$} \\

\texttt{Libero-Goal}  & $62.7$ {\tiny $\pm {1.0}$} & $\mathbf{96.5}$ {\tiny $\pm {0.2}$} & $\mathbf{96.3}$ {\tiny $\pm {0.3}$} & $\mathbf{96.3}$ {\tiny $\pm {0.3}$} & $\mathbf{96.7}$ {\tiny $\pm {0.2}$}\\

\texttt{Libero-Object}  & $59.8$ {\tiny $\pm {2.3}$} & $\mathbf{98.0}$ {\tiny $\pm {0.4}$} & $\mathbf{97.4}$ {\tiny $\pm {0.3}$} & $\mathbf{97.3}$ {\tiny $\pm {0.5}$} & $\mathbf{97.6}$ {\tiny $\pm {0.4}$}\\

\midrule
\texttt{Robomimic Can PH}  & $83.7$ {\tiny $\pm {0.4}$} & $\mathbf{97.2}$ {\tiny $\pm {0.3}$} & $93.5$ {\tiny $\pm {0.4}$} & $96.7$ {\tiny $\pm {0.2}$} & $96.2$ {\tiny $\pm {0.3}$}  \\

\texttt{Robomimic Square PH}  & $69.0$ {\tiny $\pm {0.8}$} & $\mathbf{85.4}$ {\tiny $\pm {0.5}$} & $80.8$ {\tiny $\pm {0.6}$} & $82.4$ {\tiny $\pm {0.6}$} & $80.6$ {\tiny $\pm {0.5}$}   \\

\texttt{Robomimic Transport PH}  & $3.3$ {\tiny $\pm {0.3}$} & $\mathbf{12.6}$ {\tiny $\pm {0.5}$} & $7.9$ {\tiny $\pm {0.5}$} & $\mathbf{12.1}$ {\tiny $\pm {0.5}$} & $\mathbf{12.2}$ {\tiny $\pm {0.6}$}   \\

\texttt{Robomimic Tool Hang PH}  & $28.0$ {\tiny $\pm {0.8}$} & $\mathbf{75.2}$ {\tiny $\pm {0.5}$} & $51.6$ {\tiny $\pm {0.9}$} & $71.8$ {\tiny $\pm {0.8}$} & $42.2$ {\tiny $\pm {1.0}$}   \\

\bottomrule
\end{tabular}
}
\caption{
Comparison of success rates across \texttt{Libero-90} and \texttt{Robomimic}. We see that action chunking significantly outperforms delayed policies in several settings, but that randomized delay ensembles perform comparably to action chunking in all cases.
}
\vspace{-1.0em}
\label{table:final_success}
\end{table*}

\section{Randomized Delay Ensembles Match the Performance of Action Chunking in Real-World Robotic Control}\label{sec:results_real}

We next test whether our conclusions hold in real-world robotic manipulation settings.
We utilize the Franka Emika robot arm and consider three tasks, illustrated in Figures \ref{fig:carrot}-\ref{fig:sushi}. Specifically, in the first task the goal is to pick up a carrot and put it in a bowl (``carrot in bowl''), in the second it is to remove the bread from the toaster and put it on the table (``bread toaster''), and in the final task we must pick up the sushi and put it in a cup (``sushi in cup''). 
We adopt delta joint position control and use control frequency 15Hz. See Appendix \ref{apx:details real world} for more details on real-world experiments.

For each task, we collect 50 human demonstrations and train a diffusion policy with action chunk length $k= 20$ on the demonstrations, $\pibctwenty$. We consider several different sampling strategies. First, we deploy the policy as an action chunk 10, $\pibctwenty^{10}$, and action chunk 1 (Markovian) policy, $\pibctwenty^{1}$. Next, motivated by our previous results, we test the induced delay 5 policy, $\pidelay^{5}[\pibctwenty]$ as well as the randomized delay ensemble policy introduced in the previous section, AC(10)-RDE. Unlike the simulated setting where pausing to compute the next action does not lead to changes in the environment, in the real world, pausing to compute the next action can lead to non-negligible changes in the state---in practice, the robot is not completely static even if no action is being commanded, so the state at the start and end of policy inference can be different, leading to poor policy performance. To mitigate the effect of this, we interleave policy inference and operation and start computing the action for the next step before the current action has finished executing. While this eliminates the need for pauses, it does introduce a small delay---in practice, then, all approaches here are run with an additional delay of 1 timestep.

\begin{figure*}[t]
    \centering
    \includegraphics[scale=0.4]{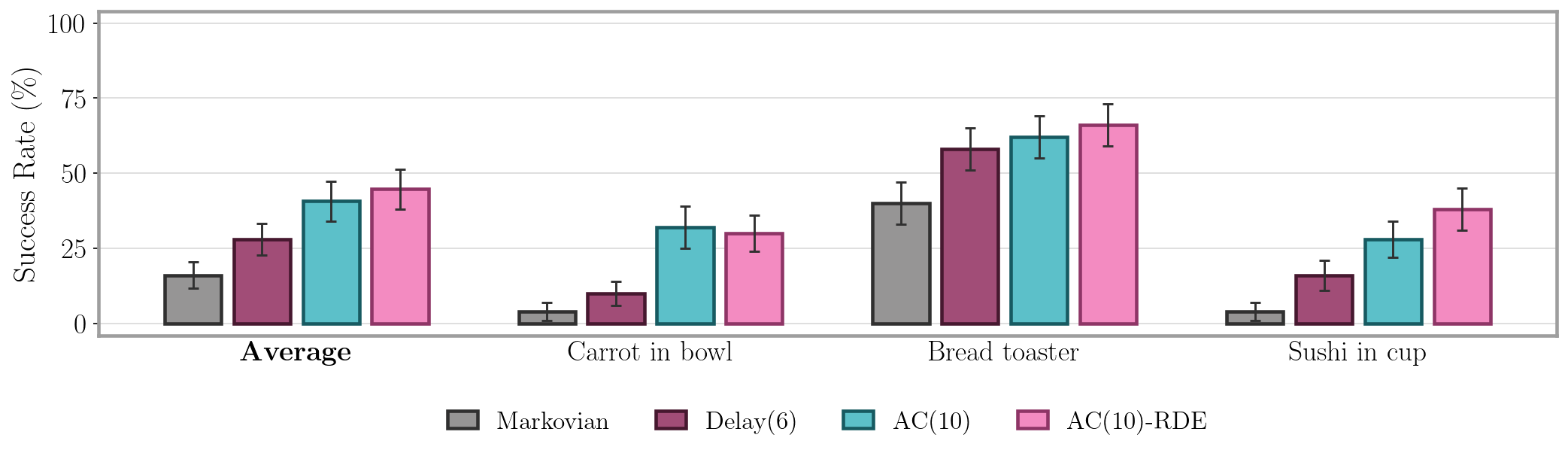}
    \caption{
\footnotesize
Comparison of success rates across the three real-world tasks considered. While action chunking significantly outperforms Markovian policies, delayed policies capture much of the performance of action chunking, and randomized delays (RDE) fully captures the performance of action chunking.
}
\label{fig:real}
\end{figure*}

\begin{wrapfigure}{r}{0.52\textwidth}
    \centering
    \vspace{-0.5em}
    \begin{subfigure}[t]{0.32\linewidth}
        \centering
        \includegraphics[width=\linewidth]{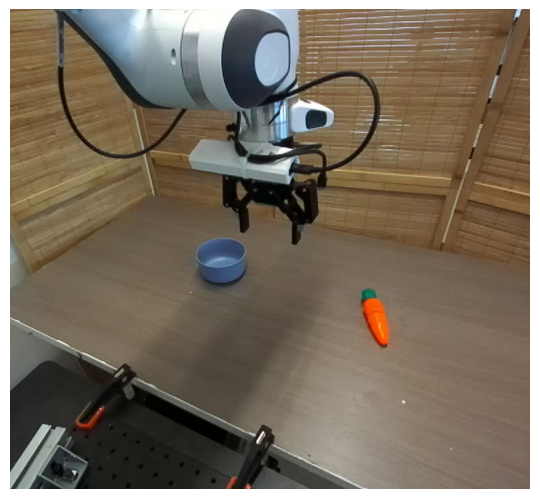}
        \caption{Carrot in bowl.}
        \label{fig:carrot}
    \end{subfigure}
    \hfill
    \begin{subfigure}[t]{0.32\linewidth}
        \centering
        \includegraphics[width=\linewidth]{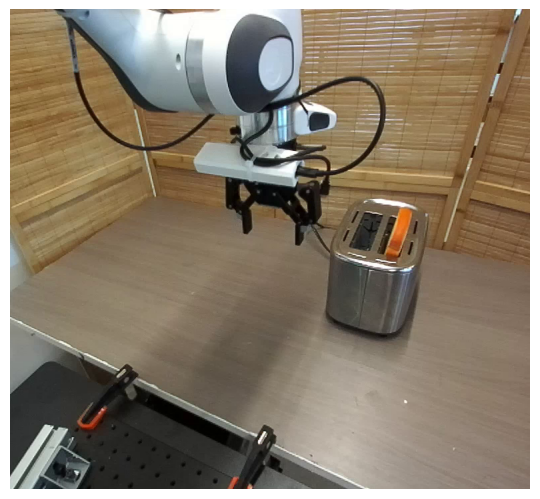}
        \caption{Bread toaster.}
        \label{fig:bread}
    \end{subfigure}
    \hfill
    \begin{subfigure}[t]{0.32\linewidth}
        \centering
        \includegraphics[width=\linewidth]{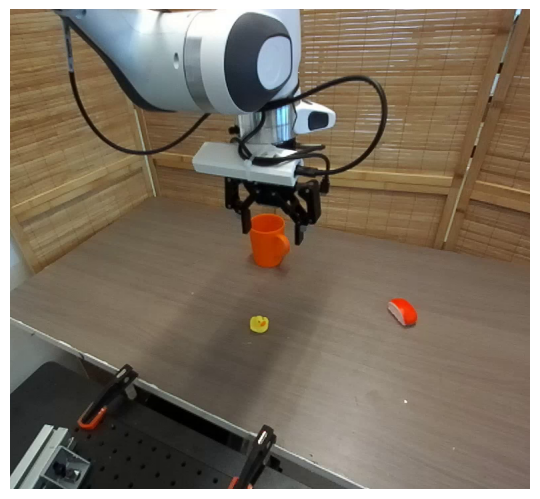}
        \caption{Sushi in cup.}
        \label{fig:sushi}
    \end{subfigure}

    \caption{
    Real-world evaluation tasks.
    }
    \label{fig:real_world}
\end{wrapfigure}
Our results are given in \Cref{fig:real}, where each bar corresponds to the average success rate rolling out a policy 50 times. We see that the Markovian policies perform poorly and are significantly outperformed by action chunking policies. The delayed policy, while outperforming the Markovian policy, cannot, in general, match the performance of the action chunked policy. However, the randomized delay ensemble matches and even exceeds the performance of the action-chunking policy. These results perfectly replicate our simulated results. While simple delayed policies are able to capture non-Markovian behavior and mitigate compounding error in a way that enables significant improvements over Markovian policies, they do not capture the implicit ensembling effect of action chunking. However, by using the randomized delay ensemble induced by the action-chunked policy, we can match or exceed the performance of action chunking.\loose

\begin{conclusionbox}{Takeaway 5: Action chunking is not, in general, necessary for effective real-world robotic control}
For certain real-world robotic control settings, even when Markovian policies are insufficient, we can match the performance of action chunking with randomized delay ensembles---playing action chunks is not necessary as long as we predict based on past observations, and replicate the implicit ensembling effect of action chunking.
\end{conclusionbox}

\section{Going Beyond Action Chunking with Ensembled BC Policies}\label{sec:ensembles}

Our previous results suggest that the key benefits of action chunking are (a) its ability to condition on delayed observations and (b) its implicit ensembling behavior. Here, we consider whether these insights motivate more  effective approaches to BC for robotic control. While a full investigation of this is beyond the scope of this work (see the following section for discussion of other promising directions), here we make a first attempt at applying our insights to improve BC performance.

Our starting point is \Cref{fig:robomimic_val}, where we see that a true ensemble of delayed policies achieves even lower action prediction loss than the temporal ensemble induced by the action-chunked policy.
Motivated by this, we hope to amplify these benefits of action chunking by deploying an actual ensemble---that is, a set of policies $\{ \pibc_{20,i} \}_{i=1}^m$ trained independently on $\frakDtrain$, as described above. If the benefits of action chunking are indeed what our analysis suggests, we would hope that by training an explicit ensemble (rather than the implicit ensemble induced by an action-chunked policy) we might achieve even better performance.

\begin{table*}[t!]
\centering
\scalebox{0.75}
{
\begin{tabular}{l!{\vrule width 1pt}ccc|ccc}
\toprule
\texttt{Task}  & Markovian & AC$(n)$ & Delay$(n)$ & AC$(n)$-Ens & Delay$(n)$-Ens & AC$(n)$-RDE-Ens\\
\midrule

\texttt{Libero-90}  & $68.9$ {\tiny $\pm {0.6}$} & $89.2$ {\tiny $\pm {0.2}$} & $94.0$ {\tiny $\pm {0.1}$} & ${94.1}$ {\tiny $\pm {0.1}$} & $\mathbf{95.0}$ {\tiny $\pm {0.1}$} & ${94.6}$ {\tiny $\pm {0.0}$}  \\

\texttt{Libero-10}  & $19.8$ {\tiny $\pm {1.5}$} & ${88.7}$ {\tiny $\pm {0.4}$} & ${86.8}$ {\tiny $\pm {0.4}$} & $\mathbf{90.6}$ {\tiny $\pm {0.4}$} &  $\mathbf{90.2}$ {\tiny $\pm {0.5}$}  & ${89.5}$ {\tiny $\pm {0.5}$} \\

\texttt{Libero-Spatial}  & $58.1$ {\tiny $\pm {0.9}$} & $91.6$ {\tiny $\pm {0.4}$} & ${92.0}$ {\tiny $\pm {0.5}$} & $\mathbf{95.4}$ {\tiny $\pm {0.4}$} & $\mathbf{94.8}$ {\tiny $\pm {0.4}$}  & $\mathbf{94.7}$ {\tiny $\pm {0.6}$}\\

\texttt{Libero-Goal}  & $62.7$ {\tiny $\pm {1.0}$} & ${96.5}$ {\tiny $\pm {0.2}$} & ${96.3}$ {\tiny $\pm {0.3}$} & $\mathbf{98.8}$ {\tiny $\pm {0.3}$} & ${97.5}$ {\tiny $\pm {0.5}$} & ${97.4}$ {\tiny $\pm {0.3}$}\\

\texttt{Libero-Object}  & $59.8$ {\tiny $\pm {2.3}$} & $\mathbf{98.0}$ {\tiny $\pm {0.4}$} & $\mathbf{97.4}$ {\tiny $\pm {0.3}$} & $\mathbf{98.8}$ {\tiny $\pm {0.0}$} & $\mathbf{98.2}$ {\tiny $\pm {0.4}$} & $\mathbf{98.3}$ {\tiny $\pm {0.4}$}\\

\midrule

\texttt{Robomimic Can PH}  & $83.7$ {\tiny $\pm {0.4}$} & $97.2$ {\tiny $\pm {0.3}$} & $93.5$ {\tiny $\pm {0.4}$} & $\mathbf{98.5}$ {\tiny $\pm {0.1}$} & ${97.7}$ {\tiny $\pm {0.4}$}   & $\mathbf{98.7}$ {\tiny $\pm {0.3}$}  \\

\texttt{Robomimic Square PH}  & $69.0$ {\tiny $\pm {0.8}$} & $85.4$ {\tiny $\pm {0.5}$} & $80.8$ {\tiny $\pm {0.6}$} & $\mathbf{88.3}$ {\tiny $\pm {0.3}$} & $\mathbf{87.7}$ {\tiny $\pm {1.1}$}   & ${85.7}$ {\tiny $\pm {0.3}$}  \\

\texttt{Robomimic Transport PH}  & $3.3$ {\tiny $\pm {0.3}$} & $12.6$ {\tiny $\pm {0.5}$} & $7.9$ {\tiny $\pm {0.5}$}  & $\mathbf{41.5}$ {\tiny $\pm {0.4}$} & ${38.9}$ {\tiny $\pm {2.1}$}   & ${38.4}$ {\tiny $\pm {0.2}$}  \\

\texttt{Robomimic Tool Hang PH}  & $28.0$ {\tiny $\pm {0.8}$} & $75.2$ {\tiny $\pm {0.5}$} & $51.6$ {\tiny $\pm {0.9}$}  & $\mathbf{87.6}$ {\tiny $\pm {1.6}$} & $\mathbf{84.9}$ {\tiny $\pm {1.2}$}   & $\mathbf{86.4}$ {\tiny $\pm {0.7}$}  \\

\bottomrule
\end{tabular}
}
\caption{
Comparison of success rates across \texttt{Libero-90} and \texttt{Robomimic} for an explicit ensemble. We see that our explicit ensemble outperforms the implicitly-ensembled AC$(n)$ policies.
}
\vspace{-1.0em}
\label{table:final_success_explicit_ensemble}
\end{table*}

Here we consider the ensemble $\{ \pibc_{20,i} \}_{i=1}^m$ (AC$(n)$-Ens) as well as the induced ensemble $\{ \pidelay^n[\pibc_{20,i}] \}_{i=1}^m$ (Delay$(n+1)$-Ens) and the randomized delay ensemble considered in Experiment 6, but where we now also sample an ensemble member $i \in \mathrm{unif}([m])$ at each step in addition to a delay (AC$(n)$-RDE-Ens). We consider two aggregation approaches for AC$(n)$-Ens and Delay$(n+1)$-Ens: either predicting an action from each ensemble and averaging their predictions, or randomly sampling an ensemble index and simply computing the action from that ensemble member. We provide results for both types of ensembles in \Cref{table:final_success_explicit_ensemble} (in all cases providing results for the best-case delay or action chunk length, and ensemble aggregation approach). We see that, \emph{across all settings, explicit ensembles improve on the performance of action chunking}, and this is true whether we use the full action-chunked ensemble or the ensemble of induced delay policy. Notably, on \texttt{Robomimic Transport}, we see an approximately +30\% boost in performance from using the explicit ensemble over the action-chunked policy.
Not only do these results corroborate our findings on action chunking---we can match the performance of action chunking by using delayed ensembles---they also suggest a promising approach for improving the performance of BC policies.

\section{Conclusion}
\label{sec:conclusion}

In this work, we present an analysis of the mechanisms enabling the effective performance of action chunking. We examine in detail three mechanisms already postulated in existing literature~\cite{chi2023diffusion,simchowitz2025pitfalls,zhang2025actionchunkingexploratorydata}: the ability to express temporally consistent behaviors, a horizon-reduction effect, and more favorable representation learning. We show that existing hypothesis fail to fully capture the benefits of action chunking, and instead demonstrate that non-Markovian expressivity and reduction in compounding error---both captured by simple delayed policies---as well as the implicit ensembling effect of action chunking policies, is able to almost fully explain the success of action chunking. Our results hold on standard simulated benchmarks, as well as real-world robotic manipulation tasks.  We further show that by taking an \textit{explicit} ensemble of delayed policies, it is possible to further increase performance beyond that of action chunking. We believe this work opens up a variety of interesting directions for future work.\loose

\textbf{Going Beyond Action Chunking.}
How can we improve over the performance of action chunking? While our results in \Cref{sec:ensembles} are directly motivated by our observations that action chunking acts as an implicit ensemble, we believe our insights motivate other approaches that could lead to further improvements. As a concrete direction, we note that in \Cref{fig:val_loss_libero}, delays of around 5-15 achieve the lowest validation error. While our results on implicit ensembling suggest that utilizing multiple delays is important, could we adaptively select the delays based on which achieve lowest validation error, and could this improve performance further?

\textbf{Non-Markovian Behavior and History Conditioning.}
Our results on non-Markovian expressivity suggest that this expressivity is a primary benefit of action chunking. While our results show that a simple delayed policy captures much of the necessary non-Markovian behavior, to capture general non-Markovian behavior we may need a fully history conditioned policy, modeling $a_t \mid o_t, o_{t-1}, o_{t-2}, \ldots$. Progress has been made developing history-conditioned policies for robotic control \cite{torne2026mem,mark2026bpp}, yet fully incorporating history-conditioning has proved challenging. Interestingly, our results suggest that history conditioning may not be the full solution. While history conditioning would capture non-Markovian behavior, it is not clear that it captures the reduction in compounding error or implicit ensembling effects induced by action chunking. How can we ensure that we obtain the other benefits of action chunking while incorporating history-conditioning to fully express non-Markovian behavior?

\textbf{The Role of Control Frequency.}
In all our experiments we utilized a control frequency in the range of 15-20 Hz (for simulated experiments, demonstrations were collected at 20 Hz). Experimentally, we found that when utilizing higher control frequencies, for example 50-60 Hz, delayed policies were not able to replicate the success of action chunking, but treating sub-chunks of length 5 as ``actions'' (so that the effective control frequency is in the range of 10-20Hz), performance was significantly improved.
In other words, the temporal consistency of action chunking \emph{does} matter, but only at high frequencies. We believe that this may be due to the frequency at which humans operate. There is evidence that human's visually guided behavior operates at a frequency in the 2-10Hz range \cite{day2000voluntary,susilaradeya2019extrinsic}, 
and, as a result, at a control frequency of 50-60Hz, actions are recorded at a much higher frequency than the human demonstrator is updating their behavior, so there may exist significant temporal correlation between actions that is effectively modeled with action chunking. While this is consistent with our analysis---we can see this as simply another form of non-Markovian behavior---it suggests that there exists a complex relationship between the behaviors expressed by the demonstrator, the control frequency of the policy, and the role of action chunking. We believe that deeper investigation into this interaction could lead to further insights into the role of action chunking.\loose

\textbf{Practical Benefits of Action Chunking.}
As noted in \Cref{sec:results_real}, unlike in simulation where there is no ``pause'' in the environment necessary for policy inference, in the real world pauses for inference can affect policy performance. Fundamentally, such pauses introduce distribution shift between the deployed policy and the demonstrator---the demonstrator has no ``inference pauses'', and if the robot state changes at inference pauses in policy deployment, this could lead to greater compounding error. As such, mitigating pauses due to inference will likely lead to improved performance. Action chunking naturally reduces inference pauses, and recent work has sought to eliminate the need for any inference calls \cite{black2025realtimeexecutionactionchunking}. To what extent does the reduction in inference pauses contribute to the benefits of action chunking in practice?

\clearpage
\acknowledgments{We acknowledge ISCRA for awarding this project access to the LEONARDO supercomputer, owned by the EuroHPC Joint Undertaking, hosted by CINECA (Italy). In addition, we acknowledge support from ONR N00014-25-1-2060.}


\bibliography{refs}  

@misc{zhao2023actionchunking,
      title={Learning Fine-Grained Bimanual Manipulation with Low-Cost Hardware}, 
      author={Tony Z. Zhao and Vikash Kumar and Sergey Levine and Chelsea Finn},
      year={2023},
      eprint={2304.13705},
      archivePrefix={arXiv},
      primaryClass={cs.RO},
      url={https://arxiv.org/abs/2304.13705}, 
}

@inproceedings{ho2020ddpm,
 author = {Ho, Jonathan and Jain, Ajay and Abbeel, Pieter},
 booktitle = {Advances in Neural Information Processing Systems},
 editor = {H. Larochelle and M. Ranzato and R. Hadsell and M.F. Balcan and H. Lin},
 pages = {6840--6851},
 publisher = {Curran Associates, Inc.},
 title = {Denoising Diffusion Probabilistic Models},
 url = {https://proceedings.neurips.cc/paper_files/paper/2020/file/4c5bcfec8584af0d967f1ab10179ca4b-Paper.pdf},
 volume = {33},
 year = {2020}
}

@article{chi2023diffusionpolicy,
author = {Cheng Chi and Zhenjia Xu and Siyuan Feng and Eric Cousineau and Yilun Du and Benjamin Burchfiel and Russ Tedrake and Shuran Song},
title ={Diffusion policy: Visuomotor policy learning via action diffusion},
journal = {The International Journal of Robotics Research},
year = {2023},
}

@misc{black2024pi0,
      title={$\pi_0$: A Vision-Language-Action Flow Model for General Robot Control}, 
      author={Kevin Black and Noah Brown and Danny Driess and Adnan Esmail and Michael Equi and Chelsea Finn and Niccolo Fusai and Lachy Groom and Karol Hausman and Brian Ichter and Szymon Jakubczak and Tim Jones and Liyiming Ke and Sergey Levine and Adrian Li-Bell and Mohith Mothukuri and Suraj Nair and Karl Pertsch and Lucy Xiaoyang Shi and James Tanner and Quan Vuong and Anna Walling and Haohuan Wang and Ury Zhilinsky},
      year={2024},
      eprint={2410.24164},
      archivePrefix={arXiv},
      primaryClass={cs.LG},
      url={https://arxiv.org/abs/2410.24164}, 
}

@article{li2025reinforcement,
  title={Reinforcement Learning with Action Chunking},
  author={Li, Qiyang and Zhou, Zhiyuan and Levine, Sergey},
  journal={arXiv preprint arXiv:2507.07969},
  year={2025}
}

@InProceedings{pmlr-v202-laroche23a,
  title = {On the Occupancy Measure of Non-{M}arkovian Policies in Continuous {MDP}s},
  author = {Laroche, Romain and Tachet Des Combes, Remi},
  booktitle = {Proceedings of the 40th International Conference on Machine Learning},
  pages = {18548--18562},
  year = {2023},
  editor = {Krause, Andreas and Brunskill, Emma and Cho, Kyunghyun and Engelhardt, Barbara and Sabato, Sivan and Scarlett, Jonathan},
  volume = {202},
  series = {Proceedings of Machine Learning Research},
  month = {23--29 Jul},
  publisher = {PMLR},
  url = {https://proceedings.mlr.press/v202/laroche23a.html}
}

@misc{pertsch2025fast,
      title={FAST: Efficient Action Tokenization for Vision-Language-Action Models}, 
      author={Karl Pertsch and Kyle Stachowicz and Brian Ichter and Danny Driess and Suraj Nair and Quan Vuong and Oier Mees and Chelsea Finn and Sergey Levine},
      year={2025},
      eprint={2501.09747},
      archivePrefix={arXiv},
      primaryClass={cs.RO},
      url={https://arxiv.org/abs/2501.09747}, 
}

@InProceedings{ross2011reduction,
  title = 	 {A Reduction of Imitation Learning and Structured Prediction to No-Regret Online Learning},
  author = 	 {Ross, Stephane and Gordon, Geoffrey and Bagnell, Drew},
  booktitle = 	 {Proceedings of the Fourteenth International Conference on Artificial Intelligence and Statistics},
  pages = 	 {627--635},
  year = 	 {2011},
  editor = 	 {Gordon, Geoffrey and Dunson, David and Dudík, Miroslav},
  volume = 	 {15},
  series = 	 {Proceedings of Machine Learning Research},
  address = 	 {Fort Lauderdale, FL, USA},
  month = 	 {11--13 Apr},
  publisher =    {PMLR},
  url = 	 {https://proceedings.mlr.press/v15/ross11a.html}
}

@inproceedings{pmerleau1988alvinn,
 author = {Pomerleau, Dean A.},
 booktitle = {Advances in Neural Information Processing Systems},
 editor = {D. Touretzky},
 pages = {},
 publisher = {Morgan-Kaufmann},
 title = {ALVINN: An Autonomous Land Vehicle in a Neural Network},
 url = {https://proceedings.neurips.cc/paper_files/paper/1988/file/812b4ba287f5ee0bc9d43bbf5bbe87fb-Paper.pdf},
 volume = {1},
 year = {1988}
}

@misc{loshchilov2019adamw,
      title={Decoupled Weight Decay Regularization}, 
      author={Ilya Loshchilov and Frank Hutter},
      year={2019},
      eprint={1711.05101},
      archivePrefix={arXiv},
      primaryClass={cs.LG},
      url={https://arxiv.org/abs/1711.05101}, 
}

@inproceedings{loshchilov2017sgdr,
title={{SGDR}: Stochastic Gradient Descent with Warm Restarts},
author={Ilya Loshchilov and Frank Hutter},
booktitle={International Conference on Learning Representations},
year={2017},
url={https://openreview.net/forum?id=Skq89Scxx}
}

@misc{black2025realtimeexecutionactionchunking,
      title={Real-Time Execution of Action Chunking Flow Policies}, 
      author={Kevin Black and Manuel Y. Galliker and Sergey Levine},
      year={2025},
      eprint={2506.07339},
      archivePrefix={arXiv},
      primaryClass={cs.RO},
      url={https://arxiv.org/abs/2506.07339}, 
}

@misc{zhang2025actionchunkingexploratorydata,
      title={Action Chunking and Exploratory Data Collection Yield Exponential Improvements in Behavior Cloning for Continuous Control}, 
      author={Thomas T. Zhang and Daniel Pfrommer and Chaoyi Pan and Nikolai Matni and Max Simchowitz},
      year={2025},
      eprint={2507.09061},
      archivePrefix={arXiv},
      primaryClass={cs.LG},
      url={https://arxiv.org/abs/2507.09061}, 
}

@article{chi2023diffusion,
  title={Diffusion policy: Visuomotor policy learning via action diffusion},
  author={Chi, Cheng and Xu, Zhenjia and Feng, Siyuan and Cousineau, Eric and Du, Yilun and Burchfiel, Benjamin and Tedrake, Russ and Song, Shuran},
  journal={The International Journal of Robotics Research},
  pages={02783649241273668},
  year={2023},
  publisher={SAGE Publications Sage UK: London, England}
}

@article{team2024octo,
  title={Octo: An open-source generalist robot policy},
  author={Team, Octo Model and Ghosh, Dibya and Walke, Homer and Pertsch, Karl and Black, Kevin and Mees, Oier and Dasari, Sudeep and Hejna, Joey and Kreiman, Tobias and Xu, Charles and others},
  journal={arXiv preprint arXiv:2405.12213},
  year={2024}
}

@article{black2024pi_0,
  title={$\pi_0$: A Vision-Language-Action Flow Model for General Robot Control},
  author={Black, Kevin and Brown, Noah and Driess, Danny and Esmail, Adnan and Equi, Michael and Finn, Chelsea and Fusai, Niccolo and Groom, Lachy and Hausman, Karol and Ichter, Brian and others},
  journal={arXiv preprint arXiv:2410.24164},
  year={2024}
}

@article{dasari2024ingredients,
  title={The ingredients for robotic diffusion transformers},
  author={Dasari, Sudeep and Mees, Oier and Zhao, Sebastian and Srirama, Mohan Kumar and Levine, Sergey},
  journal={arXiv preprint arXiv:2410.10088},
  year={2024}
}

@inproceedings{o2024open,
  title={Open x-embodiment: Robotic learning datasets and rt-x models: Open x-embodiment collaboration 0},
  author={O’Neill, Abby and Rehman, Abdul and Maddukuri, Abhiram and Gupta, Abhishek and Padalkar, Abhishek and Lee, Abraham and Pooley, Acorn and Gupta, Agrim and Mandlekar, Ajay and Jain, Ajinkya and others},
  booktitle={2024 IEEE International Conference on Robotics and Automation (ICRA)},
  pages={6892--6903},
  year={2024},
  organization={IEEE}
}

@article{zhang2025imitation,
  title={Imitation learning in continuous action spaces: mitigating compounding error without interaction},
  author={Zhang, Thomas T and Pfrommer, Daniel and Matni, Nikolai and Simchowitz, Max},
  journal={arXiv preprint arXiv},
  volume={2507},
  year={2025}
}

@article{simchowitz2025pitfalls,
  title={The pitfalls of imitation learning when actions are continuous},
  author={Simchowitz, Max and Pfrommer, Daniel and Jadbabaie, Ali},
  journal={arXiv preprint arXiv:2503.09722},
  year={2025}
}

@article{jing2025mixture,
  title={Mixture of Horizons in Action Chunking},
  author={Jing, Dong and Wang, Gang and Liu, Jiaqi and Tang, Weiliang and Sun, Zelong and Yao, Yunchao and Wei, Zhenyu and Liu, Yunhui and Lu, Zhiwu and Ding, Mingyu},
  journal={arXiv preprint arXiv:2511.19433},
  year={2025}
}

@inproceedings{liu2025bidirectional,
  title={Bidirectional decoding: Improving action chunking via guided test-time sampling},
  author={Liu, Yuejiang and Hamid, Jubayer and Xie, Annie and Lee, Yoonho and Du, Max and Finn, Chelsea},
  booktitle={International Conference on Learning Representations},
  volume={2025},
  pages={4594--4627},
  year={2025}
}

@article{park2025acg,
  title={ACG: Action Coherence Guidance for Flow-based VLA models},
  author={Park, Minho and Kim, Kinam and Hyung, Junha and Jang, Hyojin and Jin, Hoiyeong and Yun, Jooyeol and Lee, Hojoon and Choo, Jaegul},
  journal={arXiv preprint arXiv:2510.22201},
  year={2025}
}

@article{foster2024behavior,
  title={Is behavior cloning all you need? understanding horizon in imitation learning},
  author={Foster, Dylan J and Block, Adam and Misra, Dipendra},
  journal={Advances in Neural Information Processing Systems},
  volume={37},
  pages={120602--120666},
  year={2024}
}

@article{pan2025much,
  title={Much Ado About Noising: Dispelling the Myths of Generative Robotic Control},
  author={Pan, Chaoyi and Anantharaman, Giri and Huang, Nai-Chieh and Jin, Claire and Pfrommer, Daniel and Yuan, Chenyang and Permenter, Frank and Qu, Guannan and Boffi, Nicholas and Shi, Guanya and others},
  journal={arXiv preprint arXiv:2512.01809},
  year={2025}
}

@article{black2025training,
  title={Training-time action conditioning for efficient real-time chunking},
  author={Black, Kevin and Ren, Allen Z and Equi, Michael and Levine, Sergey},
  journal={arXiv preprint arXiv:2512.05964},
  year={2025}
}

@article{malhotra2025self,
  title={Self-guided action diffusion},
  author={Malhotra, Rhea and Liu, Yuejiang and Finn, Chelsea},
  journal={arXiv preprint arXiv:2508.12189},
  year={2025}
}

@article{black2026real,
  title={Real-time execution of action chunking flow policies},
  author={Black, Kevin and Galliker, Manuel and Levine, Sergey},
  journal={Advances in Neural Information Processing Systems},
  volume={38},
  pages={33383--33407},
  year={2026}
}

@article{zha2026lap,
  title={Lap: Language-action pre-training enables zero-shot cross-embodiment transfer},
  author={Zha, Lihan and Hancock, Asher J and Zhang, Mingtong and Yin, Tenny and Huang, Yixuan and Shah, Dhruv and Ren, Allen Z and Majumdar, Anirudha},
  journal={arXiv preprint arXiv:2602.10556},
  year={2026}
}

@article{susilaradeya2019extrinsic,
  title={Extrinsic and intrinsic dynamics in movement intermittency},
  author={Susilaradeya, Damar and Xu, Wei and Hall, Thomas M and Galan, Ferran and Alter, Kai and Jackson, Andrew},
  journal={Elife},
  volume={8},
  pages={e40145},
  year={2019},
  publisher={eLife Sciences Publications, Ltd}
}

@article{torne2025learning,
  title={Learning long-context diffusion policies via past-token prediction},
  author={Torne, Marcel and Tang, Andy and Liu, Yuejiang and Finn, Chelsea},
  journal={arXiv preprint arXiv:2505.09561},
  year={2025}
}

@article{wagenmaker2025posterior,
  title={Posterior Behavioral Cloning: Pretraining BC Policies for Efficient RL Finetuning},
  author={Wagenmaker, Andrew and Dong, Perry and Tsao, Raymond and Finn, Chelsea and Levine, Sergey},
  journal={arXiv preprint arXiv:2512.16911},
  year={2025}
}

@article{chen2026dream,
  title={DREAM-Chunk: Reactive Action Chunking with Latent World Model},
  author={Chen, Wenxi and Zhang, Kaidi and Lin, Chi and Zhang, Zhiyuan and She, Yu and Liu, Yuejiang and Yeh, Raymond A and Mou, Shaoshuai and Gu, Yan},
  journal={arXiv preprint arXiv:2606.18589},
  year={2026}
}

@article{weng2025temporal,
  title={Temporal action selection for action chunking},
  author={Weng, Yueyang and Zhang, Xiaopeng and Mu, Yongjin and Zhu, Yingcong and Li, Yanjie},
  journal={arXiv preprint arXiv:2511.04421},
  year={2025}
}

@article{torne2026mem,
  title={Mem: Multi-scale embodied memory for vision language action models},
  author={Torne, Marcel and Pertsch, Karl and Walke, Homer and Vedder, Kyle and Nair, Suraj and Ichter, Brian and Ren, Allen Z and Wang, Haohuan and Tang, Jiaming and Stachowicz, Kyle and others},
  journal={arXiv preprint arXiv:2603.03596},
  year={2026}
}

@article{mark2026bpp,
  title={Bpp: Long-context robot imitation learning by focusing on key history frames},
  author={Mark, Max Sobol and Liang, Jacky and Attarian, Maria and Fu, Chuyuan and Dwibedi, Debidatta and Shah, Dhruv and Kumar, Aviral},
  journal={arXiv preprint arXiv:2602.15010},
  year={2026}
}

@article{day2000voluntary,
  title={Voluntary modification of automatic arm movements evoked by motion of a visual target},
  author={Day, BL and Lyon, IN},
  journal={Experimental Brain Research},
  volume={130},
  number={2},
  pages={159--168},
  year={2000},
  publisher={Springer}
}

@article{krogh1994neural,
  title={Neural network ensembles, cross validation, and active learning},
  author={Krogh, Anders and Vedelsby, Jesper},
  journal={Advances in neural information processing systems},
  volume={7},
  year={1994}
}

@article{breiman2001random,
  title={Random forests},
  author={Breiman, Leo},
  journal={Machine learning},
  volume={45},
  number={1},
  pages={5--32},
  year={2001},
  publisher={Springer}
}

@article{ho1998random,
  title={The random subspace method for constructing decision forests},
  author={Ho, Tin Kam},
  journal={IEEE transactions on pattern analysis and machine intelligence},
  volume={20},
  number={8},
  pages={832--844},
  year={1998},
  publisher={Ieee}
}

@article{intelligence2025pi_,
  title={$\pi_{0.5}$: a Vision-Language-Action Model with Open-World Generalization},
  author={Intelligence, Physical and Black, Kevin and Brown, Noah and Darpinian, James and Dhabalia, Karan and Driess, Danny and Esmail, Adnan and Equi, Michael and Finn, Chelsea and Fusai, Niccolo and others},
  journal={arXiv preprint arXiv:2504.16054},
  year={2025}
}

@inproceedings{walke2023bridgedata,
  title={Bridgedata v2: A dataset for robot learning at scale},
  author={Walke, Homer Rich and Black, Kevin and Zhao, Tony Z and Vuong, Quan and Zheng, Chongyi and Hansen-Estruch, Philippe and He, Andre Wang and Myers, Vivek and Kim, Moo Jin and Du, Max and others},
  booktitle={Conference on Robot Learning},
  pages={1723--1736},
  year={2023},
  organization={PMLR}
}

@article{argall2009survey,
  title={A survey of robot learning from demonstration},
  author={Argall, Brenna D and Chernova, Sonia and Veloso, Manuela and Browning, Brett},
  journal={Robotics and autonomous systems},
  volume={57},
  number={5},
  pages={469--483},
  year={2009},
  publisher={Elsevier}
}

@article{bojarski2016end,
  title={End to end learning for self-driving cars},
  author={Bojarski, Mariusz},
  journal={arXiv preprint arXiv:1604.07316},
  year={2016}
}

@inproceedings{zhang2018deep,
  title={Deep imitation learning for complex manipulation tasks from virtual reality teleoperation},
  author={Zhang, Tianhao and McCarthy, Zoe and Jow, Owen and Lee, Dennis and Chen, Xi and Goldberg, Ken and Abbeel, Pieter},
  booktitle={2018 IEEE international conference on robotics and automation (ICRA)},
  pages={5628--5635},
  year={2018},
  organization={IEEE}
}

@inproceedings{rahmatizadeh2018vision,
  title={Vision-based multi-task manipulation for inexpensive robots using end-to-end learning from demonstration},
  author={Rahmatizadeh, Rouhollah and Abolghasemi, Pooya and B{\"o}l{\"o}ni, Ladislau and Levine, Sergey},
  booktitle={2018 IEEE international conference on robotics and automation (ICRA)},
  pages={3758--3765},
  year={2018},
  organization={IEEE}
}

@article{mandlekar2021matters,
  title={What matters in learning from offline human demonstrations for robot manipulation},
  author={Mandlekar, Ajay and Xu, Danfei and Wong, Josiah and Nasiriany, Soroush and Wang, Chen and Kulkarni, Rohun and Fei-Fei, Li and Savarese, Silvio and Zhu, Yuke and Mart{\'\i}n-Mart{\'\i}n, Roberto},
  journal={arXiv preprint arXiv:2108.03298},
  year={2021}
}

@article{shafiullah2022behavior,
  title={Behavior transformers: Cloning $ k $ modes with one stone},
  author={Shafiullah, Nur Muhammad and Cui, Zichen and Altanzaya, Ariuntuya Arty and Pinto, Lerrel},
  journal={Advances in neural information processing systems},
  volume={35},
  pages={22955--22968},
  year={2022}
}

@article{cui2022play,
  title={From play to policy: Conditional behavior generation from uncurated robot data},
  author={Cui, Zichen Jeff and Wang, Yibin and Shafiullah, Nur Muhammad Mahi and Pinto, Lerrel},
  journal={arXiv preprint arXiv:2210.10047},
  year={2022}
}

@article{gu2023rt,
  title={Rt-trajectory: Robotic task generalization via hindsight trajectory sketches},
  author={Gu, Jiayuan and Kirmani, Sean and Wohlhart, Paul and Lu, Yao and Arenas, Montserrat Gonzalez and Rao, Kanishka and Yu, Wenhao and Fu, Chuyuan and Gopalakrishnan, Keerthana and Xu, Zhuo and others},
  journal={arXiv preprint arXiv:2311.01977},
  year={2023}
}

@article{brohan2022rt,
  title={Rt-1: Robotics transformer for real-world control at scale},
  author={Brohan, Anthony and Brown, Noah and Carbajal, Justice and Chebotar, Yevgen and Dabis, Joseph and Finn, Chelsea and Gopalakrishnan, Keerthana and Hausman, Karol and Herzog, Alex and Hsu, Jasmine and others},
  journal={arXiv preprint arXiv:2212.06817},
  year={2022}
}

@article{khazatsky2024droid,
    title   = {DROID: A Large-Scale In-The-Wild Robot Manipulation Dataset},
    author  = {Alexander Khazatsky and Karl Pertsch and Suraj Nair and Ashwin Balakrishna and Sudeep Dasari and Siddharth Karamcheti and Soroush Nasiriany and Mohan Kumar Srirama and Lawrence Yunliang Chen and Kirsty Ellis and Peter David Fagan and Joey Hejna and Masha Itkina and Marion Lepert and Yecheng Jason Ma and Patrick Tree Miller and Jimmy Wu and Suneel Belkhale and Shivin Dass and Huy Ha and Arhan Jain and Abraham Lee and Youngwoon Lee and Marius Memmel and Sungjae Park and Ilija Radosavovic and Kaiyuan Wang and Albert Zhan and Kevin Black and Cheng Chi and Kyle Beltran Hatch and Shan Lin and Jingpei Lu and Jean Mercat and Abdul Rehman and Pannag R Sanketi and Archit Sharma and Cody Simpson and Quan Vuong and Homer Rich Walke and Blake Wulfe and Ted Xiao and Jonathan Heewon Yang and Arefeh Yavary and Tony Z. Zhao and Christopher Agia and Rohan Baijal and Mateo Guaman Castro and Daphne Chen and Qiuyu Chen and Trinity Chung and Jaimyn Drake and Ethan Paul Foster and Jensen Gao and Vitor Guizilini and David Antonio Herrera and Minho Heo and Kyle Hsu and Jiaheng Hu and Muhammad Zubair Irshad and Donovon Jackson and Charlotte Le and Yunshuang Li and Kevin Lin and Roy Lin and Zehan Ma and Abhiram Maddukuri and Suvir Mirchandani and Daniel Morton and Tony Nguyen and Abigail O'Neill and Rosario Scalise and Derick Seale and Victor Son and Stephen Tian and Emi Tran and Andrew E. Wang and Yilin Wu and Annie Xie and Jingyun Yang and Patrick Yin and Yunchu Zhang and Osbert Bastani and Glen Berseth and Jeannette Bohg and Ken Goldberg and Abhinav Gupta and Abhishek Gupta and Dinesh Jayaraman and Joseph J Lim and Jitendra Malik and Roberto Martín-Martín and Subramanian Ramamoorthy and Dorsa Sadigh and Shuran Song and Jiajun Wu and Michael C. Yip and Yuke Zhu and Thomas Kollar and Sergey Levine and Chelsea Finn},
    year    = {2024},
}

@inproceedings{robomimic2021,
  title={What Matters in Learning from Offline Human Demonstrations for Robot Manipulation},
  author={Ajay Mandlekar and Danfei Xu and Josiah Wong and Soroush Nasiriany and Chen Wang and Rohun Kulkarni and Li Fei-Fei and Silvio Savarese and Yuke Zhu and Roberto Mart\'{i}n-Mart\'{i}n},
  booktitle={arXiv preprint arXiv:2108.03298},
  year={2021}
}

@article{zhao2024aloha,
  title={Aloha unleashed: A simple recipe for robot dexterity},
  author={Zhao, Tony Z and Tompson, Jonathan and Driess, Danny and Florence, Pete and Ghasemipour, Kamyar and Finn, Chelsea and Wahid, Ayzaan},
  journal={arXiv preprint arXiv:2410.13126},
  year={2024}
}

@article{kim2024openvla,
  title={Openvla: An open-source vision-language-action model},
  author={Kim, Moo Jin and Pertsch, Karl and Karamcheti, Siddharth and Xiao, Ted and Balakrishna, Ashwin and Nair, Suraj and Rafailov, Rafael and Foster, Ethan and Lam, Grace and Sanketi, Pannag and others},
  journal={arXiv preprint arXiv:2406.09246},
  year={2024}
}

@inproceedings{ankile2024juicer,
  title={Juicer: Data-efficient imitation learning for robotic assembly},
  author={Ankile, Lars and Simeonov, Anthony and Shenfeld, Idan and Agrawal, Pulkit},
  booktitle={2024 IEEE/RSJ International Conference on Intelligent Robots and Systems (IROS)},
  pages={5096--5103},
  year={2024},
  organization={IEEE}
}

@inproceedings{sridhar2024nomad,
  title={Nomad: Goal masked diffusion policies for navigation and exploration},
  author={Sridhar, Ajay and Shah, Dhruv and Glossop, Catherine and Levine, Sergey},
  booktitle={2024 IEEE International Conference on Robotics and Automation (ICRA)},
  pages={63--70},
  year={2024},
  organization={IEEE}
}

@article{ze20243d,
  title={3d diffusion policy: Generalizable visuomotor policy learning via simple 3d representations},
  author={Ze, Yanjie and Zhang, Gu and Zhang, Kangning and Hu, Chenyuan and Wang, Muhan and Xu, Huazhe},
  journal={arXiv preprint arXiv:2403.03954},
  year={2024}
}

@article{bjorck2025gr00t,
  title={GR00T N1: An Open Foundation Model for Generalist Humanoid Robots},
  author={Bjorck, Johan and Casta{\~n}eda, Fernando and Cherniadev, Nikita and Da, Xingye and Ding, Runyu and Fan, Linxi and Fang, Yu and Fox, Dieter and Hu, Fengyuan and Huang, Spencer and others},
  journal={arXiv preprint arXiv:2503.14734},
  year={2025}
}

@article{janner2022planning,
  title={Planning with diffusion for flexible behavior synthesis},
  author={Janner, Michael and Du, Yilun and Tenenbaum, Joshua B and Levine, Sergey},
  journal={arXiv preprint arXiv:2205.09991},
  year={2022}
}

@article{liu2023libero,
  title={LIBERO: Benchmarking Knowledge Transfer for Lifelong Robot Learning},
  author={Liu, Bo and Zhu, Yifeng and Gao, Chongkai and Feng, Yihao and Liu, Qiang and Zhu, Yuke and Stone, Peter},
  journal={arXiv preprint arXiv:2306.03310},
  year={2023}
}

@article{geminirobotics,
  title   = {Gemini Robotics: Bringing AI into the Physical World},
  author  = {{Gemini Robotics Team}},
  journal = {arXiv preprint arXiv:2503.20020},
  year    = {2025},
  eprint  = {2503.20020},
  archivePrefix = {arXiv},
  primaryClass  = {cs.RO}
}

@article{lbm,
  title   = {A Careful Examination of Large Behavior Models for Multitask Dexterous Manipulation},
  author  = {{TRI LBM Team}},
  journal = {arXiv preprint arXiv:2507.05331},
  year    = {2025},
  eprint  = {2507.05331},
  archivePrefix = {arXiv},
  primaryClass  = {cs.RO},
  url     = {https://toyotaresearchinstitute.github.io/lbm1/}
}

@article{molmoact2,
  title   = {MolmoAct2: Action Reasoning Models for Real-world Deployment},
  author  = {Fang, Haoquan and Duan, Jiafei and Clay, Donovan and Wang, Sam and Liu, Shuo and Huang, Weikai and Fan, Xiang and Tsai, Wei-Chuan and Chen, Shirui and Wang, Yi Ru and Xing, Shanli and Cho, Jaemin and Park, Jae Sung and Eftekhar, Ainaz and Sushko, Peter and Farley, Karen and Wadhwa, Angad and Harrison, Cole and Han, Winson and Lee, Ying-Chun and VanderBilt, Eli and Hendrix, Rose and Ellawela, Suveen and Ngoo, Lucas and Chai, Joyce and Ren, Zhongzheng and Farhadi, Ali and Fox, Dieter and Krishna, Ranjay},
  journal = {arXiv preprint arXiv:2605.02881},
  year    = {2026},
  eprint  = {2605.02881},
  archivePrefix = {arXiv},
  primaryClass  = {cs.RO},
  url     = {https://allenai.org/blog/molmoact2}
}

@article{mimicvideo,
  title   = {{mimic-video}: Video-Action Models for Generalizable Robot Control Beyond VLAs},
  author  = {Pai, Jonas and Achenbach, Liam and Montesinos, Victoriano and Forrai, Benedek and Mees, Oier and Nava, Elvis},
  journal = {arXiv preprint arXiv:2512.15692},
  year    = {2025},
  eprint  = {2512.15692},
  archivePrefix = {arXiv},
  primaryClass  = {cs.RO},
  url     = {https://mimic-video.github.io/}
}

@inproceedings{alohaunleashed,
  title     = {ALOHA Unleashed: A Simple Recipe for Robot Dexterity},
  author    = {Zhao, Tony Z. and Tompson, Jonathan and Driess, Danny and Florence, Pete and Ghasemipour, Kamyar and Finn, Chelsea and Wahid, Ayzaan},
  booktitle = {Proceedings of The 8th Conference on Robot Learning},
  year      = {2025},
  series    = {Proceedings of Machine Learning Research},
  volume    = {270},
  publisher = {PMLR},
  eprint    = {2410.13126},
  archivePrefix = {arXiv},
  primaryClass  = {cs.RO},
  url       = {https://proceedings.mlr.press/v270/zhao25b.html}
}

@misc{zhu2025robosuitemodularsimulationframework,
      title={robosuite: A Modular Simulation Framework and Benchmark for Robot Learning}, 
      author={Yuke Zhu and Josiah Wong and Ajay Mandlekar and Roberto Martín-Martín and Abhishek Joshi and Kevin Lin and Abhiram Maddukuri and Soroush Nasiriany and Yifeng Zhu},
      year={2025},
      eprint={2009.12293},
      archivePrefix={arXiv},
      primaryClass={cs.RO},
      url={https://arxiv.org/abs/2009.12293}, 
}

@inproceedings{
dosovitskiy2021an,
title={An Image is Worth 16x16 Words: Transformers for Image Recognition at Scale},
author={Alexey Dosovitskiy and Lucas Beyer and Alexander Kolesnikov and Dirk Weissenborn and Xiaohua Zhai and Thomas Unterthiner and Mostafa Dehghani and Matthias Minderer and Georg Heigold and Sylvain Gelly and Jakob Uszkoreit and Neil Houlsby},
booktitle={International Conference on Learning Representations},
year={2021},
url={https://openreview.net/forum?id=YicbFdNTTy}
}

@misc{misra2020mishselfregularizednonmonotonic,
      title={Mish: A Self Regularized Non-Monotonic Activation Function}, 
      author={Diganta Misra},
      year={2020},
      eprint={1908.08681},
      archivePrefix={arXiv},
      primaryClass={cs.LG},
      url={https://arxiv.org/abs/1908.08681}, 
}

\newpage
\appendix
\section{Additional Experimental Details}\label{sec:app_exp_details}

In the following, we provide additional experimental details for all the simulations conducted in the paper. We begin with a discussion on the environments and datasets considered (Appendix \ref{apx:envs}), then describe the network architectures considered (Appendix \ref{apx:nets}), and, finally, we provide additional details for each simulation in the main paper (Appendix \ref{apx:more details}).

We mention that all simulations were carried out using one Nvidia H100 GPU and ten Nvidia A100 GPUs.

\newcommand{\unif}{\mathrm{unif}}
\subsection{Validation Error vs Validation Loss}\label{sec:val_error_explanation}
Here we expand on our choice of the validation error metric, $\Lval(\pi)$, that evaluates the action prediction error. In particular, $\Lval(\pi)$ measures how far the expected action predictions of policy $\pi$ are from the actual actions in the validation set. In the setting where $\pi$ is a diffusion policy, as we consider in this work, another natural choice of validation error would simply be the denoising diffusion loss on the validation set, that is:
\begin{align*}
    \Exp_{(\ba,o) \sim \unif(\mathfrak{D}_{\mathrm{val}})} \Exp_{\bm{\epsilon} \sim \cN(0,I)} \Exp_{t \sim \unif([0,1])}[\| \bm{\epsilon} - \bm{\epsilon}_{\theta}(\sqrt{\alpha_t} \ba + \sqrt{1-\alpha_t} \bm{\epsilon}, t; o) \|_2^2]
\end{align*}
where $\bm{\epsilon}_{\theta}$ is the policy's denoising network. The primary challenge with using this loss in our setting is that it does not allow us to isolate the contribution of each individual timestep. In particular, we want to evaluate the validation loss at step $i$ in the action chunk, we could compute this as
\begin{align*}
    \Exp_{(\ba,o) \sim \unif(\mathfrak{D}_{\mathrm{val}})} \Exp_{\bm{\epsilon} \sim \cN(0,I)} \Exp_{t \sim \unif([0,1])}[\| [\bm{\epsilon} - \bm{\epsilon}_{\theta}(\sqrt{\alpha_t} \ba + \sqrt{1-\alpha_t} \bm{\epsilon}, t; o)]_i \|_2^2],
\end{align*}
but this loss allows information to bleed over from other timesteps in the chunk to step $i$---as the network observes all of $\sqrt{\alpha_t} \ba + \sqrt{1-\alpha_t} \bm{\epsilon}$, it sees a noisy version of the action at steps $i+1$ and $i-1$, for example, so its prediction at step $i$ may depend strongly on this. As such, it is not possible to isolate the prediction of step $i$ in the chunk as we do, for example, in \Cref{fig:val_loss_libero}.

The second challenge with this validation loss is that it does not extend to the ensembling validation loss we consider in \Cref{sec:ac_ensembles}. Since we consider ensembled predictions across timesteps in the chunk, it is not obvious we can directly ensemble in the diffusion noise space---in particular, $\bm{\epsilon}_{\theta}(\sqrt{\alpha_t} \ba + \sqrt{1-\alpha_t} \bm{\epsilon}, t; o_h)$ and $\bm{\epsilon}_{\theta}(\sqrt{\alpha_t} \ba + \sqrt{1-\alpha_t} \bm{\epsilon}, t; o_{h-d})$ cannot be combined directly, since, while they may partially overlap, they ultimately predict different action chunks, and it is not obvious how action chunks can be combined in noise space. 

Given this, we opt instead to use the validation error that simply evaluates the actual action predictions. Note that this resolves the first issue since all action predictions begin from noise $\bm{\epsilon} \sim \cN(0,I)$, so oracle information from one step in the chunk cannot affect the prediction at another step in the chunk (since this has access to no oracle information). In addition, it also resolves the second issue since we can easily combine a given step of the chunk in action space. While in principle the validation error we consider may not incorporate multimodality effectively, we found that in practice the diffusion policies we trained were not multimodal---for a given observation, $o$, the predictions tend to be unimodal---so this did not prove an issue in practice.

\subsection{Description of Environments and Datasets}\label{apx:envs}

\texttt{Robomimic} \citep{robomimic2021} and \texttt{Libero} \citep{liu2023libero} are two widely adopted benchmarks built on top of the \texttt{Robosuite} project \citep{zhu2025robosuitemodularsimulationframework}. We consider tasks \texttt{square}, \texttt{can}, \texttt{transport} and \texttt{tool hang} for \texttt{Robomimic}, and all the 90 tasks in the \texttt{Libero 90} suite for \texttt{Libero}.

\paragraph{{Robomimic}.}

For each \texttt{Robomimic} task, we considered the PH 
(Proficient-Human) dataset, consisting of $N=200$ successful trajectories collected by a ``single, experienced teleoperator'' \citep{robomimic2021}, at a control frequency of 20Hz.
Actions for all these tasks except \texttt{transport} are 7-dimensional vectors where ``the first 3 coordinates are the desired translation from
the current end effector position, the next 3 coordinates encoder the desired delta rotation from the current end effector rotation, and the final coordinate controls the opening and closing of the gripper fingers.'' For \texttt{transport}, since there are two robot arms, then the action space is 14-dimensional.
Regarding observations, we avoid using images, but consider the full low-dimensional object observations provided by the \texttt{Robomimic} suite, made of proprioception observations (9-dimensional per arm, ``consisting of the end effector position (3-dim), quaternion (4-dim), and gripper
finger positions (2-dim)''), and ground-truth object states (with varying dimensionality depending on the task, containing in general the absolute position of the object/s and their relative position wrt the robot end effector). See Appendix E of \citet{robomimic2021} for more details.

\paragraph{{Libero}.}

For each \texttt{Libero 90} task, we are given 50 successful demonstrations from a human demonstrator. All 90 tasks share the same action space, which coincides with the action space of all the considered \texttt{Robomimic} tasks (i.e., a 7-dimensional vector encoding the delta translation/rotation of the end effector pose and the opening/closing of the gripper). Regarding the observation space, things are rather different from \texttt{Robomimic}, as we do not consider ground-truth object states, but image observations. So, our observation is made of proprioception observations (9-dimensional per arm, consisting of two components for the gripper finger positions, and seven components for the joint orientations of the arm), and image observations (that is, two (128,128,3)-dimensional images, one from the front-view camera and one from the wrist-mounted camera).

\subsection{Network Architecture}\label{apx:nets}

\begin{table}[t!]
    \centering
    \resizebox{\textwidth}{!}{%
\begin{tabular}{ccc} \toprule
    \textbf{Hyperparameter} & \texttt{Libero} & \texttt{Robomimic} \\ \midrule
    denoising net architecture & MLP & MLP \\
    denoising net size & $[4096,4096,4096,4096,4096]$ & $[3000,3000,3000,3000,3000]$ \\
    conditioning net architecture & ViT & MLP \\
    conditioning net size & 4 heads, depth 1, spatial embedding 128 & $[1024,1024,64]$ \\
    multi-task policy & yes & no \\
\bottomrule
\end{tabular}%
}
    \caption{High-level net architecture}
    \label{tab:nets_lib_rob}
\end{table}

For all the experiments, we trained diffusion policies \citep{chi2023diffusionpolicy} to fit the expert demonstrations, with some differences between \texttt{Robomimic} and \texttt{Libero} in order to handle the different types of observations. From a high-level perspective, we used conditional DDPMs \cite{ho2020ddpm} to fit (sequences of) actions given observations, where we use a multi-layer perceptron (MLP) for the denoising network, and either an MLP (\texttt{Robomimic}) or a Vision Transformer (ViT \citep{dosovitskiy2021an}, \texttt{Libero}) to encode the observation features. For \texttt{Robomimic}, we always train single-task policies (i.e., different policies for each task), where the denoising net is made of 5 layers of 3000 neurons each, while the MLP to encode observations is made of three layers of 1024, 1024, and 64 neurons. Instead, for \texttt{Libero}, we train multi-task policies, i.e., a single network to predict actions for all the 90 tasks. We accomplish this by including into the observation a 90-dim one-hot encoding of the task. The denoising net for \texttt{Libero} is made of 5 layers of 4096 neurons each, while we encode observations using a ViT encoder trained along with the network using a patch size of 8, 4 heads, and a depth
of 1. The output embedding dimension will be 128. The two images are then concatenated before being given in input to the (denoising net) MLP. Before passing the
images in input to the ViT encoder, we perform some random shift (translation) data augmentation for images: we randomly move each image 4 pixels
horizontally and vertically. See Table \ref{tab:nets_lib_rob} for a summary of this description.

\paragraph{Hyperparameters.}
The hyperparameters we adopt for training/optimization and for the DDPM are the same for both \texttt{Robomimic} and \texttt{Libero}, and are reported in Table \ref{tab:hyperparameters}. More specifically, we use 20 denoising steps and adopt a 32-dimensional embedding for the time step (we
use a sinusoidal positional embedding, as standard in DDPMs). The activation function we use is Mish \citep{misra2020mishselfregularizednonmonotonic}, and we add residual connections in the denoising MLP nets. We train our DDPMs for 3000 epochs with AdamW \citep{loshchilov2019adamw} with learning
rate 1e-4 and weight decay 1e-6. We adopt a cosine annealing scheduler for the
learning rate \citep{loshchilov2017sgdr} with 100 warmup steps and minimum and maximum learning rates of
1e-5 and 1e-4, with 3000 first cycle steps. We use a batch size of 256. We also use an empirical moving average (EMA) of the model weights, which starts at epoch 20, and is updated every 10 epochs with decay 0.995. It is worthy mentioning that, while for \texttt{Robomimic} one epoch corresponds to a number of batches that spans the whole training dataset, for \texttt{Libero}, where we have much more data as we train multi-task policies, one epoch corresponds to a randomly chosen subset of all the training samples, computed as $300/N$ (where $N$ is the total number of expert trajectories). Simply put, we are merely rescaling the true number of gradient iterations to make it more comparable to \texttt{Robomimic}, where we have less expert trajectories.

\begin{table}[t!]
    \centering
\begin{tabular}{cc} \toprule
    \textbf{Hyperparameter} & Value \\ \midrule
    activation function & Mish \\
    use residual connections & True \\\midrule
    batch size & 256 \\
    use EMA & True \\
    EMA decay & 0.995 \\
    optimizer & AdamW \\
    lr & 1e-4 \\
    weight decay & 1e-6 \\
    number of epochs & 3000 \\
    \midrule
    time embedding & sinusoidal positional \\
    time embedding size & 32 \\
    denoising steps & 20 \\ \bottomrule
\end{tabular}
    \caption{Hyperparameters}
    \label{tab:hyperparameters}
\end{table}

\paragraph{Padding and normalization.} 

We perform padding at the end of expert's trajectories every time that we train with chunk sizes $k>1$ or pure delayed policies with delay $d>0$, by replicating the
last action for $k-1$ times.
Before training, we normalize the expert trajectories as it is best practice for DDPMs.
In particular, in both \texttt{Robomimic} and \texttt{Libero}, we normalized the low-dimensional components of
the state/observation, and the action with a min-max normalization to $[-1,+1]$: given the
(300 or 50) expert's trajectories, we computed for each state and action
component, the min and max values, and then we applied min-max normalization:
$s_{\text{norm}}\coloneqq 2[(s-s_{\min}) / (s_{\max}-s_{\min}+1e-6)-0.5]$,
$a_{\text{norm}}\coloneqq 2[(a-a_{\min}) / (a_{\max}-a_{\min}+1e-6)-0.5]$.
Instead, \texttt{Libero} image observations, having discrete values in
$\{0,1,\dotsc,255\}$, have been normalized through: $i_{\text{norm}}\coloneqq i/
255.0 - 0.5$.

\subsection{Additional Details for each Simulation}\label{apx:more details}

Here we provide additional simulation-specific details for all the experiments described in the main paper, in particular for Figs. \ref{fig:val_loss_libero}-\ref{fig:robomimic_val} and Tables \ref{table:final_success}-\ref{table:final_success_explicit_ensemble}. First, we describe the evaluation protocol adopted.

\subsubsection{Evaluation Protocol}

For every \texttt{Robomimic} task, the success rate of each policy seed is assessed based on 250 rollouts played for a maximum number of timesteps of 300 for \texttt{square}, 250 for \texttt{can}, 650 for both \texttt{transport} and \texttt{tool hang}, while for each \texttt{Libero} task we consider 40 rollouts played for a maximum number of timesteps of 400. Moreover, in \texttt{Robomimic} the mean action used for the computation of the validation loss is the average of 50 action samples, while in \texttt{Libero} we used 20 action samples. As a last note, we mention that throughout all the paper, for delayed policies induced by some chunked (e.g., $\pidelay^{k}[\pihat_{20}]$), since they are not well-defined in the first timesteps of a simulation, we initialize the simulation by playing a single chunk of $\pihat_{20}^{k-1}$.

\subsubsection{\texttt{Libero} Validation Loss (\Cref{fig:val_loss_libero})}\label{apx:fig1}

For this simulation, we trained three $\pihat_{20}$ policies on a subset of all the \texttt{Libero} data. Specifically, instead of using all the 50 expert trajectories for each of the 90 tasks (overall $50*90=4500$), we used for training only 25 expert trajectories for each task (overall $25*90=2250$), and kept the remaining data for validation. Note that both policies have been trained on the \emph{same} data, but using different random seeds for training, which affect the weight initialization, the random batches selection, etc.. For validation, for every delay $k\in\{0\}\cup[19]$, for every timestep $h\ge 20$ and validation trajectory $i$,\footnote{Note that we skip the first 20 timesteps of each validation trajectory to make a fair comparison with higher delays. Indeed, while we can assess the validation error for $k=0$, we are not able to do so for, e.g., $k=15$ when $h<15$.} we sampled 20 actions from $\pidelay^{k}[\pihat_{20}](\cdot|s^i_{h-k})$, averaged them obtaining $\bar{a}$, and then computed the squared error between $\bar{a}$ and the expert's action $a^i_h$, for all coordinates except the gripper component: $\|\bar{a} - a^i_h\|_2^2$ (note that both actions are \emph{normalized} during this simulation). We repeated this procedure for each of the 90 tasks, for each of the 25 validation trajectories of each task, and for each timestep of each trajectory (whose length depends on the specific demonstration and task), and finally averaged over all these values (note that this procedure makes tasks with longer horizon more representative, compliant to the training procedure). After having obtained these 20 values for the first policy seed, we repeated the same procedure for the two other policy seeds. Finally, we compute the mean and standard error of the mean (SEM, standard deviation divided by $\sqrt{3}$ in this case) for these three 20-dim arrays, obtaining the Delay($n$) plot of Fig. \ref{fig:val_loss_libero}. Instead, the AC($n$) plot has been obtained by replacing the three 20-dim arrays with their cumulative means, and then computing the mean and SEM similarly as for Delay($n$).

\subsubsection{\texttt{Libero} Success Rate (\Cref{fig:success_libero})}\label{apx:fig2}

We trained three $\pihat_{20}$ policies on all the \texttt{Libero} data with different random seeds, rolled them out either as $\pidelay^{k}[\pihat_{20}]$ or as $\pihat_{20}^{k'}$ for $k\in\{0,2,5,8,12,15,19\}$ and $k'\in\{3,6,9,13,16,20\}$ (recall that $\pidelay^{0}[\pihat_{20}]$ coincides with $\pihat_{20}^{1}$), and finally computed the mean/SEM over these three seeds (so neglecting the variability due to using a finite number of rollouts per task).

\subsubsection{\texttt{Libero} Success Rate vs. Validation Loss (\Cref{fig:val_vs_success_libero})}\label{apx:fig3}

For this simulation, we used the same validation values as for Fig. \ref{fig:val_loss_libero} (see Appendix \ref{apx:fig1}), obtained by averaging over the three policy seeds but not over the tasks. Instead, the success rates are obtained as those in Table \ref{table:final_success} (see Appendix \ref{apx:table1}), i.e., by training fifteen policies $\pihat_{20}$ on all the \texttt{Libero} data with different random seeds, rolling them out either as $\pihat_{20}^{1}$ or as $\pihat_{20}^{10}$ (we do 40 rollouts per task), and finally averaging the success rate over the fifteen policy seeds. The line is computed as the least square solution over all the plotted points.

\subsubsection{\texttt{Libero} With and Without Action Chunking (\Cref{fig:libero_repr_learn})}\label{apx:fig4}

For $\pihat_{20}^1$ and $\pidelay^{5}[\pihat_{20}]$, we trained three $\pihat_{20}$ policies on all the \texttt{Libero} data with three different random seeds. Instead, for $\pihat_{1}^1$ we actually trained three $\pihat_{1}$ policies, while for $\pidelay^{5}$ we also trained three different policies of this kind. Then evaluation has been carried out as usual, with 40 rollouts per task, then averaging the success rate among all tasks. The SEM reported in Fig. \ref{fig:libero_repr_learn} corresponds to the three seeds only. As a last note on this experiment, we would like to mention that, for $\pidelay^{5}$, differently from the induced one $\pidelay^{5}[\pihat_{20}]$, we cannot initialize it with a chunk; thus, what we did was just to play the ``wrong'' action for the first five timesteps, that is, for timestep $0\le h<5$, we simply gave the first observation of the simulation $o_0$ in input to policy $\pidelay^{5}$, and played the outputted action.

\subsubsection{\texttt{Robomimic} Success Rate (\Cref{fig:success_robomimic})}\label{apx:fig5}

For each of the four \texttt{Robomimic} tasks considered, we trained ten $\pihat_{20}$ using different random seeds, rolled them out in the corresponding environment either as $\pidelay^{k}[\pihat_{20}]$ or as $\pihat_{20}^{k'}$ for $k\in\{0,2,4,7,10,13,16\}$ and $k'\in\{5,10,15,20\}$, and then computed the success rates. The values reported in Fig. \ref{fig:success_robomimic} correspond to the mean and SEM over these forty success rates.

\subsubsection{\texttt{Robomimic} Validation Loss (\Cref{fig:robomimic_val})}\label{apx:fig6}

The calculation of the validation loss for AC($n$) and Delay($n$) is analogous to those in Appendix \ref{apx:fig1}, with the difference that we use ten policy seeds trained on a subset of 30 demonstration trajectories per task (the remaining 170 expert trajectories are used for validation) and compute means over 50 sample actions. For AC($n$)-TE, calculations are analogous. Instead, for Delay($n$)-Ens, we considered three subsets of five policy seeds each from this group of ten policy seeds; for each group of seeds, our action predictions are the average of mean actions over the considered five seeds. Then, calculation of mean and SEM follows analogously using these three values.

\subsubsection{General Results (Tables \ref{table:final_success}-\ref{table:final_success_explicit_ensemble})}\label{apx:table1}

We trained fifteen $\pihat_{20}$ policies on all the \texttt{Libero} data (one multitask policy for each libero suite) with different random seeds, and then rolled out each policy seed as AC$(n)$ for $n\in\{1,10\}$, Delay$(n)$ for $n\in\{6\}$, AC$(n)$-RDE for $n\in\{10\}$, and AC$(n)$-TE for $n\in\{10\}$ (to be precise, we run all the fifteen seeds only for \texttt{Libero-90}, while ``only'' ten seeds for other suites, as success rate was roughly already converged). For the ensembles AC$(n)$-Ens, Delay$(n)$-Ens and AC$(n)$-RDE-Ens, we grouped the fifteen policy seeds into three disjoint groups of five seeds each, and then rolled out each ensemble and averaged the results. We used $n=10$ for AC$(n)$-Ens and AC$(n)$-RDE-Ens, and $n=5$ for Delay$(n)$-Ens. We did so for both aggregation approaches described in Section \ref{sec:ac_ensembles}.

Regarding \texttt{Robomimic}, things are similar: we trained fifteen $\pihat_{20}$ policies on \emph{each} \texttt{Robomimic} task using all the data, rolled them out, then average results. The difference is that we considered multiple values for $n$, specifically: for AC$(n)$ we used $n\in\{1,5,10\}$, for D$(n)$ we used $n\in\{3,5\}$, for AC$(n)$-Ens we used $n\in\{5,10\}$ for both aggregation methods, for Delay$(n)$-Ens we used $n\in\{3,5\}$ for both aggregation methods, while for AC$(n)$-RDE-Ens we used only $n=10$. Complete results are reported in Tables \ref{tab:complete_robomimic1}-\ref{tab:complete_robomimic4}.

\begin{table}[t!]
    \centering
    \begin{tabular}{lll}
    \hline
     Algorithm          & $\mathcal{J}$         & t succ   \\
    \hline
     Markovian           & 69.0 ± 0.8 & 157 ± 1  \\
     AC$(5)$           & 83.3 ± 0.8 & 148 ± 0  \\
     AC$(10)$          & 85.4 ± 0.5 & 144 ± 0  \\
     Delay$(5)$            & 76.8 ± 0.8 & 152 ± 1  \\
     Delay$(3)$            & 80.8 ± 0.6 & 150 ± 0  \\
     AC$(10)$-RDE      & 82.4 ± 0.6 & 146 ± 0  \\
     AC$(10)$-TE & 80.6 ± 0.5 & 147 ± 0  \\
     \hline
     AC$(5)$-Ens mean       & 87.5 ± 0.5 & 145 ± 1  \\
     AC$(5)$-Ens random     & 88.0 ± 0.3 & 146 ± 1  \\
     AC$(10)$-Ens mean      & 88.3 ± 0.3 & 143 ± 1  \\
     AC$(10)$-Ens random    & 87.1 ± 0.2 & 143 ± 0  \\
     Delay$(3)$-Ens mean        & 87.5 ± 0.2 & 148 ± 1  \\
     Delay$(3)$-Ens random      & 87.7 ± 1.1 & 147 ± 1  \\
     Delay$(5)$-Ens mean        & 85.9 ± 0.7 & 147 ± 1  \\
     Delay$(5)$-Ens random      & 84.8 ± 1.3 & 148 ± 0  \\
    \hline
    \end{tabular}
    \caption{Complete success rate results for \texttt{Robomimic square PH}.}
    \label{tab:complete_robomimic1}
\end{table}

\begin{table}[t!]
    \centering
    \begin{tabular}{lll}
    \hline
     Algorithm          & $\mathcal{J}$         & t succ   \\
    \hline
     Markovian           & 83.7 ± 0.4 & 116 ± 0  \\
     AC$(5)$           & 95.2 ± 0.4 & 112 ± 0  \\
     AC$(10)$          & 97.2 ± 0.3 & 110 ± 0  \\
     Delay$(5)$            & 93.1 ± 0.4 & 112 ± 0  \\
     Delay$(3)$            & 93.5 ± 0.4 & 112 ± 0  \\
     AC$(10)$-RDE      & 96.7 ± 0.2 & 110 ± 0  \\
     AC$(10)$-TE & 96.2 ± 0.3 & 113 ± 0  \\
     \hline
     AC$(5)$-Ens mean       & 97.7 ± 0.1 & 109 ± 0  \\
     AC$(5)$-Ens random     & 96.9 ± 0.2 & 111 ± 0  \\
     AC$(10)$-Ens mean      & 98.5 ± 0.1 & 109 ± 0  \\
     AC$(10)$-Ens random    & 97.6 ± 0.2 & 110 ± 0  \\
     Delay$(5)$-Ens random      & 96.9 ± 0.3 & 110 ± 0  \\
     Delay$(5)$-Ens mean        & 97.5 ± 0.1 & 109 ± 0  \\
     Delay$(3)$-Ens random      & 97.7 ± 0.4 & 110 ± 0  \\
     Delay$(3)$-Ens mean        & 97.5 ± 0.3 & 110 ± 0  \\
    \hline
    \end{tabular}
    \caption{Complete success rate results for \texttt{Robomimic can PH}.}
    \label{tab:complete_robomimic2}
\end{table}

\begin{table}[t!]
    \centering
    \begin{tabular}{lll}
    \hline
     Algorithm          & $\mathcal{J}$         & t succ   \\
    \hline
     Markovian           & 3.3 ± 0.3  & 512 ± 4  \\
     AC$(5)$           & 7.5 ± 0.5  & 483 ± 4  \\
     AC$(10)$          & 12.6 ± 0.5 & 472 ± 3  \\
     Delay$(5)$            & 7.9 ± 0.5  & 479 ± 4  \\
     Delay$(3)$            & 6.8 ± 0.5  & 493 ± 3  \\
     AC$(10)$-RDE      & 12.1 ± 0.5 & 477 ± 3  \\
     AC$(10)$-TE & 12.2 ± 0.6 & 483 ± 2  \\
     \hline
     AC$(5)$-Ens mean       & 39.1 ± 1.3 & 466 ± 3  \\
     AC$(5)$-Ens random     & 26.3 ± 1.0 & 467 ± 2  \\
     AC$(10)$-Ens mean      & 41.5 ± 0.4 & 461 ± 1  \\
     AC$(10)$-Ens random    & 22.8 ± 0.3 & 467 ± 3  \\
     Delay$(5)$-Ens random      & 35.6 ± 0.8 & 470 ± 3  \\
     Delay$(5)$-Ens mean        & 37.3 ± 0.2 & 469 ± 4  \\
     Delay$(3)$-Ens random      & 38.9 ± 2.1 & 468 ± 3  \\
     Delay$(3)$-Ens mean        & 37.6 ± 1.2 & 471 ± 1  \\
    \hline
    \end{tabular}
    \caption{Complete success rate results for \texttt{Robomimic transport PH}.}
    \label{tab:complete_robomimic3}
\end{table}

\begin{table}[t!]
    \centering
    \begin{tabular}{lll}
    \hline
     Algorithm          & $\mathcal{J}$         & t succ   \\
    \hline
     Markovian           & 28.0 ± 0.8 & 494 ± 3  \\
     AC$(5)$           & 65.2 ± 1.1 & 455 ± 2  \\
     AC$(10)$          & 75.2 ± 0.5 & 441 ± 2  \\
     Delay$(5)$            & 36.1 ± 0.7 & 477 ± 3  \\
     Delay$(3)$            & 51.6 ± 0.9 & 468 ± 2  \\
     AC$(10)$-RDE      & 71.8 ± 0.8 & 449 ± 1  \\
     AC$(10)$-TE & 42.2 ± 1.0 & 462 ± 2  \\
     \hline
     AC$(5)$-Ens mean       & 71.9 ± 2.3 & 441 ± 5  \\
     AC$(5)$-Ens random     & 87.6 ± 1.6 & 436 ± 1  \\
     AC$(10)$-Ens mean      & 74.9 ± 1.8 & 434 ± 1  \\
     AC$(10)$-Ens random    & 86.7 ± 1.1 & 439 ± 2  \\
     Delay$(5)$-Ens random      & 72.7 ± 1.2 & 447 ± 2  \\
     Delay$(5)$-Ens mean        & 46.8 ± 1.1 & 455 ± 4  \\
     Delay$(3)$-Ens random      & 84.9 ± 1.2 & 439 ± 1  \\
     Delay$(3)$-Ens mean        & 62.7 ± 1.7 & 447 ± 3  \\
    \hline
    \end{tabular}
    \caption{Complete success rate results for \texttt{Robomimic tool hang PH}.}
    \label{tab:complete_robomimic4}
\end{table}

\subsection{Additional Details for Real-World Experiments}\label{apx:details real world}

We trained our diffusion policies using the same hyperparameters as for \texttt{Libero}, with the differences that the MLP has five layers with 2000 neurons each, images in input have dimension $120\times 120\times 3$, and we use 15 denoising steps instead of 20 (to speed up inference for real-time evaluation of the policies). Note that, as action space, we use delta control for training and playing our policies (where delta is referred to the current state, and not the initial state of the chunk). To achieve this, we convert the actions from their original version as absolute joints, to delta joints (simply by computing the difference with the current joint absolute position), train on this space, and then at inference time compute the sum between the predicted action and the current position, and give this new (absolute joint) action to the low-level controller. We mention that we use a Robotiq gripper.

For evaluation, for each considered method, we rolled out 50 trajectories all with the same initial states (modulo some accuracy differences in re-setting the objects to the initial positions). Results are reported in Figure \ref{fig:real}. Note that, as mentioned in the main text, every policy is actually reported with a missing delay of 1 (which we denote with Delay$(2)$). So, e.g., AC($10$) should be AC($10$)-Delay$(2)$, with which we mean the policy induced from $\pihat_{20}$ which plays action chunks with size 10 using a delay of 1, i.e., which plays $a_{h:h+10}|s_{h-1}$. In a similar manner, AC($10$)-Delay$(2)$-RDE represents the policy that, at each timestep $h$, samples at random an integer $i\in[10]$ (crucially, $\{0\}\notin[10]$), and playes $a_h|s_{h-i}$.
Note that we decided to apply a delay of 1 to each of these policies in order to run them in real time: indeed, by avoiding to condition on the current state, the inference of the chunk can be executed asynchronous from at least one timestep in the past, giving the policy server enough time to make inference without having to wait for a chunk to be ready.

Crucially, these real-world results \emph{replicate our simulation results}. Note that Delay($7$) is better than Delay($2$), but still not as good as AC($10$)-Delay$(2)$. Instead, the method which involves an implicit ensemble of delayed policies, i.e., AC($10$)-Delay$(2)$-RDE, matches (or even slightly outperforms) AC($10$)-Delay$(2)$, as expected.

As a last note, we mention that no tuning of the chunk sizes nor the delay have been made. We just chosen a delay of 1 based on the inference time of our diffusion policies, and then committed in advance to the aforementioned chunk sizes (1 and 10) and delays (5). We then collected 50 rollouts per method and reported the results.

\section{Additional Simulations}

In this appendix, we provide results for additional simulations we conducted. In particular, we begin with a simulation alternative to AC$(n)$-RDE to show that we do not need temporal consistency in \texttt{Robomimic PH} and \texttt{Libero-90} (Appendix \ref{apx: prod marg}), then we show that action chunking and delayed policies are more sample efficient than Markovian policies in \texttt{Libero-90} (Appendix \ref{apx: sample eff}). Next, we show that the dynamics of both \texttt{Robomimic PH} and \texttt{Libero} are smooth, supporting our assumption for Theorem \ref{thm:delayed_upper_bound} (Appendix \ref{apx: smooth dynamics}).
Finally, we report additional results on \texttt{Robomimic MH} (Multi-Human datasets \cite{robomimic2021}) and other \texttt{Libero} suites (Appendix \ref{apx: robomimic mh}), and results with using the OpenPi policy \cite{black2024pi0} on \texttt{Libero} suites other than 90 (Appendix \ref{apx: openpi}).

\subsection{Action Chunking without Temporal Consistency}\label{apx: prod marg}

In Section \ref{sec:ac_ensembles}, we showed that AC$(n)$-RDE matches the performance of AC$(n)$, implying that, i.a., the temporal consistency property of action chunking (i.e., playing a \emph{joint} distribution over a sequence of actions) is not required. Here, we provide an alternative simulation to show this. Specifically, we trained fifteen $\pihat_{20}$ policies in both \texttt{Libero} and each individual \texttt{Robomimic} task, and then we compared the average performance of $\pihat_{20}^{10}$ with that of the policy that plays in sequence $\pidelay^0[\pihat_{20}]$, $\pidelay^1[\pihat_{20}]$, $\dotsc$, $\pidelay^9[\pihat_{20}]$, and then restarts from $\pidelay^0[\pihat_{20}]$, $\pidelay^1[\pihat_{20}]$, $\dotsc$ (we call the latter policy AC$(10)$-Ordered). Intuitively, this latter policy plays delayed policies in the same order as the action chunking $\pihat_{20}^{10}$ policy does; the only difference is that $\pihat_{20}^{10}$ is temporally consistent, i.e., tries to match the joint action distribution exhibited by the expert, while AC$(10)$-Ordered plays each action in the chunk according to its marginal probability. As shown in Table \ref{tab:prodmarg}, also this strategy matches the performance of AC$(10)$ across all the considered simulation tasks.

\begin{table}[h!]
    \centering
    \begin{tabular}{lll}
    \hline
     Task          & AC$(10)$       &  AC$(10)$-Ordered  \\
    \hline
     \texttt{Libero 90}           & 89.2 ± 0.2 & 93.5 ± 0.1  \\
     \texttt{Robomimic square PH}           & 85.4 ± 0.5 & 85.2 ± 0.5  \\
     \texttt{Robomimic can PH}           & 97.2 ± 0.3 & 96.8 ± 0.3  \\
     \texttt{Robomimic transport PH}           & 12.6 ± 0.5 & 12.5 ± 0.5  \\
     \texttt{Robomimic tool hang PH}           & 75.2 ± 0.5 & 75.2 ± 0.6  \\
    \hline
    \end{tabular}
    \caption{Another simulation about temporal consistency.}
    \label{tab:prodmarg}
\end{table}

\subsection{Sample Efficiency in \texttt{Libero}}\label{apx: sample eff}

Here we aim to show explicitly that action chunking and delayed policies are more sample efficient than Markovian policies. To this aim, we focused on \texttt{Libero}, and chose three dataset sizes: 10, 25, 50. We then splitted the expert demonstration data for each task (overall 50), into two disjoint subsets of 10 trajectories, two disjoing subsets of 25 trajectories, and then all the 50 trajectories. We then trained two multitask $\pihat_{20}$ policies on each of these five datasets, rolled them out for 100 rollouts in each task, computed the success rates and averaged these values per dataset size, obtaining the plot in Fig. \ref{fig:sample_eff}. It is clear from such figure that the gap between AC$(10)$ and AC$(1)$ decreases as we increase the number of expert demonstrations, in particular moving from 0.39 for $N=10$, to 0.31 for $N=25$, up to 0.20 for $N=50$. A similar pattern is observed also for Delay$(6)$, further supporting the insights of our theoretical analysis.

\begin{figure}[h!]
    \centering
    \includegraphics[width=0.5\linewidth]{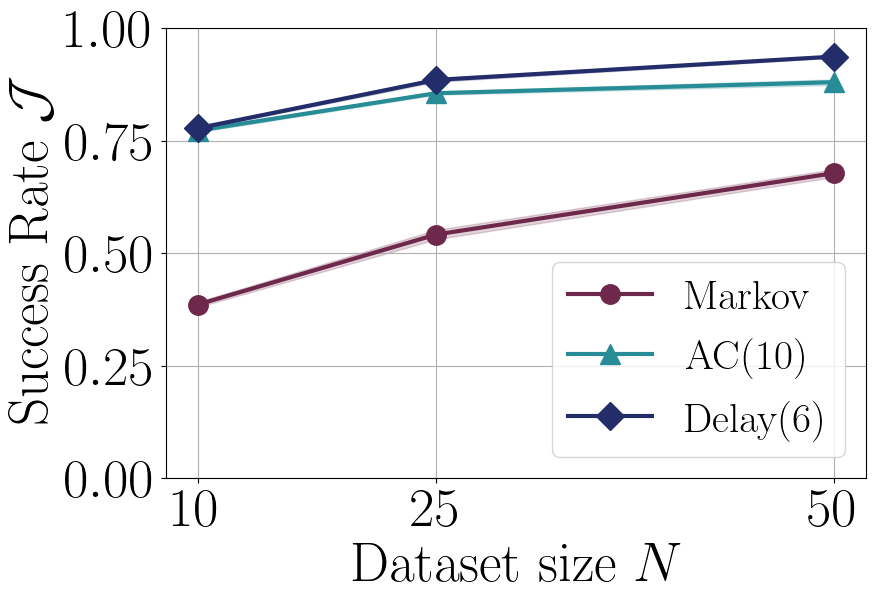}
    \caption{Comparison of Markovian, action chunking and delayed policies for increasing dataset sizes.}
    \label{fig:sample_eff}
\end{figure}

\subsection{\texttt{Robomimic} Dynamics are Smooth}\label{apx: smooth dynamics}

In Section \ref{sec:hypoth_horizon}, we made the theoretical assumption that the dynamics is smooth. This would definitely be true at least for the action translation components, if the underlying low-level controller was perfect, and could deploy the map $s_{h+1}=s_h+a_h$. In Figs. \ref{fig:compx}-\ref{fig:compz}, we took the first expert trajectory in the non-normalized \texttt{Robomimic square} dataset, and plot the sequence of next state end effector position components along with the trivial predictions $s_h+a_h$, $s_{h-3}+a_{h-3}$ and $s_{h-6}+a_{h-6}$ as a function of the timestep. It is clear that, with rare exceptions mostly due to interaction with objects, the smooth map $s_{h-3}+a_{h-3}$ is a good predictor of this environment dynamics.

Intuitively, it makes sense that the true dynamics for delta control is something like $s_{h+1}=s_{h-3}+a_{h-3}$ instead of $s_{h+1}=s_h+a_h$. Indeed, note that the low-level controller in \texttt{Libero} and \texttt{Robomimic} operates at a 500Hz frequency, while our (high-level) policies takes actions with frequency 20Hz. Thus, the low-level controller has $500/20=25$ steps of 0.002 sec to try to get the robot to the desired/commanded position $s_{h-3}+a_{h-3}$ at time step $h-3$; clearly, this amount of time might not be enough, and so the robot arm is able to get to position $s_{h-3}+a_{h-3}$ just some timesteps later, i.e., at $s_{h+1}$ (of course as long as the subsequent actions $a_{h-2},a_{h-1},a_{h}$ do not refer to completely different desired positions, which is not the case as humans cannot take completely uncorrelated actions at every $h$ (i.e., every $1/20\text{Hz}=0.05$sec), as they are not that fast).

\begin{figure*}[h!]
    \centering

    \begin{minipage}[t]{0.32\textwidth}
        \centering
        \includegraphics[width=\linewidth]{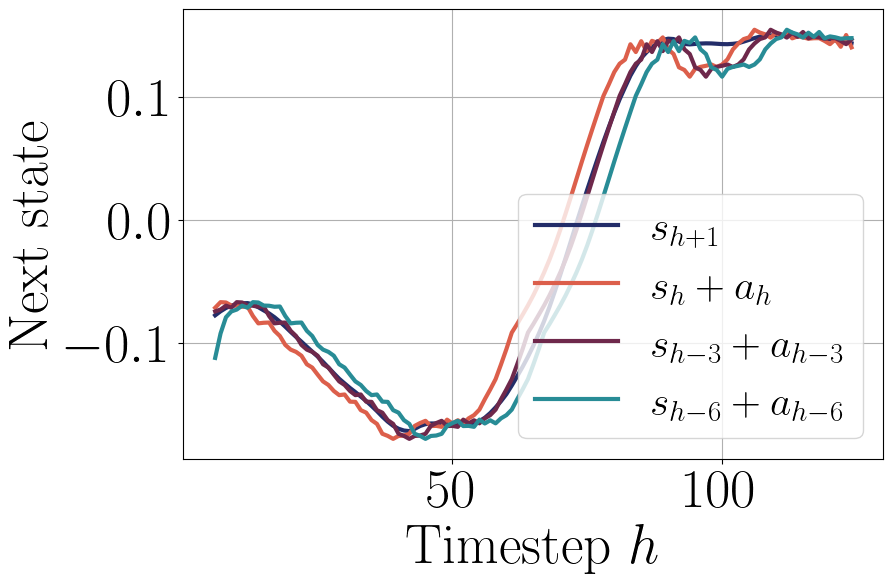}
        \caption{Action component for translation along the first axis.}
        \label{fig:compx}
    \end{minipage}
    \hfill
    \begin{minipage}[t]{0.32\textwidth}
        \centering
        \includegraphics[width=\linewidth]{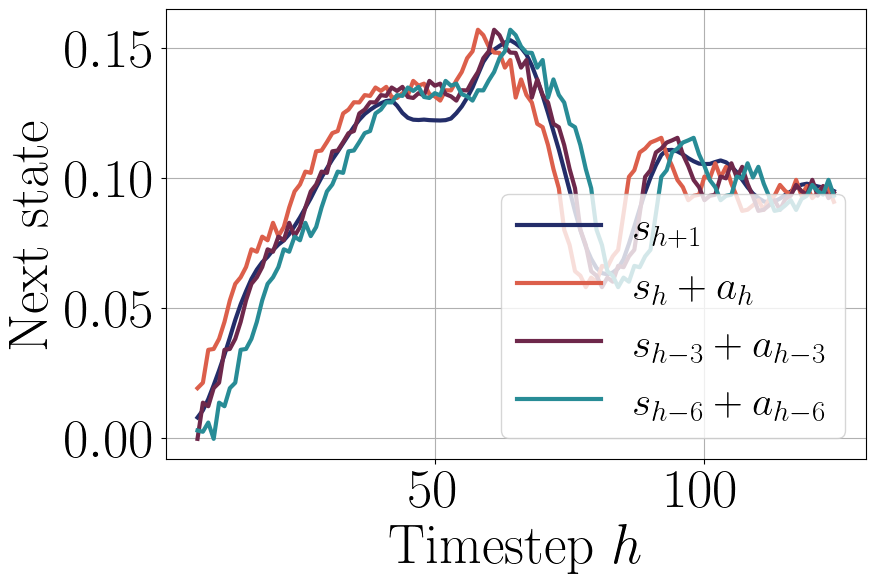}
        \caption{Action component for translation along the second axis.}
        \label{fig:compy}
    \end{minipage}
    \hfill
    \begin{minipage}[t]{0.32\textwidth}
        \centering
        \includegraphics[width=\linewidth]{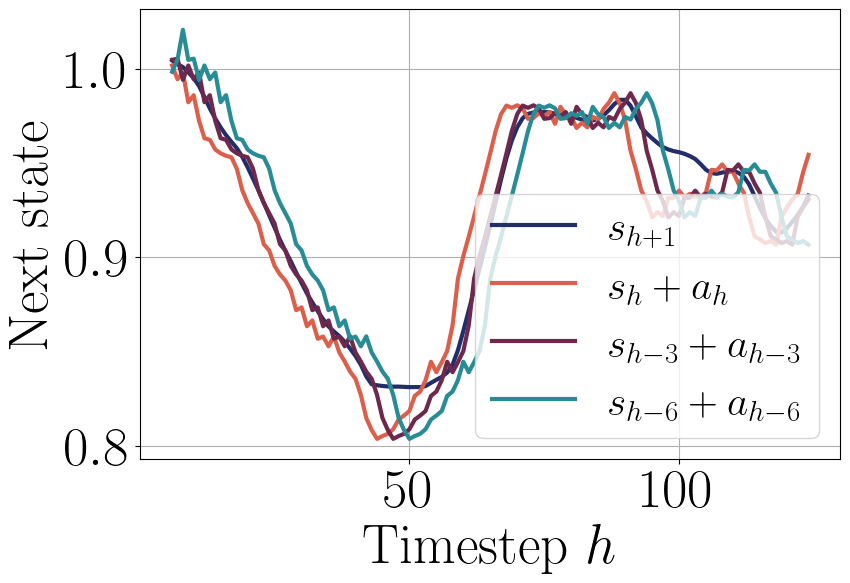}
        \caption{Action component for translation along the third axis.}
        \label{fig:compz}
    \end{minipage}
    
\end{figure*}

\subsection{Results on \texttt{Robomimic MH} and Other \texttt{Libero} Suites}\label{apx: robomimic mh}

For \texttt{Robomimic MH}, we trained 15 single task policies (and for ensemble average over three groups of 5 seeds). Results are shown in Table \ref{table:other_mh}, and are analogous to Table \ref{table:final_success}.

Additionally, we also trained 3 multitask policies per \texttt{Libero} suite, and compared the different delays with different chunk sizes, just like we did in Figure \ref{fig:success_libero} for \texttt{Libero-90}. See Tables \ref{tab:AC-D-libero_10}-\ref{tab:AC-D-libero_object}.

\begin{table*}[t!]
\centering
\scalebox{0.75}
{
\begin{tabular}{l!{\vrule width 1pt}ccc|cc|cc}
\toprule
\texttt{Task}  & Markovian & AC$(n)$ & Delay$(n)$ & AC$(n)$-RDE & AC$(n)$-TE & AC$(n)$-Ens & Delay$(n)$-Ens \\
\midrule

\texttt{Robomimic Can MH}  & $77.8$ {\tiny $\pm {0.7}$} & $93.9$ {\tiny $\pm {0.4}$} & $90.9$ {\tiny $\pm {0.5}$} & $94.3$ {\tiny $\pm {0.4}$} & $93.7$ {\tiny $\pm {0.4}$} & $\mathbf{97.7}$ {\tiny $\pm {0.8}$} & $\mathbf{96.9}$ {\tiny $\pm {0.1}$}  \\

\texttt{Robomimic Square MH}  & $43.4$ {\tiny $\pm {0.9}$} & $75.2$ {\tiny $\pm {0.7}$} & $63.4$ {\tiny $\pm {0.7}$} & $75.0$ {\tiny $\pm {0.6}$} & $67.7$ {\tiny $\pm {0.5}$} & $\mathbf{81.1}$ {\tiny $\pm {0.9}$} & ${79.2}$ {\tiny $\pm {0.6}$}  \\

\texttt{Robomimic Transport MH}  & $0.4$ {\tiny $\pm {0.1}$} & $4.7$ {\tiny $\pm {0.3}$} & $2.8$ {\tiny $\pm {0.2}$} & $4.9$ {\tiny $\pm {0.3}$} & $4.9$ {\tiny $\pm {0.3}$} & $\mathbf{30.0}$ {\tiny $\pm {0.9}$} & ${24.5}$ {\tiny $\pm {0.9}$}  \\

\bottomrule
\end{tabular}
}
\caption{
Comparison of success rates across other \texttt{Robomimic MH}.
}
\label{table:other_mh}
\end{table*}

\begin{table}[t!]
    \centering
    \begin{tabular}{lll}
\hline
 Algorithm   & SR         & t succ   \\
\hline
 Markovian    & 19.2 ± 1.5 & 350 ± 7  \\
 AC$(5)$    & 80.3 ± 1.1 & 288 ± 2  \\
 AC$(10)$   & 88.0 ± 0.4 & 271 ± 1  \\
 AC$(15)$   & 88.2 ± 0.8 & 264 ± 1  \\
 Delay$(3 )$    & 78.1 ± 0.7 & 287 ± 2  \\
 Delay$(6 )$    & 86.7 ± 1.6 & 271 ± 1  \\
 Delay$(9 )$    & 84.1 ± 0.6 & 264 ± 1  \\
 Delay$(13)$    & 74.5 ± 1.3 & 264 ± 1  \\
 Delay$(16)$    & 64.1 ± 0.4 & 267 ± 0  \\
\hline
\end{tabular}
    \caption{Comparison different chunk sizes and delays for \texttt{Libero-10}.}
    \label{tab:AC-D-libero_10}
\end{table}

\begin{table}[t!]
    \centering
    \begin{tabular}{lll}
\hline
 Algorithm   & SR         & t succ   \\
\hline
 Markovian    & 55.5 ± 1.4 & 126 ± 1  \\
 AC$(5)$    & 87.8 ± 1.5 & 112 ± 0  \\
 AC$(10)$   & 90.6 ± 0.3 & 109 ± 0  \\
 AC$(15)$   & 91.7 ± 0.7 & 107 ± 0  \\
 Delay$(3 )$    & 86.9 ± 1.2 & 112 ± 0  \\
 Delay$(6 )$    & 92.0 ± 0.9 & 109 ± 0  \\
 Delay$(9 )$    & 88.6 ± 1.5 & 108 ± 0  \\
 Delay$(13)$    & 79.7 ± 0.3 & 109 ± 0  \\
 Delay$(16)$    & 72.7 ± 0.8 & 108 ± 0  \\
\hline
\end{tabular}
    \caption{Comparison different chunk sizes and delays for \texttt{Libero-Spatial}.}
    \label{tab:AC-D-libero_spatial}
\end{table}

\begin{table}[t!]
    \centering
    \begin{tabular}{lll}
\hline
 Algorithm   & SR         & t succ   \\
\hline
 Markovian    & 63.7 ± 2.1 & 137 ± 0  \\
 AC$(5)$    & 92.2 ± 0.8 & 119 ± 0  \\
 AC$(10)$   & 96.5 ± 0.3 & 113 ± 0  \\
 AC$(15)$   & 96.4 ± 0.9 & 112 ± 0  \\
 Delay$(3 )$    & 92.5 ± 0.7 & 120 ± 1  \\
 Delay$(6 )$    & 96.5 ± 0.2 & 114 ± 0  \\
 Delay$(9 )$    & 93.9 ± 0.3 & 113 ± 0  \\
 Delay$(13)$    & 92.0 ± 0.3 & 112 ± 1  \\
 Delay$(16)$    & 86.9 ± 0.3 & 111 ± 0  \\
\hline
\end{tabular}
    \caption{Comparison different chunk sizes and delays for \texttt{Libero-Goal}.}
    \label{tab:AC-D-libero_goal}
\end{table}

\begin{table}[t!]
    \centering
    \begin{tabular}{lll}
\hline
 Algorithm   & SR         & t succ   \\
\hline
 Markovian    & 67.0 ± 0.9 & 160 ± 1  \\
 AC$(5)$    & 95.9 ± 0.5 & 142 ± 1  \\
 AC$(10)$   & 96.9 ± 0.6 & 137 ± 1  \\
 AC$(15)$   & 98.4 ± 0.3 & 138 ± 1  \\
 Delay$(3 )$    & 94.7 ± 0.3 & 143 ± 1  \\
 Delay$(6 )$    & 97.2 ± 0.3 & 138 ± 0  \\
 Delay$(9 )$    & 96.6 ± 0.2 & 138 ± 0  \\
 Delay$(13)$    & 86.4 ± 0.6 & 142 ± 1  \\
 Delay$(16)$    & 74.6 ± 1.5 & 141 ± 1  \\
\hline
\end{tabular}
    \caption{Comparison different chunk sizes and delays for \texttt{Libero-Object}.}
    \label{tab:AC-D-libero_object}
\end{table}

\subsection{Results with $\pi_{0.5}$}\label{apx: openpi}

We used the open-source weights for the $\pi_{0.5}$ policy \cite{intelligence2025pi_} fine-tuned on \texttt{Libero} available at \url{gs://openpi-assets/checkpoints/pi05_libero} (see repository \url{https://github.com/Physical-Intelligence/openpi} \cite{black2024pi_0}). We ran 40 rollouts for each \texttt{Libero} task in every suite different than 90, as such policy has not been fine-tuned for this suite. Results are reported in Table \ref{table:other_openpi}. As you can see, they are compatible with those we obtained by training our own DDPM \cite{ho2020ddpm} policies. Note that this fine-tuned policy has an overall chunk size of 10. We also conjecture that the performance of the Markovian policy is much better here because, i.a., a pre-processing of the demonstration data removing pauses has been applied.

\begin{table*}[t!]
\centering
\scalebox{0.75}
{
\begin{tabular}{l!{\vrule width 1pt}ccc|ccc|cc}
\toprule
\texttt{Task}  & Markovian & AC$(5)$ & AC$(10)$ & Delay$(3)$ & Delay$(6)$ & Delay$(9)$ & AC$(10)$-RDE & AC$(10)$-TE\\
\midrule

\texttt{Libero-10}  & $84.0$ {\tiny $\pm {6.1}$} & $\mathbf{91.2}$ {\tiny $\pm {4.1}$} & $\mathbf{93.8}$ {\tiny $\pm {1.9}$} & $\mathbf{90.5}$ {\tiny $\pm {4.7}$} & $\mathbf{91.0}$ {\tiny $\pm {2.7}$}& ${88.5}$ {\tiny $\pm {1.9}$}  & $\mathbf{93.8}$ {\tiny $\pm {2.5}$} & $\mathbf{92.2}$ {\tiny $\pm {1.4}$}\\

\texttt{Libero-Spatial}  & $\mathbf{96.7}$ {\tiny $\pm {1.0}$} & $\mathbf{96.7}$ {\tiny $\pm {1.1}$} & $\mathbf{97.8}$ {\tiny $\pm {1.0}$} & $\mathbf{96.8}$ {\tiny $\pm {1.1}$} & $\mathbf{98.0}$ {\tiny $\pm {0.8}$}& $\mathbf{95.0}$ {\tiny $\pm {1.3}$}  & $\mathbf{97.0}$ {\tiny $\pm {0.9}$} & $\mathbf{98.0}$ {\tiny $\pm {0.9}$}\\

\texttt{Libero-Goal}  & $92.5$ {\tiny $\pm {2.6}$} & $\mathbf{97.5}$ {\tiny $\pm {1.1}$} & $97.5$ {\tiny $\pm {0.6}$} & $96.7$ {\tiny $\pm {1.2}$} & $\mathbf{98.5}$ {\tiny $\pm {0.4}$}& $97.8$ {\tiny $\pm {0.6}$}  & $\mathbf{99.0}$ {\tiny $\pm {0.5}$} & $\mathbf{98.5}$ {\tiny $\pm {0.5}$}\\

\texttt{Libero-Object}  & $91.2$ {\tiny $\pm {1.7}$} & $\mathbf{98.5}$ {\tiny $\pm {0.5}$} & $\mathbf{99.2}$ {\tiny $\pm {0.5}$} & $95.5$ {\tiny $\pm {1.1}$} & $97.8$ {\tiny $\pm {0.7}$}& $97.5$ {\tiny $\pm {0.9}$}  & $\mathbf{98.8}$ {\tiny $\pm {0.6}$} & $\mathbf{98.8}$ {\tiny $\pm {0.6}$}\\

\bottomrule
\end{tabular}
}
\caption{
Comparison of success rates for $\pi_{0.5}$ \cite{intelligence2025pi_} on \texttt{Libero}.
}
\label{table:other_openpi}
\end{table*}

\section{Stationary Markov Policies are Not Expressive Enough for Non-Markov Experts}\label{apx:randomwalk}

In this appendix, we aim to clarify that \emph{stationary} Markovian learners are not expressive enough for perfectly replicating (potentially) non-Markovian human behavior, even in the limit of infinite data. 
We provide here a very simple example of environment to show that \emph{stationary} Markovian policies are not expressive enough to achieve the same success rate of a \emph{non-Markovian} expert, even if the environment is Markov. 

To see this, take an MDP with $n$ states labeled $s_1\dots s_n$, where the action simply consists of moving either left or right ($s_n \to s_{n\pm 1}$). Success is defined by reaching either extremal state $1$ or $n$, while the initial state distribution is uniform over all states different from goal states. Assume the following simple non-Markovian demonstrator policy: select an action $\Delta = \pm 1$ uniformly at random at the first timestep and apply the same action for the rest of the episode. If we fit a Markov policy $\pi(a|s)$ to this dataset, we exhibit a random walk. While the demonstration trajectories all succeed in at most $n$ steps, the Markovian behavior-cloning policy $\pi$ succeeds with $p<1$ (Fig. \ref{fig:lineworld}). The discrepancy with standard literature~\cite{pmlr-v202-laroche23a} comes from the assumption of a finite horizon for success.

Note that if we were to simply condition $\pi$ on the step in the episode, we could perfectly match the state-action distribution of the demonstrator (though the trajectory distribution would differ) \cite{pmlr-v202-laroche23a}---the key piece in this example is that we do not condition $\pi$ on the timestep. While timestep conditioning would resolve the challenge here, in practice, it is non-standard to condition on timestep so we would not expect policies in practice to necessarily express non-Markovian demonstrators.
Being more non-Markovian than a Markovian policy, action chunking policies can provide more favorable rates here (intuitively, action chunking policies induce random walks with longer steps, and so are faster to reach the extrema of the interval).

\begin{figure}
    \centering
    \includegraphics[width=0.75\linewidth]{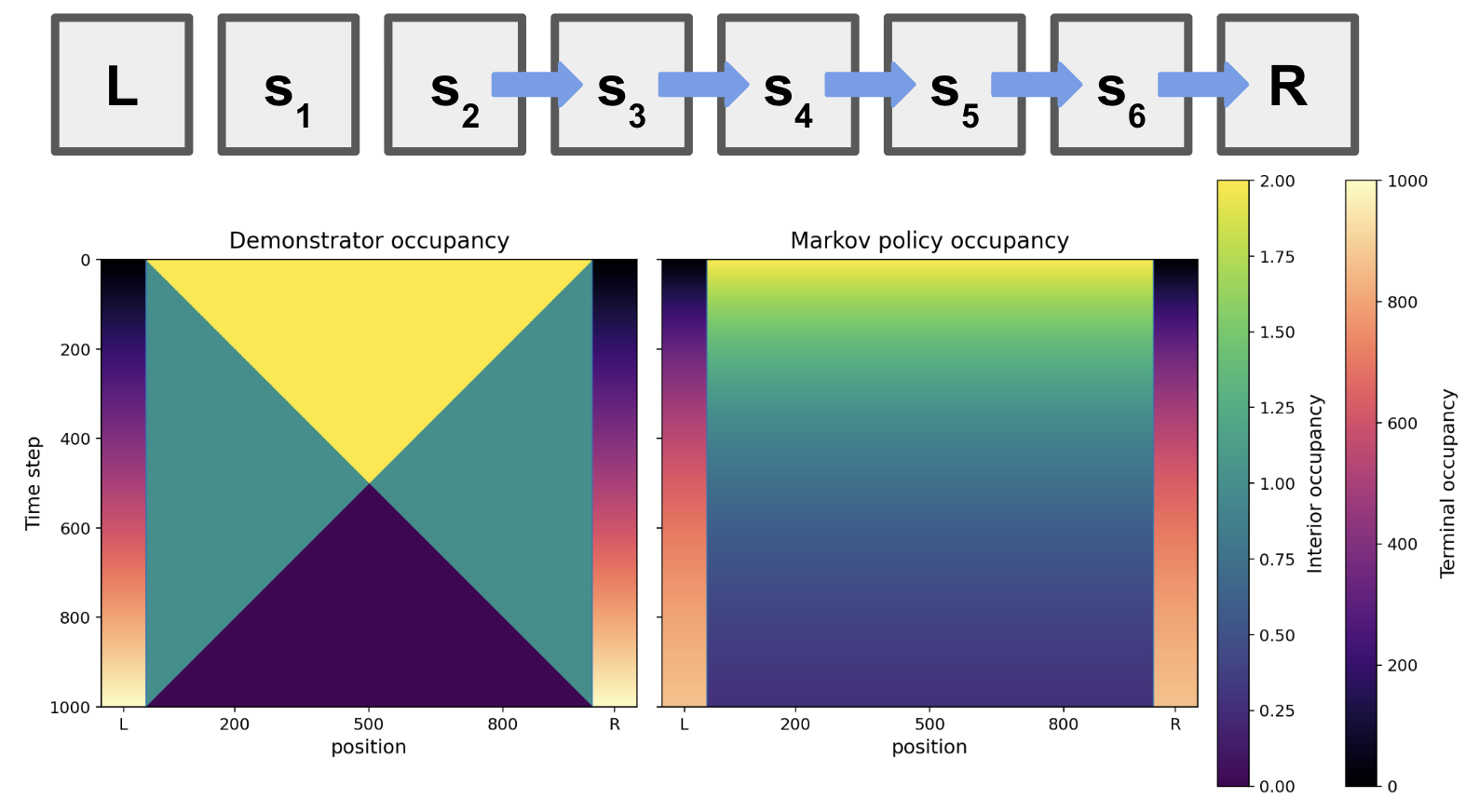}
    \caption{Illustrative example of random-walk behavior on a 1D line-world with $n=1000$ states. While the (non-Markovian) demonstrator moves constantly left or constantly right and arrives at either the left or right terminating states with 100\% probability before the chosen horizon, marginalizing this into a Markovian policy (i.e. via BC) results in significantly different state occupancy and in the limit of large $N$ only achieves $1-e^{-2}$ success rate at time $T=N$.}
    \label{fig:lineworld}
\end{figure}

\section{Predicting Actions Based on Delayed States Provably Mitigates Compounding Error}\label{sec:proofs}

Throughout, for some metrics $\ds$ and $\da$ in, respectively, the state $\cS$ and action $\cA$ spaces, we will make the following assumptions.

\begin{ass}[Dynamics is smooth in the state]\label{ass:dyn_state}
    The dynamics $P$ is 1-Lipschitz in the state. Formally:
    \begin{align*}
        \ds\bigr{P(s,a),P(s',a)}\le \ds\bigr{s,s'} \qquad \forall s,s'\in\cS, \forall a\in\cA.
    \end{align*}
\end{ass}

\begin{ass}[Dynamics is smooth in the action]\label{ass:dyn_action}
    The dynamics $P$ is 1-Lipschitz in the action. Formally:
    \begin{align*}
        \ds\bigr{P(s,a),P(s,a')}\le \da\bigr{a,a'} \qquad \forall s\in\cS, \forall a,a'\in\cA.
    \end{align*}
\end{ass}

\begin{ass}[Reward is smooth]\label{ass:reward}
    The reward $r$ is 1-Lipschitz in the state. Formally:
    \begin{align*}
        \biga{r(s)-r(s')}\le \ds\bigr{s,s'} \qquad \forall s,s'\in\cS.
    \end{align*}
\end{ass}

\begin{ass}[Learned policies are smooth]\label{ass:policy}
    The learned policy $\pihat$ is 1-Lipschitz in the state. Formally, if $\pihat$ is a Markovian or action chunked policy $\pihat_k$, then:
    \begin{align*}
       \wass\bigr{[\pihat_k(\cdot|s)]_i,[\pihat_k(\cdot|s')]_i} \le \ds(s,s') \qquad \forall s,s'\in\cS, \forall i\in[k],
    \end{align*}
    where index $i$ denotes the $i$-th action in the chunk outputted by $\pihat_k$.
    If instead $\pihat$ is a (potentially induced) delayed policy $\pidelay^d$, then:
    \begin{align*}
       \wass\bigr{\pidelay^d(\cdot|s),\pidelay^d(\cdot|s')} \le \ds(s,s') \qquad \forall s,s'\in\cS.
    \end{align*}
\end{ass}
To be precise, we would like to mention the fact that, in general, it might not be possible to construct a ``smooth'' learner as prescribed by Assumption \ref{ass:policy}, which also has small generalization error w.r.t. the expert's policy, as assumed by Theorem \ref{thm:delayed_upper_bound}. Thus, to assume this, we are implicitly assuming that Assumption \ref{ass:policy} holds also for the expert's policy.

We would like to remark that, in the following proofs, we will \emph{not} assume a binary success/failure reward, but we will consider a smooth reward as prescribed by Assumption \ref{ass:reward}. As such, our episodes will not terminate when any non-zero reward is received, but they will always terminate after $H$ timesteps.

\subsection{Upper Bound for Delayed Policy}
Throughout, we define $s_{i} := s_0$ for $i < 0$.

\begin{proof}[Proof of \Cref{thm:delayed_upper_bound}]

    We begin by proving the result for the delayed policy, that is, let our learner $\pihat$ be a delayed policy $\pidelay^k$ for some $k\ge0$ (note that, for $k=0$, we are effectively studying a Markovian policy). Formally, as the behavior of policy $\pidelay^k$ is undefined for $t<k$, we assume that we are going to play the chunked policy there (i.e., we use some $\pihat$ policy) and assume that the same upper bound to the generalization error $\epsilon$ holds for all the actions there.

    \textbf{Delayed policy.}
    
    By definition, for any policy $\pi$, we have that $\cJ(\pi) = \Exp^\pi[\sum_{t=0}^{H-1} r(s_t)]$. Therefore, by triangle inequality, we can write:
    \begin{align*}
        \cJ(\pidemo) - \cJ(\pihat) \le \sum_{t=0}^{H-1} \biga{\Exp^{\pidemo}[r(s_t)] - \Exp^{\pihat}[r(s_t)]},
    \end{align*}
    where we would like to recall that symbol $\Exp^\pi$ denotes the trajectory distribution of policy $\pi$, so that $\Exp^\pi[r(s_t)]=\sum_\tau \P^\pi(\tau)r(s_t)$, where $\tau$ denotes some trajectory of length $h$, $\P^\pi(\cdot)$ is the induced probability measure over trajectories by $\pi$, and $s_t$ denotes the state at timestep $t$ of trajectory $\tau$.
    
    Consider some $t \ge 1$. Let $\sdemo_t$ denote the state generated in a trajectory induced by playing $\pidemo$ and $\shat_t$ the state generated by playing $\pihat$. Similarly define $\ademo_t$ and $\ahat_t$ the corresponding actions. Since we have assumed that $r$ is 1-Lipschitz, by Assumption \ref{ass:reward} and Jensen's inequality, we have
    \begin{align*}
        \biga{\Exp^{\pidemo}[r(s_t)] - \Exp^{\pihat}[r(s_t)]} = \biga{\Exp[r(\sdemo_t) - r(\shat_t)]} \le \Exp[\ds(\sdemo_t,\shat_t)],
    \end{align*}
    where $\Exp$ is the expectation w.r.t. some coupling between the trajectory distributions $\P^{\pidemo}$ and $\P^{\pihat}$ up to timestep $t$.

    By the definition of the dynamics, we have $\sdemo_t = P(\sdemo_{t-1},\ademo_{t-1})$ and $\shat_t = P(\shat_{t-1},\ahat_{t-1})$. Since we have assumed that $P$ is 1-Lipschitz in both state and action, then by Assumptions \ref{ass:dyn_state}-\ref{ass:dyn_action} and by the triangle inequality we have
    \begin{align*}
        \Exp[\ds(\sdemo_t,\shat_t)] \le \Exp[\ds(\sdemo_{t-1},\shat_{t-1})] + \Exp[\da(\ademo_{t-1},\ahat_{t-1})].
    \end{align*}
    By assumption, $\ademo_{t-1} \sim \pidemo(\cdot \mid \histdemo_{t-1})$ (for $\histdemo$ the history generated under $\pidemo$) and $\ahat_{t-1} \sim \pihat(\cdot \mid \shat_{t-k-1})$. 
    By the triangle inequality, we have
    \begin{align}\label{eq:triang}
        \Exp[\da(\ademo_{t-1},\ahat_{t-1})] & \le \Exp[\Exp_{a \sim \pihat(\cdot \mid \sdemo_{t-k-1})}[\da(\ademo_{t-1},a)]] + \Exp[\Exp_{a \sim \pihat(\cdot \mid \sdemo_{t-k-1})}[\da(a,\ahat_{t-1})]].
    \end{align}
    Considering the optimal coupling between $\ademo_{t-1} \mid \cF_{t-1}$, for $\cF_{t-1}$ the filtration up to step $t-1$, and $a \sim \pihat(\cdot \mid \sdemo_{t-k-1})$, we can bound
    \begin{align*}
         \Exp[\Exp_{a \sim \pihat(\cdot \mid \sdemo_{t-k-1})}[\da(\ademo_{t-1},a)]] \le \Exp[\wass(\pidemo(\cdot \mid \histdemo_{t-1}),\pihat(\cdot \mid \sdemo_{t-k-1}))]] \le \epsilon,
    \end{align*}
    where the last inequality follows by assumption. Note that
    \begin{align*}
         \Exp[\Exp_{a \sim \pihat(\cdot \mid \sdemo_{t-k-1})}[\da(a,\ahat_{t-1})]] & =  \Exp[\Exp_{a \sim \pihat(\cdot \mid \sdemo_{t-k-1}), a' \sim \pihat(\cdot \mid \shat_{t-k-1})}[\da(a,a')]] \\
         & \le \Exp[\wass(\pihat(\cdot \mid \sdemo_{t-k-1}),\pihat(\cdot \mid \shat_{t-k-1}))] \\
         & \le \Exp[\ds(\sdemo_{t-k-1},\shat_{t-k-1})]
    \end{align*}
    where the first inequality follows taking the optimal coupling between $a$ and $a'$, and the second inequality follows from Assumption \ref{ass:policy}.
    Altogether then, we have shown that
    \begin{align*}
         \Exp[\ds(\sdemo_t,\shat_t)] \le  \Exp[\ds(\sdemo_{t-1},\shat_{t-1})] + \Exp[\ds(\sdemo_{t-k-1},\shat_{t-k-1})] + \epsilon.
    \end{align*}
    By \Cref{lem:recursion_bound}, we can bound
    \begin{align*}
    \Exp[\ds(\sdemo_t,\shat_t)] \le 2 (k+1)^{\ceil{t/k}} \cdot \epsilon.
    \end{align*}
    The result then follows by summing over $t$ and using geometric sums:
    \begin{align*}
    \sum\limits_{t=0}^{H-1}(k+1)^{\ceil{t/k}}&=1+\sum\limits_{m=1}^k (k+1)
    +\sum\limits_{m=k+1}^{2k} (k+1)^2+\dotsc\\
    &\le 1+\sum\limits_{m=1}^{\ceil{(H-1)/k}} k(k+1)^m\\
    &\le 1+k\sum\limits_{m=0}^{\ceil{(H-1)/k}} (k+1)^m\\
    &= 1+k\frac{(k+1)^{\ceil{(H-1)/k}+1}-1}{(k+1)-1}\\
    &= (k+1)^{\ceil{(H-1)/k}+1}.
    \end{align*}

    \textbf{Action chunking (sketch).}

    Note that the proof applied above can be easily modified for using any sequence of delayed policies with different delays. In particular, the key difference is just that, in Eq. \eqref{eq:triang}, instead of using $\pihat(\cdot \mid \sdemo_{t-k-1})$, we should use the desired delay $d$, i.e., $\pihat(\cdot \mid \sdemo_{t-d})$. Then all other passages follow analogously. In case of action chunking, we play a sequence of delayed policies in ascending order by delay, up to reach the maximum delay $k-1$ (corresponding to chunk size $k$), before restarting. So, if we set our learner $\pihat$ be a chunked policy $\pihat_k$, then it is easy to observe that we end up to the following recursion, for any timestep $t=kx+r$, where $x\coloneqq\floor{t/k}$ and $r\coloneqq t-k\floor{t/k}$:
    \begin{align*}
         \Exp[\ds(\sdemo_t,\shat_t)] \le  \Exp[\ds(\sdemo_{kx+r-1},\shat_{kx+r-1})] + \Exp[\ds(\sdemo_{n},\shat_{n})] + \epsilon,
    \end{align*}
    where $n=kx$ if $r>0$ and $n=k(x-1)$ otherwise.
    
    Thanks to Lemma \ref{lem:recursion_bound2}, we obtain:
    \begin{align*}
    \Exp[\ds(\sdemo_t,\shat_t)] \le (k+1)^{\floor{t/k}+1} \cdot \epsilon.
    \end{align*}
The result then follows by summing over $h$ and using geometric sums:
    \begin{align*}
    \sum\limits_{t=0}^{H-1}(k+1)^{\floor{t/k}+1}&=\sum\limits_{m=0}^{k-1} (k+1)
    +\sum\limits_{m=k}^{2k-1} (k+1)^2+\dotsc\\
    &\le \sum\limits_{m=0}^{\floor{(H-1)/k}+1} (k+1)^m\\
    &= \frac{(k+1)^{\floor{(H-1)/k}+2}-1}{(k+1)-1}\\
    &\le 2(k+1)^{\floor{(H-1)/k}+1}.
    \end{align*}

    This concludes the proof.\footnote{Note that the same rate would be obtained even if we randomized or used any other kind of order for playing delayed policies. Moreover, note that the upper bound for the delayed policies is not tight, and could probably obtain a bound where the contribution of the $k$ factor is doubled, since a delayed policy with delay $k$ roughly corresponds to an action chunking policy with chunk size $2k$.}
\end{proof}

\begin{lemma}\label{lem:recursion_bound}
Consider a sequence $\{\delta_t\}_t$ with $\delta_t = 0$ for $t \le 0$, and that satisfies
\begin{align}\label{eq:delta_seq_def}
    \delta_{t+1} \le \delta_{t} + \delta_{t-k} + \epsilon,
\end{align}
for some constant $\epsilon\ge 0$. Then we can bound $\delta_t \le 2 (k+1)^{\ceil{t/k}} \cdot \epsilon$.
\end{lemma}
\begin{proof}
    Define $\{\deltabar_t\}_t$ as the sequence which satisfies $\deltabar_t = 0$ for $t \le 0$ and $\deltabar_{t+1} = \deltabar_t + \deltabar_{t-k} + \epsilon$. By construction, we then have that $\{\deltabar_t\}_t$ satisfies \eqref{eq:delta_seq_def}, and $\deltabar_t \ge \delta_t$ $\forall t$ for any other sequence satisfying \eqref{eq:delta_seq_def}.\footnote{Formally, this can be shown by induction. At $t\le0$, we have $\deltabar_t = \delta_t=0$ by definition. Then, let us make the inductive hypothesis that $\deltabar_t \ge \delta_t$ $\forall t\le\overline{t}$. Then, at $t=\overline{t}$, we have that: $\deltabar_{t+1} \coloneqq \deltabar_t + \deltabar_{t-k} + \epsilon \ge \delta_{t} + \delta_{t-k} + \epsilon\ge \delta_{t+1}$, where we used first the definition of $\{\deltabar_t\}_t$, then the inductive hypothesis, and finally the definition of $\{\delta_t\}_t$.} Note also that $\deltabar_{t+1} \ge \deltabar_t$.\footnote{This follows easily from the definition and the fact that all terms are non-negative: $\deltabar_{t+1} = \deltabar_t + \deltabar_{t-k} + \epsilon\ge \deltabar_t+0+0\ge \deltabar_t$.}

    Recursing backwards, for any $k\ge1$, we have
    \begin{align*}
        \deltabar_{t} & = \deltabar_{t-1} + \deltabar_{t-k-1} + \epsilon \\
        & = \deltabar_{t-2} + \deltabar_{t-k-1} + \deltabar_{t-k-2} + 2 \epsilon \\
        & \vdots \\
        & = \deltabar_{t-k} + \sum_{i=1}^{k} \deltabar_{t-k-i} + k \epsilon\\
        & = \sum_{i=0}^{k} \deltabar_{t-k-i} + k \epsilon.
    \end{align*}
    Since $\deltabar_{t+1} \ge \deltabar_t$ for any $t$, we can bound this as
    \begin{align*}
     \deltabar_{t}   \le (k+1) \deltabar_{t-k} + k \epsilon.
    \end{align*}
    Recursing this backwards, we get
    \begin{align*}
    \deltabar_{t} &\le (k+1) \deltabar_{t-k} + (k+1) \epsilon \le (k+1)^2 \deltabar_{t-2k} + (k+1)^2 \epsilon + (k+1) \epsilon \\
    &\le \ldots \le \sum_{i=1}^{\ceil{t/k}} (k+1)^i \epsilon,
    \end{align*}
    where we note that $x=\ceil{t/k}$ is the minimum integer value guaranteeing $t-xk\le0$.

    Lastly, we use the bound for geometric sums:
    \begin{align*}
    \deltabar_{t} &\le\sum_{i=1}^{\ceil{t/k}} (k+1)^i \epsilon \le \epsilon \sum_{i=0}^{\ceil{t/k}} (k+1)^i
    = \epsilon\frac{(k+1)^{\ceil{t/k}+1}-1}{k}\\
    & \le \epsilon\frac{k+1}{k}(k+1)^{\ceil{t/k}}\le 2\epsilon (k+1)^{\ceil{t/k}}.
    \end{align*}
    
    This completes the proof.
\end{proof}

\begin{lemma}\label{lem:recursion_bound2}
Consider a sequence $\{\delta_t\}_t$ with $\delta_t = 0$ for $t = 0$, and that satisfies, for any timestep $t=kx+r$, where $x\coloneqq\floor{t/k}$ and $r\coloneqq t-k\floor{t/k}$
\begin{align}\label{eq:delta_seq_def2}
    \delta_{kx+r} \le \delta_{kx+r-1} + \delta_{n} + \epsilon,
\end{align}
where $n=kx$ if $r>0$ and $n=k(x-1)$ otherwise, for some constant $\epsilon\ge 0$. Then we can bound $\delta_t \le (k+1)^{\floor{t/k}+1} \cdot \epsilon$.
\end{lemma}
\begin{proof}
    Define $\{\deltabar_t\}_t$ as the sequence which satisfies $\deltabar_t = 0$ for $t = 0$ and $\deltabar_{kx+r} = \deltabar_{kx+r-1} + \deltabar_{n} + \epsilon$. By construction, we then have that $\{\deltabar_t\}_t$ satisfies \eqref{eq:delta_seq_def2}, and $\deltabar_t \ge \delta_t$ $\forall t$ for any other sequence satisfying \eqref{eq:delta_seq_def2}.\footnote{This can be proved analogously to the proof of Lemma \ref{lem:recursion_bound}.}

    Recursing backwards, for any $k\ge1$, we have
    \begin{align*}
        \deltabar_{kx+r} & \le \deltabar_{kx+k} \\
        & = \deltabar_{kx+(k-1)} + \deltabar_{kx} + \epsilon \\
        & = \deltabar_{kx+(k-2)} + 2\deltabar_{kx} + 2\epsilon \\
        & \vdots \\
        & = \deltabar_{kx} + k \deltabar_{kx} + k \epsilon\\
        & = (k+1)\deltabar_{kx} + k \epsilon.
    \end{align*}
    Recursing this backwards, we get, for $t=kx+r$:
    \begin{align*}
    \deltabar_{t} &\le (k+1)\deltabar_{kx} + k \epsilon \le (k+1)^2 \deltabar_{k(x-1)} + k(k+1) \epsilon + k \epsilon \\
    &\le \ldots \le k\sum_{i=0}^{\floor{t/k}} (k+1)^i \epsilon.
    \end{align*}

    Lastly, we use the bound for geometric sums:
    \begin{align*}
    \deltabar_{t} &\le k\epsilon \sum_{i=0}^{\floor{t/k}} (k+1)^i
    = k\epsilon\frac{(k+1)^{\floor{t/k}+1}-1}{k}\\
    & \le \epsilon(k+1)^{\floor{t/k}+1}.
    \end{align*}
    
    This completes the proof.
\end{proof}

\subsection{Lower Bound for Markovian Policies}

In this appendix, we will use symbol $[x_1,x_2,\dotsc,x_d]$ to denote a $d$-dimensional vector, and symbol $[x]_i$ to denote the $i$th component of vector $x$.

\begin{proof}[Proof of \Cref{thm:lb_markov}]
  Let the environment $\cM$ be such that: $\cS=\cA=\RR^2$, $P_0([0,0])=1$, and
  $P(s,a)=s+a$ for all $s,a$. Let $r(s) = 1 - |[s]_1|$. Let $\ds$ and $\da$ be the metrics induced by any norm $\|\cdot\|$. Then, note that the Lipschitz assumption on $P$ (that is, Assumptions \ref{ass:dyn_state}-\ref{ass:dyn_action}) holds
  since:
  \begin{align*}
        &\|P(s,a)-P(s',a)\|= \|(s+a)-(s'+a)\|=\|s-s'\|,\\
        &\|P(s,a)-P(s,a')\|= \|(s+a)-(s+a')\|=\|a-a'\|.    
  \end{align*}
  Furthermore, the Lipschitz assumption on $r$ (that is, Assumption \ref{ass:reward}) holds since
  \begin{align*}
      |r(s) - r(s')| = ||[s]_1| - |[s']_1|| \le |[s]_1 - [s']_1| \le \| s - s' \|_2.
  \end{align*}
Let the expert's policy $\pidemo:\cS\to\cA$ be $\pidemo(s)=[0,1]$ $\forall s$. This
induces a support: $\text{supp}(\pidemo)=\{[0,k]\in\cS\,:\,k\in\Nat\}$.

Let $\pihat:\cS\to\cA$ be:
$\pihat([s^1,s^2])=[\epsilon+\min_{s'\in\text{supp}(\pidemo)}\|s-s'\|_2,1]$, and note that it is 1-Lipschitz (i.e., satisfies Assumption \ref{ass:policy}) since:
\begin{align*}
  \|\pihat(s)-\pihat(s'')\|_2  &=\big|\min_{s'\in\text{supp}(\pidemo)}\|s-s'\|_2-\min_{s'\in\text{supp}(\pidemo)}\|s''-s'\|_2\big|\\
  &\le \|s-s''\|_2,
\end{align*}
where we used that the map $d_\cX(u)\coloneqq \min_{x\in\cX}\|u-x\|_2$ is
1-Lipschitz for any $\cX\subseteq\RR^2$. Moreover, note that this $\pihat$
satisfies also the (last) hypothesis of this theorem, which is that of having a bounded generalization error, since:
\begin{align*}
  \Exp^{\pidemo}[\wass(\pidemo(s_t), \pihat(s_t))] = \Exp^{\pidemo}[\| \pidemo(s_t) - \pihat(s_t)\|_2] = \epsilon.
\end{align*}

As shown also in Figure \ref{fig: lower bound bc}, it is easy to see that the
sequence of expert states is
$\sdemo_0=[0,0],\sdemo_1=[0,1],\sdemo_2=[0,2],\dotsc,\sdemo_t=[0,t]$, while the sequence of
states followed by $\pihat$ starts at $\shat_0=[0,0]$ and grows based on
$s_{t+1}^{\pihat}=s_t^{\pihat} + \pihat(s_t^{\pihat})$ as:
\begin{align*}
\begin{cases}
  &\shat_0=[0,0],\\
  &\shat_1=[0,0]+[\epsilon+0,1]=[\epsilon,1],\\
  &\shat_2=[\epsilon,1]+[\epsilon+\epsilon,1]=[3\epsilon,2],\\
  &\shat_3=[3\epsilon,2]+[\epsilon+3\epsilon,1]=[7\epsilon,3],\\
  &\dotsc
\end{cases}.
\end{align*}
Thus:
\begin{align*}
  [\shat_t]_1 = \epsilon + 2 [\shat_{t-1}]_1 = \epsilon (2^t - 1).
\end{align*}
Note that $[\sdemo_t]_1 = 0$, so that $\cJ(\pidemo) = H$. However,
\begin{align*}
    \cJ(\pihat) & = \sum_{t=0}^{H-1} r(\shat_t) = \sum_{t=0}^{H-1} (1 - [\shat_t]_1) = H - \epsilon \cdot \sum_{t=0}^{H-1} (2^h - 1) = H - \epsilon \cdot (2^{H} - H - 1).
\end{align*}
This completes the proof.
\end{proof}

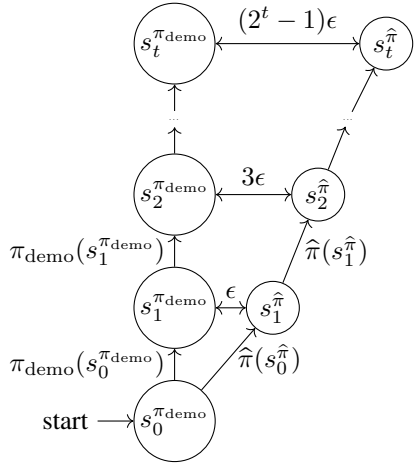
\begin{figure}[t!]
  \centering
  \begin{tikzpicture}[state/.append style={minimum size=0.7cm, inner sep=1pt}]
    \node[state, initial] at (0,0) (s0) {$s_0^{\pidemo}$};
    \node[state] at (0,1.5) (s1) {$s_1^{\pidemo}$};
    \node[state] at (0,3) (s2) {$s_2^{\pidemo}$};
    \node[state,draw=none,scale=0.3] at (0,4) (ss) {$\dotsc$};
    \node[state,draw=none,scale=0.3] at (2.3,4) (ssbc) {$\dotsc$};
    \node[state] at (0,5) (st) {$s_t^{\pidemo}$};
    \node[state] at (1.3,1.5) (s1bc) {$s_1^{\pihat}$};
    \node[state] at (1.9,3) (s2bc) {$s_2^{\pihat}$};
    \node[state] at (2.8,5) (stbc) {$s_t^{\pihat}$};
    \draw (s0) edge[->, solid, left] node{$\pidemo(s_0^{\pidemo})$} (s1);
    \draw (s1) edge[->, solid, left] node{$\pidemo(s_1^{\pidemo})$} (s2);
    \draw (s0) edge[->, solid, right] node{$\pihat(s_0^{\pihat})$} (s1bc);
    \draw (s1bc) edge[->, solid, right] node{$\pihat(s_1^{\pihat})$} (s2bc);
    \draw (s2) edge[->, solid, left] node{} (ss);
    \draw (s2bc) edge[->, solid, left] node{} (ssbc);
    \draw (ss) edge[->, solid, left] node{} (st);
    \draw (ssbc) edge[->, solid, left] node{} (stbc);
    \draw (s1) edge[<->, solid, above] node{$\epsilon$} (s1bc);
    \draw (s2) edge[<->, solid, above] node{$3\epsilon$} (s2bc);
    \draw (st) edge[<->, solid, above] node{$(2^t-1)\epsilon$} (stbc);
  \end{tikzpicture}
  \caption{Construction for lower bound on performance of Markovian policy (\Cref{thm:lb_markov})}
  \label{fig: lower bound bc}
\end{figure}

\subsection{Lower Bound for Action-Chunked Policies}

\begin{theorem}\label{thm:lb_ac}
    There exists an environment satisfying the above assumptions, a Markov demonstrator $\pidemo$, and an action chunking learner $\pihat_k$ satisfying $\max_h \E^{\pi_{\rm demo}}[\max_{i\in[k]} W_1(\pi_{\rm demo}(s_{t+i-1}),[\hat\pi_k(s_t)]_i)]\le\epsilon$, for $W_1(\cdot, \cdot)$ the Wasserstein-1 metric,
    such that $\cJ(\pidemo) \ge \cJ(\pihat) + \Omega((k+1)^{k/H} \cdot \epsilon)$.
\end{theorem}
\begin{proof}
We consider the same environment, reward, and expert as in \Cref{thm:lb_markov}.

Let 
$[\pihat([s^1,s^2])]_i=[\epsilon+\min_{s'\in\text{supp}(\pidemo)}\|s-s'\|_2,1]$, and note that it is 1-Lipschitz (i.e., satisfies Assumption \ref{ass:policy}) since, for any action in the chunk $i\in[k]$:
\begin{align*}
  \|[\pihat(s)]_i-[\pihat(s'')]_i\|_2  &=\big|\min_{s'\in\text{supp}(\pidemo)}\|s-s'\|_2-\min_{s'\in\text{supp}(\pidemo)}\|s''-s'\|_2\big|\\
  &\le \|s-s''\|_2,
\end{align*}
where we used that the map $d_\cX(u)\coloneqq \min_{x\in\cX}\|u-x\|_2$ is
1-Lipschitz for any $\cX\subseteq\RR^2$. Moreover, note that this $\pihat$
satisfies the assumption of uniform bound on the generalization error along the chunk, since:
\begin{align*}
  \Exp^{\pidemo}[\wass(\pidemo(s_t), \pihat_i(s_{t-i}))] = \Exp^{\pidemo}[\| \pidemo(s_t) - \pihat_i(s_{t-i})\|_2] = \epsilon.
\end{align*}
With this choice of $\pihat$, since we start at $s_0^{\pidemo} = [0,0]$, for the first $k$ timesteps we play $\pihat_i(s_0^{\pidemo}) = [\epsilon,1]$, so that $\shat_k = [k \epsilon, k]$. Then, we have $\pihat(\shat_k) = [k \epsilon + \epsilon, 1]$, so that $\pihat(\shat_{2k}) = [k\epsilon + k(k\epsilon + \epsilon), 2k]$. More generally, we see that 
\begin{align*}
    \shat_{ik} = [(k+1)^{i} \cdot \epsilon, ik]
\end{align*}
and, more generally, $\shat_t = [(k+1)^{i_t} \cdot \epsilon + (t - i_t k) \cdot ((k+1)^{i_t} \cdot \epsilon + \epsilon), t]$, where $i_t := \lfloor t/k \rfloor$ denotes the closest action chunk boundary index.
It follows that
\begin{align*}
    \cJ(\pihat) & = H - \epsilon \cdot \sum_{i=0}^{H/k-1} \sum_{j=0}^k \left [ (k+1)^i + j \cdot ((k+1)^i+1) \right ] \\
    &  = H - \epsilon \cdot  \left [  \frac{k+1}{k} \cdot (k+1)^{H/k} - \frac{k+1}{k} + H \right ].
\end{align*}
This proves the result.
\end{proof}

\section{Individualized Result Plots}\label{apx:addplots}

In this section we report the plots of Figs. \ref{fig:val_loss_libero} (see Figs. \ref{fig:val loss each libero1}-\ref{fig:val loss each libero3}), \ref{fig:success_libero} (see Figs. \ref{fig:succ each libero1}-\ref{fig:succ each libero3}), \ref{fig:robomimic_val} (see Fig. \ref{fig:val each robomimic}), \ref{fig:success_robomimic} (see Fig. \ref{fig:succ each robomimic}), individually for each task instead of averaging over all tasks.

\begin{figure*}[t]
    \centering

    \begin{minipage}[t]{0.23\textwidth}
        \centering
        \includegraphics[width=\linewidth]{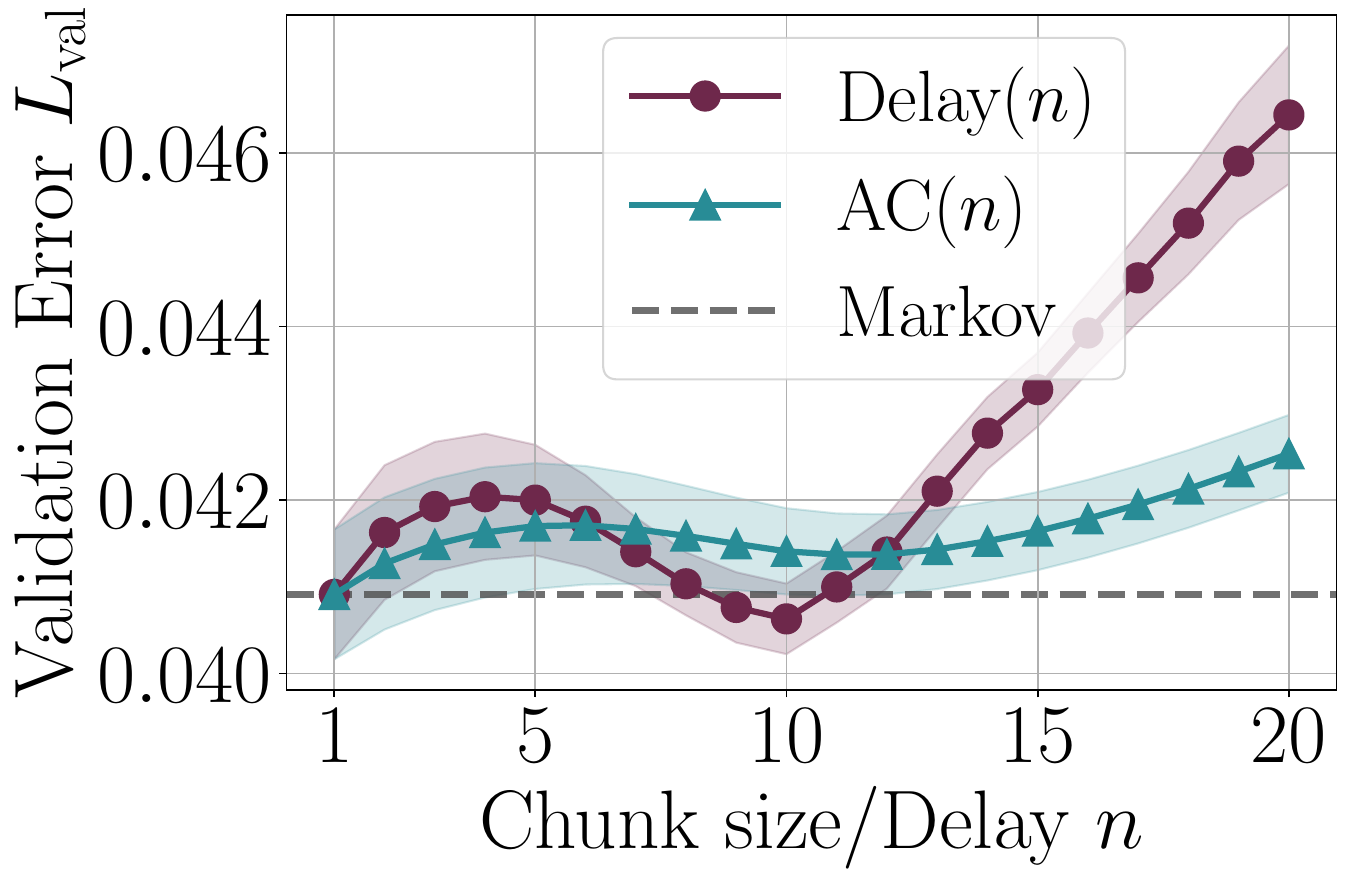}
    \end{minipage}
    \hfill
        \begin{minipage}[t]{0.23\textwidth}
        \centering
        \includegraphics[width=\linewidth]{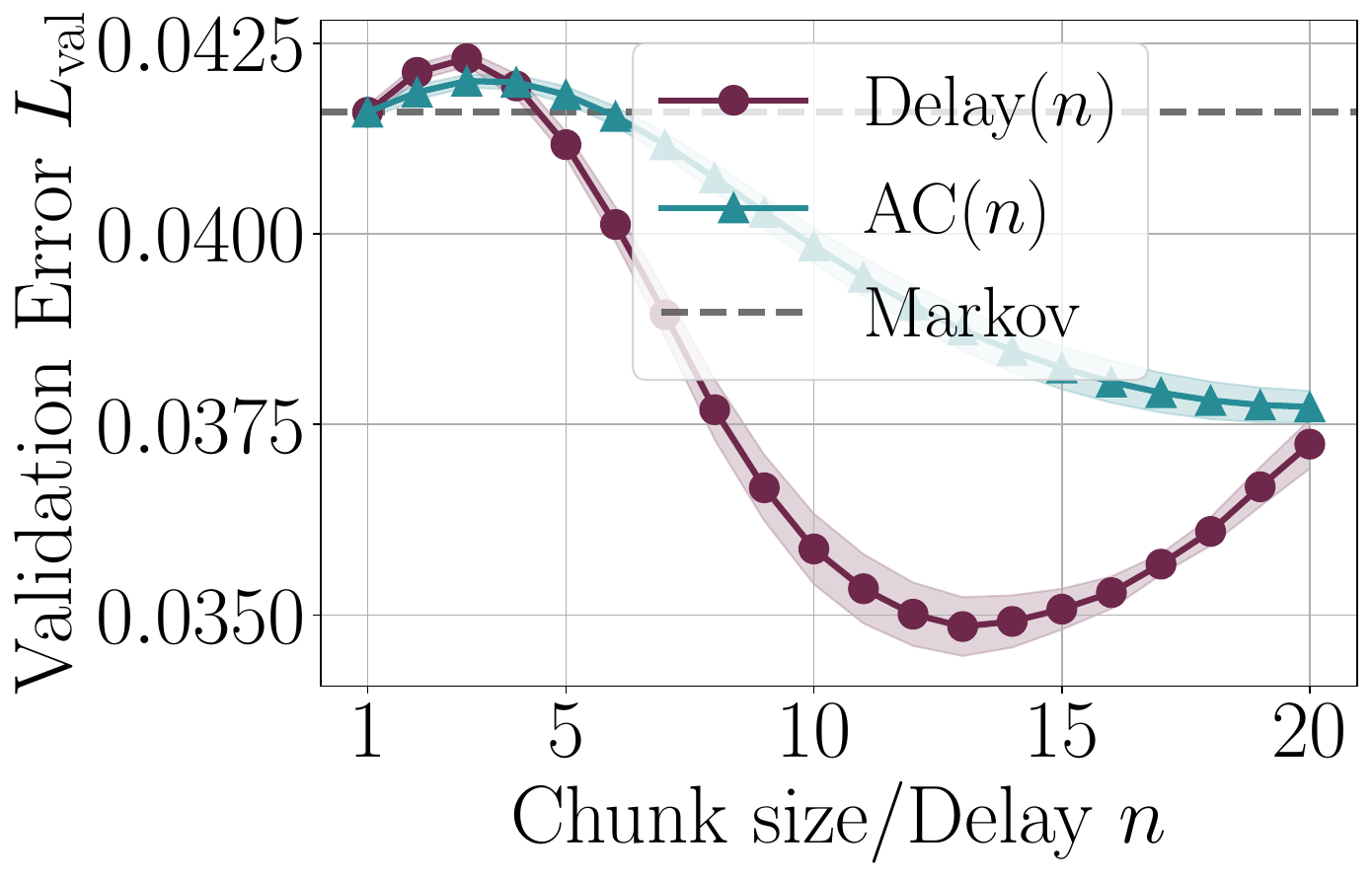}
    \end{minipage}
    \hfill
        \begin{minipage}[t]{0.23\textwidth}
        \centering
        \includegraphics[width=\linewidth]{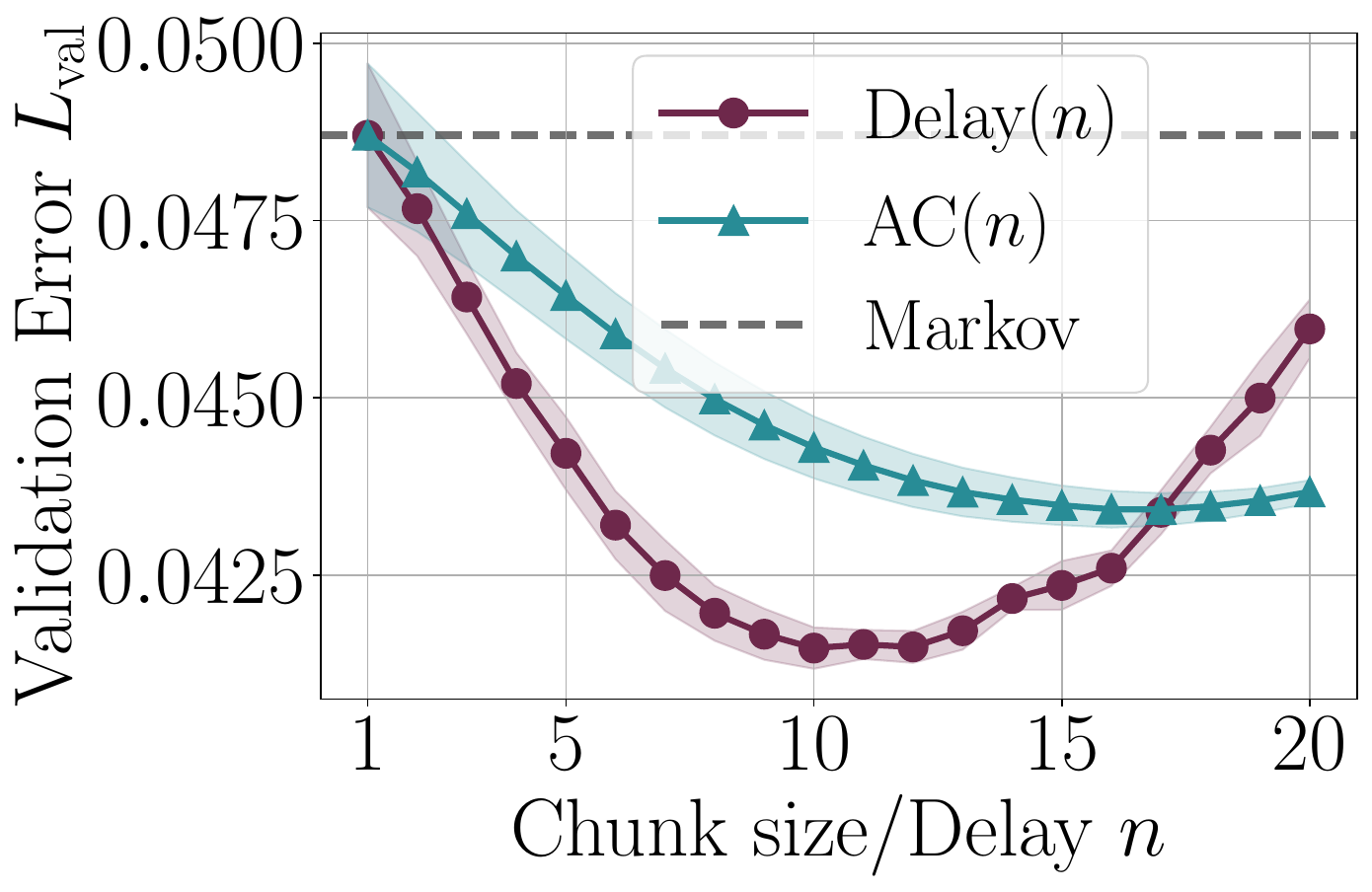}
    \end{minipage}
    \hfill
        \begin{minipage}[t]{0.23\textwidth}
        \centering
        \includegraphics[width=\linewidth]{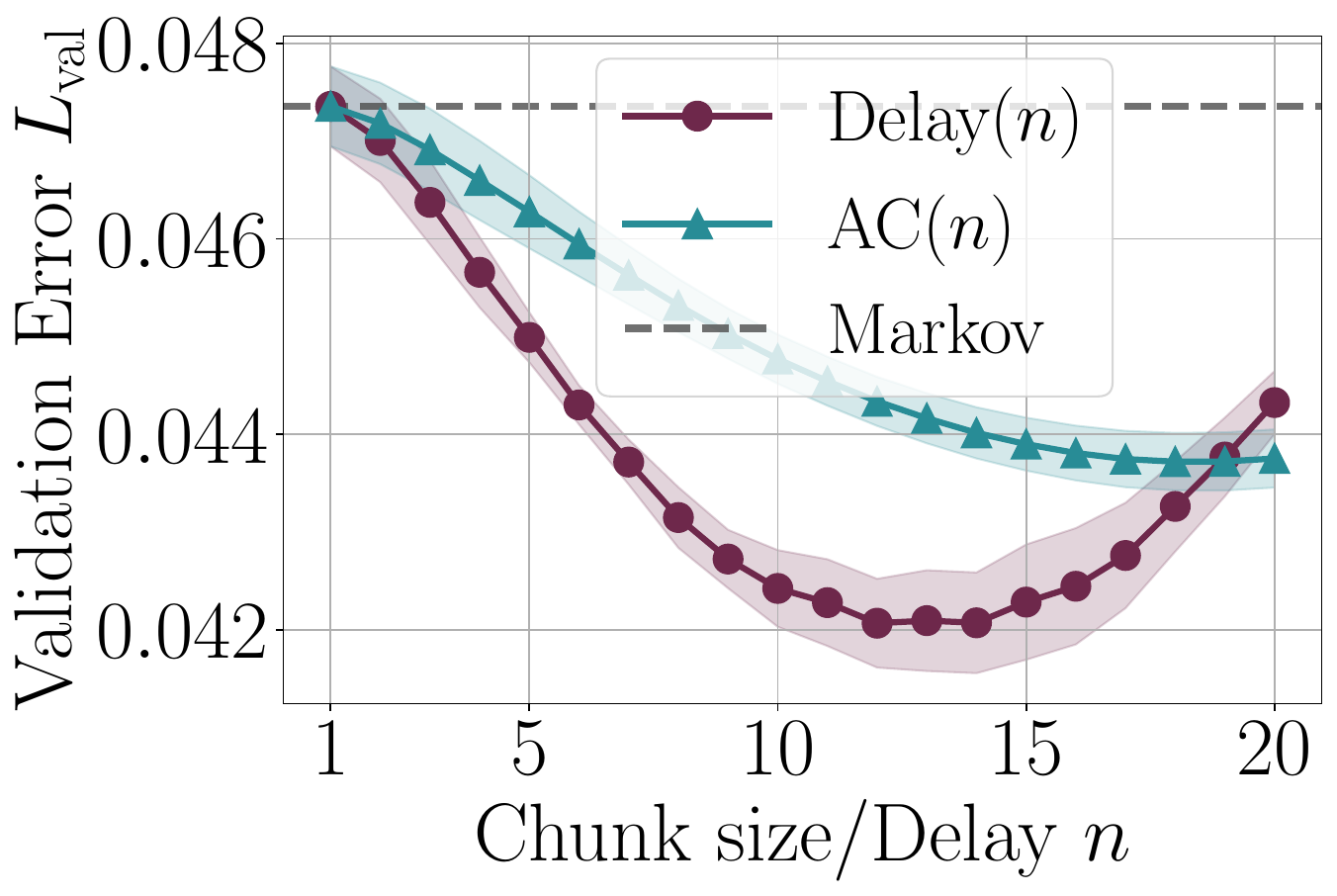}
    \end{minipage}
        \begin{minipage}[t]{0.23\textwidth}
            \includegraphics[width=\linewidth]{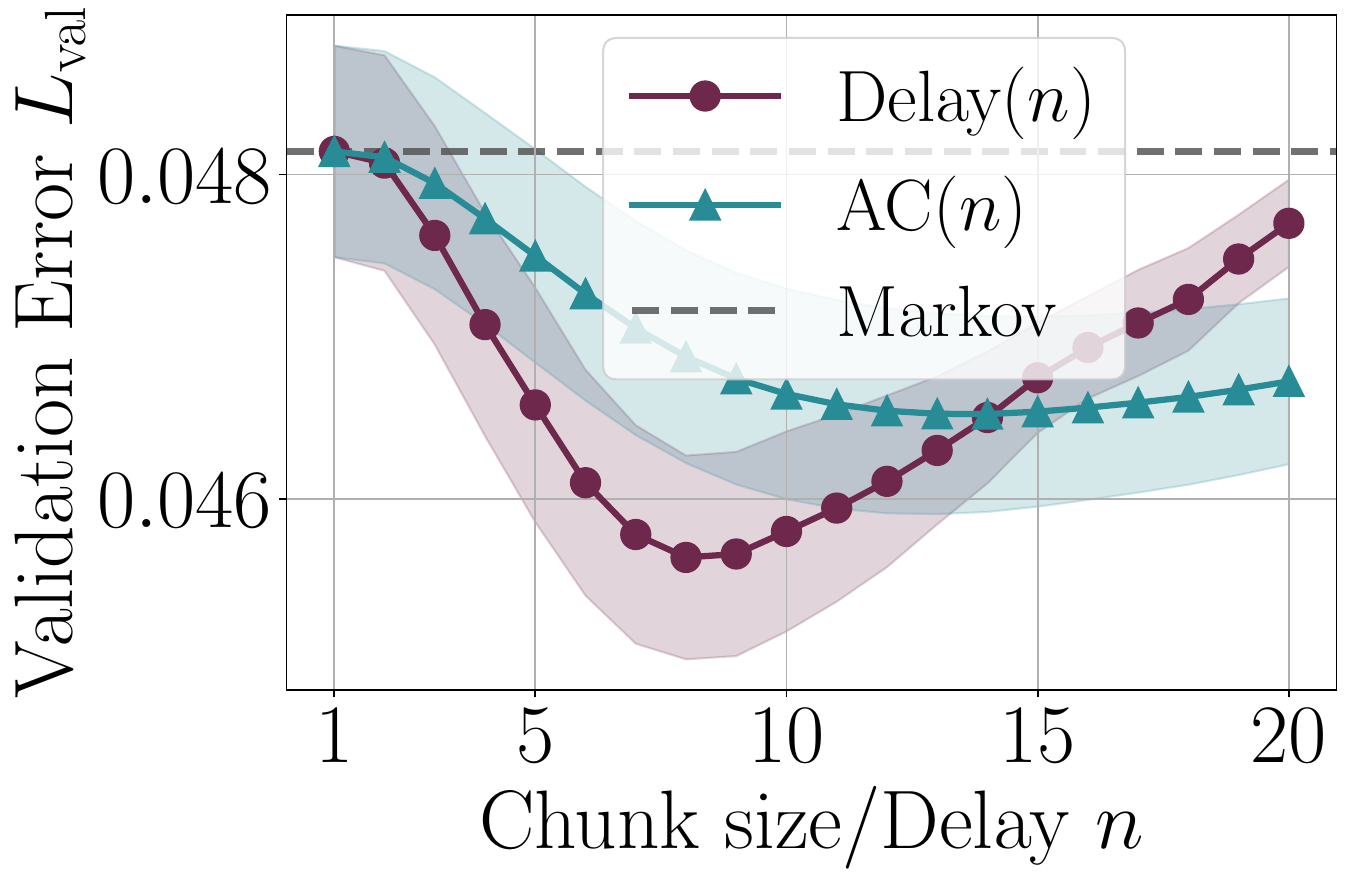}
    \end{minipage}
    \hfill
        \begin{minipage}[t]{0.23\textwidth}
        \centering
        \includegraphics[width=\linewidth]{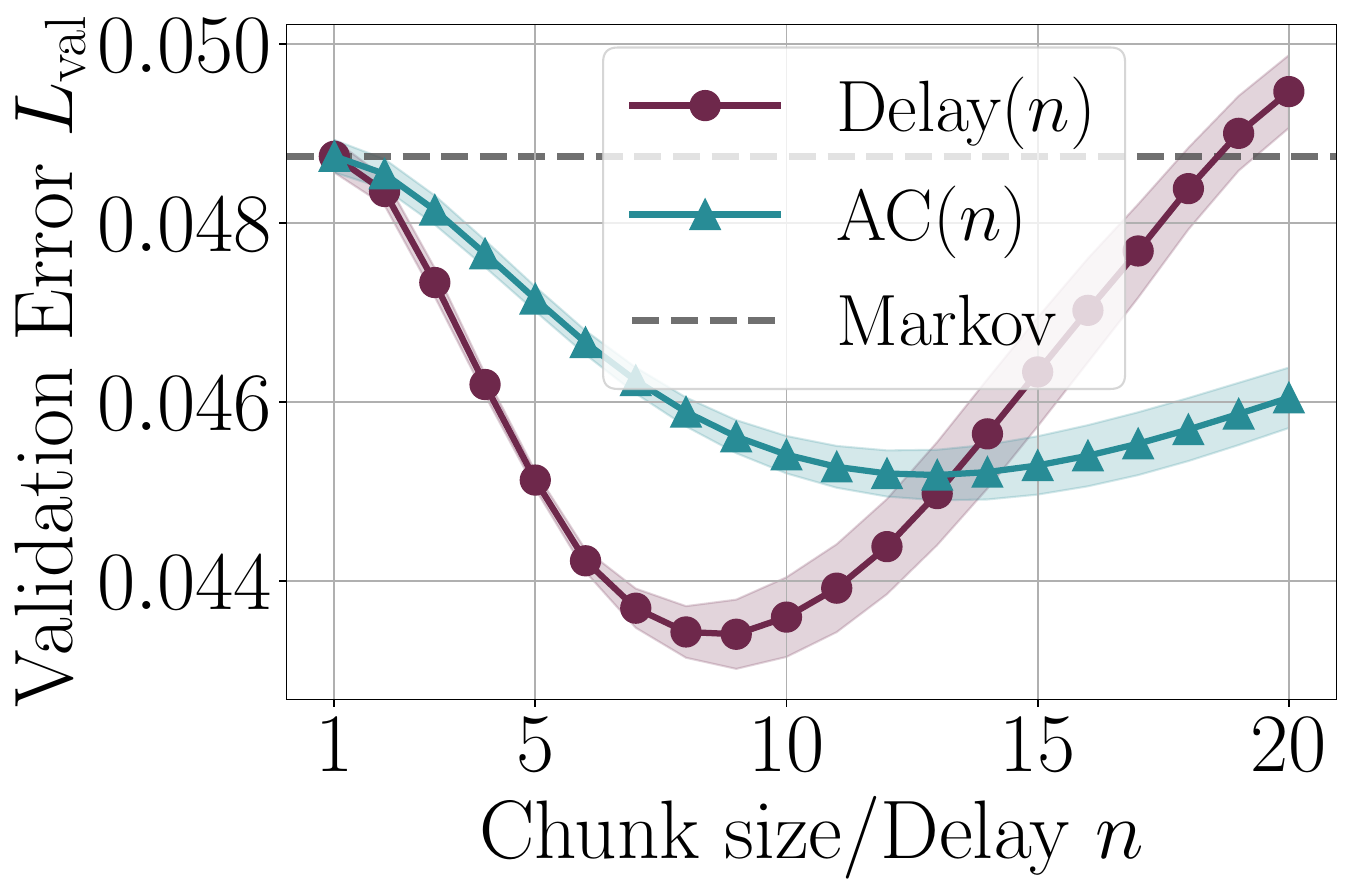}
    \end{minipage}
    \hfill
        \begin{minipage}[t]{0.23\textwidth}
        \centering
        \includegraphics[width=\linewidth]{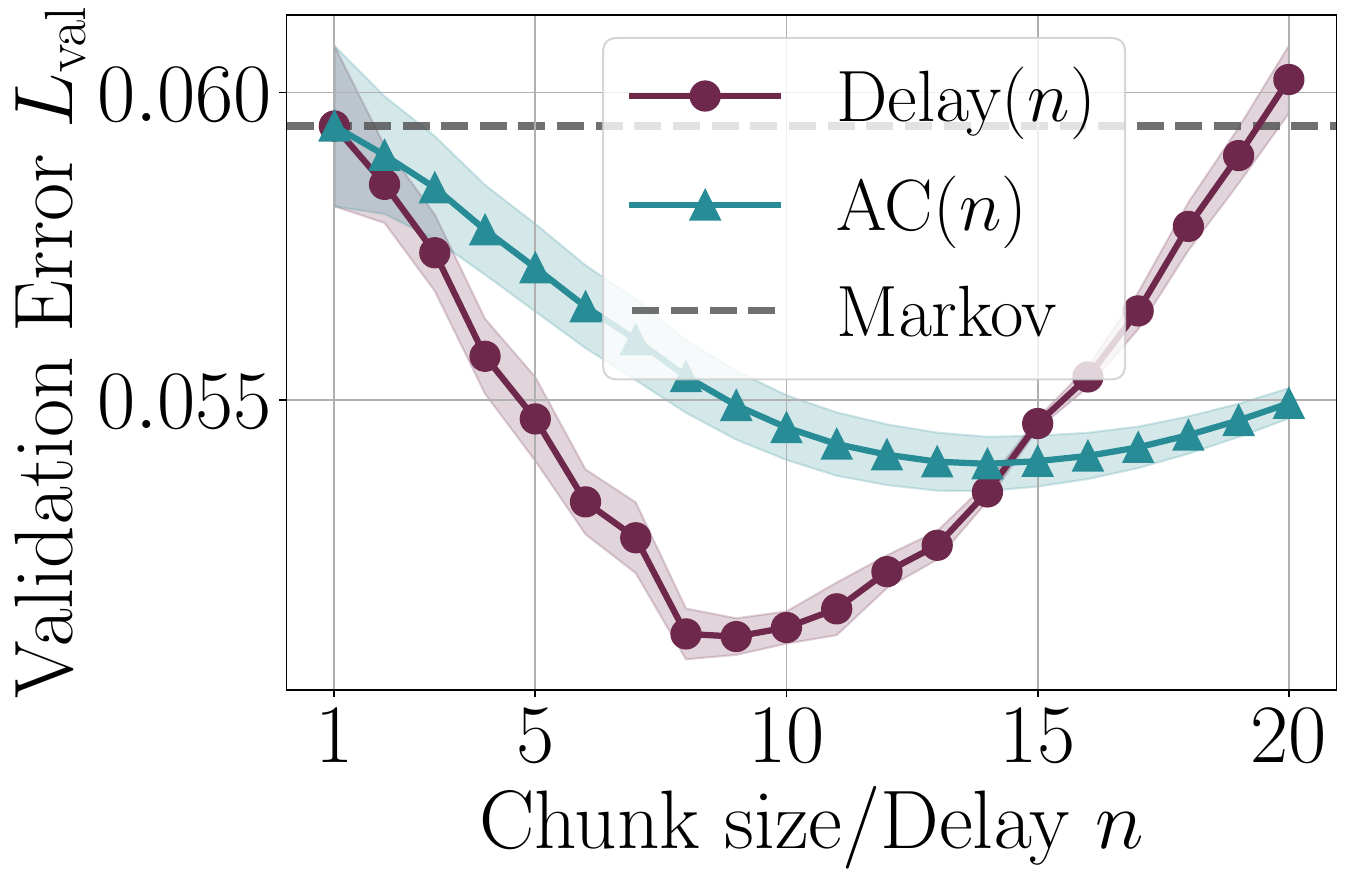}
    \end{minipage}
    \hfill
        \begin{minipage}[t]{0.23\textwidth}
        \centering
        \includegraphics[width=\linewidth]{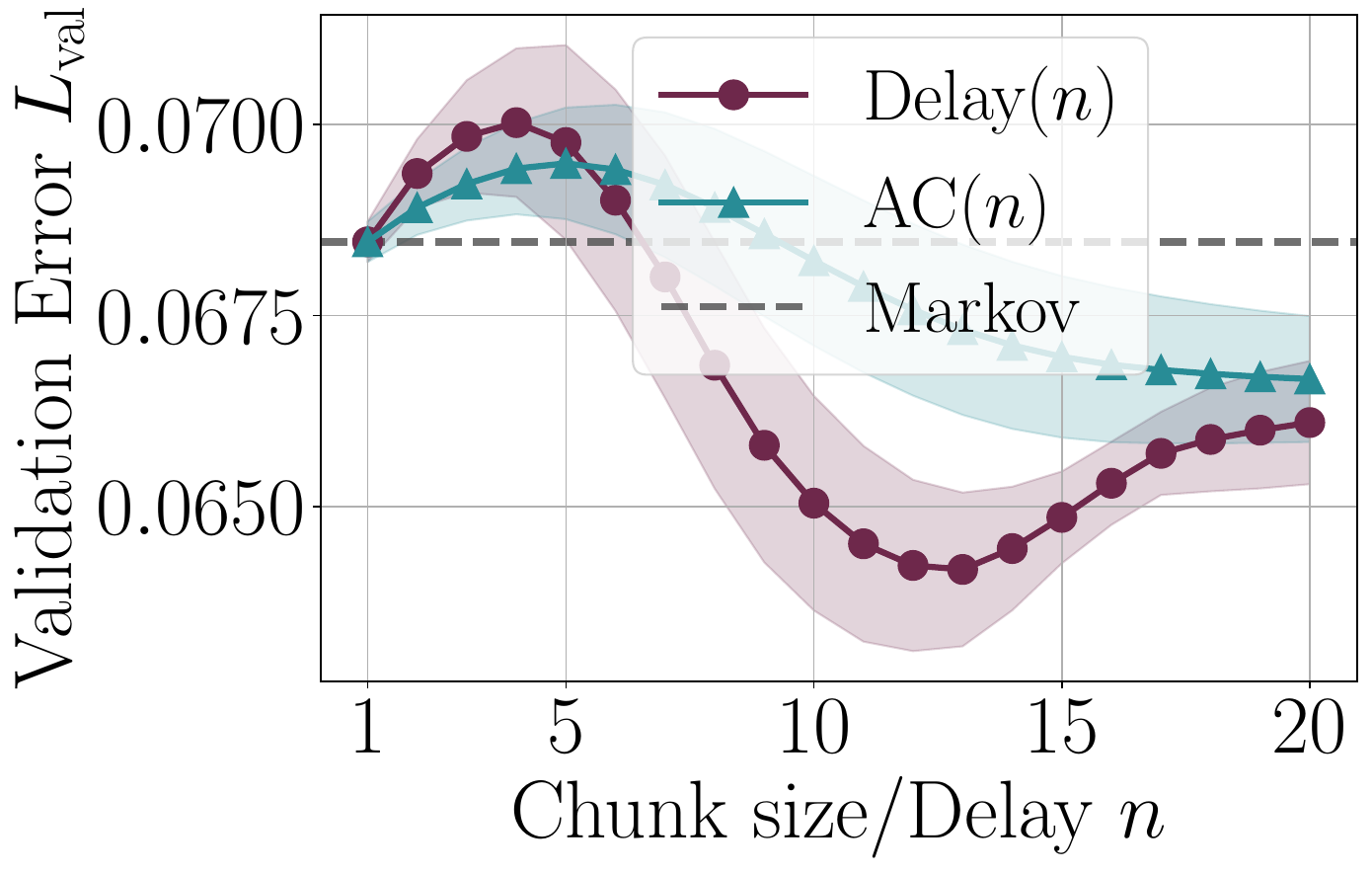}
    \end{minipage}
        \begin{minipage}[t]{0.23\textwidth}
            \includegraphics[width=\linewidth]{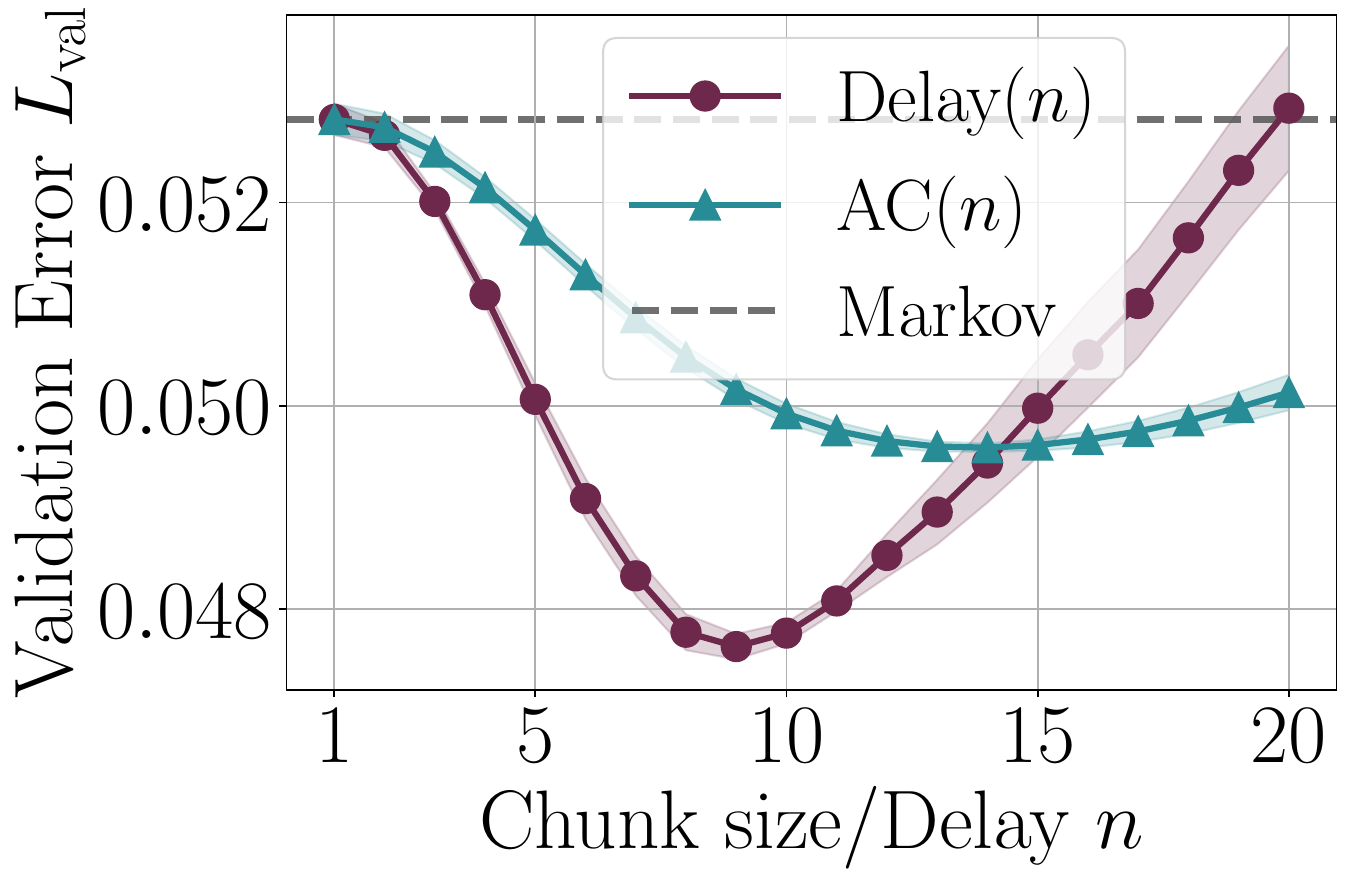}
    \end{minipage}
    \hfill
        \begin{minipage}[t]{0.23\textwidth}
        \centering
        \includegraphics[width=\linewidth]{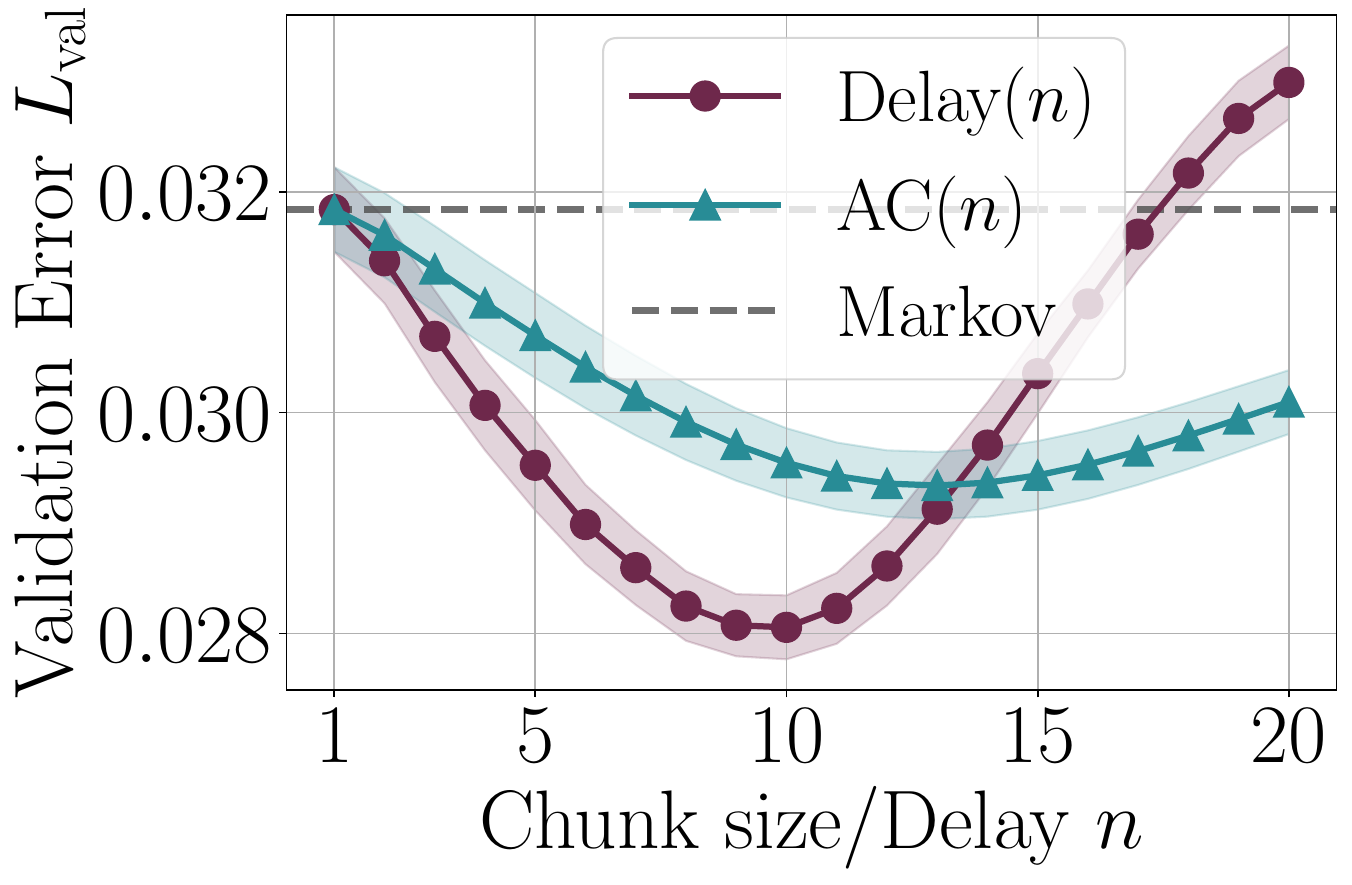}
    \end{minipage}
    \hfill
        \begin{minipage}[t]{0.23\textwidth}
        \centering
        \includegraphics[width=\linewidth]{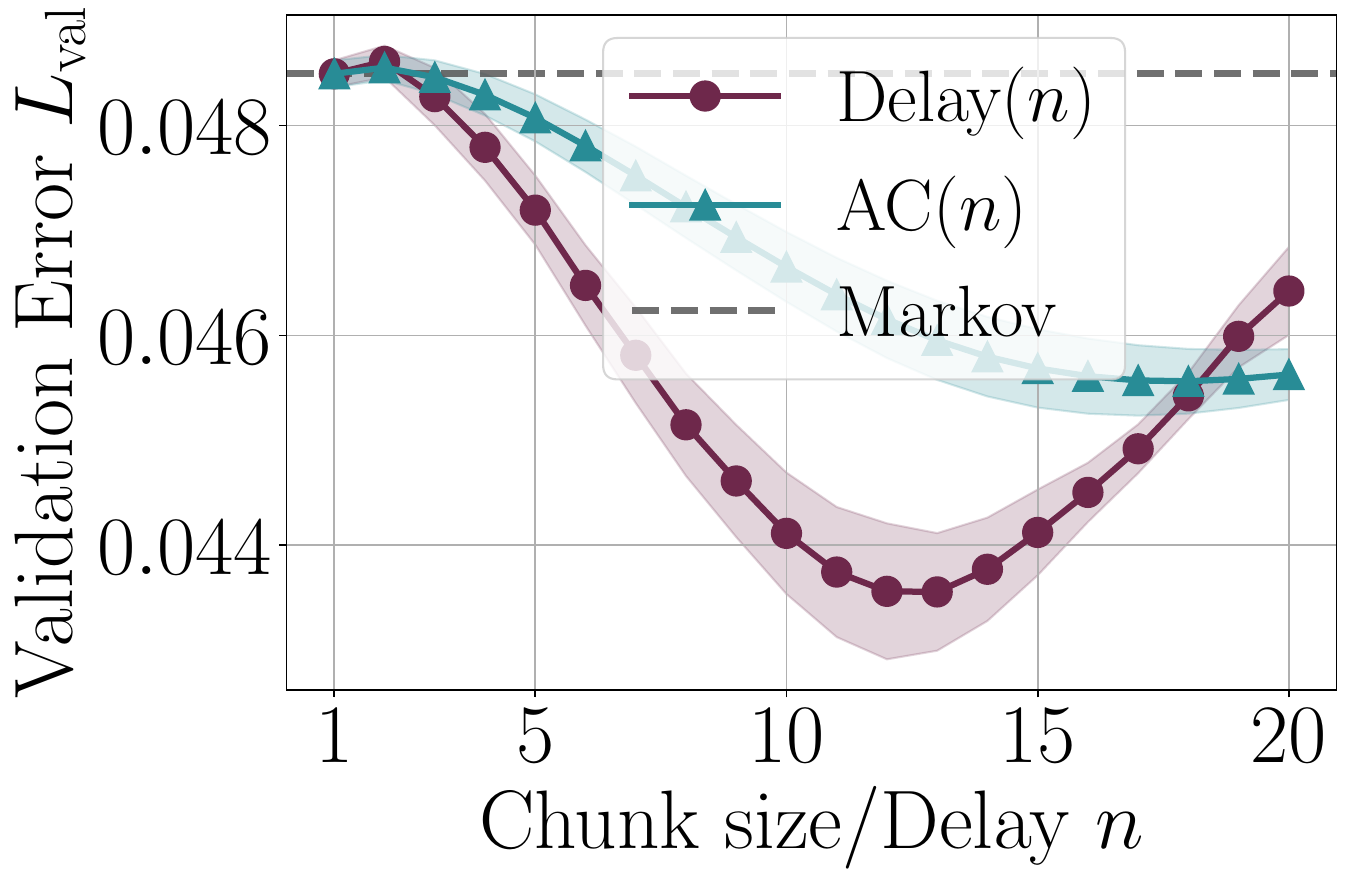}
    \end{minipage}
    \hfill
        \begin{minipage}[t]{0.23\textwidth}
        \centering
        \includegraphics[width=\linewidth]{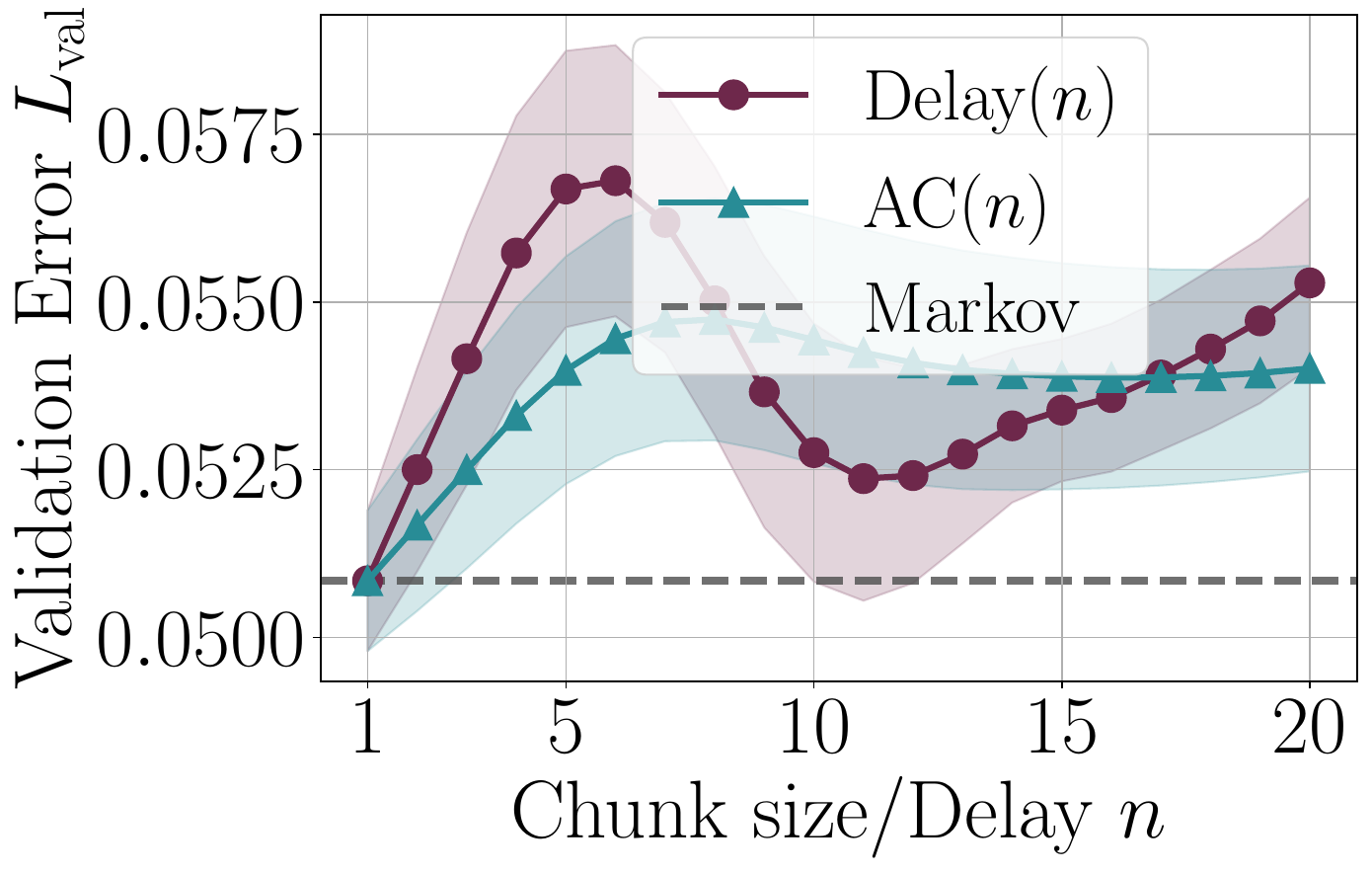}
    \end{minipage}
        \begin{minipage}[t]{0.23\textwidth}
            \includegraphics[width=\linewidth]{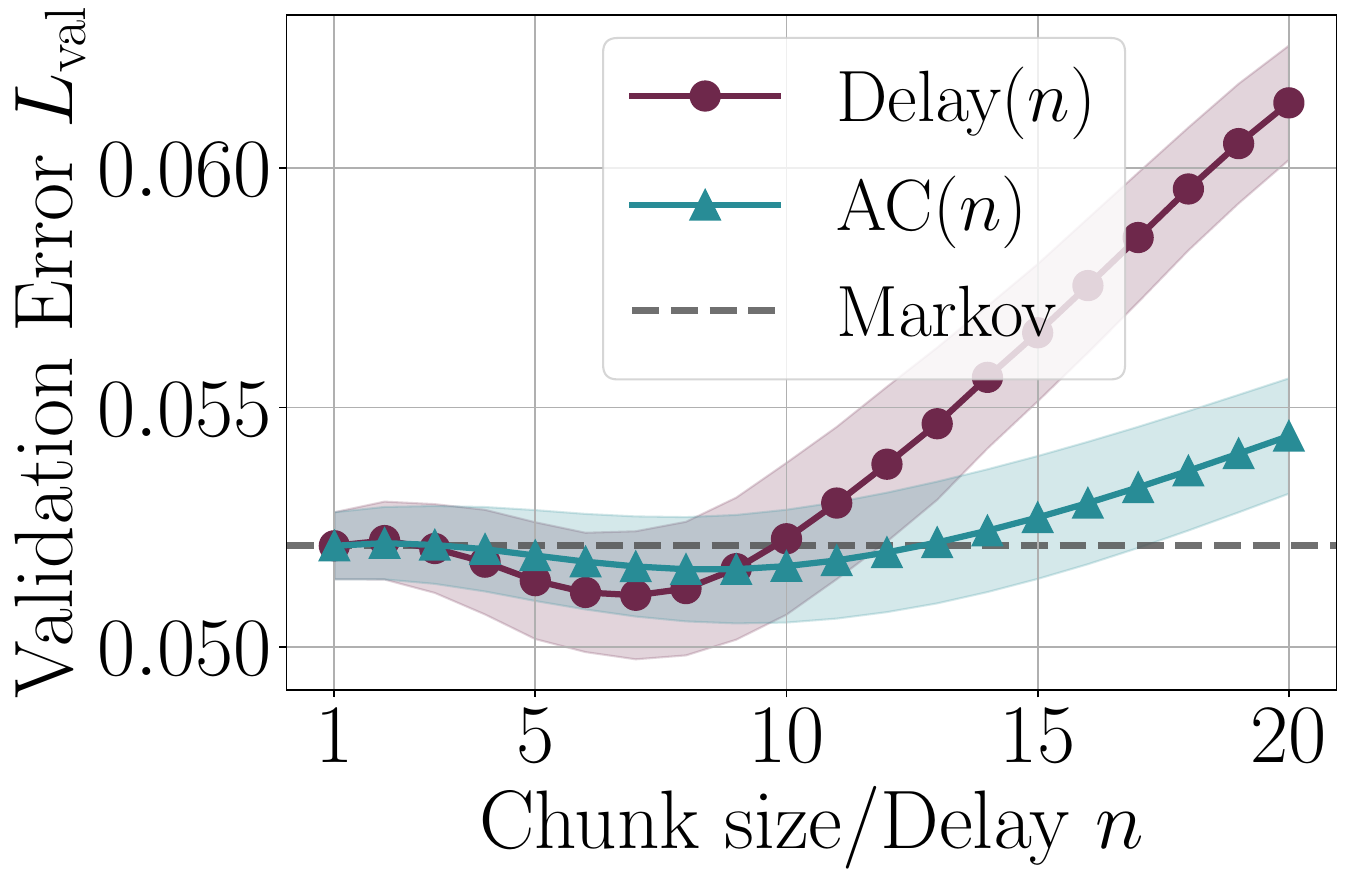}
    \end{minipage}
    \hfill
        \begin{minipage}[t]{0.23\textwidth}
        \centering
        \includegraphics[width=\linewidth]{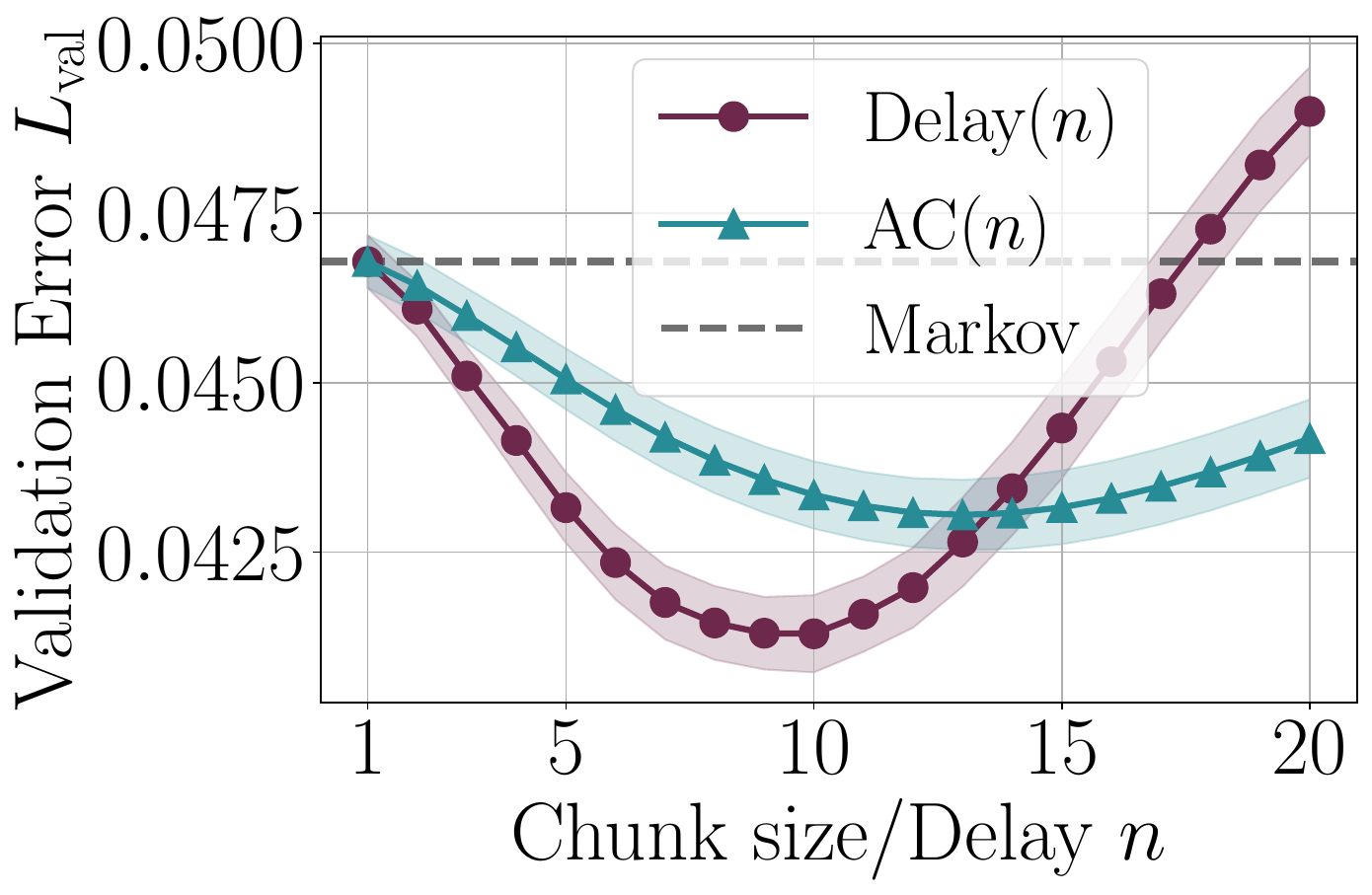}
    \end{minipage}
    \hfill
        \begin{minipage}[t]{0.23\textwidth}
        \centering
        \includegraphics[width=\linewidth]{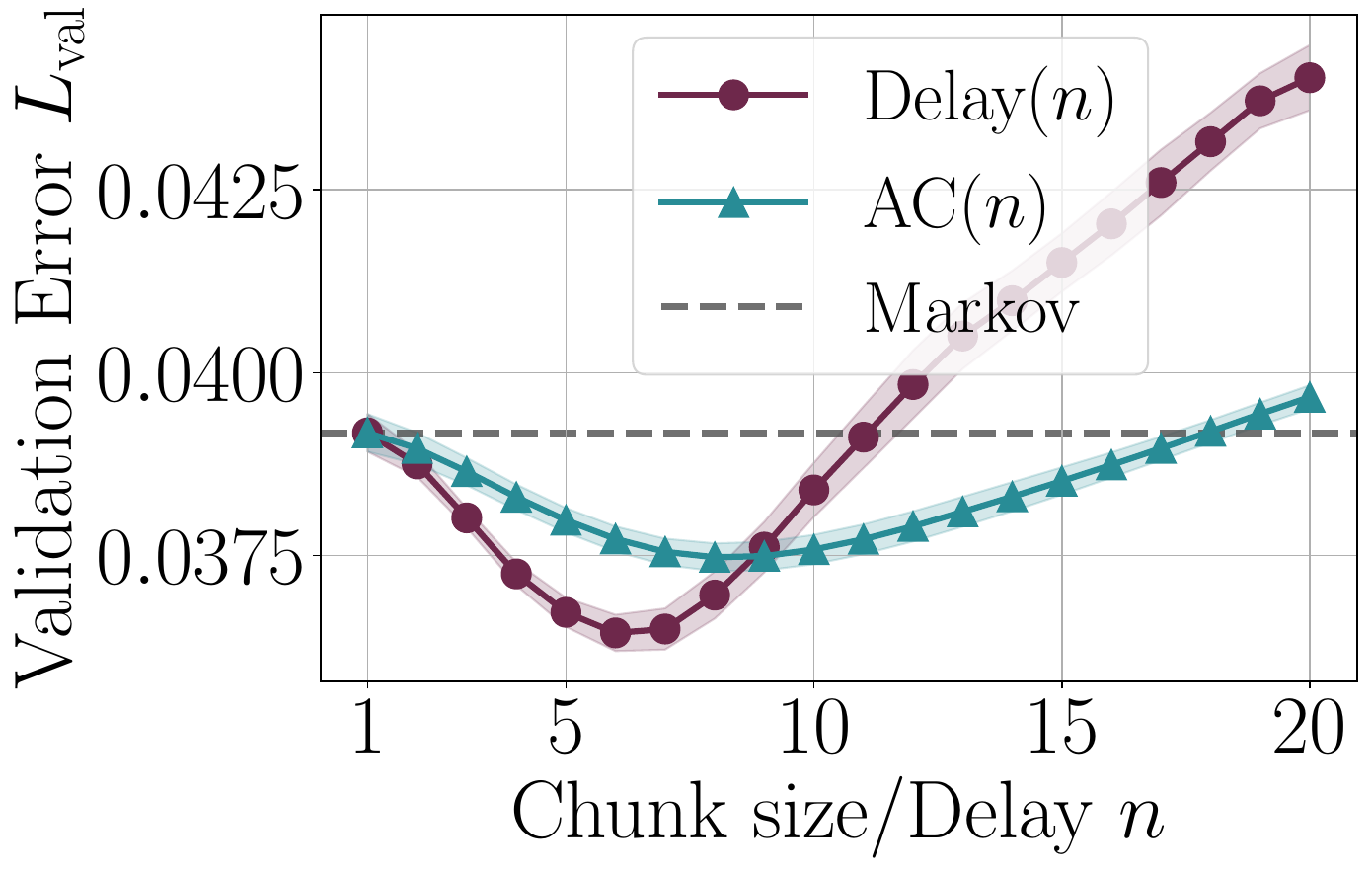}
    \end{minipage}
    \hfill
        \begin{minipage}[t]{0.23\textwidth}
        \centering
        \includegraphics[width=\linewidth]{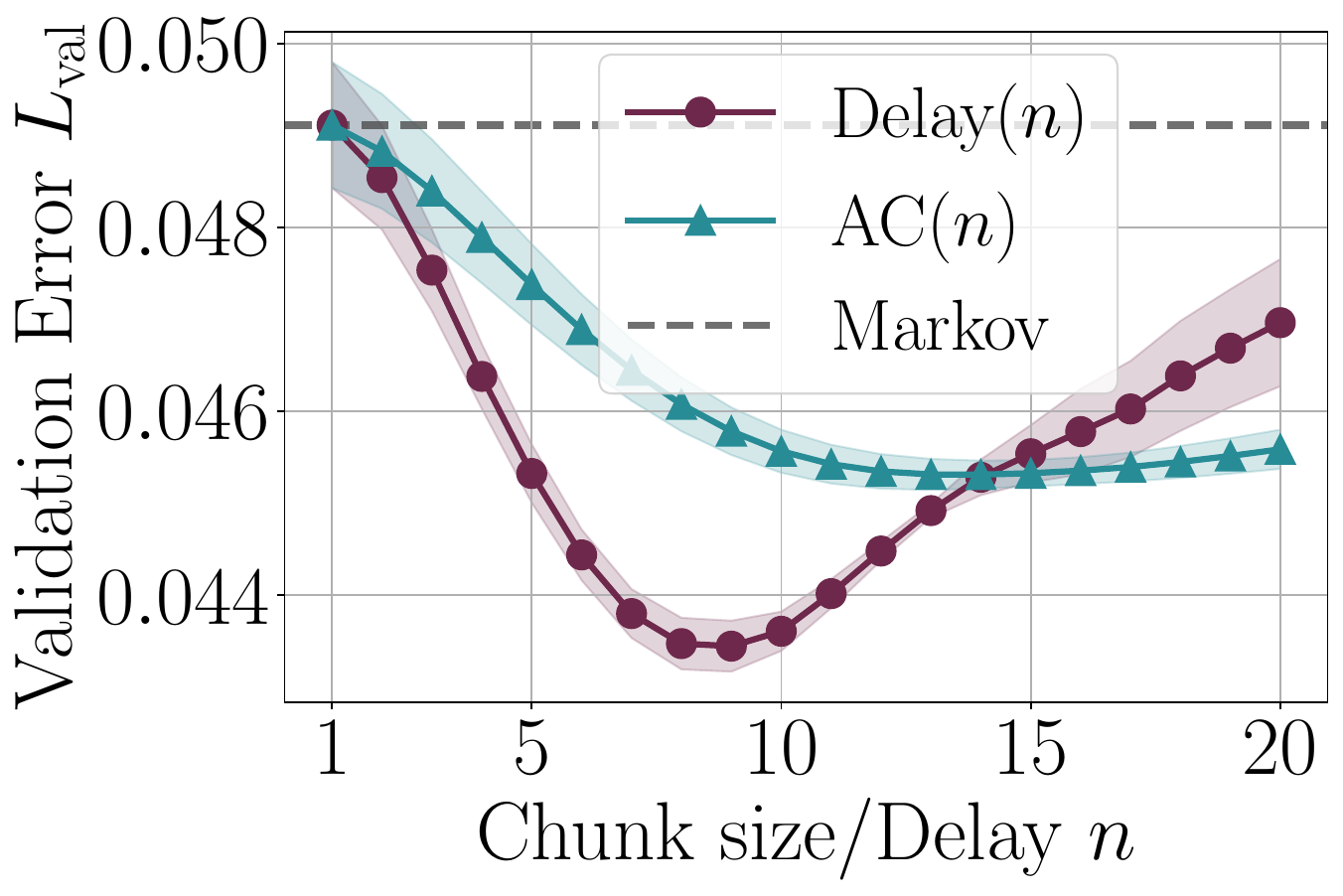}
    \end{minipage}
        \begin{minipage}[t]{0.23\textwidth}
            \includegraphics[width=\linewidth]{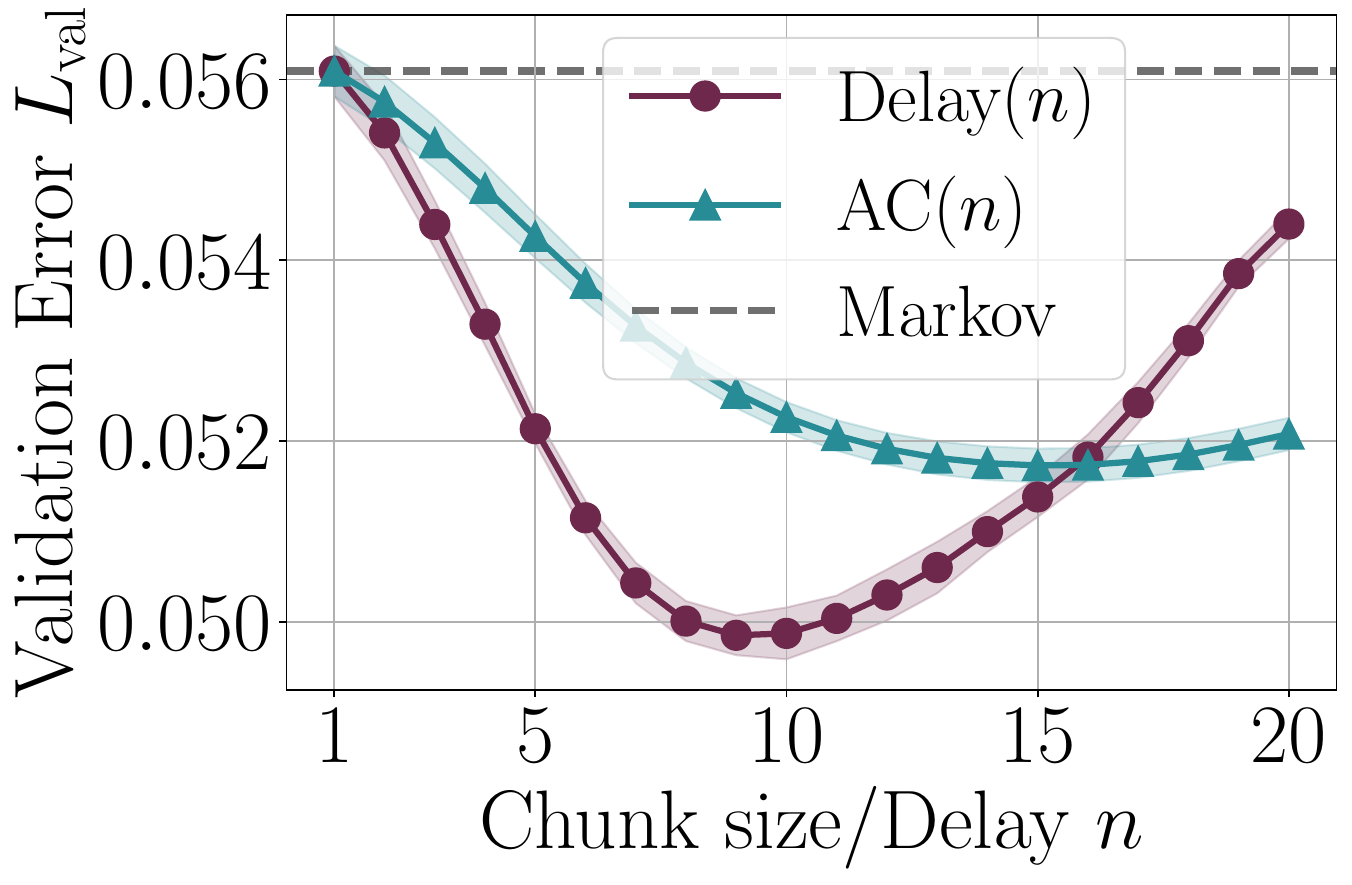}
    \end{minipage}
    \hfill
        \begin{minipage}[t]{0.23\textwidth}
        \centering
        \includegraphics[width=\linewidth]{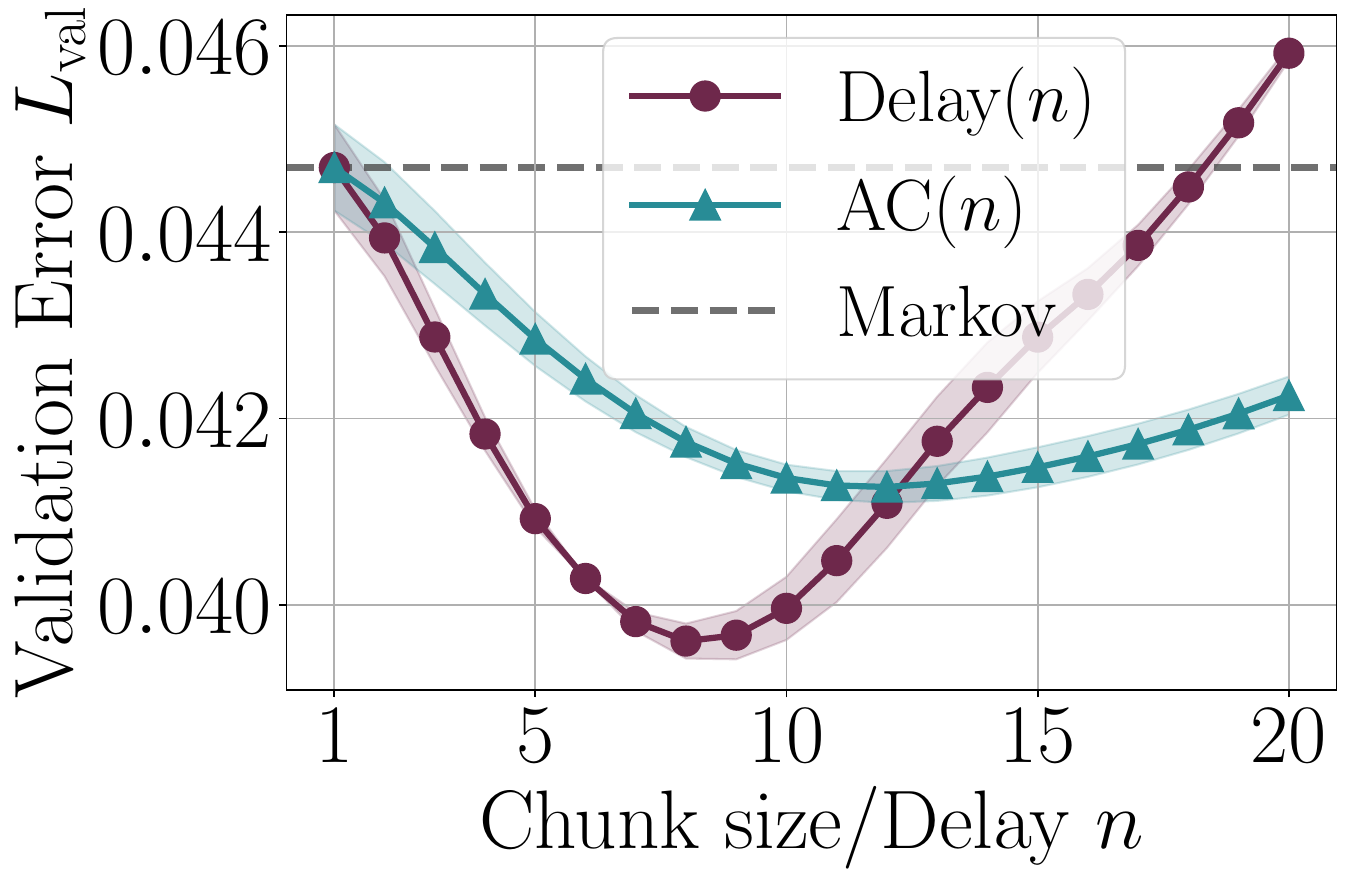}
    \end{minipage}
    \hfill
        \begin{minipage}[t]{0.23\textwidth}
        \centering
        \includegraphics[width=\linewidth]{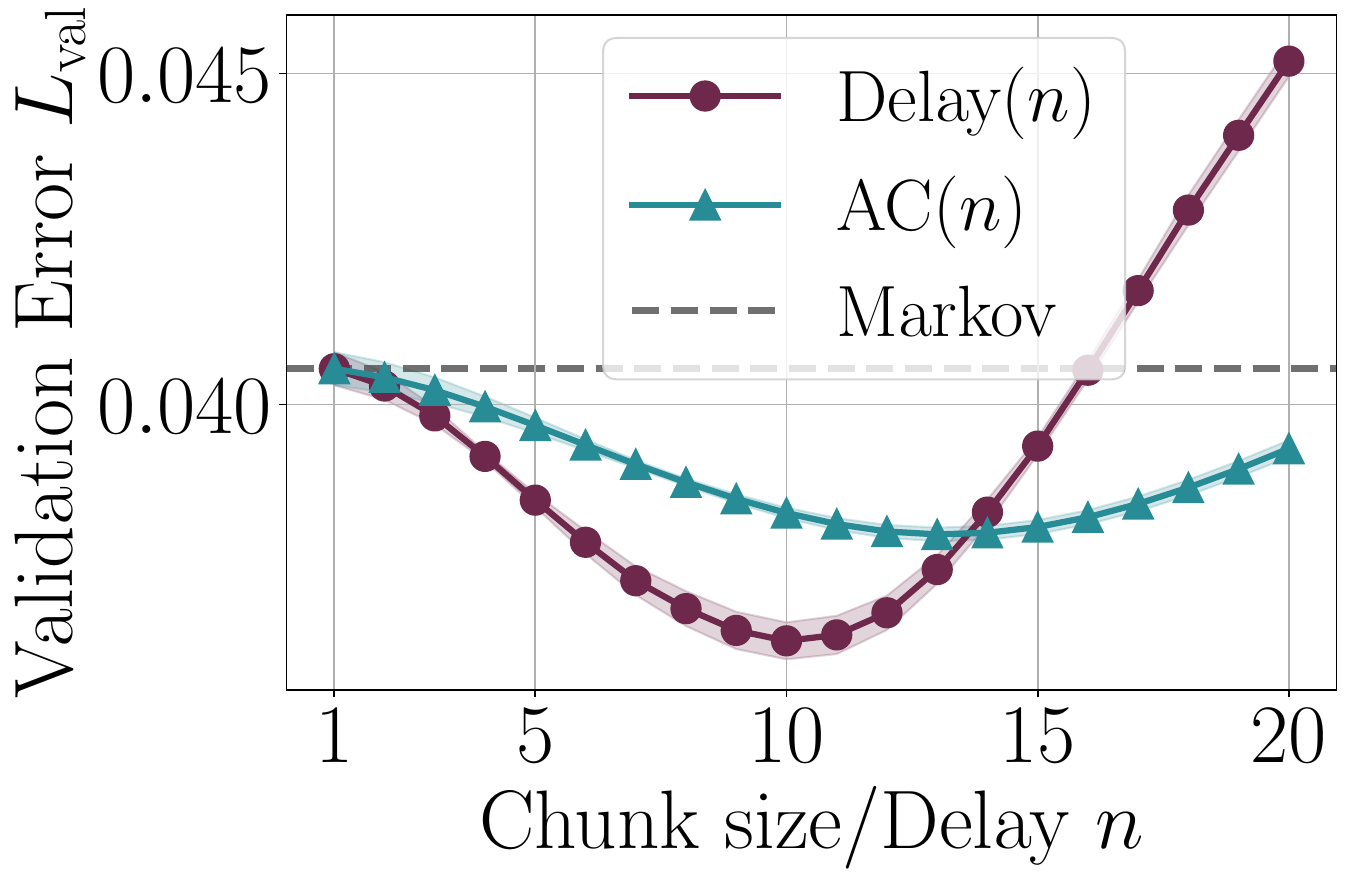}
    \end{minipage}
    \hfill
        \begin{minipage}[t]{0.23\textwidth}
        \centering
        \includegraphics[width=\linewidth]{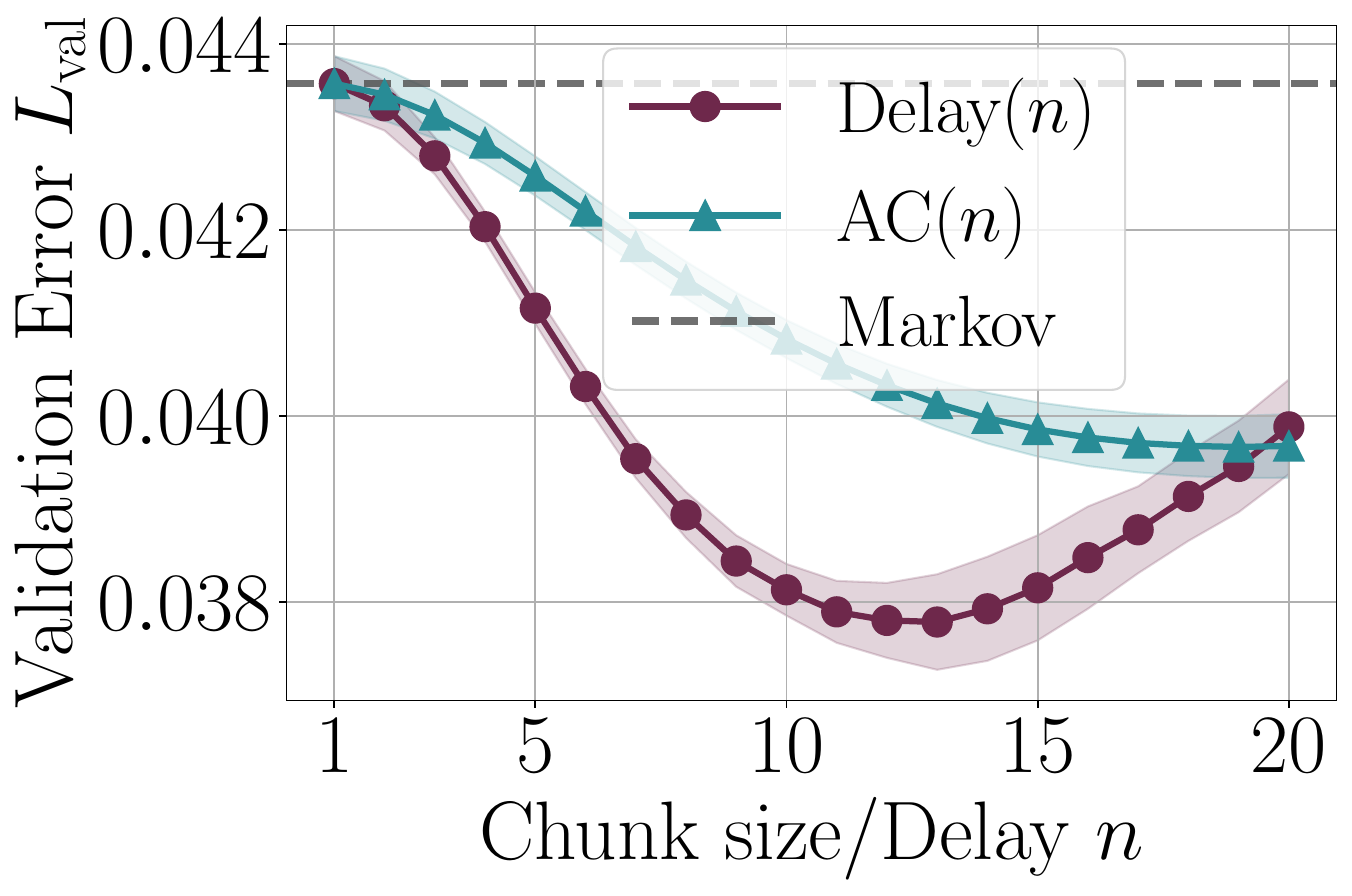}
    \end{minipage}
        \begin{minipage}[t]{0.23\textwidth}
            \includegraphics[width=\linewidth]{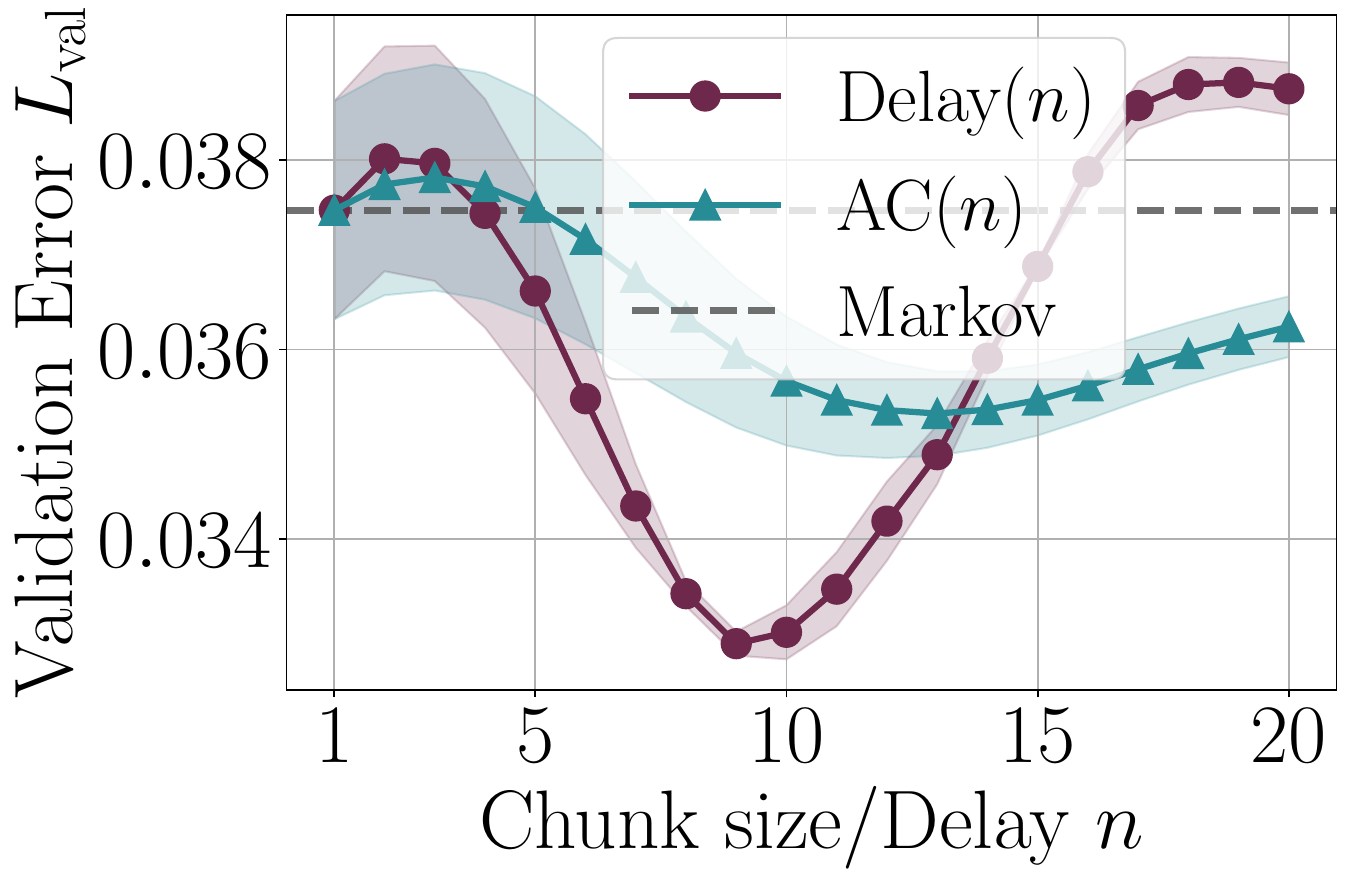}
    \end{minipage}
    \hfill
        \begin{minipage}[t]{0.23\textwidth}
        \centering
        \includegraphics[width=\linewidth]{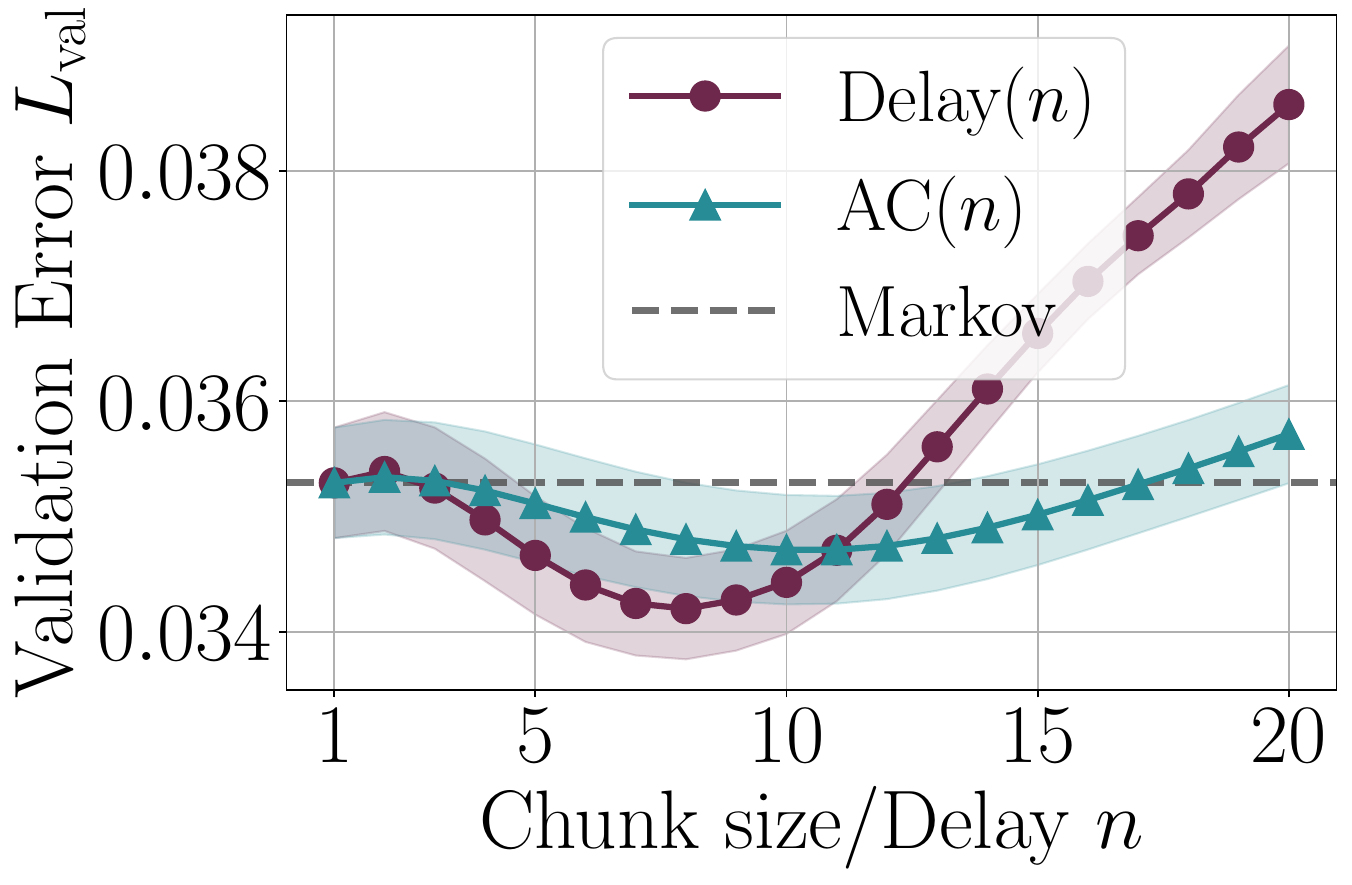}
    \end{minipage}
    \hfill
        \begin{minipage}[t]{0.23\textwidth}
        \centering
        \includegraphics[width=\linewidth]{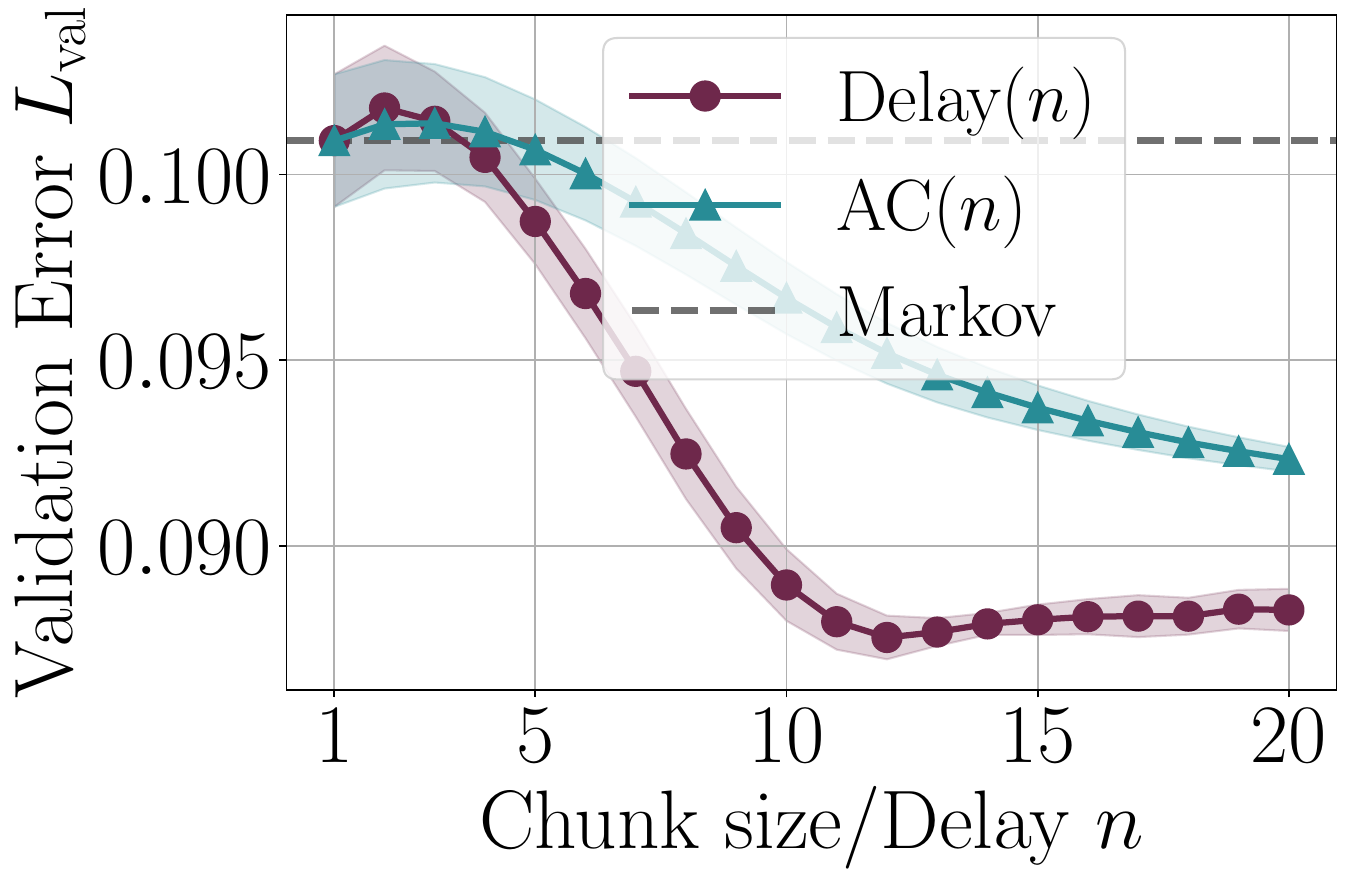}
    \end{minipage}
    \hfill
        \begin{minipage}[t]{0.23\textwidth}
        \centering
        \includegraphics[width=\linewidth]{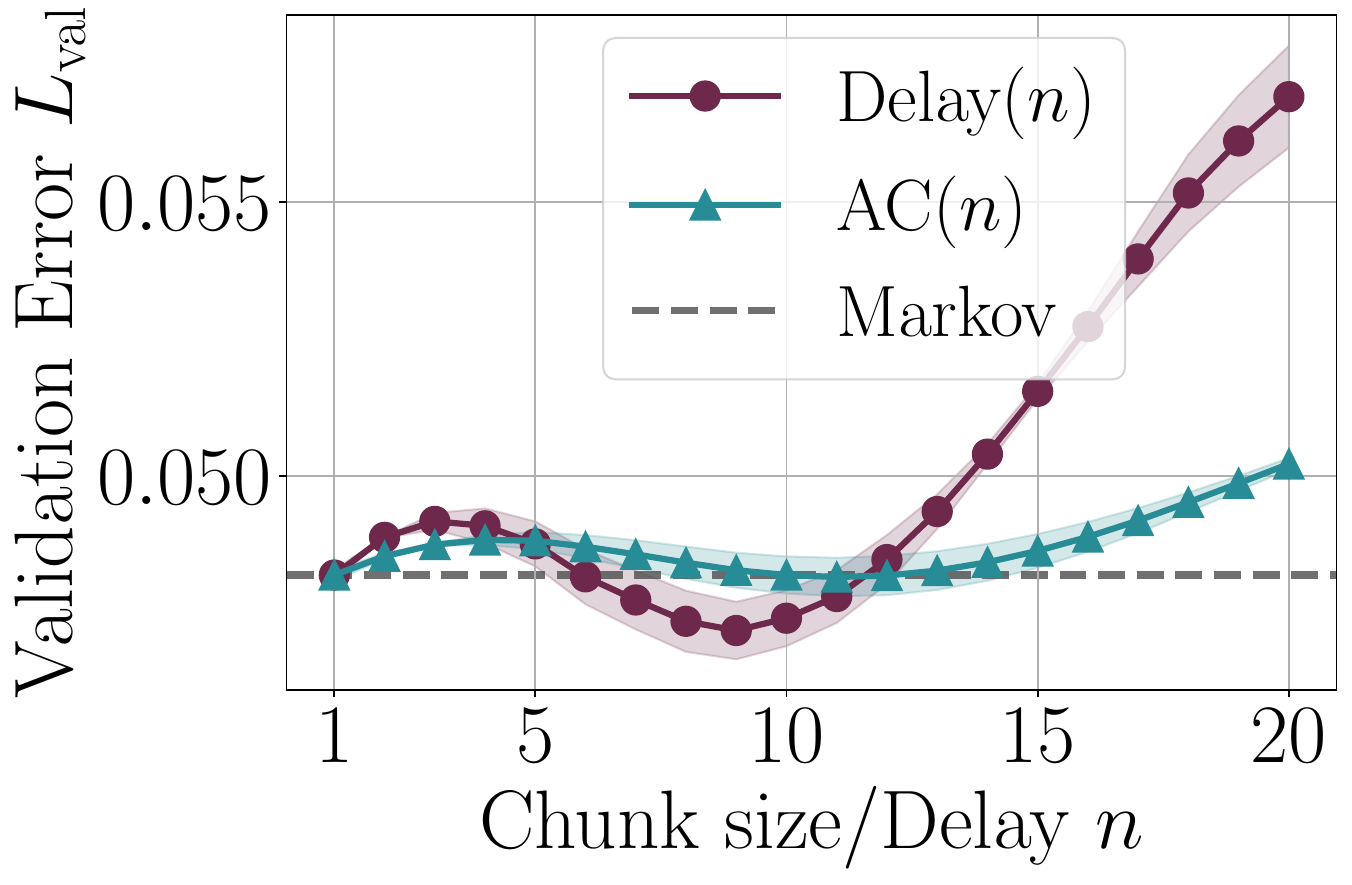}
    \end{minipage}
        \begin{minipage}[t]{0.23\textwidth}
            \includegraphics[width=\linewidth]{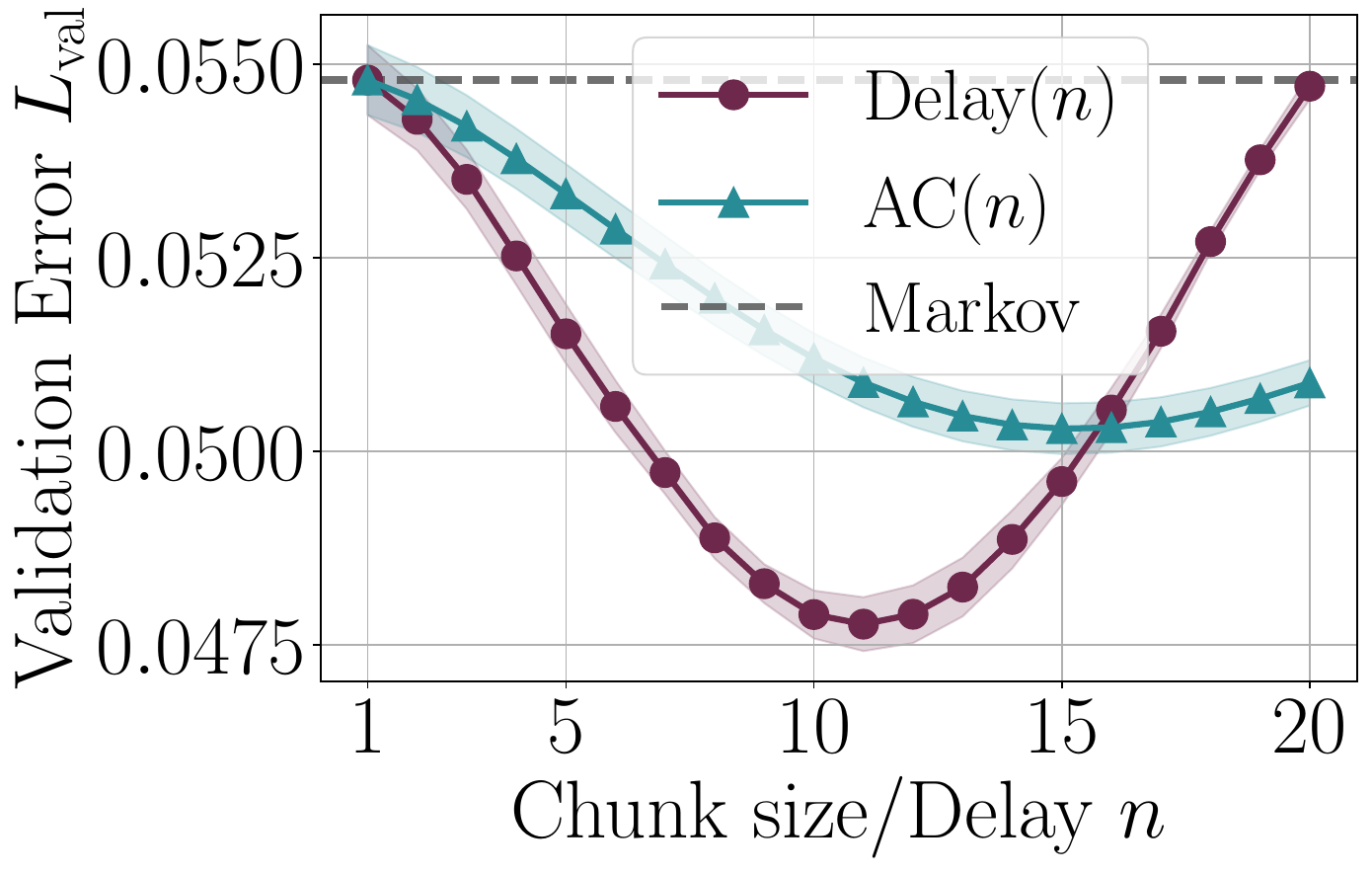}
    \end{minipage}
    \hfill
        \begin{minipage}[t]{0.23\textwidth}
        \centering
        \includegraphics[width=\linewidth]{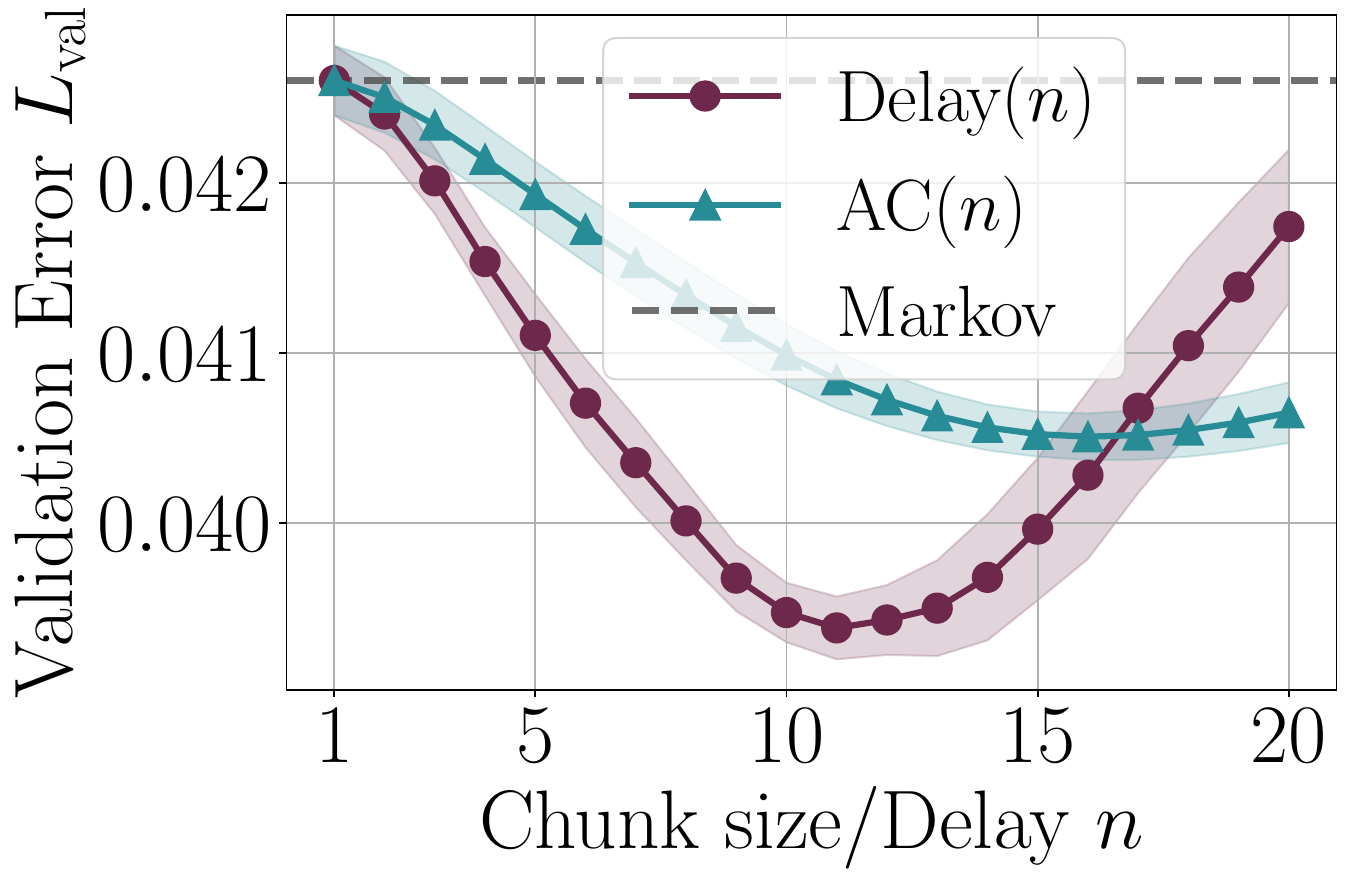}
    \end{minipage}
    \hfill
        \begin{minipage}[t]{0.23\textwidth}
        \centering
        \includegraphics[width=\linewidth]{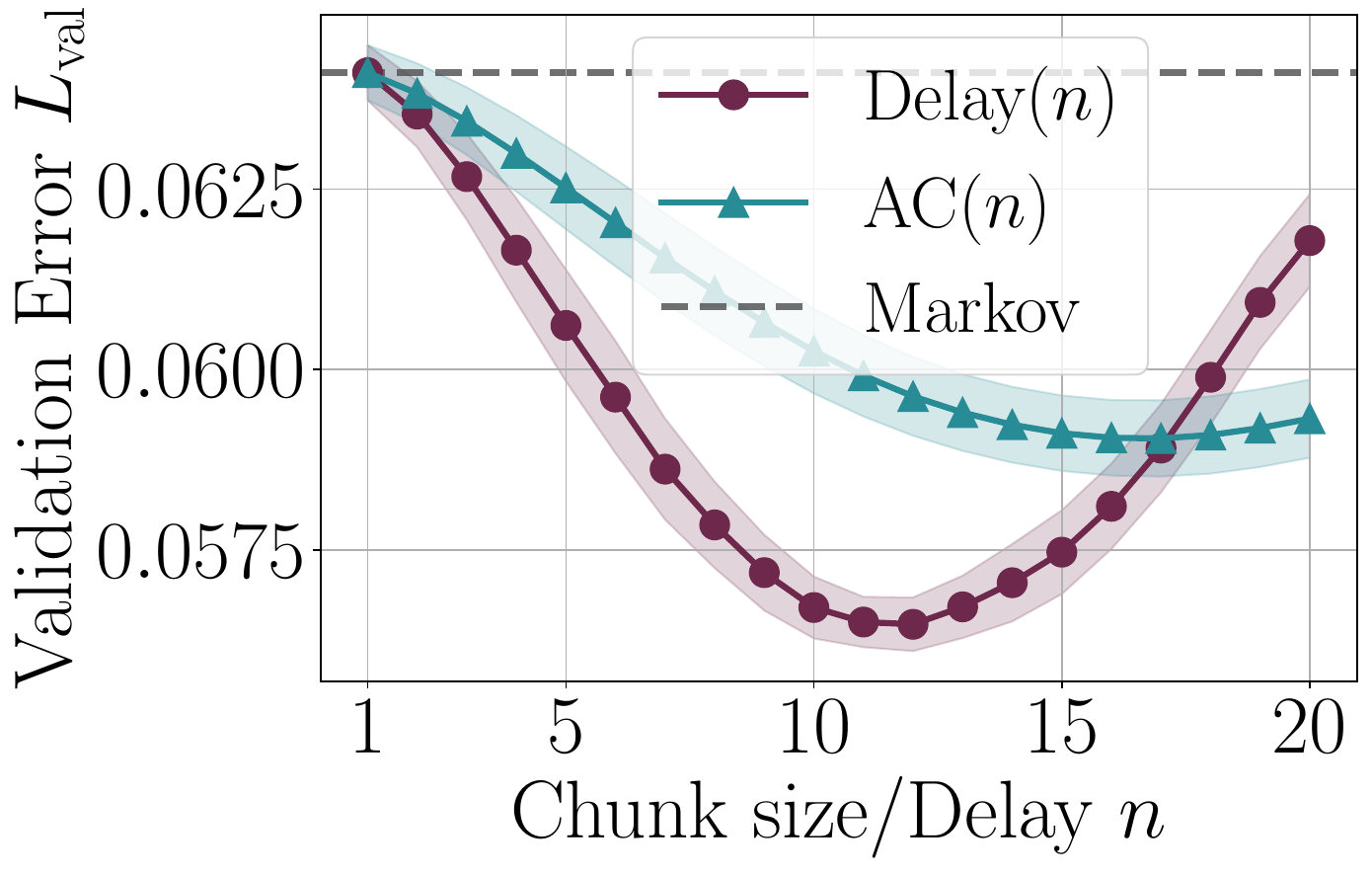}
    \end{minipage}
    \hfill
        \begin{minipage}[t]{0.23\textwidth}
        \centering
        \includegraphics[width=\linewidth]{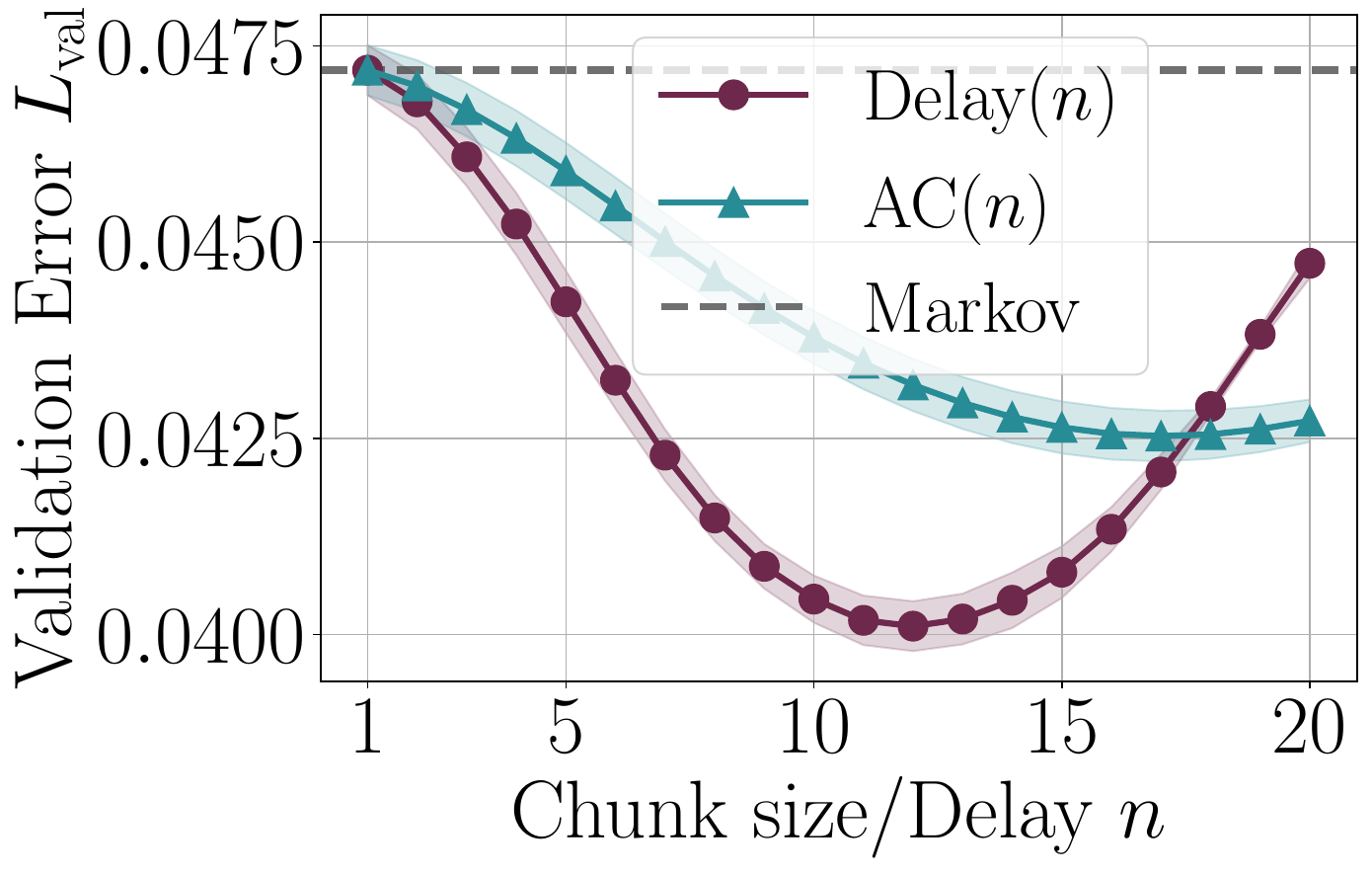}
    \end{minipage}
        \begin{minipage}[t]{0.23\textwidth}
            \includegraphics[width=\linewidth]{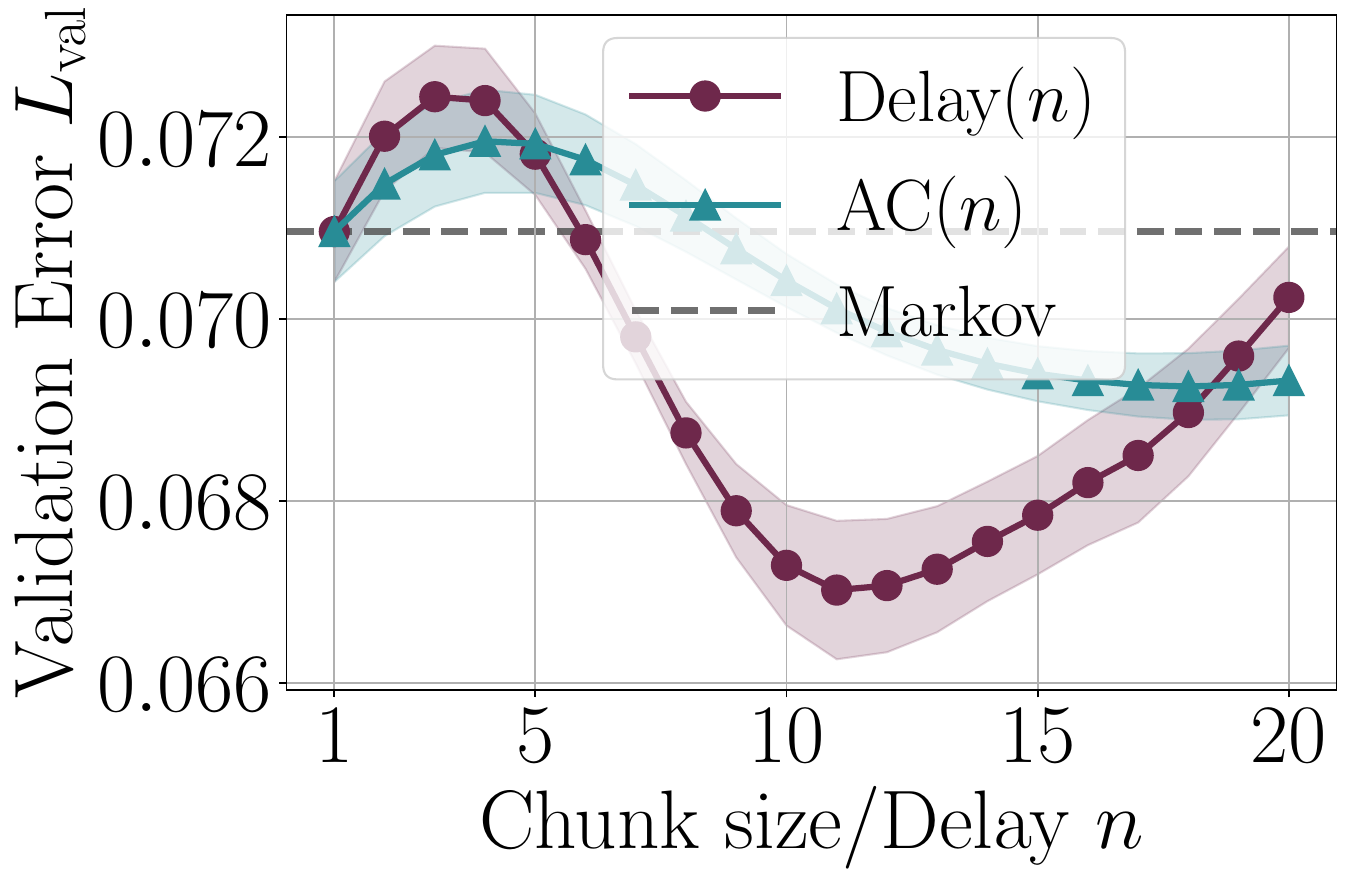}
    \end{minipage}
    \hfill
        \begin{minipage}[t]{0.23\textwidth}
        \centering
        \includegraphics[width=\linewidth]{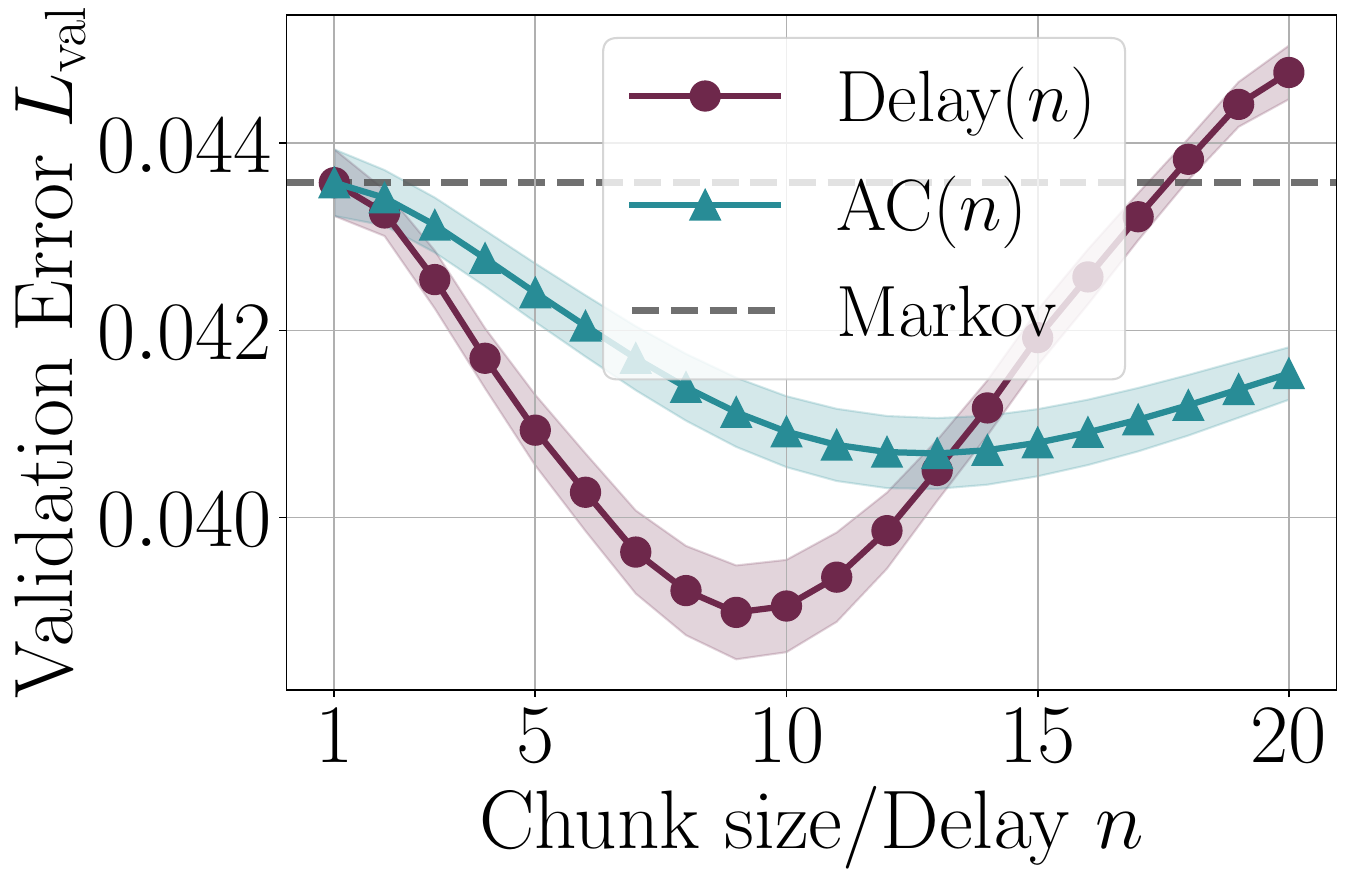}
    \end{minipage}
    \hfill
        \begin{minipage}[t]{0.23\textwidth}
        \centering
        \includegraphics[width=\linewidth]{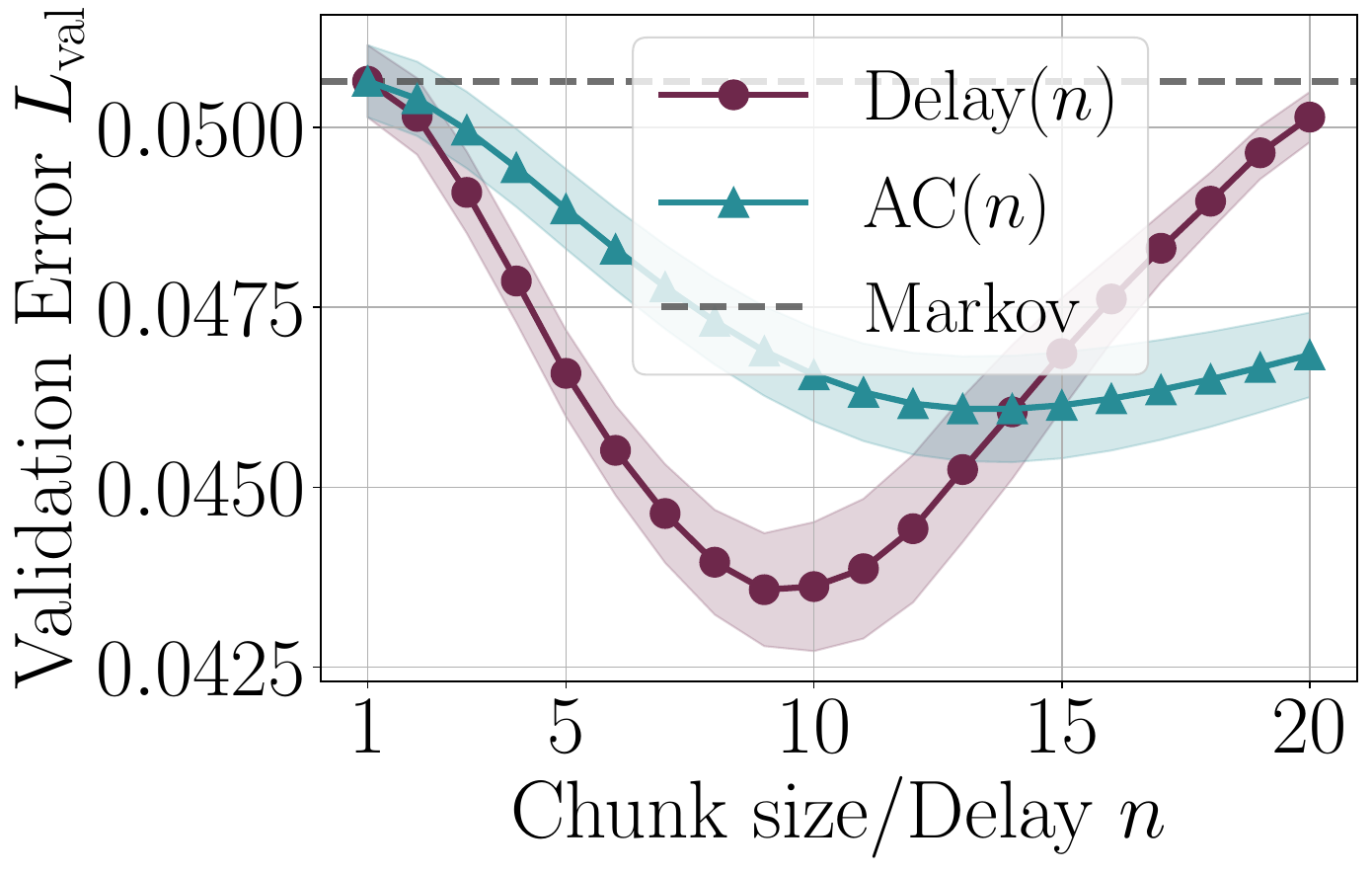}
    \end{minipage}
    \hfill
        \begin{minipage}[t]{0.23\textwidth}
        \centering
        \includegraphics[width=\linewidth]{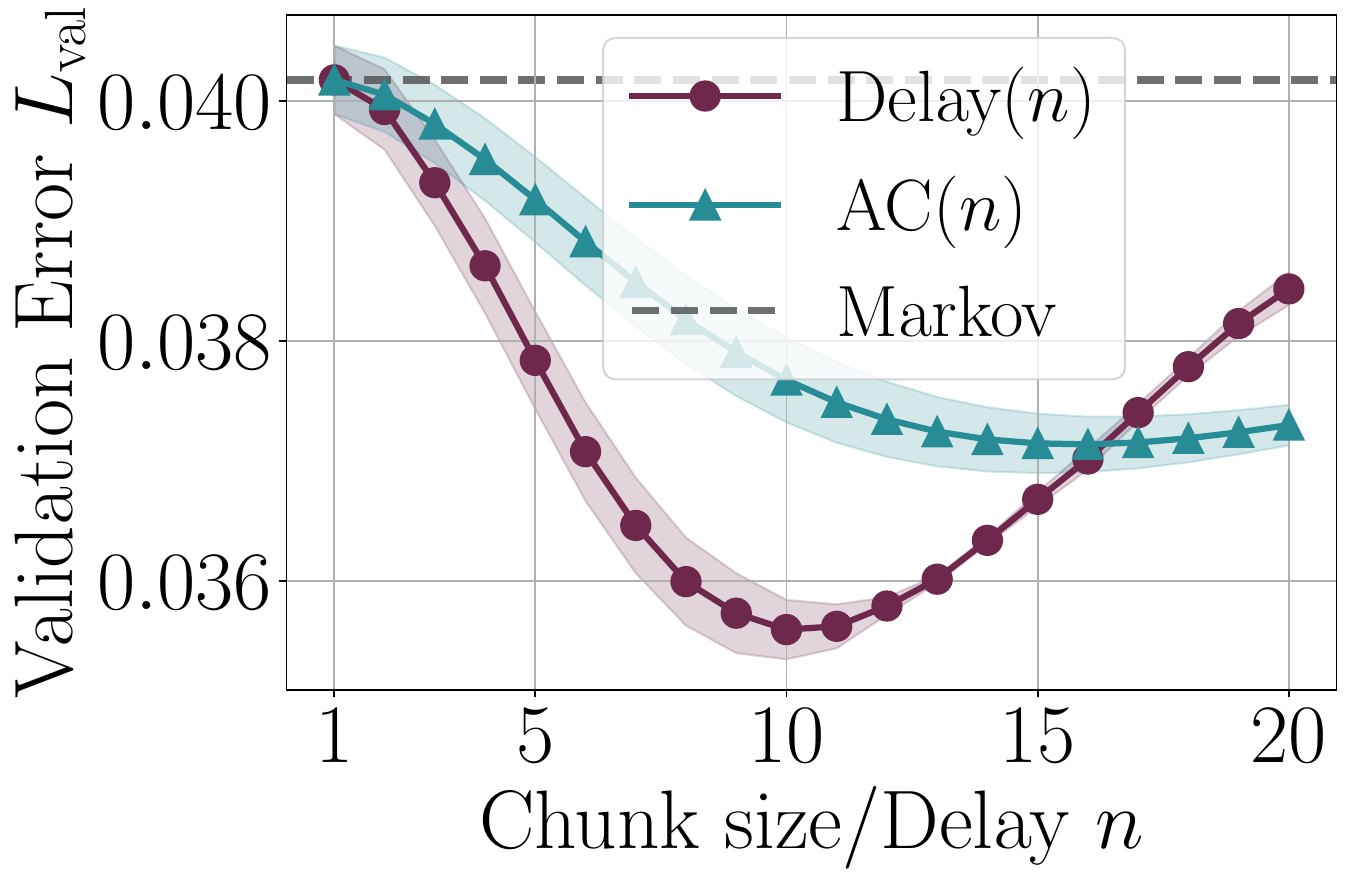}
    \end{minipage}
        \begin{minipage}[t]{0.23\textwidth}
            \includegraphics[width=\linewidth]{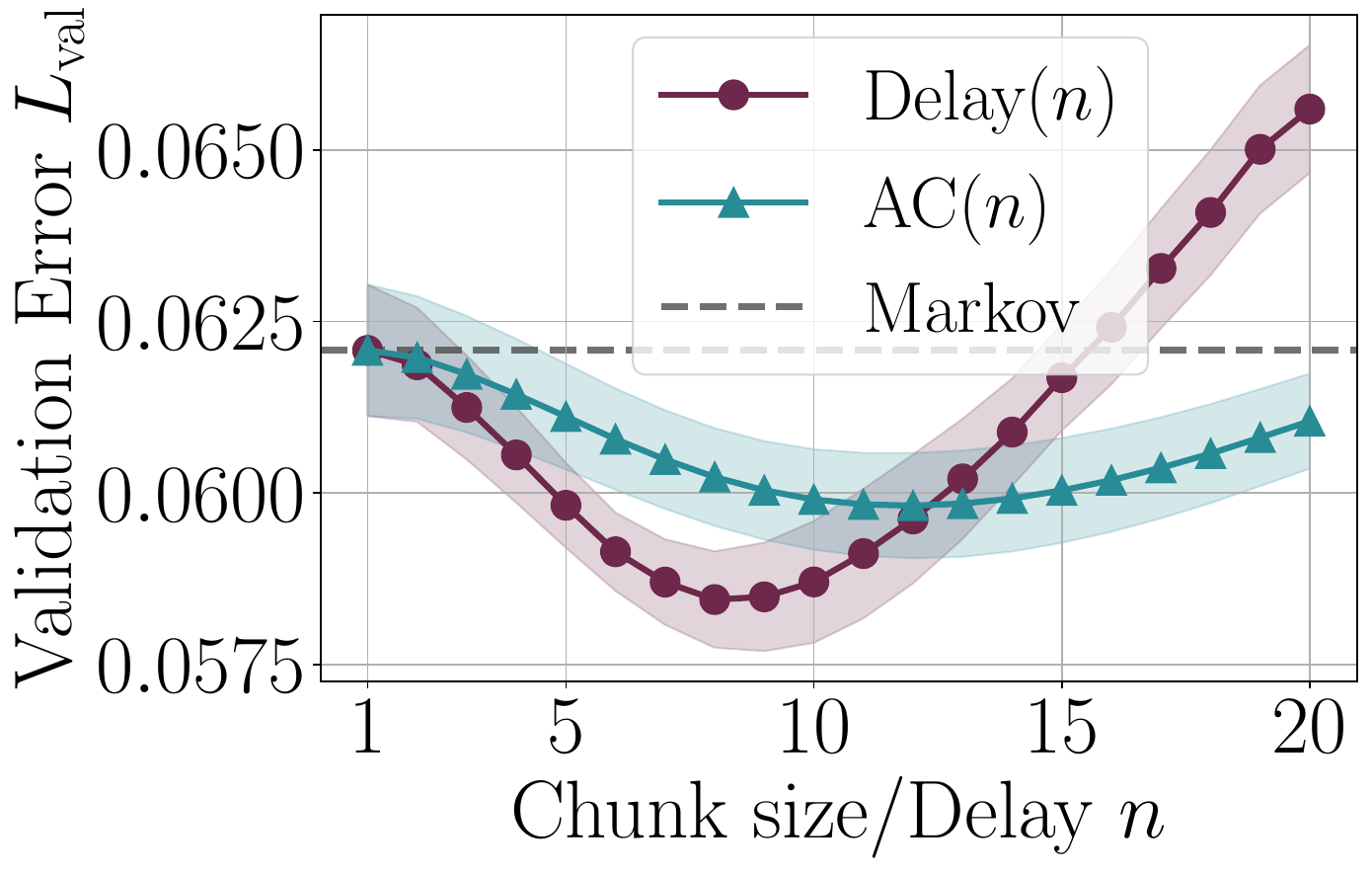}
    \end{minipage}
    \hfill
        \begin{minipage}[t]{0.23\textwidth}
        \centering
        \includegraphics[width=\linewidth]{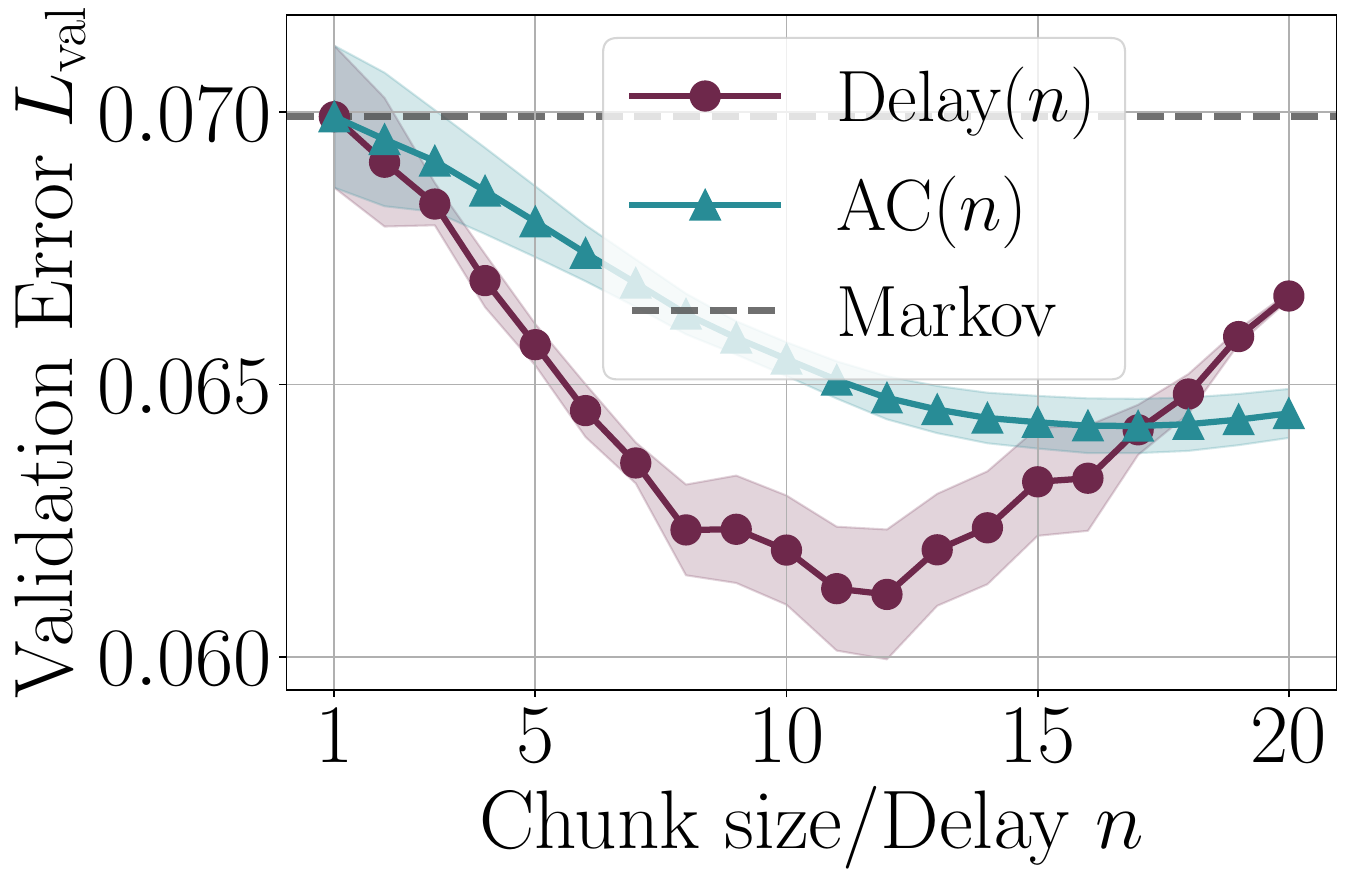}
    \end{minipage}
    \hfill
        \begin{minipage}[t]{0.23\textwidth}
        \centering
        \includegraphics[width=\linewidth]{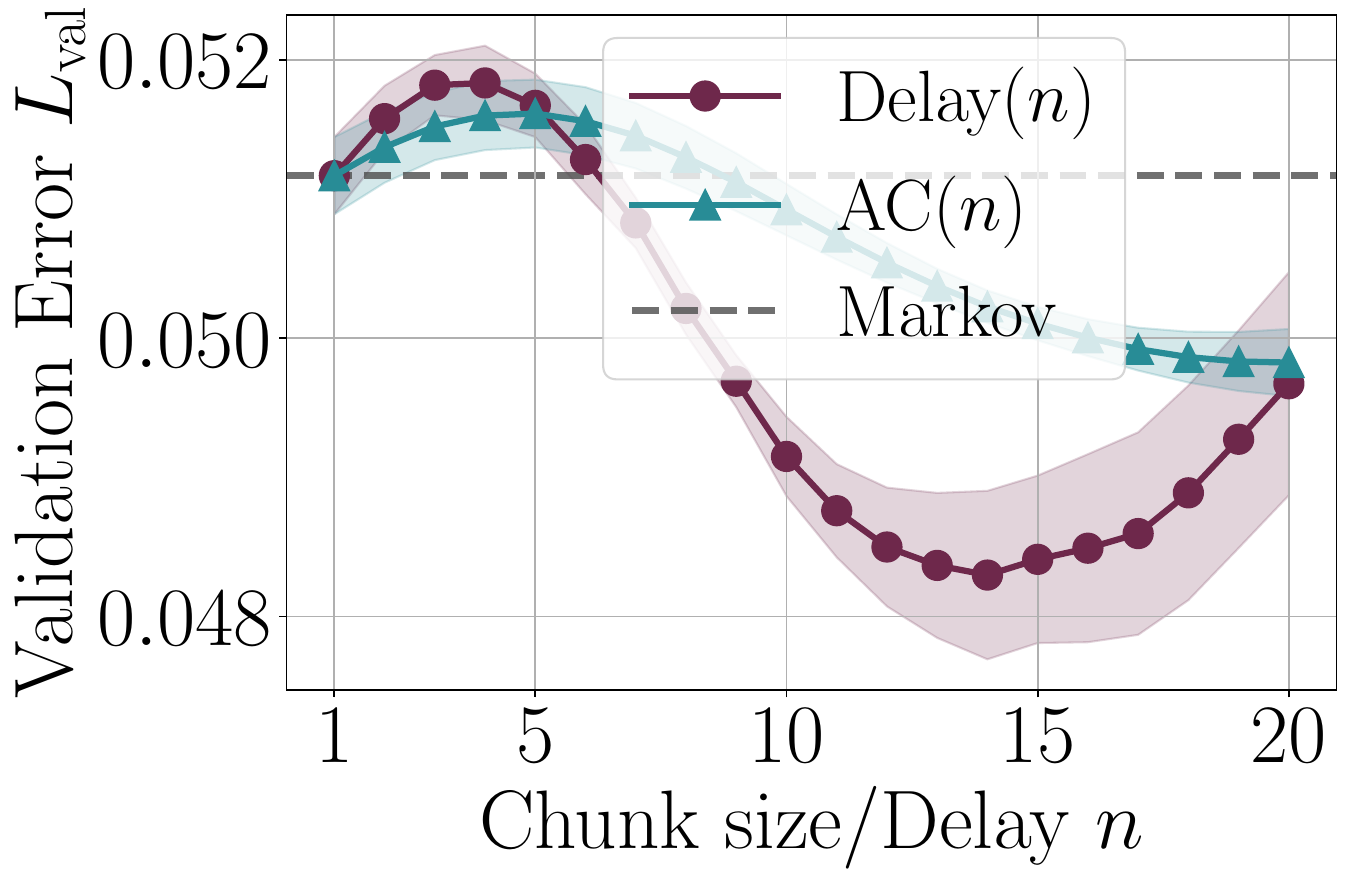}
    \end{minipage}
    \hfill
        \begin{minipage}[t]{0.23\textwidth}
        \centering
        \includegraphics[width=\linewidth]{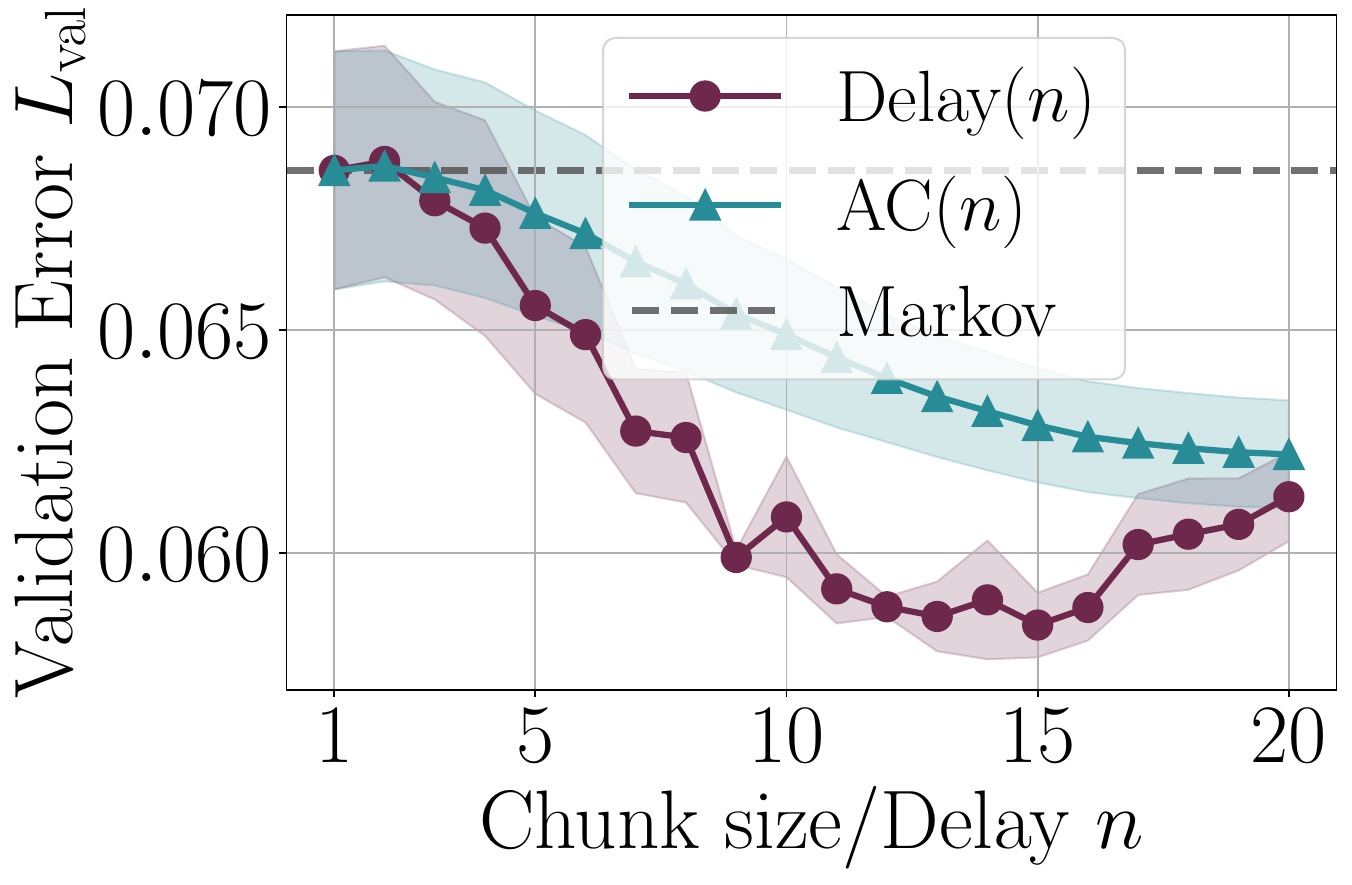}
    \end{minipage}
        \caption{Validation loss for each \texttt{Libero} task from 0 to 35 (corresponding to Fig. \ref{fig:val_loss_libero}), part 1.}
    \label{fig:val loss each libero1}
\end{figure*}

\begin{figure*}

        \begin{minipage}[t]{0.23\textwidth}
            \includegraphics[width=\linewidth]{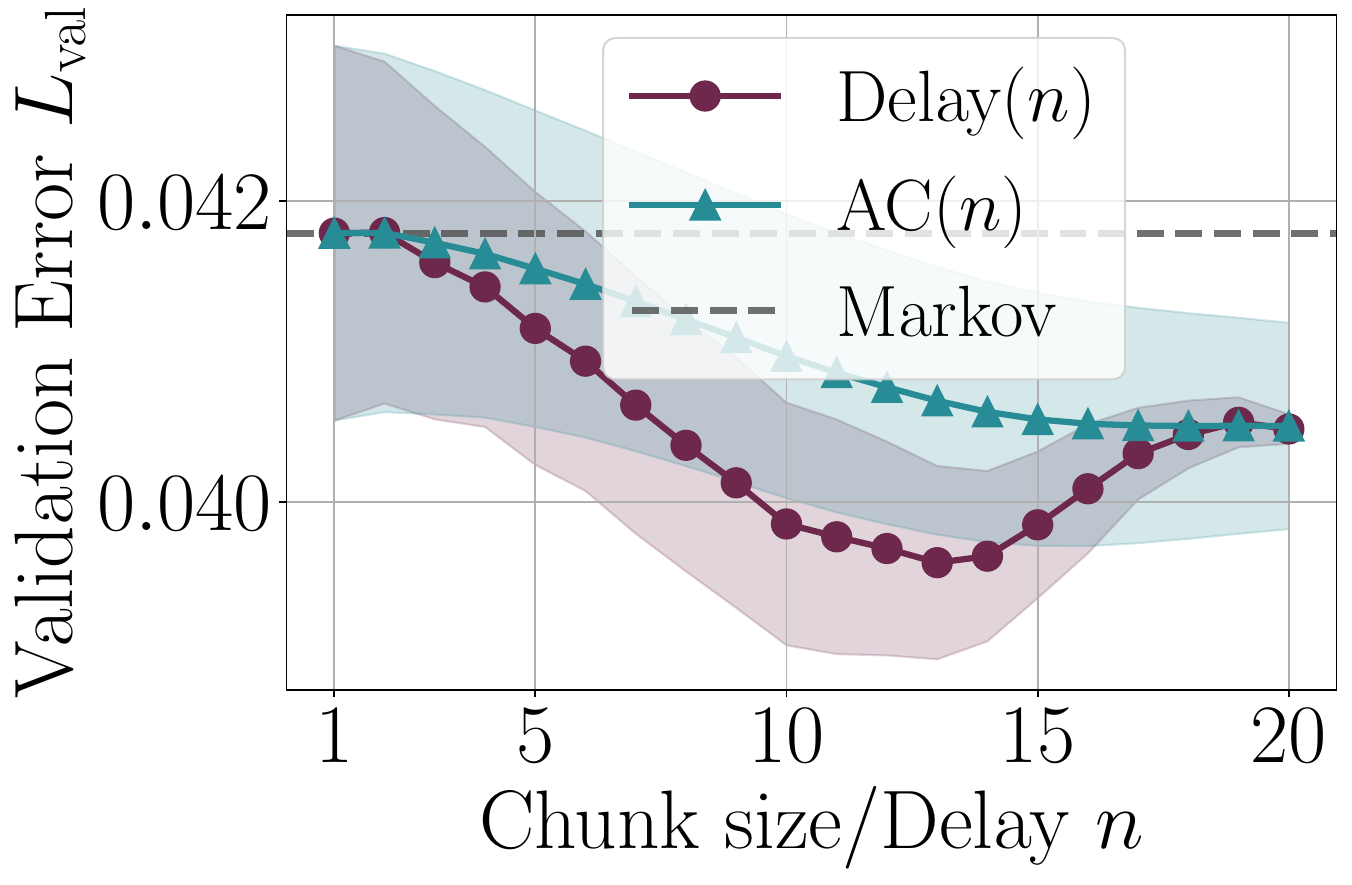}
    \end{minipage}
    \hfill
        \begin{minipage}[t]{0.23\textwidth}
        \centering
        \includegraphics[width=\linewidth]{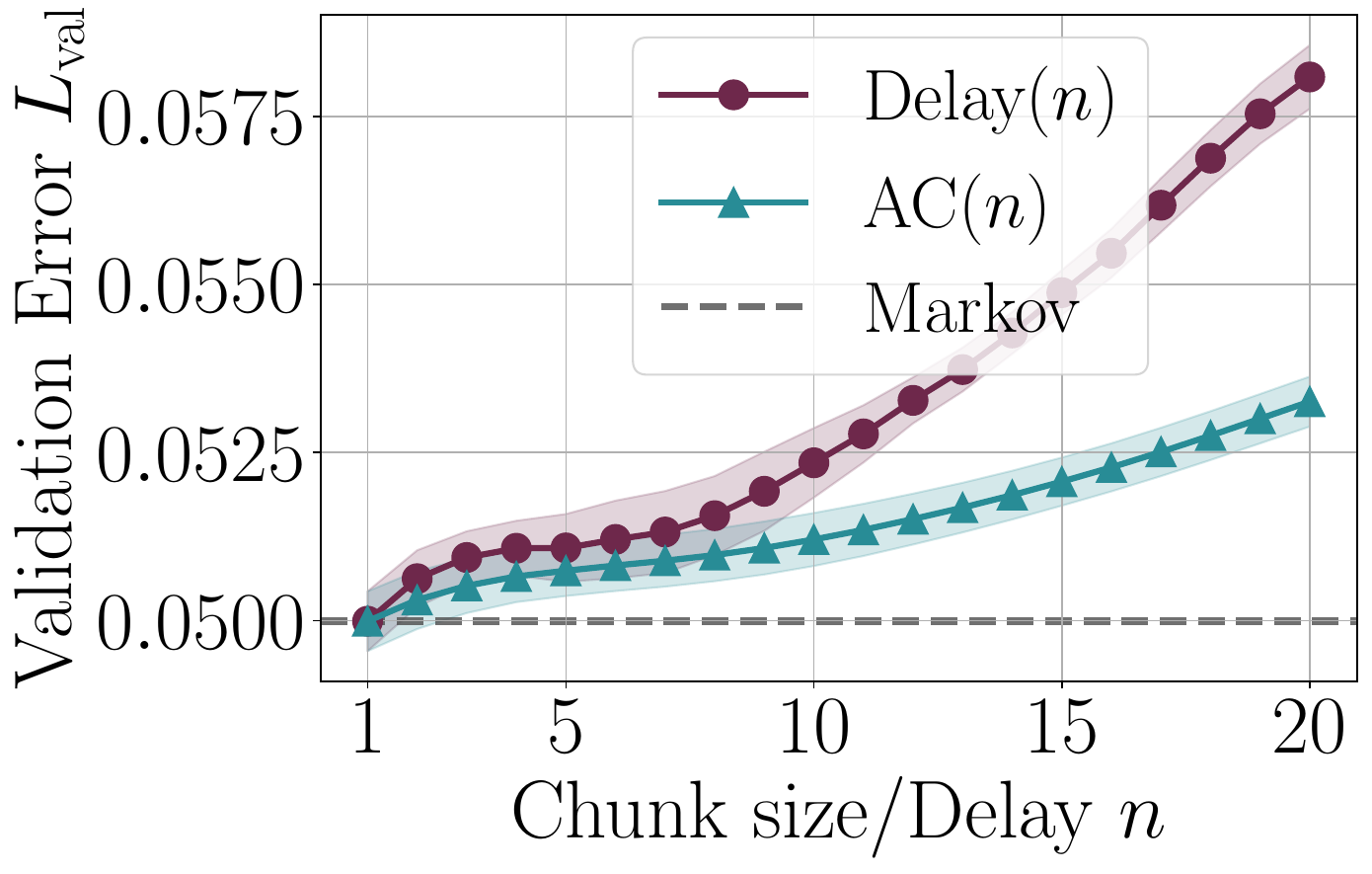}
    \end{minipage}
    \hfill
        \begin{minipage}[t]{0.23\textwidth}
        \centering
        \includegraphics[width=\linewidth]{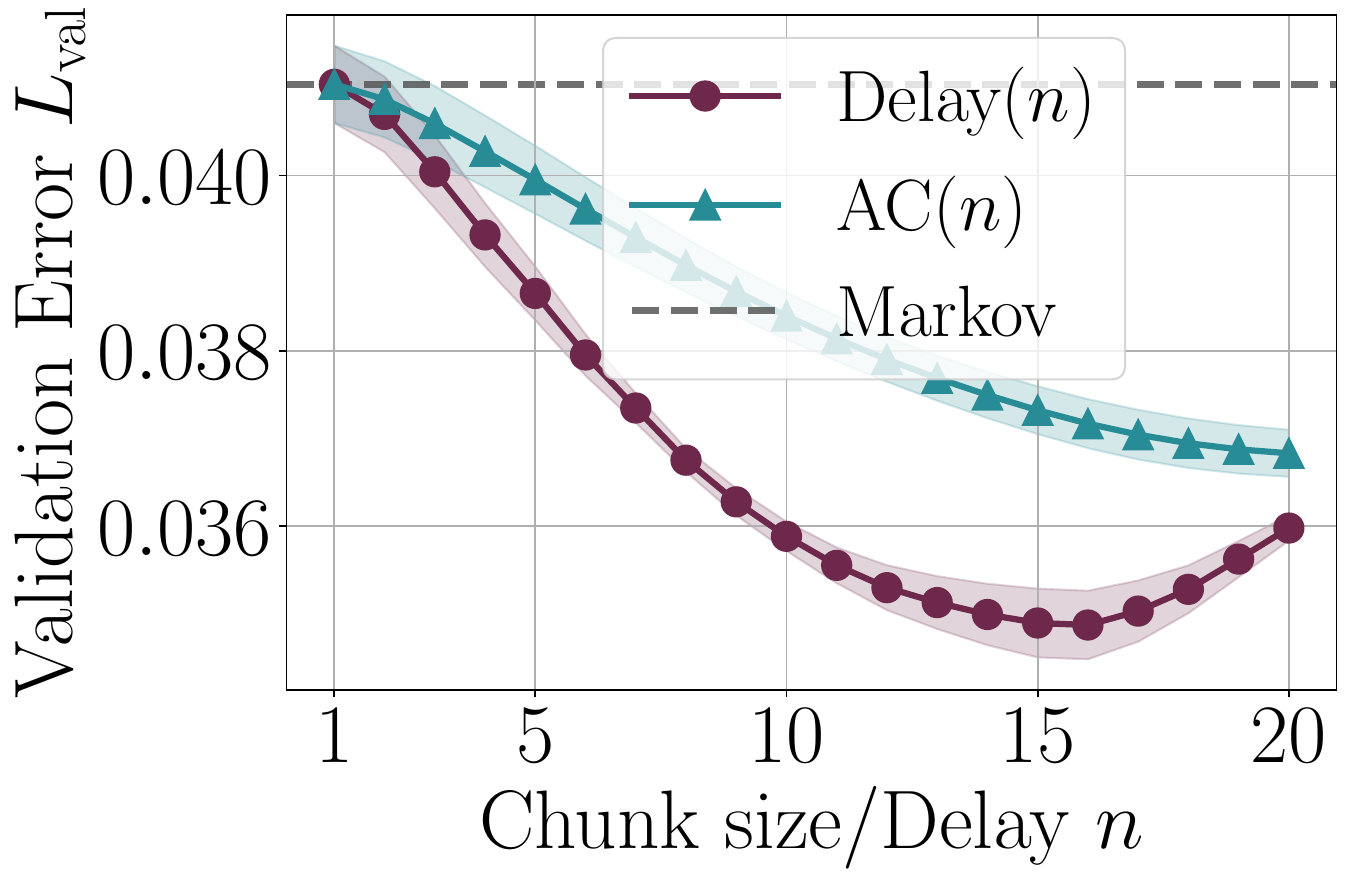}
    \end{minipage}
    \hfill
        \begin{minipage}[t]{0.23\textwidth}
        \centering
        \includegraphics[width=\linewidth]{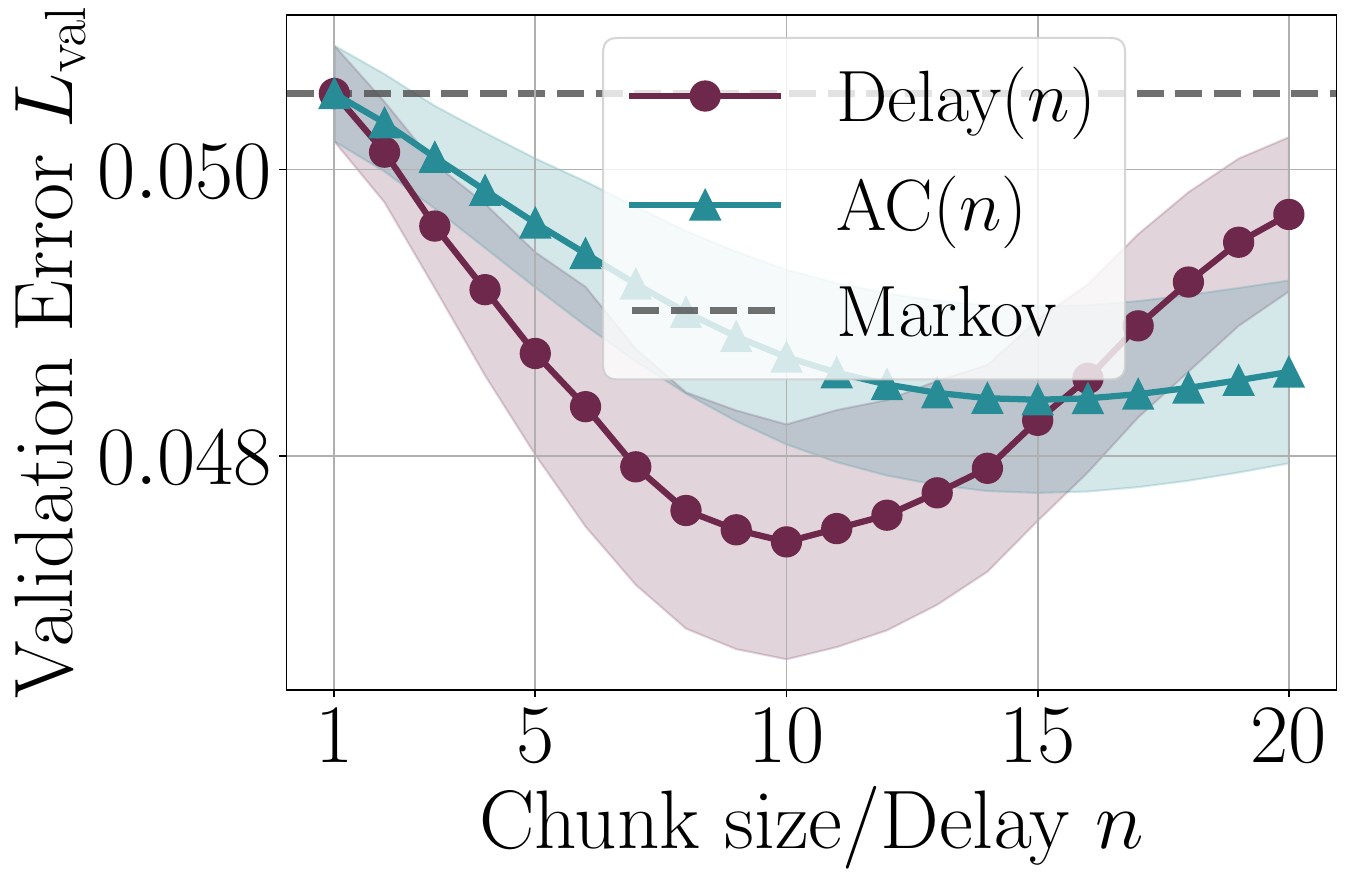}
    \end{minipage}
        \begin{minipage}[t]{0.23\textwidth}
            \includegraphics[width=\linewidth]{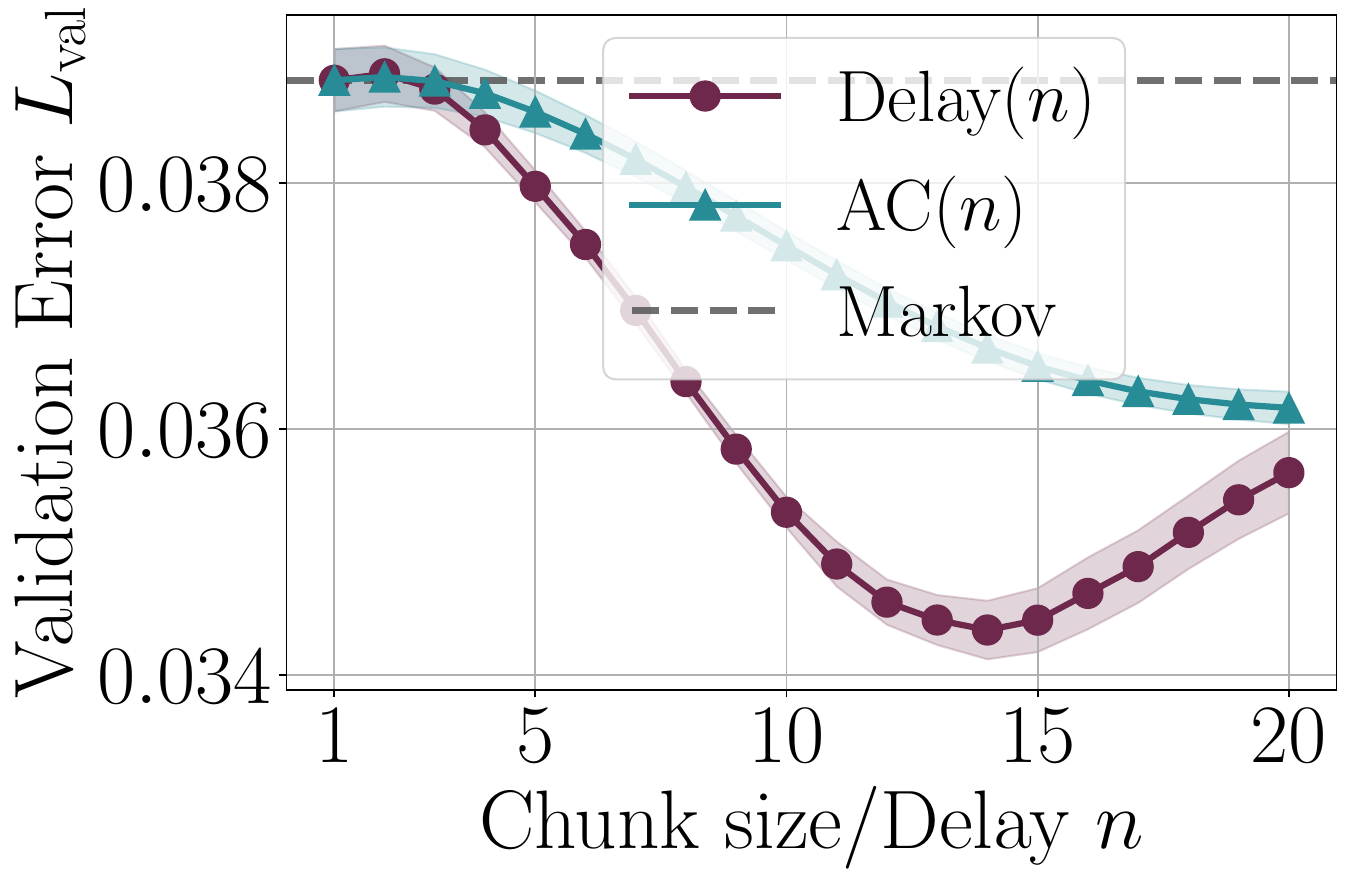}
    \end{minipage}
    \hfill
        \begin{minipage}[t]{0.23\textwidth}
        \centering
        \includegraphics[width=\linewidth]{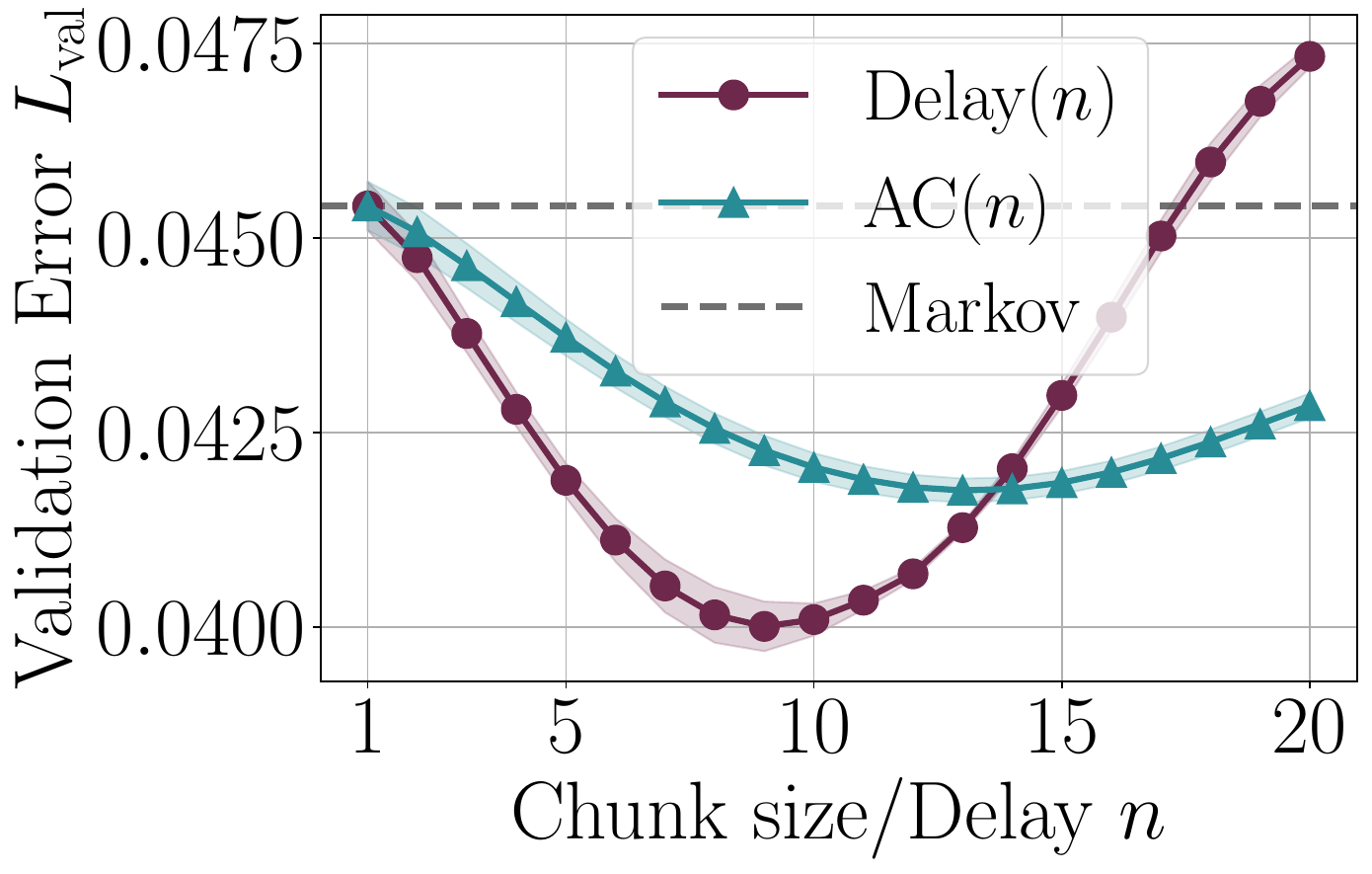}
    \end{minipage}
    \hfill
        \begin{minipage}[t]{0.23\textwidth}
        \centering
        \includegraphics[width=\linewidth]{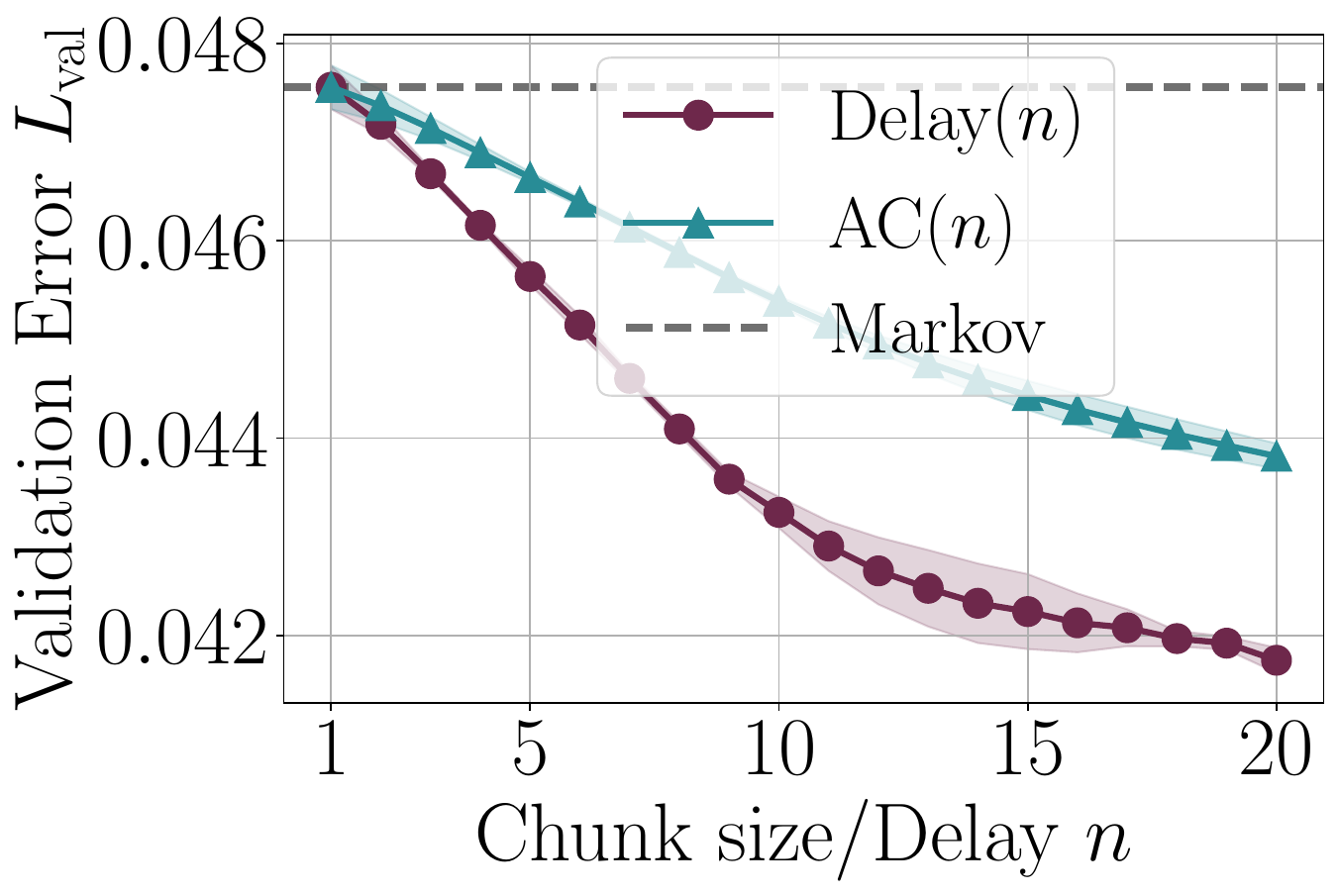}
    \end{minipage}
    \hfill
        \begin{minipage}[t]{0.23\textwidth}
        \centering
        \includegraphics[width=\linewidth]{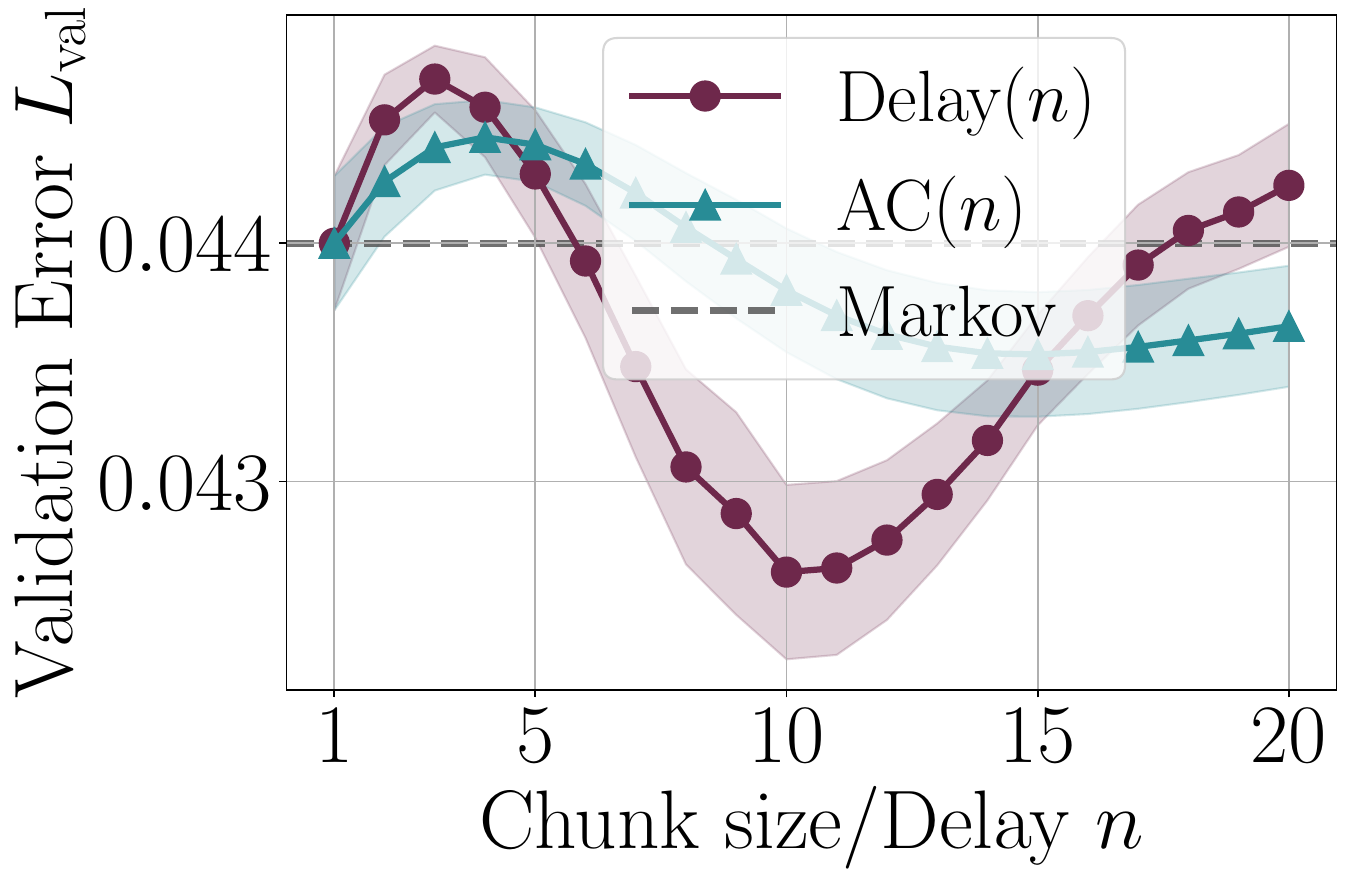}
    \end{minipage}
        \begin{minipage}[t]{0.23\textwidth}
            \includegraphics[width=\linewidth]{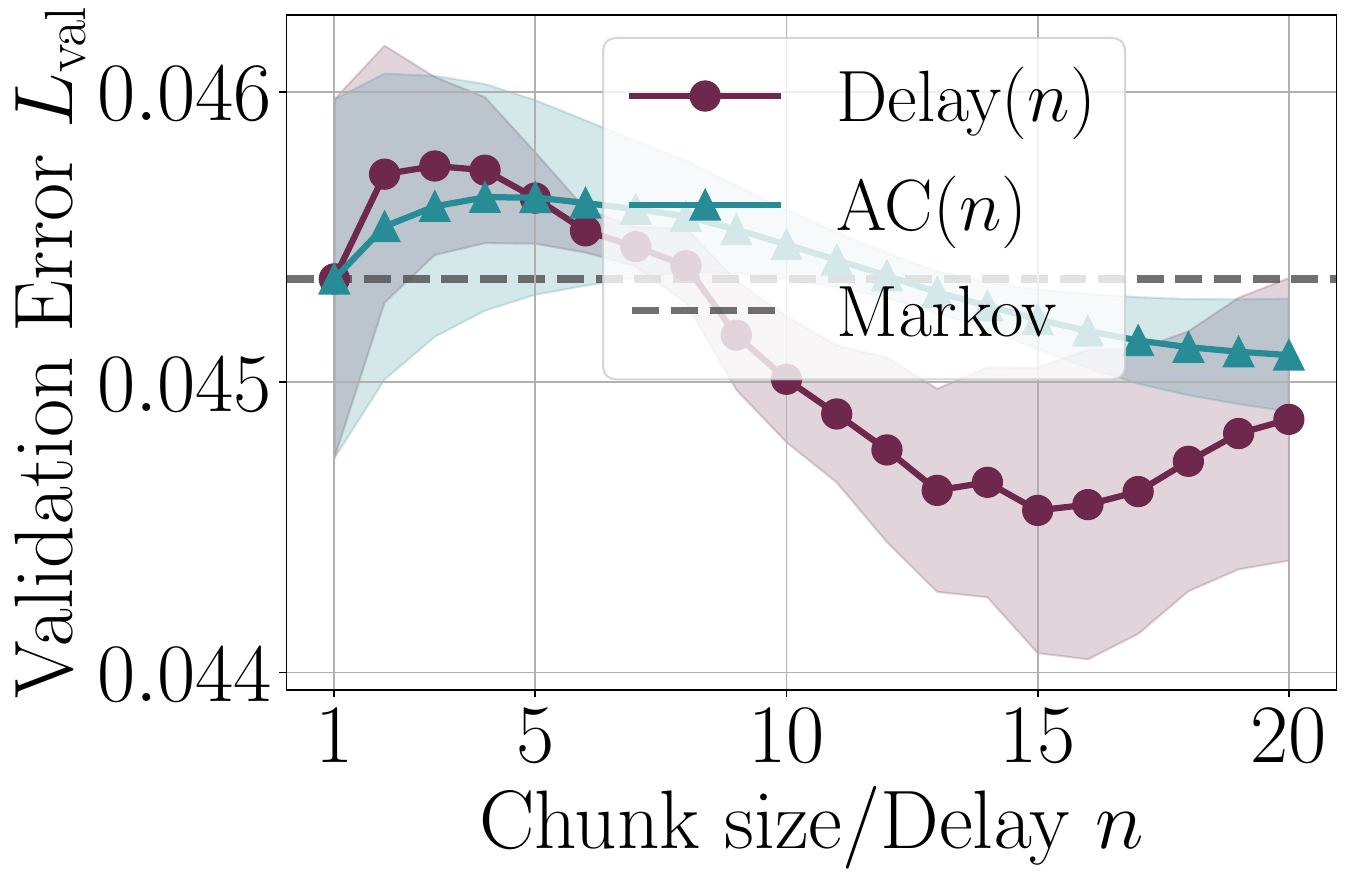}
    \end{minipage}
    \hfill
        \begin{minipage}[t]{0.23\textwidth}
        \centering
        \includegraphics[width=\linewidth]{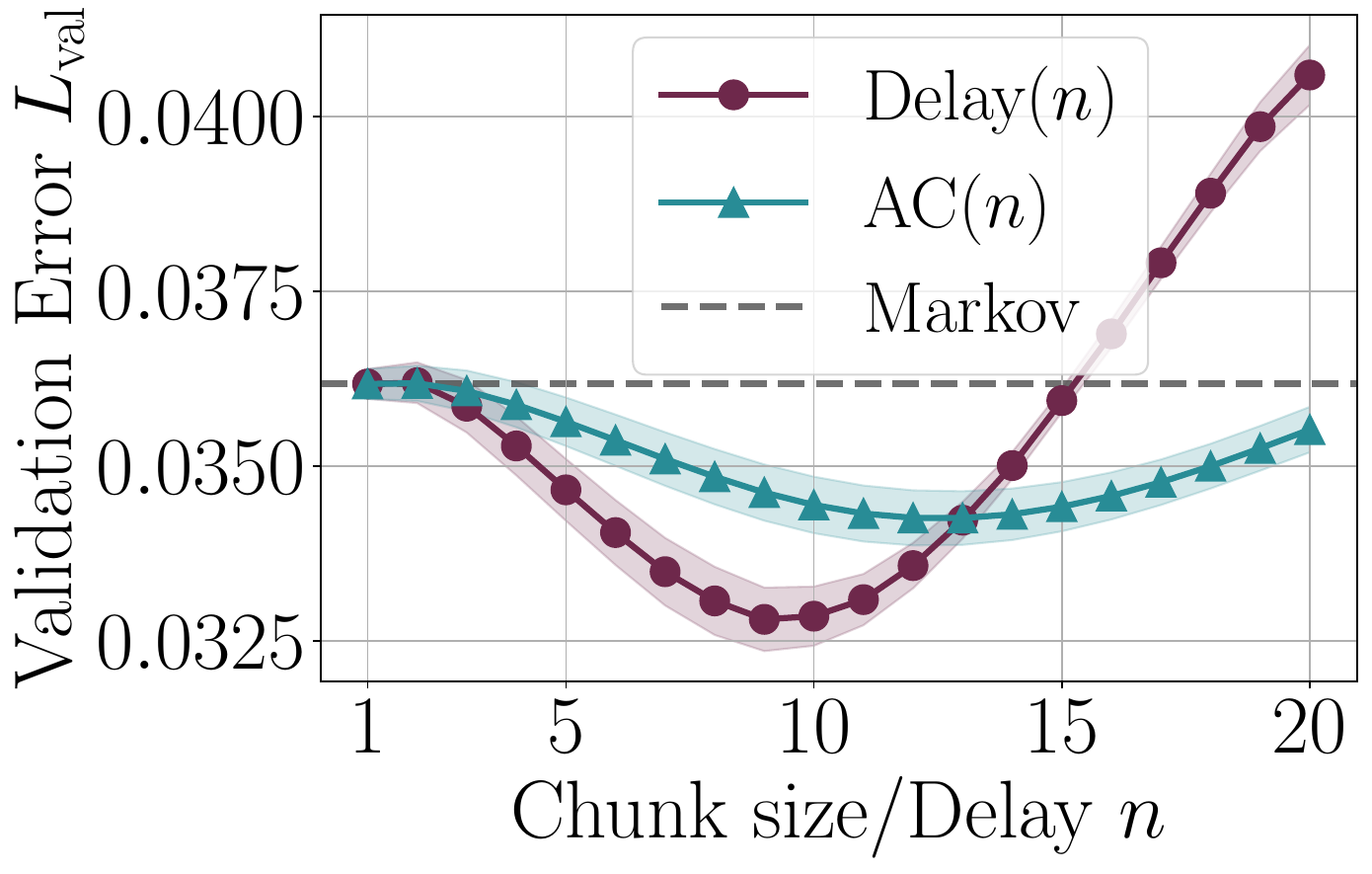}
    \end{minipage}
    \hfill
        \begin{minipage}[t]{0.23\textwidth}
        \centering
        \includegraphics[width=\linewidth]{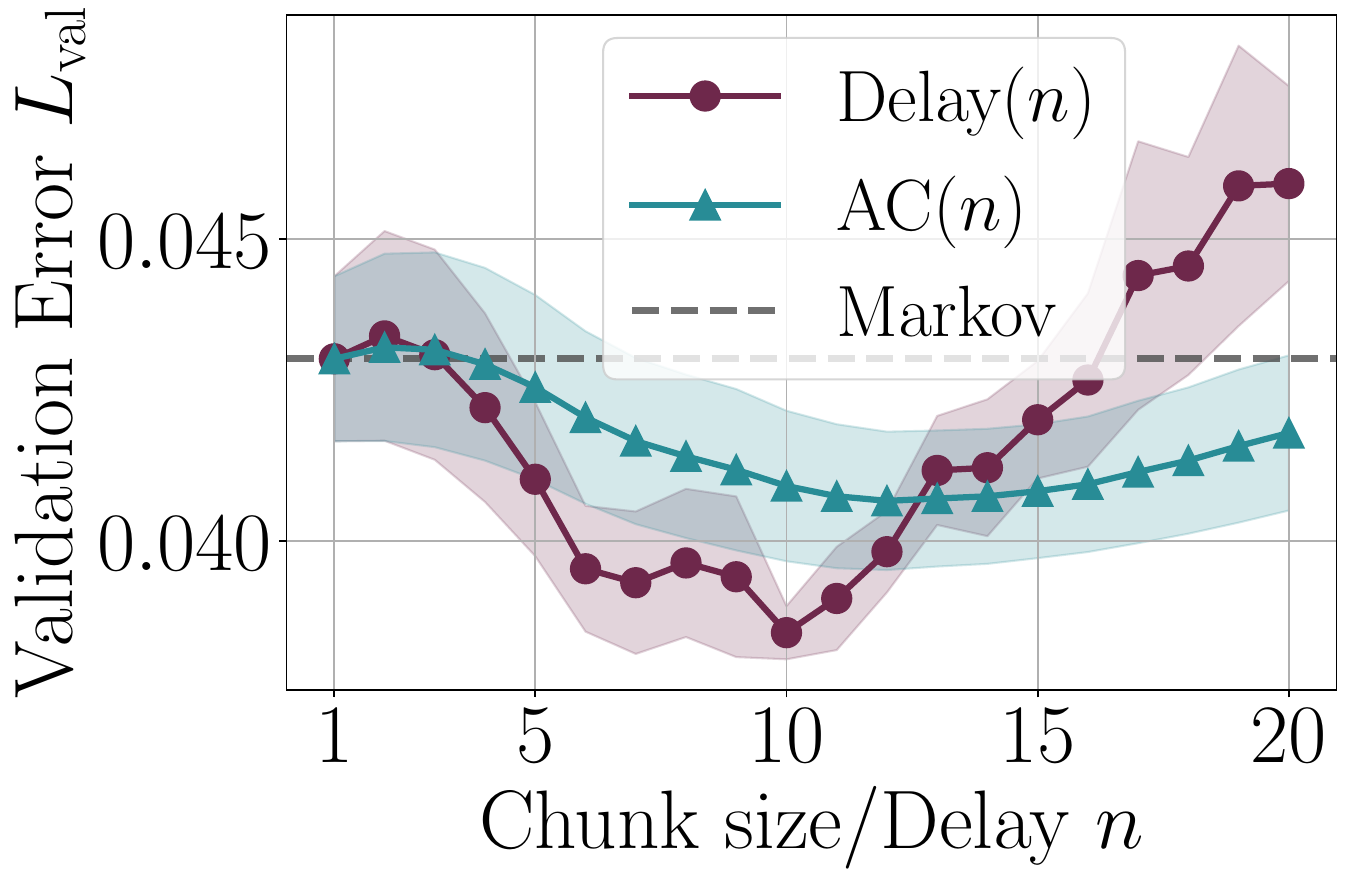}
    \end{minipage}
    \hfill
        \begin{minipage}[t]{0.23\textwidth}
        \centering
        \includegraphics[width=\linewidth]{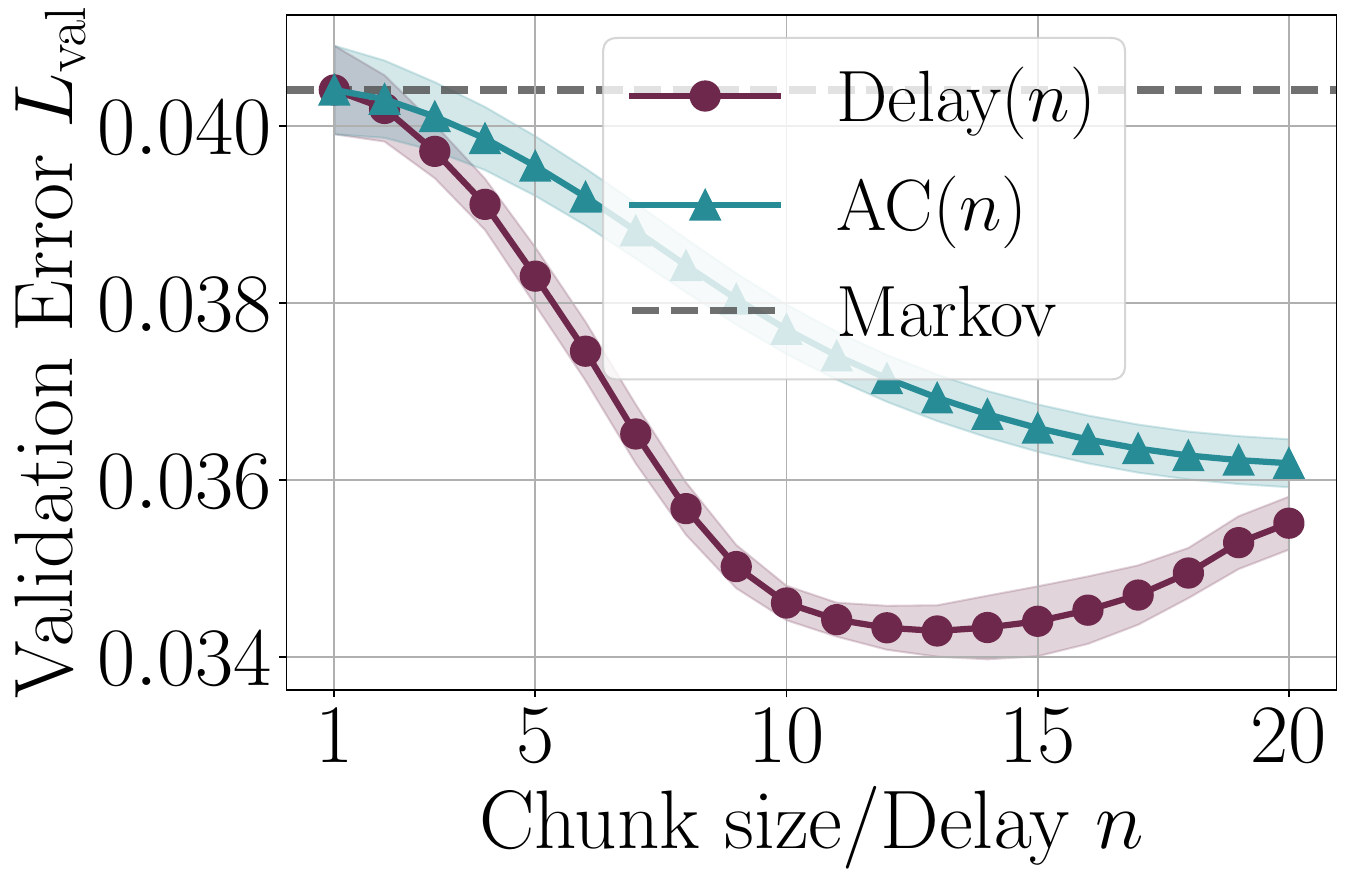}
    \end{minipage}
        \begin{minipage}[t]{0.23\textwidth}
            \includegraphics[width=\linewidth]{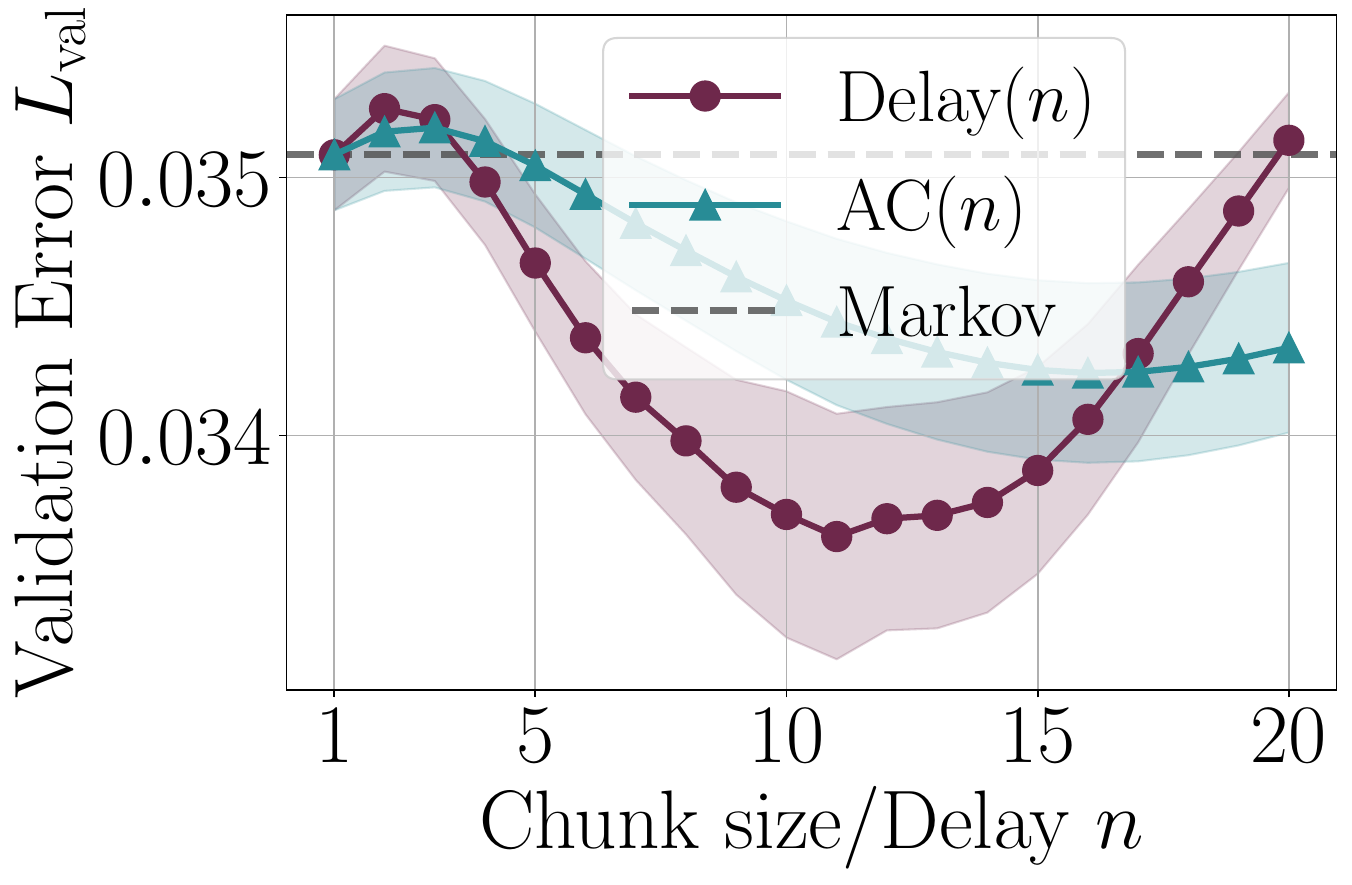}
    \end{minipage}
    \hfill
        \begin{minipage}[t]{0.23\textwidth}
        \centering
        \includegraphics[width=\linewidth]{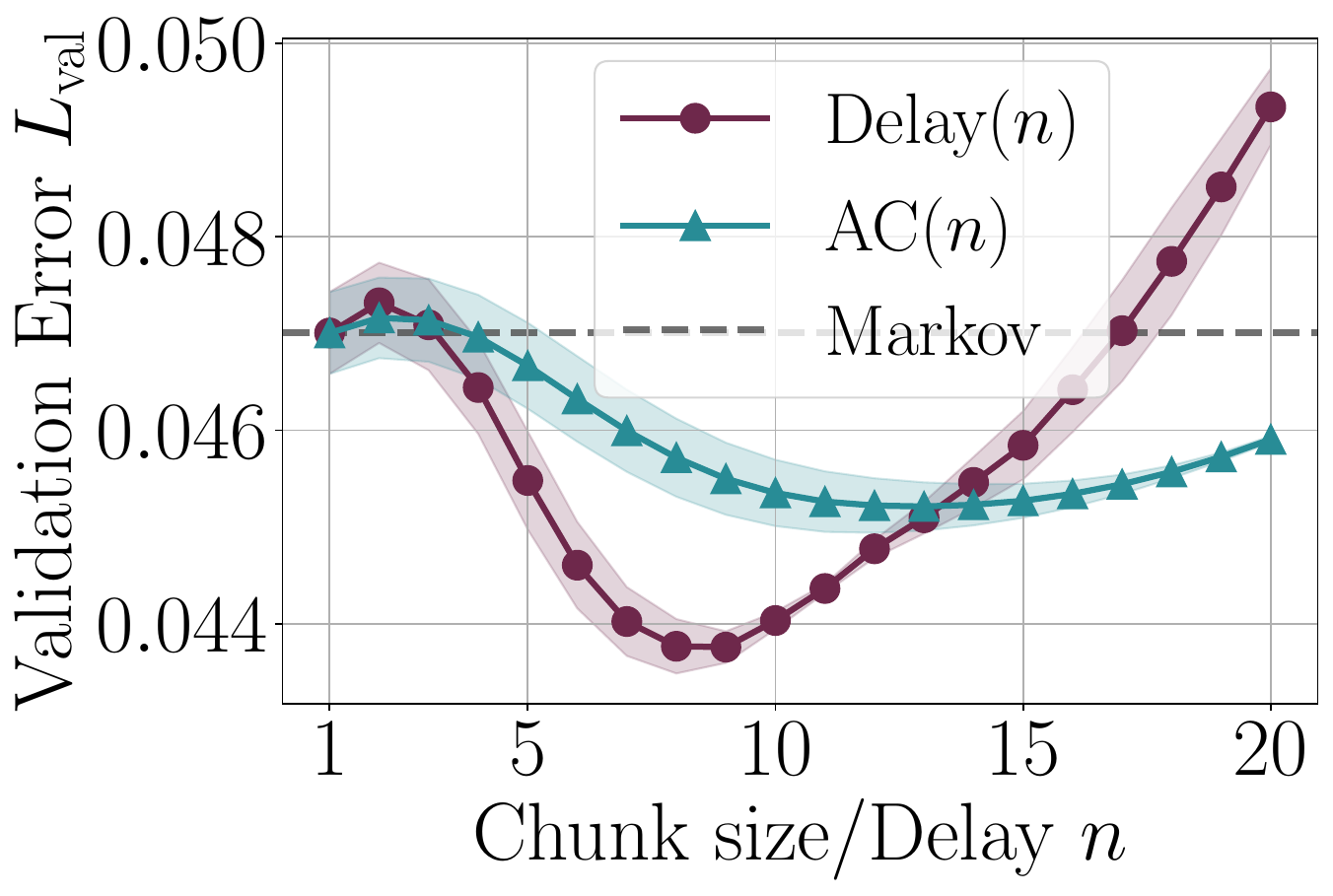}
    \end{minipage}
    \hfill
        \begin{minipage}[t]{0.23\textwidth}
        \centering
        \includegraphics[width=\linewidth]{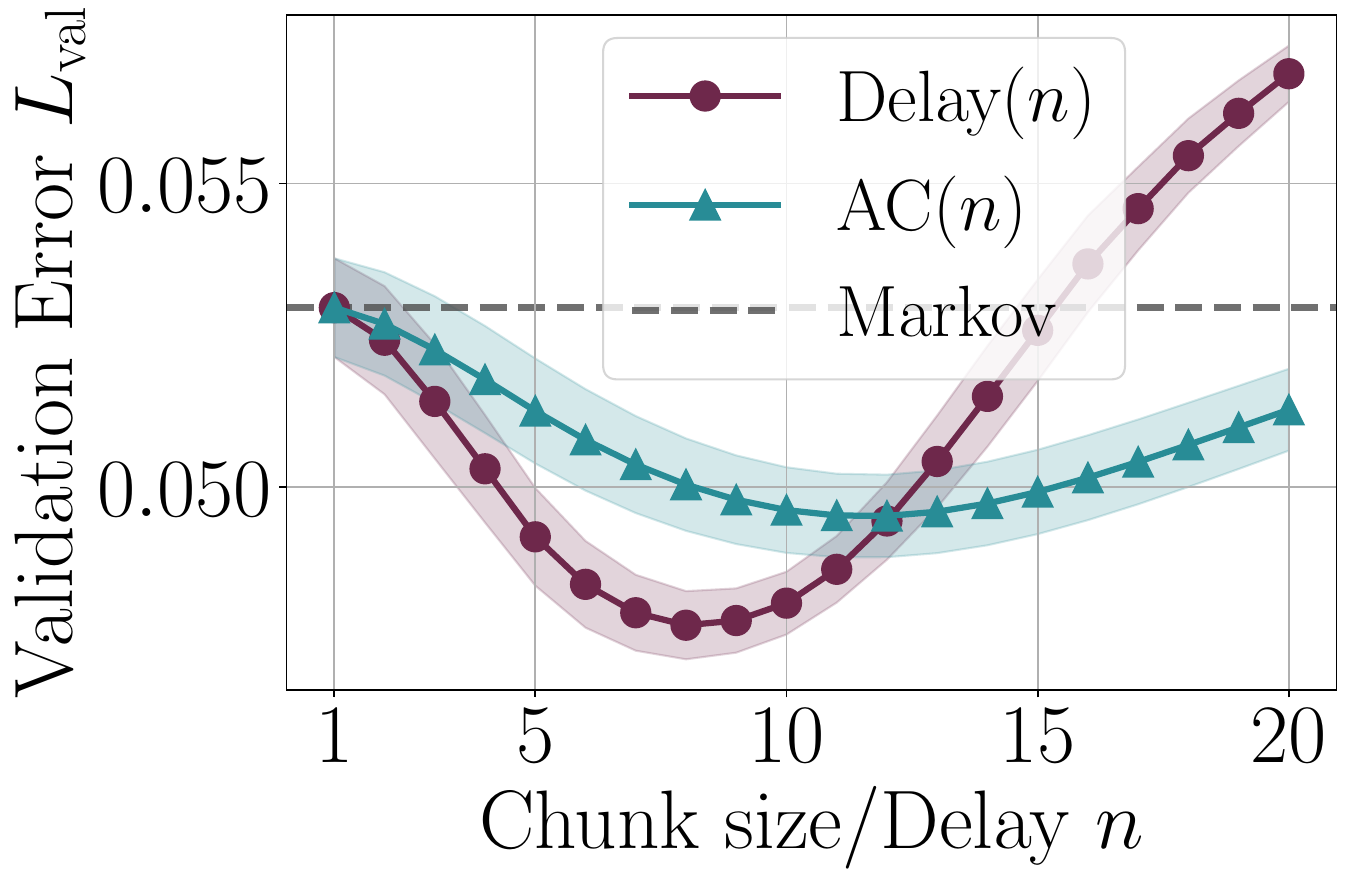}
    \end{minipage}
    \hfill
        \begin{minipage}[t]{0.23\textwidth}
        \centering
        \includegraphics[width=\linewidth]{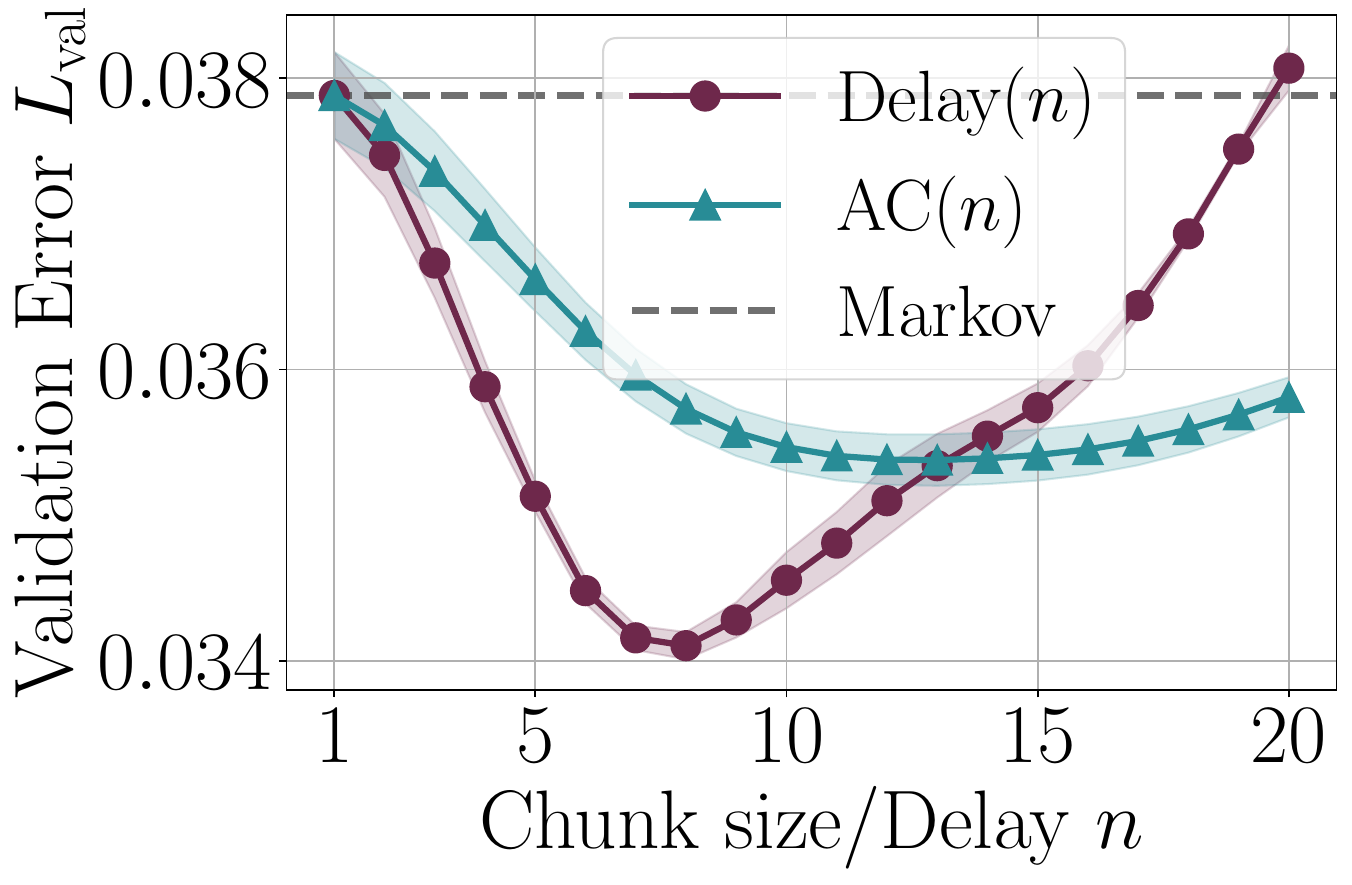}
    \end{minipage}
        \begin{minipage}[t]{0.23\textwidth}
            \includegraphics[width=\linewidth]{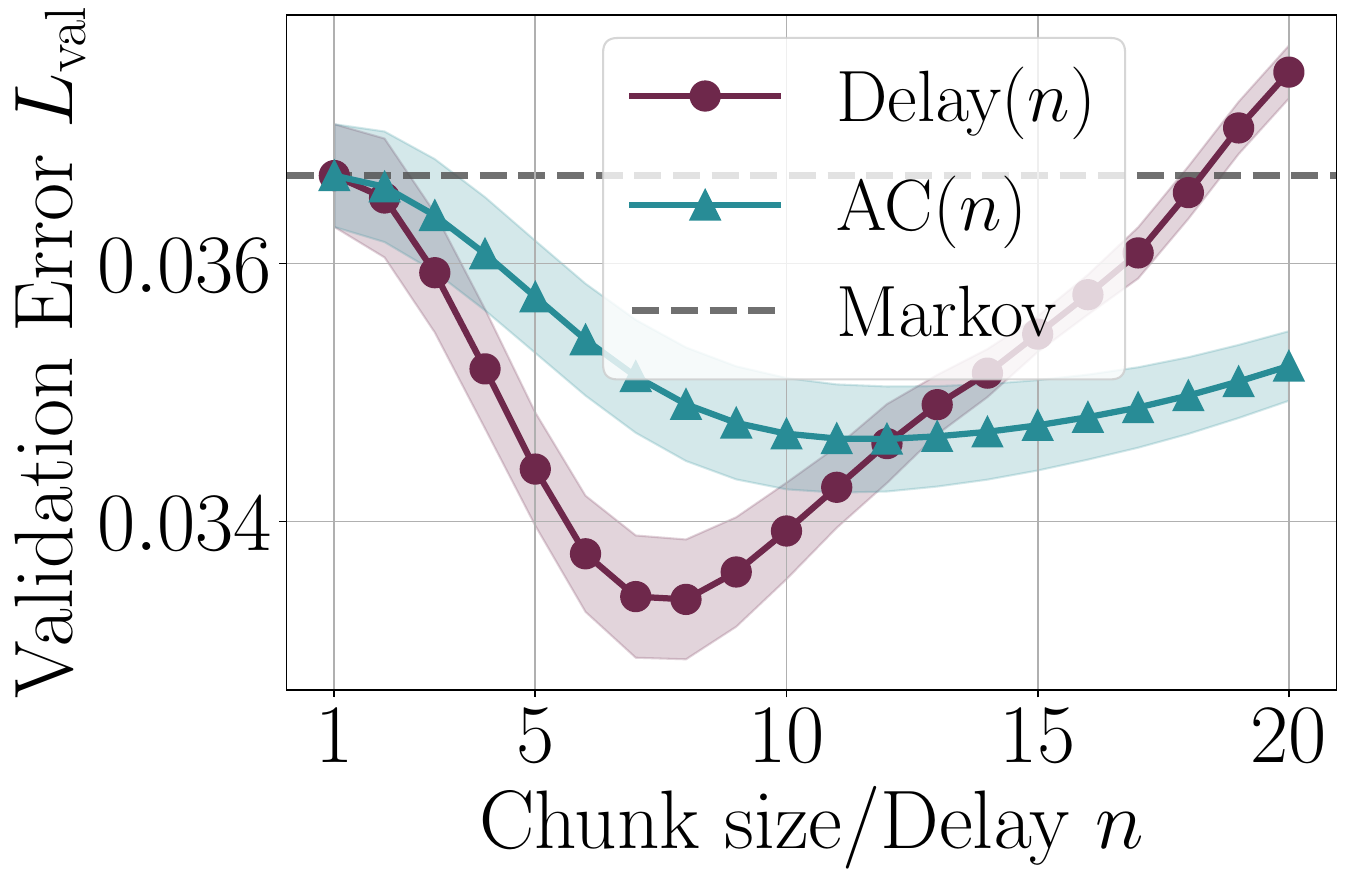}
    \end{minipage}
    \hfill
        \begin{minipage}[t]{0.23\textwidth}
        \centering
        \includegraphics[width=\linewidth]{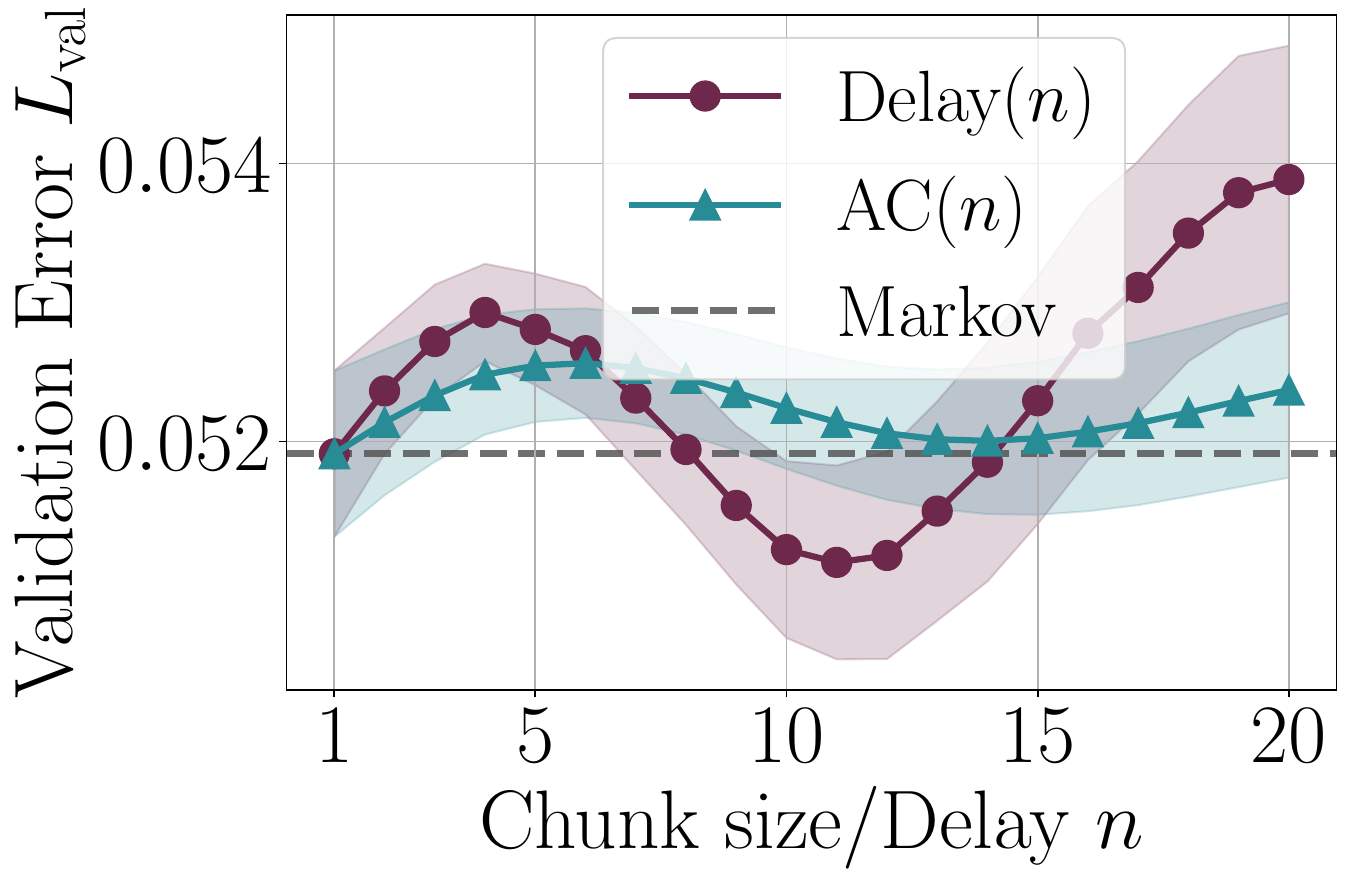}
    \end{minipage}
    \hfill
        \begin{minipage}[t]{0.23\textwidth}
        \centering
        \includegraphics[width=\linewidth]{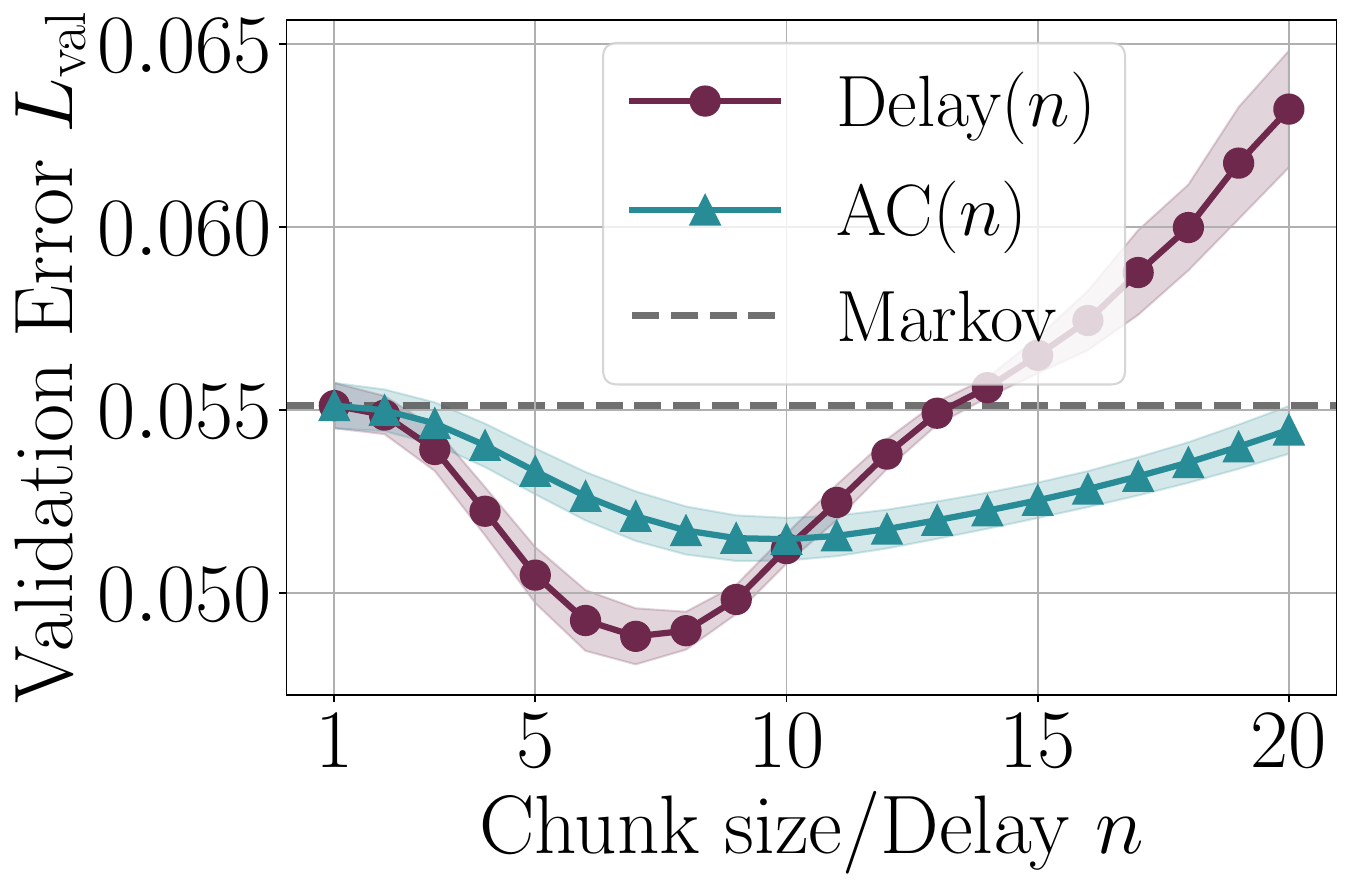}
    \end{minipage}
    \hfill
        \begin{minipage}[t]{0.23\textwidth}
        \centering
        \includegraphics[width=\linewidth]{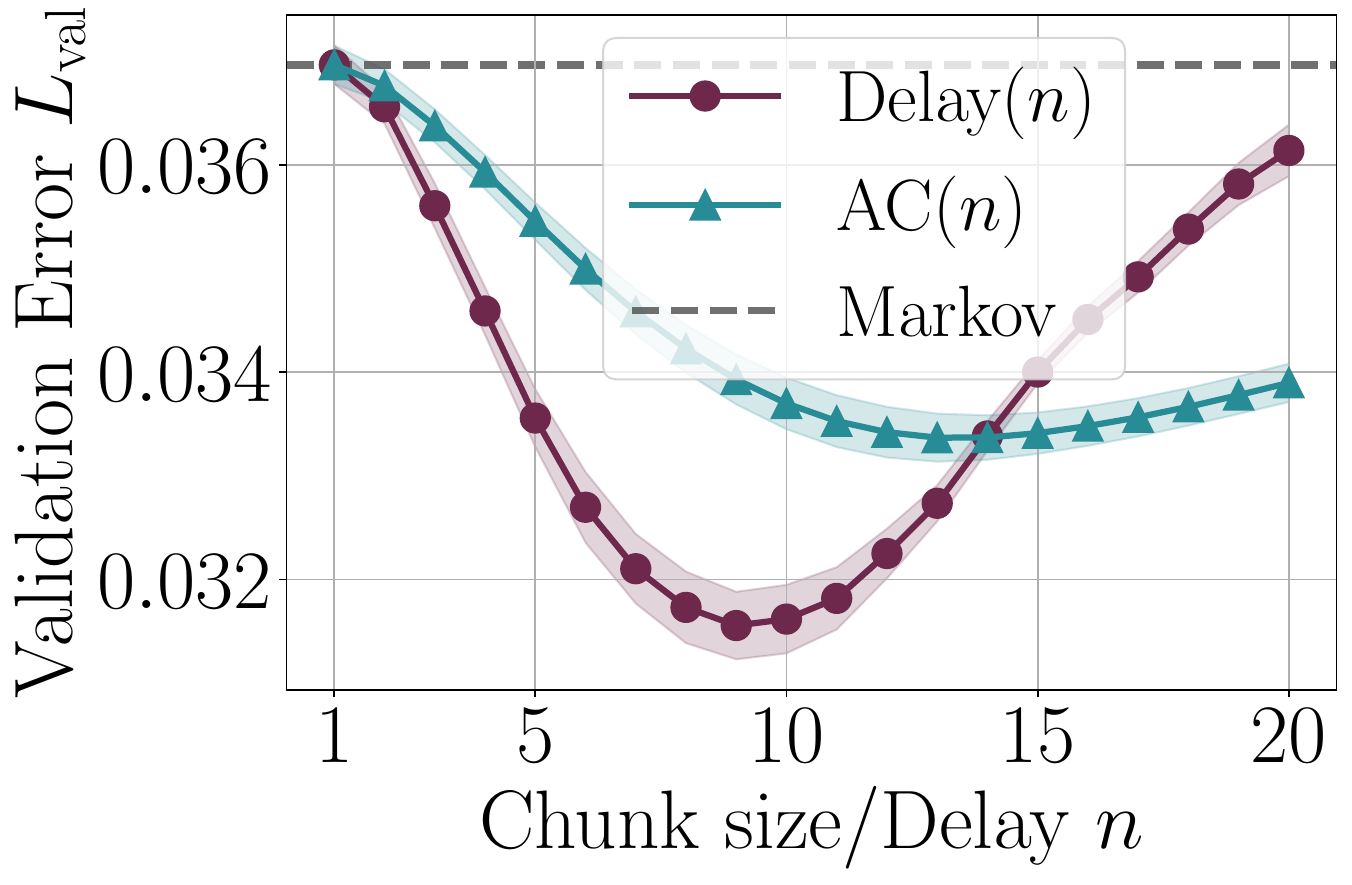}
    \end{minipage}
        \begin{minipage}[t]{0.23\textwidth}
            \includegraphics[width=\linewidth]{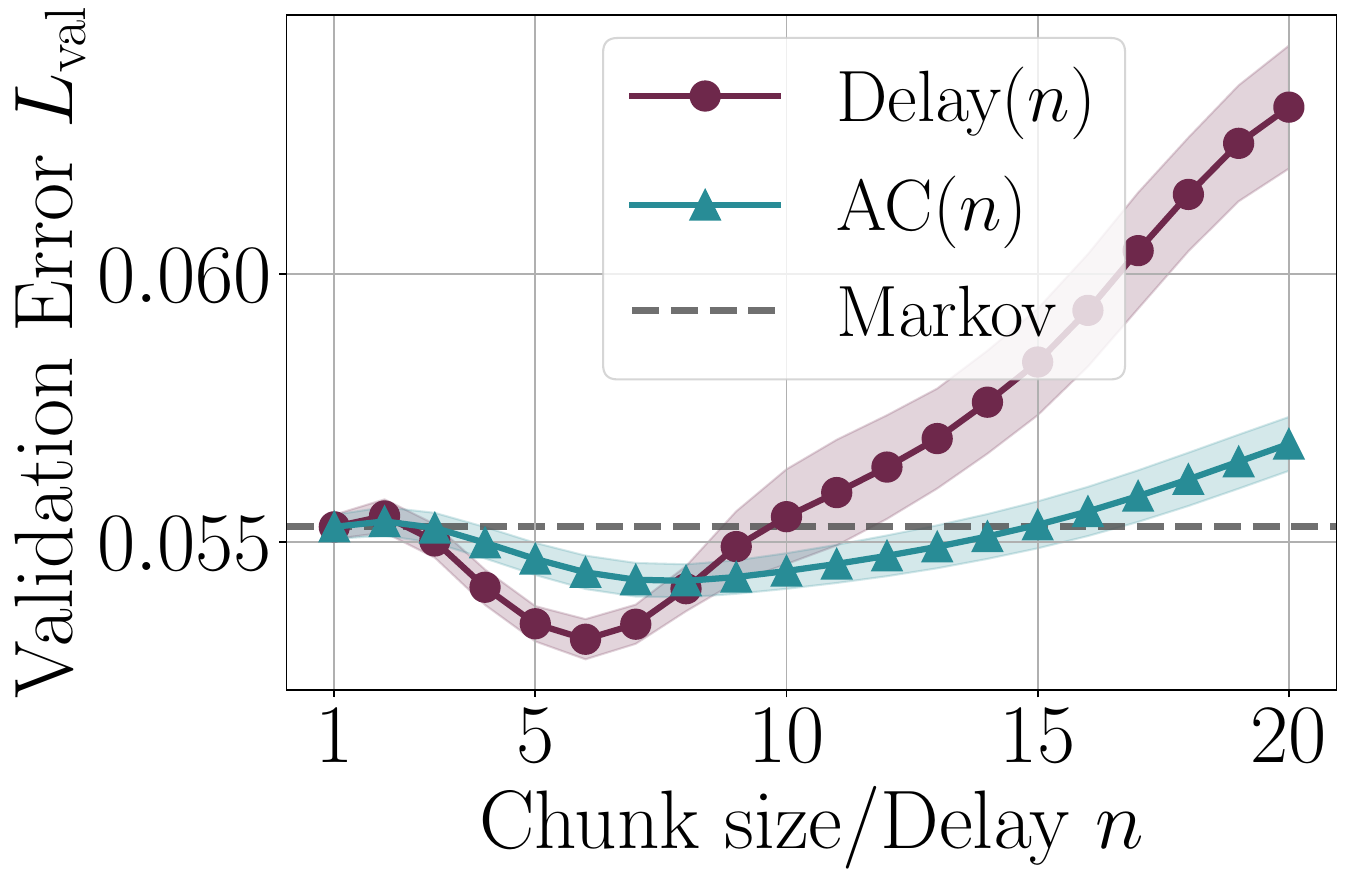}
    \end{minipage}
    \hfill
        \begin{minipage}[t]{0.23\textwidth}
        \centering
        \includegraphics[width=\linewidth]{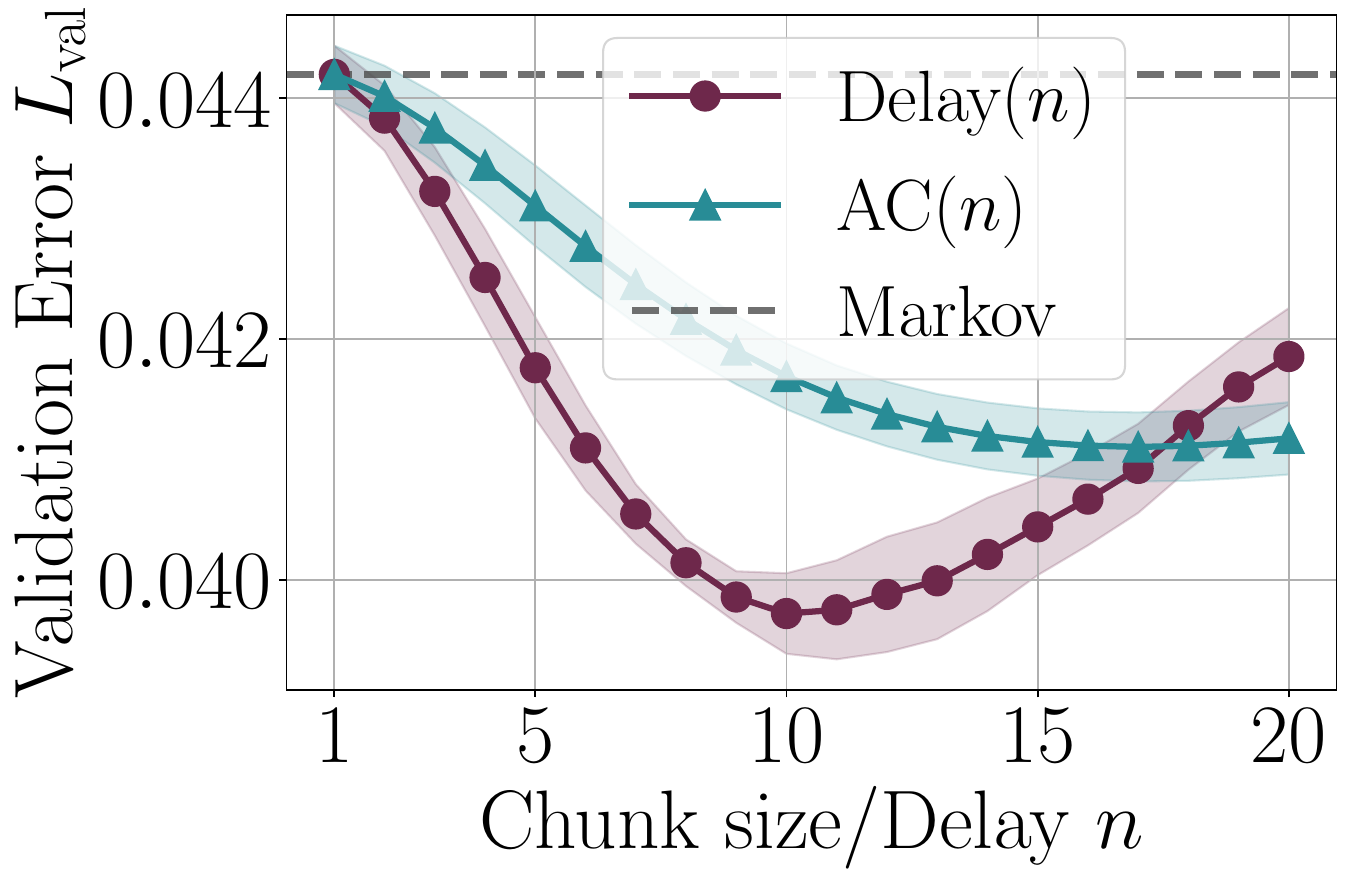}
    \end{minipage}
    \hfill
        \begin{minipage}[t]{0.23\textwidth}
        \centering
        \includegraphics[width=\linewidth]{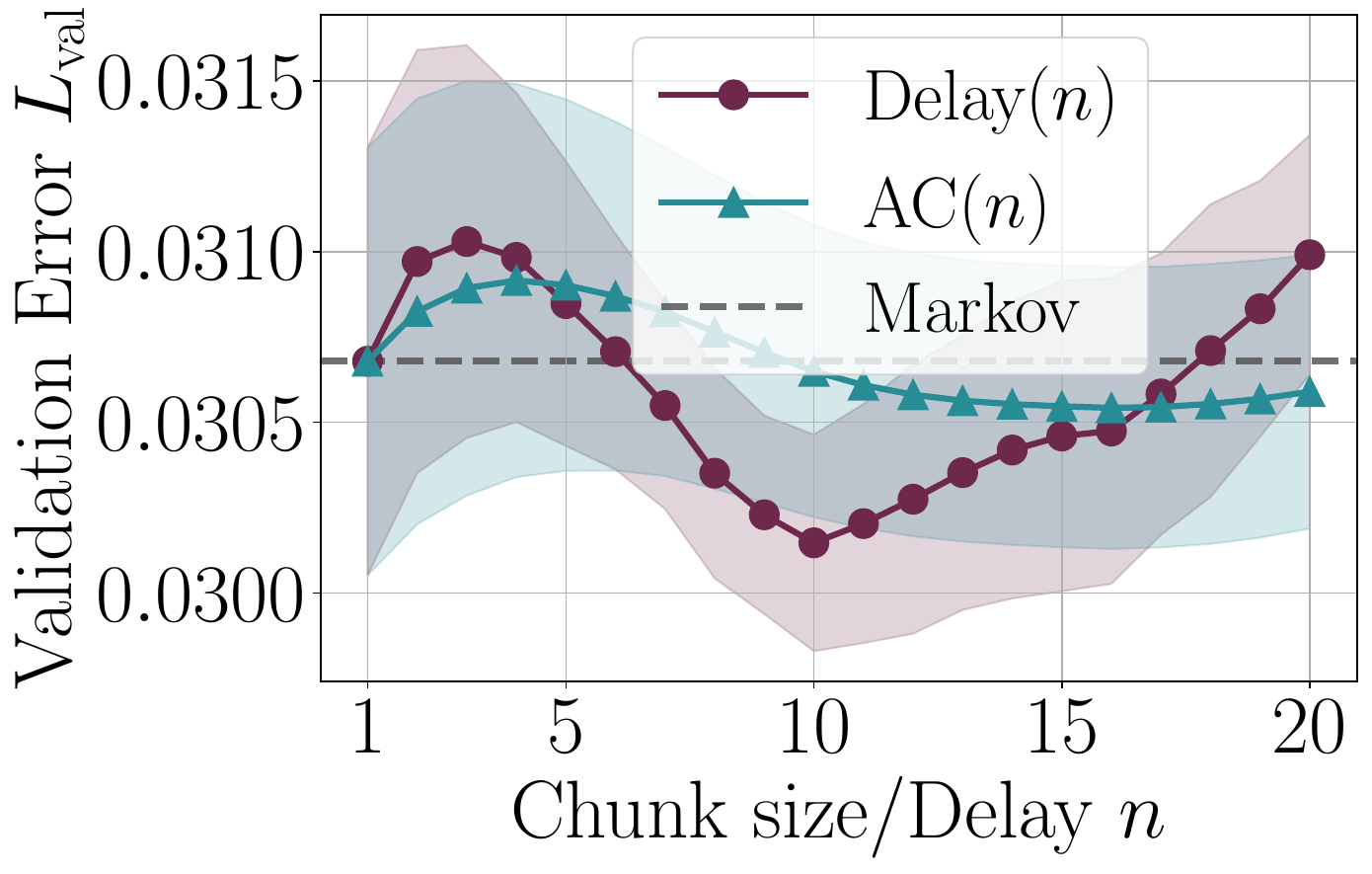}
    \end{minipage}
    \hfill
        \begin{minipage}[t]{0.23\textwidth}
        \centering
        \includegraphics[width=\linewidth]{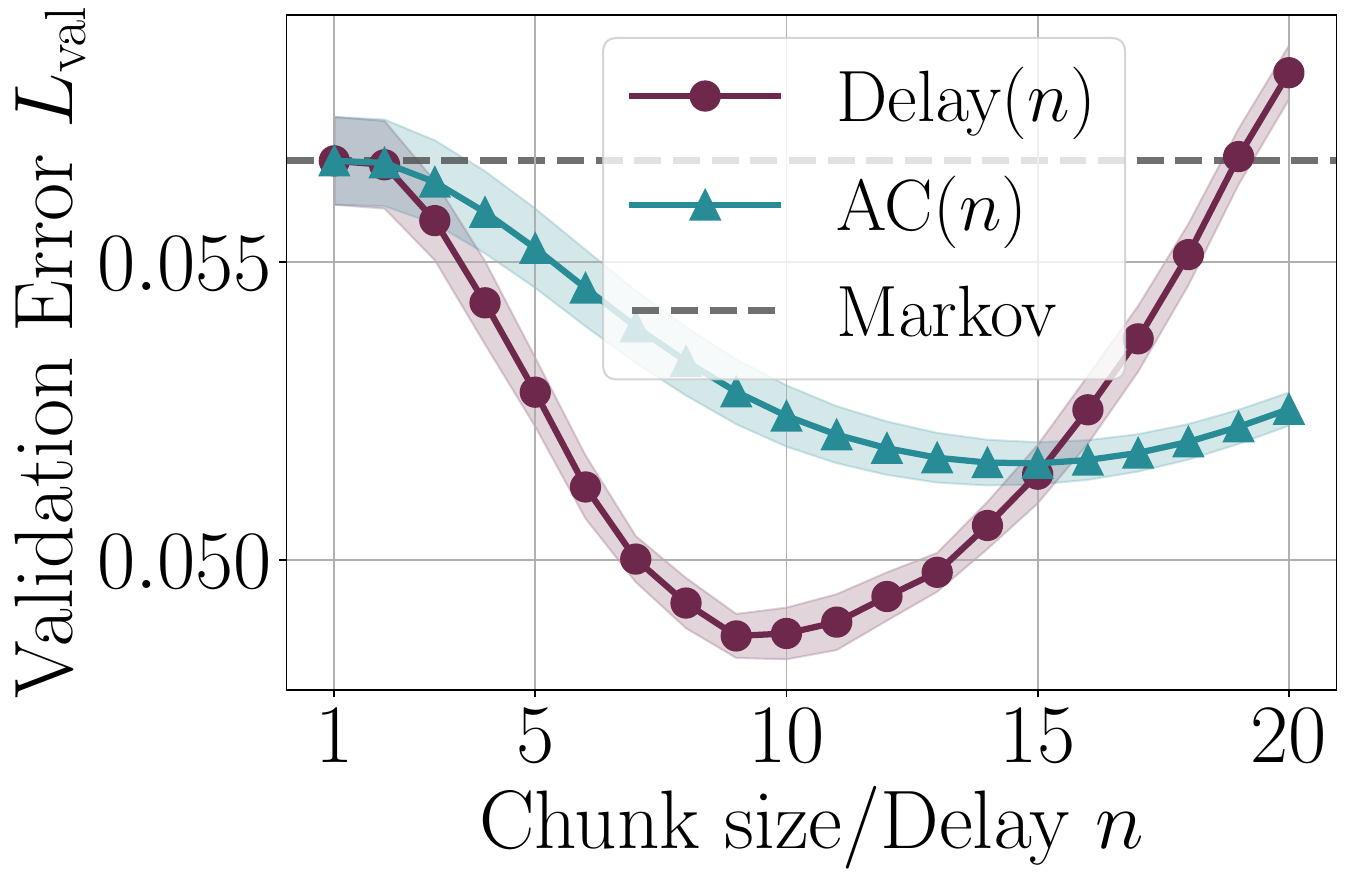}
    \end{minipage}
        \begin{minipage}[t]{0.23\textwidth}
            \includegraphics[width=\linewidth]{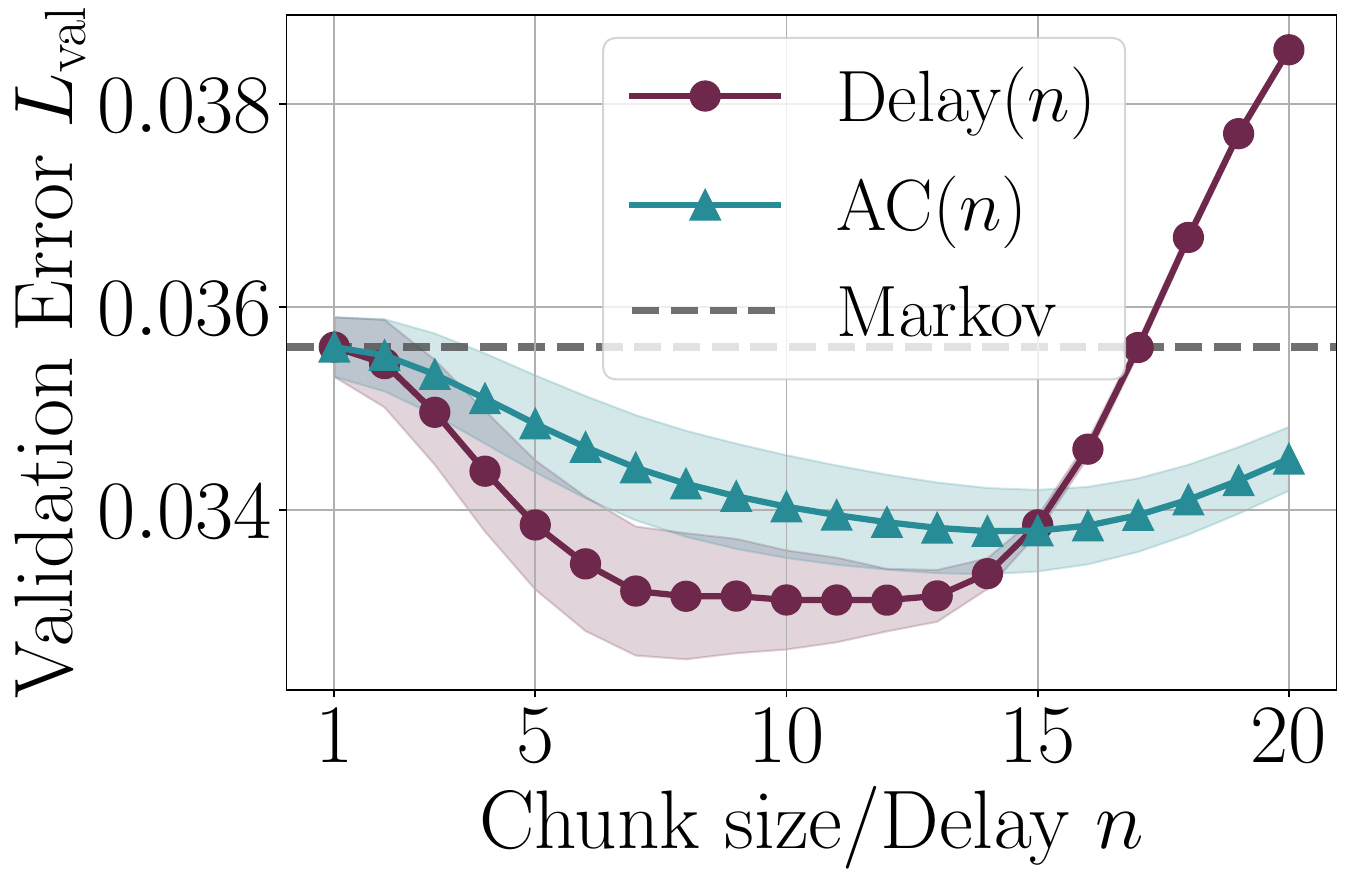}
    \end{minipage}
    \hfill
        \begin{minipage}[t]{0.23\textwidth}
        \centering
        \includegraphics[width=\linewidth]{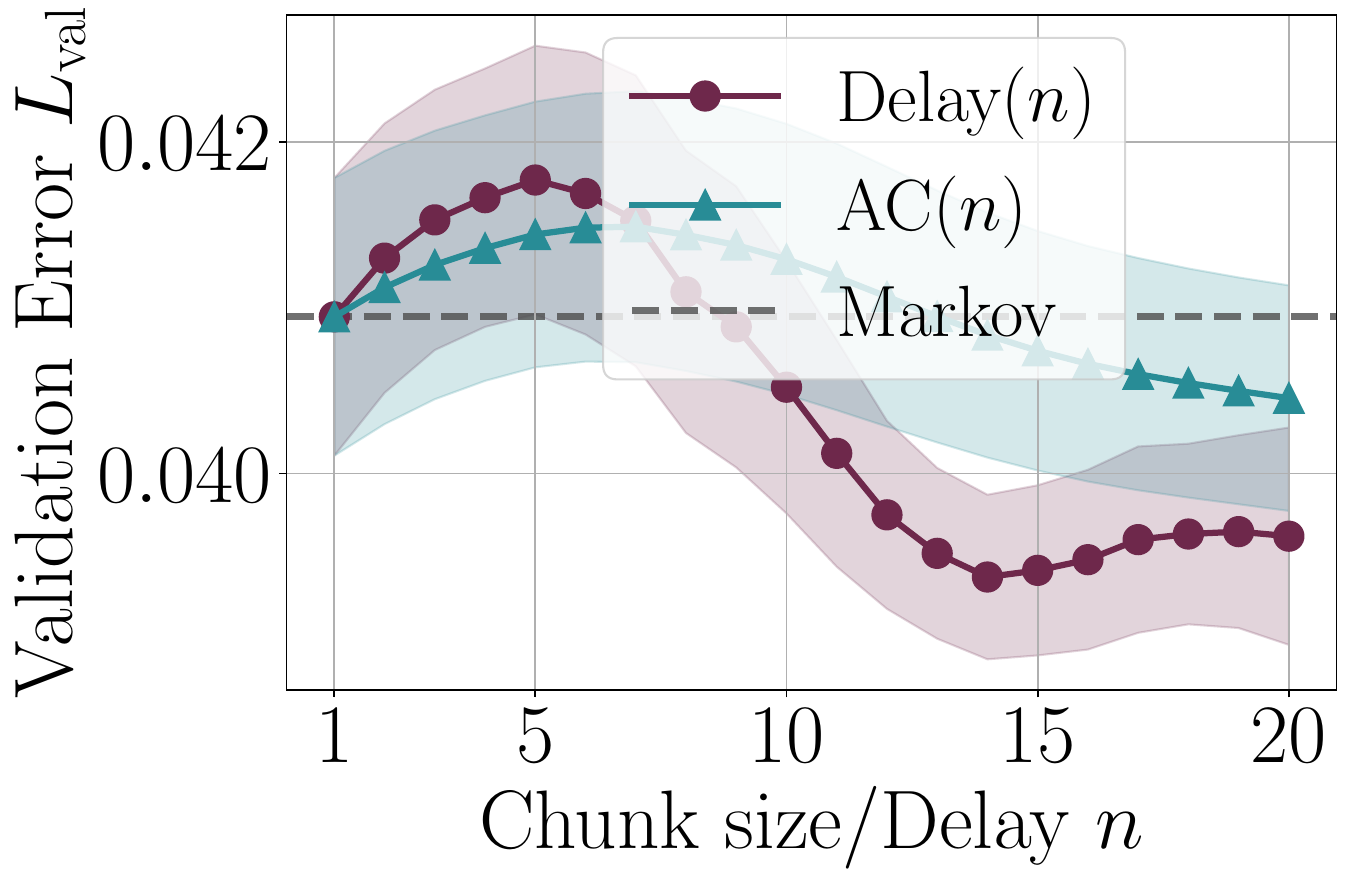}
    \end{minipage}
    \hfill
        \begin{minipage}[t]{0.23\textwidth}
        \centering
        \includegraphics[width=\linewidth]{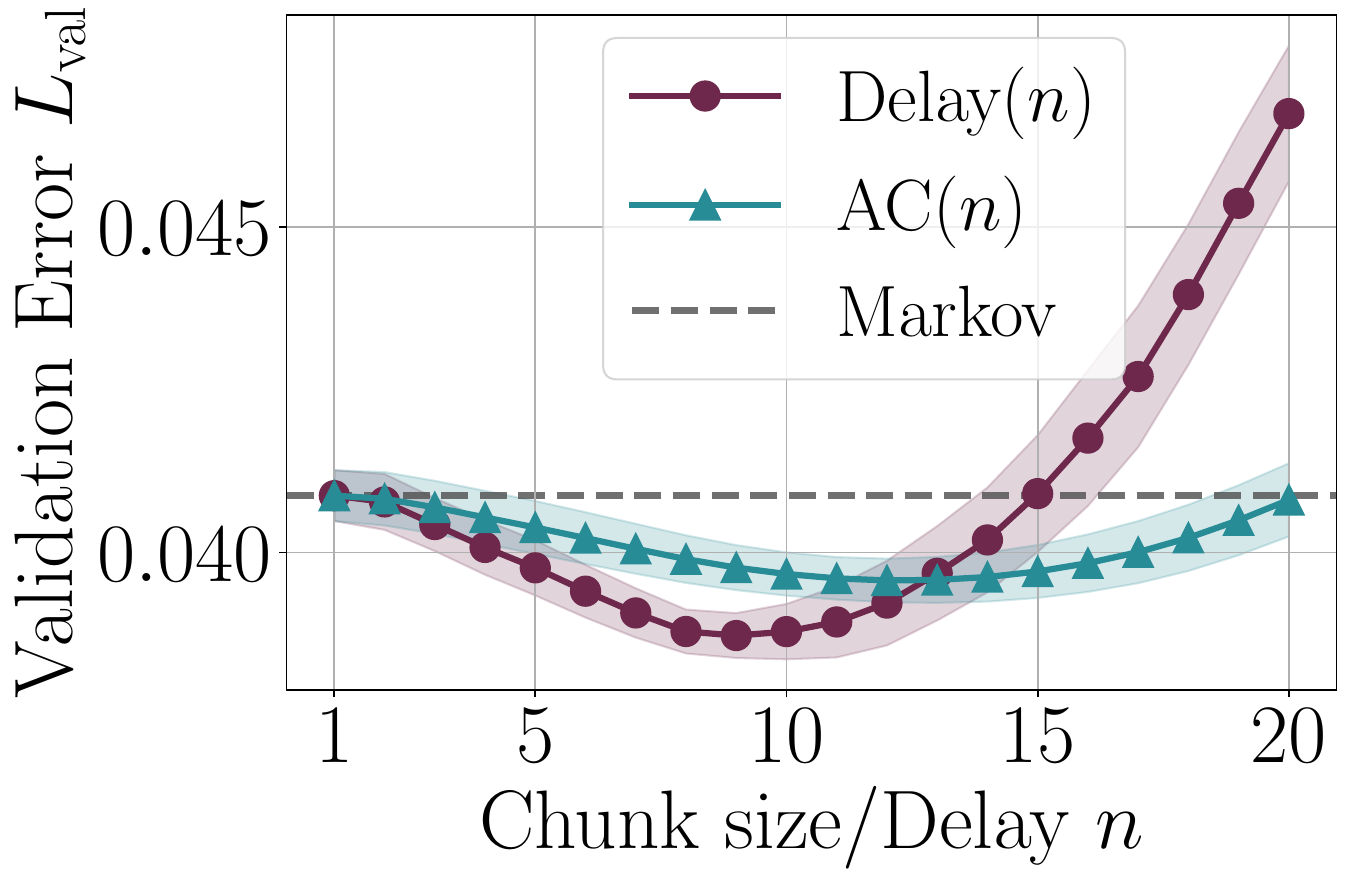}
    \end{minipage}
    \hfill
        \begin{minipage}[t]{0.23\textwidth}
        \centering
        \includegraphics[width=\linewidth]{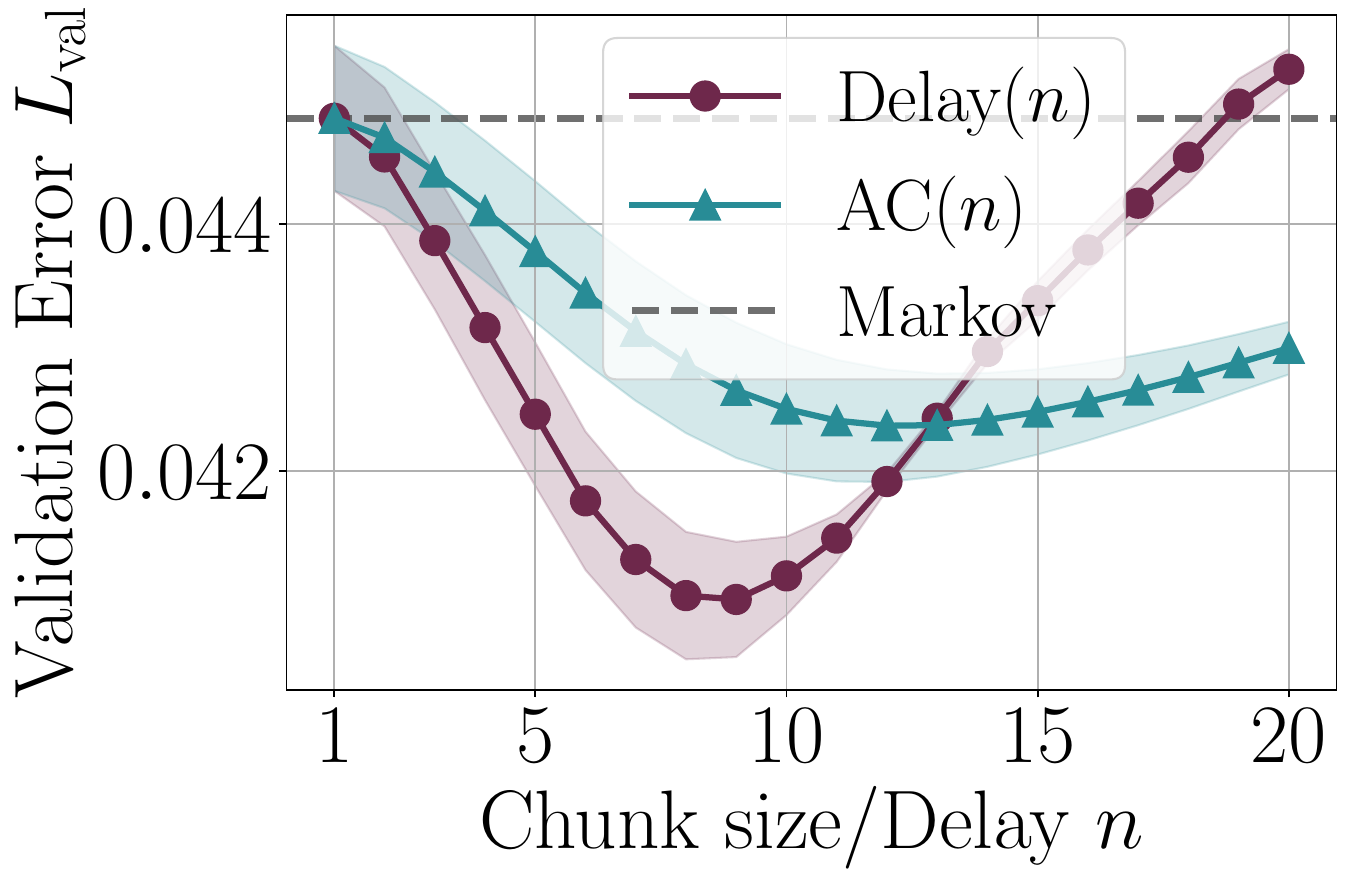}
    \end{minipage}
        \begin{minipage}[t]{0.23\textwidth}
            \includegraphics[width=\linewidth]{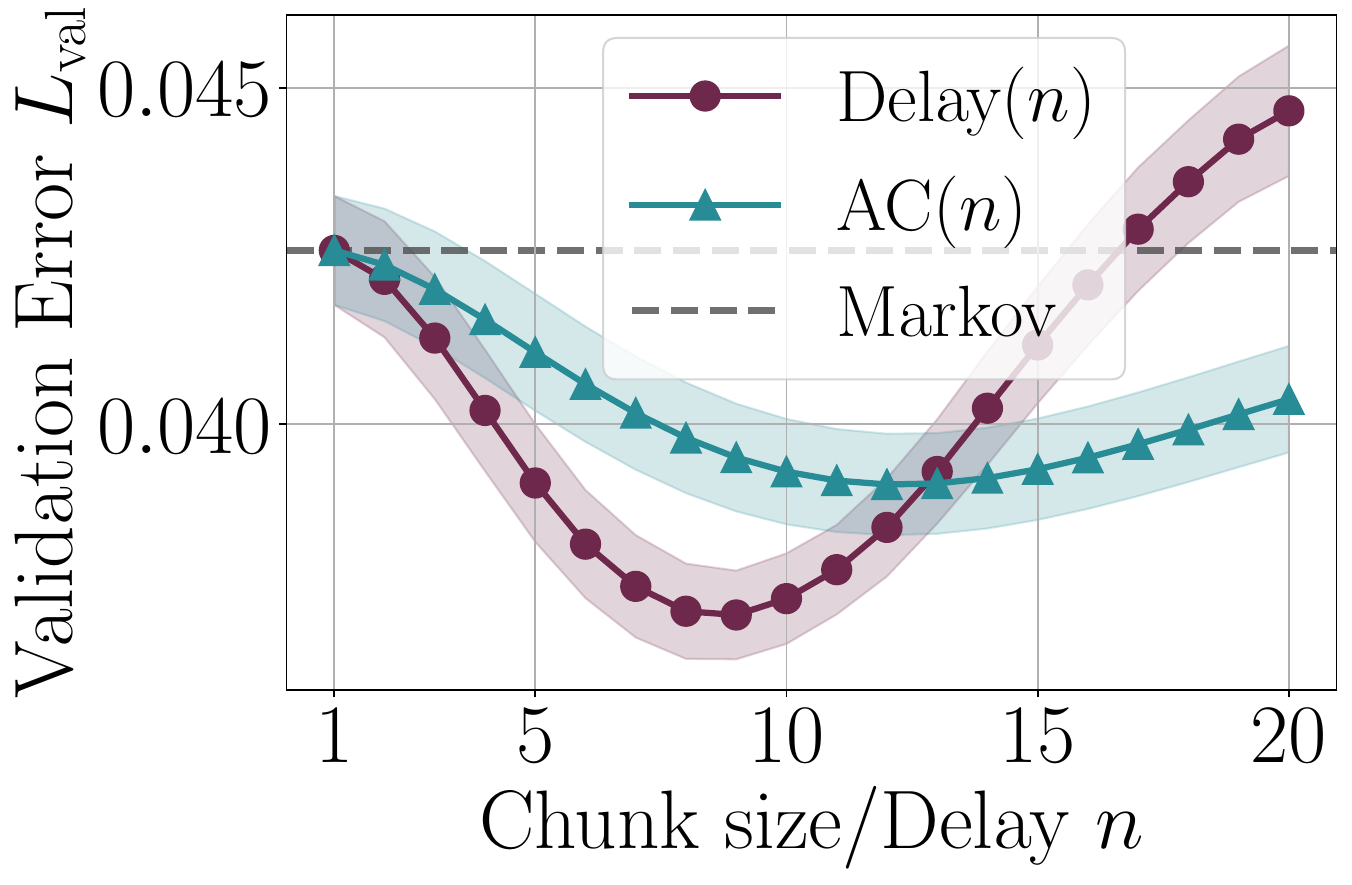}
    \end{minipage}
    \hfill
        \begin{minipage}[t]{0.23\textwidth}
        \centering
        \includegraphics[width=\linewidth]{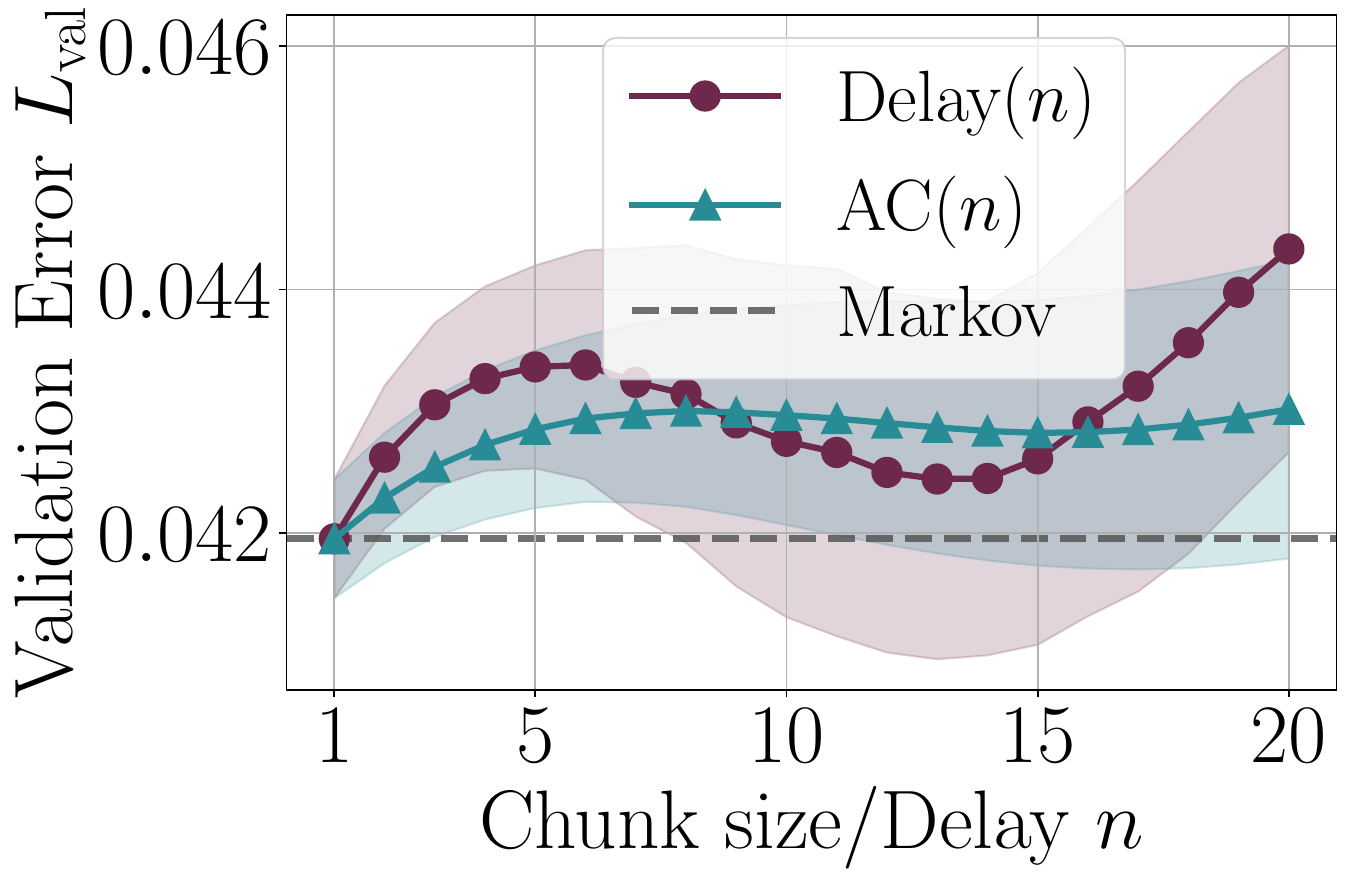}
    \end{minipage}
    \hfill
        \begin{minipage}[t]{0.23\textwidth}
        \centering
        \includegraphics[width=\linewidth]{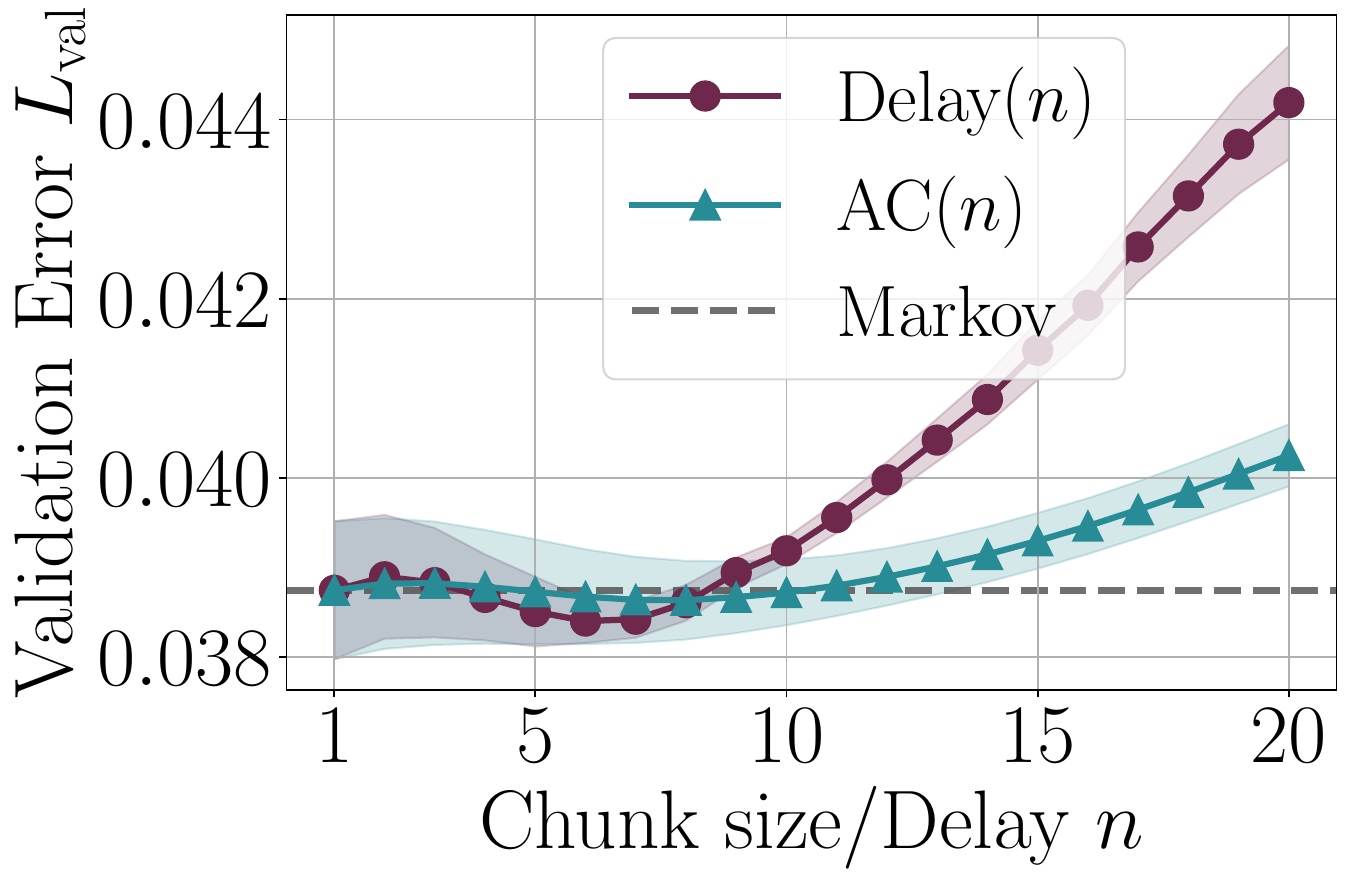}
    \end{minipage}
    \hfill
        \begin{minipage}[t]{0.23\textwidth}
        \centering
        \includegraphics[width=\linewidth]{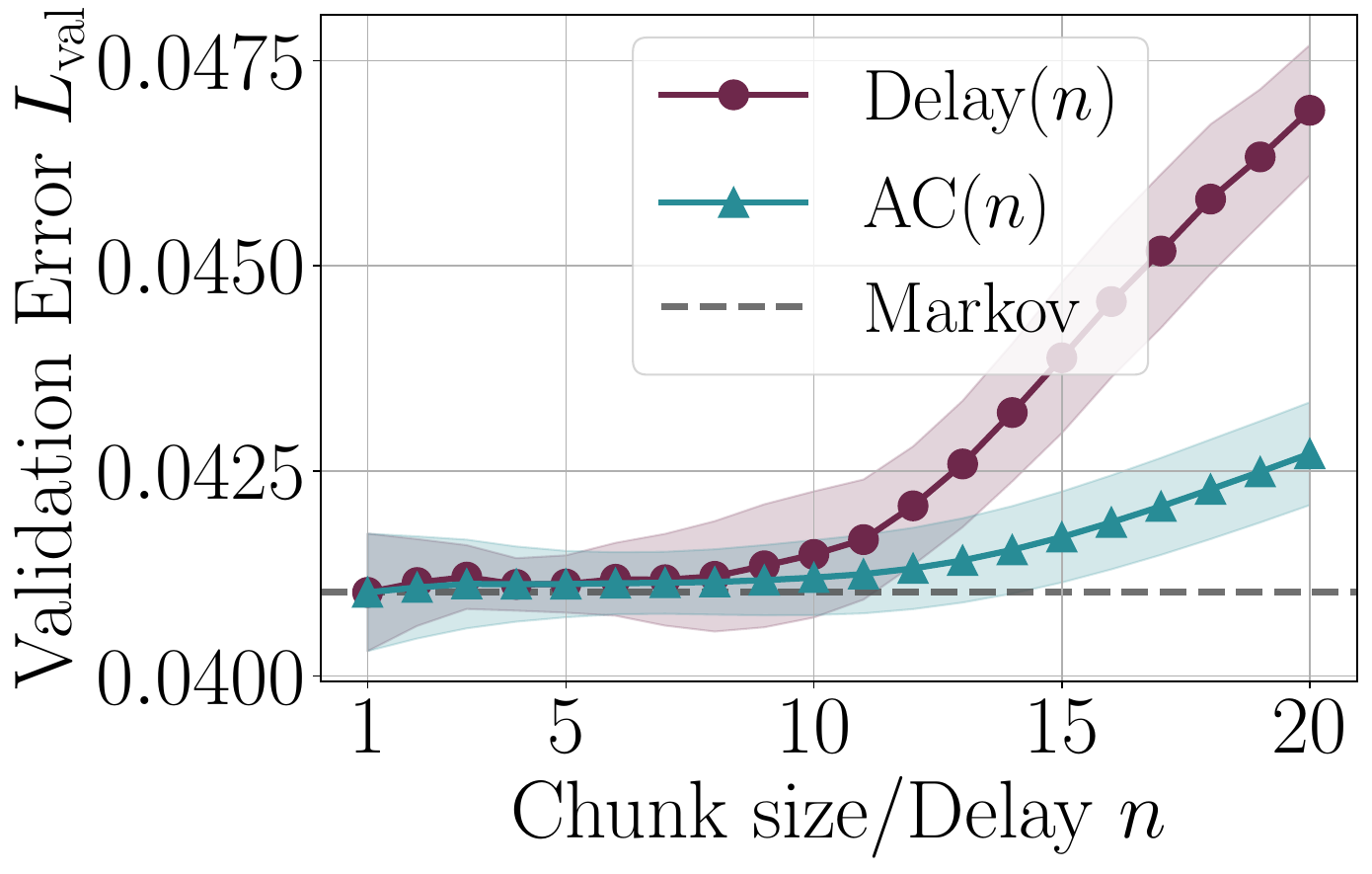}
    \end{minipage}
        \caption{Validation loss for each \texttt{Libero} task from 36 to 67 (corresponding to Fig. \ref{fig:val_loss_libero}), part 2.}
    \label{fig:val loss each libero2}
\end{figure*}

\begin{figure*}
        \begin{minipage}[t]{0.23\textwidth}
            \includegraphics[width=\linewidth]{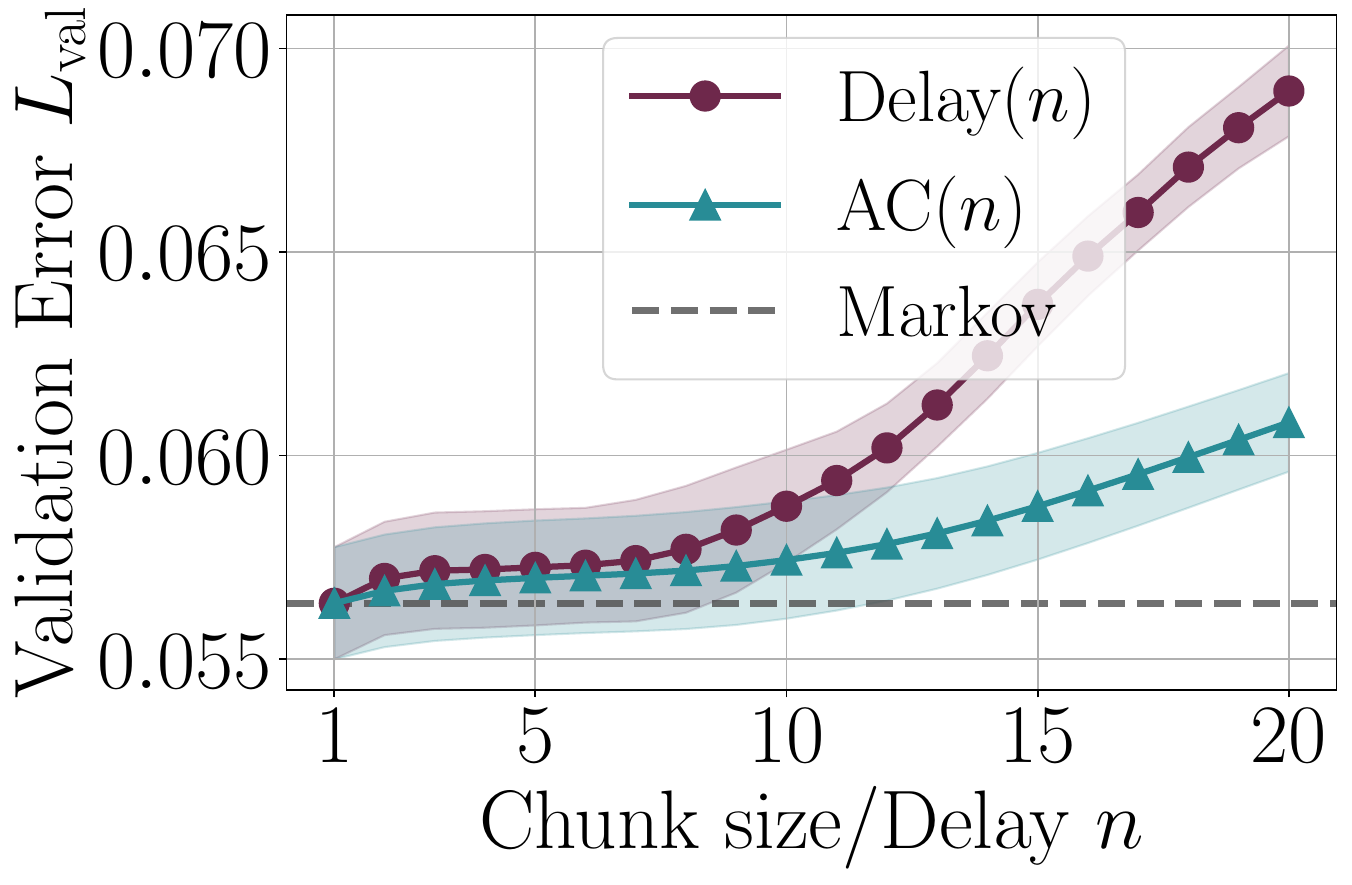}
    \end{minipage}
    \hfill
        \begin{minipage}[t]{0.23\textwidth}
        \centering
        \includegraphics[width=\linewidth]{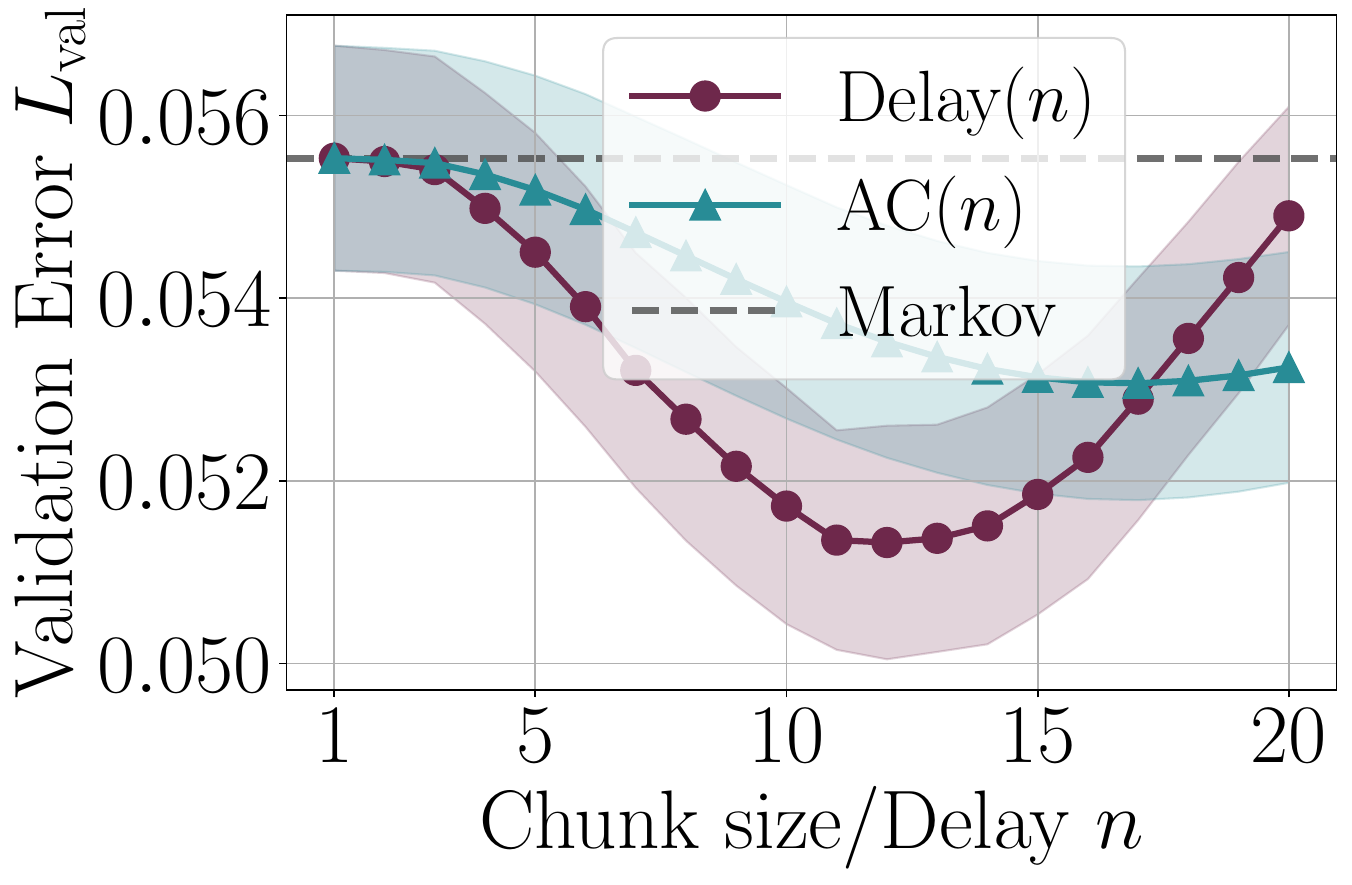}
    \end{minipage}
    \hfill
        \begin{minipage}[t]{0.23\textwidth}
        \centering
        \includegraphics[width=\linewidth]{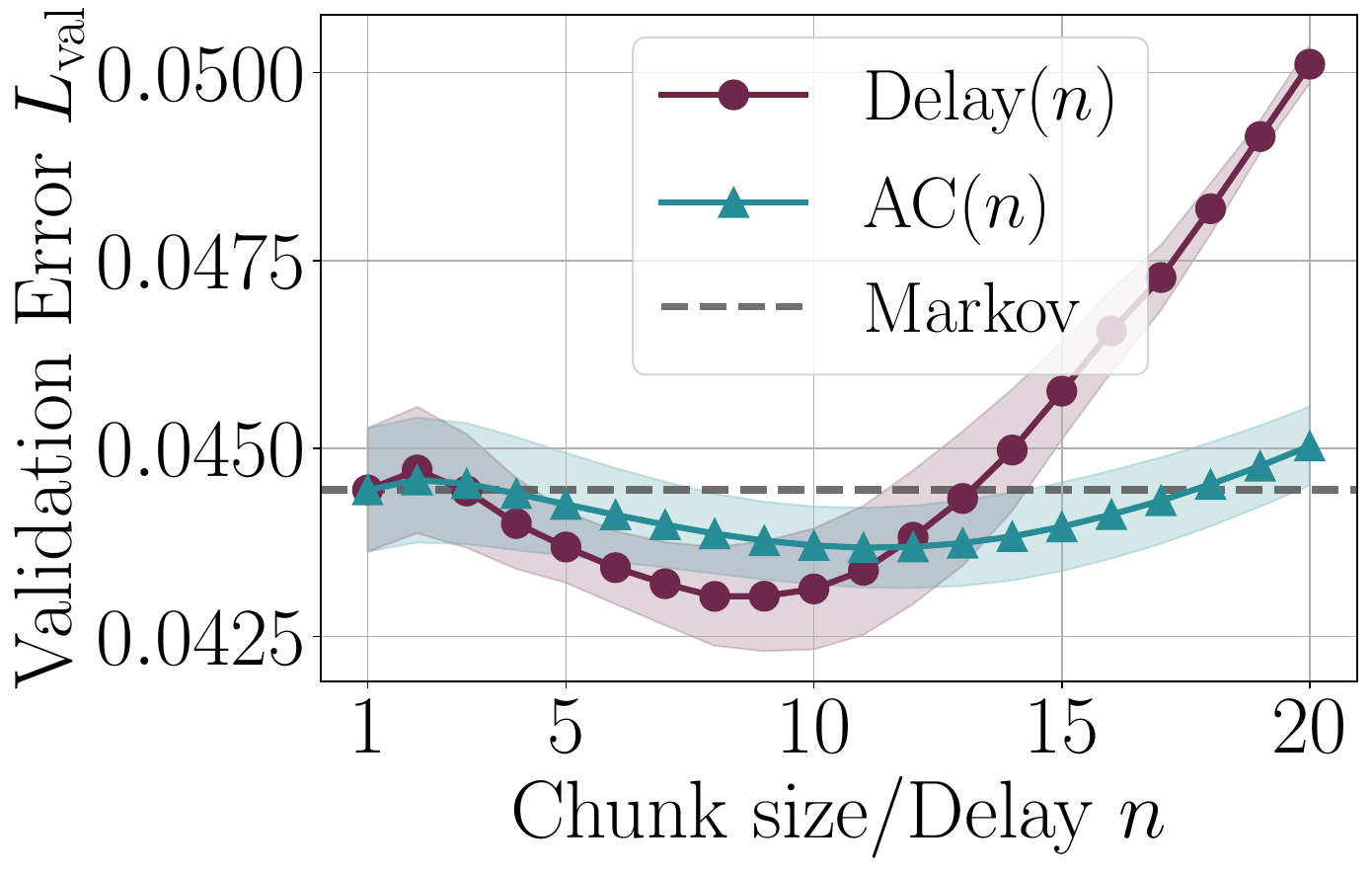}
    \end{minipage}
    \hfill
        \begin{minipage}[t]{0.23\textwidth}
        \centering
        \includegraphics[width=\linewidth]{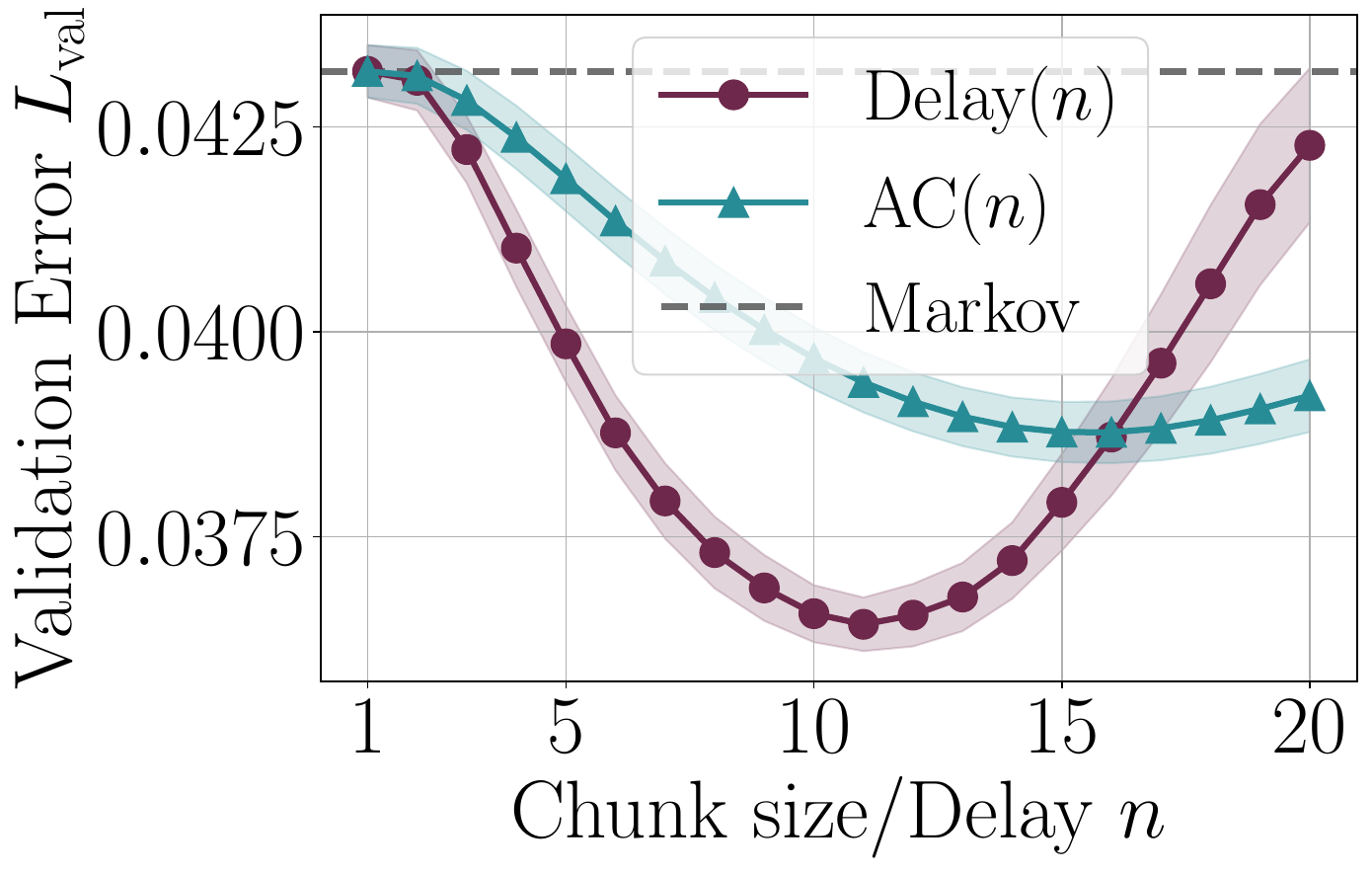}
    \end{minipage}
        \begin{minipage}[t]{0.23\textwidth}
            \includegraphics[width=\linewidth]{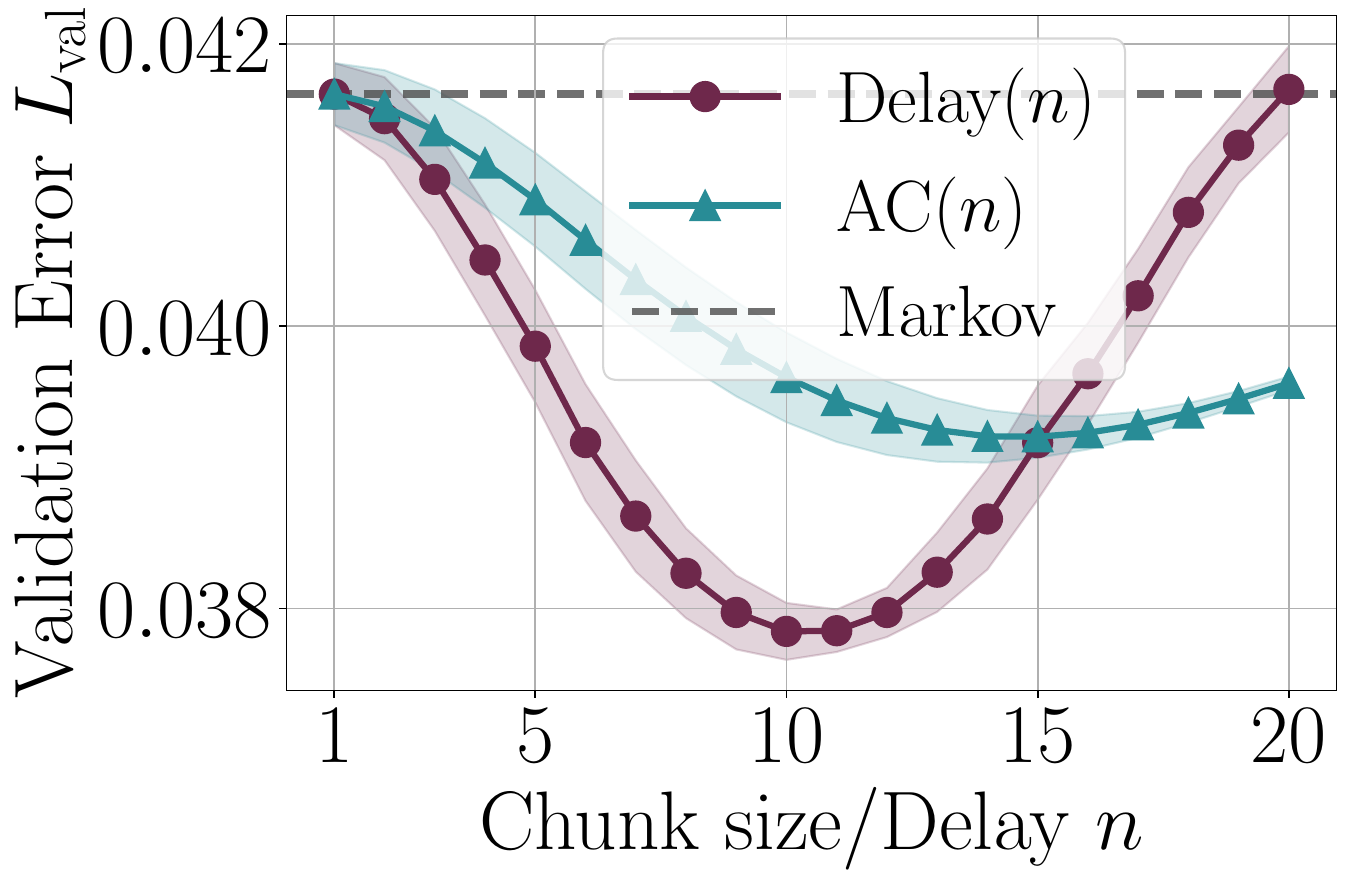}
    \end{minipage}
    \hfill
        \begin{minipage}[t]{0.23\textwidth}
        \centering
        \includegraphics[width=\linewidth]{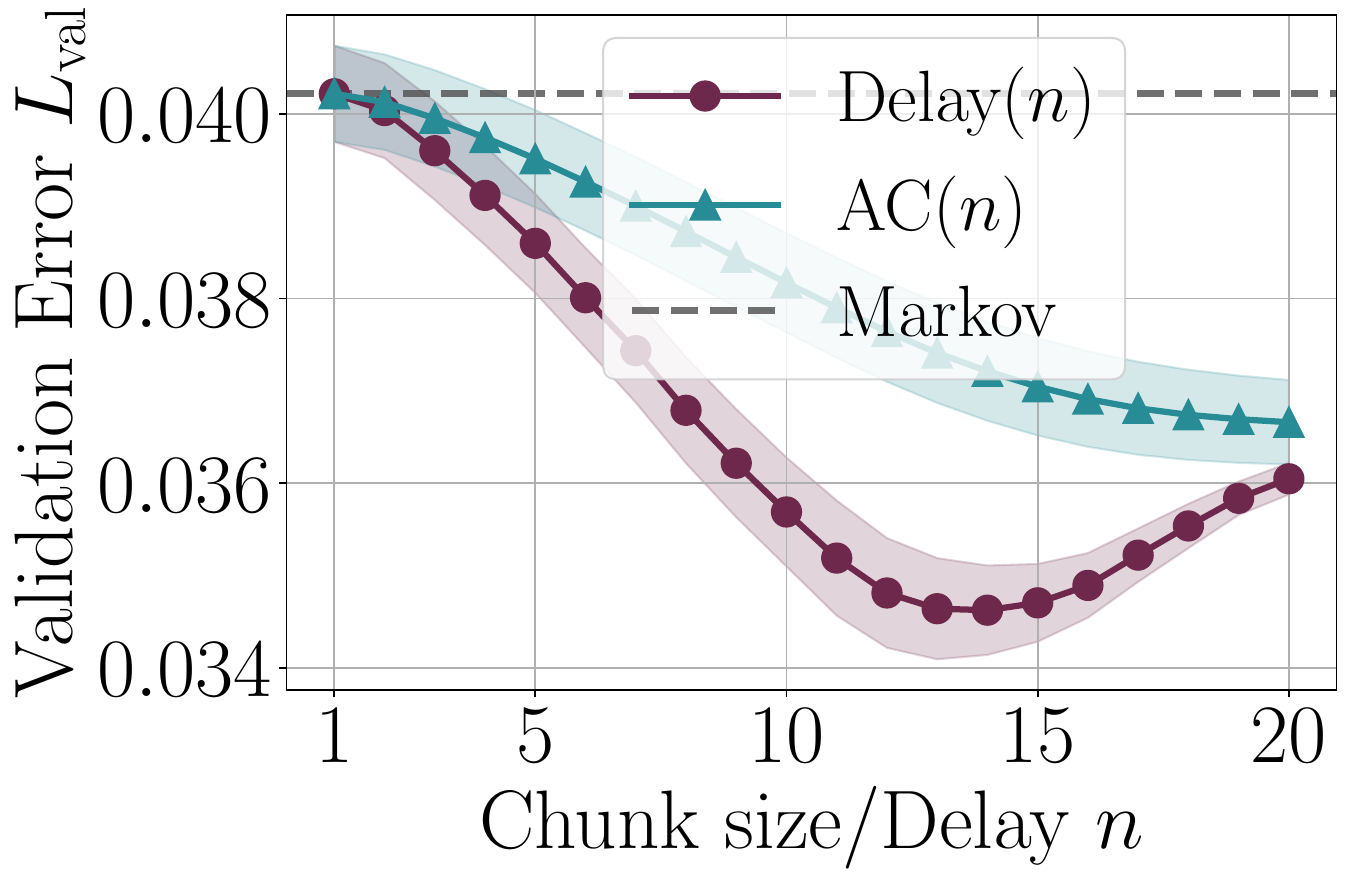}
    \end{minipage}
    \hfill
        \begin{minipage}[t]{0.23\textwidth}
        \centering
        \includegraphics[width=\linewidth]{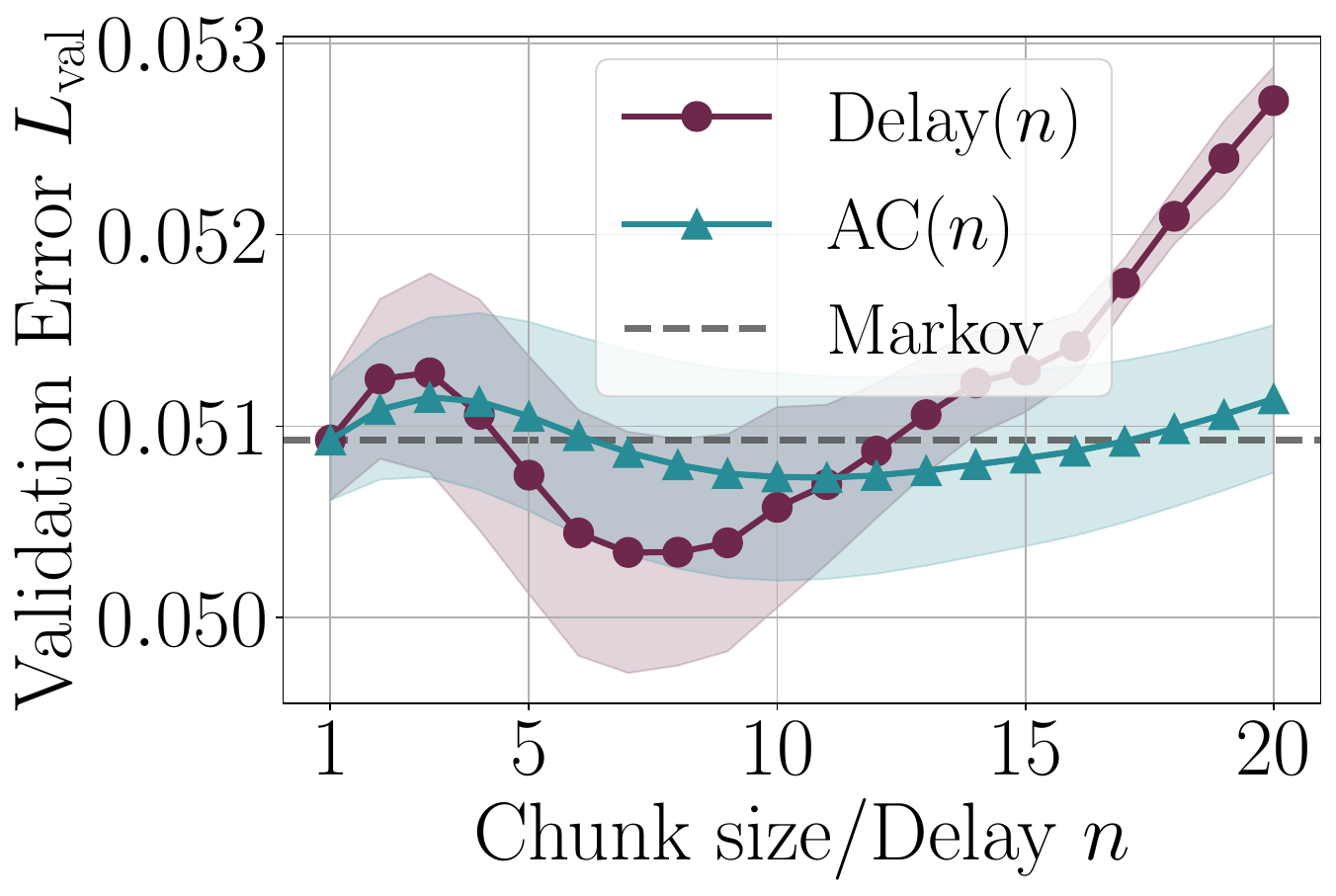}
    \end{minipage}
    \hfill
        \begin{minipage}[t]{0.23\textwidth}
        \centering
        \includegraphics[width=\linewidth]{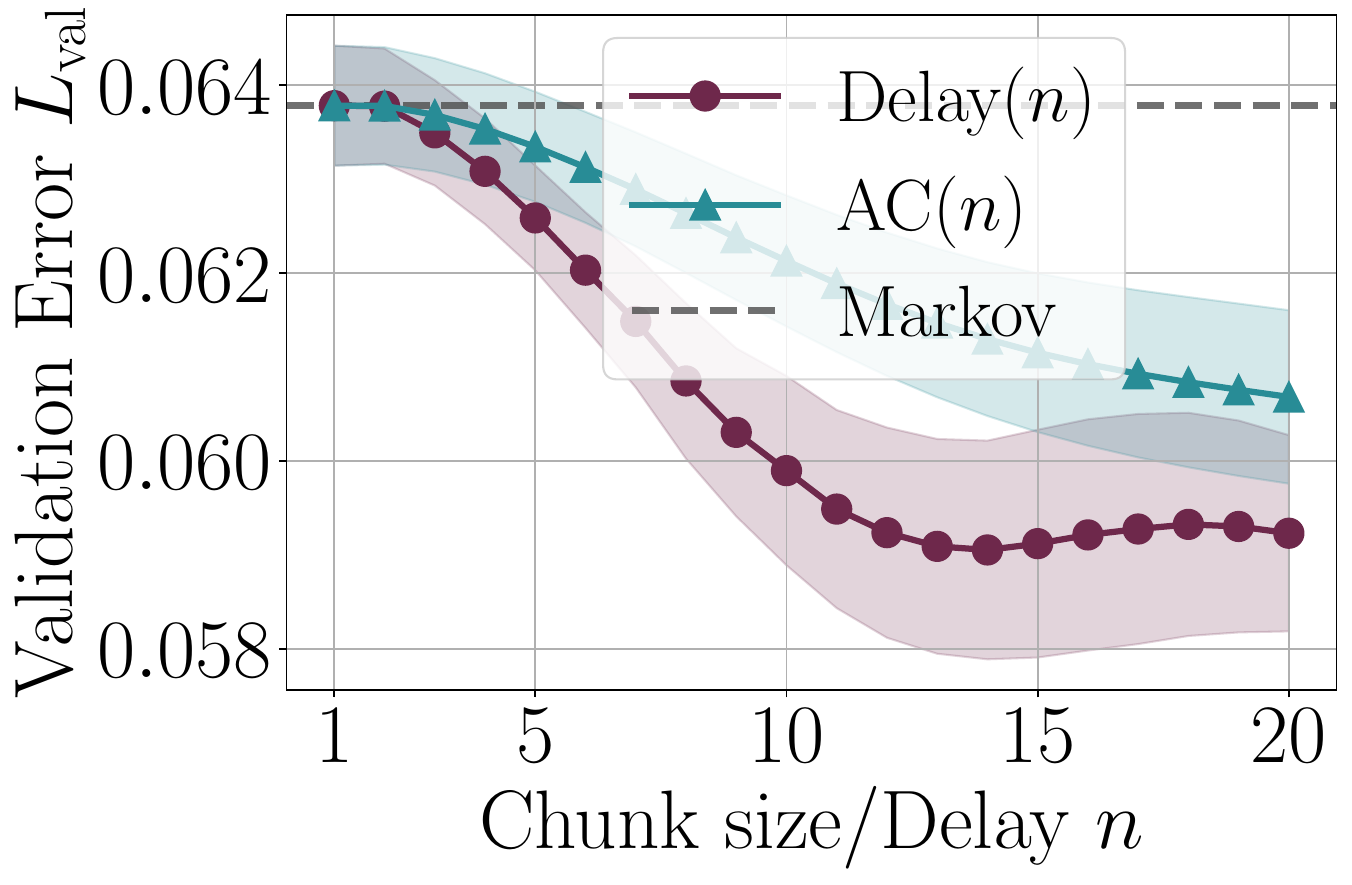}
    \end{minipage}
        \begin{minipage}[t]{0.23\textwidth}
            \includegraphics[width=\linewidth]{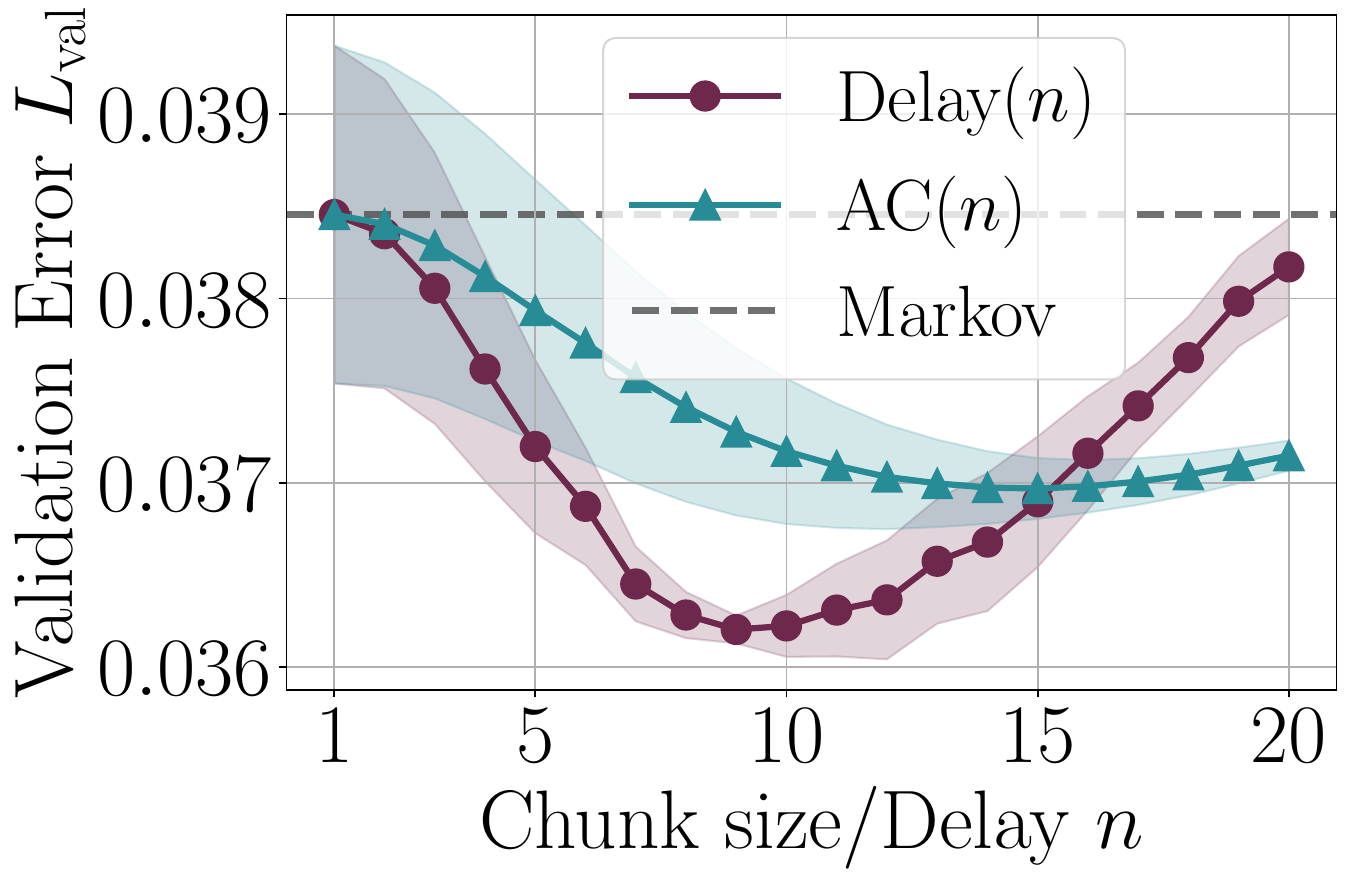}
    \end{minipage}
    \hfill
        \begin{minipage}[t]{0.23\textwidth}
        \centering
        \includegraphics[width=\linewidth]{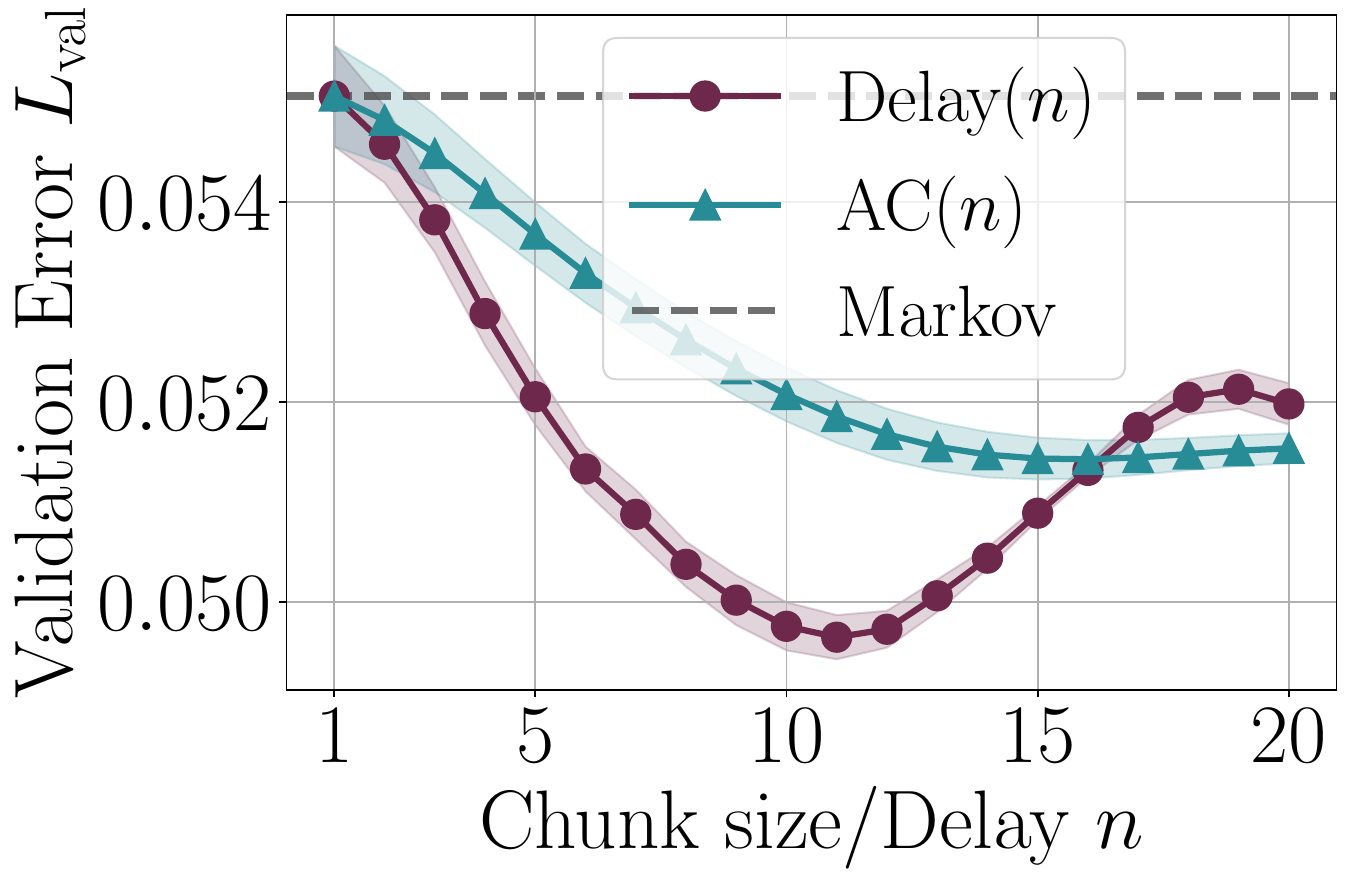}
    \end{minipage}
    \hfill
        \begin{minipage}[t]{0.23\textwidth}
        \centering
        \includegraphics[width=\linewidth]{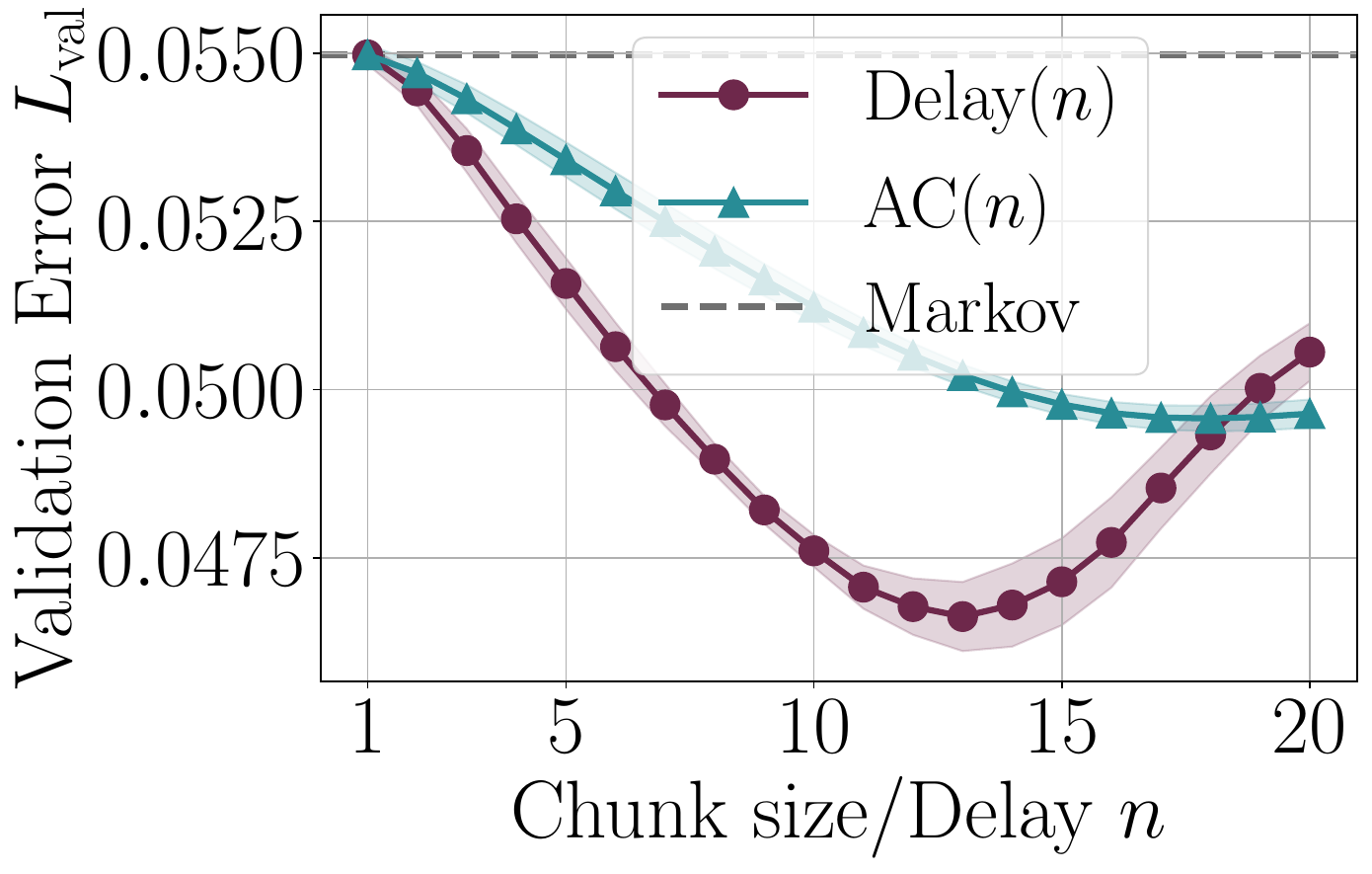}
    \end{minipage}
    \hfill
        \begin{minipage}[t]{0.23\textwidth}
        \centering
        \includegraphics[width=\linewidth]{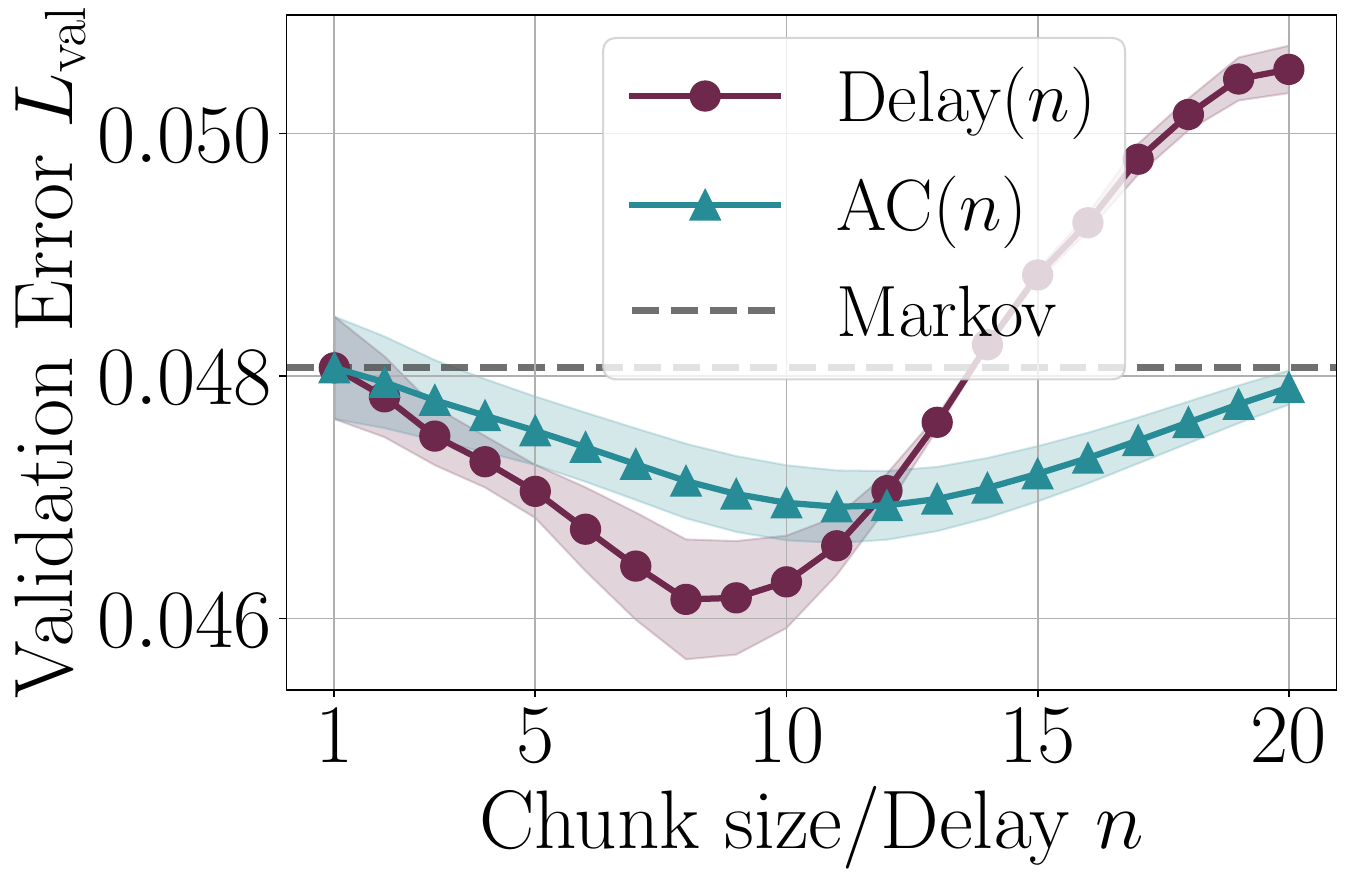}
    \end{minipage}
        \begin{minipage}[t]{0.23\textwidth}
            \includegraphics[width=\linewidth]{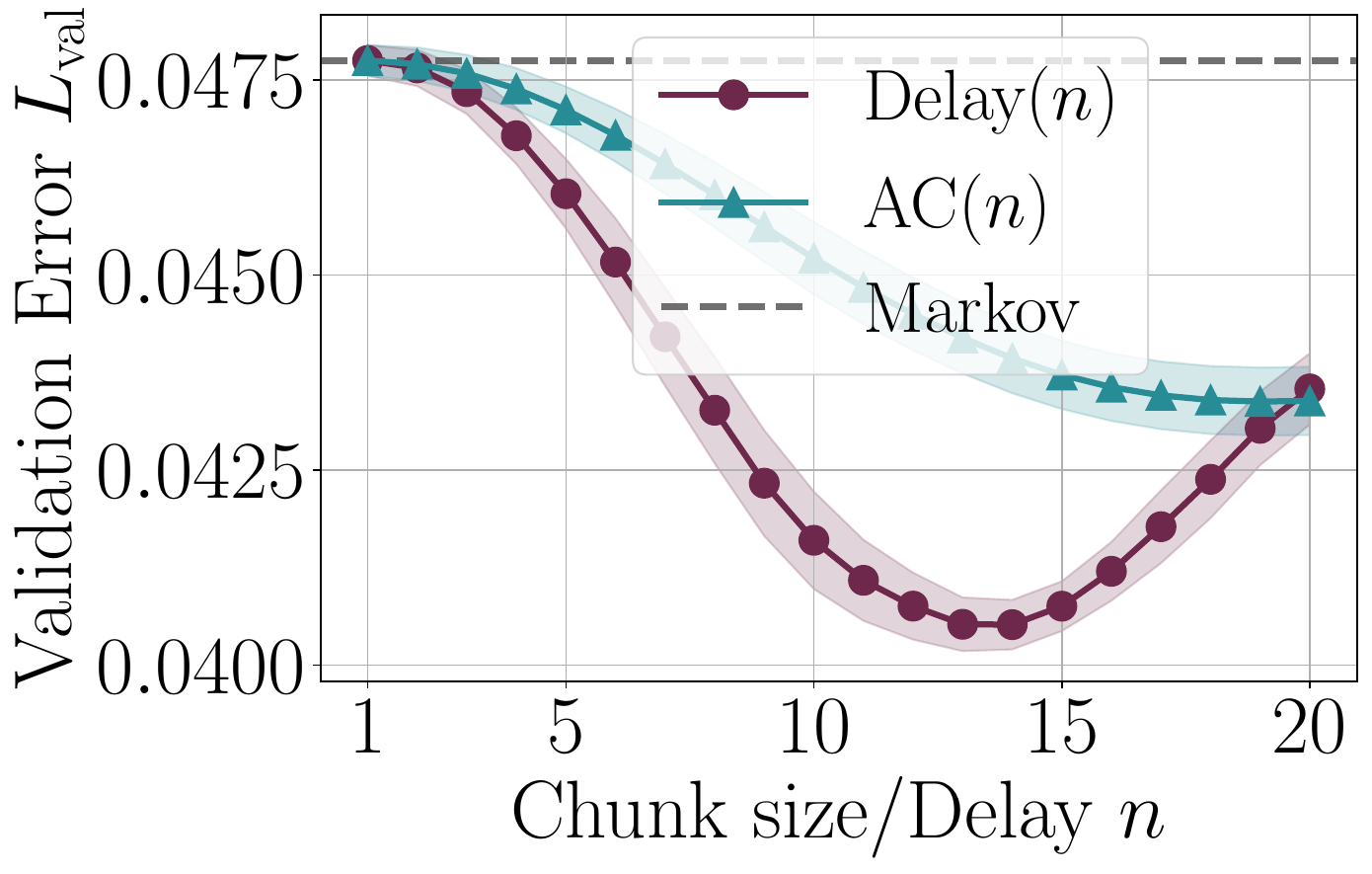}
    \end{minipage}
    \hfill
        \begin{minipage}[t]{0.23\textwidth}
        \centering
        \includegraphics[width=\linewidth]{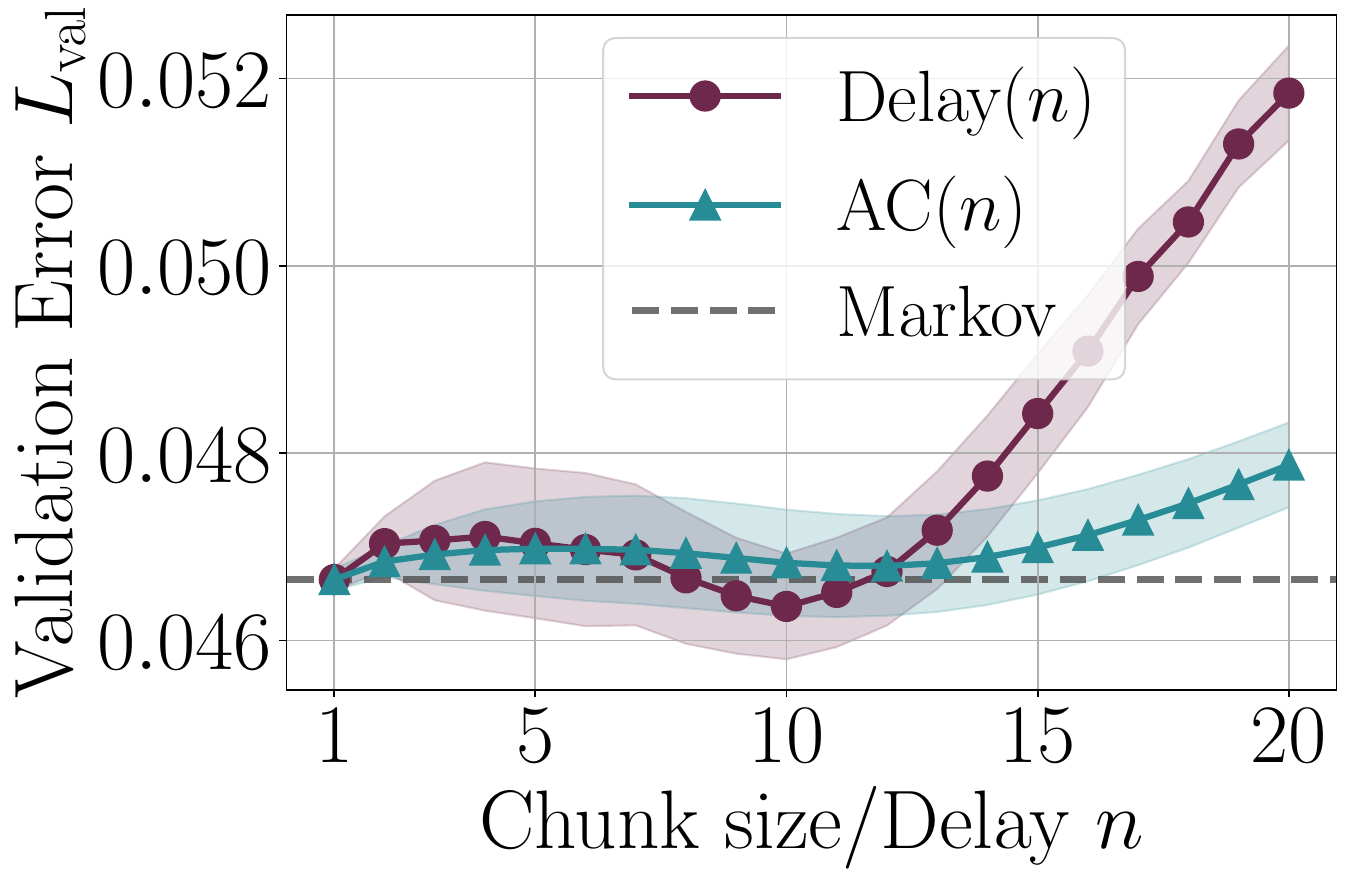}
    \end{minipage}
    \hfill
        \begin{minipage}[t]{0.23\textwidth}
        \centering
        \includegraphics[width=\linewidth]{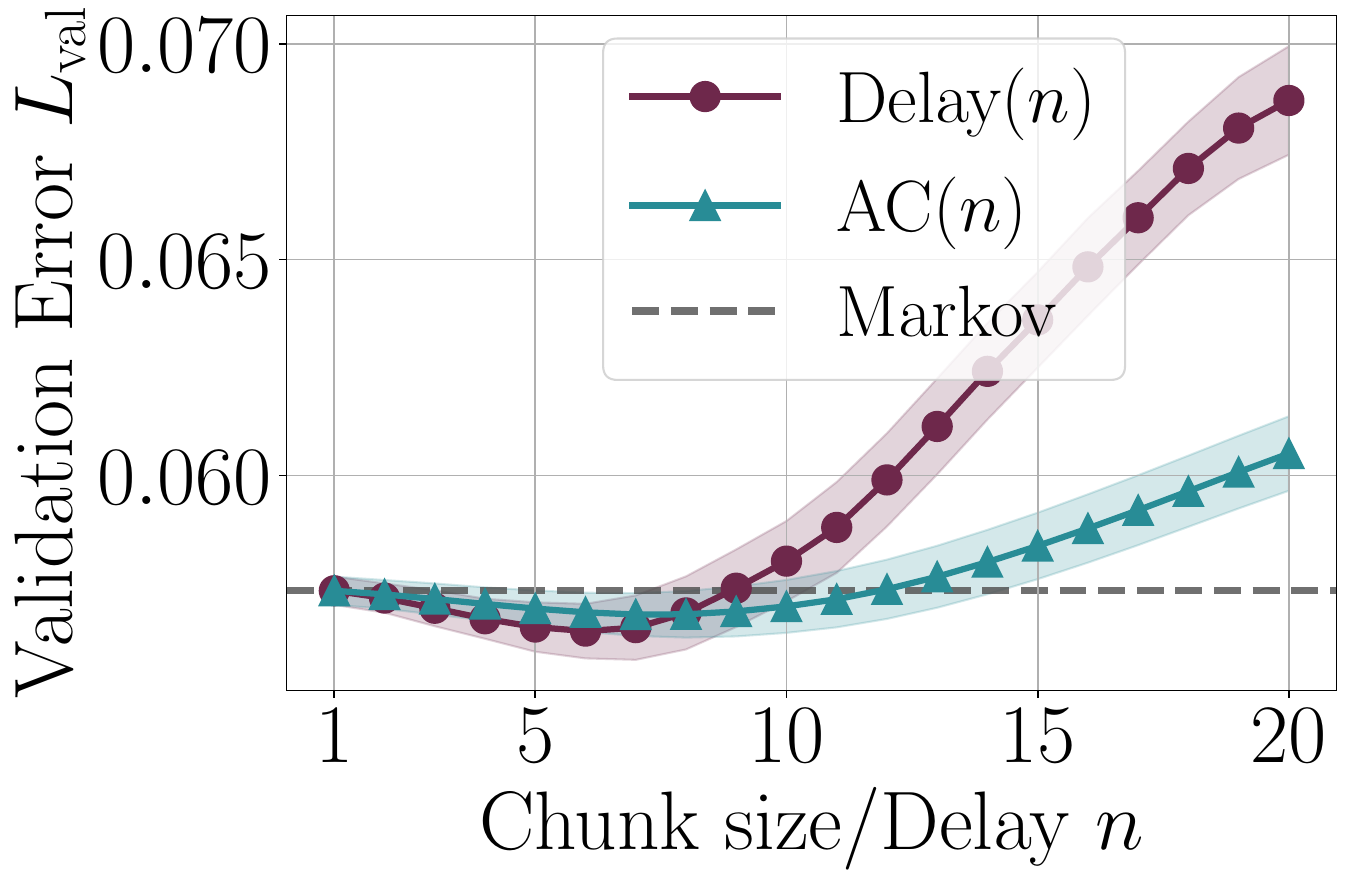}
    \end{minipage}
    \hfill
        \begin{minipage}[t]{0.23\textwidth}
        \centering
        \includegraphics[width=\linewidth]{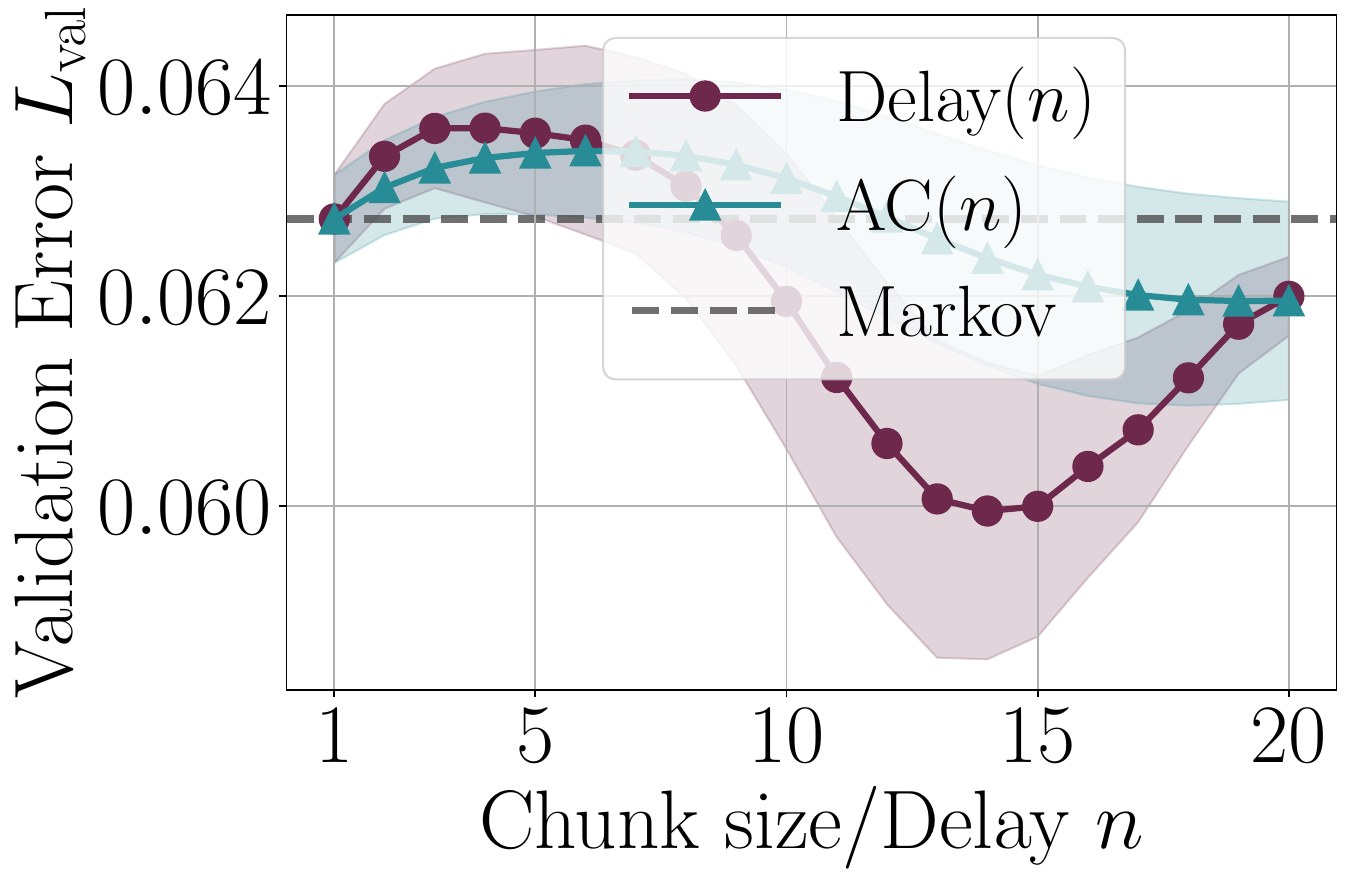}
    \end{minipage}
        \begin{minipage}[t]{0.23\textwidth}
            \includegraphics[width=\linewidth]{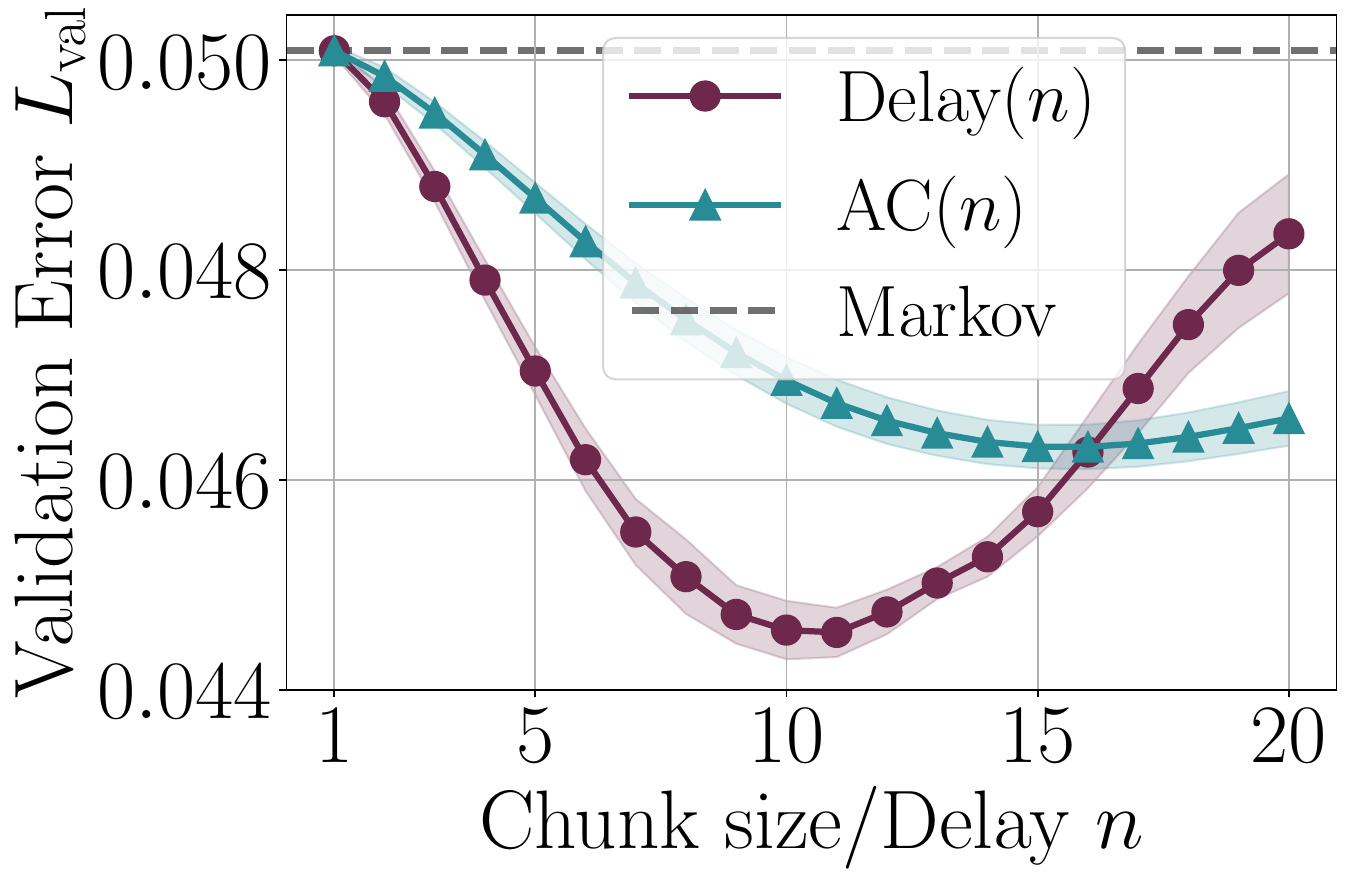}
    \end{minipage}
    \hfill
        \begin{minipage}[t]{0.23\textwidth}
        \centering
        \includegraphics[width=\linewidth]{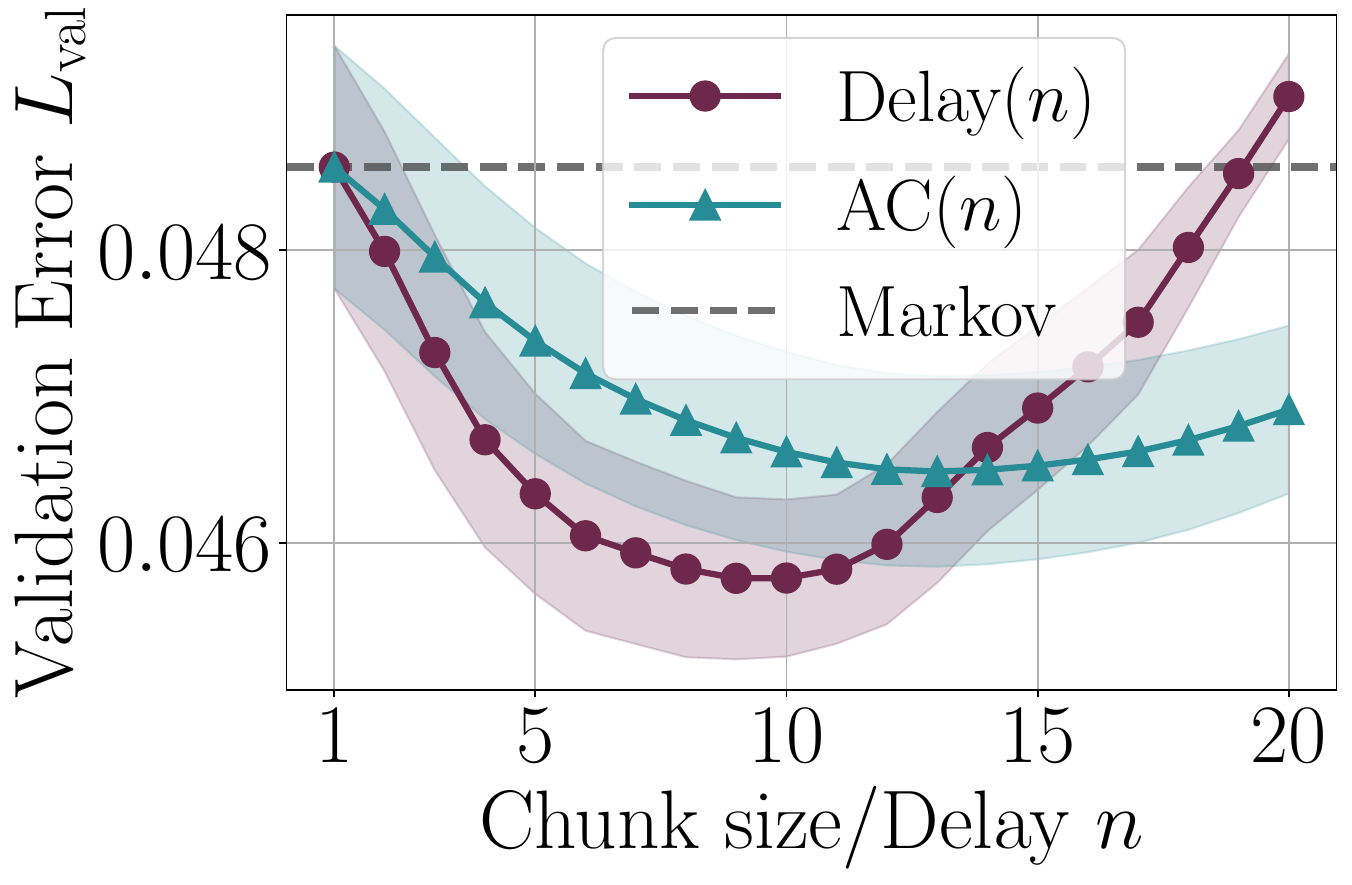}
    \end{minipage}
    \hfill
        \begin{minipage}[t]{0.23\textwidth}
        \centering
        \includegraphics[width=\linewidth]{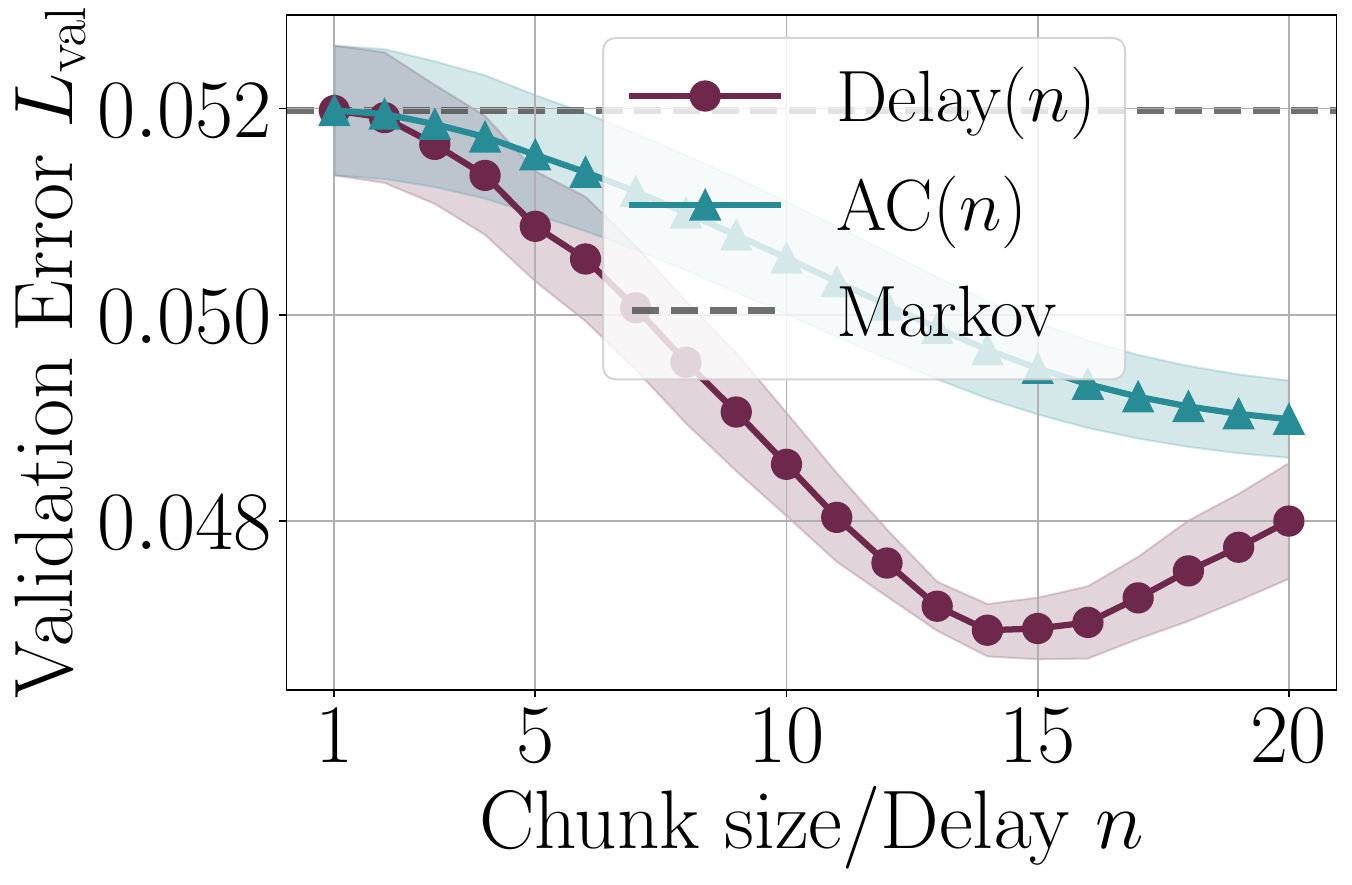}
    \end{minipage}
    \hfill
        \begin{minipage}[t]{0.23\textwidth}
        \centering
        \includegraphics[width=\linewidth]{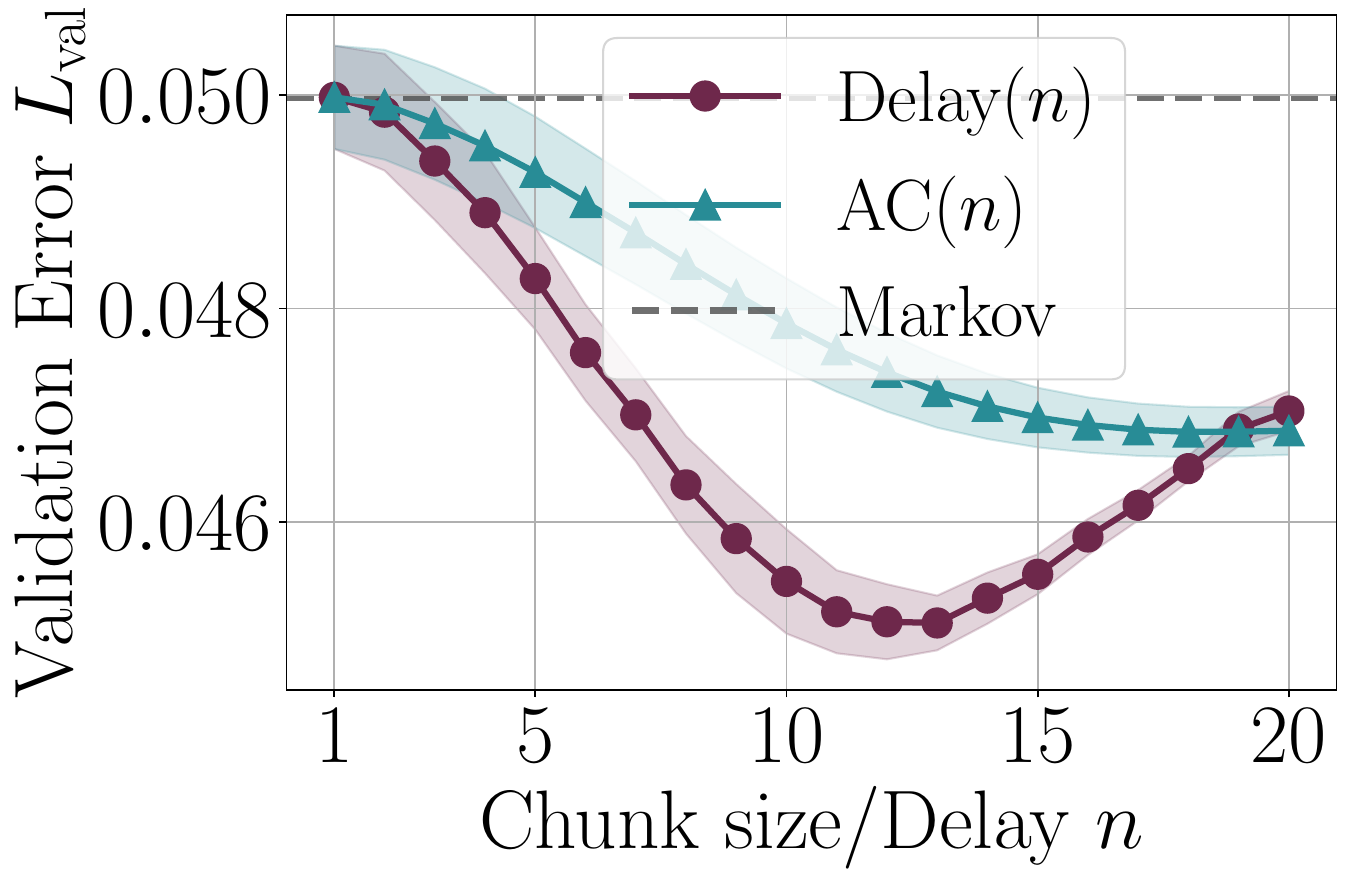}
    \end{minipage}
        \begin{minipage}[t]{0.23\textwidth}
            \includegraphics[width=\linewidth]{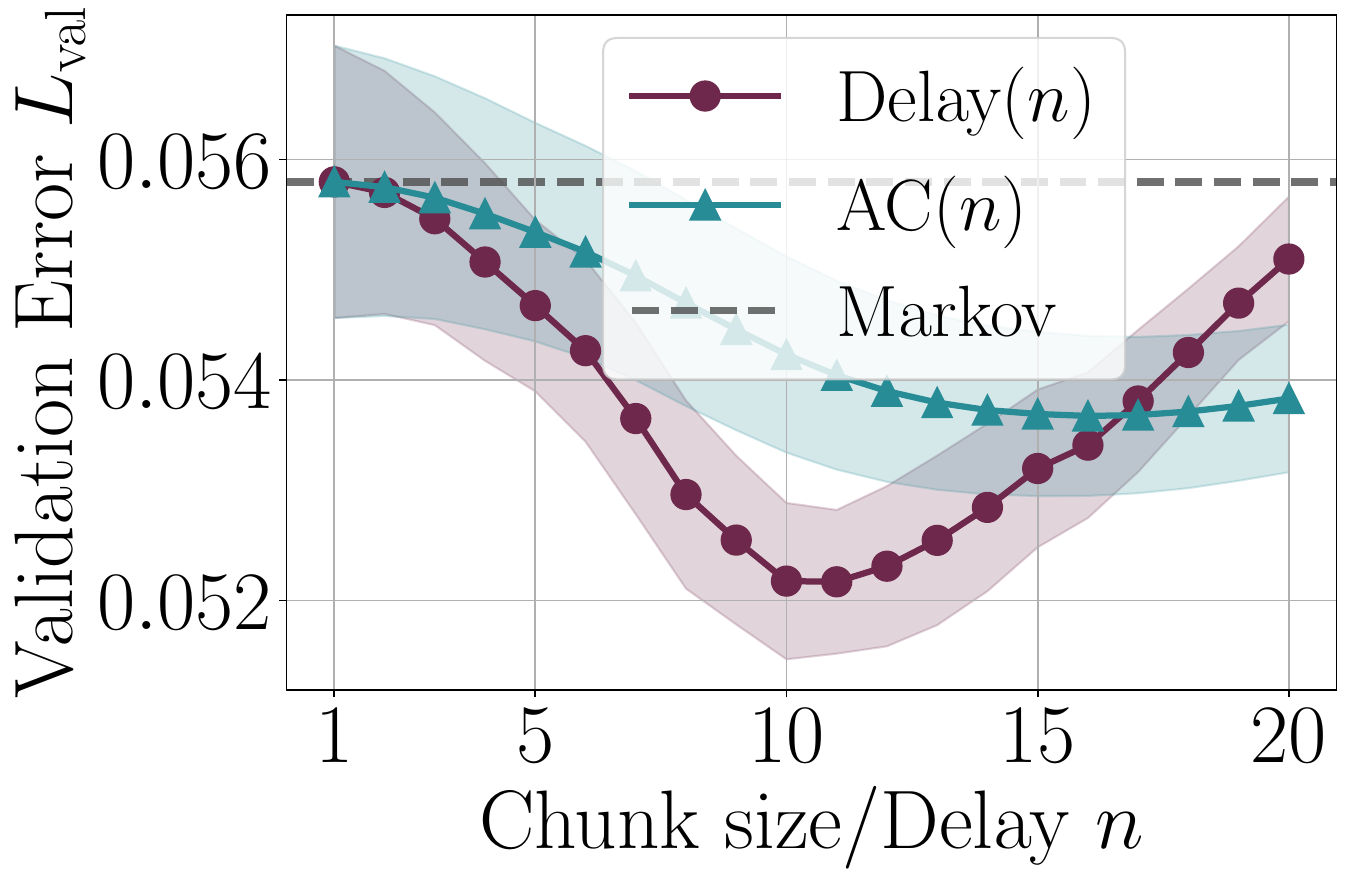}
    \end{minipage}
    \hfill
        \begin{minipage}[t]{0.23\textwidth}
        \centering
        \includegraphics[width=\linewidth]{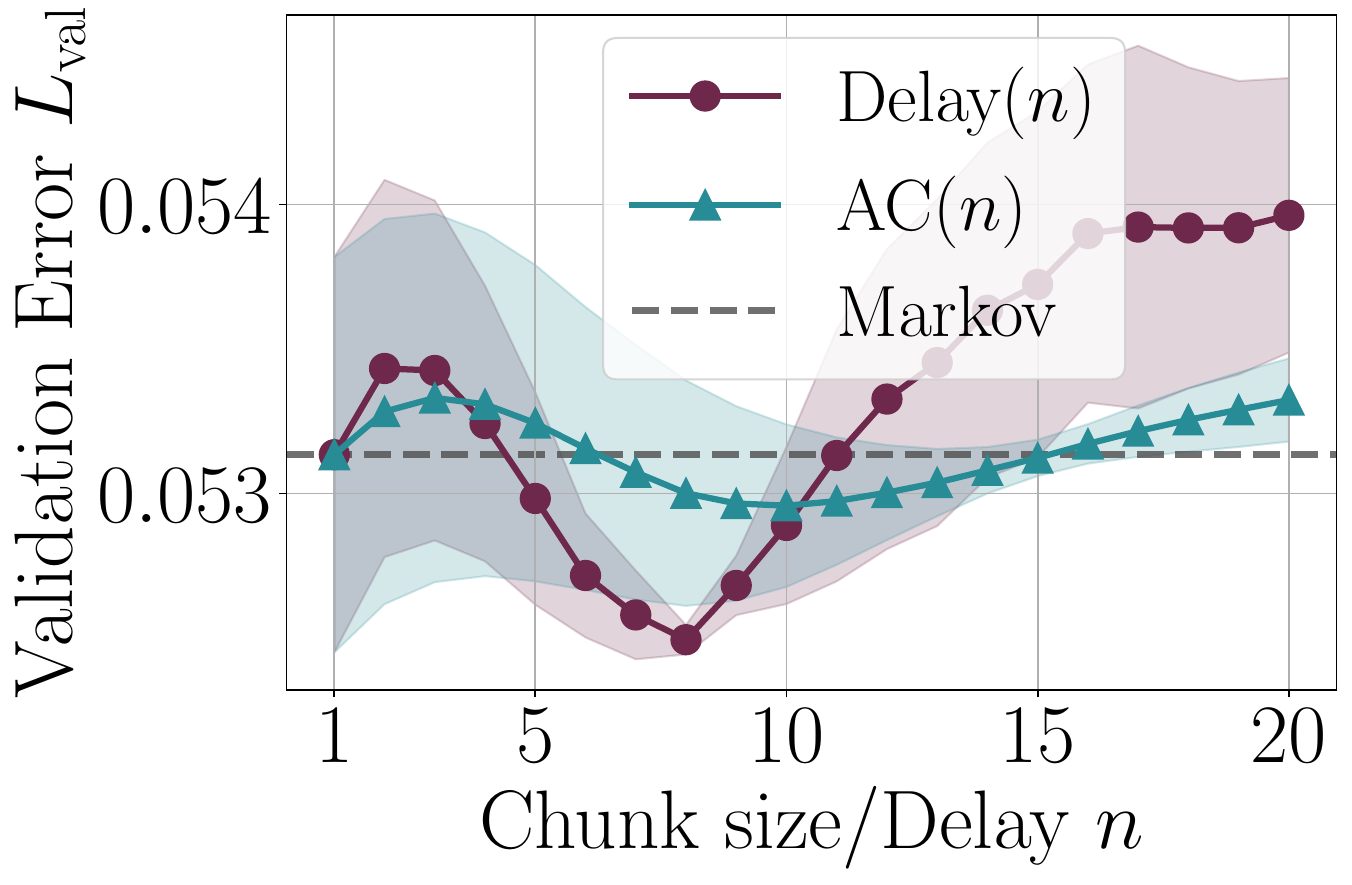}
    \end{minipage}

        \caption{Validation loss for each \texttt{Libero} task from 68 to 89 (corresponding to Fig. \ref{fig:val_loss_libero}), part 3.}
    \label{fig:val loss each libero3}
\end{figure*}

\begin{figure*}[t]
    \centering

    \begin{minipage}[t]{0.23\textwidth}
        \centering
        \includegraphics[width=\linewidth]{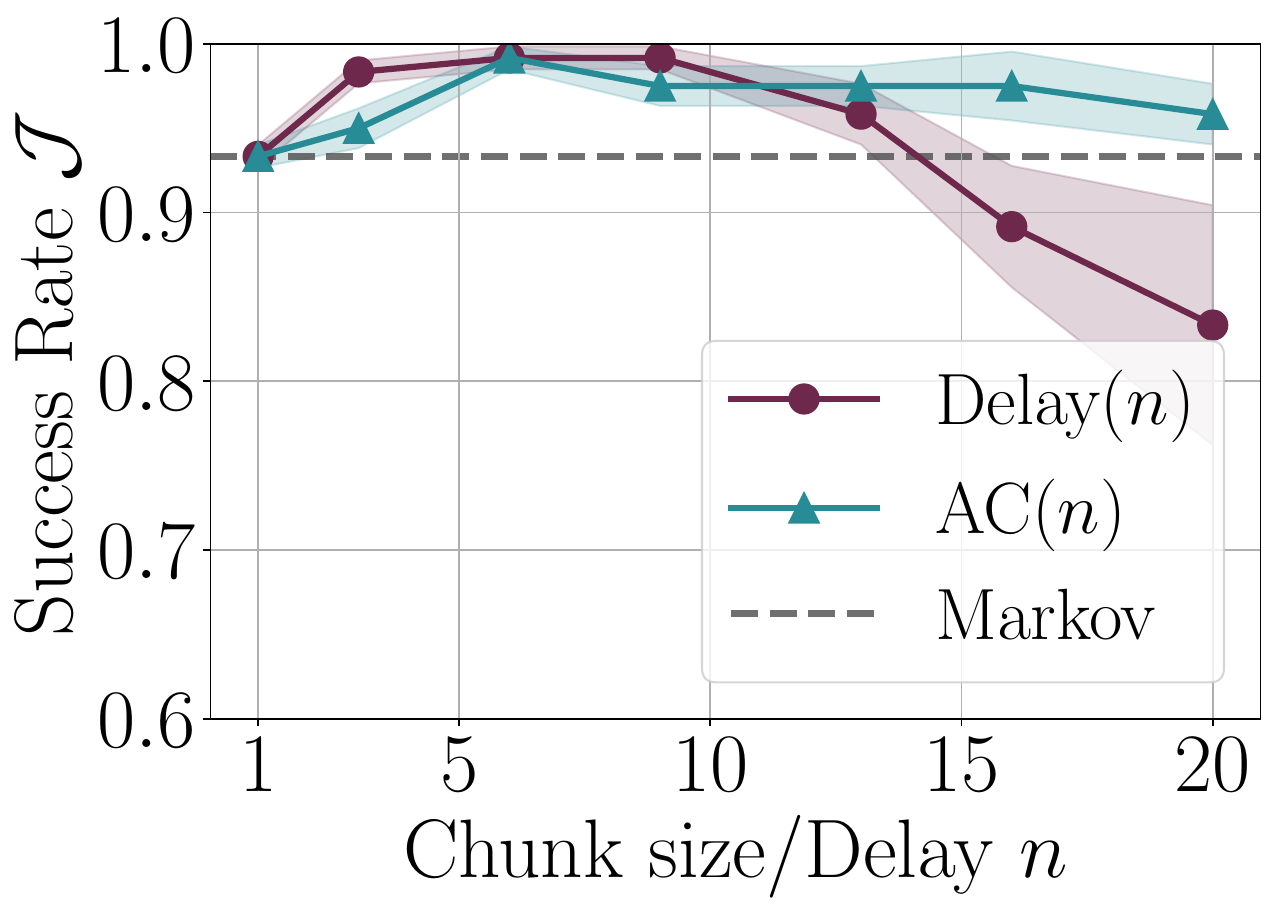}
    \end{minipage}
    \hfill
        \begin{minipage}[t]{0.23\textwidth}
        \centering
        \includegraphics[width=\linewidth]{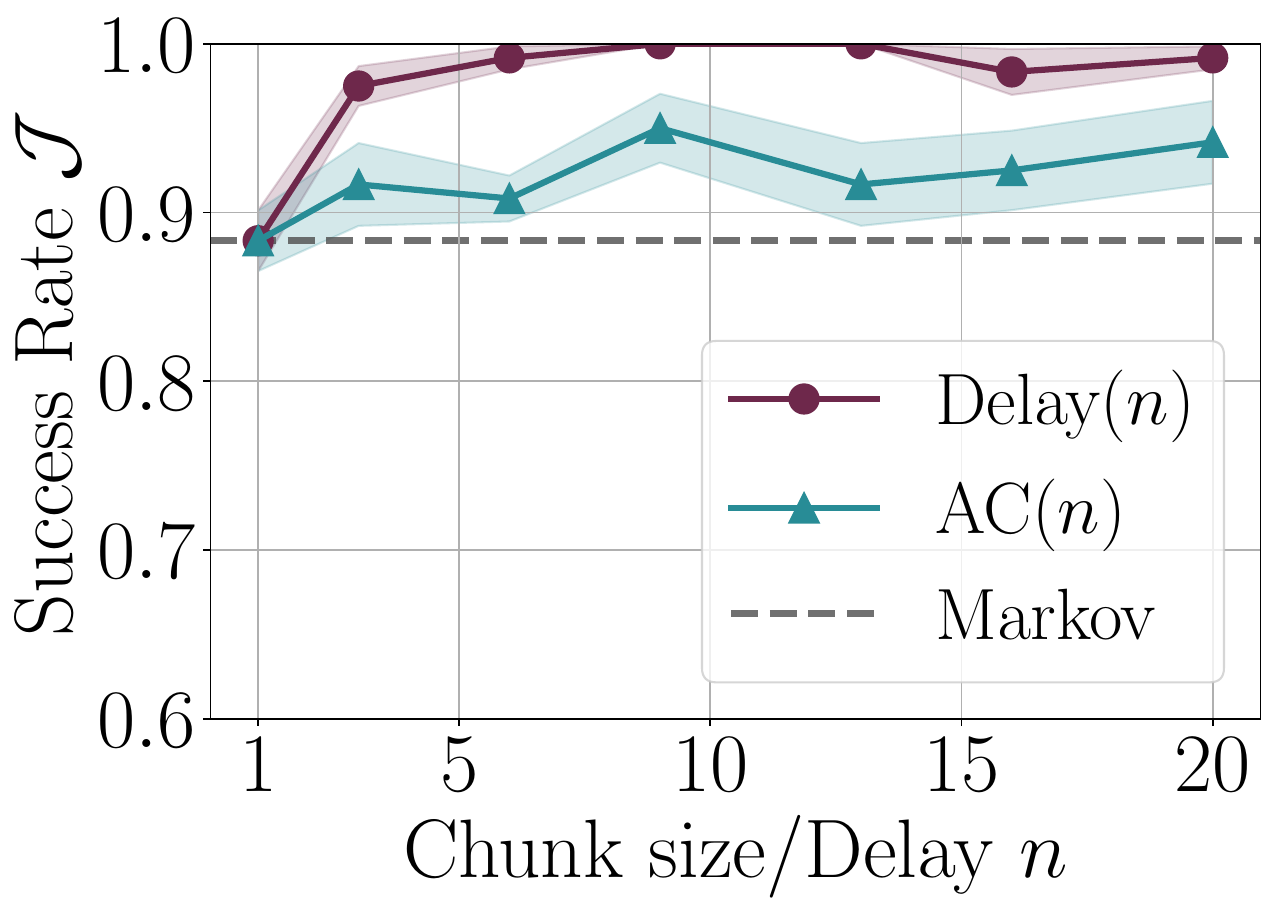}
    \end{minipage}
    \hfill
        \begin{minipage}[t]{0.23\textwidth}
        \centering
        \includegraphics[width=\linewidth]{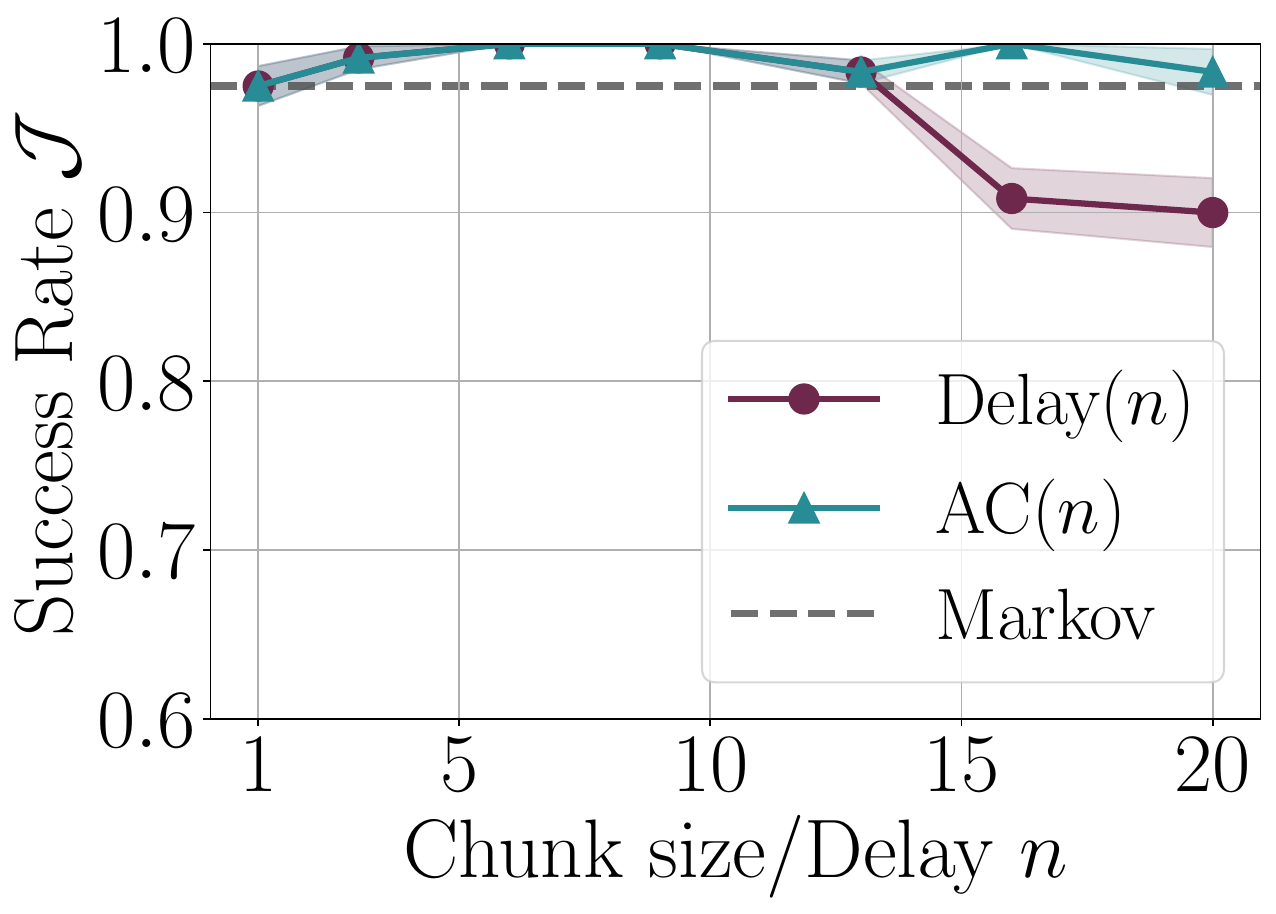}
    \end{minipage}
    \hfill
        \begin{minipage}[t]{0.23\textwidth}
        \centering
        \includegraphics[width=\linewidth]{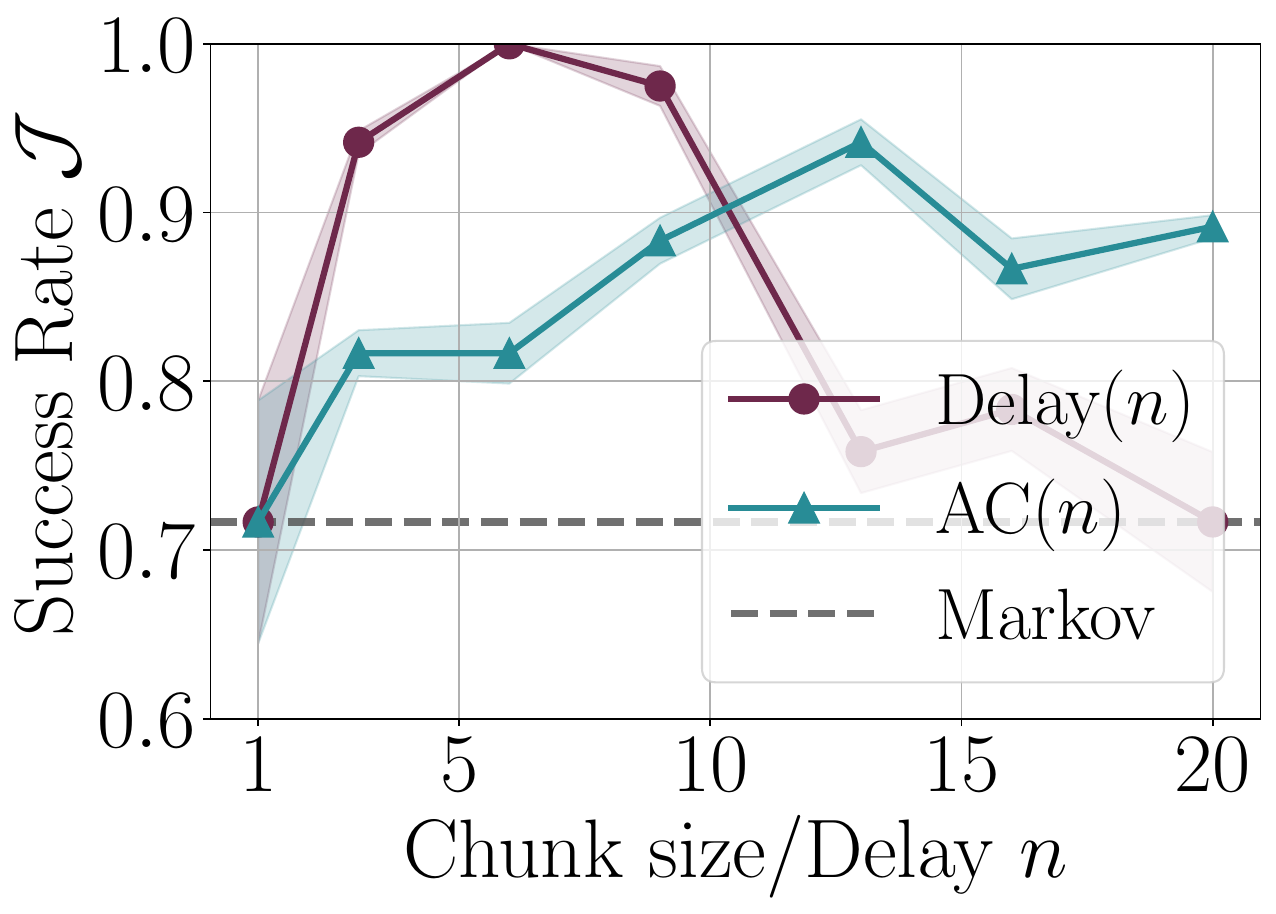}
    \end{minipage}
        \begin{minipage}[t]{0.23\textwidth}
            \includegraphics[width=\linewidth]{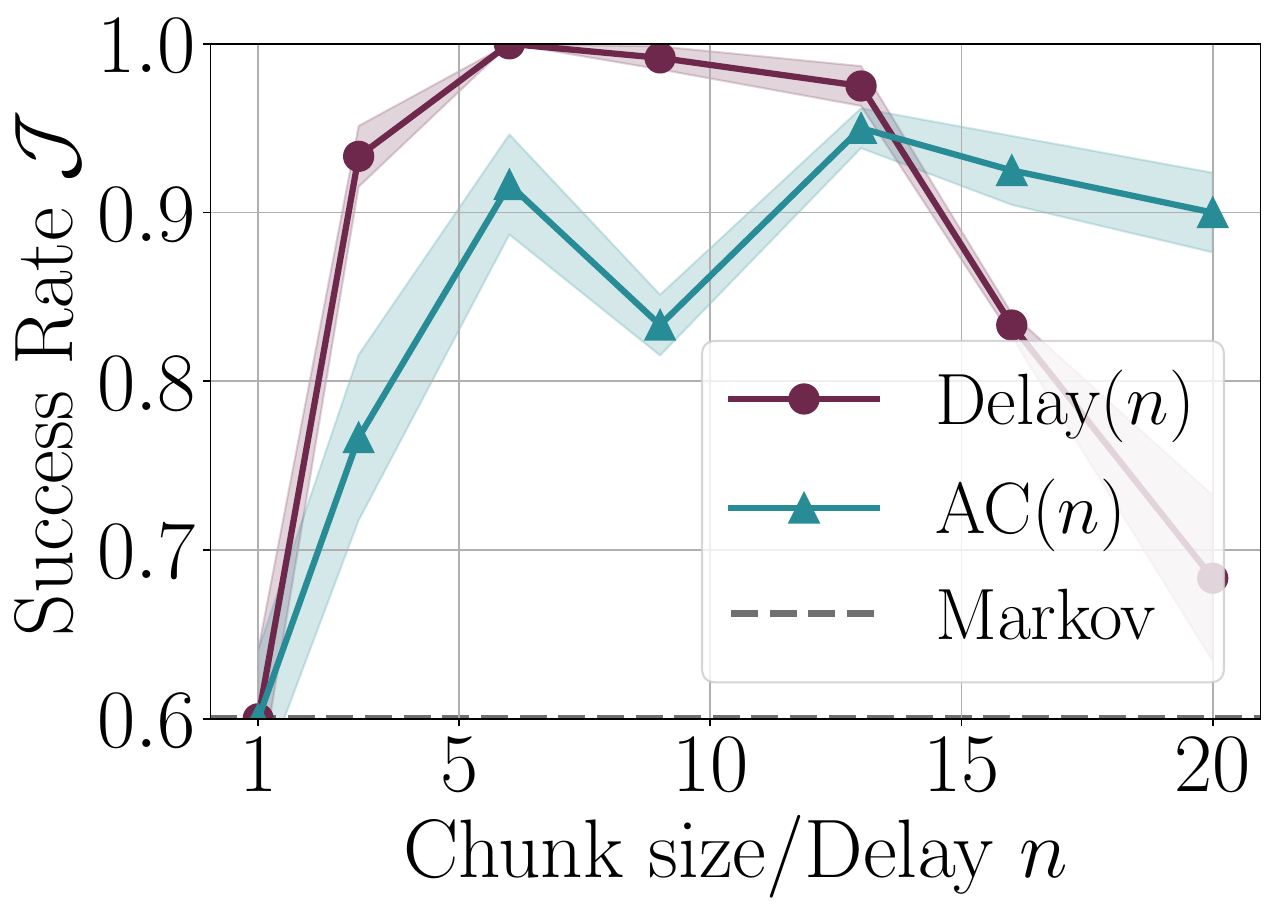}
    \end{minipage}
    \hfill
        \begin{minipage}[t]{0.23\textwidth}
        \centering
        \includegraphics[width=\linewidth]{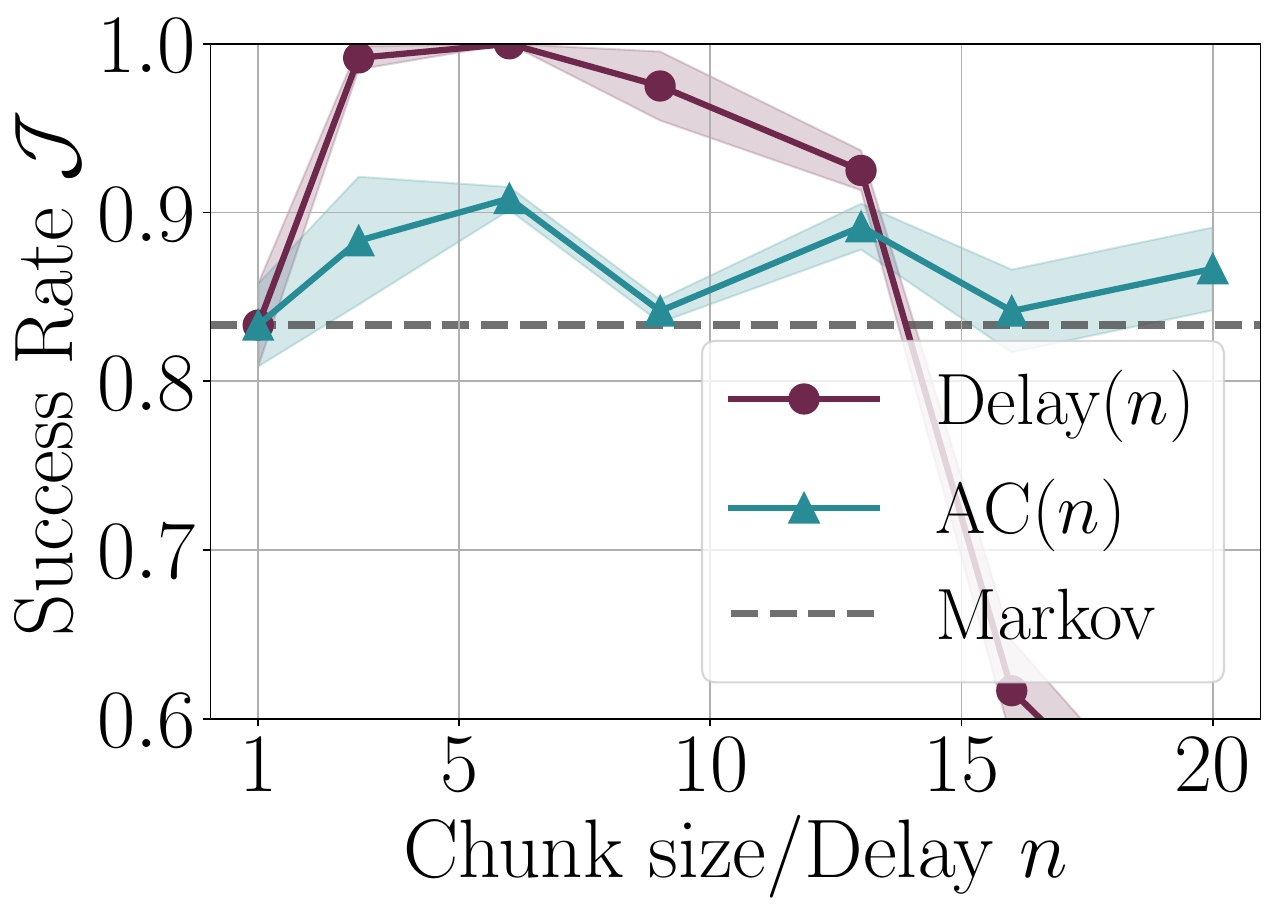}
    \end{minipage}
    \hfill
        \begin{minipage}[t]{0.23\textwidth}
        \centering
        \includegraphics[width=\linewidth]{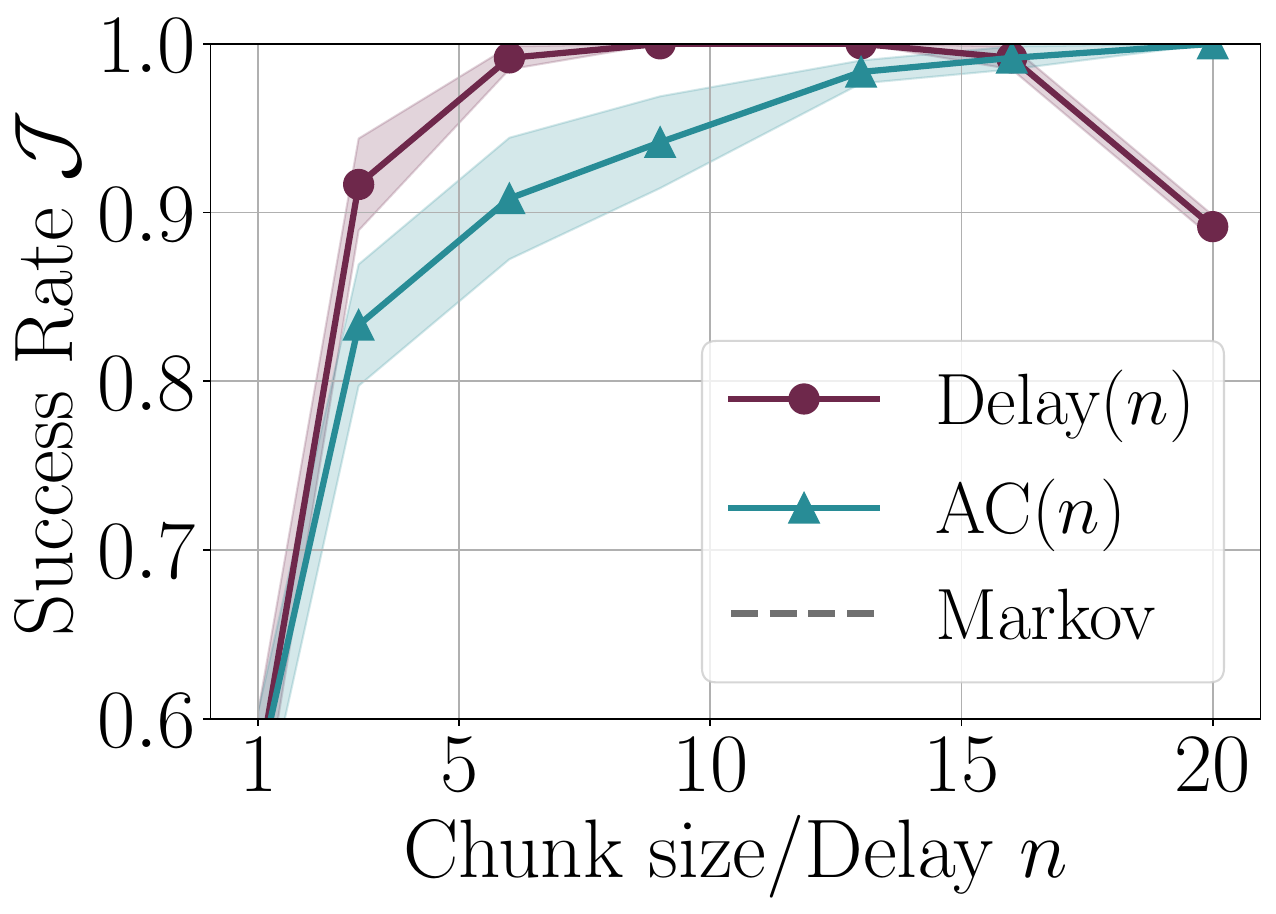}
    \end{minipage}
    \hfill
        \begin{minipage}[t]{0.23\textwidth}
        \centering
        \includegraphics[width=\linewidth]{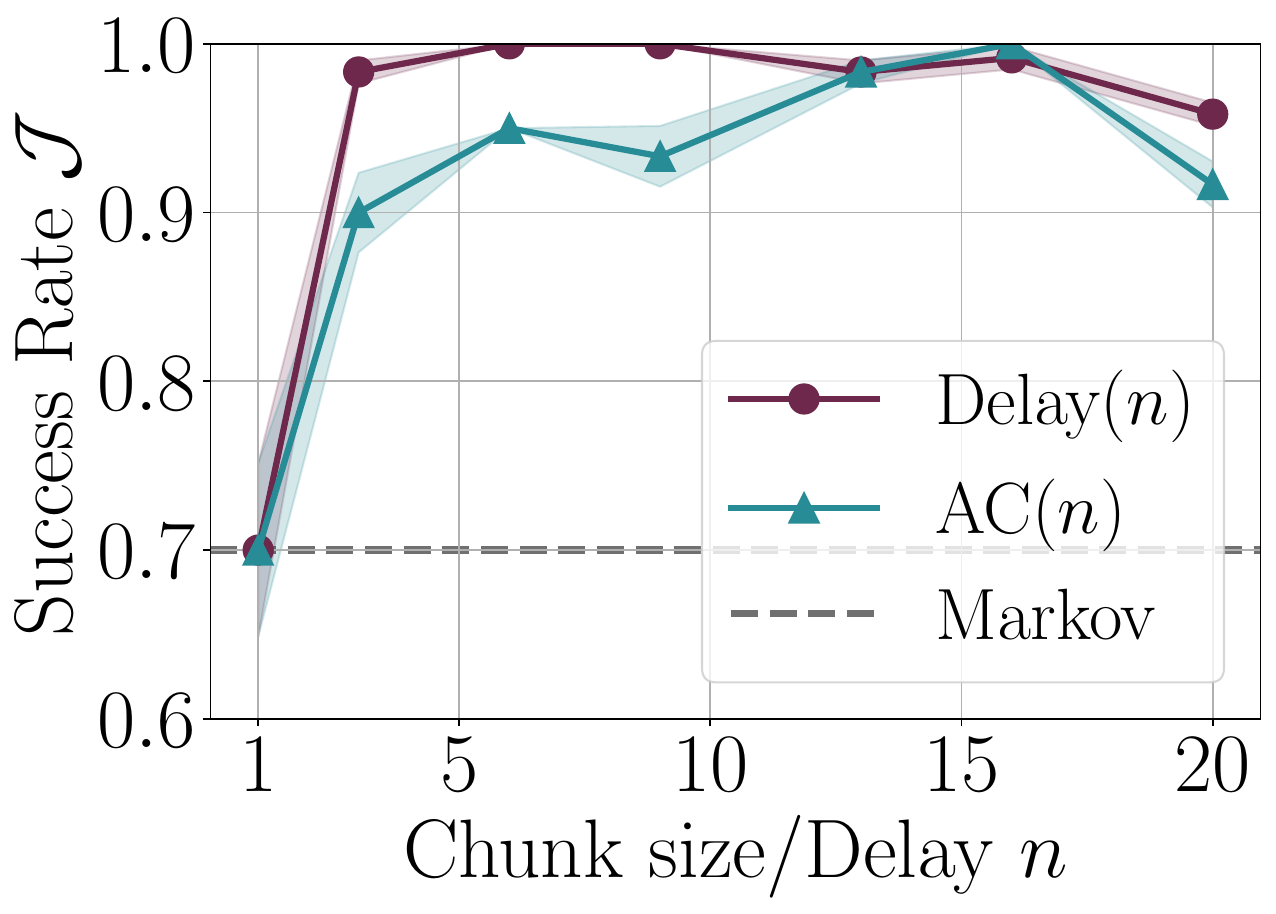}
    \end{minipage}
        \begin{minipage}[t]{0.23\textwidth}
            \includegraphics[width=\linewidth]{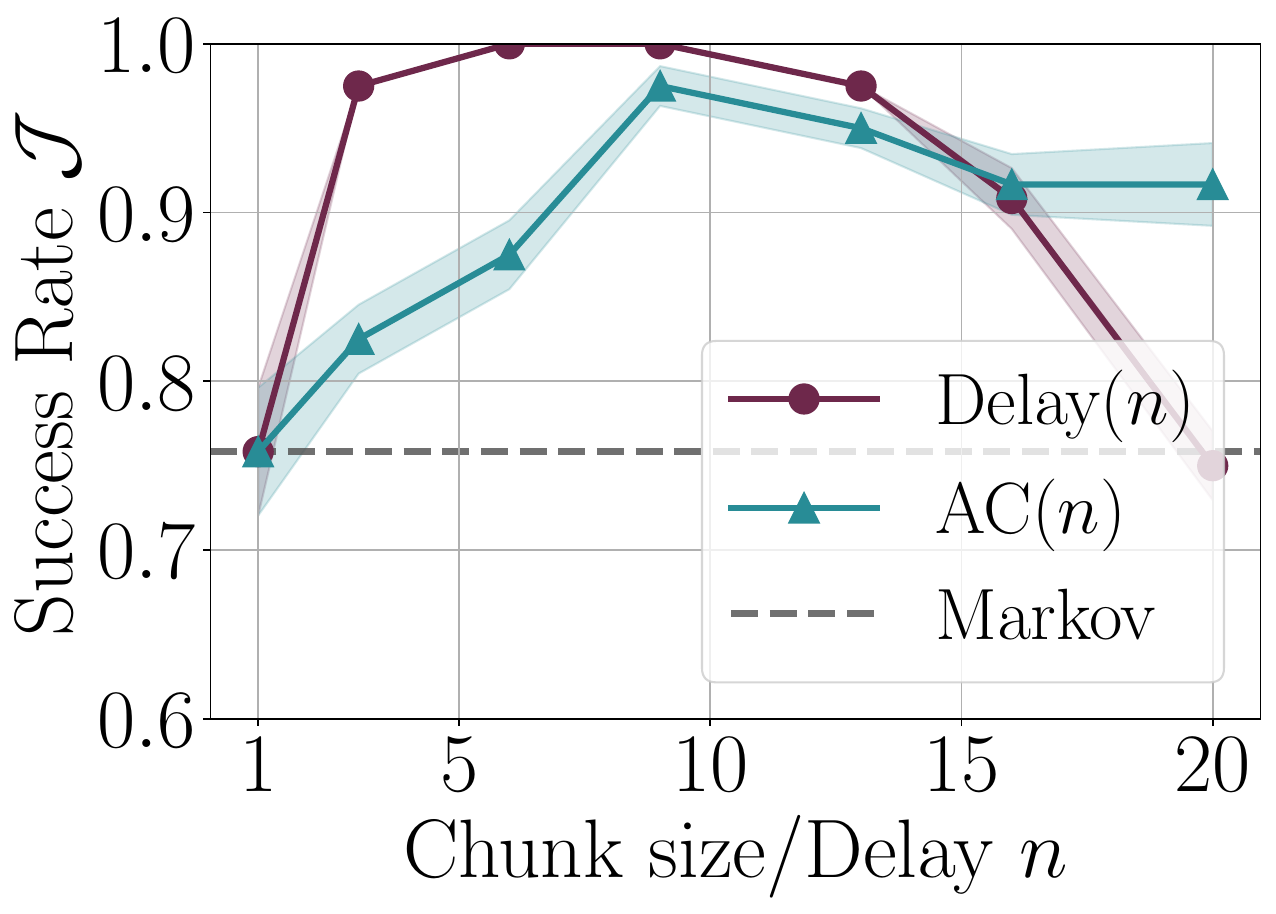}
    \end{minipage}
    \hfill
        \begin{minipage}[t]{0.23\textwidth}
        \centering
        \includegraphics[width=\linewidth]{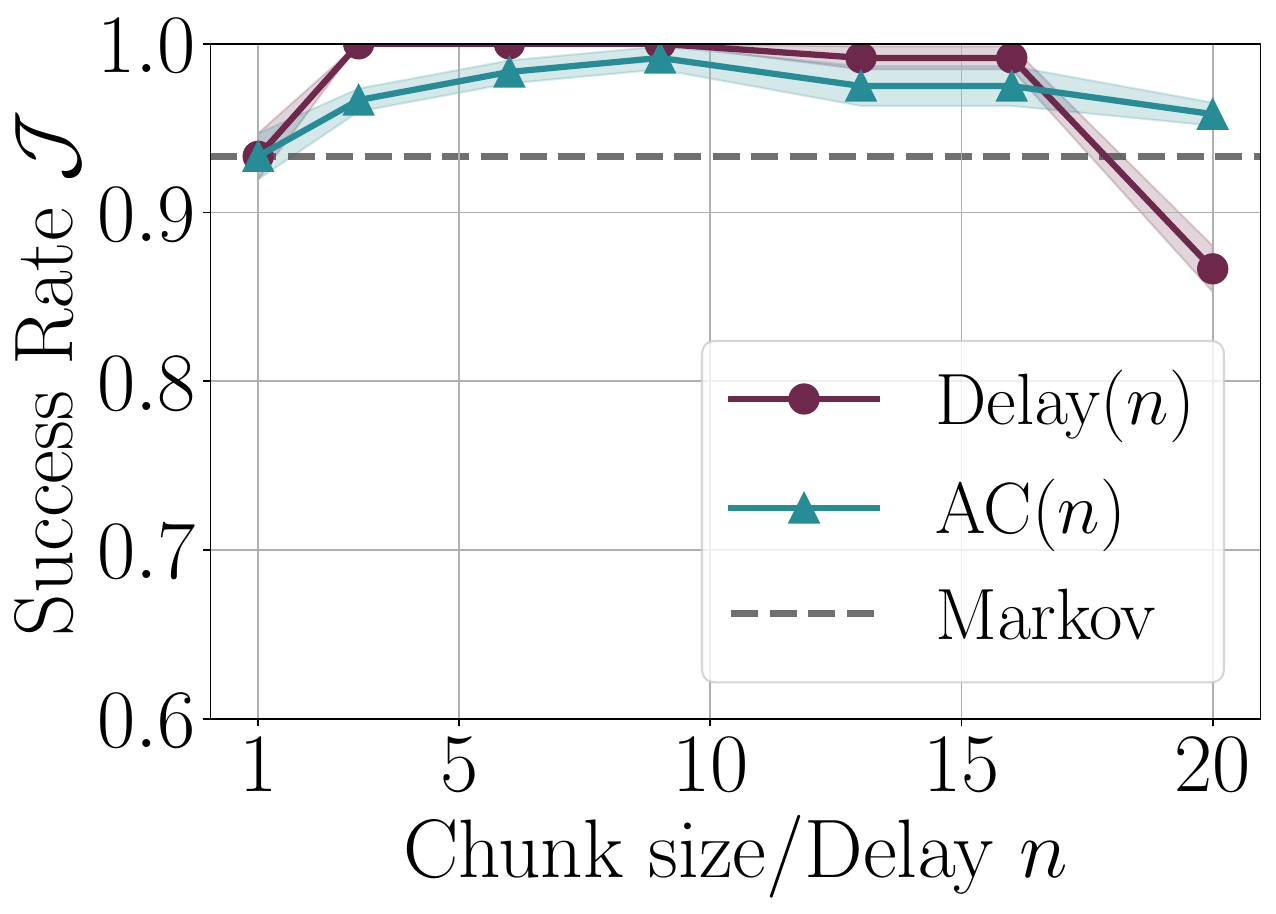}
    \end{minipage}
    \hfill
        \begin{minipage}[t]{0.23\textwidth}
        \centering
        \includegraphics[width=\linewidth]{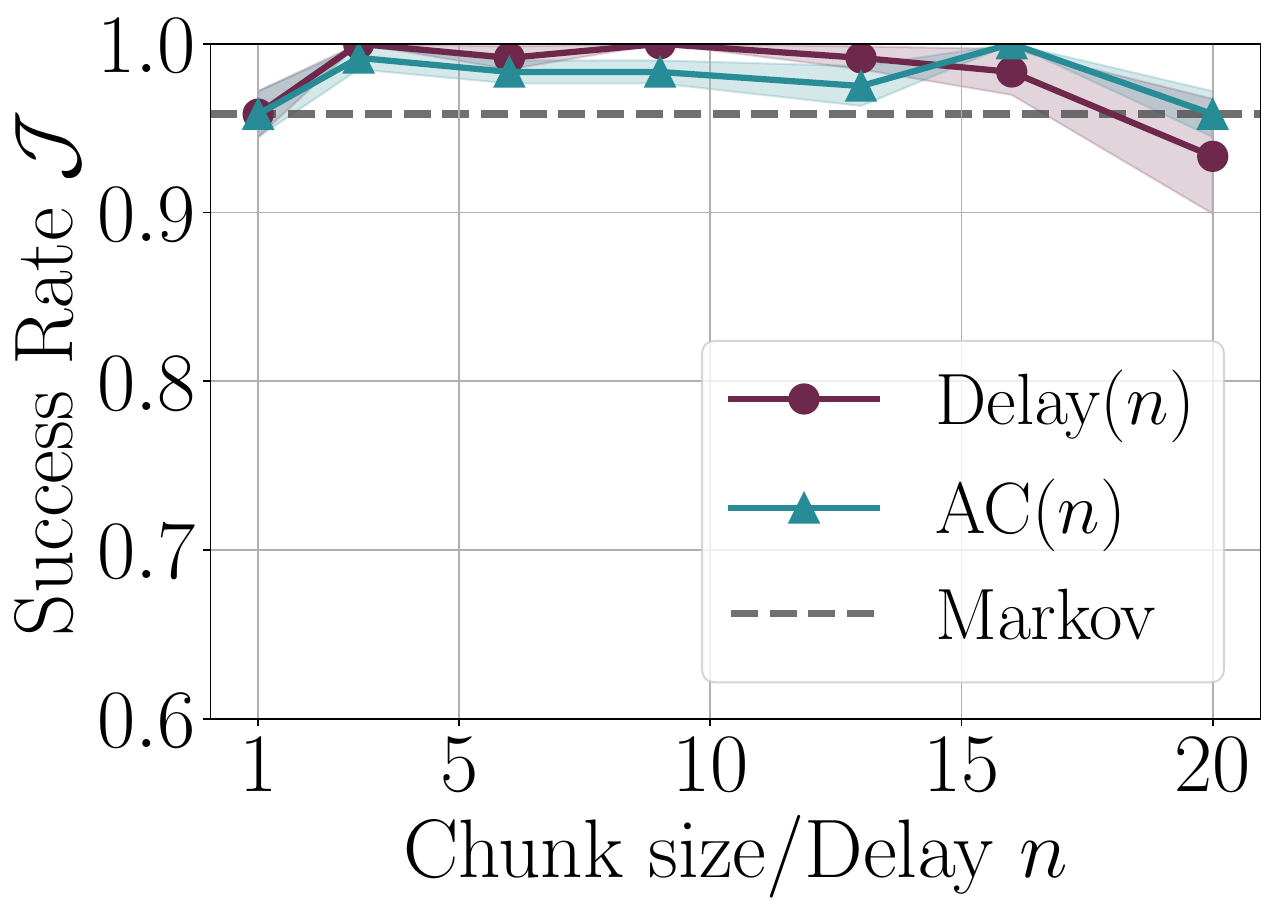}
    \end{minipage}
    \hfill
        \begin{minipage}[t]{0.23\textwidth}
        \centering
        \includegraphics[width=\linewidth]{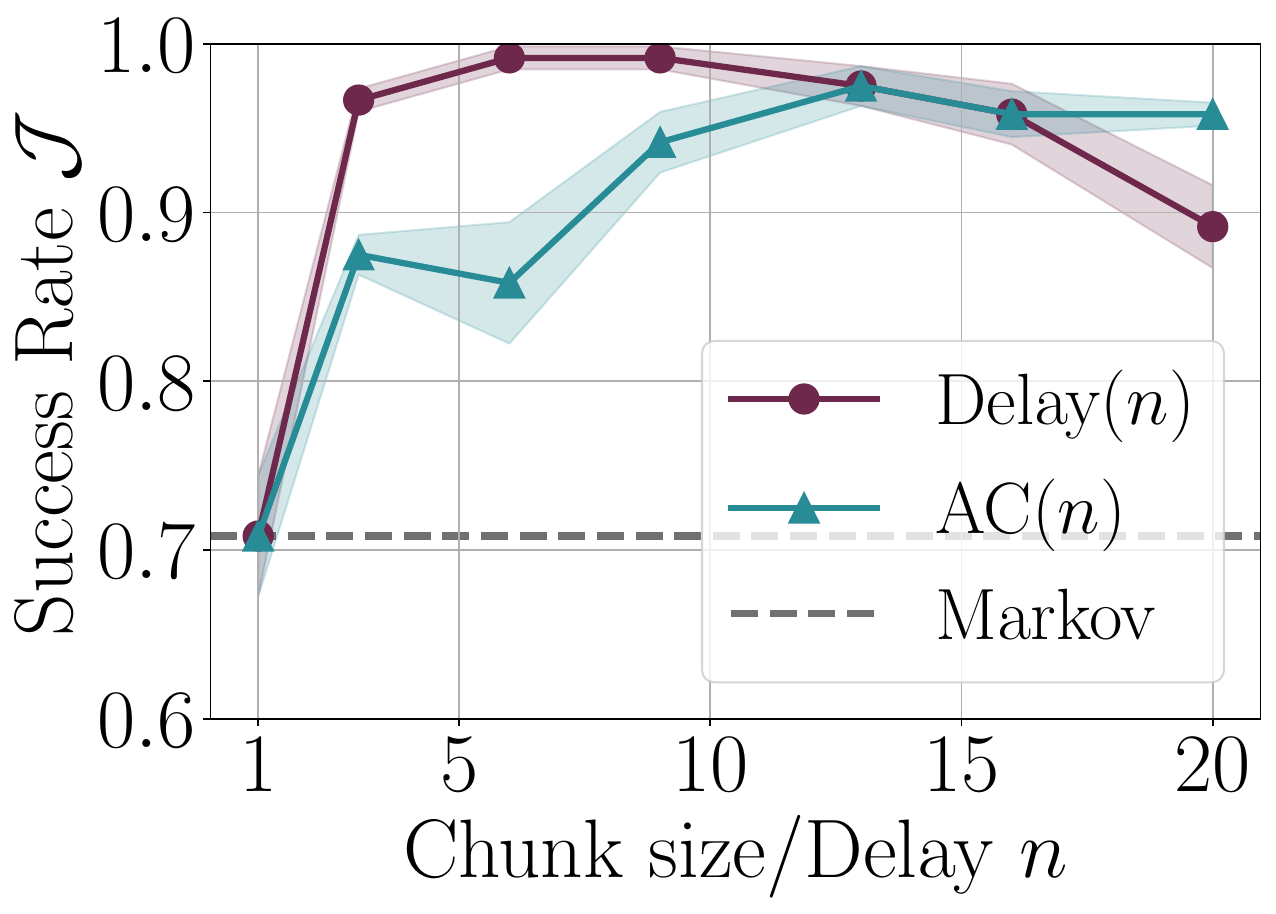}
    \end{minipage}
        \begin{minipage}[t]{0.23\textwidth}
            \includegraphics[width=\linewidth]{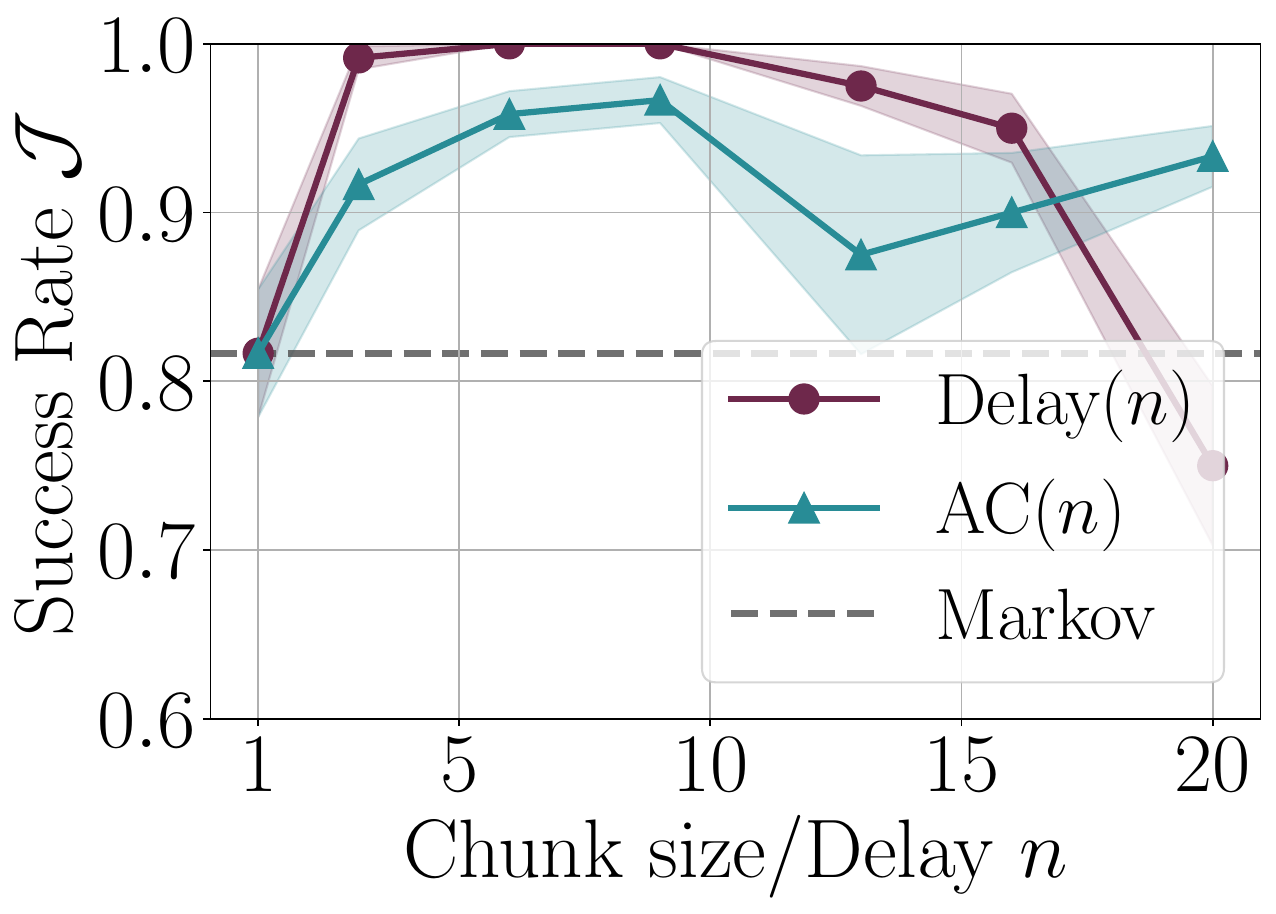}
    \end{minipage}
    \hfill
        \begin{minipage}[t]{0.23\textwidth}
        \centering
        \includegraphics[width=\linewidth]{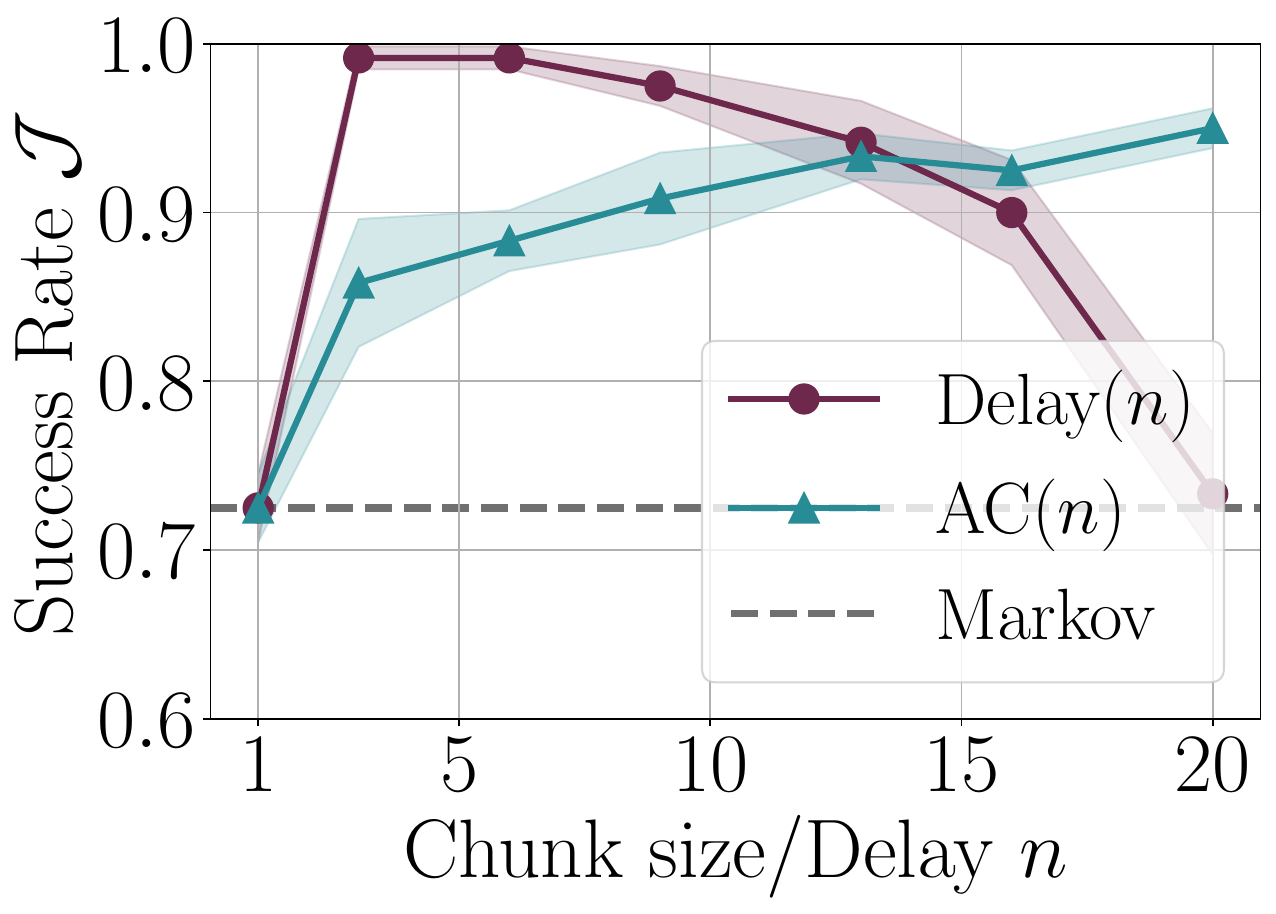}
    \end{minipage}
    \hfill
        \begin{minipage}[t]{0.23\textwidth}
        \centering
        \includegraphics[width=\linewidth]{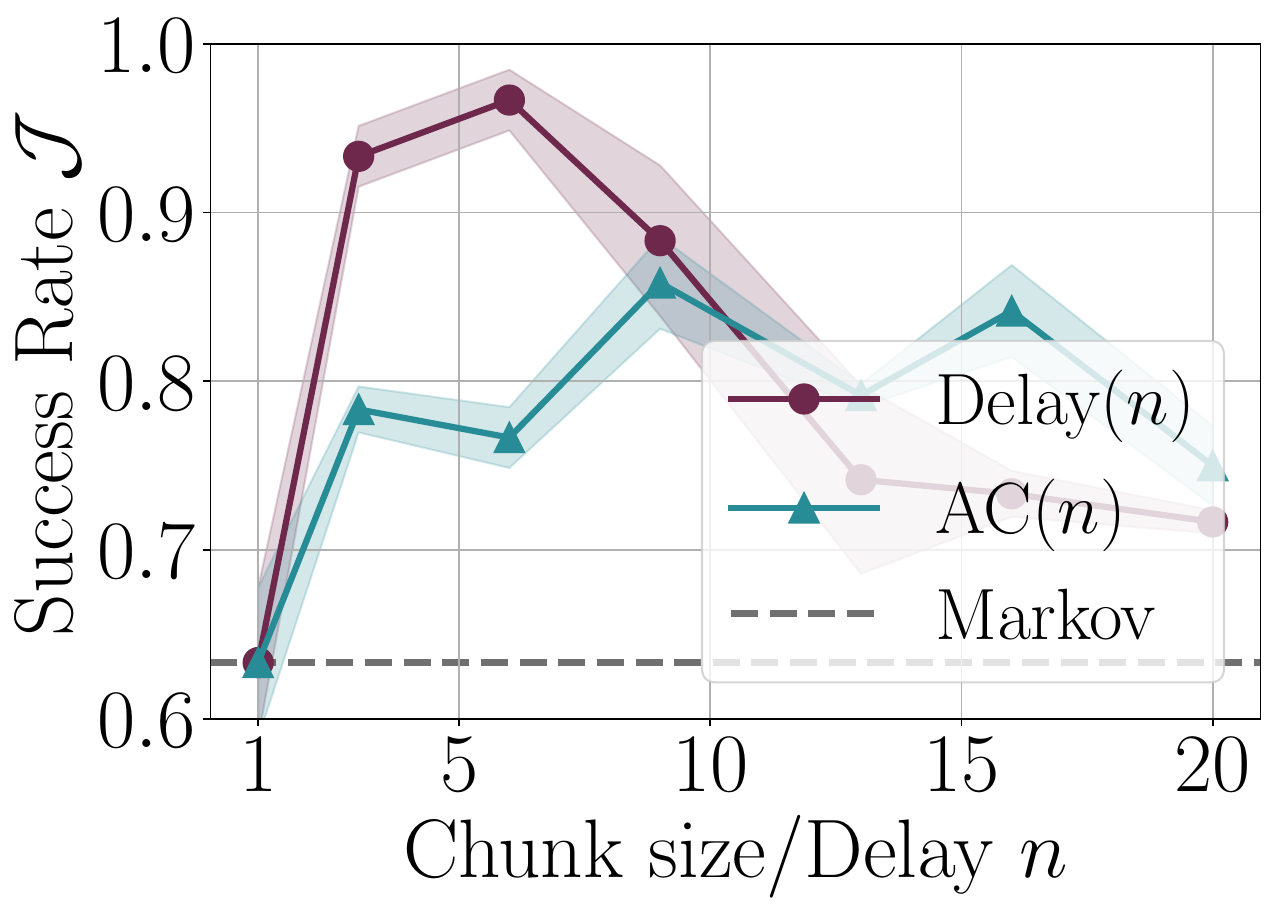}
    \end{minipage}
    \hfill
        \begin{minipage}[t]{0.23\textwidth}
        \centering
        \includegraphics[width=\linewidth]{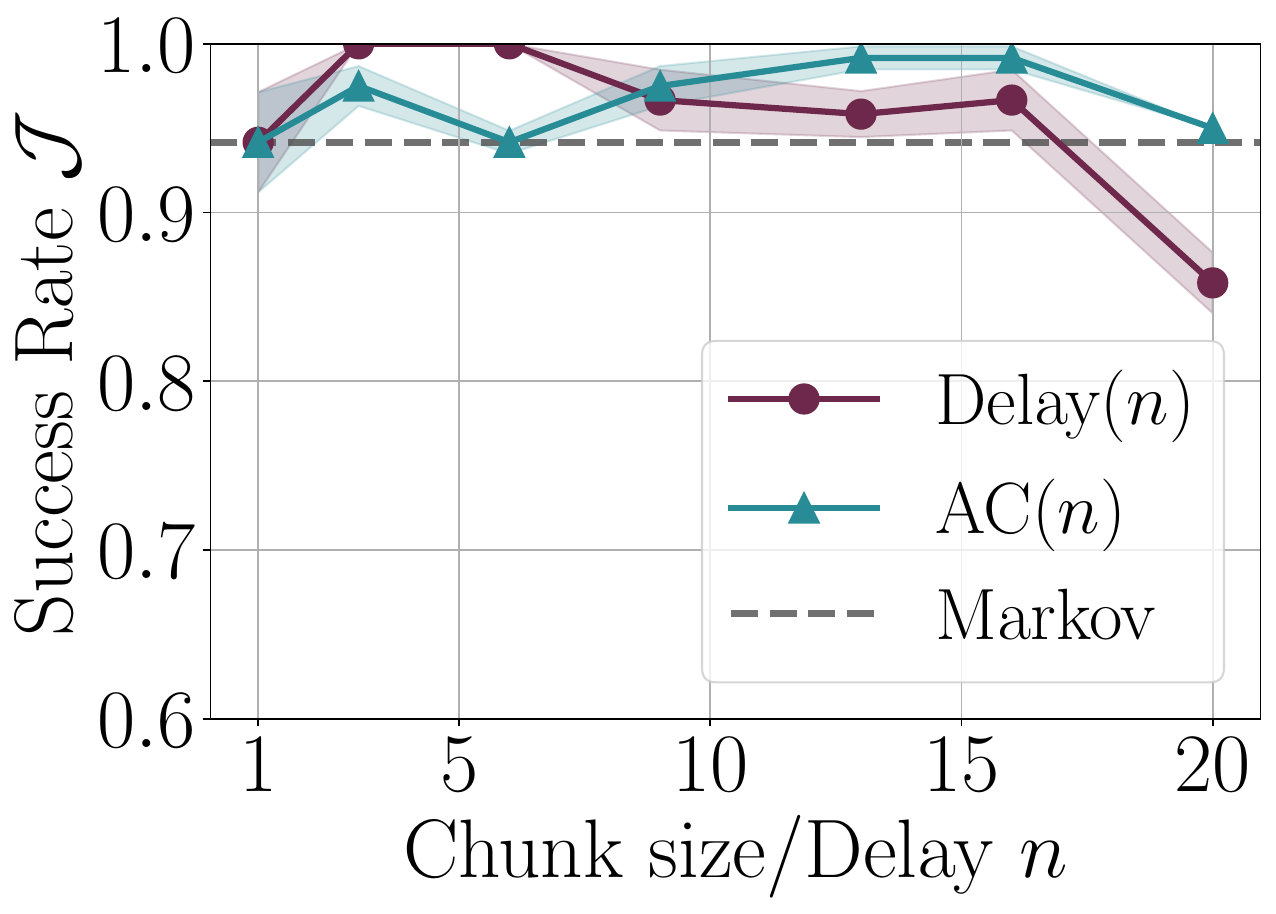}
    \end{minipage}
        \begin{minipage}[t]{0.23\textwidth}
            \includegraphics[width=\linewidth]{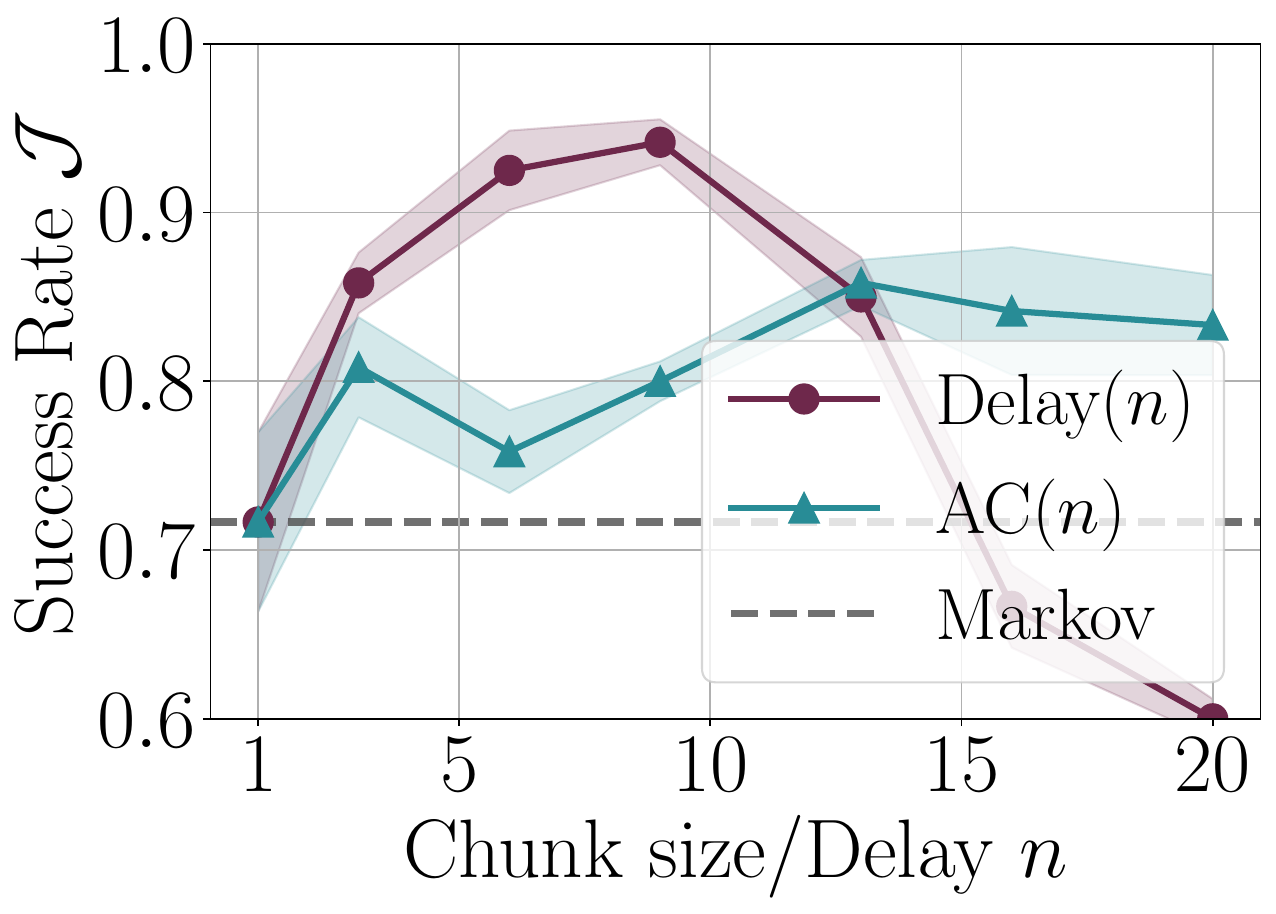}
    \end{minipage}
    \hfill
        \begin{minipage}[t]{0.23\textwidth}
        \centering
        \includegraphics[width=\linewidth]{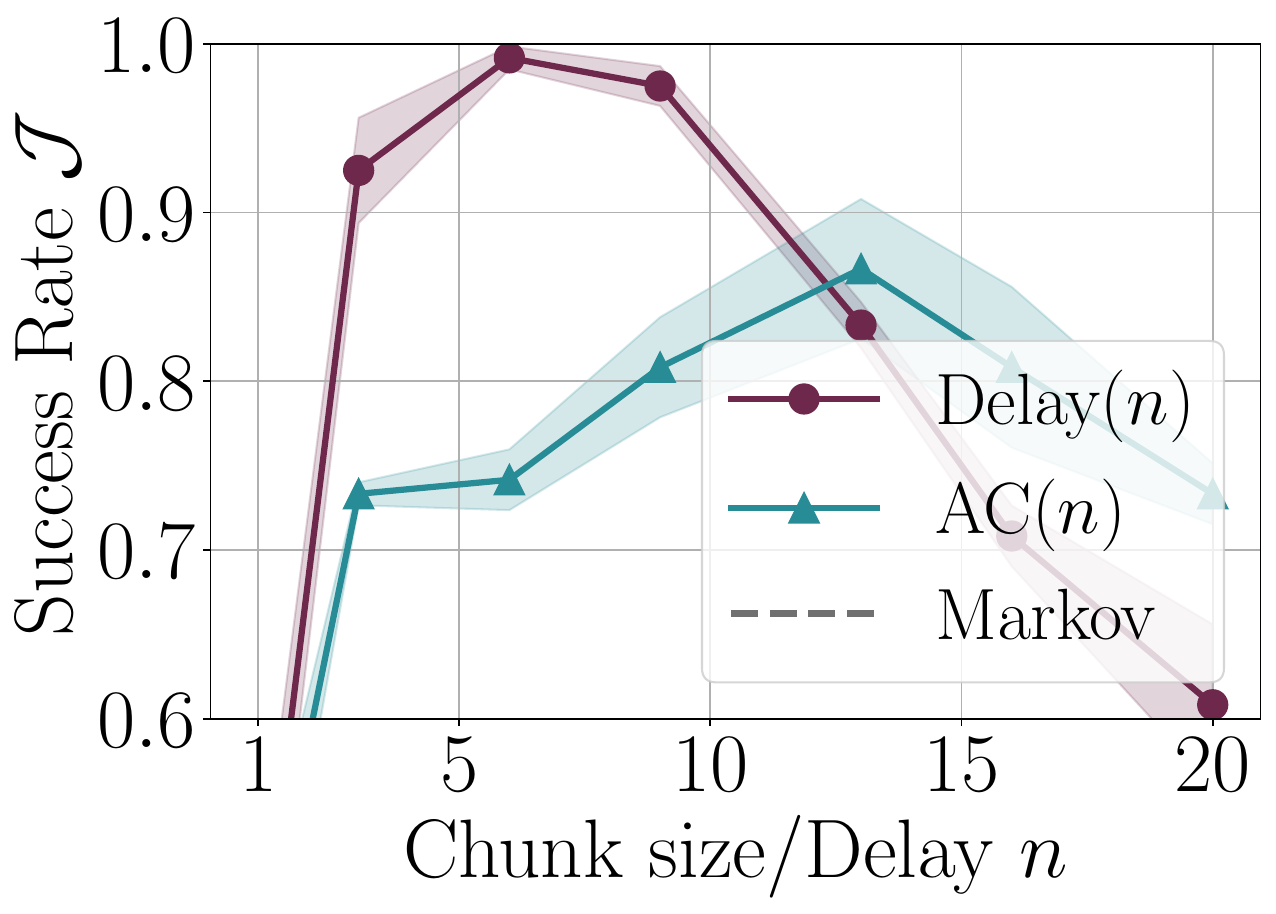}
    \end{minipage}
    \hfill
        \begin{minipage}[t]{0.23\textwidth}
        \centering
        \includegraphics[width=\linewidth]{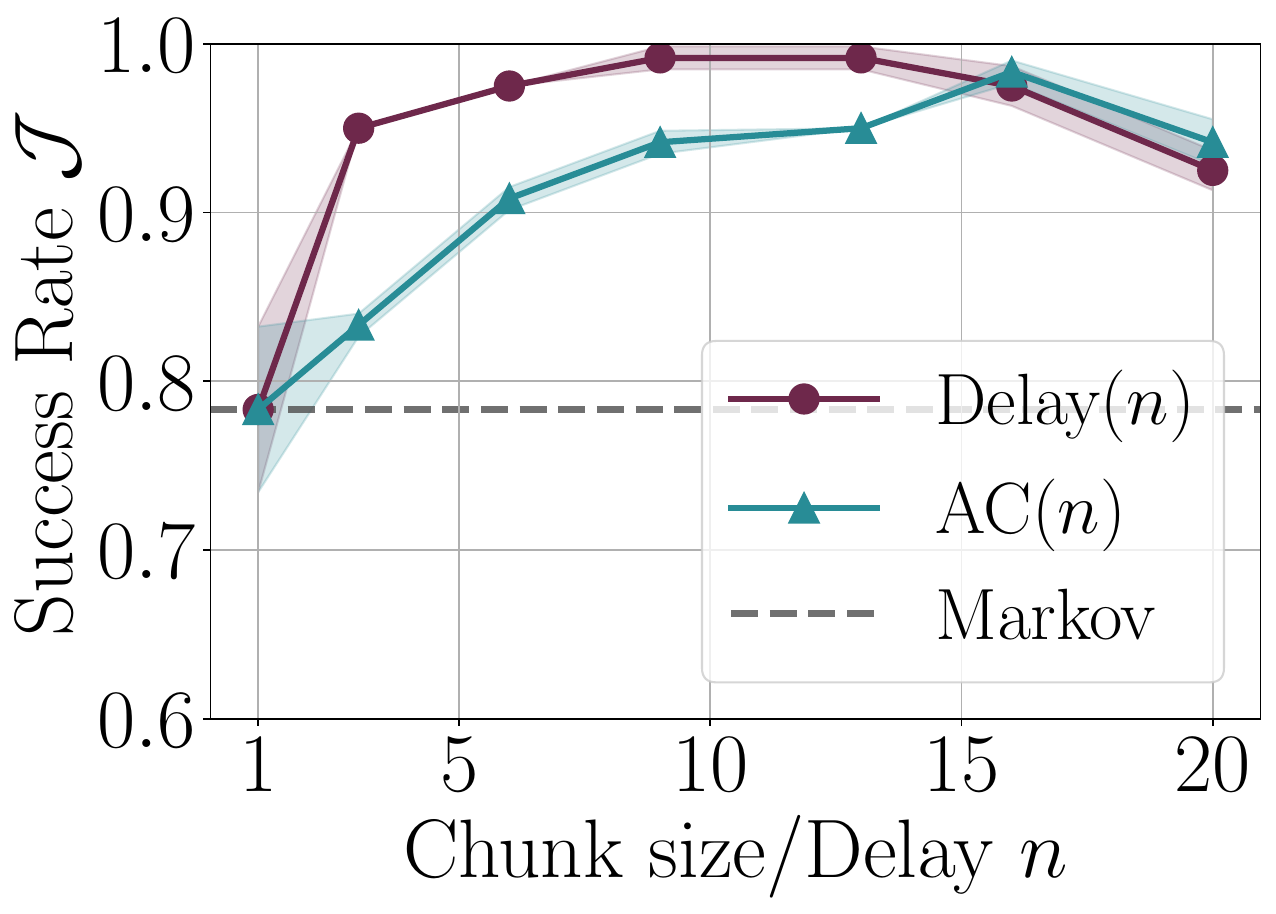}
    \end{minipage}
    \hfill
        \begin{minipage}[t]{0.23\textwidth}
        \centering
        \includegraphics[width=\linewidth]{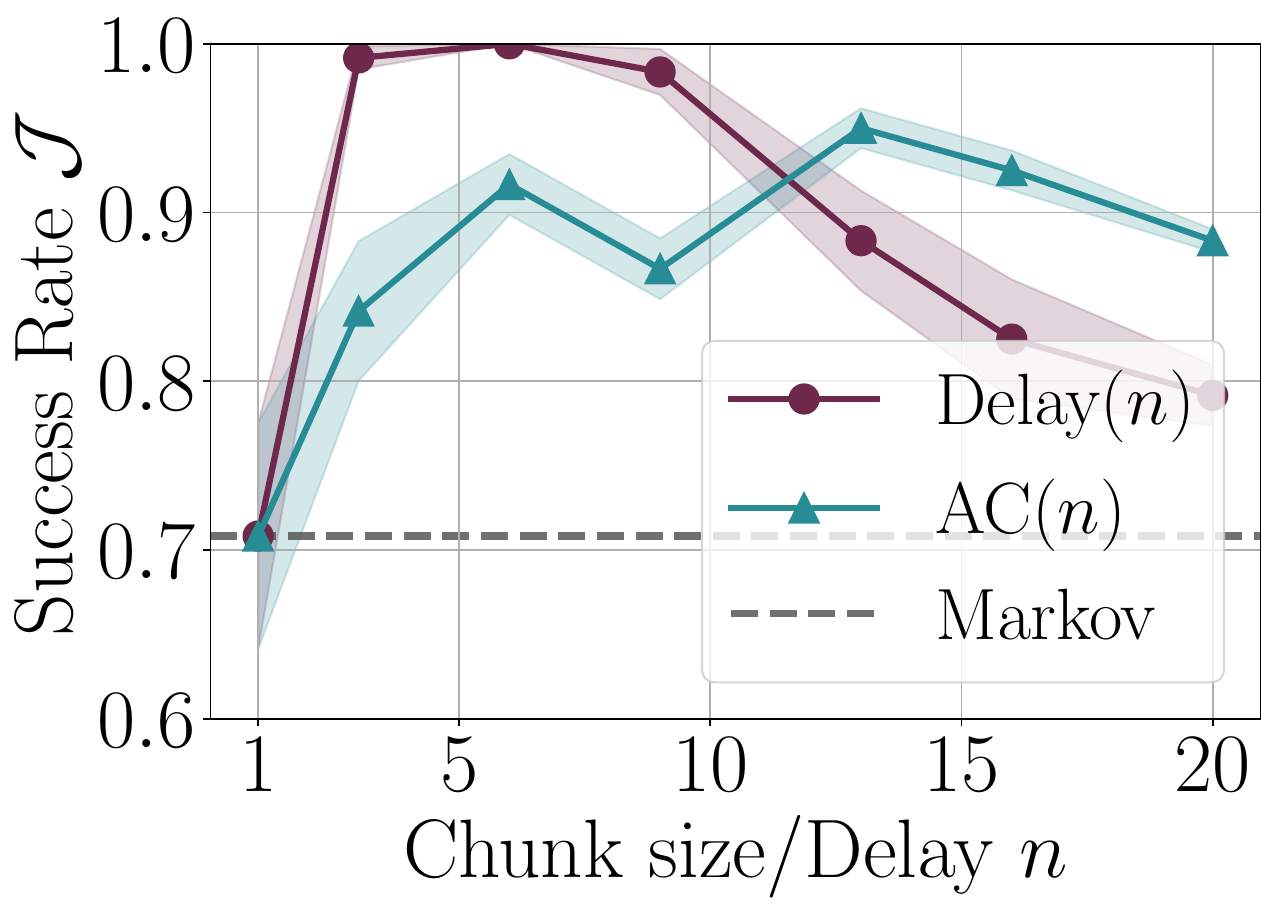}
    \end{minipage}
        \begin{minipage}[t]{0.23\textwidth}
            \includegraphics[width=\linewidth]{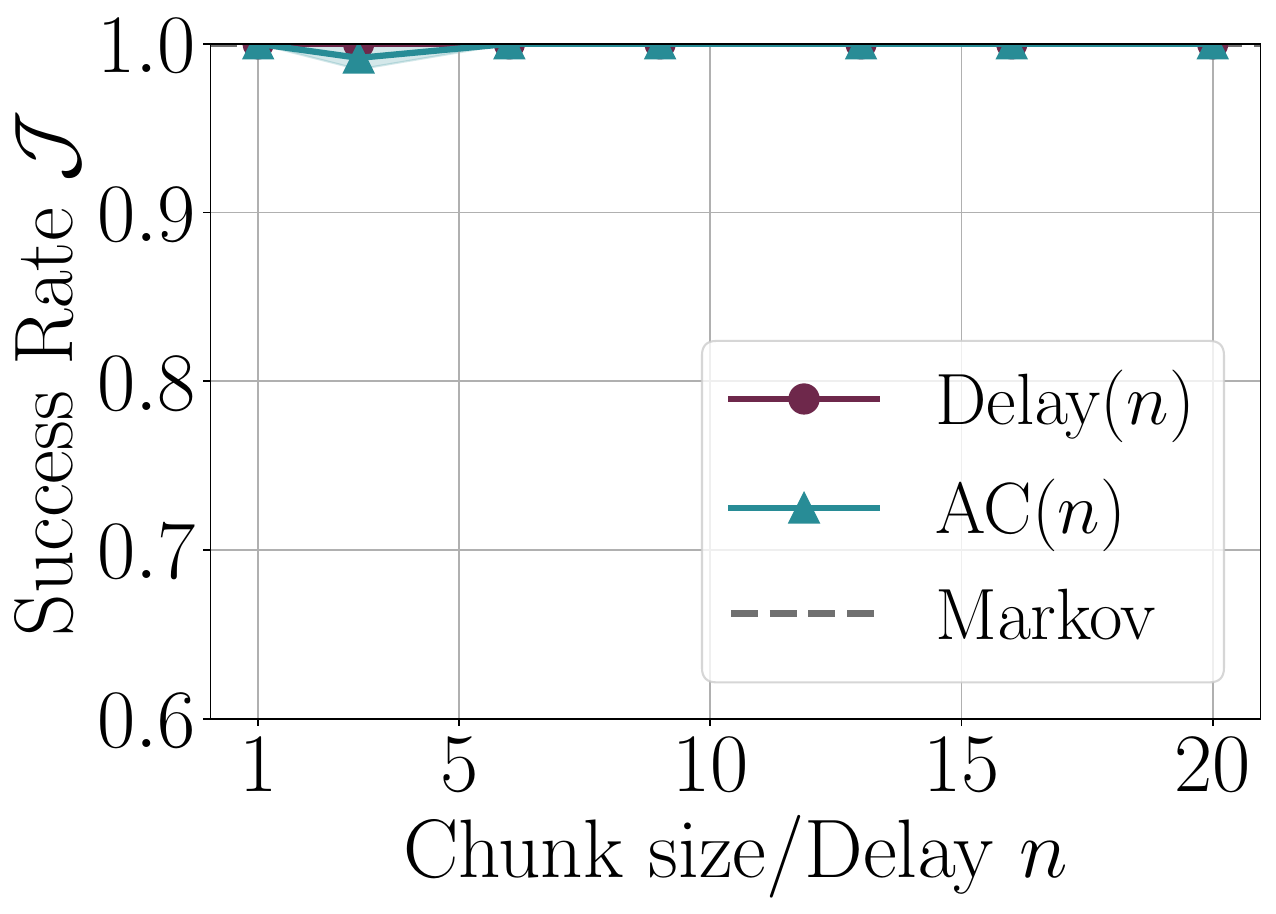}
    \end{minipage}
    \hfill
        \begin{minipage}[t]{0.23\textwidth}
        \centering
        \includegraphics[width=\linewidth]{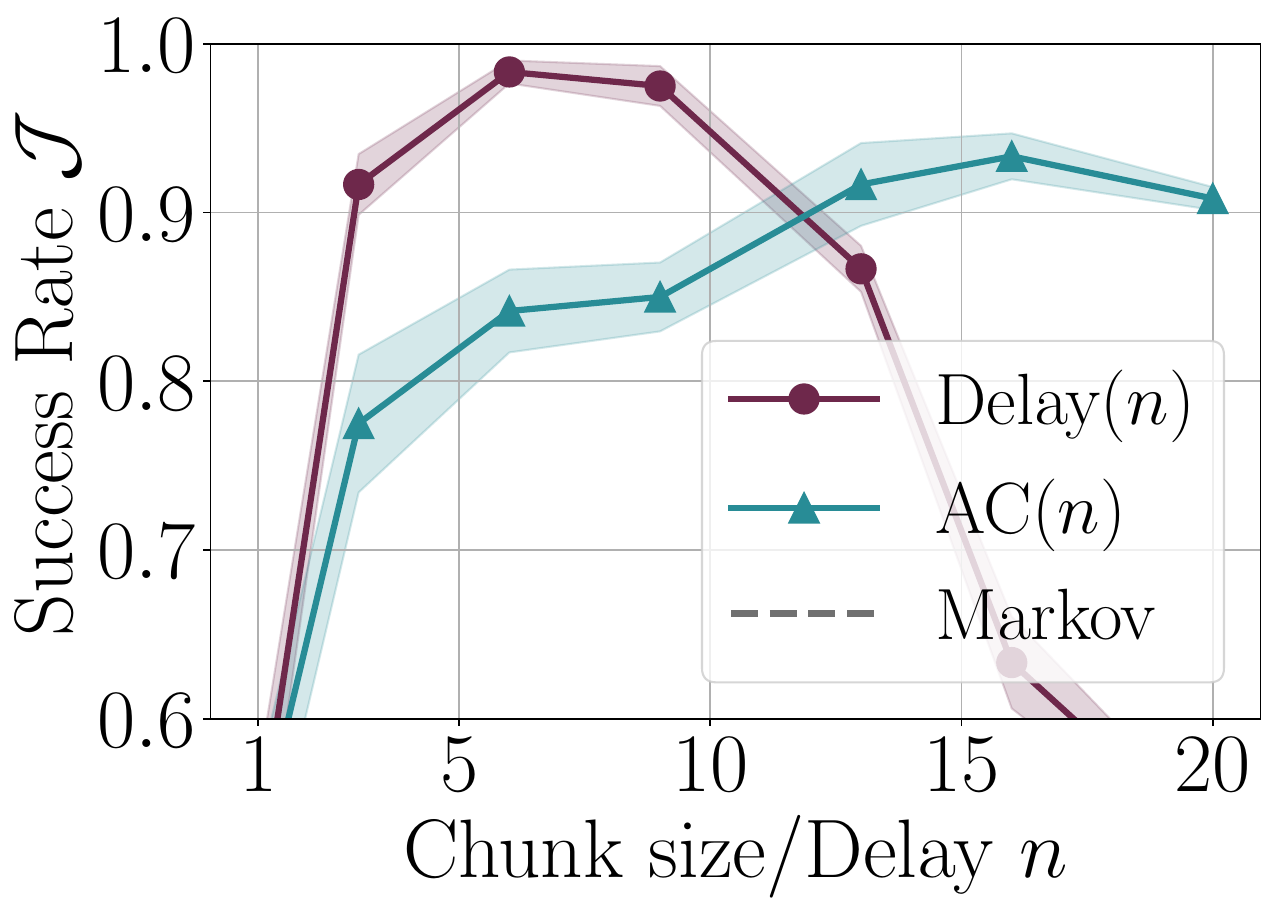}
    \end{minipage}
    \hfill
        \begin{minipage}[t]{0.23\textwidth}
        \centering
        \includegraphics[width=\linewidth]{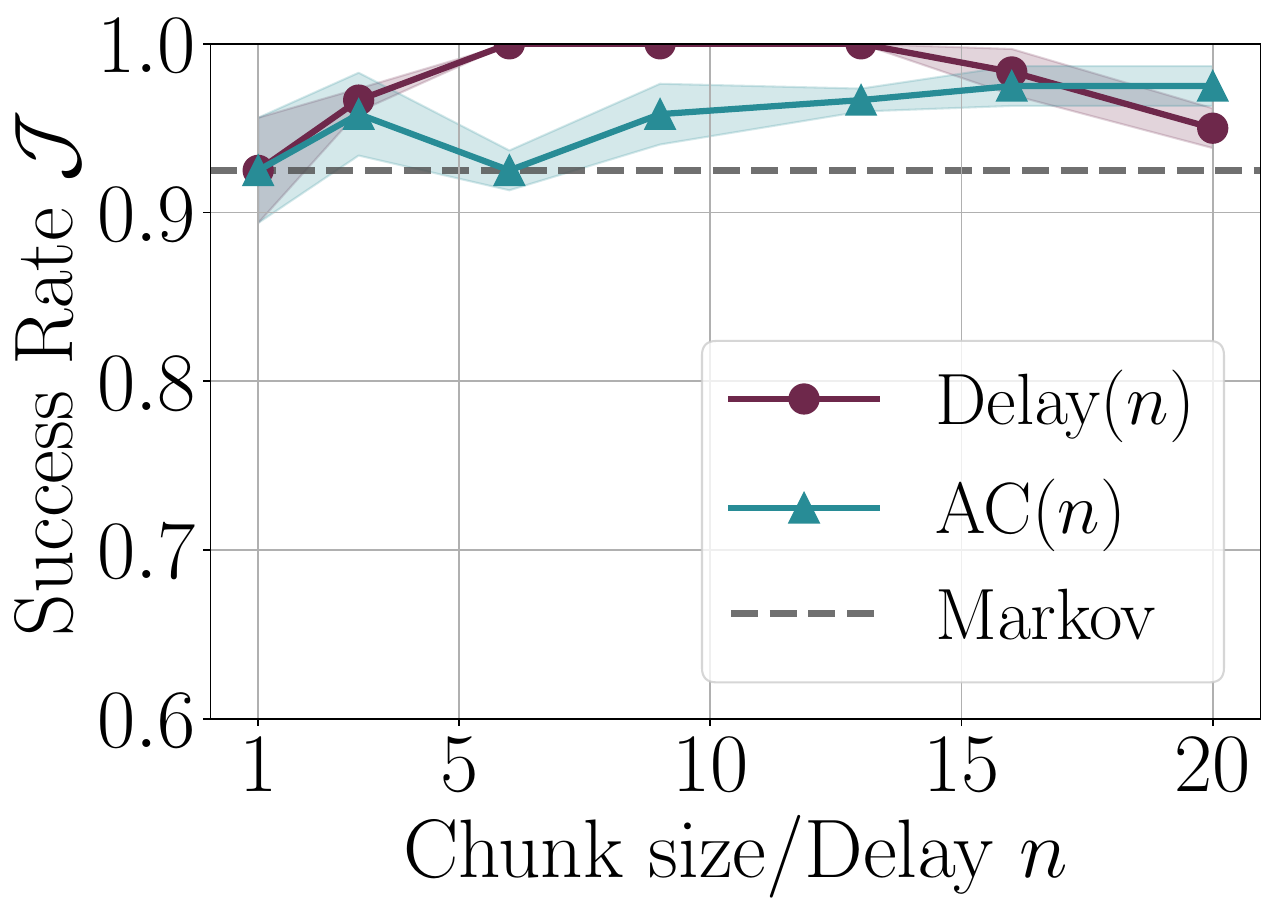}
    \end{minipage}
    \hfill
        \begin{minipage}[t]{0.23\textwidth}
        \centering
        \includegraphics[width=\linewidth]{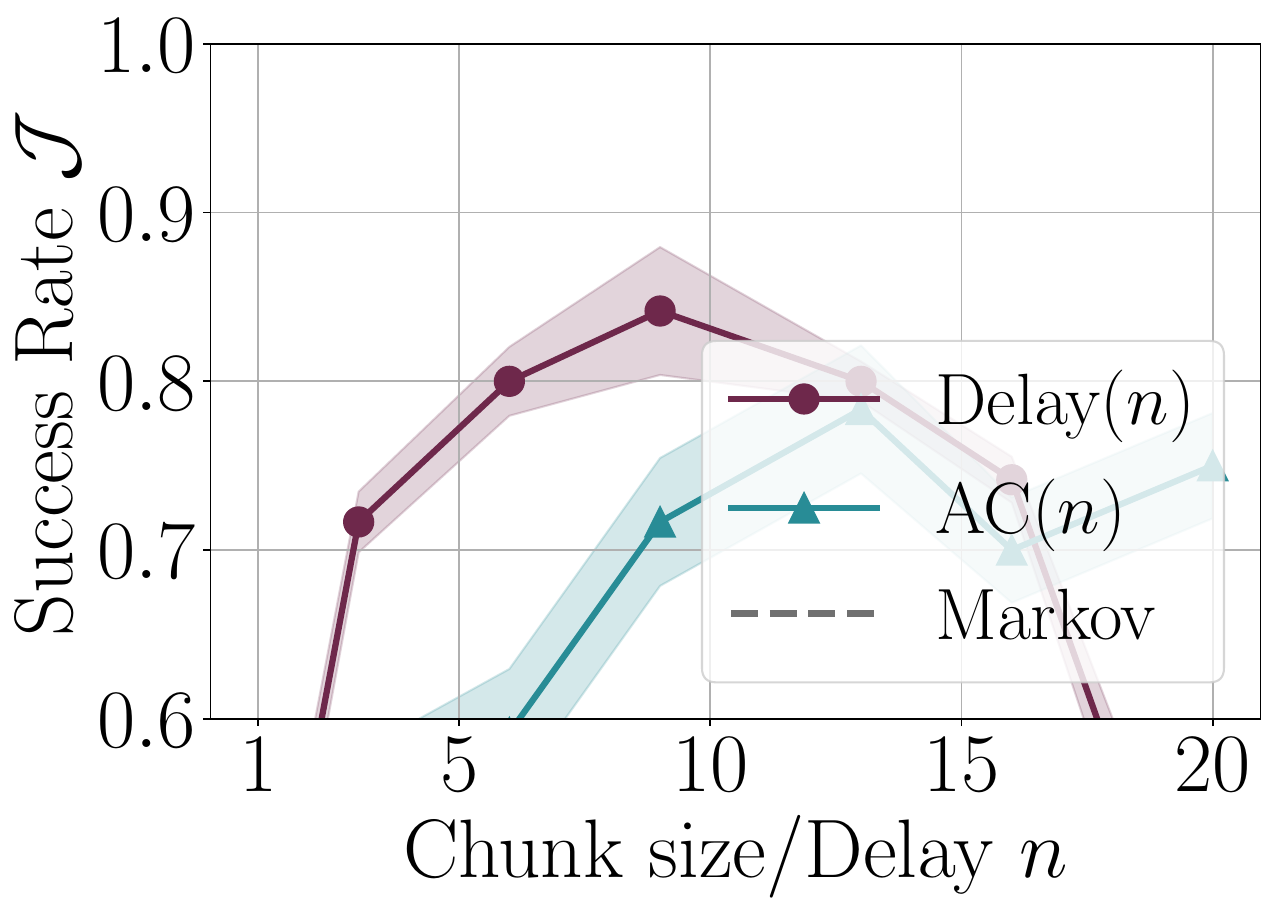}
    \end{minipage}
        \begin{minipage}[t]{0.23\textwidth}
            \includegraphics[width=\linewidth]{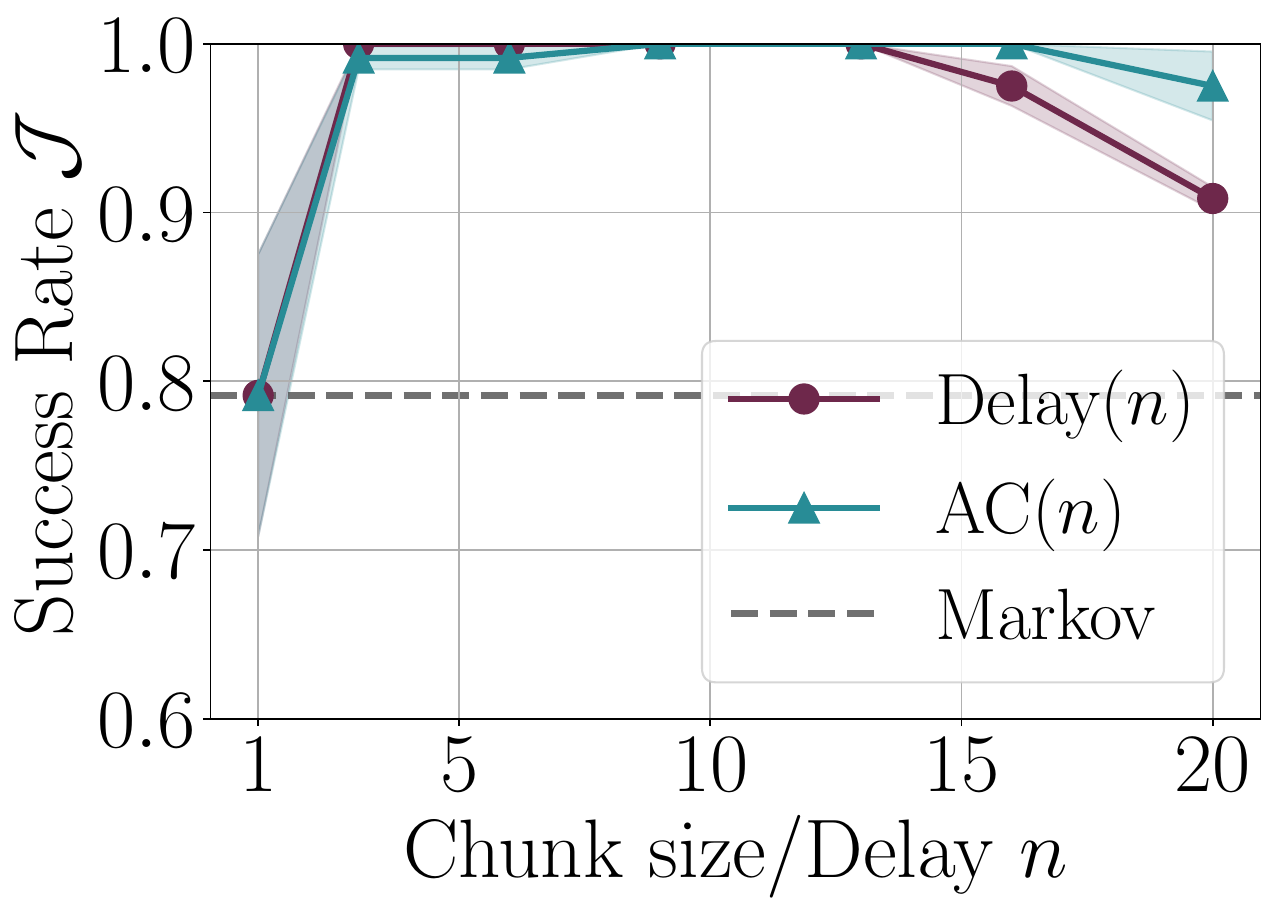}
    \end{minipage}
    \hfill
        \begin{minipage}[t]{0.23\textwidth}
        \centering
        \includegraphics[width=\linewidth]{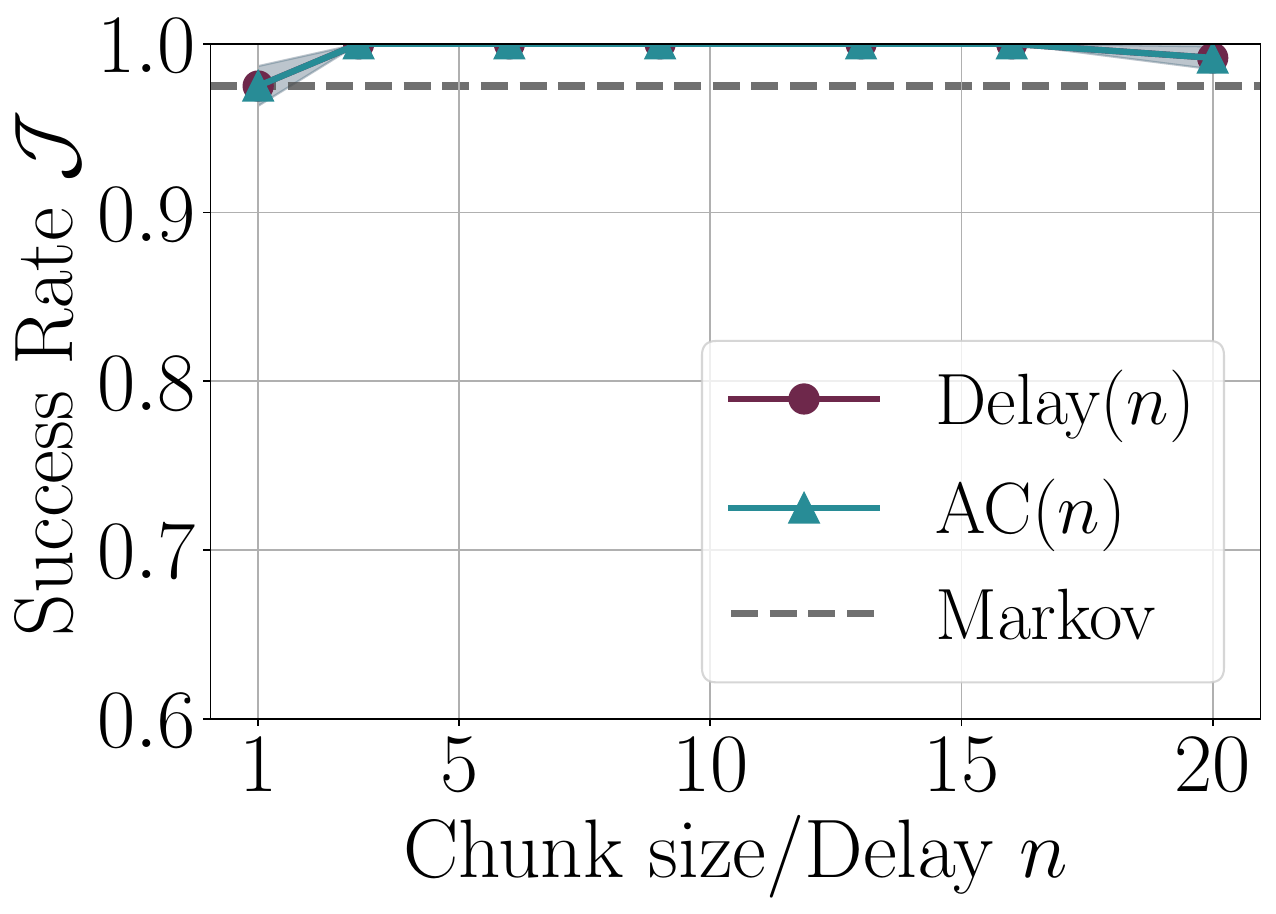}
    \end{minipage}
    \hfill
        \begin{minipage}[t]{0.23\textwidth}
        \centering
        \includegraphics[width=\linewidth]{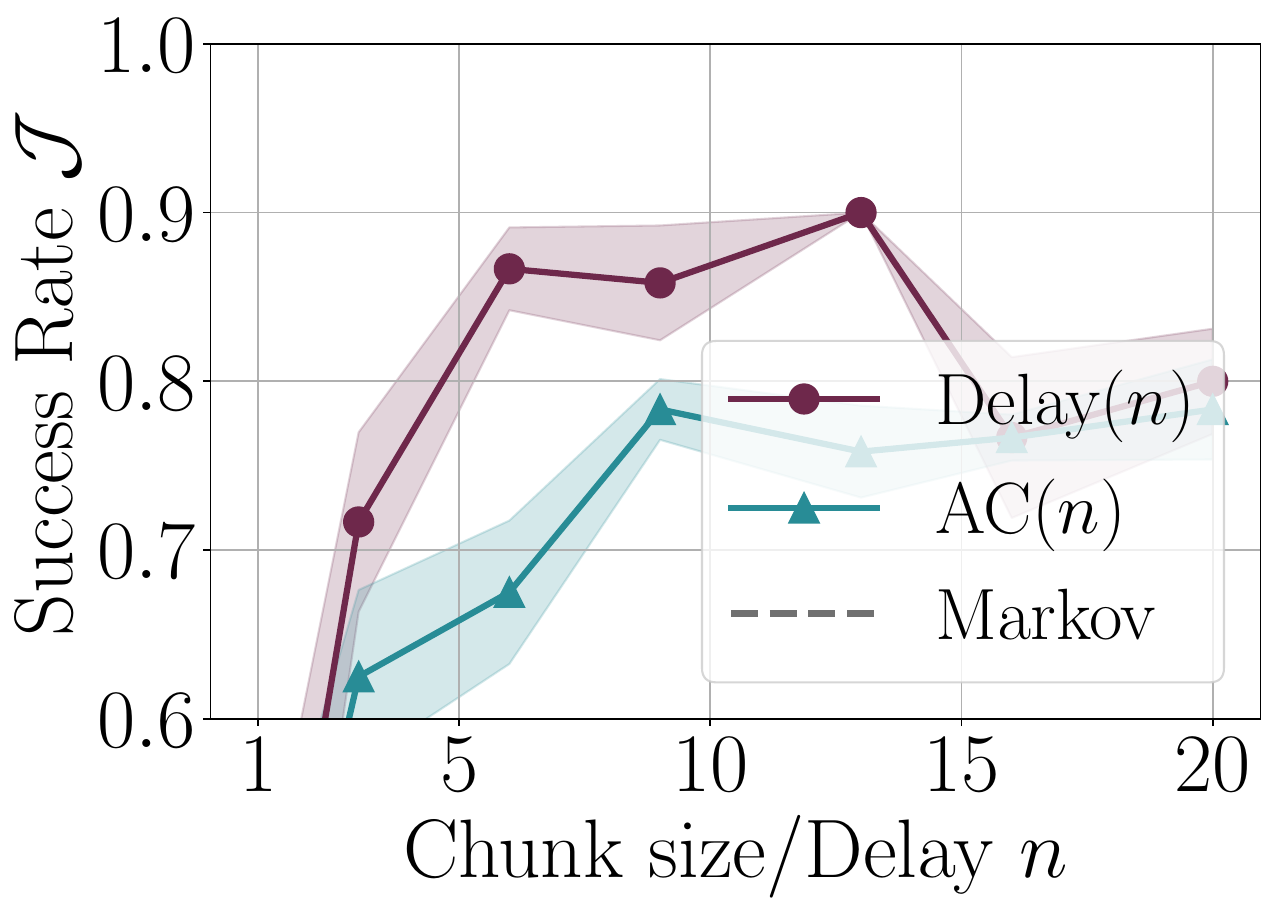}
    \end{minipage}
    \hfill
        \begin{minipage}[t]{0.23\textwidth}
        \centering
        \includegraphics[width=\linewidth]{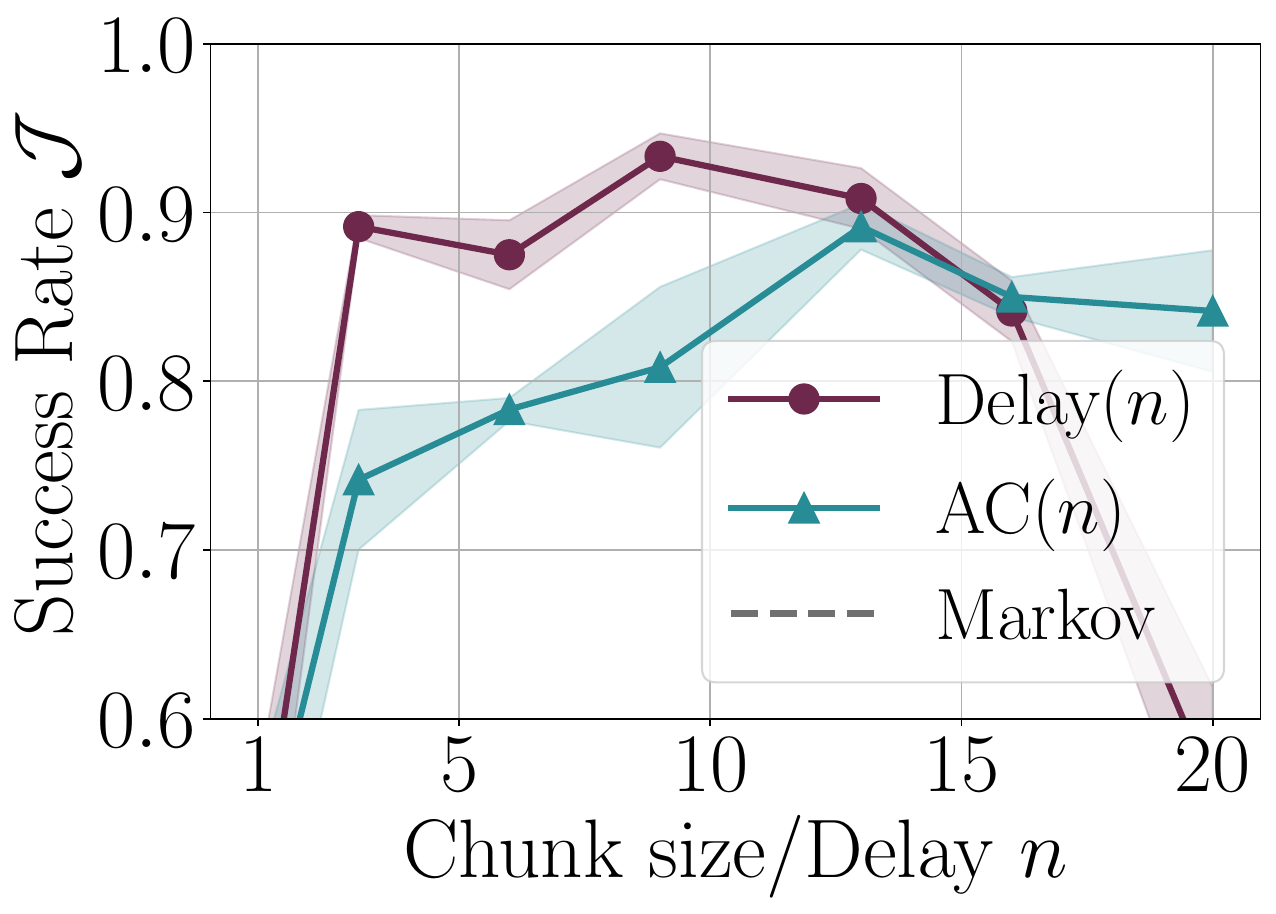}
    \end{minipage}
        \begin{minipage}[t]{0.23\textwidth}
            \includegraphics[width=\linewidth]{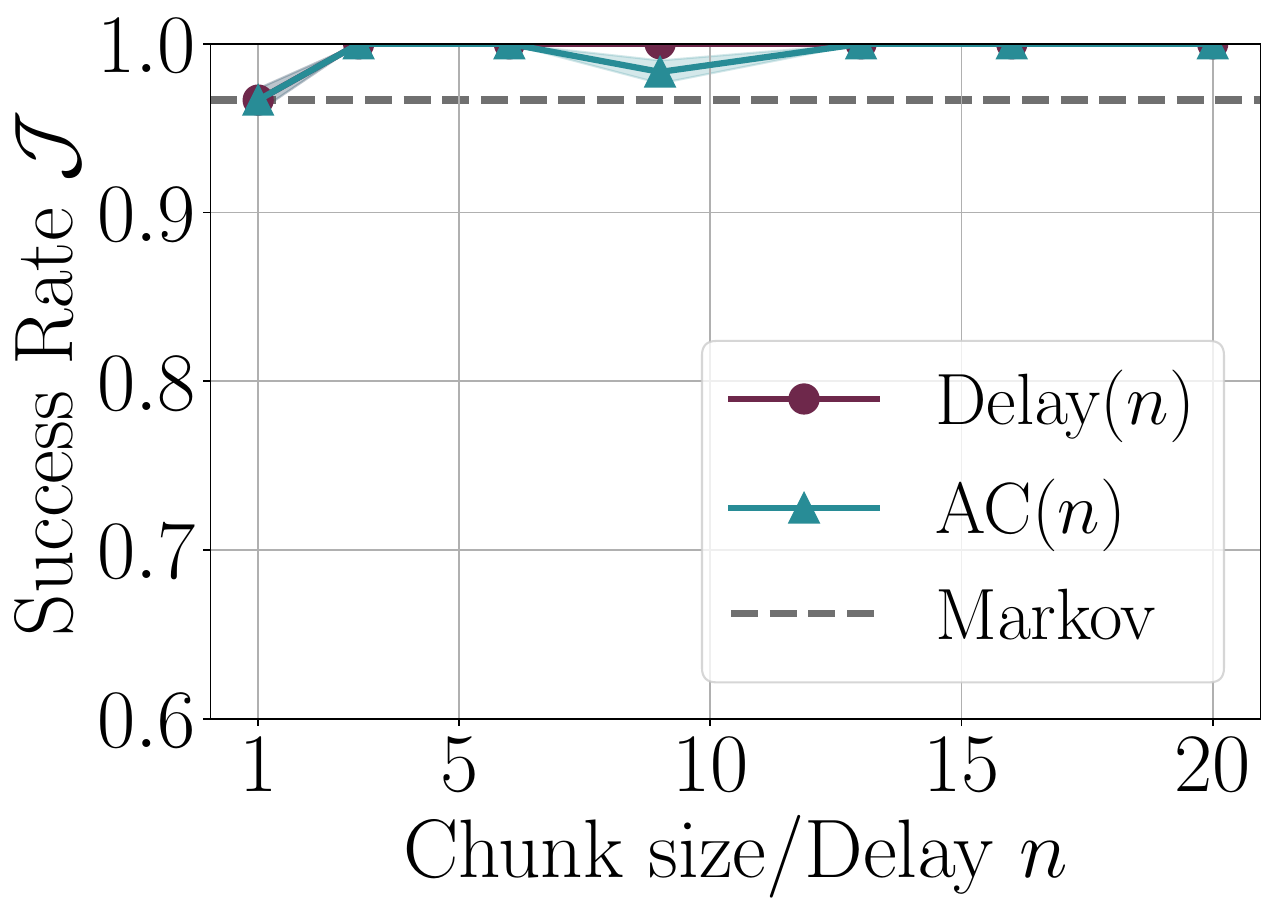}
    \end{minipage}
    \hfill
        \begin{minipage}[t]{0.23\textwidth}
        \centering
        \includegraphics[width=\linewidth]{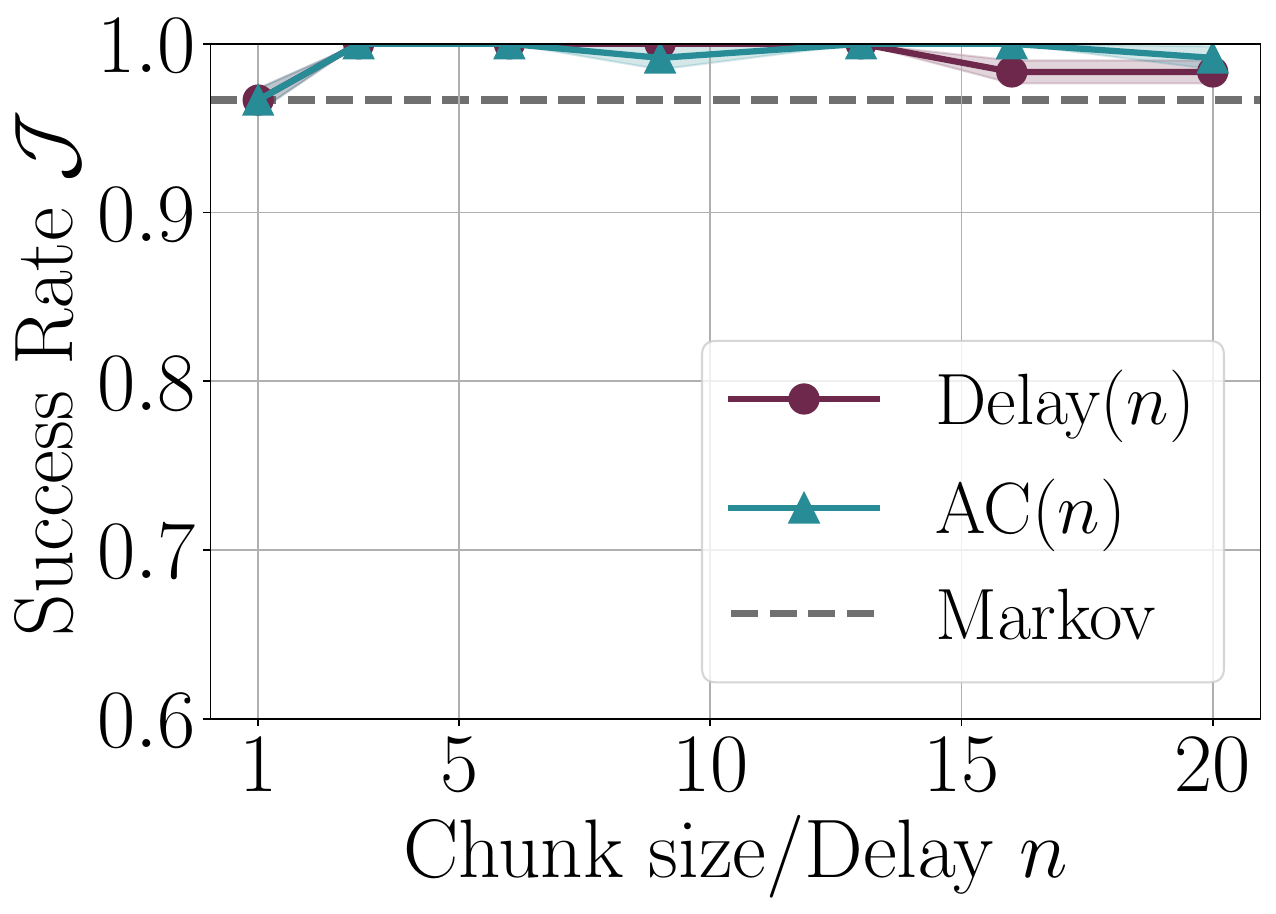}
    \end{minipage}
    \hfill
        \begin{minipage}[t]{0.23\textwidth}
        \centering
        \includegraphics[width=\linewidth]{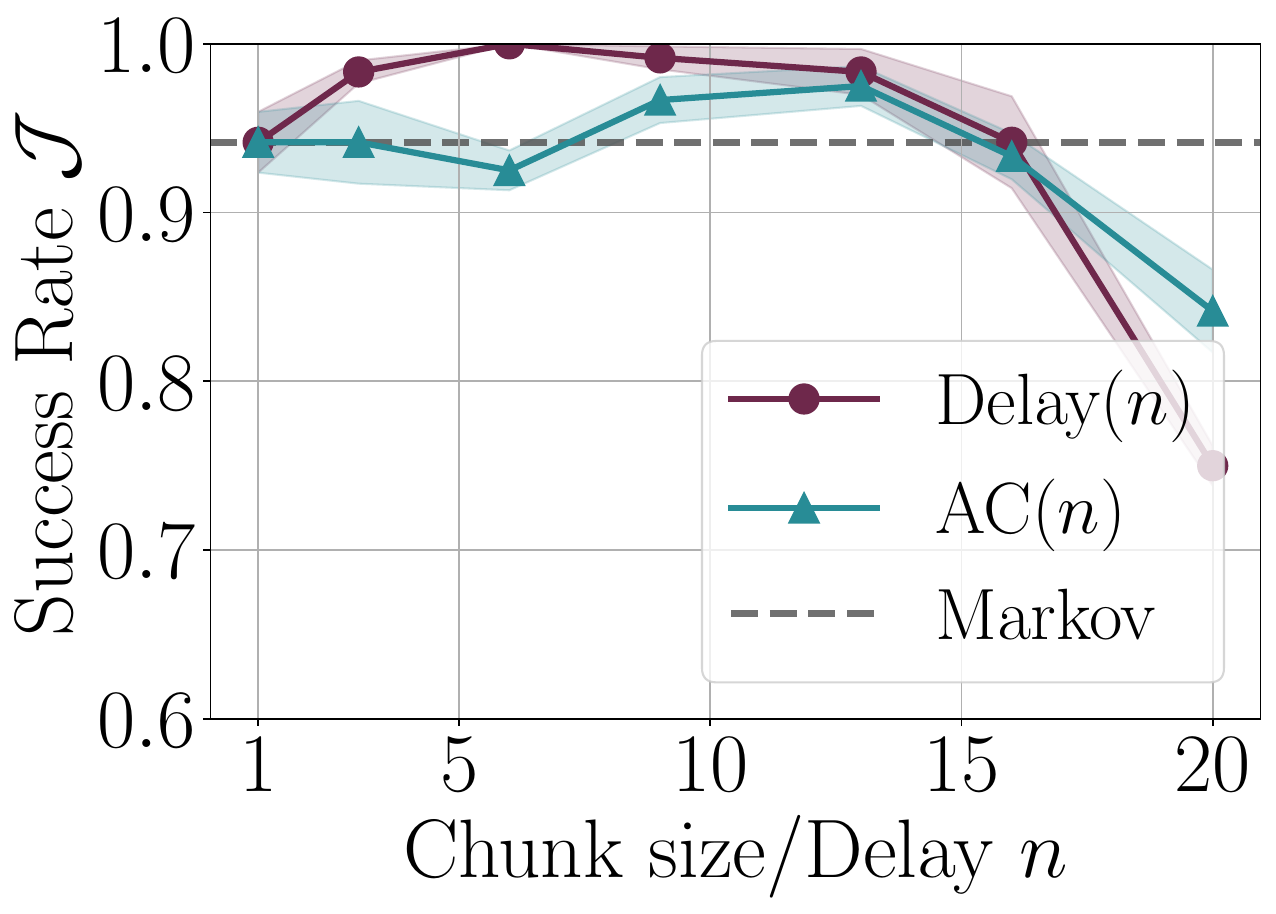}
    \end{minipage}
    \hfill
        \begin{minipage}[t]{0.23\textwidth}
        \centering
        \includegraphics[width=\linewidth]{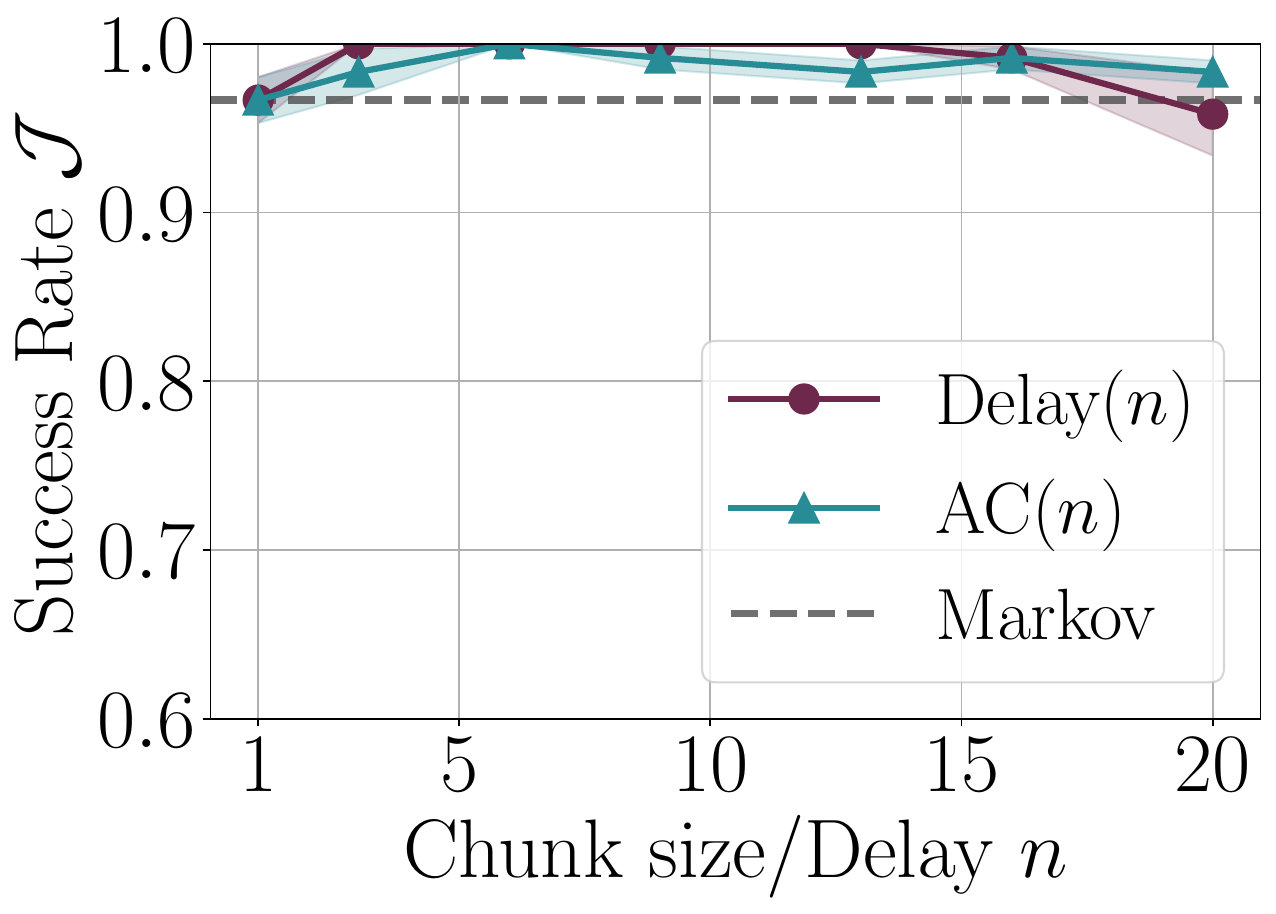}
    \end{minipage}
        \begin{minipage}[t]{0.23\textwidth}
            \includegraphics[width=\linewidth]{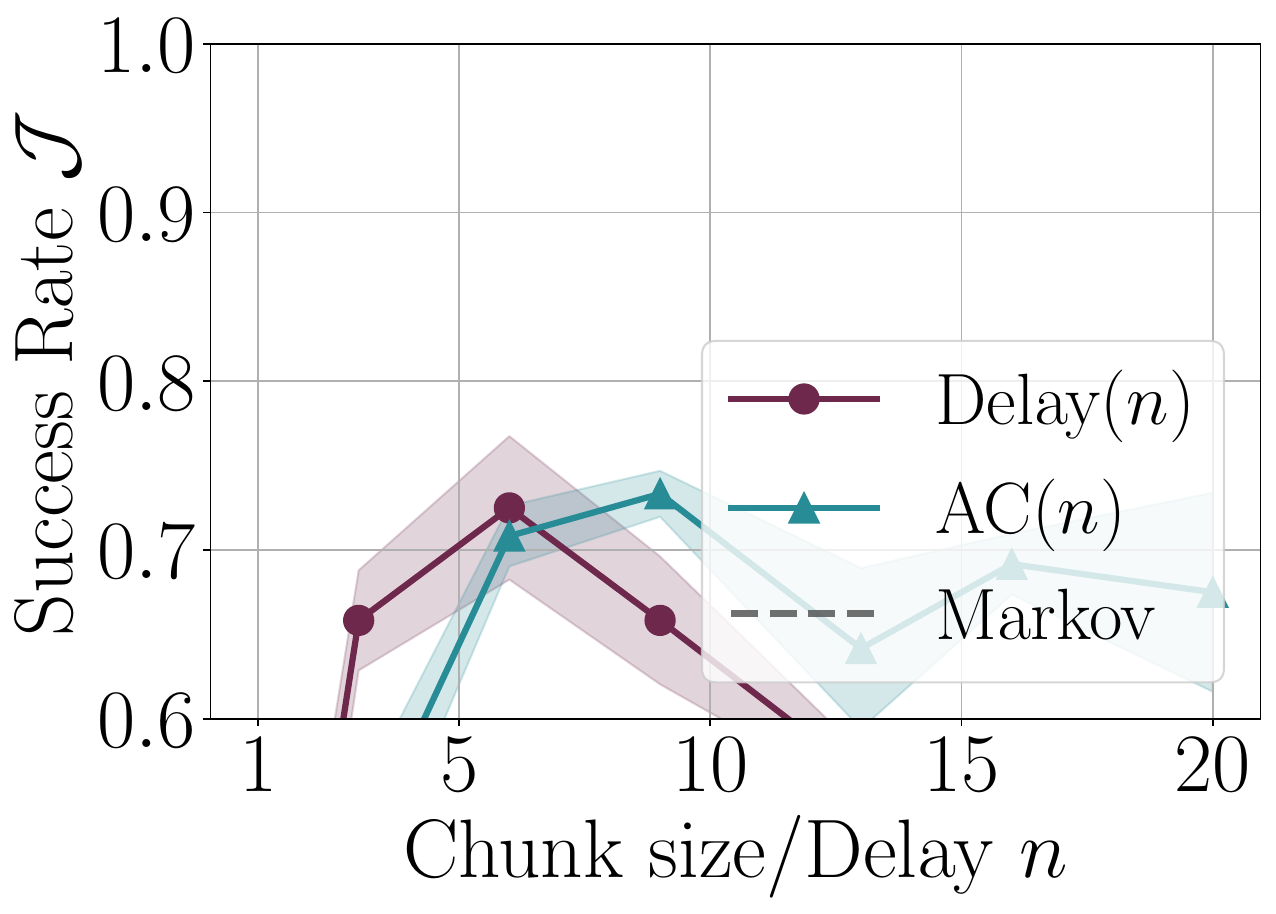}
    \end{minipage}
    \hfill
        \begin{minipage}[t]{0.23\textwidth}
        \centering
        \includegraphics[width=\linewidth]{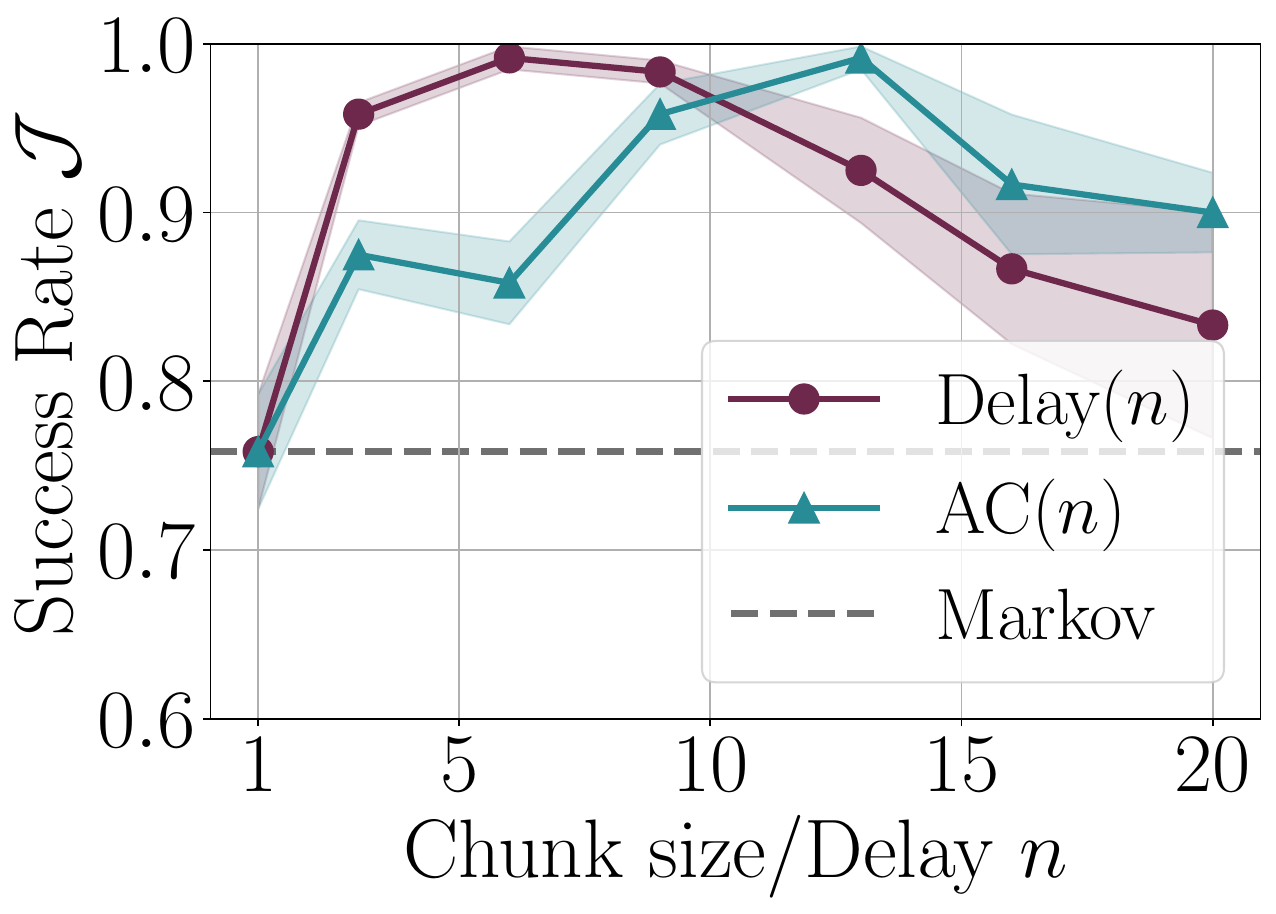}
    \end{minipage}
    \hfill
        \begin{minipage}[t]{0.23\textwidth}
        \centering
        \includegraphics[width=\linewidth]{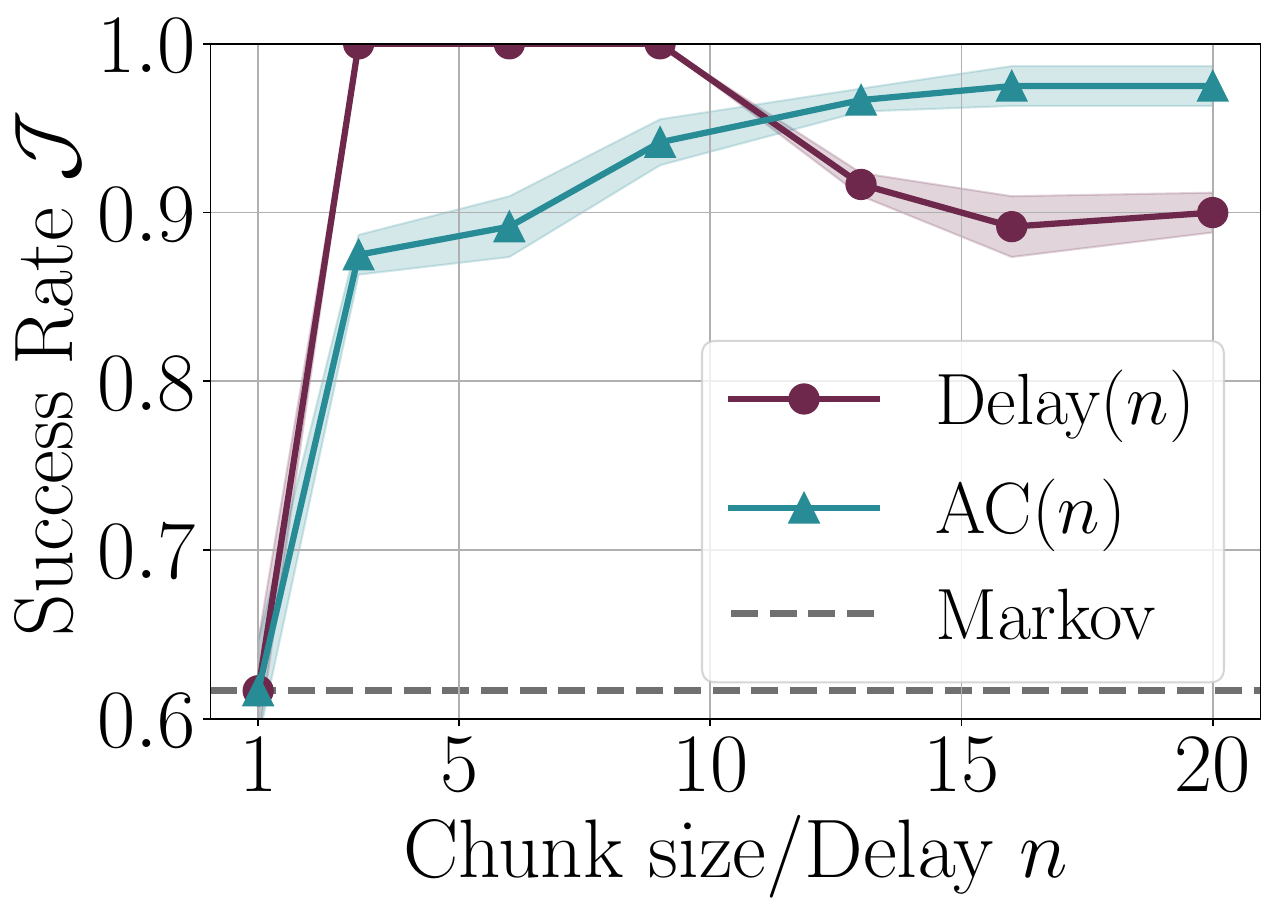}
    \end{minipage}
    \hfill
        \begin{minipage}[t]{0.23\textwidth}
        \centering
        \includegraphics[width=\linewidth]{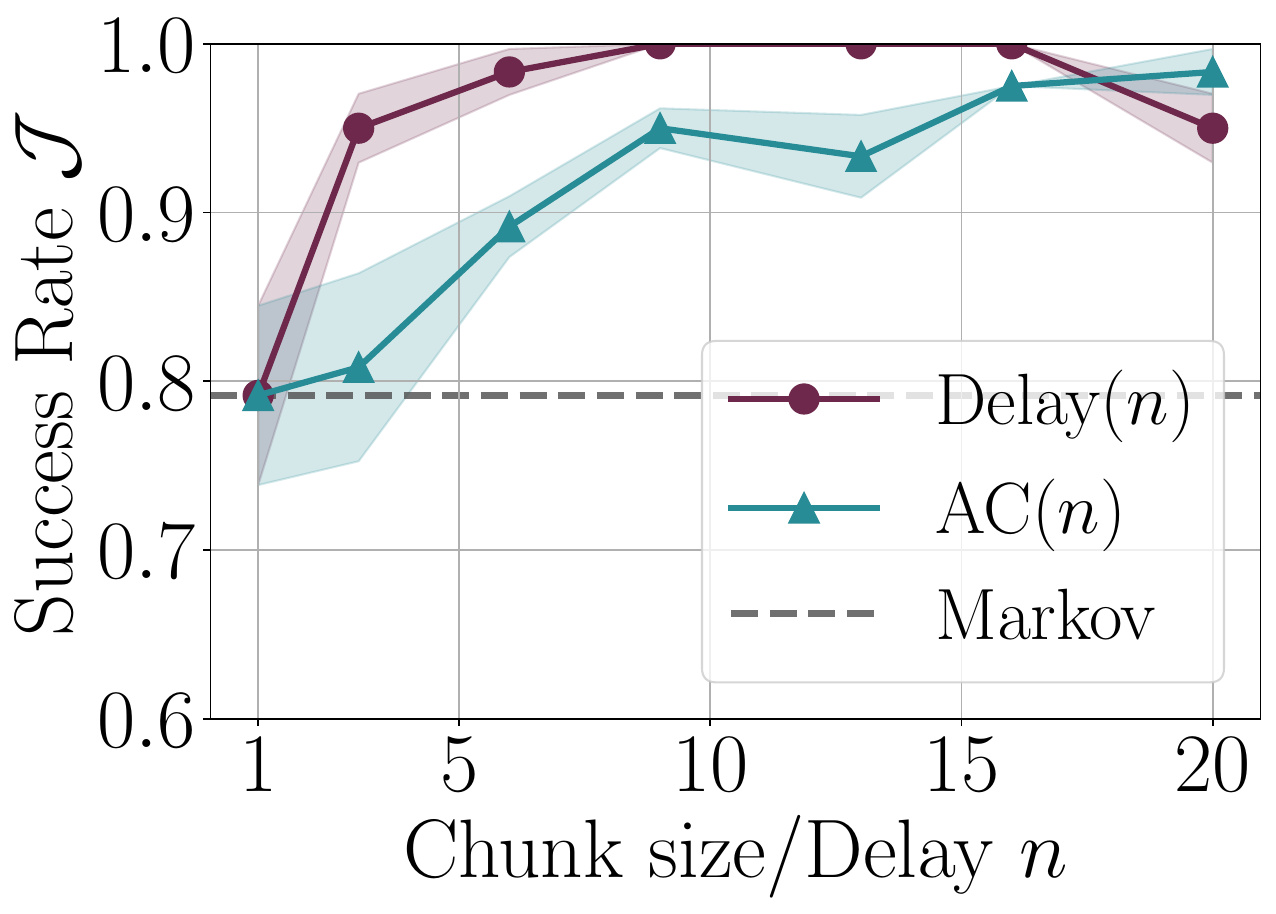}
    \end{minipage}
        \caption{Success rate for each \texttt{Libero} task from 0 to 35 (corresponding to Fig. \ref{fig:success_libero}), part 1.}
    \label{fig:succ each libero1}
\end{figure*}

\begin{figure*}

        \begin{minipage}[t]{0.23\textwidth}
            \includegraphics[width=\linewidth]{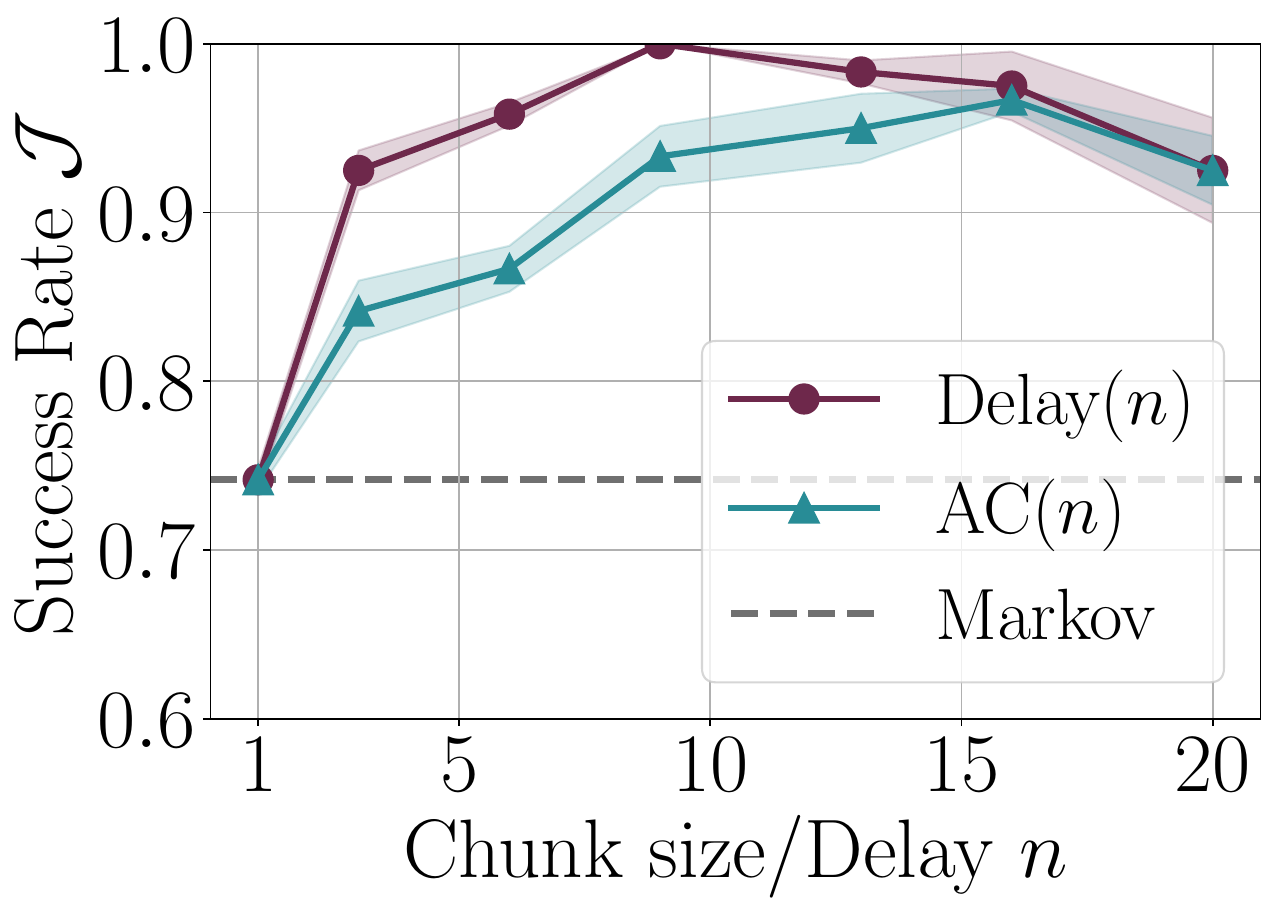}
    \end{minipage}
    \hfill
        \begin{minipage}[t]{0.23\textwidth}
        \centering
        \includegraphics[width=\linewidth]{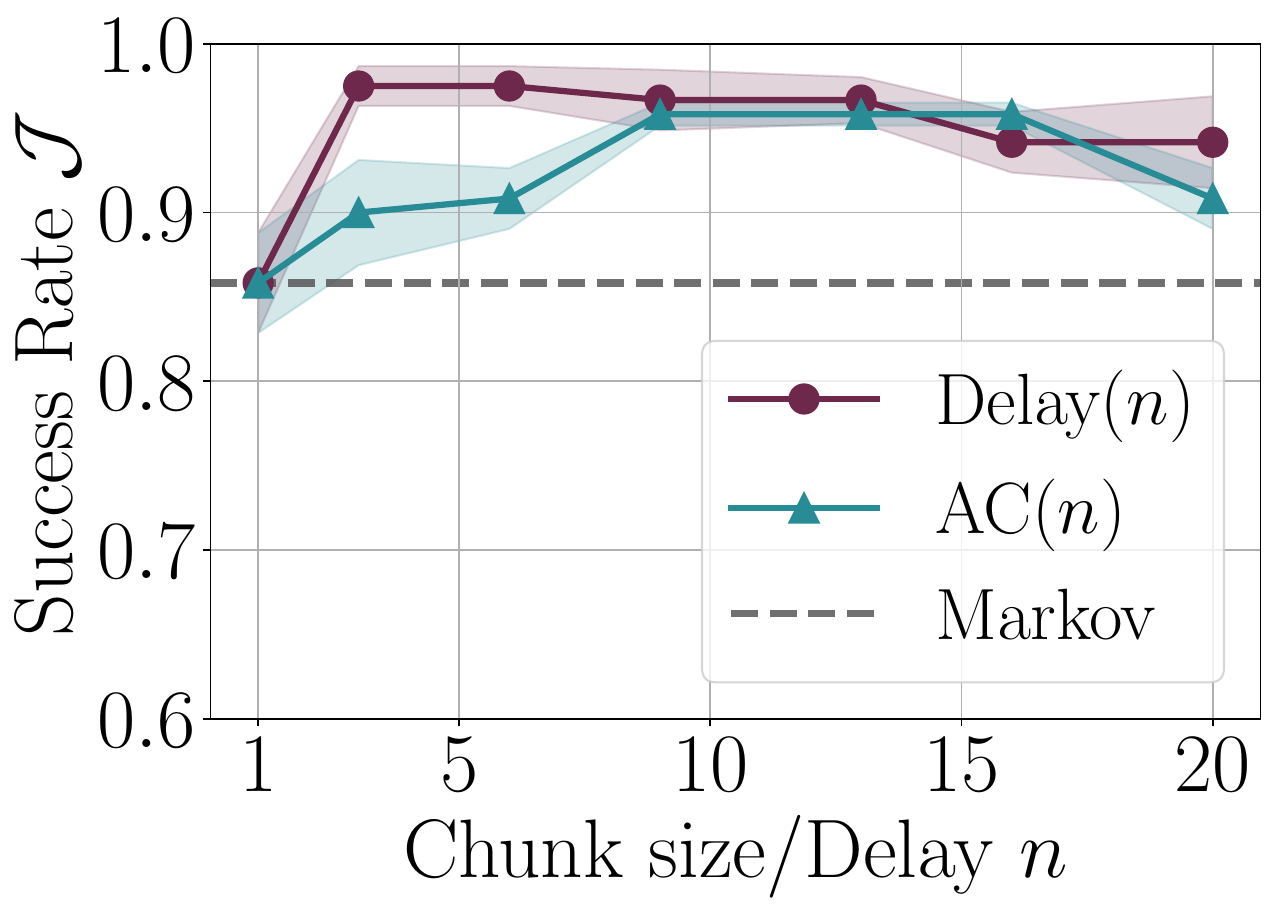}
    \end{minipage}
    \hfill
        \begin{minipage}[t]{0.23\textwidth}
        \centering
        \includegraphics[width=\linewidth]{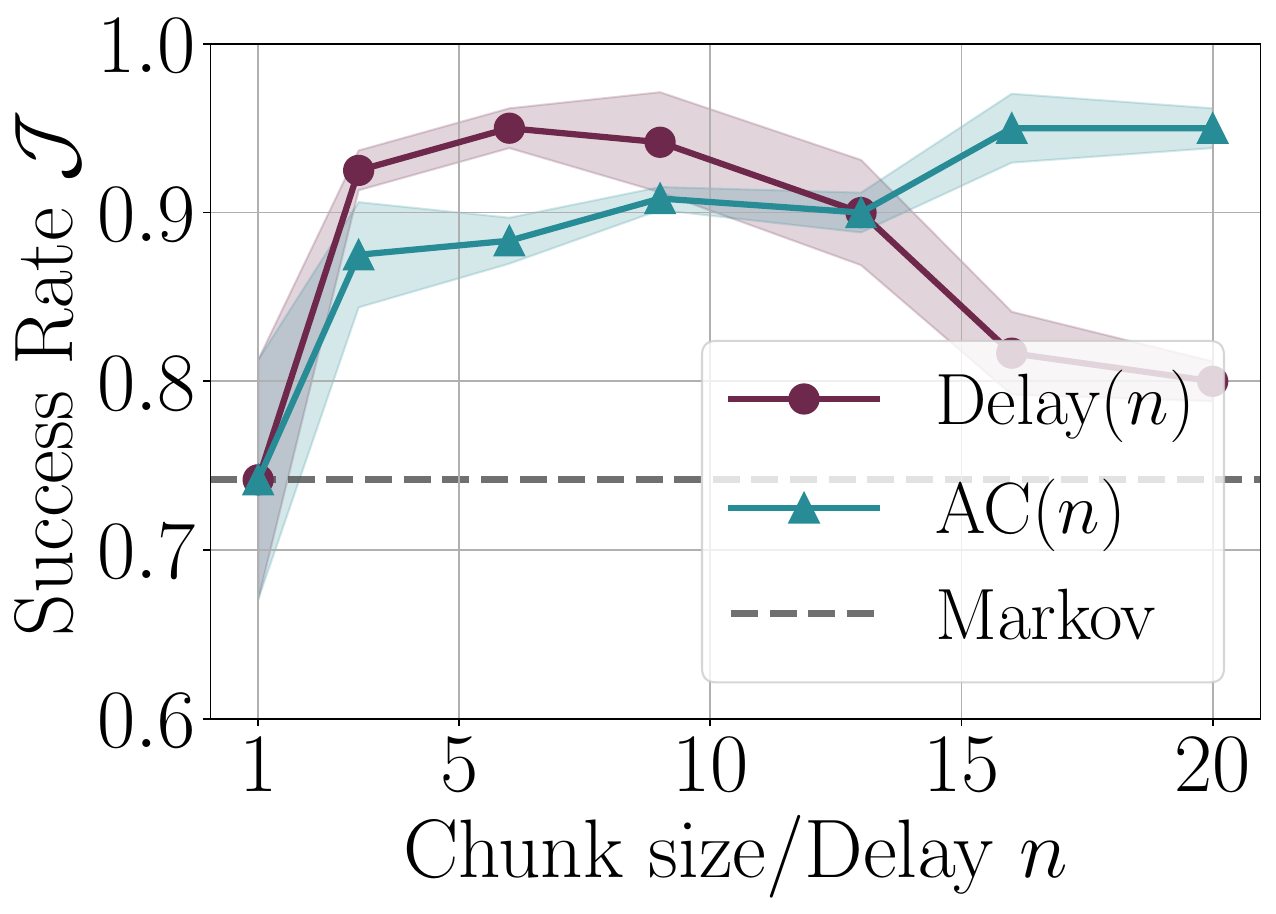}
    \end{minipage}
    \hfill
        \begin{minipage}[t]{0.23\textwidth}
        \centering
        \includegraphics[width=\linewidth]{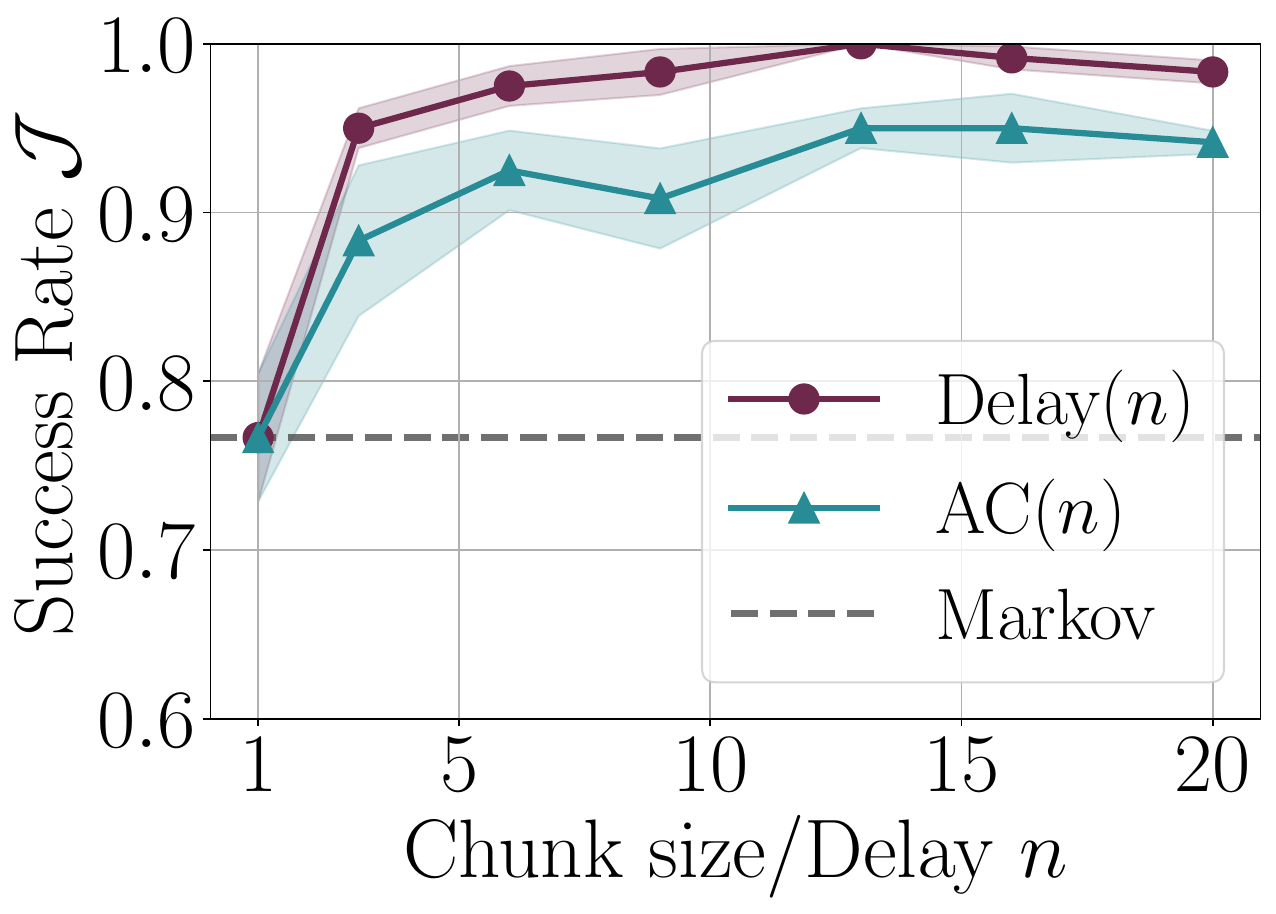}
    \end{minipage}
        \begin{minipage}[t]{0.23\textwidth}
            \includegraphics[width=\linewidth]{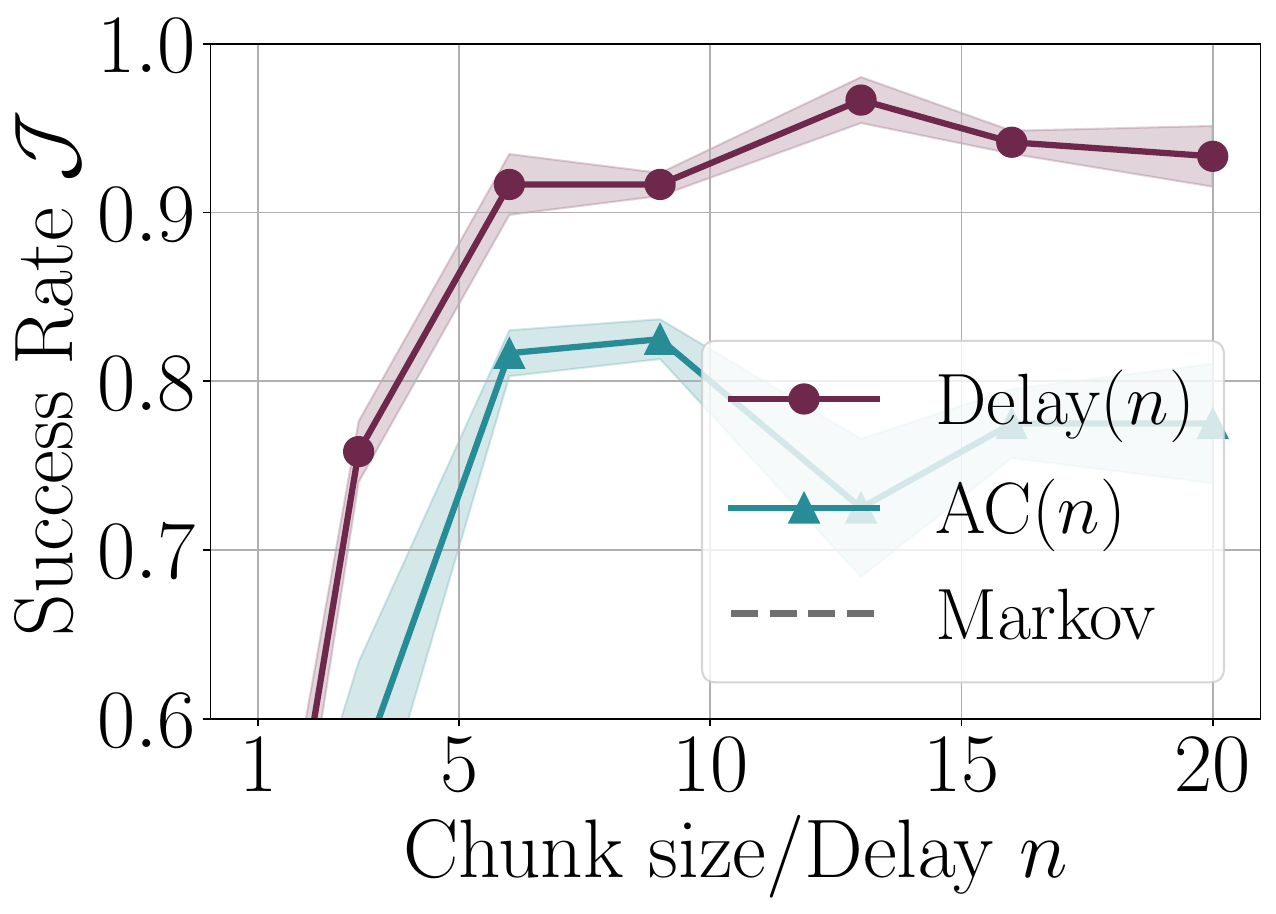}
    \end{minipage}
    \hfill
        \begin{minipage}[t]{0.23\textwidth}
        \centering
        \includegraphics[width=\linewidth]{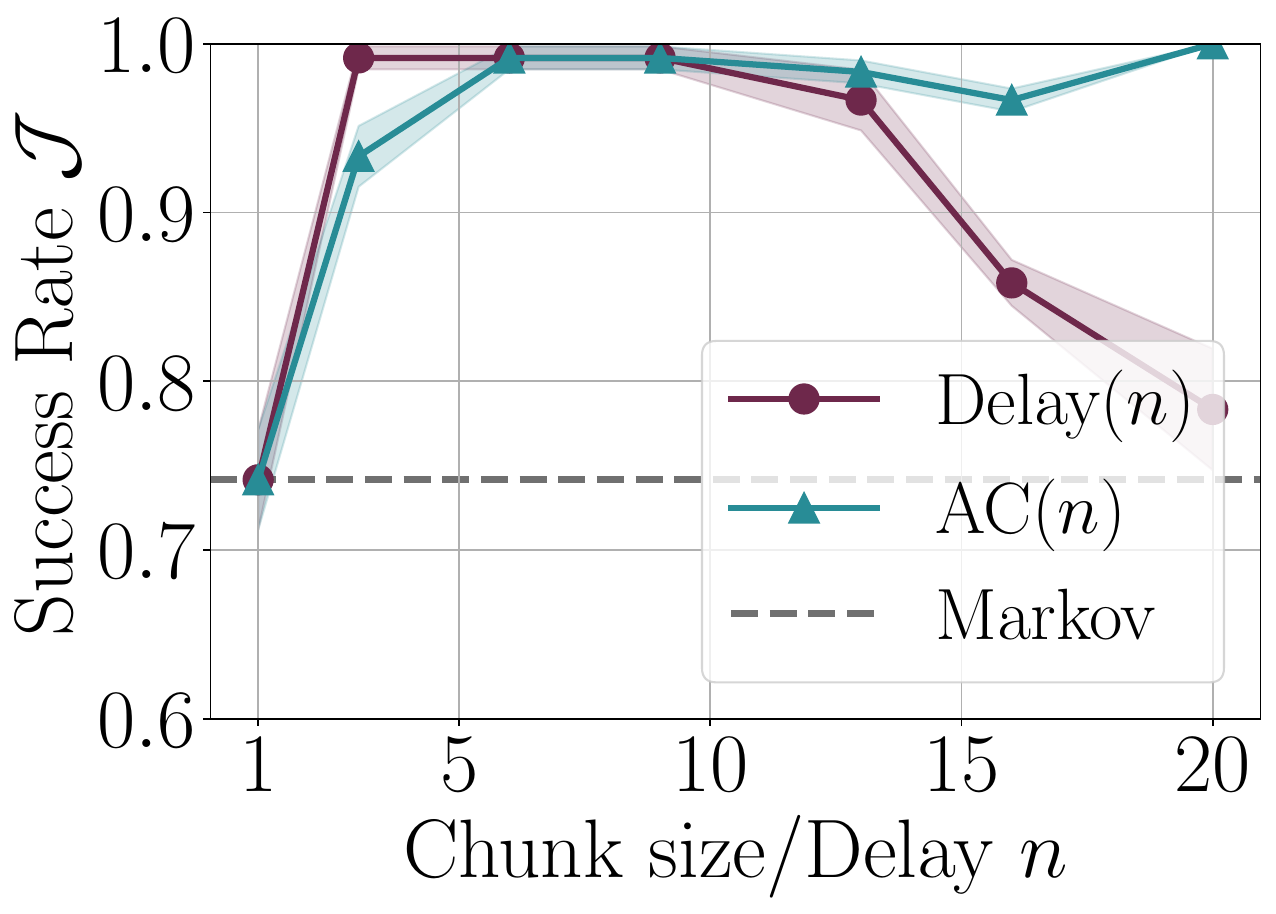}
    \end{minipage}
    \hfill
        \begin{minipage}[t]{0.23\textwidth}
        \centering
        \includegraphics[width=\linewidth]{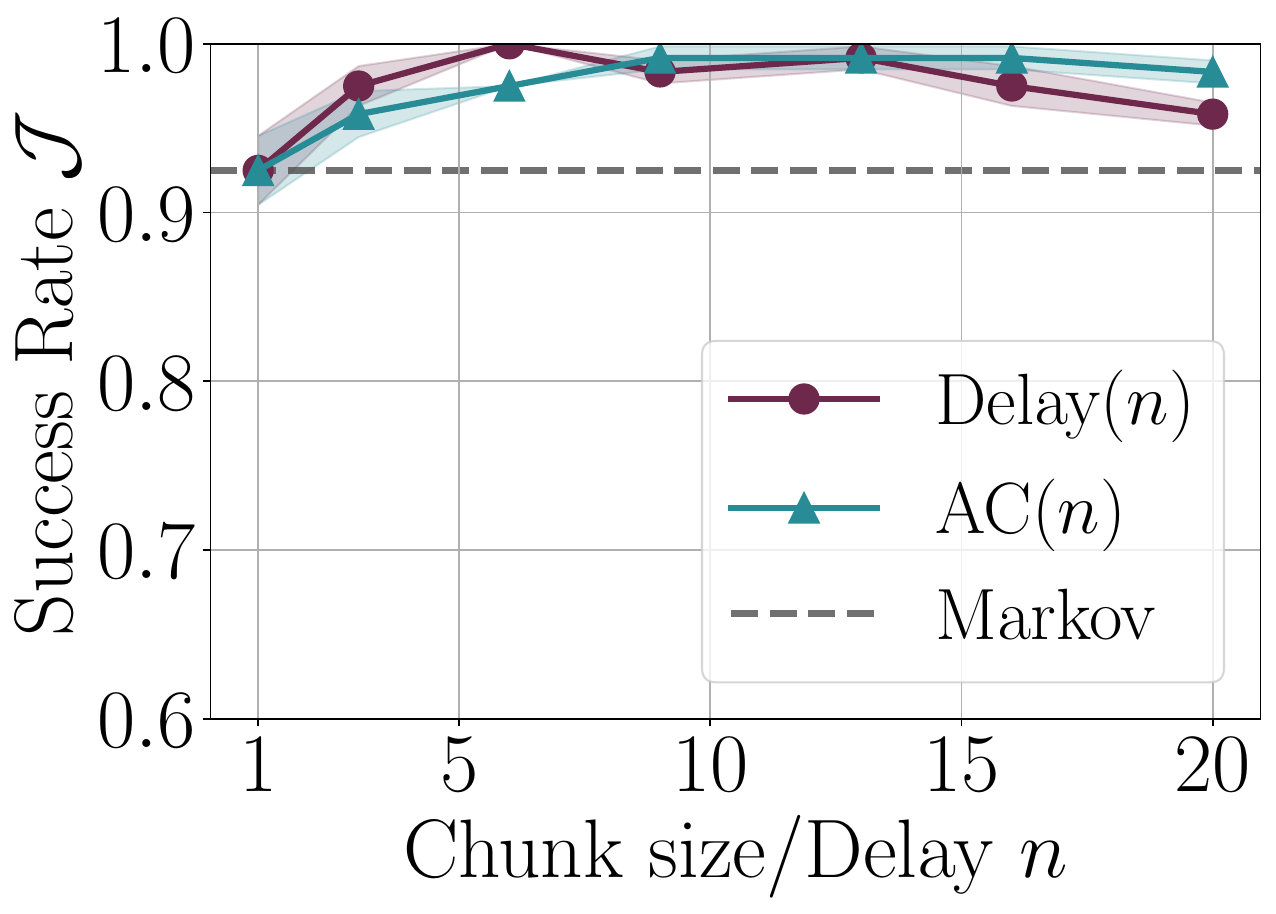}
    \end{minipage}
    \hfill
        \begin{minipage}[t]{0.23\textwidth}
        \centering
        \includegraphics[width=\linewidth]{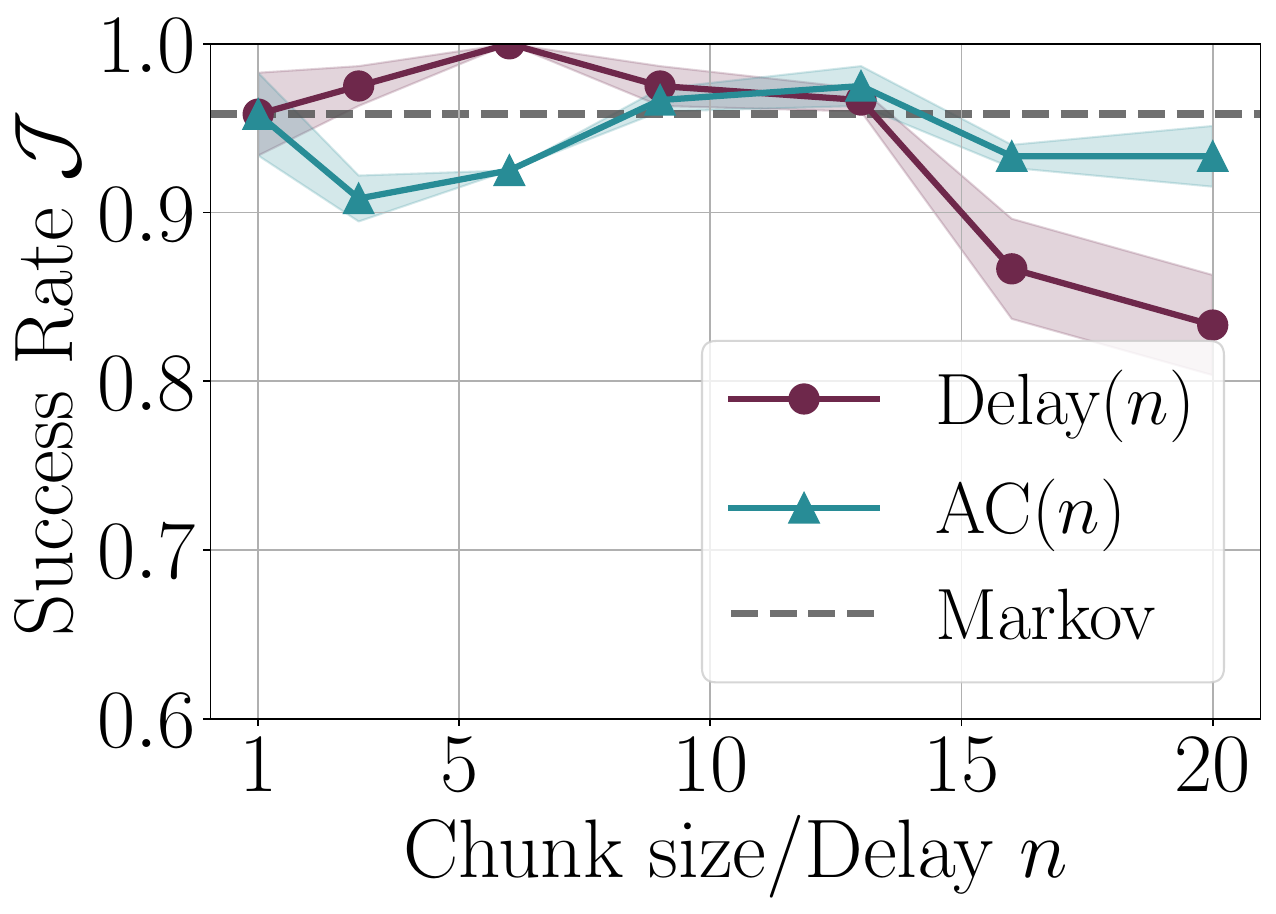}
    \end{minipage}
        \begin{minipage}[t]{0.23\textwidth}
            \includegraphics[width=\linewidth]{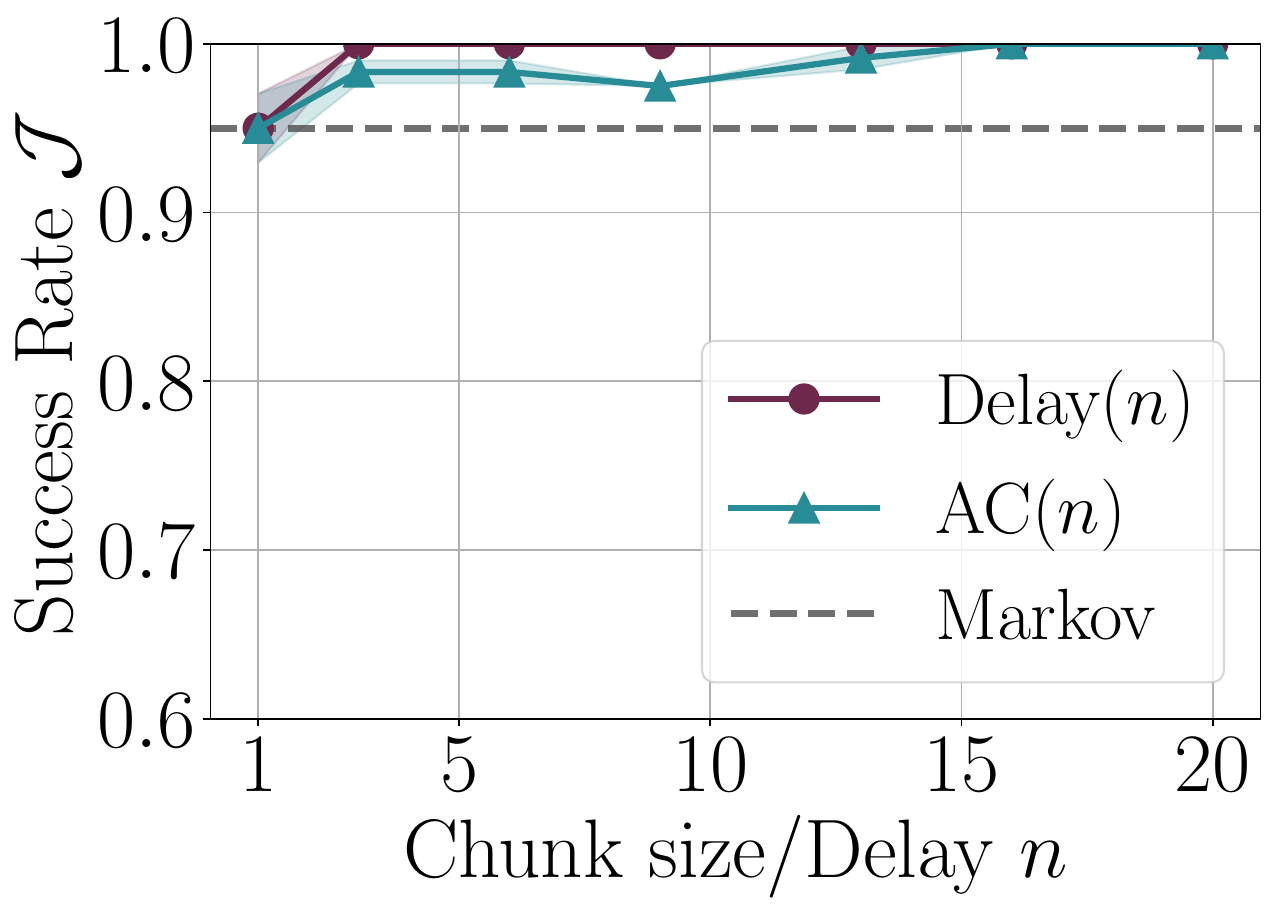}
    \end{minipage}
    \hfill
        \begin{minipage}[t]{0.23\textwidth}
        \centering
        \includegraphics[width=\linewidth]{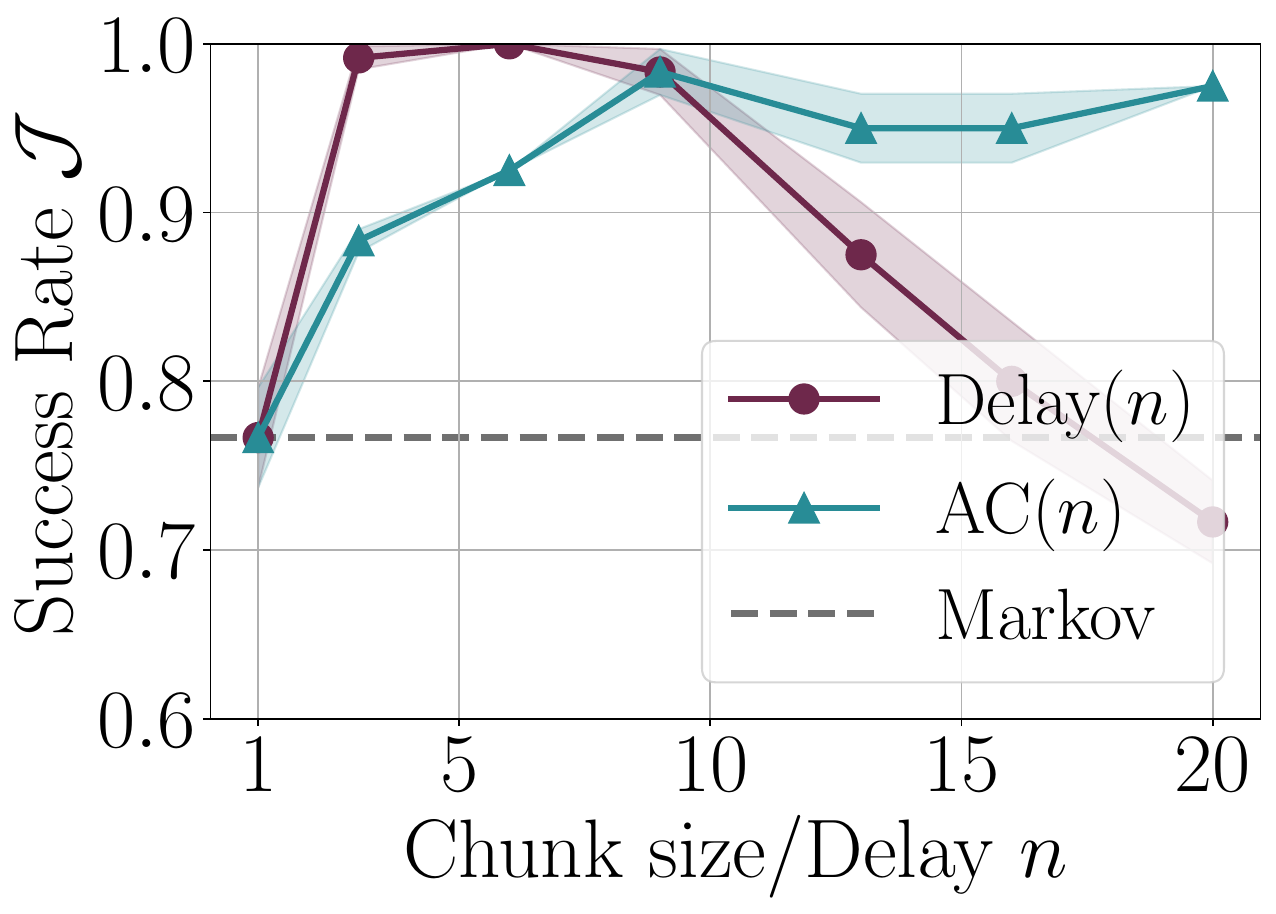}
    \end{minipage}
    \hfill
        \begin{minipage}[t]{0.23\textwidth}
        \centering
        \includegraphics[width=\linewidth]{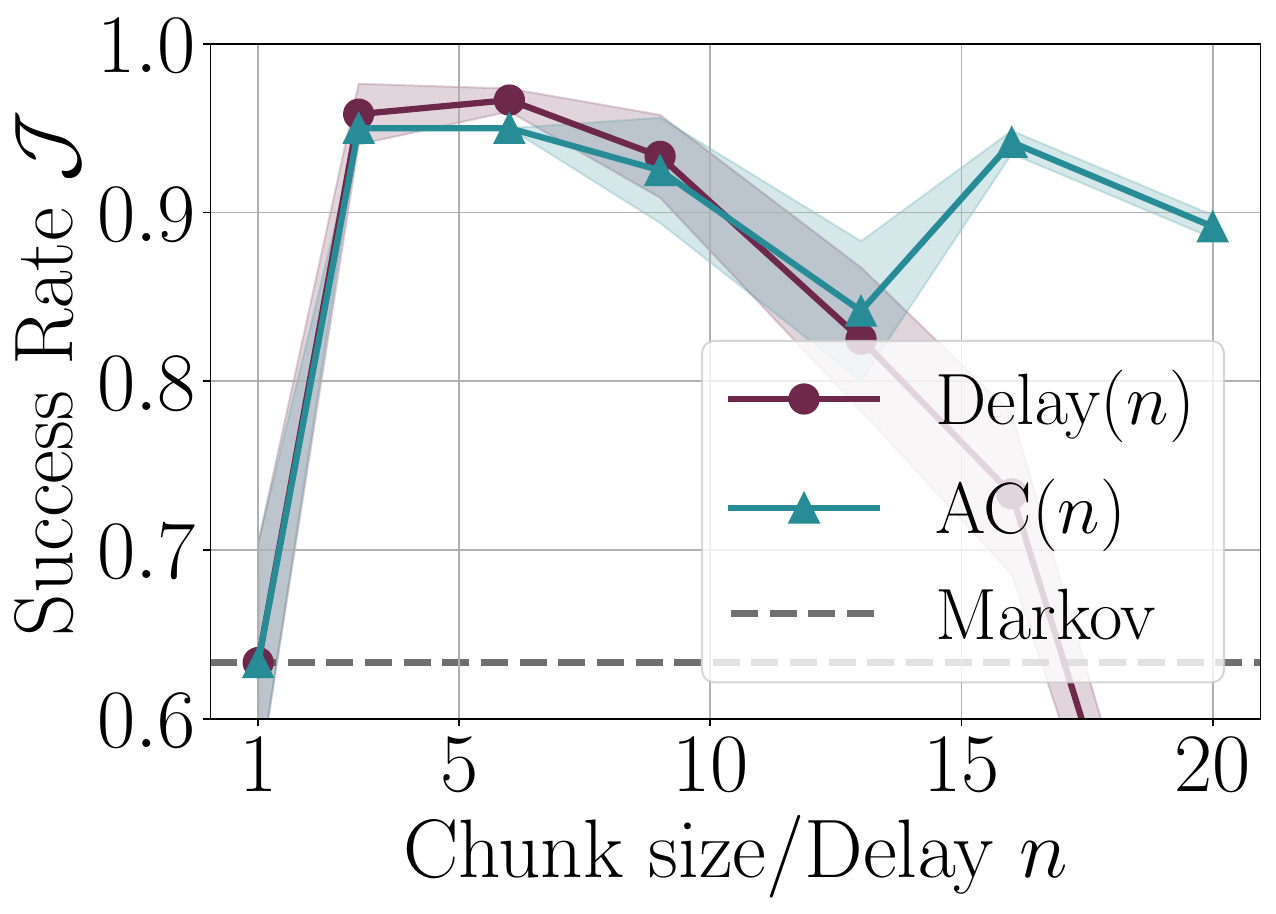}
    \end{minipage}
    \hfill
        \begin{minipage}[t]{0.23\textwidth}
        \centering
        \includegraphics[width=\linewidth]{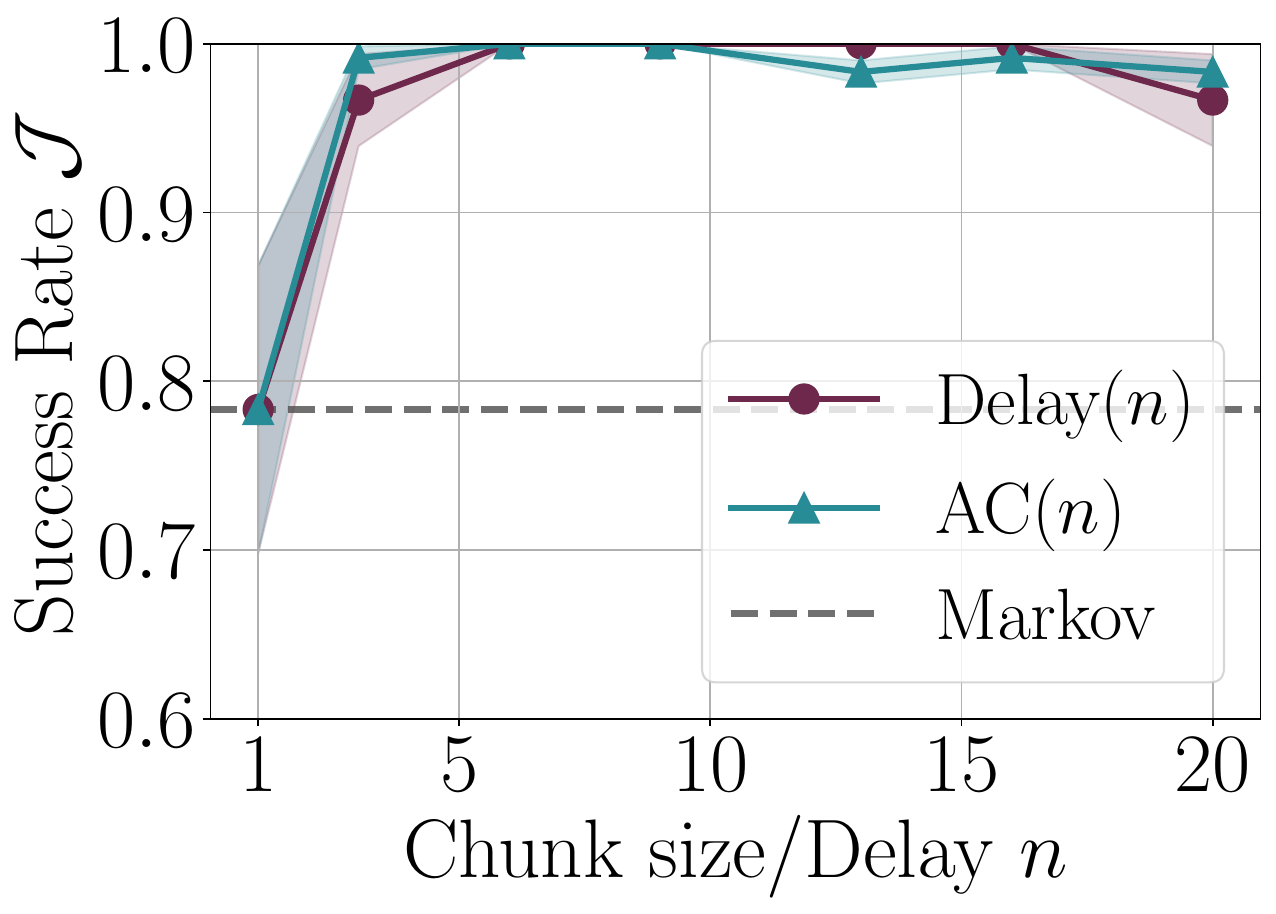}
    \end{minipage}
        \begin{minipage}[t]{0.23\textwidth}
            \includegraphics[width=\linewidth]{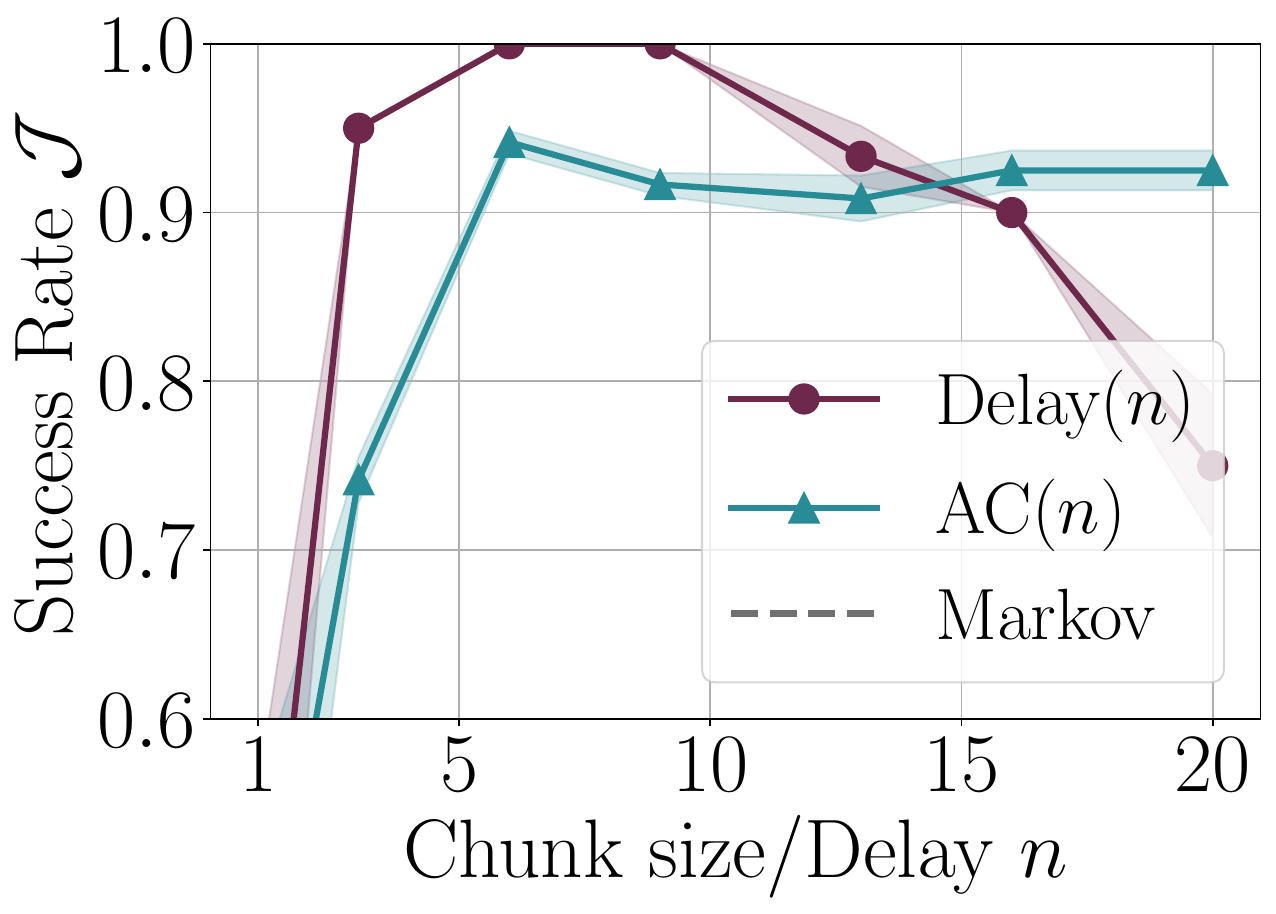}
    \end{minipage}
    \hfill
        \begin{minipage}[t]{0.23\textwidth}
        \centering
        \includegraphics[width=\linewidth]{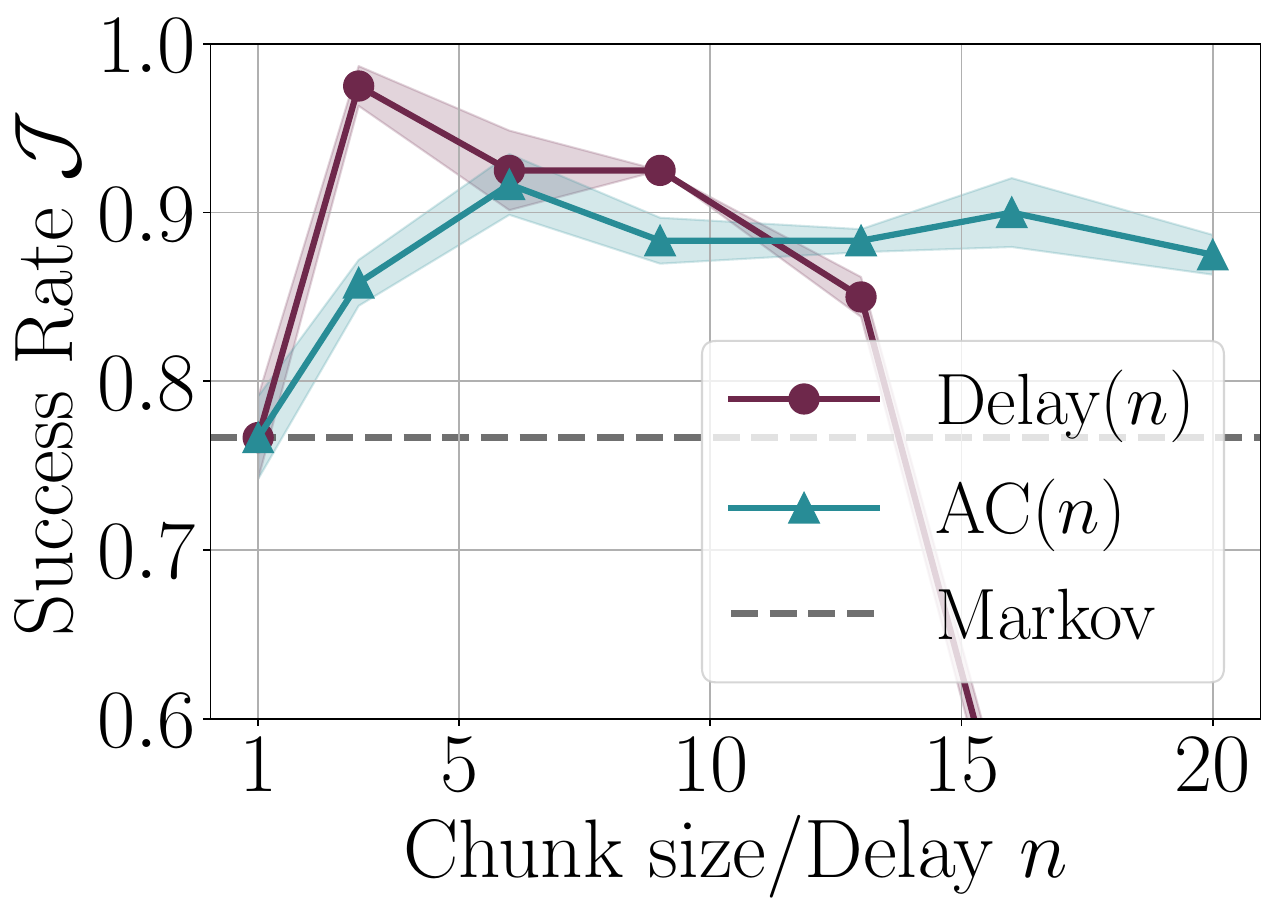}
    \end{minipage}
    \hfill
        \begin{minipage}[t]{0.23\textwidth}
        \centering
        \includegraphics[width=\linewidth]{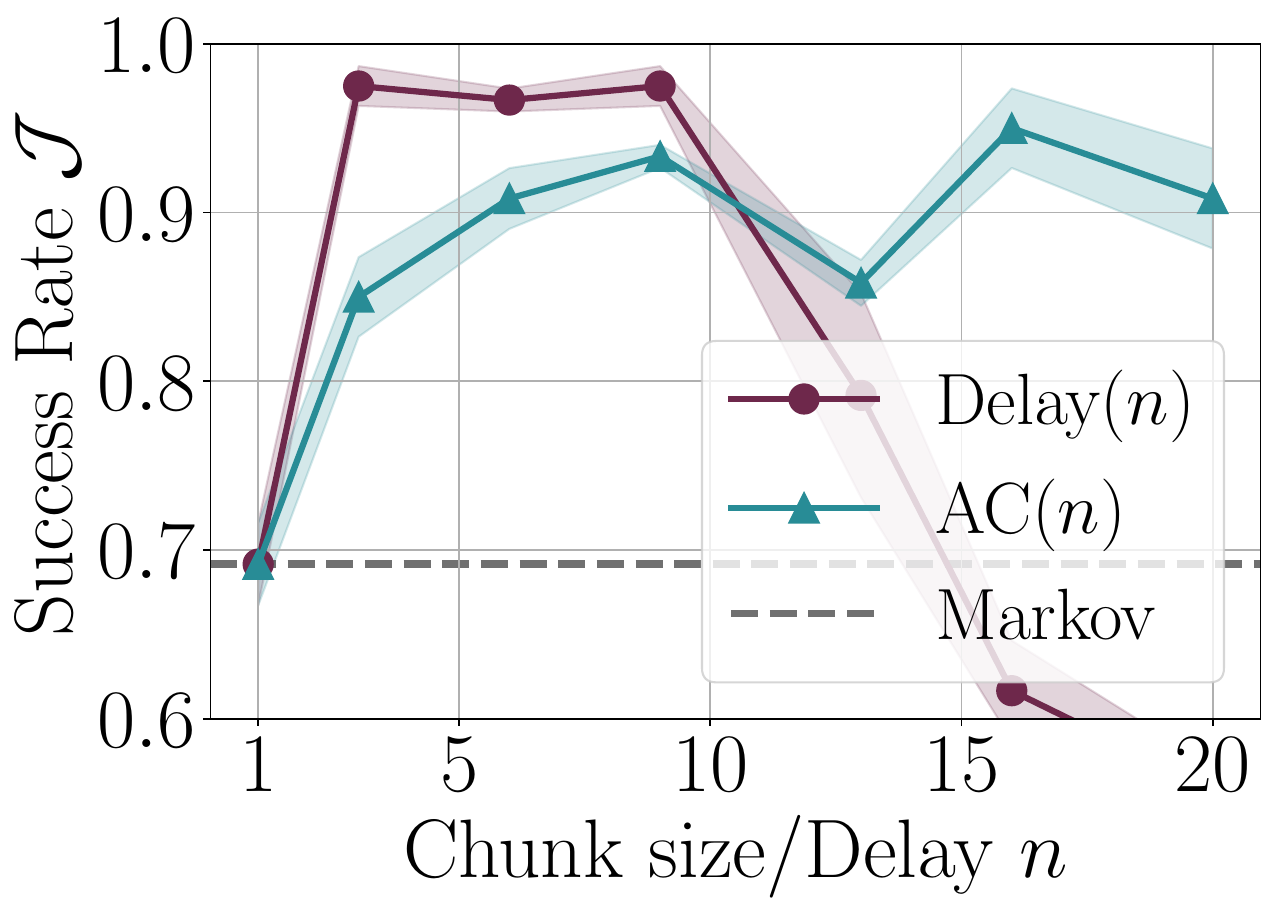}
    \end{minipage}
    \hfill
        \begin{minipage}[t]{0.23\textwidth}
        \centering
        \includegraphics[width=\linewidth]{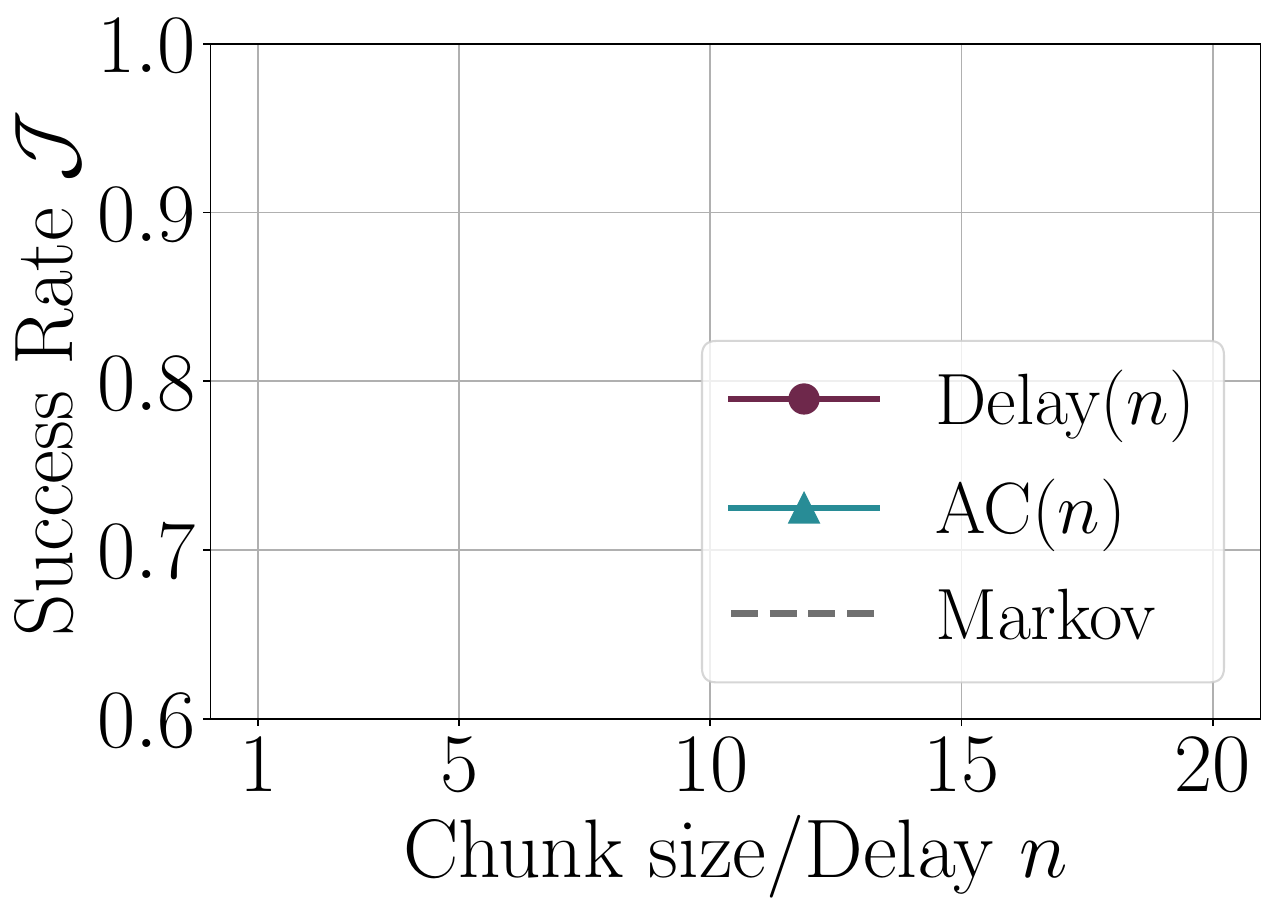}
    \end{minipage}
        \begin{minipage}[t]{0.23\textwidth}
            \includegraphics[width=\linewidth]{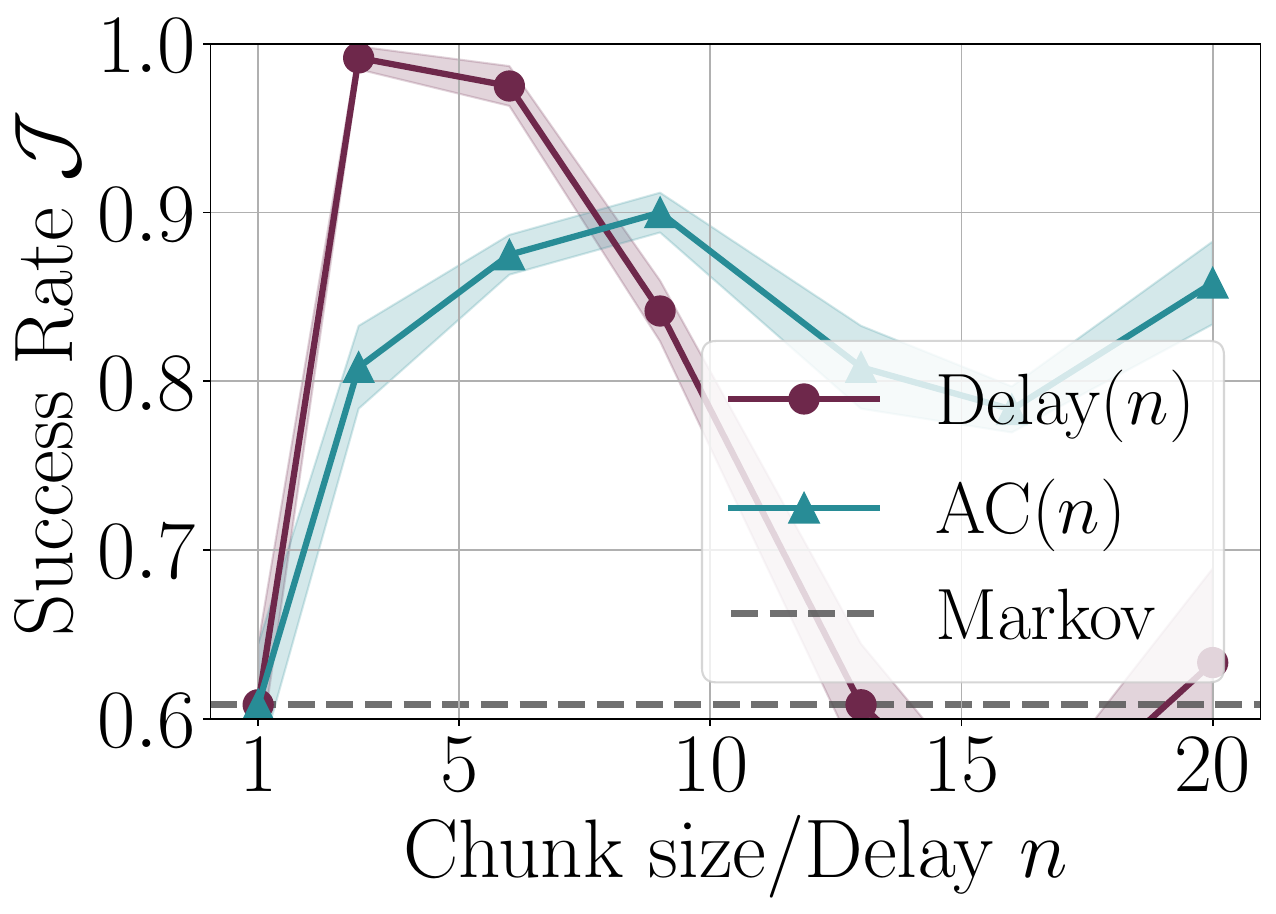}
    \end{minipage}
    \hfill
        \begin{minipage}[t]{0.23\textwidth}
        \centering
        \includegraphics[width=\linewidth]{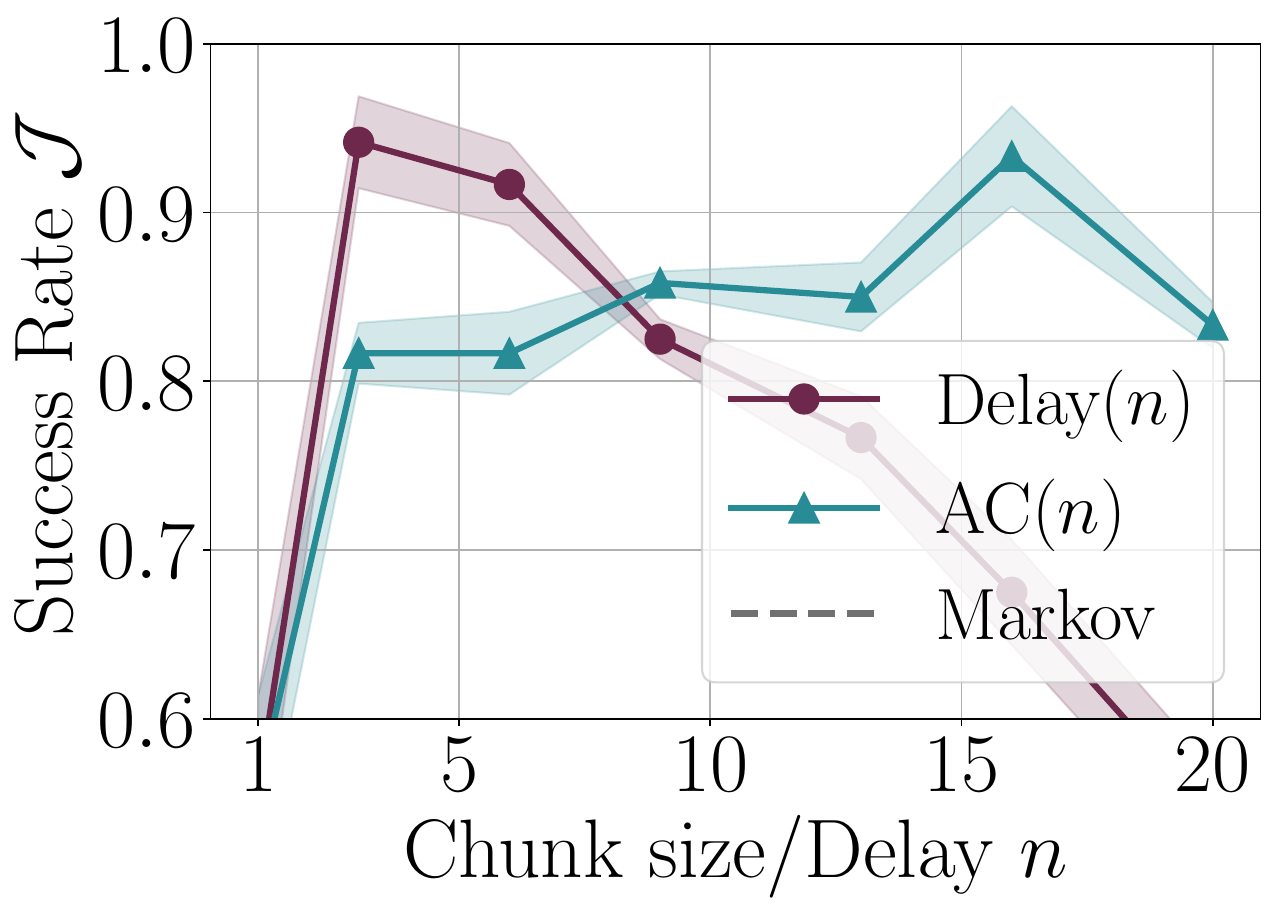}
    \end{minipage}
    \hfill
        \begin{minipage}[t]{0.23\textwidth}
        \centering
        \includegraphics[width=\linewidth]{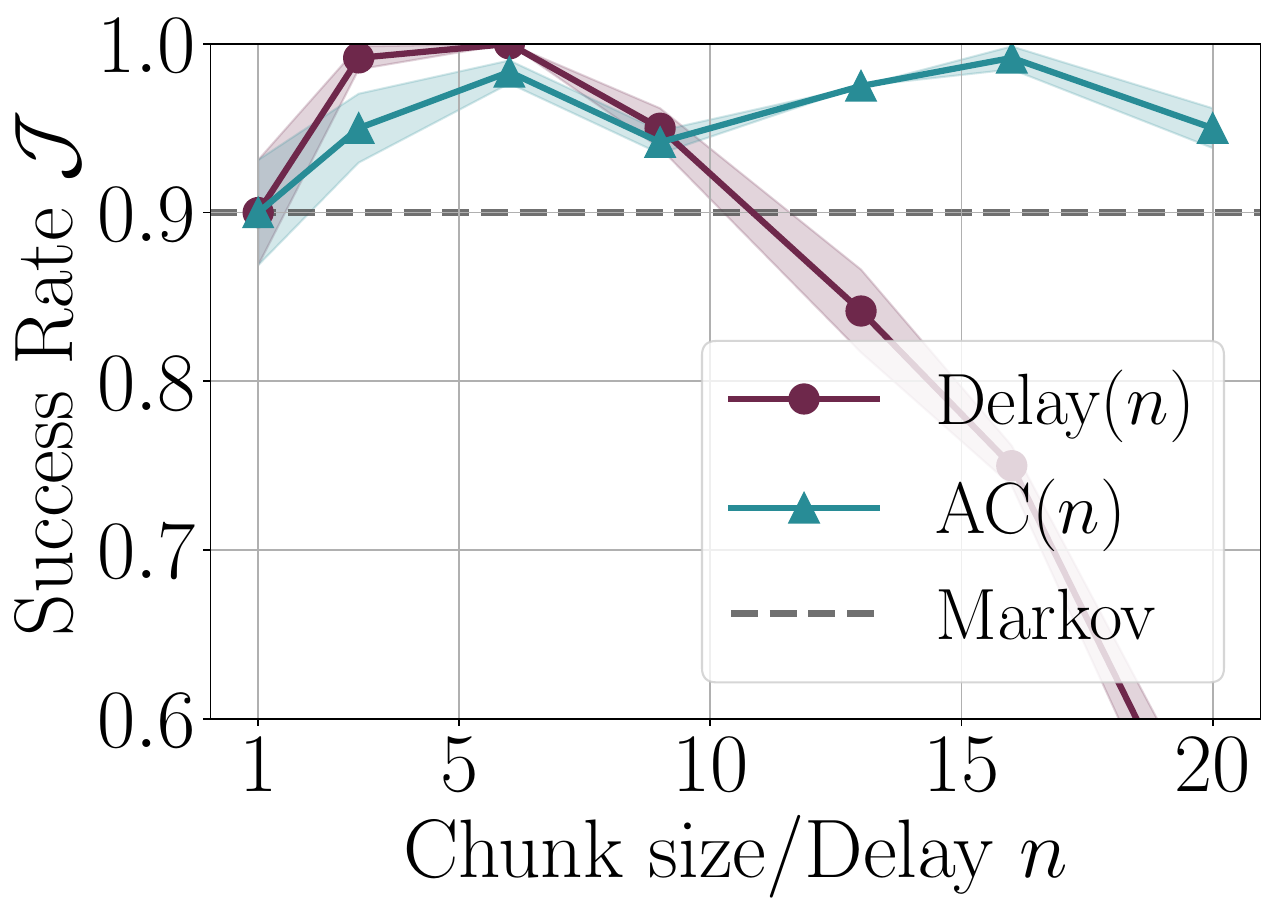}
    \end{minipage}
    \hfill
        \begin{minipage}[t]{0.23\textwidth}
        \centering
        \includegraphics[width=\linewidth]{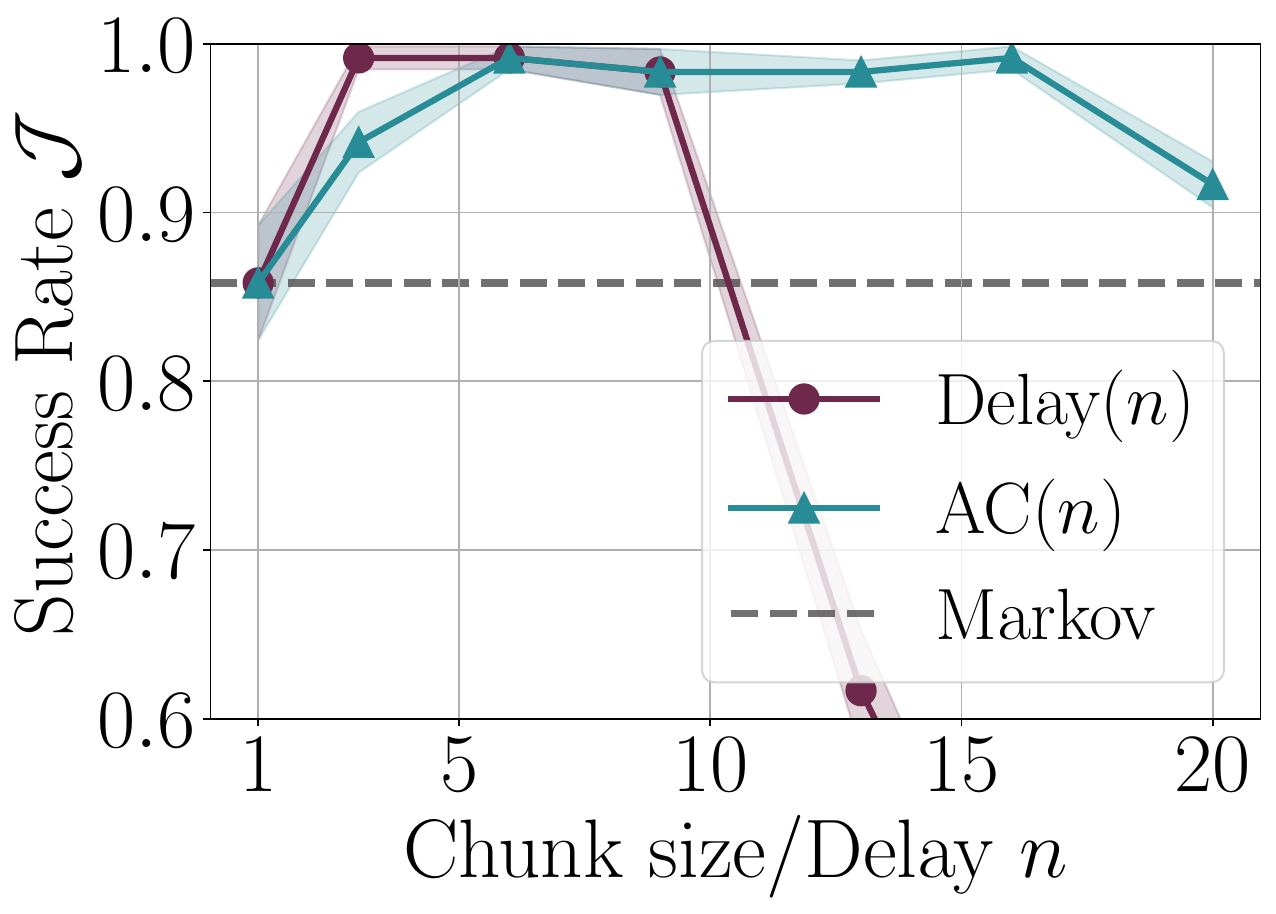}
    \end{minipage}
        \begin{minipage}[t]{0.23\textwidth}
            \includegraphics[width=\linewidth]{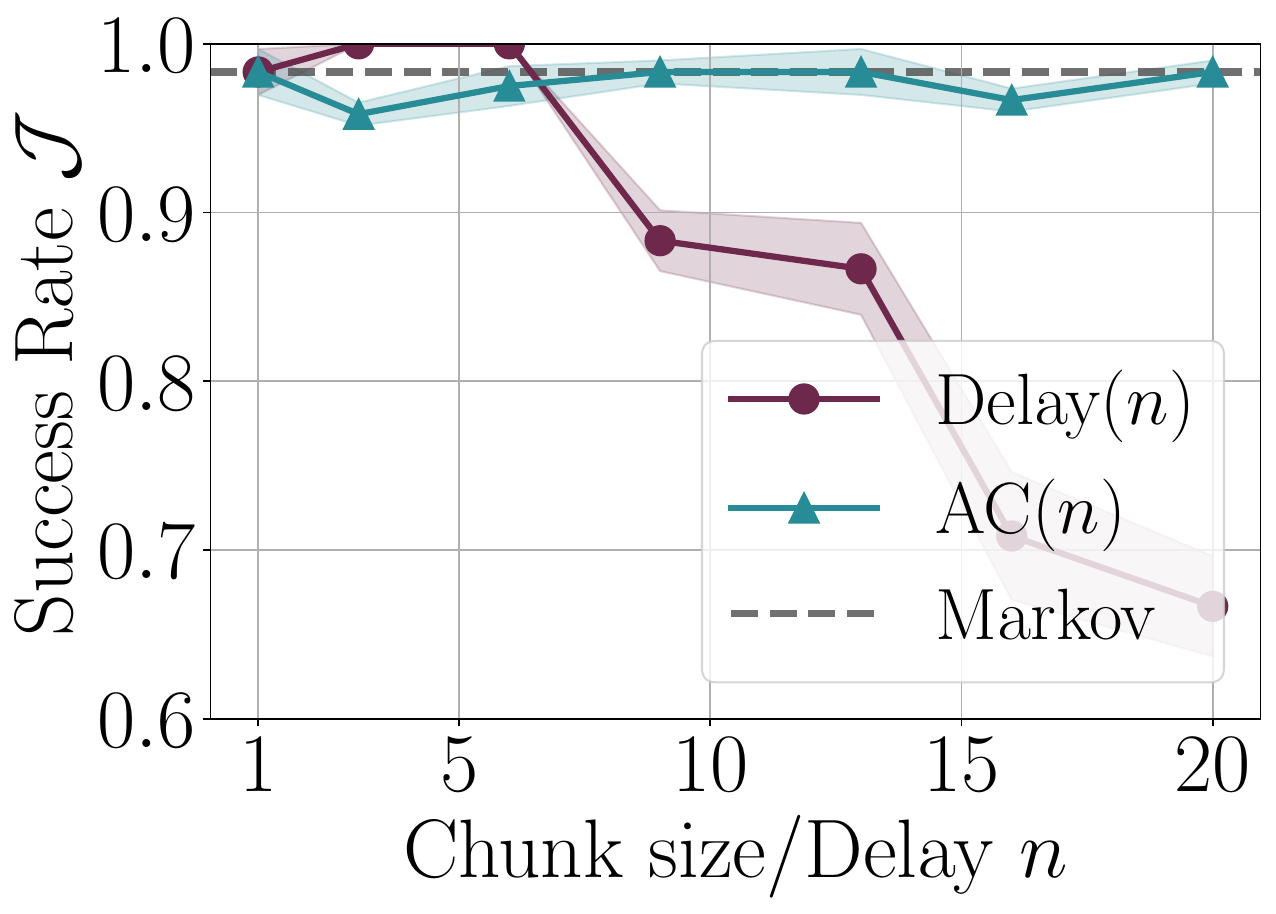}
    \end{minipage}
    \hfill
        \begin{minipage}[t]{0.23\textwidth}
        \centering
        \includegraphics[width=\linewidth]{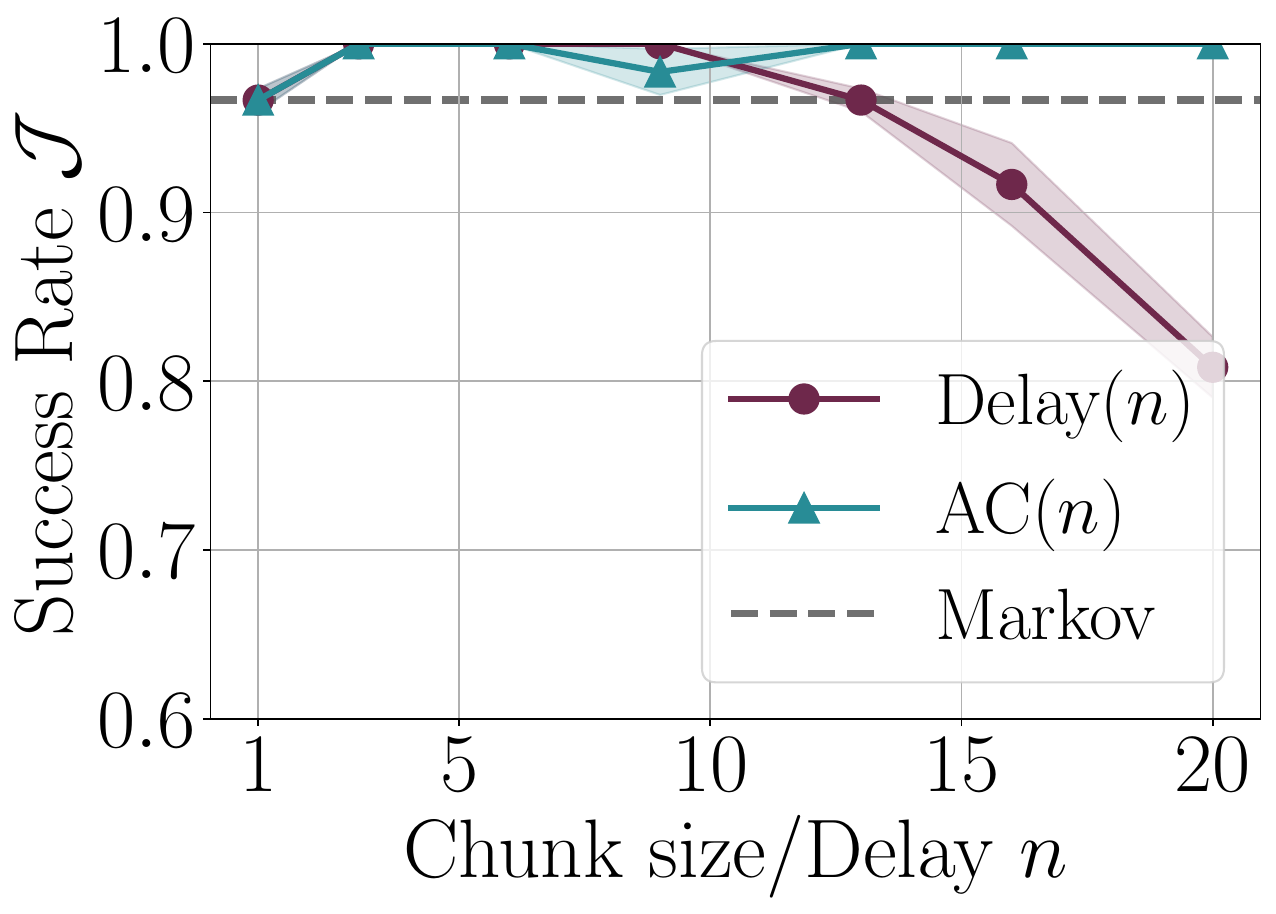}
    \end{minipage}
    \hfill
        \begin{minipage}[t]{0.23\textwidth}
        \centering
        \includegraphics[width=\linewidth]{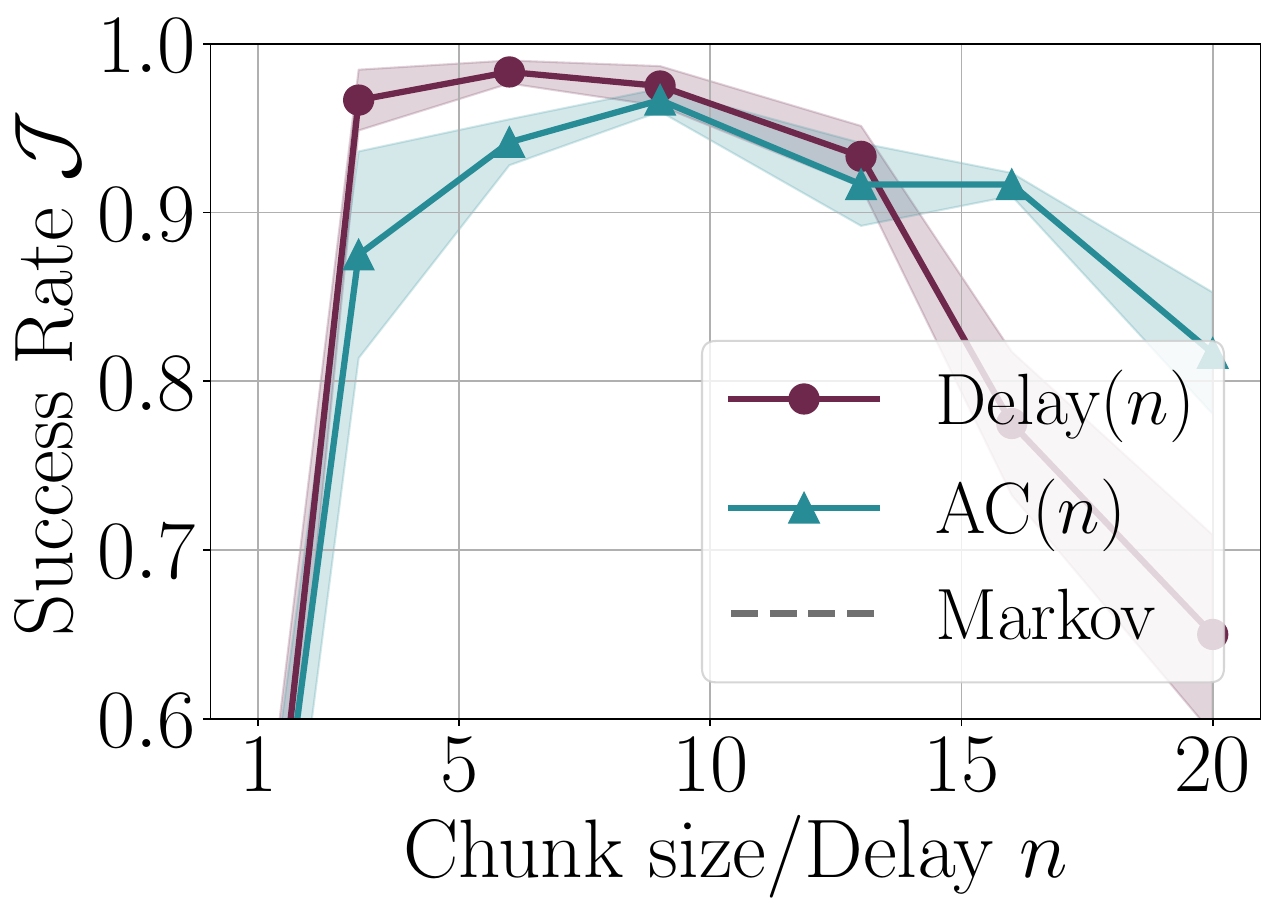}
    \end{minipage}
    \hfill
        \begin{minipage}[t]{0.23\textwidth}
        \centering
        \includegraphics[width=\linewidth]{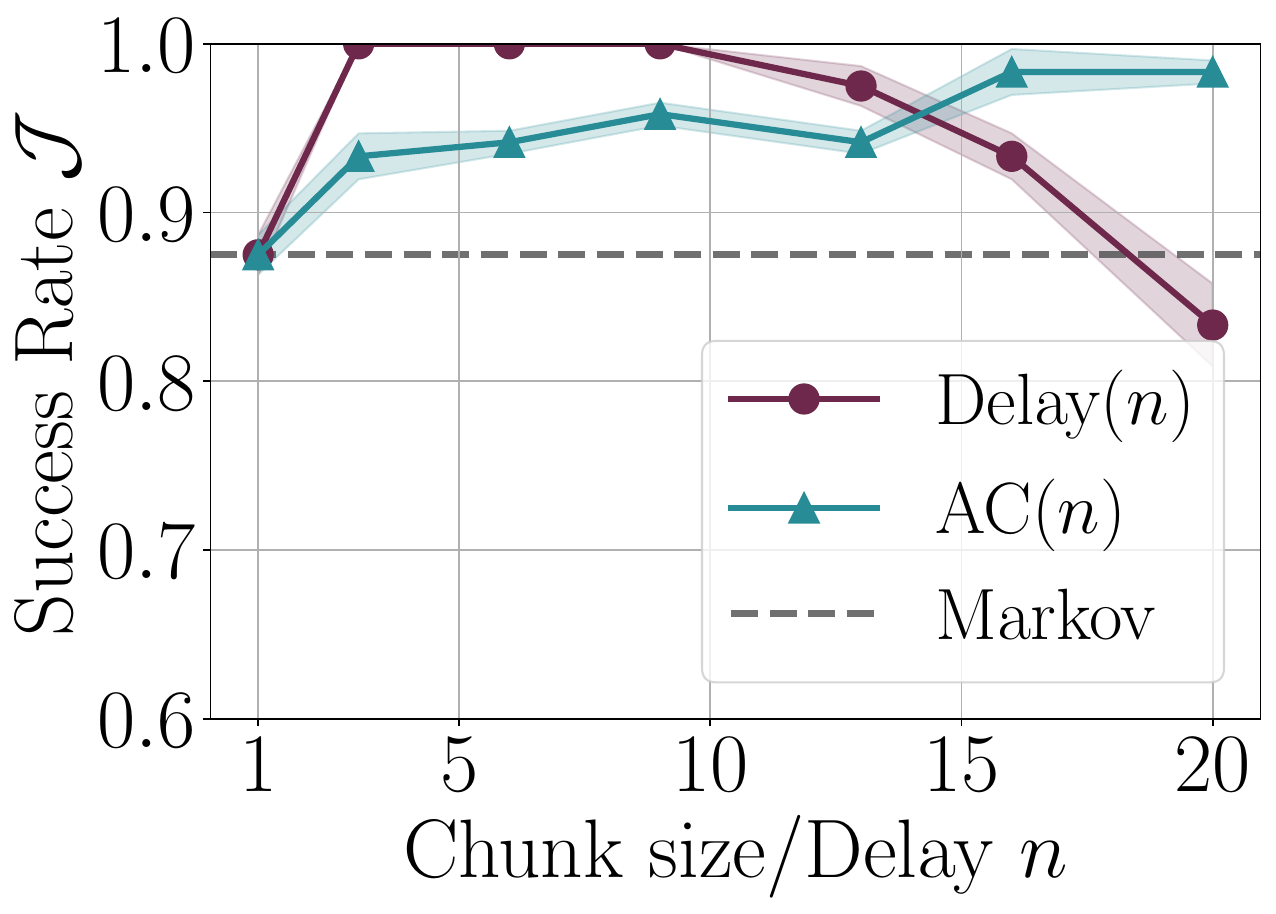}
    \end{minipage}
        \begin{minipage}[t]{0.23\textwidth}
            \includegraphics[width=\linewidth]{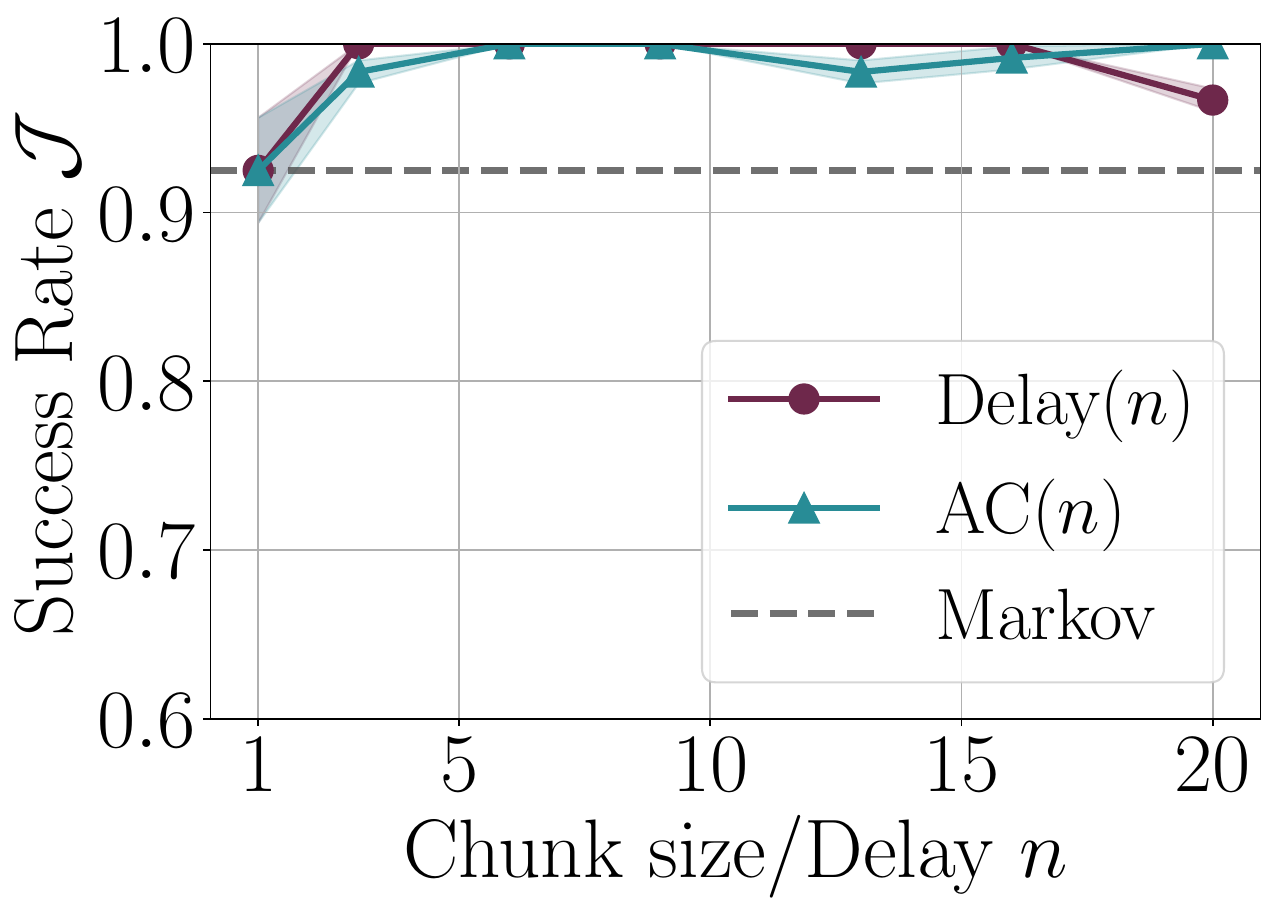}
    \end{minipage}
    \hfill
        \begin{minipage}[t]{0.23\textwidth}
        \centering
        \includegraphics[width=\linewidth]{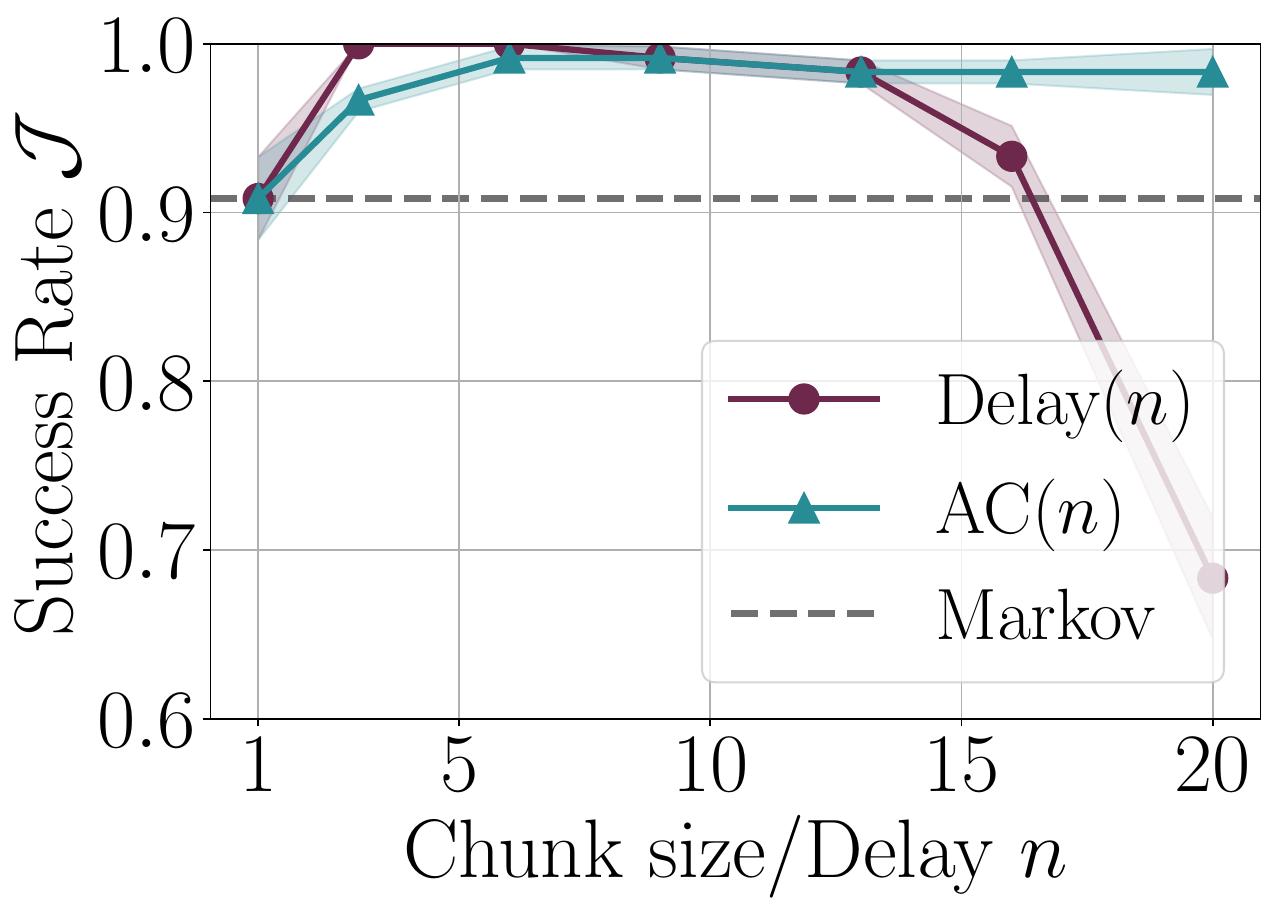}
    \end{minipage}
    \hfill
        \begin{minipage}[t]{0.23\textwidth}
        \centering
        \includegraphics[width=\linewidth]{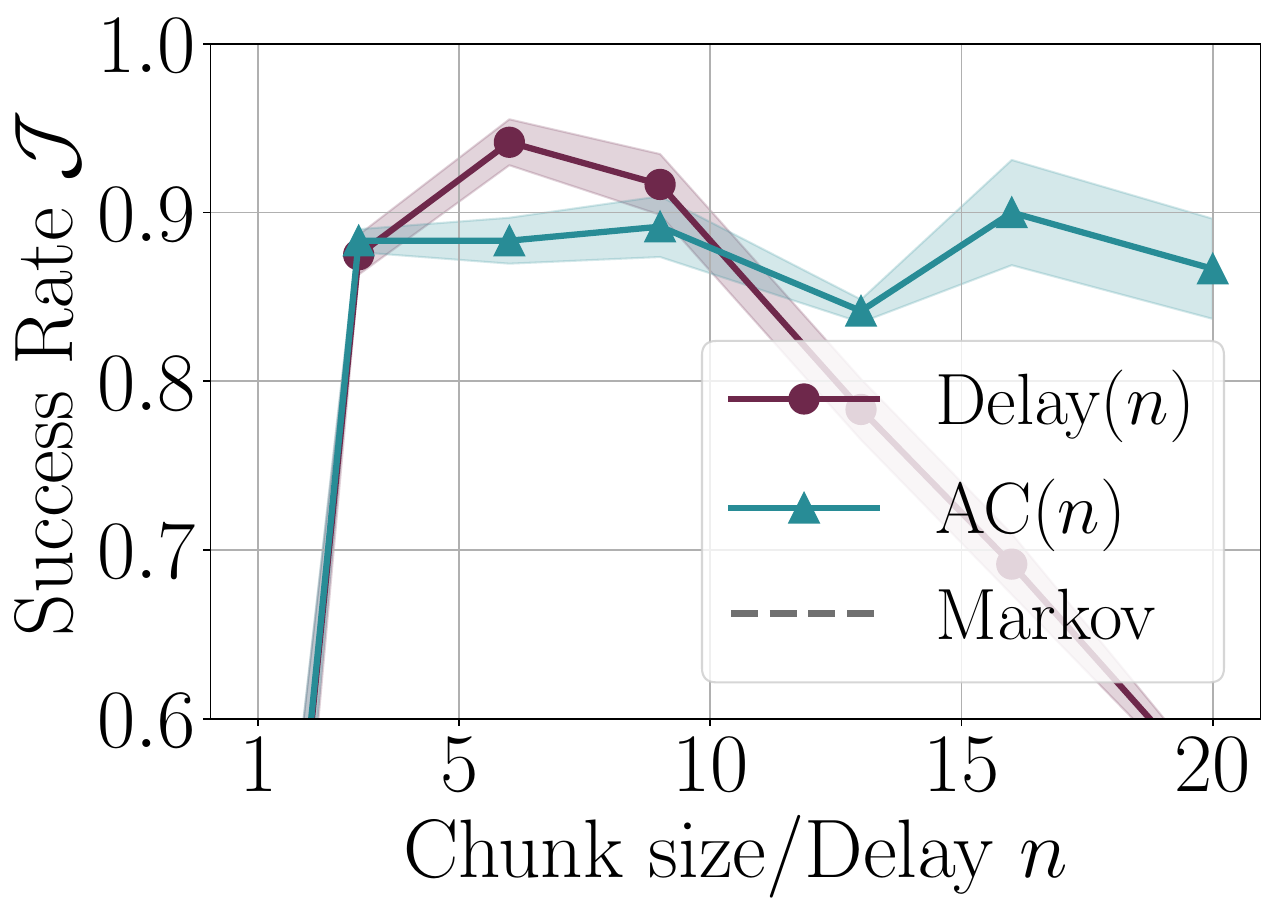}
    \end{minipage}
    \hfill
        \begin{minipage}[t]{0.23\textwidth}
        \centering
        \includegraphics[width=\linewidth]{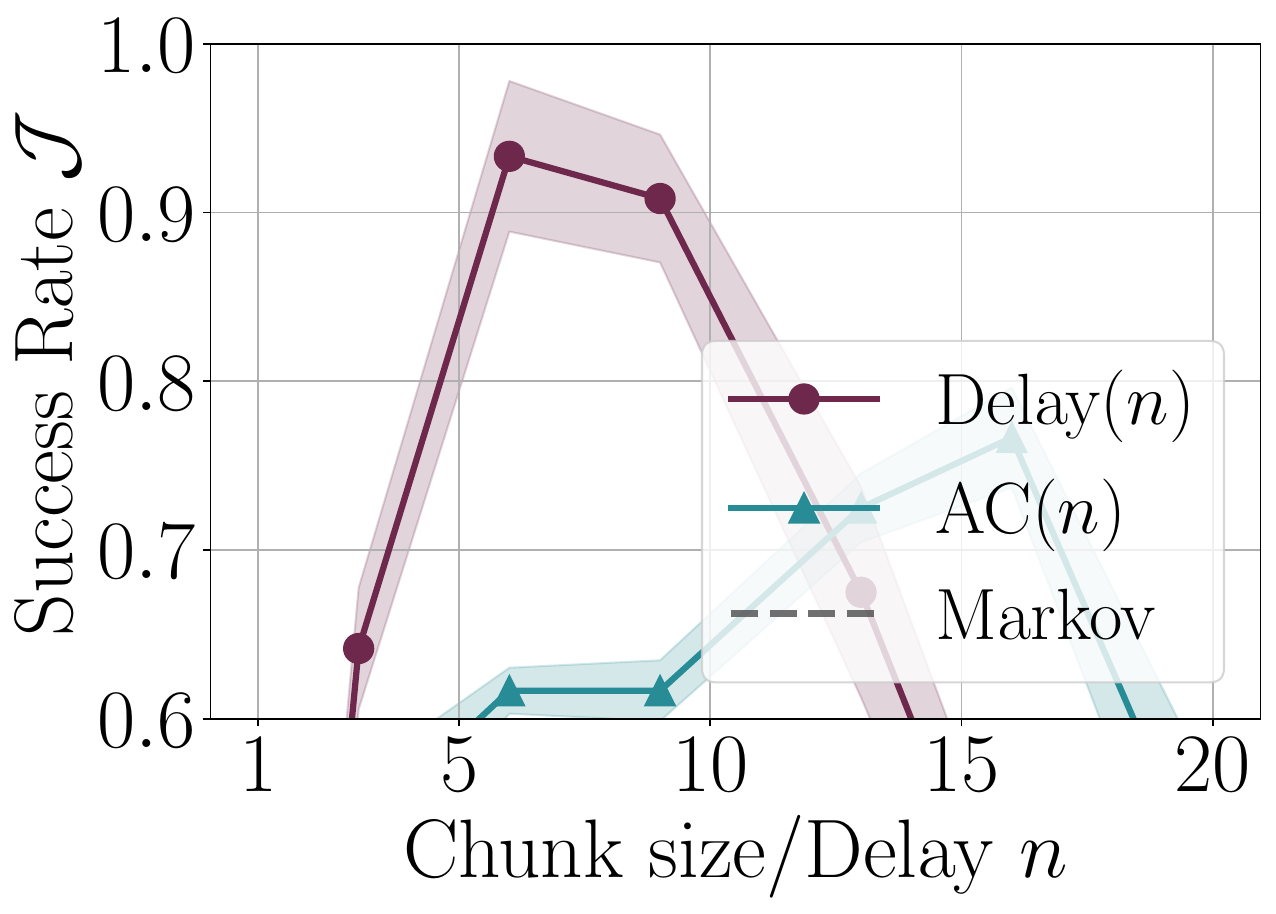}
    \end{minipage}
        \begin{minipage}[t]{0.23\textwidth}
            \includegraphics[width=\linewidth]{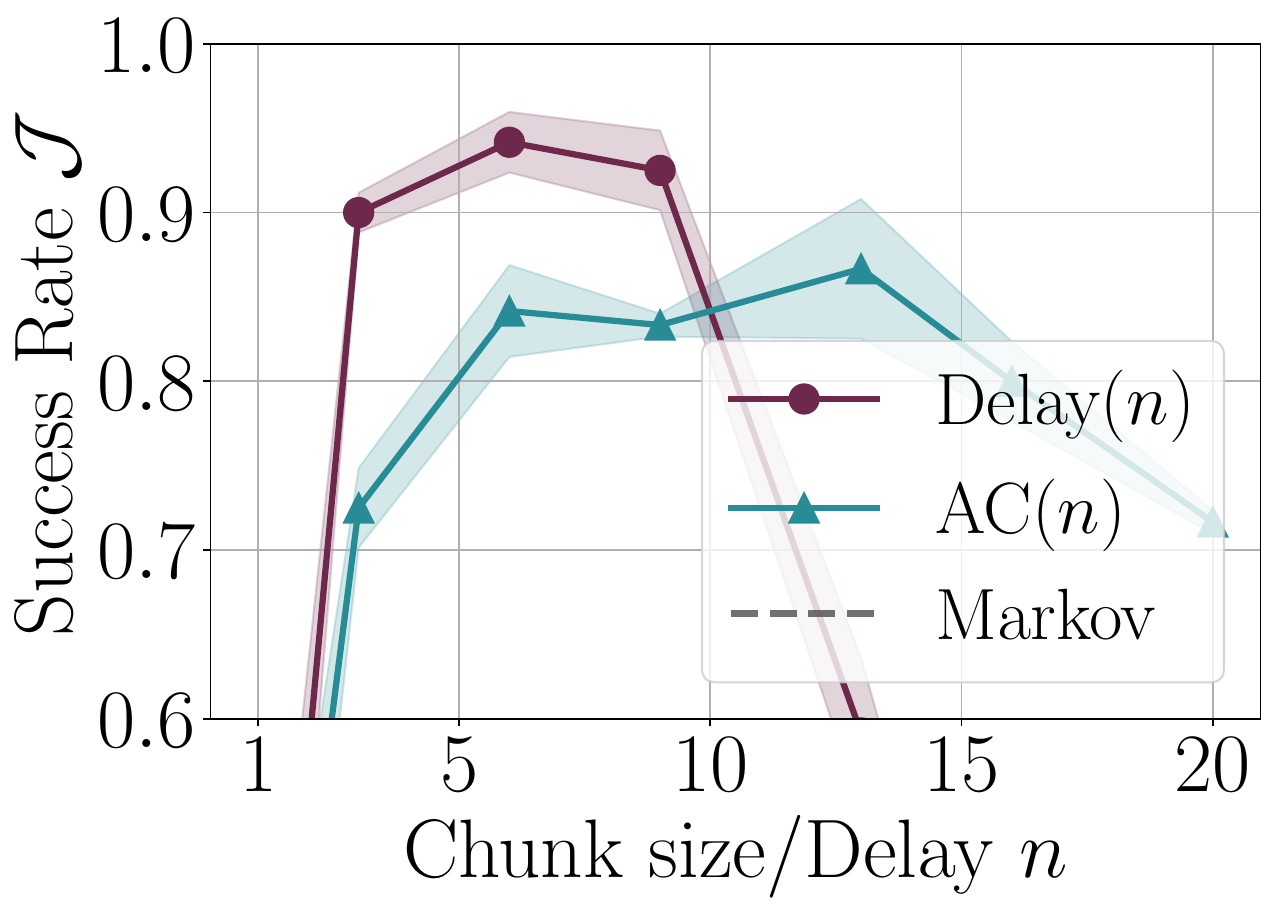}
    \end{minipage}
    \hfill
        \begin{minipage}[t]{0.23\textwidth}
        \centering
        \includegraphics[width=\linewidth]{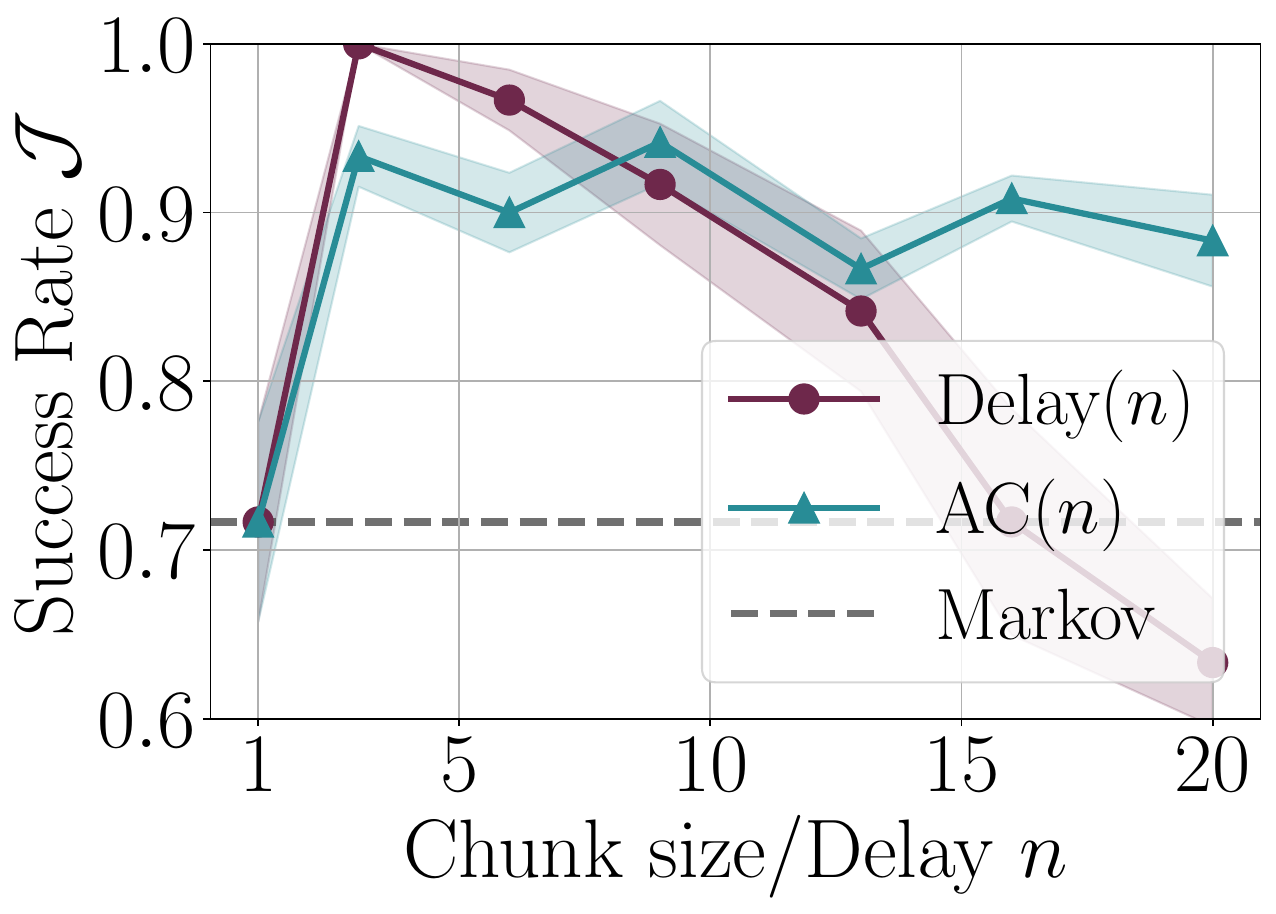}
    \end{minipage}
    \hfill
        \begin{minipage}[t]{0.23\textwidth}
        \centering
        \includegraphics[width=\linewidth]{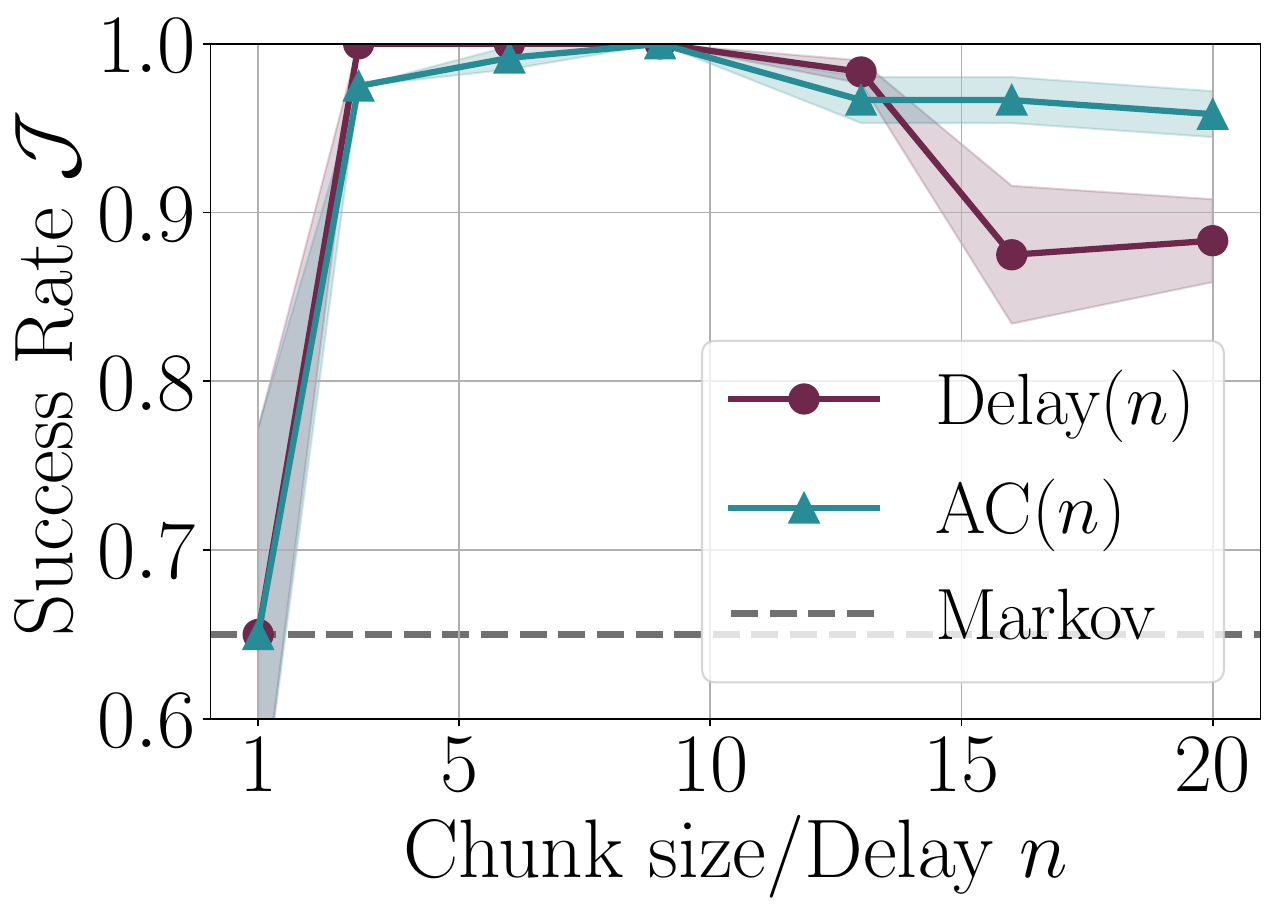}
    \end{minipage}
    \hfill
        \begin{minipage}[t]{0.23\textwidth}
        \centering
        \includegraphics[width=\linewidth]{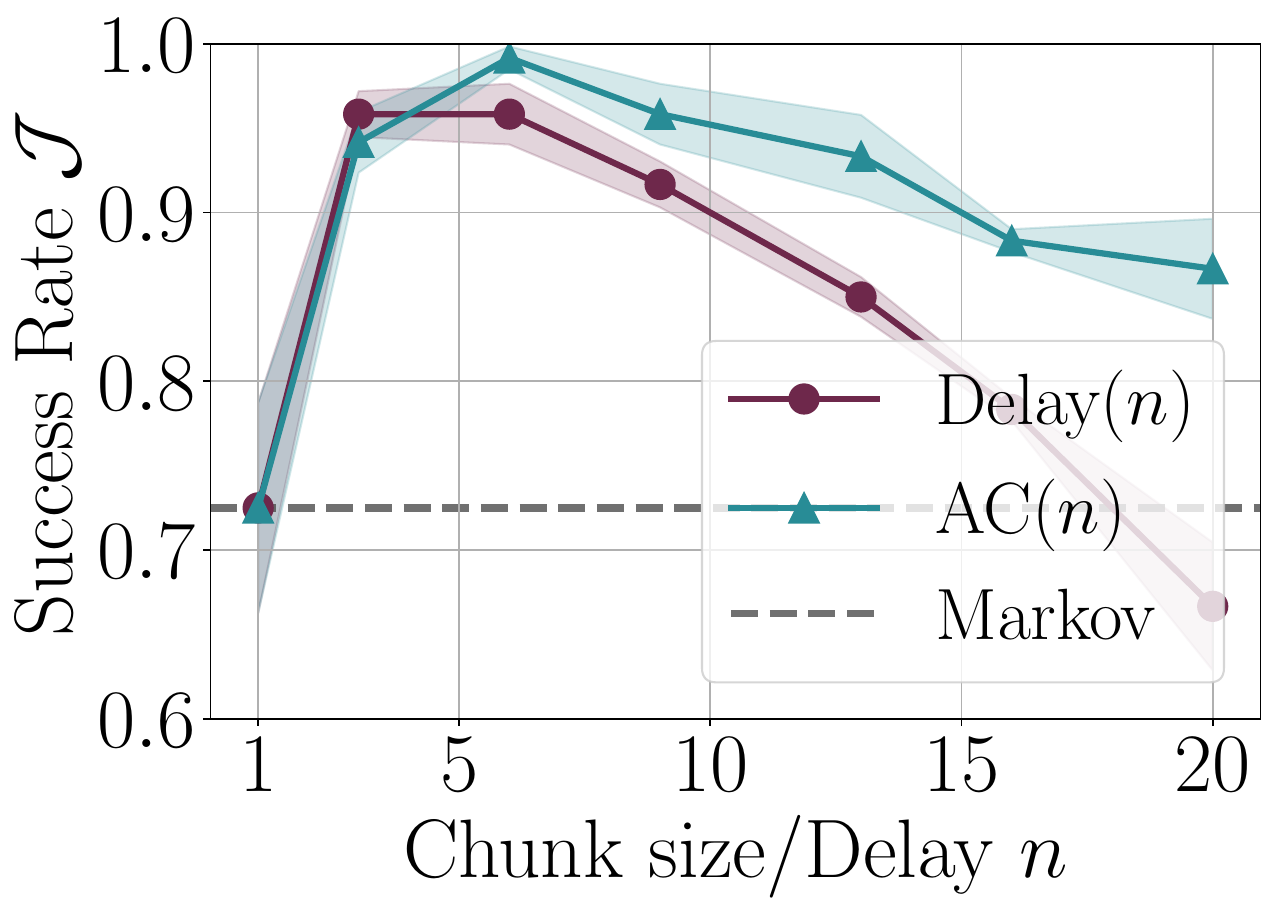}
    \end{minipage}
        \caption{Success rate for each \texttt{Libero} task from 36 to 67 (corresponding to Fig. \ref{fig:success_libero}), part 2.}
    \label{fig:succ each libero2}
\end{figure*}

\begin{figure*}
        \begin{minipage}[t]{0.23\textwidth}
            \includegraphics[width=\linewidth]{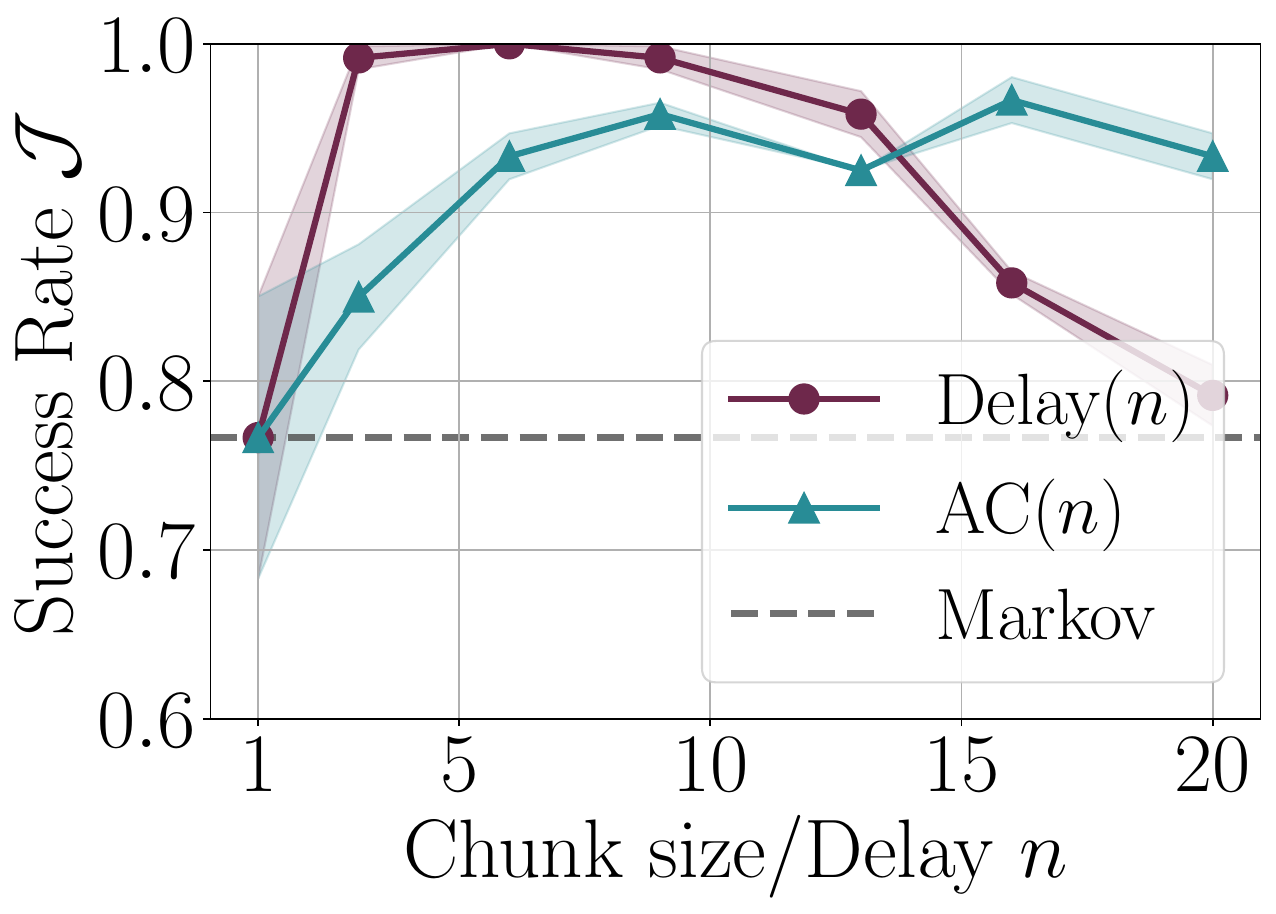}
    \end{minipage}
    \hfill
        \begin{minipage}[t]{0.23\textwidth}
        \centering
        \includegraphics[width=\linewidth]{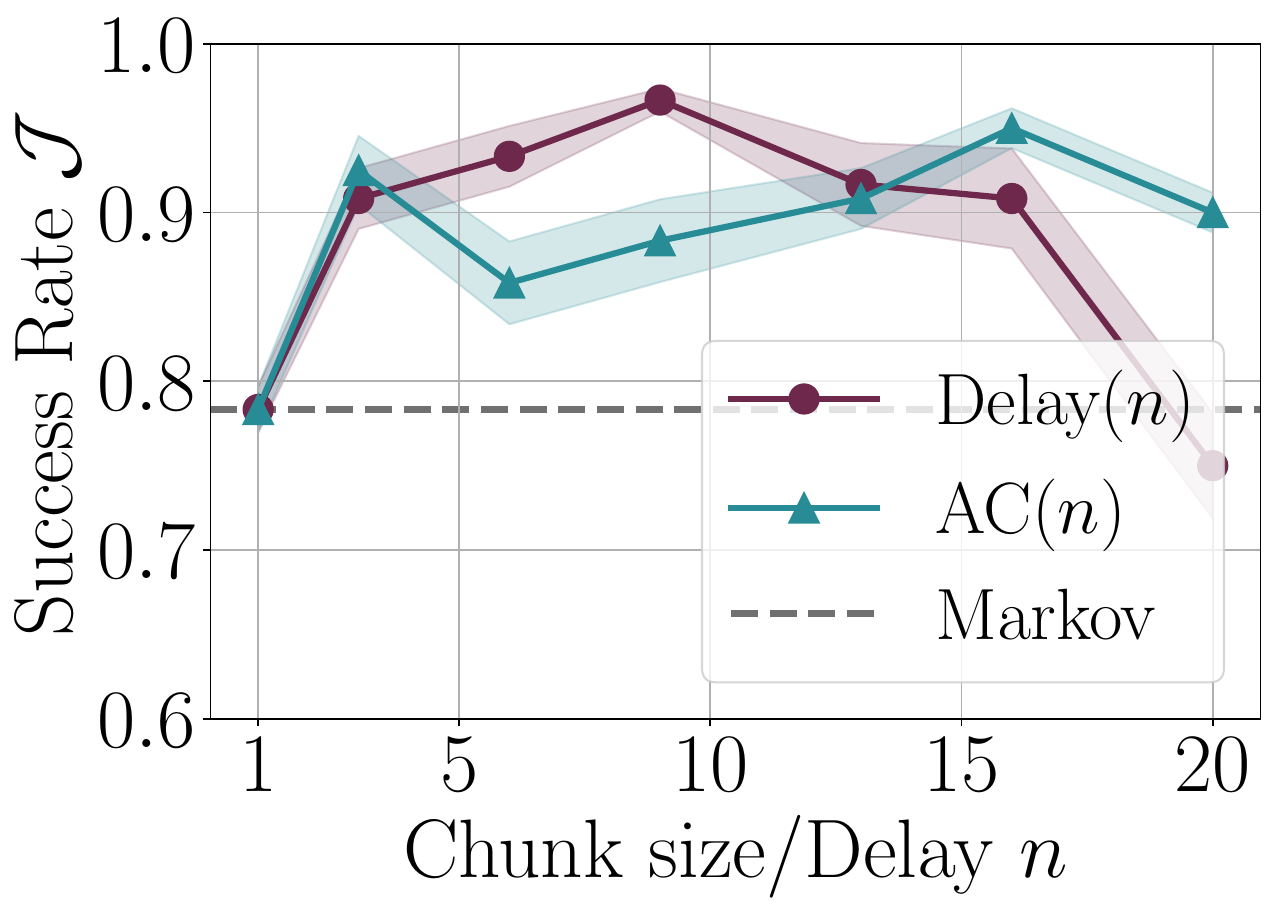}
    \end{minipage}
    \hfill
        \begin{minipage}[t]{0.23\textwidth}
        \centering
        \includegraphics[width=\linewidth]{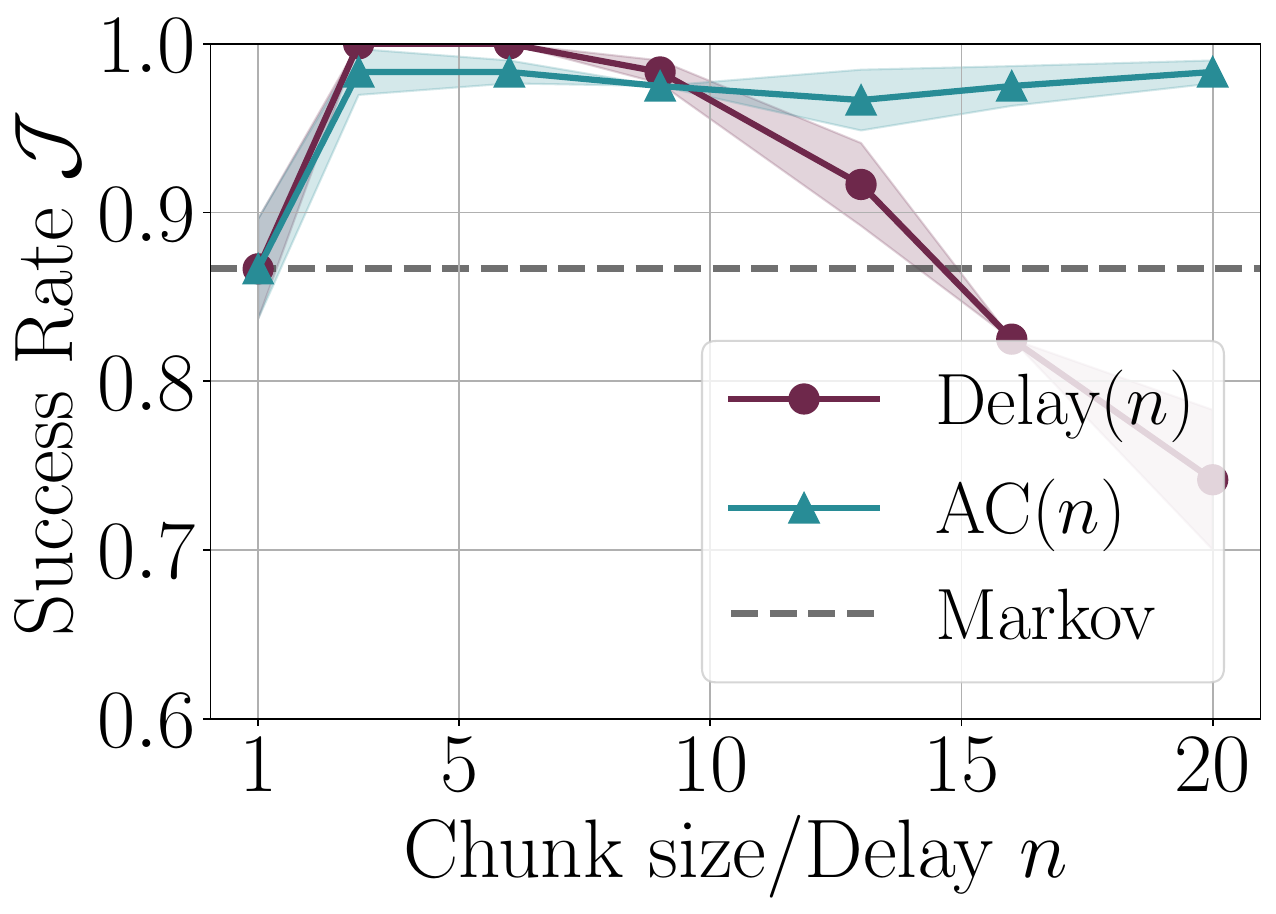}
    \end{minipage}
    \hfill
        \begin{minipage}[t]{0.23\textwidth}
        \centering
        \includegraphics[width=\linewidth]{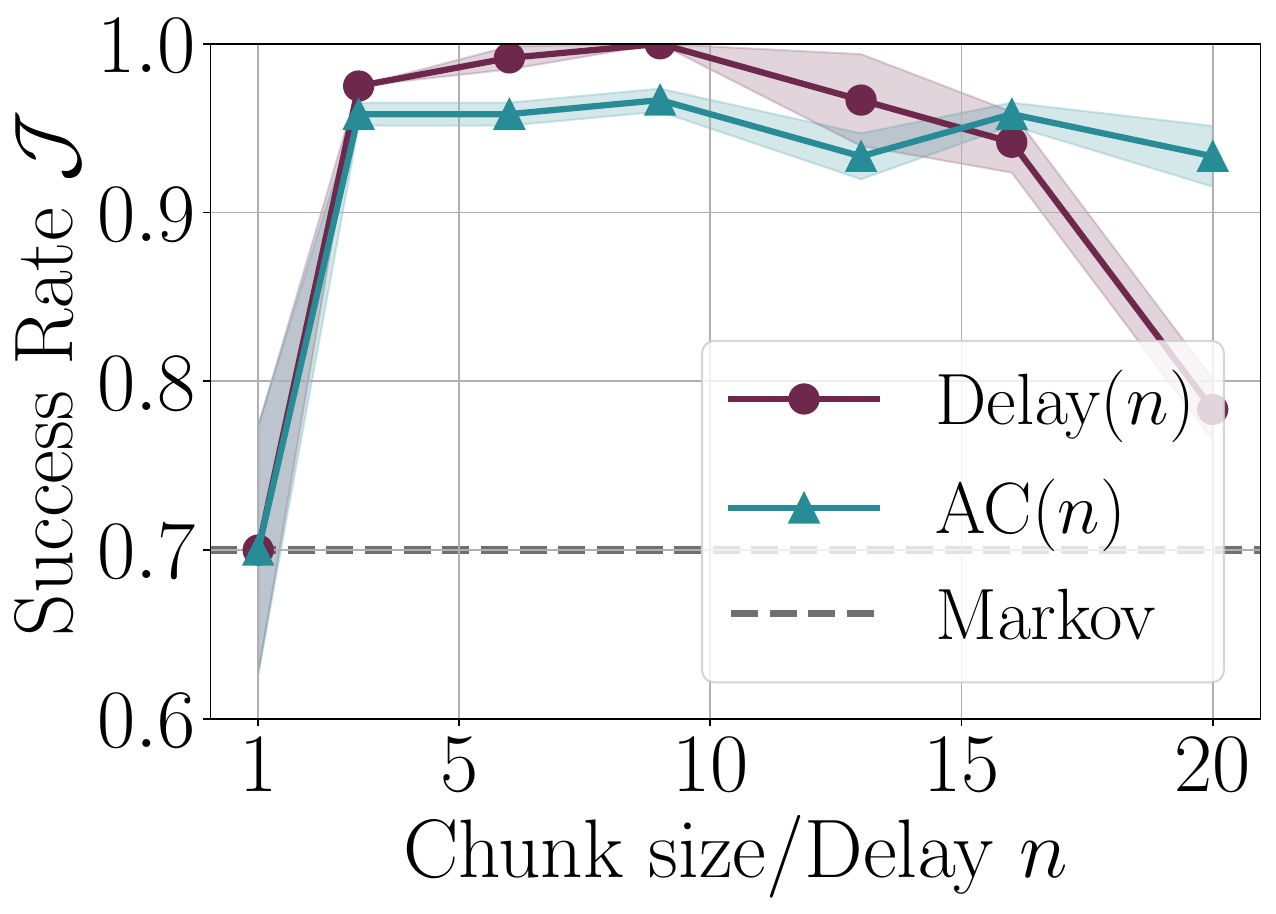}
    \end{minipage}
        \begin{minipage}[t]{0.23\textwidth}
            \includegraphics[width=\linewidth]{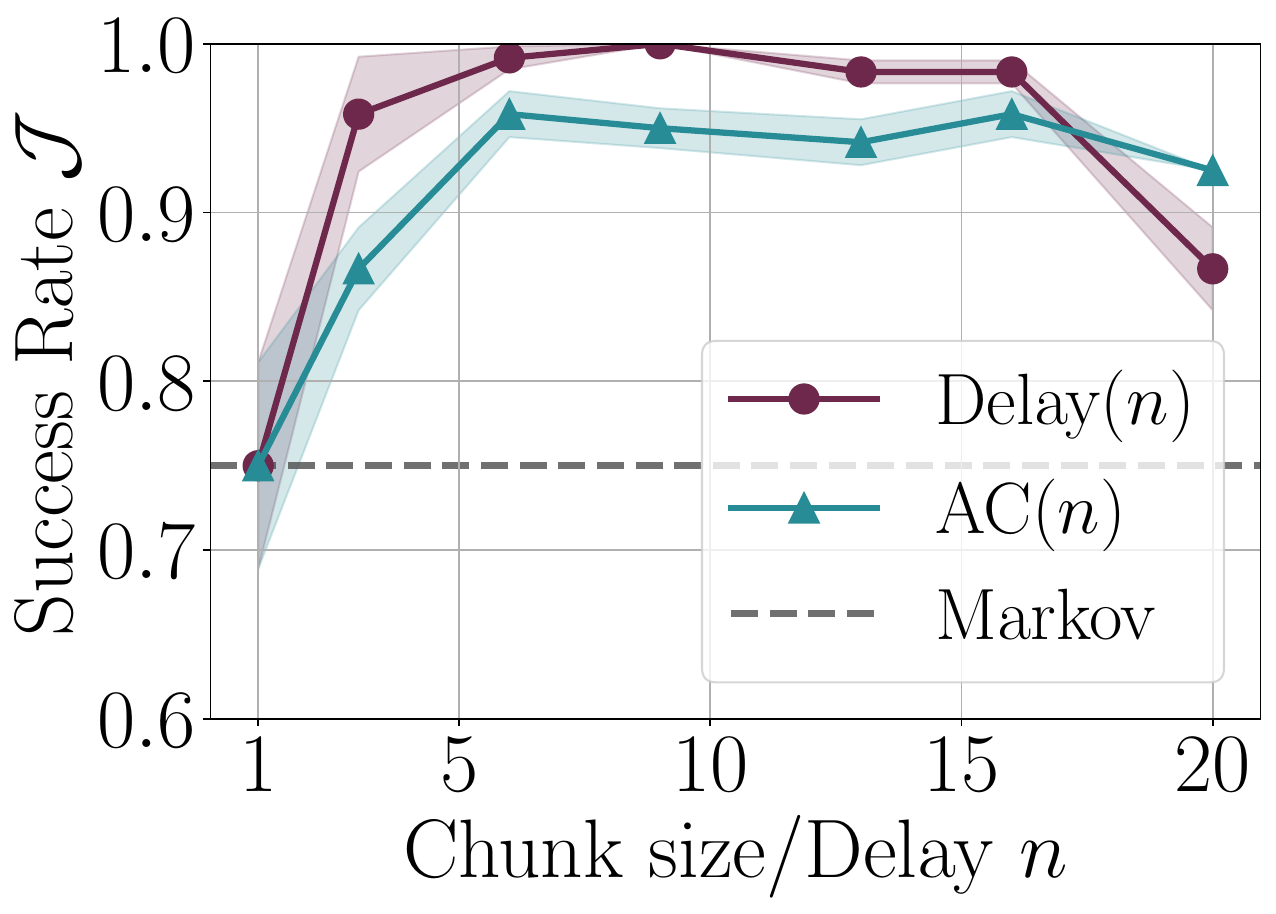}
    \end{minipage}
    \hfill
        \begin{minipage}[t]{0.23\textwidth}
        \centering
        \includegraphics[width=\linewidth]{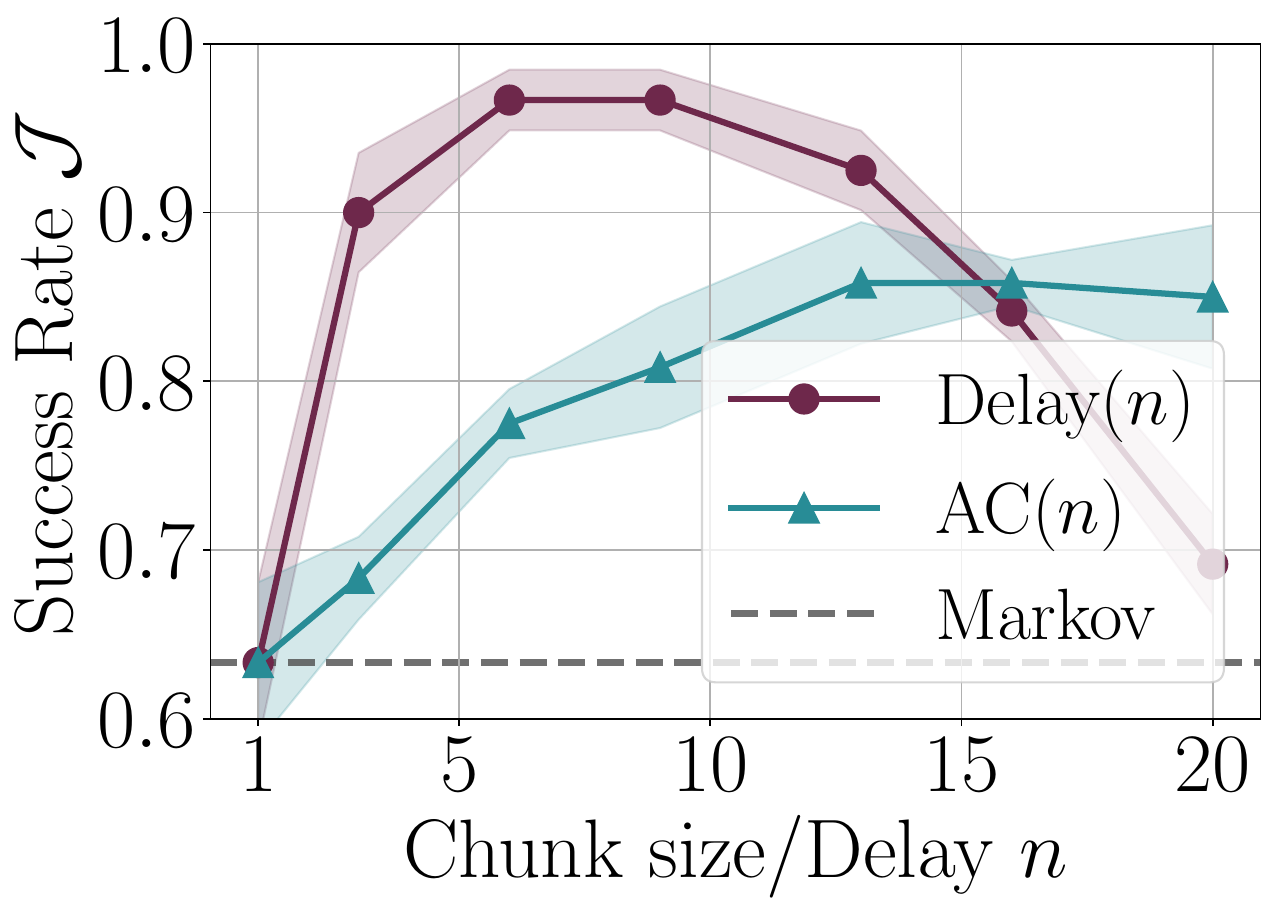}
    \end{minipage}
    \hfill
        \begin{minipage}[t]{0.23\textwidth}
        \centering
        \includegraphics[width=\linewidth]{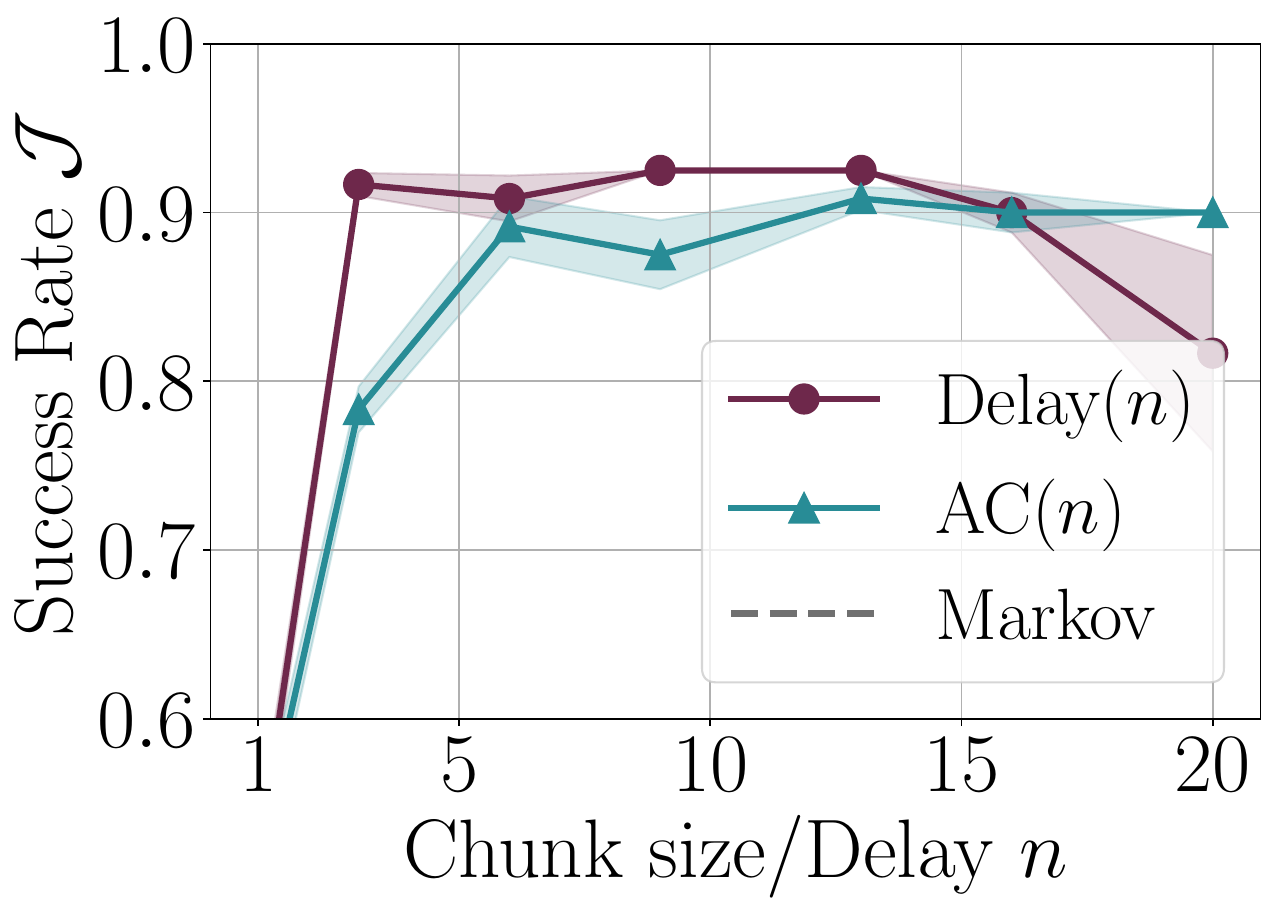}
    \end{minipage}
    \hfill
        \begin{minipage}[t]{0.23\textwidth}
        \centering
        \includegraphics[width=\linewidth]{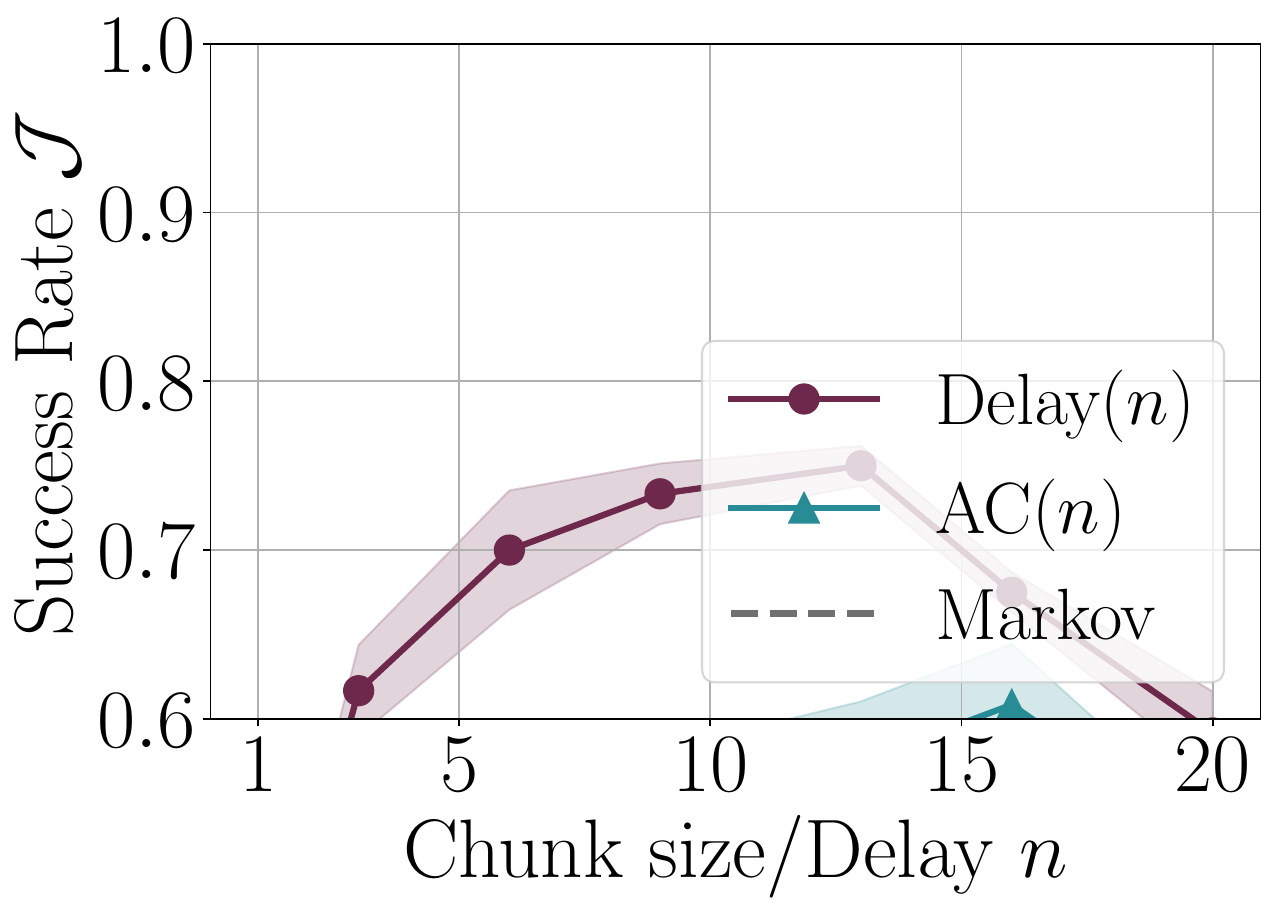}
    \end{minipage}
        \begin{minipage}[t]{0.23\textwidth}
            \includegraphics[width=\linewidth]{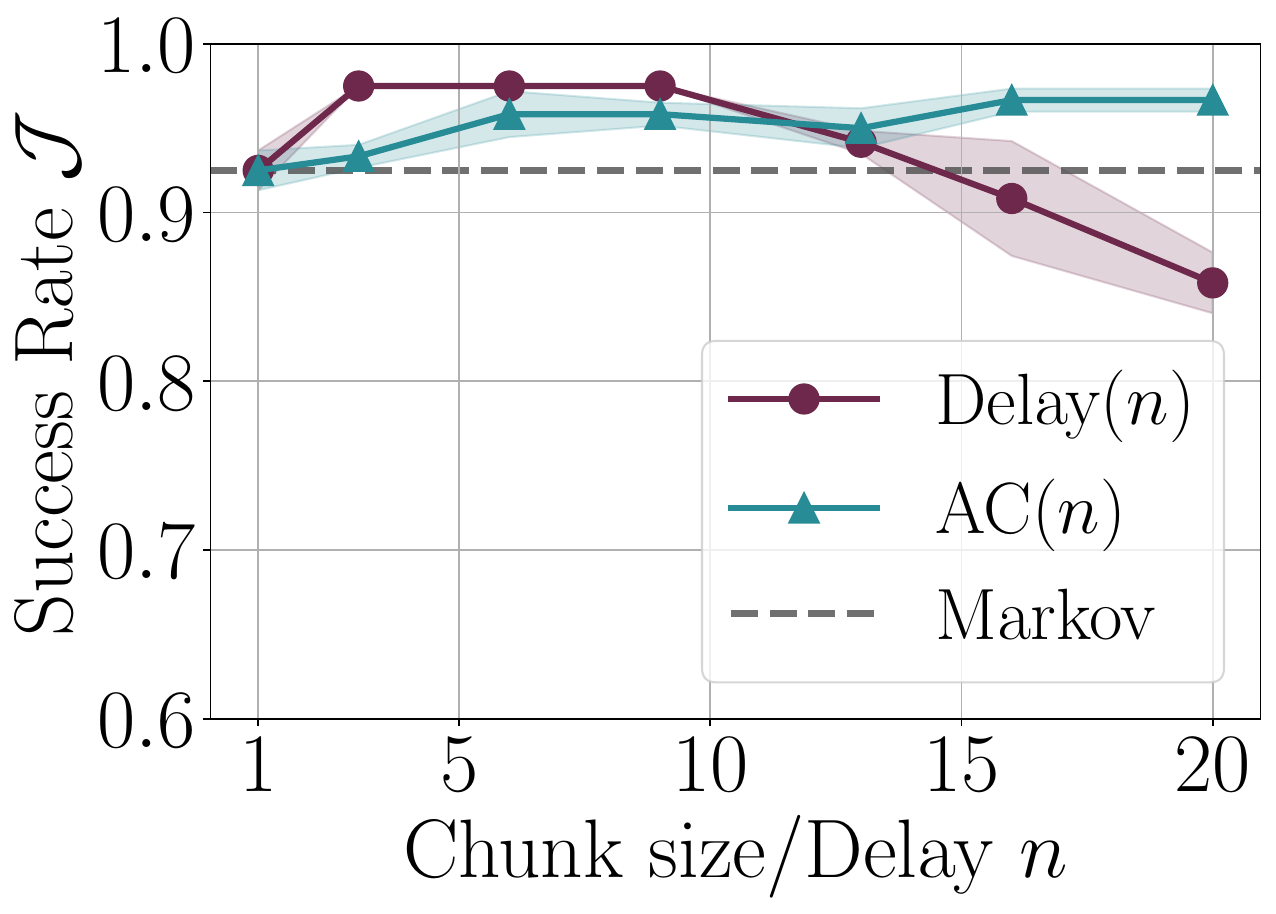}
    \end{minipage}
    \hfill
        \begin{minipage}[t]{0.23\textwidth}
        \centering
        \includegraphics[width=\linewidth]{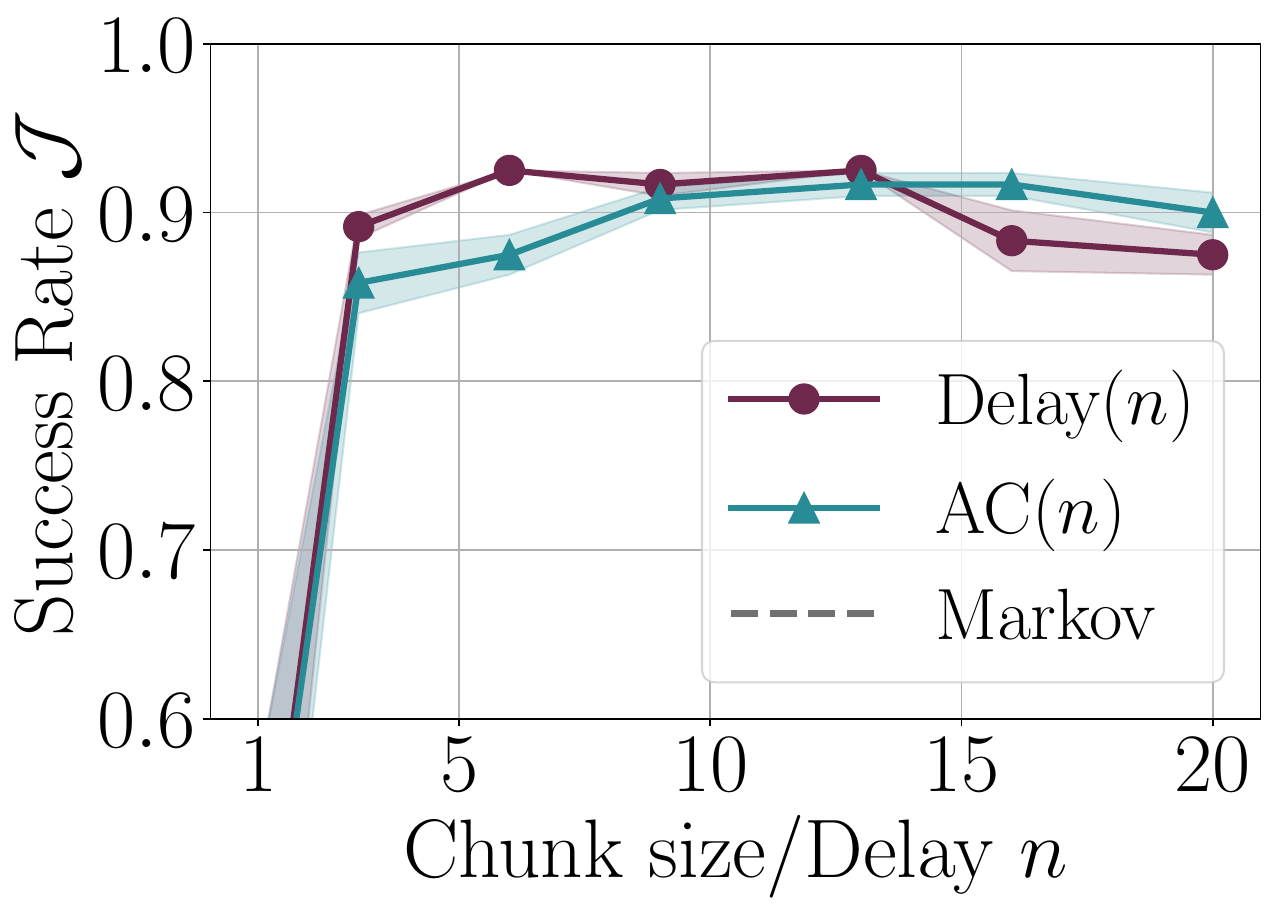}
    \end{minipage}
    \hfill
        \begin{minipage}[t]{0.23\textwidth}
        \centering
        \includegraphics[width=\linewidth]{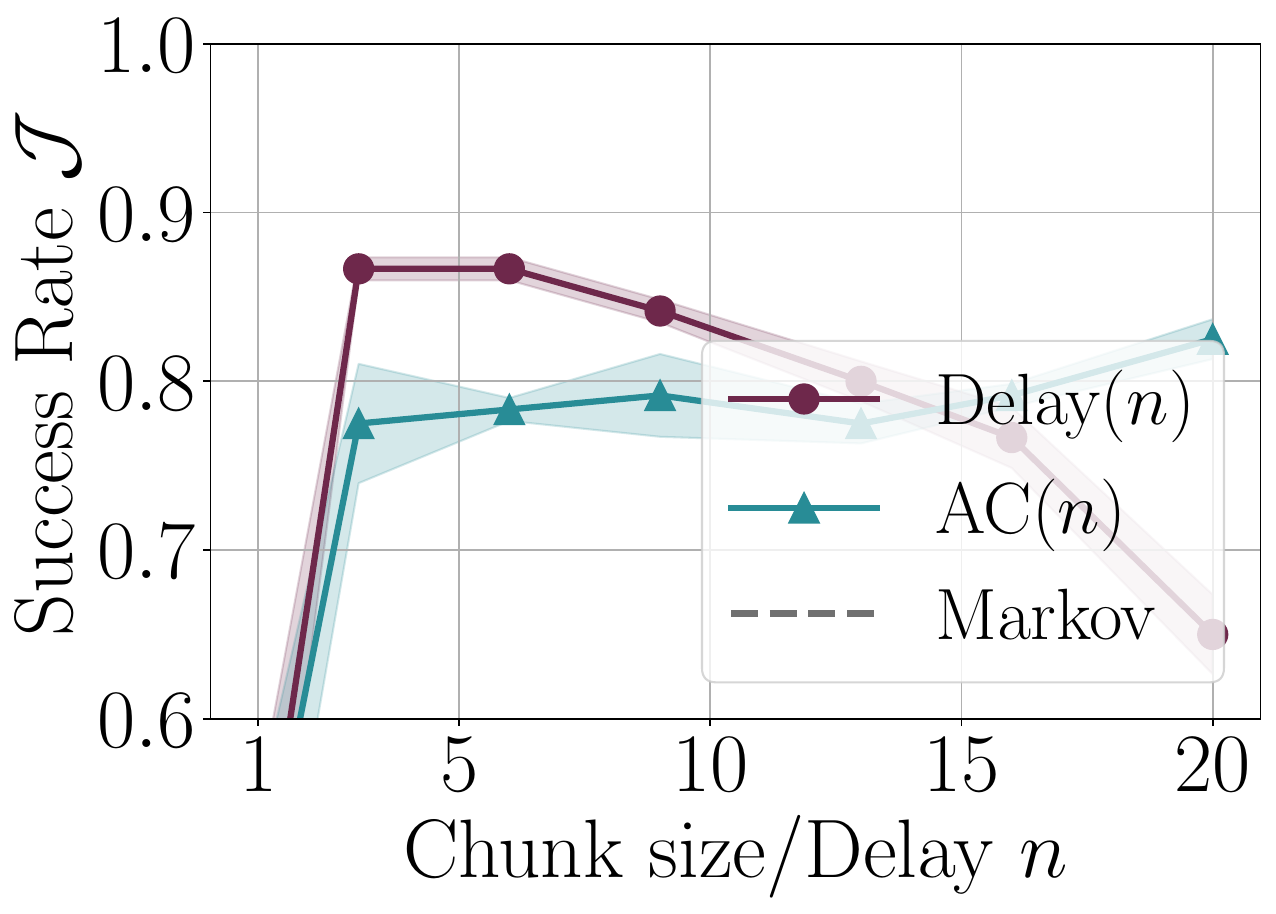}
    \end{minipage}
    \hfill
        \begin{minipage}[t]{0.23\textwidth}
        \centering
        \includegraphics[width=\linewidth]{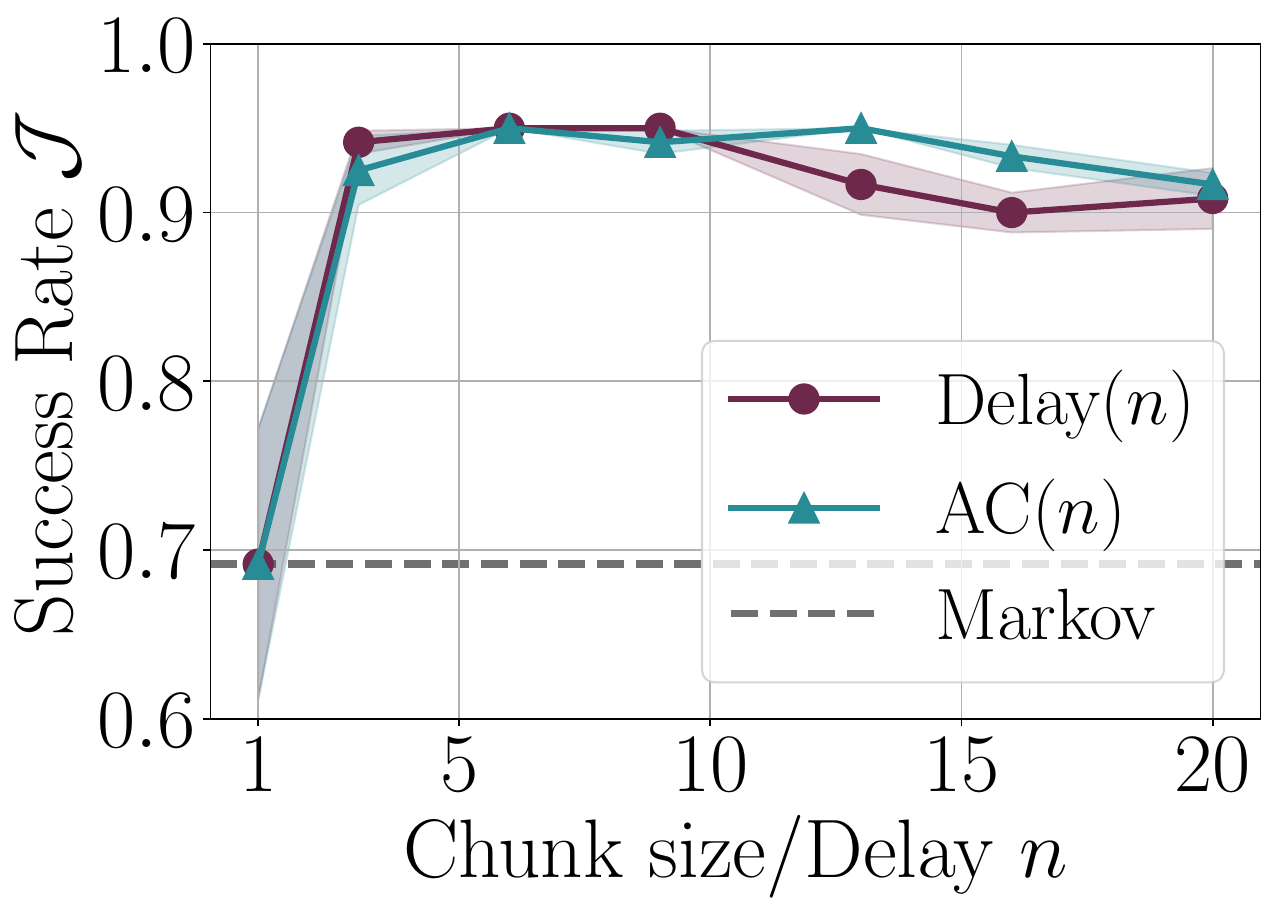}
    \end{minipage}
        \begin{minipage}[t]{0.23\textwidth}
            \includegraphics[width=\linewidth]{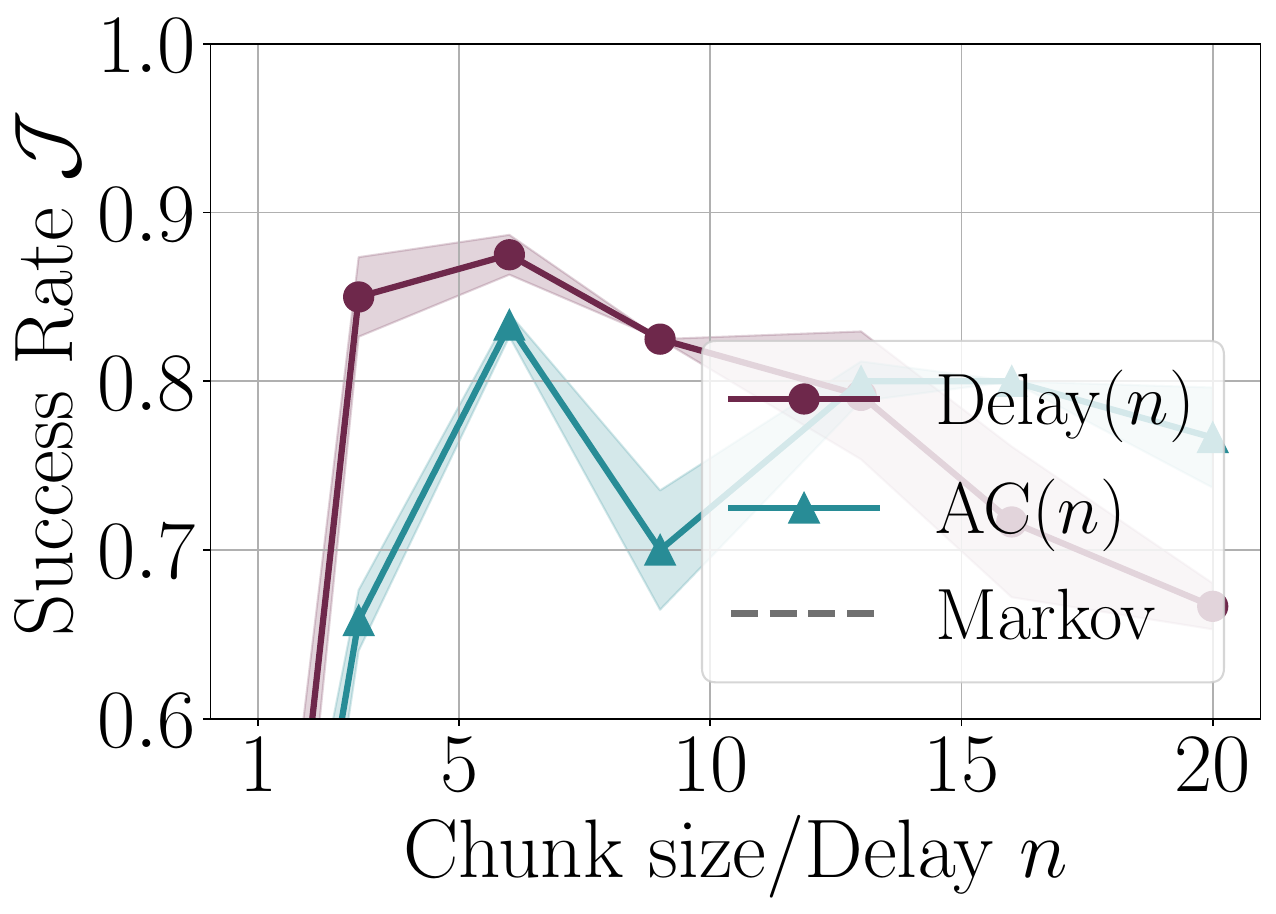}
    \end{minipage}
    \hfill
        \begin{minipage}[t]{0.23\textwidth}
        \centering
        \includegraphics[width=\linewidth]{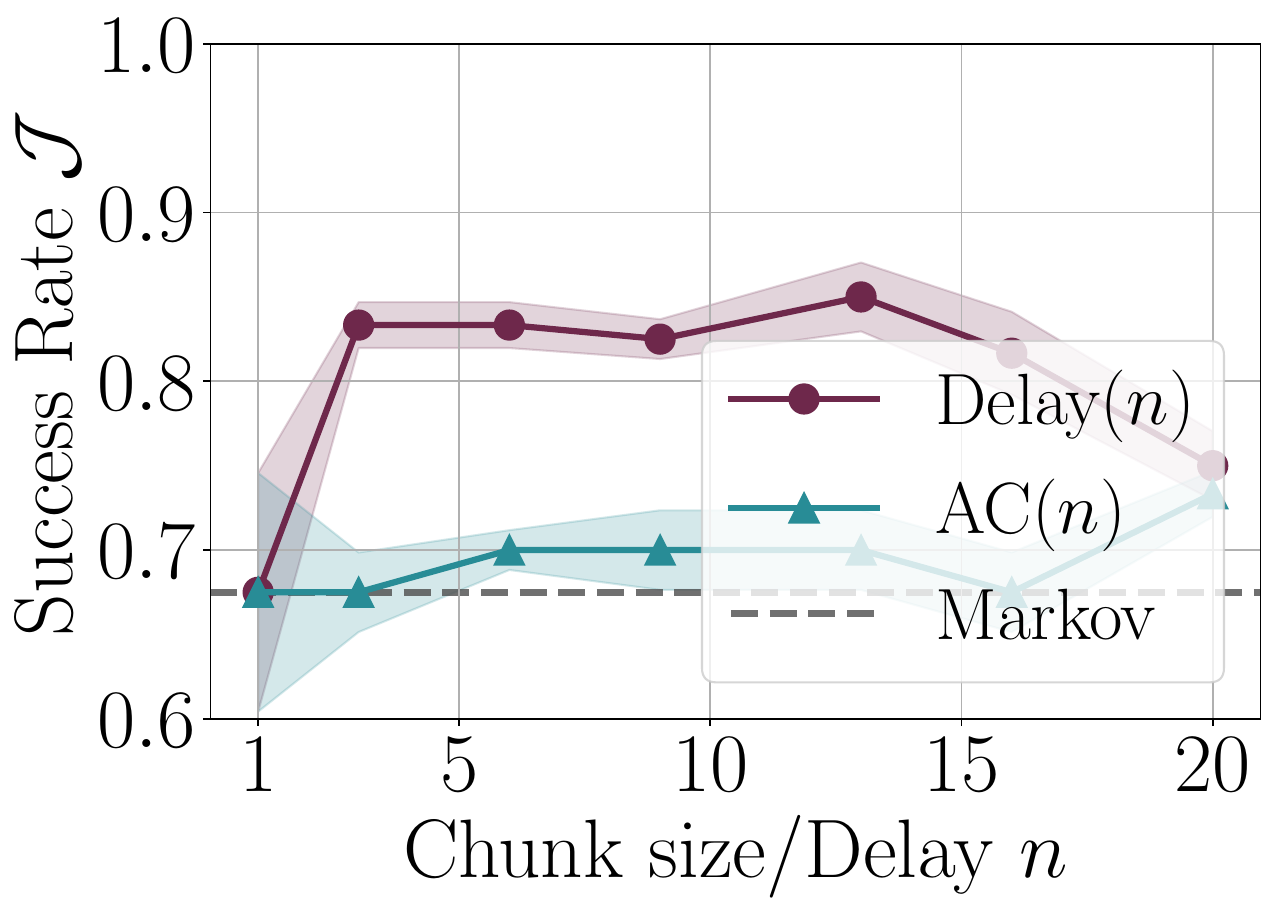}
    \end{minipage}
    \hfill
        \begin{minipage}[t]{0.23\textwidth}
        \centering
        \includegraphics[width=\linewidth]{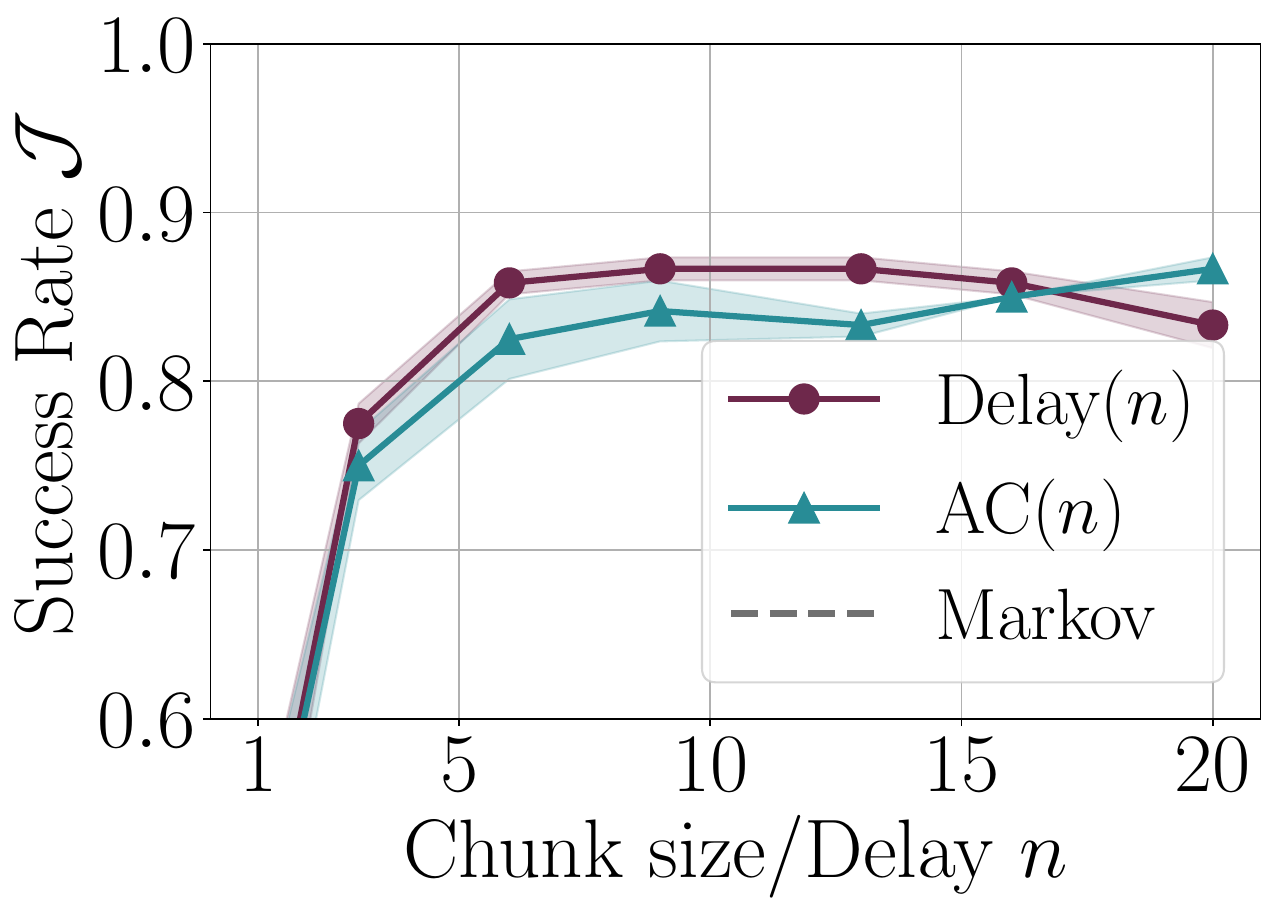}
    \end{minipage}
    \hfill
        \begin{minipage}[t]{0.23\textwidth}
        \centering
        \includegraphics[width=\linewidth]{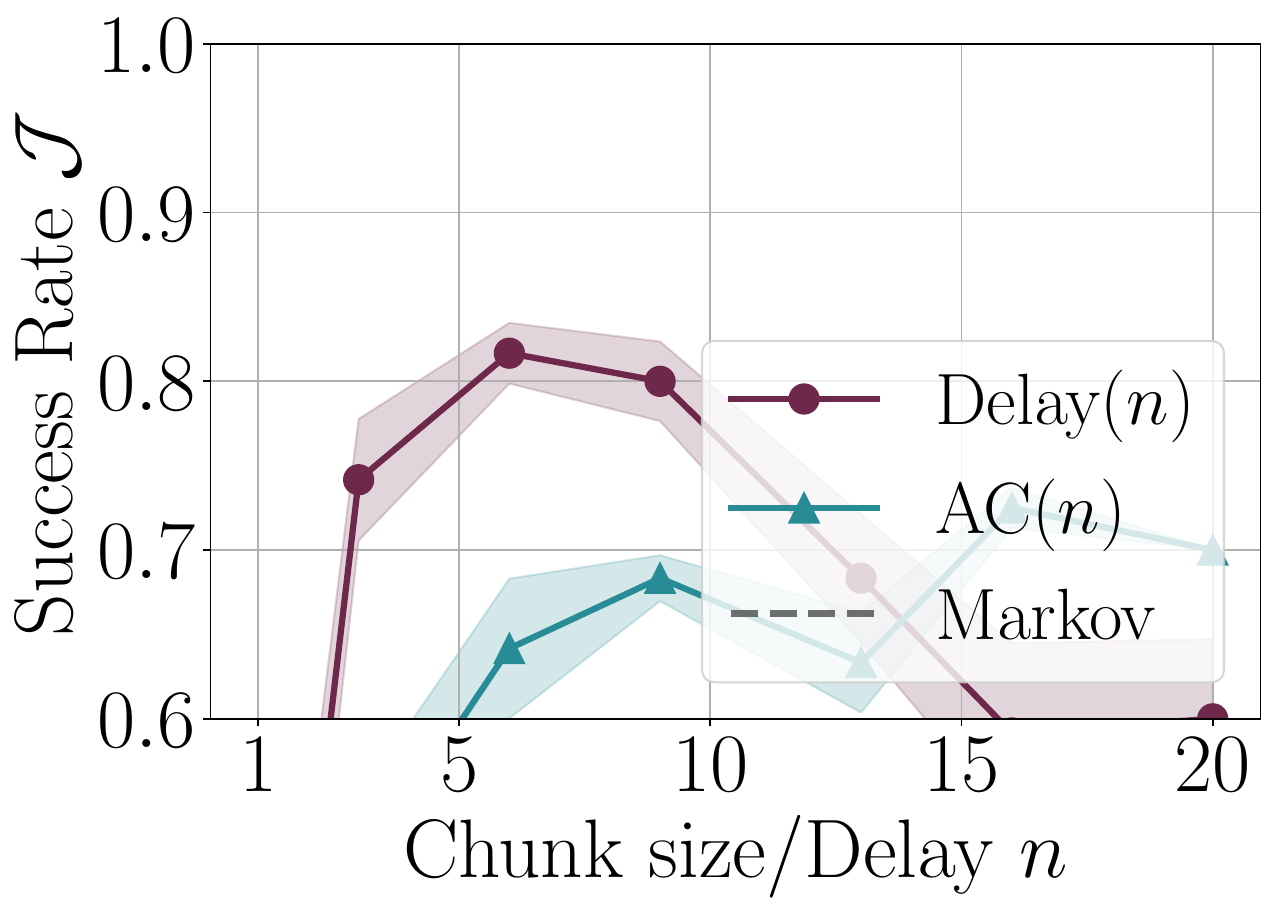}
    \end{minipage}
        \begin{minipage}[t]{0.23\textwidth}
            \includegraphics[width=\linewidth]{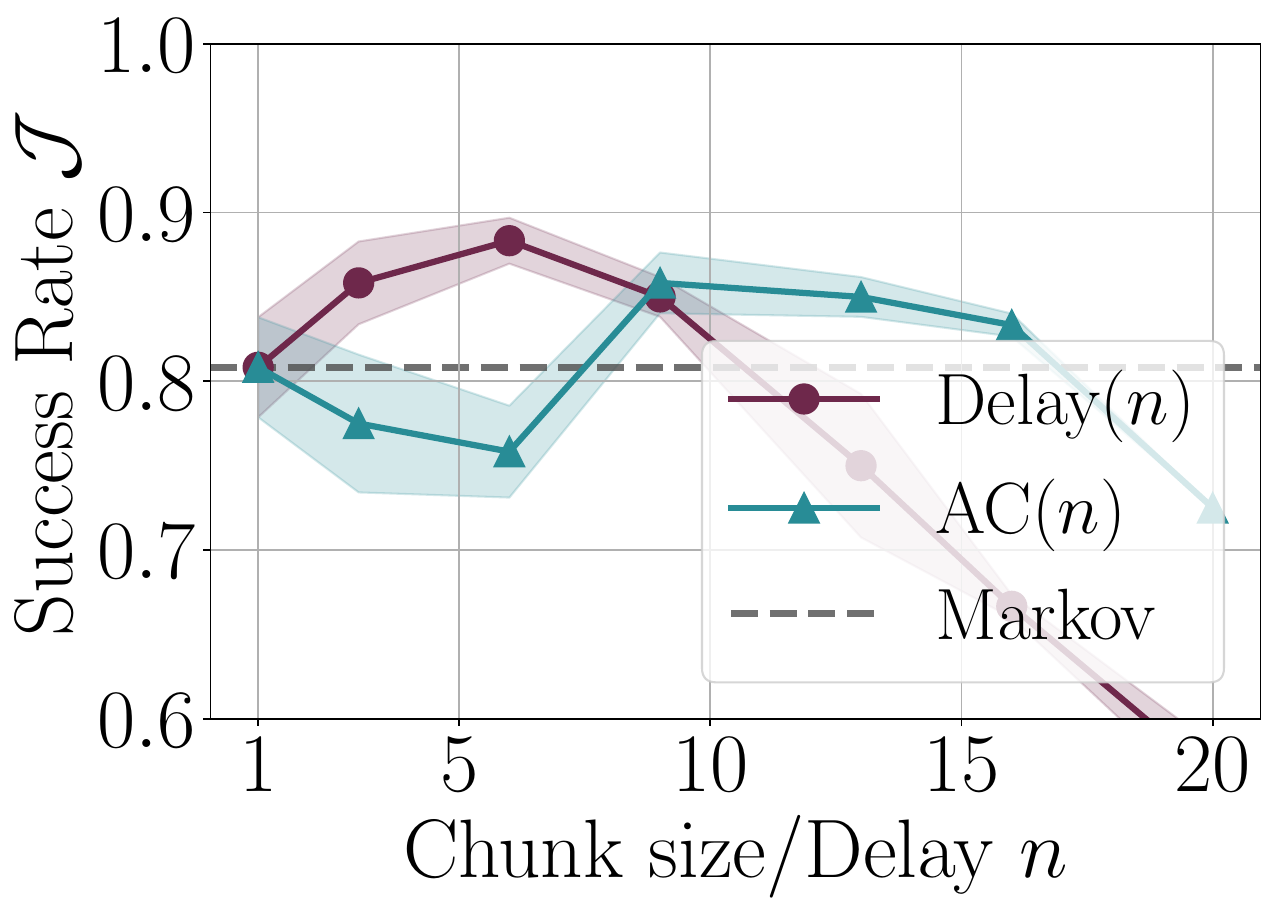}
    \end{minipage}
    \hfill
        \begin{minipage}[t]{0.23\textwidth}
        \centering
        \includegraphics[width=\linewidth]{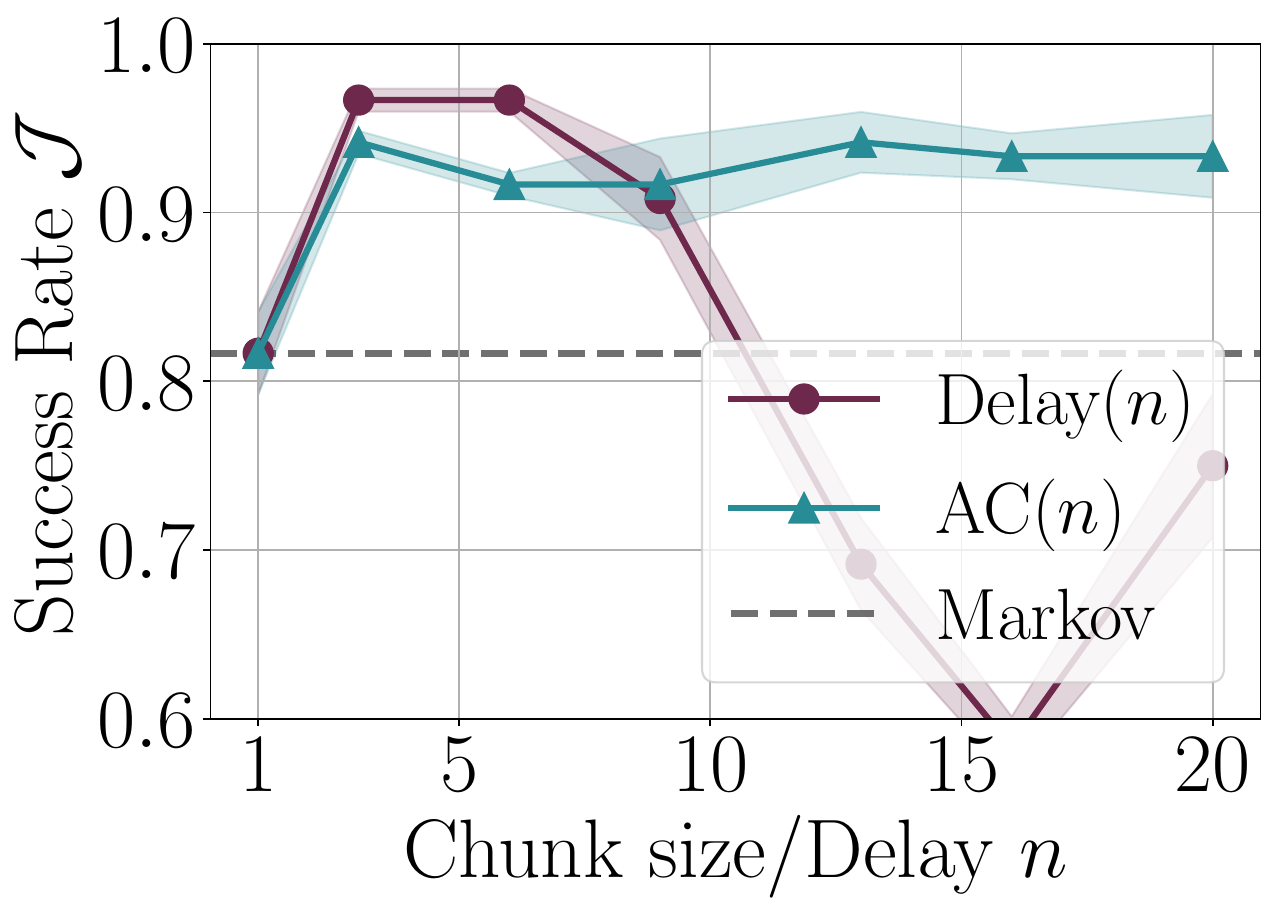}
    \end{minipage}
    \hfill
        \begin{minipage}[t]{0.23\textwidth}
        \centering
        \includegraphics[width=\linewidth]{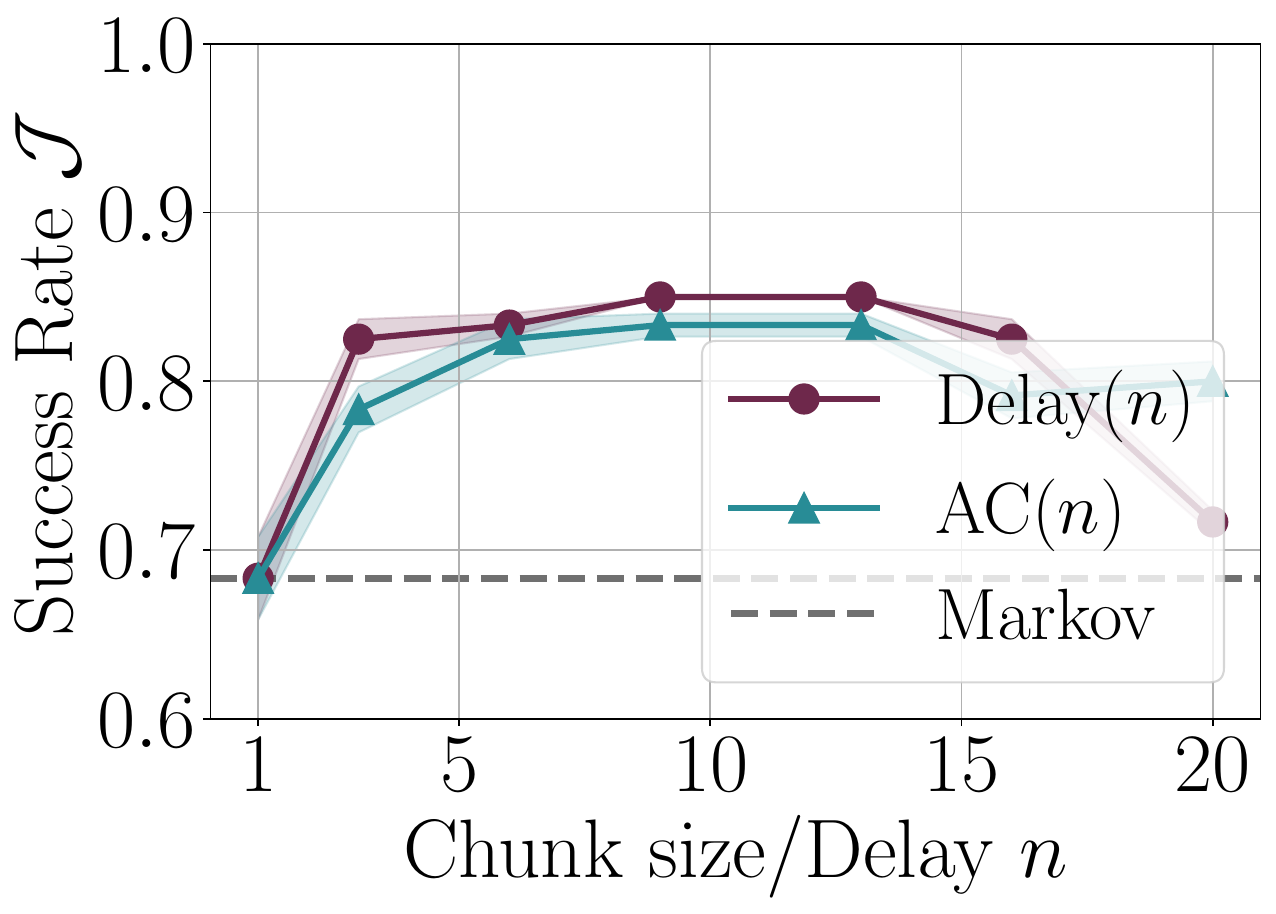}
    \end{minipage}
    \hfill
        \begin{minipage}[t]{0.23\textwidth}
        \centering
        \includegraphics[width=\linewidth]{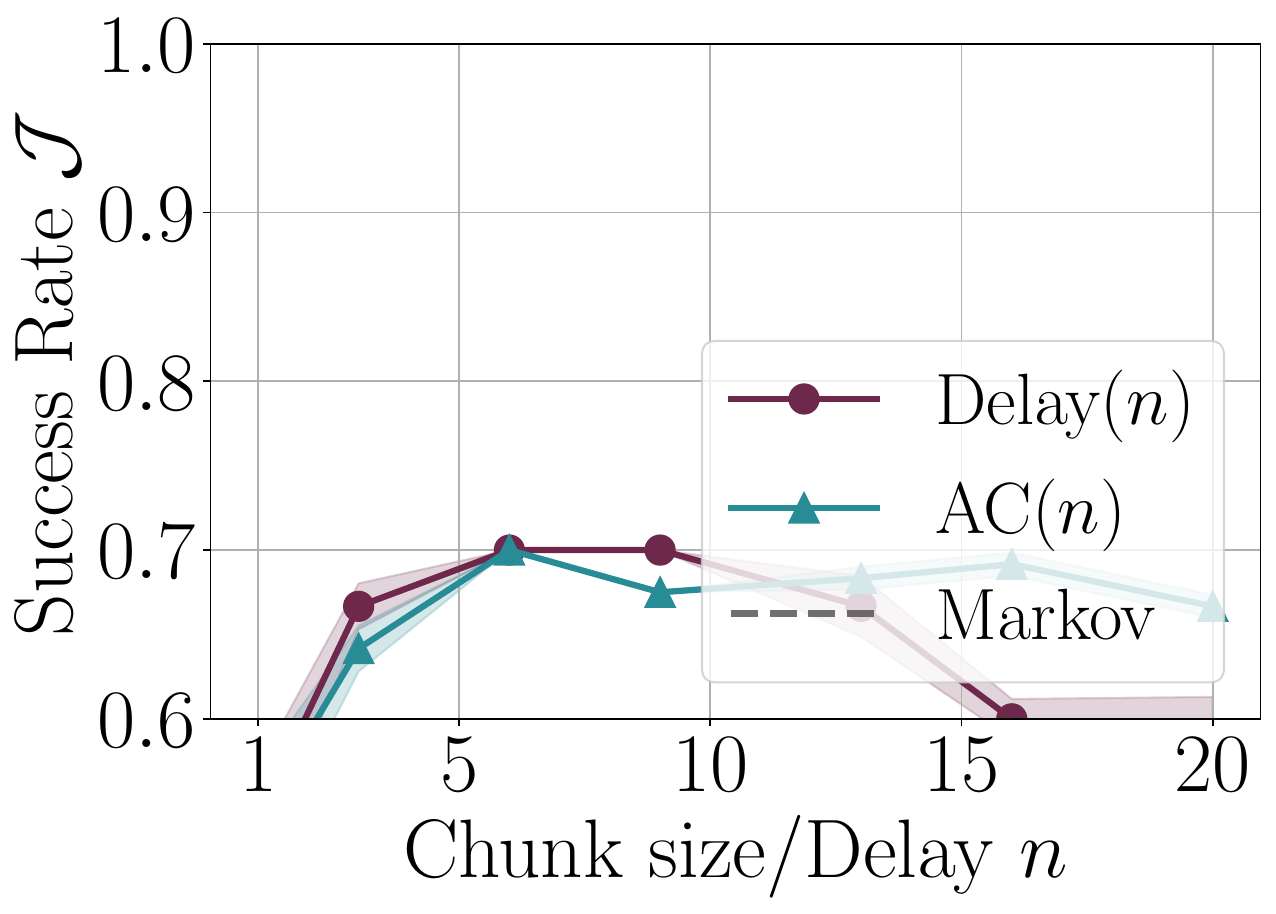}
    \end{minipage}
        \begin{minipage}[t]{0.23\textwidth}
            \includegraphics[width=\linewidth]{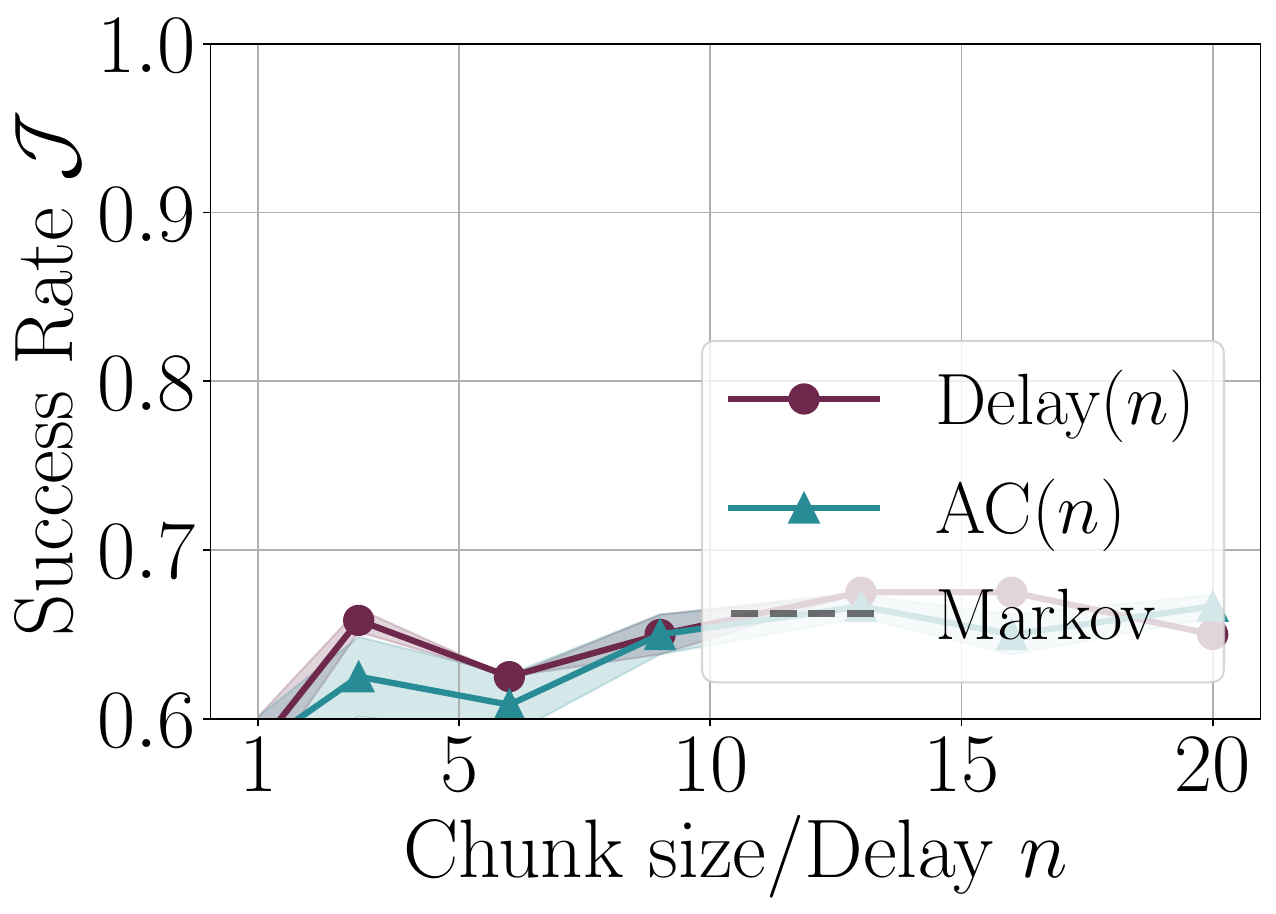}
    \end{minipage}
    \hfill
        \begin{minipage}[t]{0.23\textwidth}
        \centering
        \includegraphics[width=\linewidth]{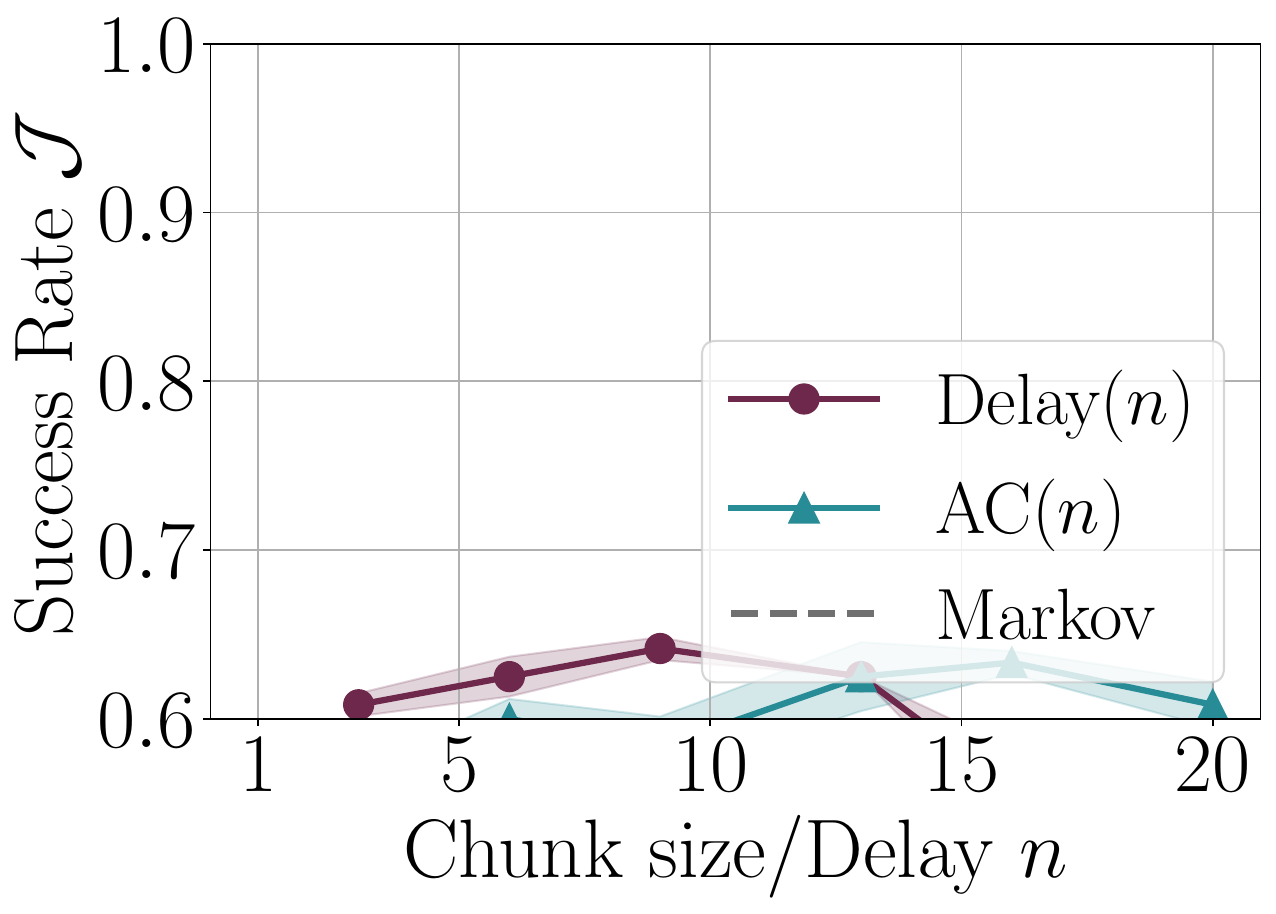}
    \end{minipage}

        \caption{Success rate for each \texttt{Libero} task from 68 to 89 (corresponding to Fig. \ref{fig:success_libero}), part 3.}
    \label{fig:succ each libero3}
\end{figure*}

\begin{figure*}[t]
    \centering
        \begin{minipage}[t]{0.23\textwidth}
            \includegraphics[width=\linewidth]{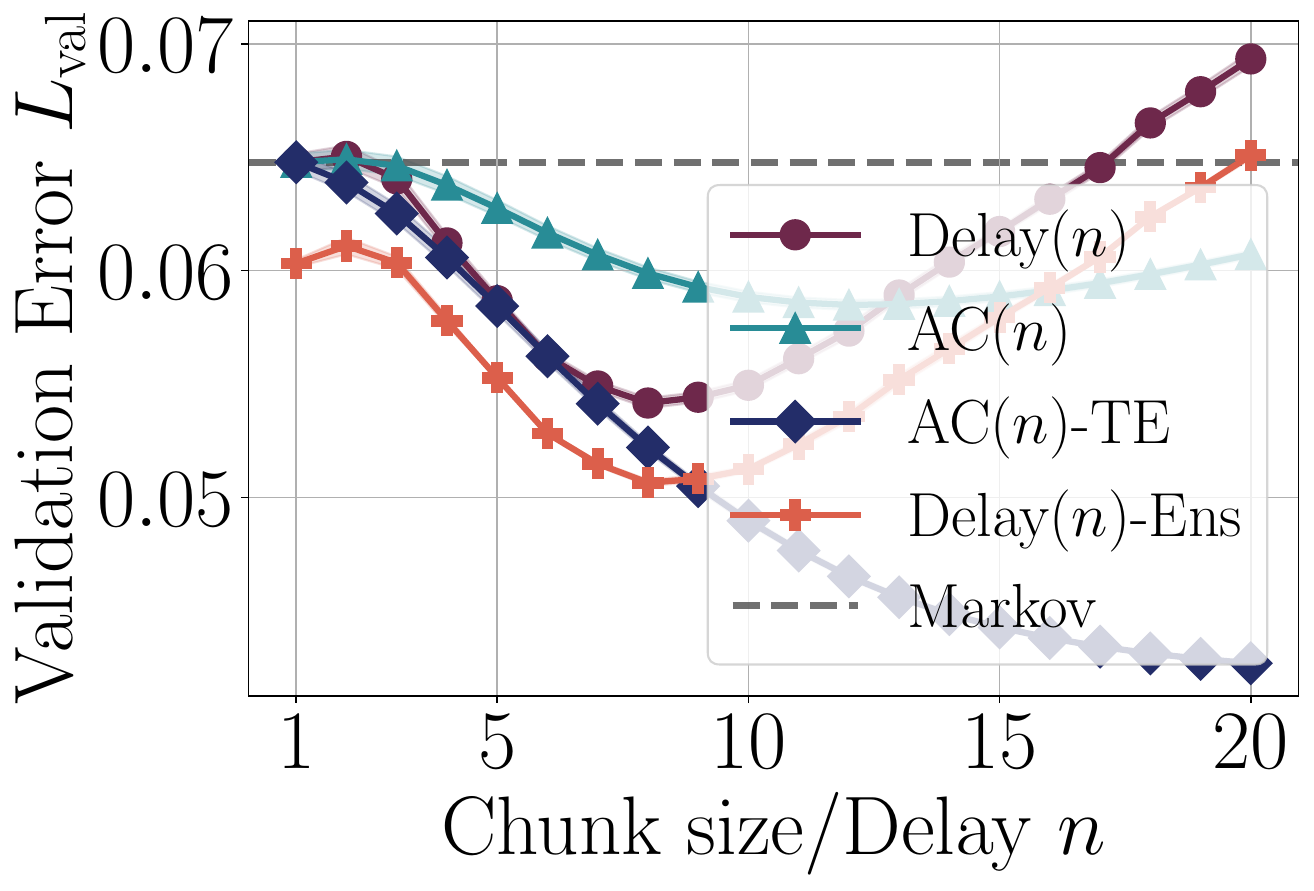}
    \end{minipage}
    \hfill
        \begin{minipage}[t]{0.23\textwidth}
        \centering
        \includegraphics[width=\linewidth]{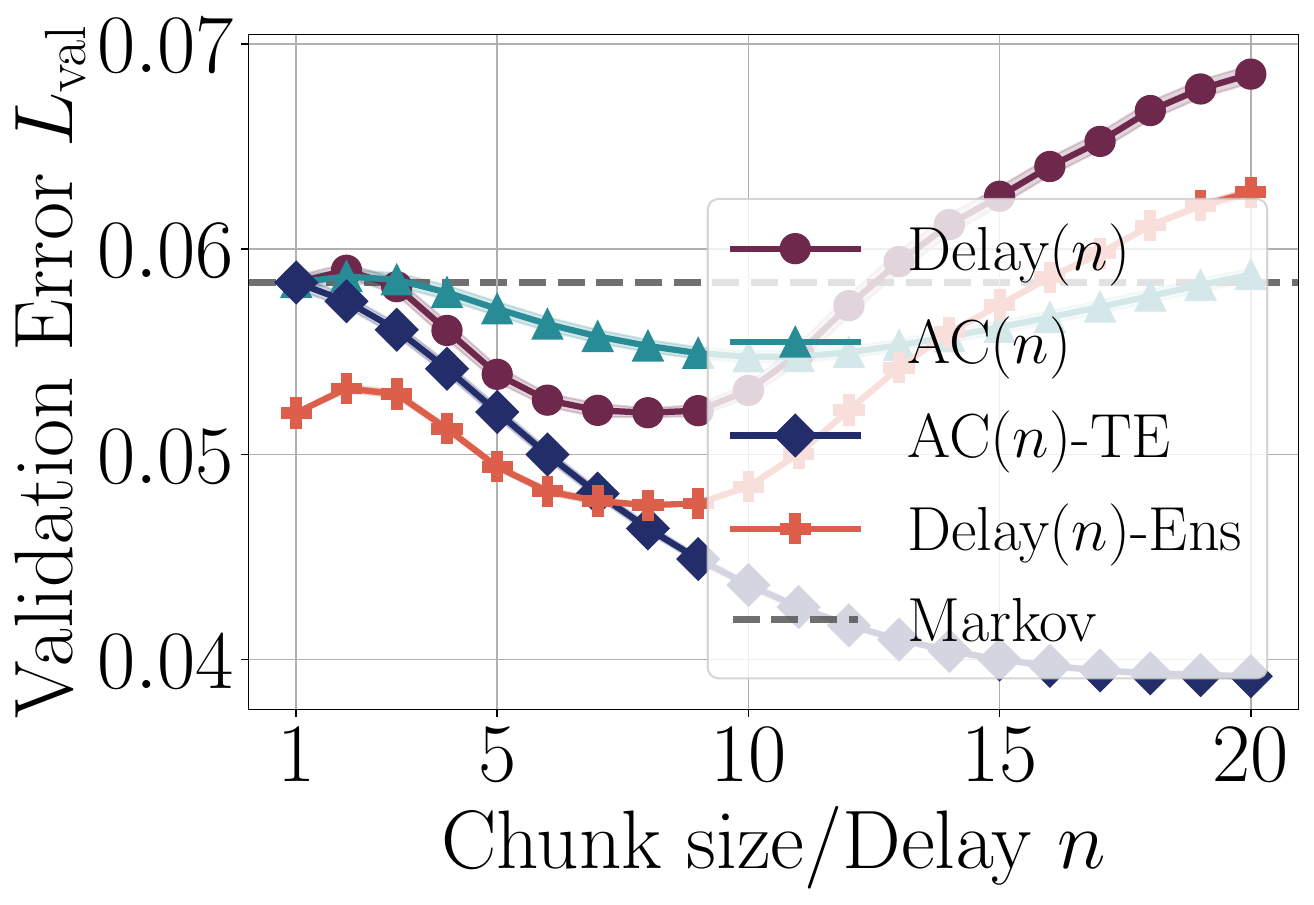}
    \end{minipage}
    \hfill
        \begin{minipage}[t]{0.23\textwidth}
        \centering
        \includegraphics[width=\linewidth]{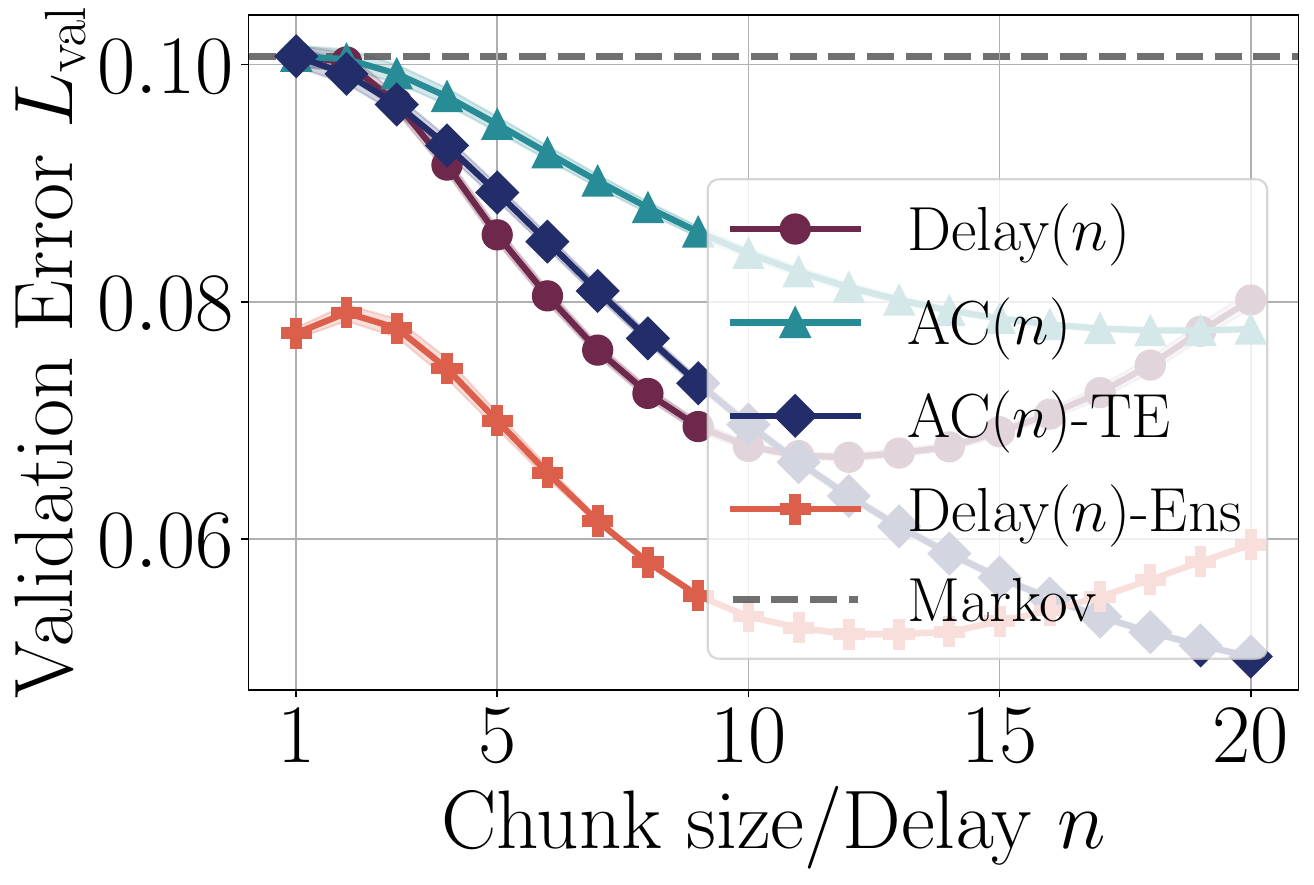}
    \end{minipage}
    \hfill
        \begin{minipage}[t]{0.23\textwidth}
        \centering
        \includegraphics[width=\linewidth]{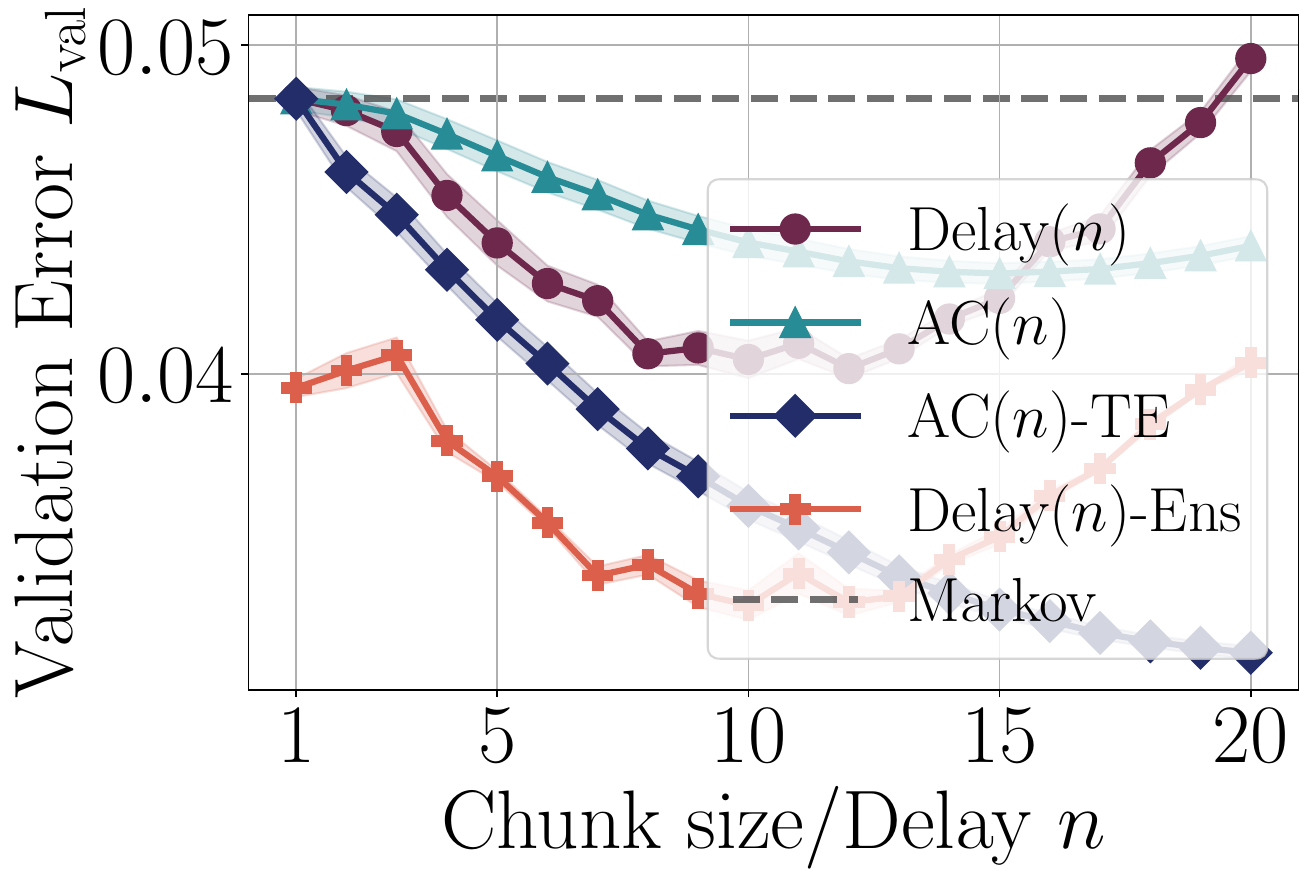}
    \end{minipage}
        \caption{Val loss for each \texttt{Robomimic} task in this order from left to right: can, square, transport, tool hang (corresponding to Fig. \ref{fig:robomimic_val}), part 1.}
    \label{fig:val each robomimic}
\end{figure*}

\begin{figure*}[t]
    \centering
        \begin{minipage}[t]{0.23\textwidth}
            \includegraphics[width=\linewidth]{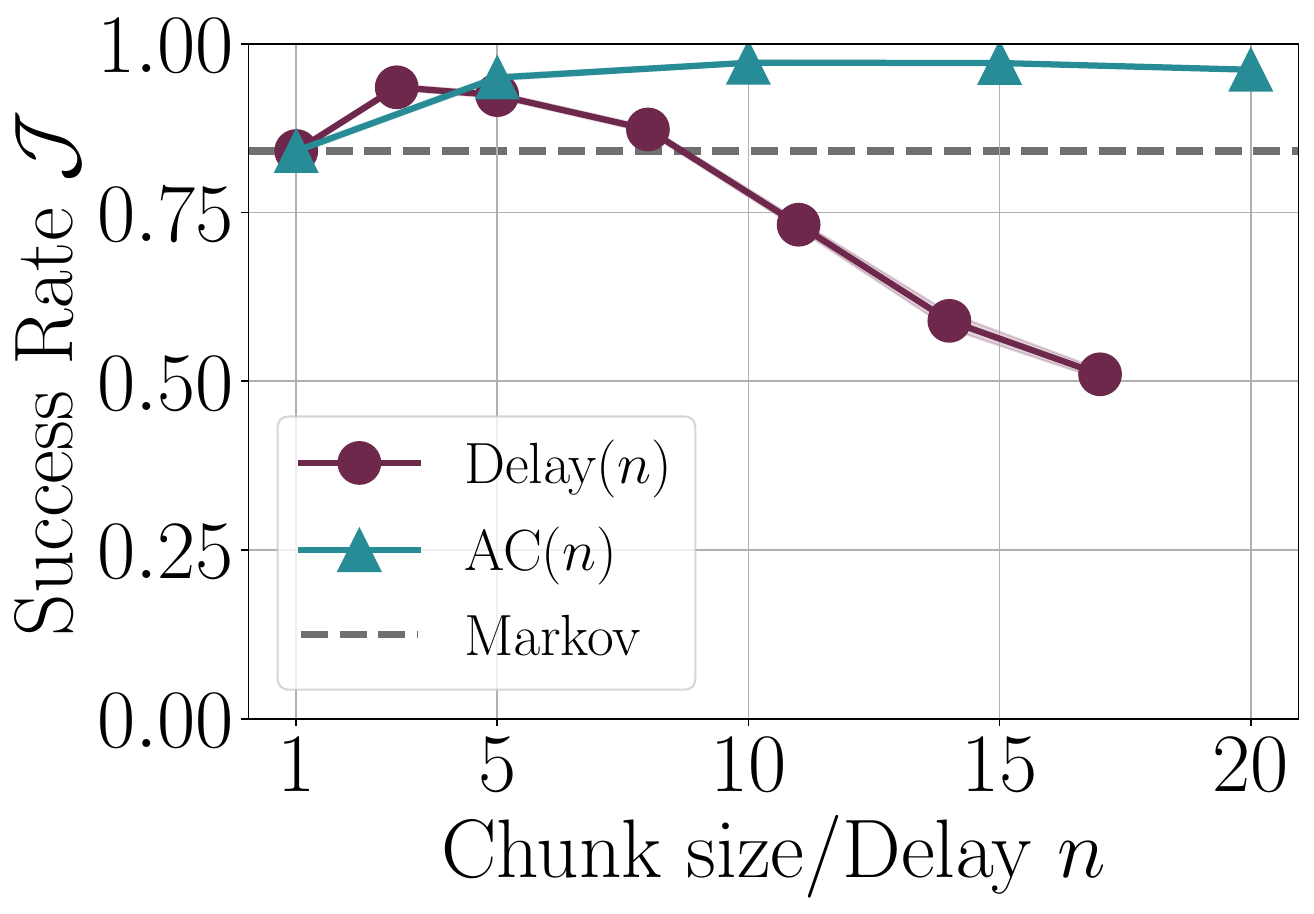}
    \end{minipage}
    \hfill
        \begin{minipage}[t]{0.23\textwidth}
        \centering
        \includegraphics[width=\linewidth]{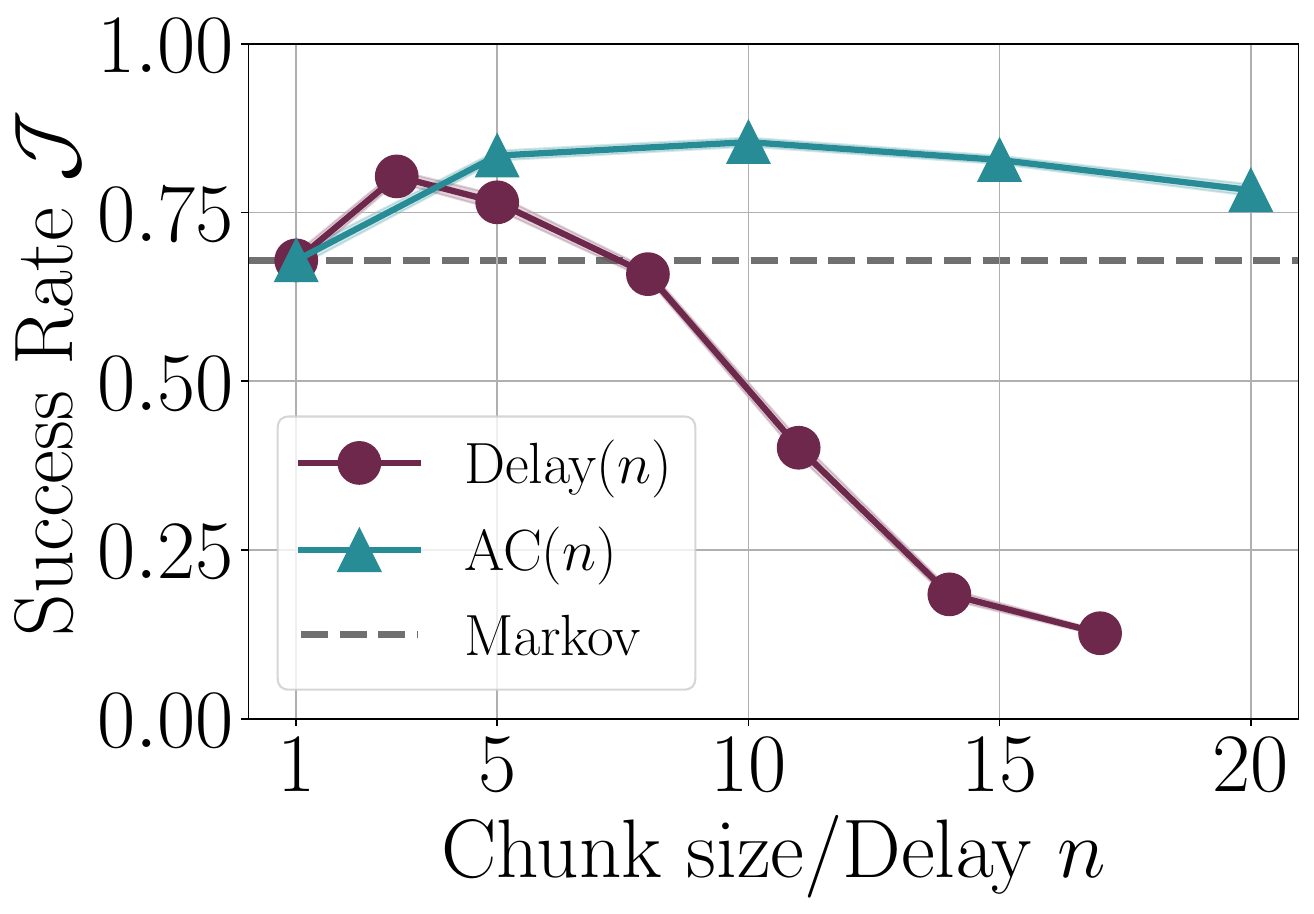}
    \end{minipage}
    \hfill
        \begin{minipage}[t]{0.23\textwidth}
        \centering
        \includegraphics[width=\linewidth]{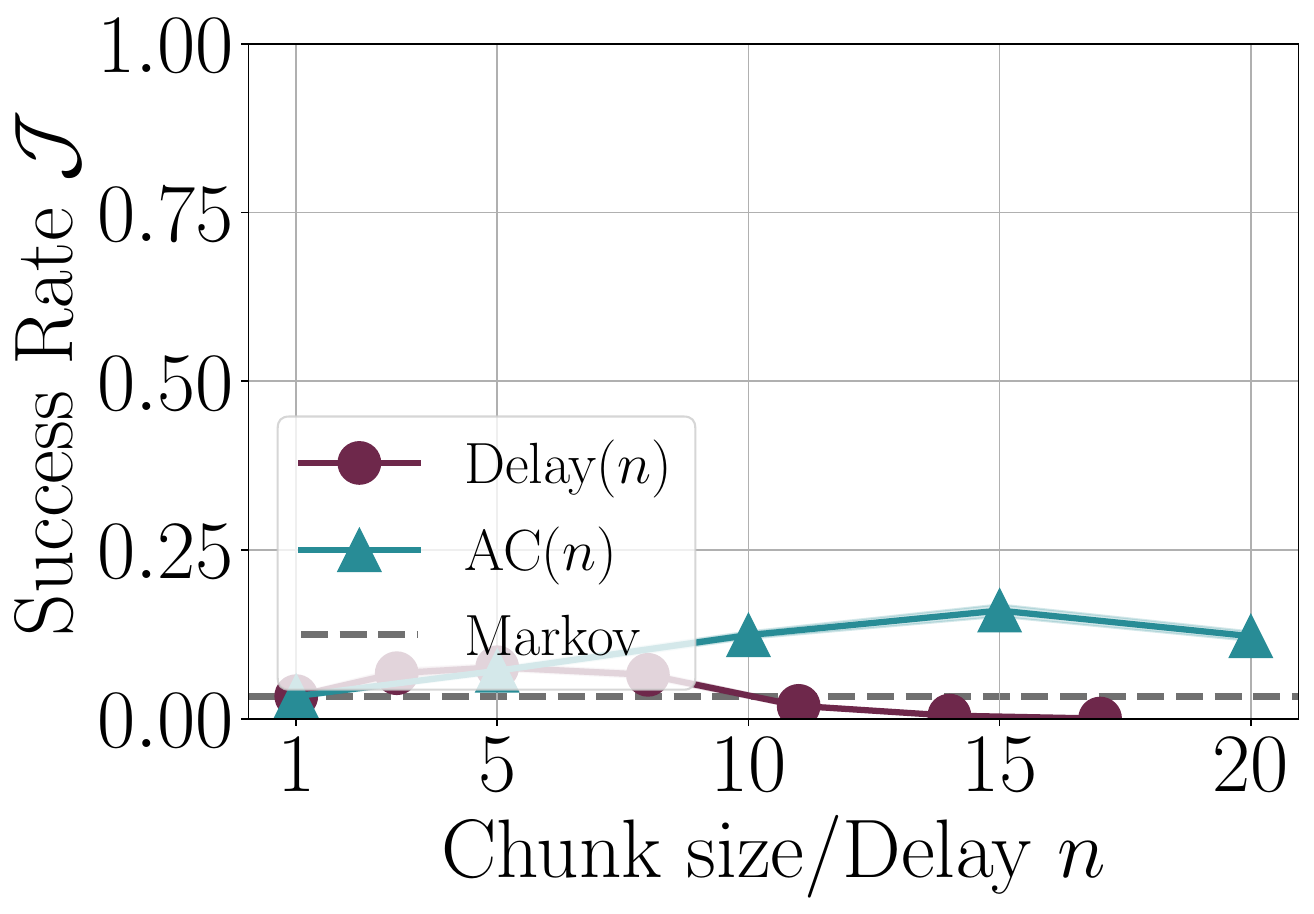}
    \end{minipage}
    \hfill
        \begin{minipage}[t]{0.23\textwidth}
        \centering
        \includegraphics[width=\linewidth]{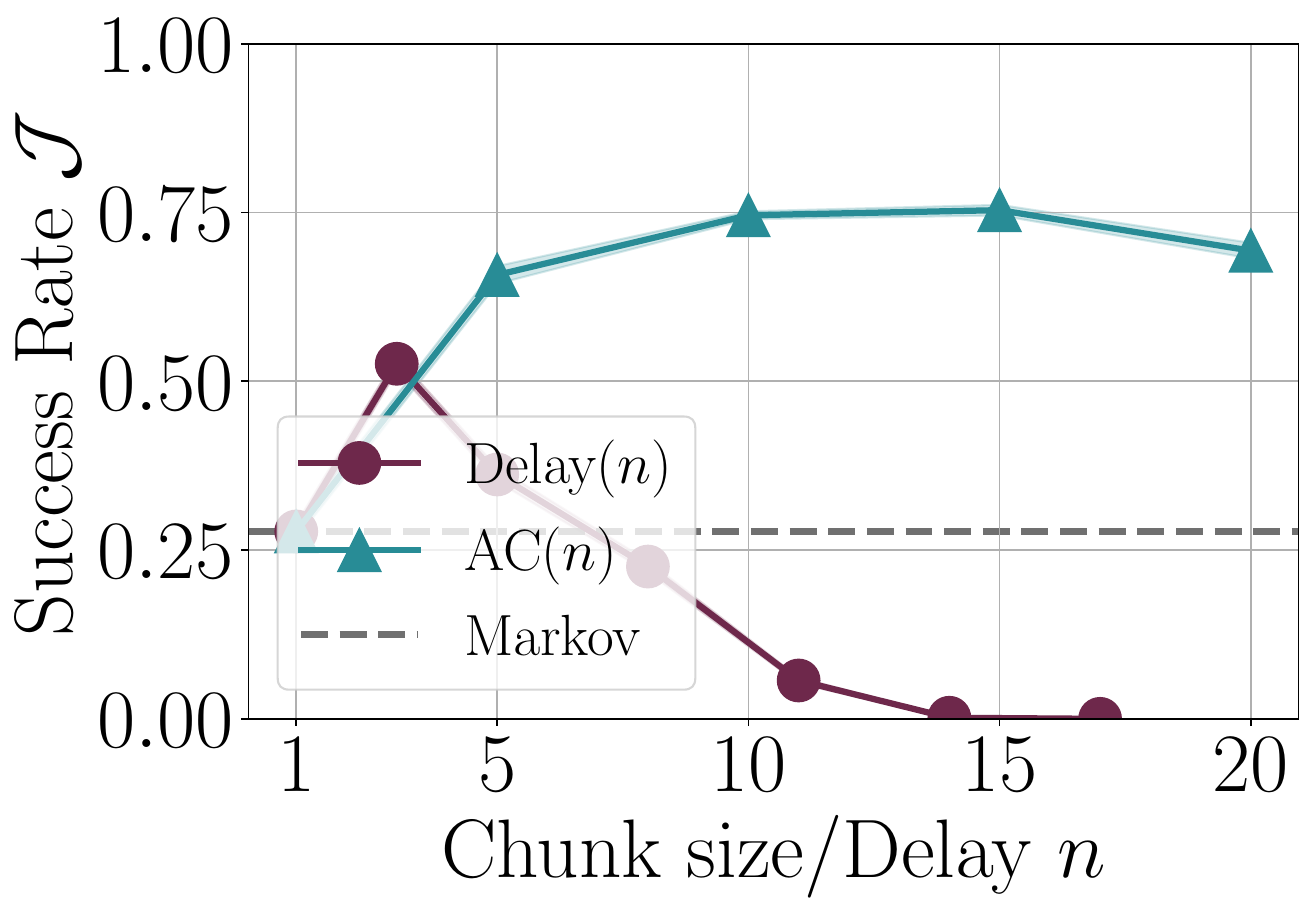}
    \end{minipage}
        \caption{Success rate for each \texttt{Robomimic} task in this order from left to right: can, square, transport, tool hang (corresponding to Fig. \ref{fig:success_robomimic}), part 1.}
    \label{fig:succ each robomimic}
\end{figure*}

\end{document}